%% file: main.tex
\documentclass{memoir}

\aliaspagestyle{title}{empty}
\setstocksize{239mm}{169mm}
\settrimmedsize{239mm}{169mm}{*}
\setlrmarginsandblock{2.5cm}{2cm}{*}
\setulmarginsandblock{2.5cm}{2.5cm}{*}
\checkandfixthelayout 
\chapterstyle{verville}
\input{packages}
\input{macros}

\setsecnumdepth{subsection}

\begin{document}

\title{\textsc{\huge{An Information-Theoretic Approach to Generalization Theory}\\}}
\author{\\Borja Rodríguez-Gálvez, Ragnar Thobaben, and Mikael Skoglund \\
KTH Royal Institute of Technology \\
\texttt{borjabrg12@gmail.com}, \texttt{\{ragnart, skoglund\}@kth.se} \vspace{30em}}
\date{$^*$Based on the \href{https://kth.diva-portal.org/smash/record.jsf?dswid=901&pid=diva2:1848094}{PhD dissertation manuscript} of Borja Rodríguez-Gálvez. The text is taken almost verbose from there.}

\maketitle

\frontmatter %
\begin{abstract}
\input{chapters/abstract}
\end{abstract}

\mainmatter %
\renewcommand{\thefigure}{\thechapter.\arabic{figure}}
\tableofcontents

\chapter{Introduction}
\label{ch:introduction}
\input{chapters/introduction}

\chapter{Preliminaries}
\label{ch:preliminaries}
\input{chapters/preliminaries}

\chapter{A Primer on Generalization Theory}
\label{ch:generalization_theory}
\input{chapters/primer_generalization_theory}

\chapter{Generalization Guarantees in Expectation}
\label{ch:expected_generalization_error}
\input{chapters/generalization_expectation}

\chapter{PAC-Bayesian Generalization Guarantees}
\label{ch:pac_bayesian_generalization}
\renewcommand{\thefigure}{\thechapter.\arabic{figure}}
\input{chapters/pac_bayes}

\chapter{Privacy and Generalization}
\label{ch:privacy_and_generalization}
\renewcommand{\thefigure}{\thechapter.\arabic{figure}}
\input{chapters/privacy_and_generalization}

\chapter{Discussion and Future Work}
\label{ch:discussion_and_future_work}
\input{chapters/discussion}

\chapter*{Acknowledgements}
\input{chapters/acknowledgements}

\bibliographystyle{IEEEtranN}
\renewcommand{\bibname}{References}%
\bibliography{Refs} %

\end{document}

%% file: packages.tex
\usepackage[utf8]{inputenc}
\usepackage[swedish,english]{babel}
\usepackage{enumerate} 

\usepackage{amsmath, amsthm, amssymb, amsfonts, mathrsfs} 
\usepackage[cal=cm,
	scr=boondoxo,
	bb=txof]{mathalfa}
\usepackage{nicefrac}
\usepackage{mathtools}
\usepackage{bm, bbm}
\usepackage{thm-restate}
\usepackage{mleftright}

\usepackage[
    dvipsnames
]{xcolor} 
 \usepackage[
     colorlinks=true,
     citecolor=MidnightBlue,
     linkcolor=MidnightBlue,
     filecolor=MidnightBlue,      
     urlcolor=MidnightBlue
 ]{hyperref}  
 \usepackage[nameinlink]{cleveref}
\usepackage{letltxmacro,xparse}

\usepackage{tcolorbox}

\usepackage{caption} 
\usepackage{subcaption}  
\usepackage{lscape}
\usepackage{epstopdf}
\usepackage[
    square, 
    numbers, 
    compress,
    sort
]{natbib} 
\usepackage{graphicx}
\usepackage{bm}	%
\usepackage{algorithm}
\usepackage{algpseudocode}
\usepackage{algpascal}
\usepackage{multicol}
\usepackage{lipsum}

\usepackage{appendix}
\AtBeginEnvironment{subappendices}{%
\chapter*{Appendices}
\addcontentsline{toc}{chapter}{Appendices}
\counterwithin{figure}{section}
\counterwithin{table}{section}
}

\usepackage{tikz}
\usetikzlibrary{tikzmark,calc,decorations.pathreplacing}

%% file: macros.tex
\newcommand{\bA}{\ensuremath{\mathbb{A}}}

\newcommand{\bE}{\ensuremath{\mathbb{E}}}

\newcommand{\bI}{\ensuremath{\mathbb{I}}}

\newcommand{\bM}{\ensuremath{\mathbb{M}}}
\newcommand{\bN}{\ensuremath{\mathbb{N}}}

\newcommand{\bP}{\ensuremath{\mathbb{P}}}
\newcommand{\bQ}{\ensuremath{\mathbb{Q}}}
\newcommand{\bR}{\ensuremath{\mathbb{R}}}

\newcommand{\bW}{\ensuremath{\mathbb{W}}}

\newcommand{\cA}{\ensuremath{\mathcal{A}}}
\newcommand{\cB}{\ensuremath{\mathcal{B}}}
\newcommand{\cC}{\ensuremath{\mathcal{C}}}

\newcommand{\cE}{\ensuremath{\mathcal{E}}}
\newcommand{\cF}{\ensuremath{\mathcal{F}}}
\newcommand{\cG}{\ensuremath{\mathcal{G}}}
\newcommand{\cH}{\ensuremath{\mathcal{H}}}

\newcommand{\cJ}{\ensuremath{\mathcal{J}}}
\newcommand{\cK}{\ensuremath{\mathcal{K}}}
\newcommand{\cL}{\ensuremath{\mathcal{L}}}

\newcommand{\cN}{\ensuremath{\mathcal{N}}}
\newcommand{\cO}{\ensuremath{\mathcal{O}}}
\newcommand{\cP}{\ensuremath{\mathcal{P}}}

\newcommand{\cR}{\ensuremath{\mathcal{R}}}
\newcommand{\cS}{\ensuremath{\mathcal{S}}}
\newcommand{\cT}{\ensuremath{\mathcal{T}}}
\newcommand{\cU}{\ensuremath{\mathcal{U}}}

\newcommand{\cW}{\ensuremath{\mathcal{W}}}
\newcommand{\cX}{\ensuremath{\mathcal{X}}}
\newcommand{\cY}{\ensuremath{\mathcal{Y}}}
\newcommand{\cZ}{\ensuremath{\mathcal{Z}}}

\newcommand{\ccA}{\ensuremath{\mathscr{A}}}
\newcommand{\ccB}{\ensuremath{\mathscr{B}}}

\newcommand{\ccF}{\ensuremath{\mathscr{F}}}

\newcommand{\ccR}{\ensuremath{\mathscr{R}}}

\newcommand{\ccX}{\ensuremath{\mathscr{X}}}
\newcommand{\ccY}{\ensuremath{\mathscr{Y}}}

\newcommand{\sff}{\ensuremath{\mathsf{f}}}

\newcommand{\sfp}{\ensuremath{\mathsf{p}}}
\newcommand{\sfq}{\ensuremath{\mathsf{q}}}

\newcommand{\rmb}{\ensuremath{\mathrm{b}}}

\newcommand{\rmd}{\ensuremath{\mathrm{d}}}

\newcommand{\rmC}{\ensuremath{\mathrm{C}}}
\newcommand{\rmD}{\ensuremath{\mathrm{D}}}

\newcommand{\rmH}{\ensuremath{\mathrm{H}}}
\newcommand{\rmI}{\ensuremath{\mathrm{I}}}

\newcommand{\ent}{\rmH}
\newcommand{\entber}{\rmH_\rmb}
\newcommand{\relent}{\ensuremath{\rmD_{\mathrm{KL}}}}
\newcommand{\relentber}{\ensuremath{\rmd_{\mathrm{KL}}}}
\newcommand{\renyidiv}[1]{\ensuremath{\rmD}_{#1}}
\newcommand{\minf}{\rmI}
\newcommand{\fdiv}{\ensuremath{\rmD_{f}}}
\newcommand{\tv}{\ensuremath{\mathrm{TV}}}
\newcommand{\hel}{\ensuremath{\mathrm{Hel}}}

\newcommand{\var}{\ensuremath{\mathrm{Var}}}

\newcommand{\supp}{\ensuremath{\mathrm{supp}}}

\newcommand{\poprisk}{\ensuremath{\ccR}}
\newcommand{\emprisk}{\ensuremath{\widehat{\ccR}}}
\newcommand{\gen}{\ensuremath{\mathrm{gen}}}
\newcommand{\empgen}{\ensuremath{\widehat{\mathrm{gen}}}}
\newcommand{\excess}{\ensuremath{\mathrm{excess}}}
\newcommand{\growth}{\ensuremath{\Pi}}
\newcommand{\vcdim}{\ensuremath{\mathrm{VCdim}}}
\newcommand{\ndim}{\ensuremath{\mathrm{Ndim}}}
\newcommand{\rad}{\ensuremath{\mathrm{Rad}}}

\newcommand{\desc}{\ensuremath{\mathrm{desc}}}
\newcommand{\comp}{\mathfrak{C}_{n,\beta,S}}

\newcommand{\SCOprob}{(\cW,\cZ,\ell)}
\newcommand{\clb}{\cC_{L,B}}
\newcommand{\coorvec}[1]{\mathsf{e}({#1})}

\newcommand{\diam}{\ensuremath{\mathrm{diam}}}
\newcommand{\sign}{\ensuremath{\mathrm{sign}}}

\newtheorem{definition}{Definition}[chapter]
\newtheorem{theorem}{Theorem}[chapter]
\newtheorem{lemma}{Lemma}[chapter]
\newtheorem{proposition}{Proposition}[chapter]
\newtheorem{remark}{Remark}[chapter]
\newtheorem{corollary}{Corollary}[chapter]
\newtheorem{example}{Example}[chapter]

\definecolor{kth-blue}{RGB}{25,84,166}
\definecolor{kth-pink}{RGB}{216,84,151}
\definecolor{kth-gray}{RGB}{101,101,108}

\crefdefaultlabelformat{#2\textbf{#1}#3}
\Crefname{definition}{\protect\textbf{Definition}}{\protect\textbf{Definitions}}
\Crefname{theorem}{\protect\textbf{Theorem}}{\protect\textbf{Theorems}}
\Crefname{lemma}{\protect\textbf{Lemma}}{\protect\textbf{Lemmata}}
\Crefname{proposition}{\protect\textbf{Proposition}}{\protect\textbf{Propositions}}
\Crefname{remark}{\protect\textbf{Remark}}{\protect\textbf{Remarks}}
\Crefname{example}{\protect\textbf{Example}}{\protect\textbf{Examples}}
\Crefname{corollary}{\protect\textbf{Corollary}}{\protect\textbf{Corollaries}}
\Crefname{figure}{\protect\textbf{Figure}}{\protect\textbf{Figures}}
\Crefname{section}{\protect\textbf{Section}}{\protect\textbf{Sections}}
\Crefname{chapter}{\protect\textbf{Chapter}}{\protect\textbf{Chapters}}
\Crefname{appendix}{\protect\textbf{Appendix}}{\protect\textbf{Appendices}}
\Crefname{objective}{\protect\textbf{Objective}}{\protect\textbf{Objectives}}

\DeclareMathOperator*{\esssup}{ess\,sup}

\DeclareMathOperator*{\argmin}{arg\,min}

\newcommand{\stack}[2]{\stackrel{\mathclap{(#1)}}{#2}}

\makeatletter
\DeclareRobustCommand\ltseries
{\not@math@alphabet\ltseries\relax
\fontseries\ltdefault\selectfont}
\newcommand{\ltdefault}{l}
\DeclareTextFontCommand{\textlt}{\ltseries}
\DeclareRobustCommand\hbseries
{\not@math@alphabet\hbseries\relax
\fontseries\hbdefault\selectfont}
\newcommand{\hbdefault}{hb}
\DeclareTextFontCommand{\texthb}{\hbseries}
\makeatother

%% file: chapters/abstract.tex
In this manuscript, we investigate the \emph{in-distribution generalization} of machine learning algorithms, focusing on establishing \emph{rigorous upper bounds on the generalization error}. We depart from traditional complexity-based approaches by introducing and analyzing \emph{information-theoretic bounds} that quantify the \emph{dependence between a learning algorithm and the training data}.
We consider two categories of generalization guarantees:
\begin{itemize}
    \item \emph{Guarantees in expectation}. These bounds measure performance in the average case. Here, the dependence between the algorithm and the data is often captured by the \emph{mutual information} or other information measures based on \emph{$f$-divergences}. While these measures offer an intuitive interpretation, they might overlook the \emph{geometry of the algorithm's hypothesis class}. To address this limitation, we introduce bounds using the \emph{Wasserstein distance}, which incorporates geometric considerations at the cost of being mathematically more involved. Furthermore, we propose a \emph{structured, systematic method} to derive bounds capturing the dependence between the algorithm and an individual datum, and between the algorithm and subsets of the training data, conditioned on knowing the rest of the data. These types of bounds provide deeper insights, as we demonstrate by applying them to derive generalization error bounds for the stochastic gradient Langevin dynamics algorithm.
    
    \item \emph{PAC-Bayesian guarantees}. These bounds measure the performance level with high probability. Here, the dependence between the algorithm and the data is often measured by the \emph{relative entropy}. We establish connections between the Seeger--Langford and Catoni's bounds, revealing that the former is optimized by the \emph{Gibbs posterior}. Additionally, we introduce \emph{novel, tighter bounds} for various types of loss functions, including those with a bounded range, cumulant generating function, moment, or variance. To achieve this, we introduce a new technique to optimize parameters in probabilistic statements.

\end{itemize}

\emph{We also study the limitations of these approaches}. We present a counter-example where most of the existing (relative entropy-based) information-theoretic bounds fail and where traditional approaches do not. 
Finally, we \emph{explore the relationship between privacy and generalization}. We show that algorithms with a bounded maximal leakage generalize. Moreover, for discrete data, we derive new bounds for differentially private algorithms that vanish as the number of samples increases, thus guaranteeing their generalization even with a constant privacy parameter. This contrasts with previous bounds in the literature, which require the privacy parameter to decrease with the number of samples to ensure generalization.

\vfill

\noindent \textbf {Keywords:} Generalization, Information-Theoretic Bounds

%% file: chapters/introduction.tex
Currently, \emph{machine learning} silently shapes the world around us: from the curated content we see on social media to high-stakes industries such as healthcare, finance, or autonomous driving~\citep{jiang2017artificial,yurtsever2020survey,badue2021self,bolton2002statistical}. Machine learning algorithms are used to analyze data for early disease detection, fraud prevention, or steering the wheels, breaking, and speeding up autonomous vehicles. 
Therefore, it is crucial to have a robust theoretical understanding of \emph{when} and \emph{why} these algorithms will perform appropriately in critical situations.

\section{Background}

A machine learning algorithm (or just a \emph{learning algorithm}) is a mechanism that observes a sequence of instances of a problem (often referred to as a \emph{training set}) and returns a hypothesis for the said problem. For example, given a training set of medical records containing two biological markers and a label stating if a person has a certain disease or not, a learning algorithm returns a model that predicts the presence of that disease from a medical record. In this case, the hypothesis is the mapping between medical records and the presence of the disease characterized by the predictive model, as depicted in~\Cref{fig:medical_example}.

The hypothesis returned by the algorithm may perform really well in the training set. However, this does not necessarily mean that it will perform well on new, unseen data. Consider, for example, the complex hypothesis from~\Cref{fig:medical_example}; while it can perfectly explain the relationship between the medical markers and the presence of disease for all instances in the training data, it fails when it comes to describing unseen data from the same distribution.
The field of \emph{generalization theory} focuses on finding mathematical guarantees on the performance of the hypotheses returned by learning algorithms applied to new, unseen data.

\begin{figure}[t]
    \centering
    \includegraphics[width=0.9\textwidth]{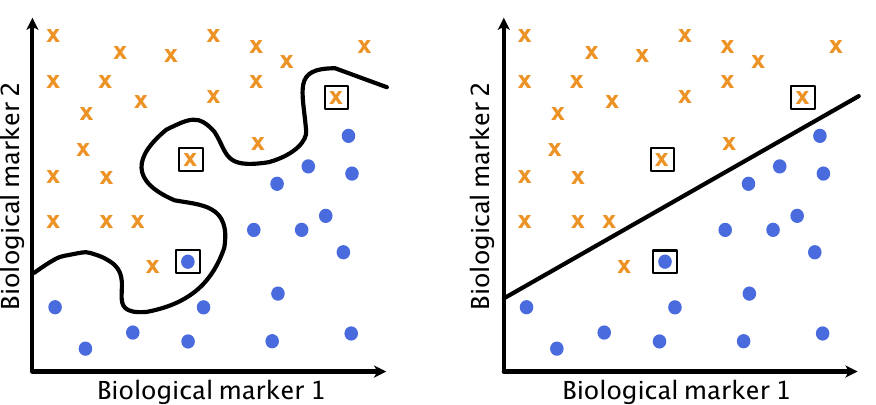}
    \caption{Fictitious relationship between two biological markers and the presence (orange crosses) or absence (blue dots) of a disease. The marks without a box represent the training set available to the algorithm, and the ones with a box represent new, unseen data. On the left, there is a complex hypothesis that overfits the training set. On the right, there is a simple hypothesis that performs slightly worse in the training set but generalizes to new instances.}
    \label{fig:medical_example}
\end{figure}

\looseness=-1 Often, we describe the performance of an algorithm through a \emph{loss function}. In the example above, the loss function returns 0 if the algorithm correctly determines if a person has a disease from their medical record, and 1 otherwise. Therefore, smaller values of the loss function mean a better performance. The aforementioned \emph{generalization guarantees} are often in the form of an upper bound on the expected value of the loss function on samples from the problem distribution, also known as the \emph{population risk}. Alternatively, the guarantees come as an upper bound on the \emph{generalization error}, which is defined as the difference between the population risk and the average loss function on the instances from the training set, or \emph{empirical risk}.

When studying the generalization of learning algorithms, we may consider either that the new data come from the same distribution (\emph{in-distribution generalization}) or from a close, albeit different distribution (\emph{out-of-distribution generalization}). For example, the medical records from patients from the same demographic collected in the same hospital following the same protocol may be considered from the same distribution. However, the medical records from patients from a different demographic are often considered from a different distribution. As an example, the blood oxygen levels of people living in high-altitude cities are lower than those of people living at sea level. This change in distribution is often known as a~\emph{distribution shift}. In this monograph, and thus throughout the forthcoming text, we will focus only on in-distribution generalization.

There are two main types of generalization guarantees:
\begin{itemize}
    \item \emph{Guarantees in expectation.} These guarantees are the less specific of the two. They ensure that the average performance of the hypotheses returned by the algorithm is never smaller than a certain amount.
    \item \emph{High-probability guarantees.} These guarantees ensure that, with high probability, the performance of the hypotheses returned by the algorithm is never smaller than a certain amount. In the context of information-theoretic generalization, this kind of guarantees are referred to as \emph{PAC-Bayesian guarantees}.
\end{itemize}

Classical approaches to studying the generalization of learning algorithms focus on the complexity of the class of hypotheses considered by the algorithm. This refers to the richness of the possible hypotheses the algorithm can select. The idea behind these approaches is that hypotheses from more complex classes can fit the training set really well, but they may potentially \emph{overfit}, that is, they may adapt to the sampling noise for the specific data they observe, which may differ from future instances. On the other hand, hypotheses from simpler classes cannot adapt to the sampling noise and need to focus on the high signal patterns in the data, representative of the data distribution. The downside of the more simple classes is that their lack of richness may mean that none of their hypotheses can describe well the patterns in the training set. However, they still have a small \emph{generalization error}; their bad performance in the training set is predictive of a bad performance on new, unseen data. See \Cref{fig:medical_example} for an example of this dichotomy.

An important feature of classical approaches like those based on \emph{uniform convergence} or on the \emph{Rademacher complexity} is that they provide generalization guarantees for the \emph{whole hypothesis class} considered by the algorithm. This can be too restrictive, as within a complex hypothesis class, there may still be hypotheses that generalize well. A step away from this restriction is given by the \emph{algorithmic stability framework}, which gives generalization guarantees for \emph{stable algorithms}, that is, algorithms that return similar hypotheses when presented with a similar training set. In this way, algorithmic stability can guarantee the generalization of algorithms that consider complex hypotheses, as long as they are consistent in returning these hypotheses for similar training sets. Often, the concept of ``similar training sets'' only takes into account every possible value of the observed data, however unlikely, rendering the framework too restrictive for certain situations. Finally, the framework of \emph{information-theoretic generalization} gives guarantees that are \emph{specific} to both the algorithm and the data distribution. The idea is to use an \emph{information measure} to capture the dependence between the hypothesis returned by the algorithm and the training set. Intuitively, the more the hypothesis depends on the specific instances it has seen, the more it overfits and the less it generalizes. This individualized approach, while more nuanced, allows for stronger theoretical results in situations where the previous approaches fail.

The information-theoretic framework to generalization elucidates a \emph{connection between privacy and generalization}. An algorithm is private if the hypothesis it returns does not leak or rely on much information about the instances used for training. Therefore, the more private an algorithm is, the smaller the dependence between the hypothesis returned by the algorithm and training set, and the better it generalizes.

\section{Overview of the Monograph}

In this monograph, we study the information-theoretic framework for generalization. We introduce the topic in this chapter, which is followed by two other chapters to prepare the reader for the later content. 

In \Cref{ch:preliminaries}, we detail the notation and definitions necessary to follow the monograph. Particularly, we describe the elements of probability and information theory that we use throughout the monograph. Moreover, we discuss the concepts of convex conjugacy and Wasserstein distances and present a straightforward extension of the Bobkov–Götze theorem connecting the Wasserstein distance to the relative entropy that will be of use later on.

In \Cref{ch:generalization_theory}, we introduce the notation and definitions related to generalization theory and go into greater detail about the classical frameworks to study it. We use this chapter to further motivate and contextualize the information-theoretic framework to generalization with respect to the other alternatives.

Once the reader has the proper context of the elements of information theory, probability, and generalization theory, we commence with our investigation. Our study may be separated into three main blocks: the study of guarantees in expectation, the study of PAC-Bayesian guarantees, and the connections between privacy and generalization. Each of these topics has a dedicated chapter to them, namely~\Cref{ch:expected_generalization_error,ch:pac_bayesian_generalization,ch:privacy_and_generalization}. Nonetheless, the analysis and findings of each block are interconnected. Below, in~\Cref{subsec:introduction_guarantees_expectation,subsec:introduction_pac_bayes,subsec:introduction_privacy_generalization}, we outline the scope and contents of these chapters.

Finally, in~\Cref{ch:discussion_and_future_work}, we reflect on the state of the art and the results from the monograph and we discuss potential future avenues of research.

\subsection{Guarantees in Expectation}
\label{subsec:introduction_guarantees_expectation}

\sloppy \citet{xu2017information}, based on~\citep{russo2016controlling}, sparked the interest in information-theoretic generalization guarantees in expectation. Their result captured the dependence between the hypothesis returned by the algorithm and the training set using mutual information. After that, \citet{bu2020tightening} showed that considering the information that the returned hypothesis contains about each individual datum provides a better characterization of the expected generalization performance. Similarly, \citet{negrea2019information} realized that there are situations where considering the information contained in a subset of the training data, given that the rest of the data is known, is beneficial to characterize the generalization of certain algorithms. Later, \citet{steinke2020reasoning} discussed how instead of studying the information the hypothesis contains about the training data, studying if the training data can be \emph{identified} from the hypothesis was more appropriate. 

Throughout \Cref{ch:expected_generalization_error}, especially in~\Cref{sec:bounds_using_mutual_information,sec:bounds_using_conditional_mutual_information,sec:random_subset_and_single_letter,sec:bounds_using_wasserstein_distance}, we organize all these different observations into a structured, systematic methodology to derive information-theoretic guarantees in expectation. Furthermore, this methodology is agnostic to the measure considered to characterize the dependence between the algorithm's returned hypothesis and the training set. 

Despite the fact that the methodology can be made agnostic to the metric used to capture the dependence between the hypothesis and the training data, traditionally, the relative entropy and the mutual information are often chosen for this endeavor. This is due to their simplicity, interpretability, and mathematical tractability. 

In~\Cref{sec:noisy_iterative_learning_algos,sec:bounds_using_wasserstein_distance,sec:limitations_bounds_using_mutual_information}, we delve deeper into the choice of this metric. First, in \Cref{sec:bounds_using_wasserstein_distance}, we show that the relative entropy-based guarantees are tight in non-trivial scenarios. Then, in~\Cref{sec:noisy_iterative_learning_algos}, we demonstrate how the relative entropy-based guarantees can be employed to gain a better understanding of the generalization of noisy, iterative algorithms such as the stochastic gradient Langevin dynamics. Second, in~\Cref{sec:limitations_bounds_using_mutual_information}, we show that there are situations where the relative entropy-based guarantees cannot predict the generalization of an algorithm, while other frameworks like uniform stability can.

Usually, to capture the hypothesis' dependence on the training data, we measure the discrepancy between the algorithm's output hypothesis distribution and a reference distribution that does not depend on the data. As mentioned above, this discrepancy is often determined by the relative entropy, albeit there exist other works that consider other $f$-divergences or the Rényi divergence~\citep{esposito2021generalization,ohnishi2021novel,hellstrom2020generalization}. Unfortunately, all these divergences lead to vacuous guarantees if the reference distribution places no density on a hypothesis that may be returned by the algorithm, even if it very rarely returns it. The reason for this failure is that these divergences are agnostic to the geometry, and their value solely depends on the ratio of densities of the algorithm's output and the reference distributions.

In~\Cref{sec:bounds_using_wasserstein_distance}, we study how to capture the algorithm's returned hypothesis dependence on the training set using the Wasserstein distance. This metric takes the geometry of the hypothesis' space into consideration and does not depend on the ratio of the densities of the algorithm's output and the reference distributions.

Throughout~\Cref{ch:expected_generalization_error}, we obtain other findings and insights. For example, the development of new generalization guarantees when the performance metric considered has heavy tails.

\subsection{PAC-Bayesian Guarantees}
\label{subsec:introduction_pac_bayes}

The study of information-theoretic high-probability generalization guarantees can be traced back to~\citet{shawe1996framework} or, even earlier, to the minimum description length principle from~\citet{rissanen1978modeling} and its relationship with generalization from, for example, Barron, Cover, Kearns, and others~\citep{barron1991complexity,barron1991minimum,kearns1995experimental}. Nonetheless, the field started to gain traction with the seminal works from~\citet{mcallester1998some,mcallester1999pac,mcallester2003pac}, where he introduced the nomenclature of PAC-Bayesian generalization.

\looseness=-1 Similarly to what we mentioned above in~\Cref{subsec:introduction_guarantees_expectation}, the main result from \citet{mcallester1998some,mcallester1999pac,mcallester2003pac} related the generalization of a learning algorithm to the relative entropy of the returned hypothesis distribution (also known as \emph{posterior}) with respect to a reference measure (also known as \emph{prior}) independent of the training data. His results were restricted to settings where the performance was evaluated by a loss function with a bounded range. This guarantee presents a \emph{slow rate} with respect to the number of samples, that is, the generalization error is inversely proportional to the square root of the number of samples $(\nicefrac{1}{\sqrt{n}})$. Moreover, it is of \emph{high-probability}, that is, it depends logarithmically with the probability of failure $(\log \nicefrac{1}{\beta})$.

For losses with a bounded range, Seeger and Langford~\citep{seeger2002pac,langford2001bounds} and \citet{catoni2003pac,catoni2007pac} improved upon the works from~\citet{mcallester1998some,mcallester1999pac,mcallester2003pac} and developed some of the tightest guarantees for this setting. Unfortunately, these results have the shortcoming of being rather involved and/or not having a closed-form solution for the posterior characterizing the optimal algorithm (that is, the algorithm that optimizes the generalization guarantee). After that, there have been some other developments bringing more interpretable guarantees~\citep{tolstikhin2013pac,thiemann2017strongly,rivasplata2019pac}, albeit at the cost of providing a worse characterization of the generalization error.

In \Cref{sec:bounds_bounded_losses}, we recover a connection between the results from Seeger and Langford~\citep{seeger2002pac,langford2001bounds} and \citet{catoni2003pac,catoni2007pac}. Using this connection, we then obtain an equivalent guarantee that highlights that the bound has a \emph{fast rate} of $\nicefrac{1}{n}$. Moreover, this guarantee reveals that the Gibbs posterior characterizes the algorithm that optimizes the bound from Seeger and Langford~\citep{seeger2002pac,langford2001bounds}.

When the loss has an unbounded range, there have been several attempts at extending~\citet{mcallester1998some,mcallester1999pac,mcallester2003pac}'s theory~\citep{catoni2004statistical, hellstrom2020generalization, guedj2021still, alquier2006transductive, alquier2018simpler, holland2019pac, kuzborskij2019efron, haddouche2021pac, haddouche2023pacbayes, wang2015pac}. However, most of these guarantees are either not of high probability (that is, they depend \emph{linearly} or \emph{polynomially} with the probability of failure $\beta$), contain terms that often make the guarantees non-decreasing with the number of samples $n$, decrease at a slower rate than the results from~\citet{mcallester1998some,mcallester1999pac,mcallester2003pac} when restricted to the case of losses with a bounded range, or depend on parameters that need to be chosen before the draw of the training data.

In~\Cref{sec:bounds_unbounded_losses}, we devote our attention to these shortcomings. First, in~\Cref{subsec:losses_with_bounded_CGF}, we obtain a PAC-Bayesian Chernoff analogue for losses with light tails, that is, with a bounded cumulant generating function. Then, in~\Cref{subsec:losses_with_bounded_moment}, we build on the previous results to obtain PAC-Bayesian guarantees for losses with a bounded raw moment and a bounded variance. We emphasize that all the guarantees developed in~\Cref{subsec:losses_with_bounded_CGF,subsec:losses_with_bounded_moment} are of high-probability, improve upon (or are equivalent) to the bound from~\citet{mcallester1998some,mcallester1999pac,mcallester2003pac} when restricted to losses with a bounded range, and are parameter independent (or, in case to be parameterized, they hold simultaneously for all values of the parameter, permitting a free optimization).

Finally, it is worth mentioning that the PAC-Bayesian results from~\citet{mcallester1998some,mcallester1999pac,mcallester2003pac}, even though they hold with high probability with respect to the training set, they provide guarantees for the average hypothesis returned by the algorithm. That is, the guarantees state that, with high probability, we will observe a training set for which the algorithm will generalize \emph{on average}. If the algorithm is deterministic, this does not mean much, as the algorithm always returns the same hypothesis when it observes the same training set. However, for randomized algorithms, this subtlety makes this kind of guarantees less specific.

An alternative to this kind of guarantees is provided by~\citet{catoni2003pac,catoni2007pac}, who studied single-draw PAC-Bayesian guarantees.\footnote{This nomenclature is due to~\citet{hellstrom2020generalization}.} These guarantees state that, with high probability, we will observe a training set for which the algorithm returns a hypothesis that generalizes. 

In~\Cref{sec:single-draw-pac-bayes}, we obtain single-draw PAC-Bayesian analogues to every previous PAC-Bayesian guarantee seen previously in~\Cref{ch:pac_bayesian_generalization}, where the relative entropy is substituted by the Radon--Nikodym derivative between the posterior and the prior.

\looseness=-1 Throughout the chapter, we also obtain other findings and realizations. For instance, to develop the results from~\Cref{sec:bounds_unbounded_losses}, we developed a technique to optimize parameters in probabilistic statements that can be of independent interest.

\subsection{Connection Between Privacy and Generalization}
\label{subsec:introduction_privacy_generalization}

There are three main frameworks to define and describe the privacy of mechanisms. The first one is \emph{semantic security}~\citep{Goldwasser1984ProbabilisticE}, which guarantees an absolute disclosure prevention of the information from the data. Unfortunately, applied to machine learning algorithms, this implies that the algorithm cannot even learn the patterns regarding the underlying population distribution~\citep{dwork2010difficulties}. A second framework is \emph{quantitative information flow}~\citep{smith2009foundations}. This framework encompasses many different sub-frameworks defining notions of information leakage and relate the amount of privacy to the amount of leakage of the mechanism. In this framework, one first defines a \emph{threat model}, that is, an adversary with certain capabilities that has a specific objective, and the leakage is defined as the adversary's ability to achieve their objective. Two particularly compelling leakage definitions, due to their strong operational meaning, are \emph{maximal leakage}~\citep{issa2020operational} and \emph{maximal $g$-leakage}~\citep{m2012measuring}. Luckily, it was shown that the two are equivalent~\citep{issa2020operational,alvim2014additive}. The third framework is~\emph{differential privacy}~\citep{dwork2006calibrating,dwork2014algorithmic}, which ensures that two neighbouring training sets, that is, two training sets differing in a single instance, cannot be distinguished based on their returned hypothesis.

With respect to the maximal leakage, previous studies have shown that, for potentially unbounded losses with light tails, algorithms with a finite maximal leakage generalize~\citep{esposito2021generalization,hellstrom2020generalization}. However, it is generally known that most of the dependence measures employed to characterize the information-theoretic guarantees are smaller than the maximal leakage. Therefore, algorithms with a bounded maximal leakage generalize more broadly.

In~\Cref{sec:maximal_leakage_privacy}, we use the results from~\Cref{ch:expected_generalization_error,ch:pac_bayesian_generalization} to recover known results and develop generalization guarantees for algorithms with a bounded maximal leakage. Although the presented results in this section are considered ``folklore'', they are now formalized and collected explicitly.

The level of privacy of differentially private algorithms depends on two privacy parameters, $\varepsilon$ and $\delta$. There have been several dedicated studies deriving generalization guarantees for differentially private algorithms~\citep{dwork2015preserving,dwork2015generalization,JMLR:v17:15-313,jung2021new}. However, in most of these results, the generalization guarantees do not improve as the number of samples $n$ increases. In these guarantees, the generalization error is proportional to the privacy parameter $\varepsilon$, and thus they require this parameter to decrease with the number of samples to imply a strong generalization performance.

Similarly to what happened with the maximal leakage, most of the dependence measures used to derive information-theoretic guarantees can be bounded by the privacy parameter $\varepsilon$. Unfortunately, this analysis also leads to generalization guarantees where the generalization error is proportional to the privacy parameter $\varepsilon$ but does not decrease with the number of samples $n$.

In~\Cref{sec:differential_privacy_generalization}, we first recover results similar to those in the state of the art, which guarantee the generalization of differentially private algorithms as long as their privacy parameters decrease with the number of samples. Then, we restrict ourselves to permutation invariant algorithms that operate on discrete data. We show that every algorithm of that kind generalizes in this setting. Moreover, we show that if the algorithm is also differentially private, then the generalization guarantees are stronger and vanish even for constant privacy parameters.

%% file: chapters/preliminaries.tex
In this chapter, we detail the notation and definitions necessary to follow the manuscript. In~\Cref{sec:prob_theory}, we discuss the notation employed to treat the elements of probability theory and justify the usage and assumption of standard Borel spaces throughout the manuscript. This first section will also serve to avoid measurability problems later on. After that, in~\Cref{sec:information_theory}, we review some important quantities in information theory, present some of their properties, and give some intuition about them. Finally, in~\Cref{sec:miscellaneous}, we discuss the concepts of convex conjugacy and Wasserstein distances and offer a trivial extension of the Bobkov--Götze theorem to functions with a bounded cumulant generating function.

The intention of this chapter is briefly introduce the above topics and the notation we will use to refer to them. Thus, an advanced reader may skip this chapter and only come back when they do not recognize some of the notation. The other readers are welcome to go over this chapter since, even if the treatment of the subjects is simplified, it contains insights that aid the understanding of the following chapters. Moreover, some of the insights may help more novice readers to introduce themselves into the fields before inquire into them using the bibliography provided.

\section{Probability Theory}
\label{sec:prob_theory}

Probability theory deals with the study of random phenomena. Since this article is centered around the generalization properties of (potentially stochastic) algorithms, this section lays down the notation used throughout the manuscript and some fundamental definitions and results that aid the understanding of what follows. The covering of the topic, even at an introductory and fundamental level, is far from complete, and we refer the reader to the excellent books \cite{gray2009probability,kallenberg1997foundations,mcdonald1999course} for further reading. 

\subsection{Probability Spaces and Random Objects}
\label{subsec:probability_spaces_and_random_objects}

A \emph{sample space} $\cA$ is a collection of elements that nature can select randomly. A set of elements from the sample space is an \emph{event}, and any collection $\ccA$ of events that form a $\sigma$-field (or $\sigma$-algebra~\cite[Section 1.3]{gray2009probability}) is an \emph{event space}. Together, the tuple $(\cA, \ccA)$, or simply $\cA$ if the event space $\ccA$ is clear from the context, is a \emph{measurable space}. Let $\varnothing$ be the empty set. A function $\bP: \ccA \to [0,1]$ such that $\bP \big[ \varnothing \big] = 0$, $\bP \big[ \cA \big] = 1$, and $\bP\big[\bigcup_{i=1}^\infty \cE_i \big] = \sum_{i=1}^\infty \bP\big[ \cE_i \big]$ for all disjoint events $\cE_i$ is a \emph{probability measure}, and $\bP \big[ \cE \big]$ is interpreted as the probability of the event $\cE \in \ccA$.

\begin{definition}
\label{def:probability_space}
A \emph{probability space} $(\cA,\ccA,\bP)$ consists of a measurable space $(\cA,\ccA)$ and a probability measure $\bP:\ccA \to [0,1]$ that assigns probabilities to events.
\end{definition}

Throughout the manuscript, there are times where a condition \texttt{C} is said to hold $\bP$-almost surely (or $\bP$-a.s.).\footnote{This is also known as almost everywhere (a.e.), almost always (a.a.), for almost all $a \in \cA$, almost certainly (a.c.), and with probability one (w.p.1).} This means that the set of elements for which the condition does not hold is contained in a set that has probability zero, that is, there is a set $\cS$ such that  $\{ a \in \cA : \textnormal{\texttt{C} does not hold} \} \subseteq \cS$ and $\bP \big[ \cS \big] = 0$. This concept is useful, for example, to define the concepts of \emph{essential supremum} and \emph{essential infimum}. Consider a function $f : \cA \to \bR$, the essential supremum of $f$ with respect to $\bP$ is defined as the smallest value $b \in \bR$ that bounds $f$ from above $\bP$-almost surely. That is, $\esssup_\bP f \coloneqq \inf \{ b \in \bR : f \leq b \ \bP\textnormal{-a.s.} \}$. The essential infimum is defined analogously. 

\begin{definition}
\label{def:random_object}
A \emph{random object} $X$ is an $\ccA$-measurable function $X: \cA \to \cX$ from a source measurable space $(\cA, \ccA)$ to a target measurable space $(\cX, \ccX)$. That is, for all events $\cE$ in $\ccX$, their anti-images $X^{-1}(\cE)$ are in $\ccA$.
\end{definition}

The pair of an abstract probability space $(\cA, \ccA, \bP)$ and a random object $X$ on a measurable space $(\cX, \ccX)$ give a model for random phenomena: nature chooses an outcome at random from $\cA$ and we observe the outcome $X(a)$. In particular, the probability that we observe an event $\cE \in \ccX$ is
\begin{equation*}
    \bP \big[ \{ a \in \cA : X(a) \in \cE \} \big] = \bP \circ X^{-1} \big[ \cE \big].
\end{equation*}
The probability measure $\bP_X := \bP \circ X^{-1} :  \ccX \to [0,1]$ is the \emph{probability distribution} (or just distribution) of $X$, and it describes the probability space $(\cX, \ccX, \bP_X)$. The set of all probability distributions on $\cX$ is denoted as $\cP(\cX)$. Often, we make probabilistic statements involving random objects such as: ``condition $\texttt{C}(X)$ holds''. When we want to measure the probability of such statements we abuse notation and write $\bP \big[ \texttt{C}(X) \big] = \bP \big[ \{a \in \cA : \texttt{C} \circ X (a) \textnormal{ holds} \} \big]$.

Finally, the \emph{expectation} (or expected value or mean) of a random object $X$ is 
\begin{equation*}
	\bE[X] := \int_\cA X(a) \rmd \bP[a] = \int_\cX x  \rmd \bP_X[x] 
\end{equation*}
whenever each integral exists.

Throughout the manuscript, random objects $X$ are written in capital letters, their realizations or outcomes $x$ in lowercase letters, their target sample space $\cX$ in calligraphic letters, their target event space $\ccX$ in script-style letters, and their probability distribution is written as $\bP_X$. Expectations of random objects are written as $\bE[X]$ unless when their probability distribution $\bQ = \bP_X$ is not clear, in which case they are written as $\bE_{x \sim \bQ} [x]$. 

\subsection{Polish and Standard Borel Spaces}
\label{subsec:polish_borel_spaces}

Modeling random phenomena as pairs of abstract probability spaces $(\cA, \ccA, \bP)$ and random objects $X$ allows us to study the probability space $(\cX, \ccX, \bP_X)$ induced by the said objects. In this way, one may demand certain structural properties to the target space $(\cX, \ccX)$ generated by the outcomes of the random object and maintain the source space $(\cA, \ccA)$ from nature abstract. Ideally, the demanded properties are restrictive enough to allow us to develop fundamental results in probability theory, while being general enough to capture the behavior of many random phenomena of interest.

A natural\footnote{Arguably the most general measurable spaces that allow us to proof fundamental results in probability theory are standard spaces~\cite{gray2009probability}. For brevity and simplicity, the discussion and results developed are for standard Borel spaces, though some of the results can be directly extended to standard spaces.} approach to equipping a target sample space $\cX$ with structure is to consider a metric $\rho: \cX \times \cX \to \bR$. The tuple $(\cX, \rho)$, or simply $\cX$ if the metric $\rho$ is clear from the context, is called a \emph{metric space}. A desirable property of a metric space $\cX$ is that every element $x \in \cX$ can be well approximated. This is formalized with the following two properties:
\begin{enumerate}
    \item \emph{Completeness}: The metric space $(\cX,\rho)$ is \emph{complete} if all Cauchy sequences converge to points in $\cX$. A sequence $\{ x_n \in \cX \}_{n=1}^\infty$ is Cauchy if for all $\varepsilon > 0$ there is a $k \in \bN$ such that $\rho(x_n, x_m) < \varepsilon$ whenever $n,m \geq k$.
    \item \emph{Separability}: The sample space $\cX$ is \emph{separable} if there is a countable, dense subset $\cE \subseteq \cX$. That is, a set $\cE$ whose closure is the whole space ($\overline{\cE} = \cX$).
\end{enumerate}

\begin{definition}
\label{def:polish_space}
A \emph{Polish space} $(\cX, \rho)$ is a complete and separable metric space.
\end{definition}

Given a metric space $(\cX, \rho)$, one may consider the collection of open balls $\cB_r(x) \coloneqq \{ y \in \cX : \rho(x,y) < r \}$ (or its induced topology~\cite[Section 1.2]{gray2009probability}). Then, the \emph{Borel $\sigma$-field} is the $\sigma$-field generated by such a collection; that is,
\begin{equation*}
    \ccB(\cX) \coloneqq \sigma \big(\{ \cB_r(x): x \in \cX, r > 0 \} \big).
\end{equation*}

Therefore, this manuscript focuses on target spaces $(\cX, \ccX)$ where the sample space $\cX$ is Polish and the event space is its Borel $\sigma$-field $\ccX = \ccB(\cX)$. These measurable spaces are called \emph{standard Borel spaces}. 

\begin{definition}
\label{def:standard_borel_spaces}
A \emph{standard Borel space} $(\cX, \ccX)$ is a measurable space that is Borel isomorphic to $(\cS,\cB(\cS))$, where $\cS$ is some subset of $[0,1]$. That is, there is a measurable one-to-one and onto function $f: \cX \to \cS$ such that $f^{-1}$ is also measurable.
\end{definition}

\begin{theorem}[Polish and standard Borel spaces~{\cite[Theorem~4.3]{gray2009probability}}]
\label{th:polish_and_borel_spaces}
Let $(\cX,\rho)$ be a Polish space. If $\cS$ is a subset of $\cX$, then $(\cS, \cB(\cS))$ is a standard Borel space.
\end{theorem}

\Cref{def:standard_borel_spaces} tells us that any uncountable standard Borel space inherits the properties of the measurable space $\big((0,1),\cB((0,1))\big)$ and \Cref{th:polish_and_borel_spaces} states that all the target spaces $(\cX, \ccX)$ constructed as described before are indeed standard Borel. At first glance, these attractive properties may seem very limiting, but in fact many common and useful measurable spaces are standard Borel. For example, all the following sample spaces construct standard Borel spaces with the following selected, usual metrics:
\begin{itemize}
    \item Any finite set $\cX$ such that $|\cX| < \infty$ with the discrete metric $\rho(x,y) = \bI[x \neq y]$.
    \item The rational $\cX = \bQ$, the irrational $\cX = \bR \setminus \bQ$, and the real numbers $\cX = \bR$ with the metric $\rho(x,y) = |x - y|$.
    \item The product of the rational $\cX = \bQ^n$, the irrational $\cX = (\bR \setminus \bQ)^n$, and the real numbers $\cX = \bR^n$ with the metric $\rho(x,y) = \lVert x - y \rVert_2$.
    \item Any separable Hilbert space, such as the space of square-integrable functions $\cX = \cL^2(\mu)$ with metric $\rho(f,g)^2 = \int (f-g)^2 d\mu$.
\end{itemize}

For any distribution $\bP$ on a Polish space $\cX$, the \emph{support} of the distribution is defined as the subset $\supp(\cX)$ of $\cX$ such that every neighbourhood of an element of $\supp(\cX)$ has a positive probability, that is, $\supp(\bP) = \{x \in \cX : \bP \big[\cB_r(x) \big] > 0 \textnormal{ for all } r > 0 \}$.

\subsection{Conditional Probability and Conditional Expectation}

Consider a probability space $(\cA, \ccA, \bP)$, a random object $X$ in $(\cX,\ccX)$, and an event $\cF \in \ccA$ such that $\bP[\cF] > 0$. Given that the event $\cF$ occurred, the probability that the event $\cE \in \ccA$ occurs is the \emph{conditional probability of $\cE$ given $\cF$}
\begin{equation*}
    \bP^{\cF} \big[ \cE \big] = \bP \big[ \cE \mid \cF \big] := \frac{\bP \big[ \cE \cap \cF \big]}{\bP \big[ \cF \big]},
\end{equation*}
and the expectation of $X$ is the \emph{conditional expectation of $X$ given $\cF$}
\begin{equation*}
	\bE^\cF \big[ X \big] = \bE[X \mid \cF] := \int_\cA X \rmd \bP^\cF  = \frac{1}{\bP \big[\cF \big]} \int_\cF X \rmd \bP.
\end{equation*}
Therefore, combining the definition of a probability measure from~\Cref{subsec:probability_spaces_and_random_objects} and the definition of the conditional probability and conditional expectation, we obtain the so called \emph{law of total probability} and \emph{total expectation}. Namely, that for every collection of disjoint events $\cF_i$ such that $\bP \big[ \cF_i \big] > 0$ and that $\sum_{i} \bP \big[ \cF_i \big] = 1$, for every event $\cE$ and every random object $X$
\begin{equation*}
    \bP \big[ \cE \big] = \sum_{i} \bP^{\cF_i} \big[ \cE \big] \cdot \bP \big[ \cF_i \big] \hspace{1em} \textnormal{ and } \hspace{1em} \bE \big[ X \big] = \sum_{i} \bE^{\cF_i} \big[ X \big] \cdot \bP \big[ \cF_i \big],
\end{equation*}
where the summation over $i$ is kept vague as this holds true for both a finite and infinitely countable number of events $\cF_i$.

In this way, the \emph{conditional probability of an event $\cE \in \ccA$ given a sub $\sigma$-field $\ccF \subseteq \ccA$} is the probability that the event $\cE$ occurred knowing if each event $\cF \in \ccF$ occurred or not. It is written as $\bP^\ccF \big[ \cE \big]$ or $\bP\big[ \cE \mid \ccF \big]$, and $\bP^\ccF$ is the $\bP$-a.s. unique $\ccF$-measurable function such that for all events $\cF \in \ccF$,
\begin{equation*}
    \bP \big[ \cE \cap \cF \big] = \int_\cF \bP^\ccF \big[ \cE \big] \rmd\bP.
\end{equation*}
Similarly, the \emph{conditional expectation of $X$ given a sub $\sigma$-field $\ccF \subseteq \ccA$} is written as $\bE^\ccF \big[ X \big]$ or $\bE \big [X \mid \ccF \big]$ and $\bE^\ccF$ is the $\bP$-a.s. unique $\ccF$-measurable function such that for all events $\cF \in \ccF$,
\begin{equation*}
	\int_\cF X d\bP = \int_\cF \bE^\ccF \big[X \big] \rmd\bP.
\end{equation*}

Finally, the \emph{conditional distribution of an event $\cE \in \ccA$ given a random object $Y$} is defined as $\bP^Y \big[ \cE \big] := \bP^{\sigma(Y)} \big[ \cE \big]$, where $\sigma(Y) \subseteq \ccA$ is the smallest $\sigma$-field for which $Y$ is measurable. Moreover, the \emph{conditional expectation of the random object $X$ given $Y$} is $\bE^Y \big[X \big] = \bE^{\sigma(Y)} \big[X \big]$.

\subsection{Several Random Objects}

Consider a couple of random objects\footnote{The following text applies to a countable set of random objects analogously, but the exposition is reduced to only two for simplicity and clarity.} $X: \cA \to \cX$ and $Y: \cA \to \cY$, where $(\cX,\ccX)$ and $(\cY,\ccY)$ are standard Borel spaces. Then, they may be treated as a single random object $(X,Y): \cA \to \cX \times \cY$ with target space $(\cX \times \cY, \ccX \otimes \ccY)$, where $\ccX \otimes \ccY$ is the $\sigma$-field generated by the measurable rectangles of $\cX \times \cY$. That is,
\begin{equation*}
    \ccX \otimes \ccY \coloneqq \sigma \big( \{ \cE \times \cF : \cE \in \ccX, \cF \in \ccY \} \big).
\end{equation*}
Moreover, the space $(\cX \times \cY, \ccX \otimes \ccY)$ is also standard Borel and the event space is the Borel $\sigma$-field of the sample space, i.e. $\ccX \otimes \ccY = \cB(\cX \times \cY)$\cite[Lemma 1.2]{kallenberg1997foundations}. 

The distribution $\bP_{X,Y}$ of the random object $(X,Y)$ is called the \emph{joint distribution} of the random objects. The \emph{marginal distribution} on $\cX$ of a probability distribution $\bQ$ on $\cX \times \cY$ is defined as $\bQ \big[\cdot, \cY \big]: \ccX \to [0,1]$. In particular, the marginal distribution on $\cX$ of the joint distribution agrees with the distribution of $X$, that is, $\bP_X = \bP_{X,Y}\big[ \cdot, \cY \big]$. However, there are more distributions in $\cX \times \cY$ that have $\bP_X$ and $\bP_Y$ as marginals. For example, the \emph{product distribution} $\bP_X \otimes \bP_Y$, where $\bP_X \otimes \bP_Y [\cE \times \cF] = \bP_X[\cE]  \bP_Y[\cF]$ for all $\cE \in \ccX, \cF \in \ccY$. The set of all such distributions is called the \emph{couplings} of $\bP_X$ and $\bP_Y$ and is denoted as $\Pi(\bP_X, \bP_Y)$. The random objects $X$ and $Y$ are said to be \emph{independent} if $\bP_{X,Y} = \bP_X \otimes \bP_Y$.

Given two standard Borel spaces $(\cX, \ccX)$ and $(\cY, \ccY)$, a \emph{probability (or Markov) kernel} from $\cX$ to $\cY$ is a mapping $\kappa: \cX \times \ccY \to [0,1]$ such that (i) $\kappa(x,\cdot)$ is a probability measure on  $(\cY, \ccY)$ for all $x \in \cX$; and (ii) $\kappa(\cdot, \cF)$ is a measurable function from $\cX$ to the set of all probability measures on $\cY$, $\cP(\cY)$, for all $\cF \in \ccY$.

\begin{definition}
\label{def:conditional_distribution}
Let $X$ and $Y$ be random objects with target standard Borel spaces $(\cX, \ccX)$ and $(\cY, \ccY)$. The (regular) conditional distribution of $Y$ given $X$ is the $\bP_X$-a.s. unique probability kernel $\kappa$ from $\cX$ to $\cY$ satisfying that $\bP^X \big[ Y \in \cF] = \kappa(X,\cF)$ $\bP$-a.s. for all $\cF \in \ccY$. In what follows, it is denoted as $\bP^X_Y \coloneqq \kappa(X,\cdot)$
\end{definition}
This way, the conditional distribution $\bP^X_Y$ of $Y$ given $X$ is a random object on $\cP(\cY)$ where for each realization $x$ of $X$, the realization $\bP^{X=x}_Y \in \cP(\cY)$ is a probability distribution on $\cY$. The marginal distribution of $Y$ is obtained from the conditional distribution as $\bP_Y = \bP_Y^X \circ \bP_X \coloneqq \bE_{x \sim \bP_X}\big[\bP^{X=x}_Y \big]$ and the joint distribution of $(X,Y)$ satisfies that  
\begin{equation*}
	\bP_{X,Y}\big[\cE\big] =  \int \bigg( \int \bI_\cE(x,y) \rmd\bP^{X=x}_{Y}(y) \bigg) \rmd\bP_X (x)
\end{equation*}
for all $\cE \in \ccX \otimes \ccY$, where $\bI_\cE$ is the indicator function of the set $\cE$ which returns 1 if $(x,y) \in \cE$ and 0 otherwise. Therefore, the joint distribution is sometimes written as $\bP_{X,Y} = \bP_X \otimes \bP^X_{Y}$.

\subsection{Densities}

At times, working with probability measures and distributions is cumbersome: for instance, when comparing two probability measures $\bP$ and $\bQ$, or when calculating expectations or probabilities of events in a target space under certain distributions $\bP_X$. In these cases, it is often helpful to consider a suitable $\sigma$-finite~\cite[Section 1.6]{gray2009probability} base measure $\mu$ and study the \emph{densities} of said probability measures with respect to $\mu$.

Given two measures $\bP$ and $\mu$ on a measurable space $(\cA, \ccA)$, the measure $\bP$ is \emph{absolutely continuous} with respect to $\mu$ (or $\bP \ll \mu$) if $\bP[\cE] = 0$ whenever $\mu(\cE) = 0$ for all events $\cE \in \ccA$. 

\begin{theorem}[{Radon--Nikodym theorem~\cite [Theorem~6.10]{mcdonald1999course}}]
\label{th:radon-nikodym}
Consider a probability measure\footnote{Technically, it is enough for $\bP$ to be $\sigma$-finite, but we restrict the theorem to probability measures for our purposes.} $\bP$ and a $\sigma$-finite measure $\mu$ on a measurable space $(\cA,\ccA)$ such that $\bP \ll \mu$. Then, there is a $\mu$-a.s. unique $\ccA$-measurable function $\sfp: \ccA \to \bR_+$ such that for all $\cE \in \ccA$
\begin{equation*}
	\bP \big[ \cE \big] = \int_\cE \sfp \rmd\mu.
\end{equation*}
Such a function is written $\sfp \coloneqq \frac{\rmd\bP}{\rmd\mu}$ and is called the Radon--Nikodym derivative of $\bP$ with respect to $\mu$.
\end{theorem}
\begin{definition}
\label{def:density}
\looseness=-1 Given a probability measure $\bP$ and a $\sigma$-finite measure $\mu$ on a measurable space $(\cA, \ccA)$ such that $\bP \ll \mu$, the \emph{density} of $\bP$ with respect to $\mu$ (or just the density of $\bP$ if $\mu$ is clear from context) is the Radon--Nikodym derivative $\frac{\rmd\bP}{\rmd\mu}$.
\end{definition}

An immediate consequence of \Cref{th:radon-nikodym} is that the expected value of a random object $X$ with density $\sfp_X$ with respect to $\mu$ is 
\begin{equation*}
	\bE \big[ X \big] = \int_{\cE} x \sfp_X \rmd \mu.
\end{equation*}

This result is often also presented in a slightly different way and is referred to as the \emph{change of measure} result. More precisely, let $\bP$ and $\bQ$ be two distributions on $\cX$ such that $\bP \ll \bQ$, then 
\begin{equation}
\label{eq:change_of_measure}
	\bE_{x \sim \bP} \big[ x \big] = \bE_{x \sim \bQ} \bigg[x \cdot \frac{\rmd \bP}{\rmd \bQ}(x)  \bigg].
\end{equation}

Throughout the manuscript, densities of probability measures $\bP$ and distributions $\bP_X$ are written with their lowercase sans serif letters $\sfp$, $\sfp_X$. 
A canonical example when densities are helpful are \emph{probability density functions} (pdfs) $\sff_X$ of a random object $X: \cA \to \bR^n$, where the base measure is the Lebesgue measure $\lambda$~\cite[Sections 3.4 and 7.2]{gray2009probability} and $\sff_X = \sfp_X = \frac{\rmd\bP_X}{\rmd\lambda}$. In this situation, the densities of many distributions are known and Lebesgue integration~\cite[Section 4.3]{mcdonald1999course} (which is an extension of the standard Riemann integral that coincides with it in the measurable rectangles on $\bR^n$~\cite[Section 4.3 and Section 6.4]{mcdonald1999course}) can be employed to calculate events' probabilities. We will refer to this kind of random objects as \emph{(absolutely) continuous random variables}.
Another example are \emph{probability mass functions} (pmfs) of a random object $X: \cA \to [k]$, where the base measure is the counting measure $\mathfrak{c}$. In this situation, the density of $X$ is $\sfp_X = \frac{\rmd\bP_X}{\rmd\mathfrak{c}}$ and $\sfp_X (i) = \bP_X \big[ \lbrace i \rbrace \big]$. We will refer to this kind of random objects as \emph{discrete random variables}. More generally, we will refer to random objects that can either be continuous or discrete random variables simply as \emph{random variables}.

\subsection{Moments and Generating Functions}
\label{subsec:moments_and_generating_functions}

Consider a random variable $X$ with a target space $\cX \subseteq \bR$. The $p$-th \emph{absolute moment} of $X$ is $\bE \big[ |X|^p \big]$ and it is said to exist when $\bE \big[ |X|^p \big] < \infty$. By either Hölder's or Jensen's inequality it follows that the absolute moments of a random variable are non-decreasing, that is, $\bE \big[ |X|^p \big] \leq \bE \big[ |X|^q \big]$ for all $p \leq q$. Hence, if the $p$-th absolute moment does not exist, no larger moment does. 
The $p$-th \emph{moment} (or \emph{raw moment}) of $X$ is $\bE \big[X^p\big]$ and it is said to exist when the $p$-th absolute moment exists. When the random variable $X$ is non-negative, then its $p$-th moment and its $p$-th absolute moment coincide. This will be the case throughout most of this manuscript and therefore we will use both terms interchangeably.
For $p > 1$, we are often interested in the $p$-th \emph{central moment}, which is defined as $\bE \big[ |X - \bE[X]|^p \big]$ and gives us properties of the random variable independent of translation. For example, a notable central moment is the \emph{variance}, which is defined as
\begin{equation*}
	\var[X] \coloneqq \bE \big[ (X - \bE[X])^2 \big] = \bE \big[ X^2 \big] - \bE \big[X \big]^2.
\end{equation*}

Another important quantity related to a random variable $X$ is its \emph{moment} (or \emph{probability}) \emph{generating function} (MGF), which is defined as the Laplace transform of $-X$.

\begin{definition}
\label{def:mgf}
The \emph{moment-generating function} of a random variable $X$ is $M_X(\lambda) \coloneqq \bE \big[e^{\lambda X} \big]$ for all $\lambda \in \bR$. We say that it exits if it is bounded in $(-b,b)$ for some $b > 0$.
\end{definition}

The MGF of a random variable $X$ completely characterizes it, to the point that $M_X(\lambda) = M_Y(\lambda)$ if and only if $X = Y$ almost surely. Moreover, every moment of a random variable $X$ can be obtained from its MGF as
\begin{equation}
	\bE \big[ X^p \big] \coloneqq \frac{\rmd^p M_X}{\rmd \lambda^p} (0).
\end{equation}
Therefore, the existence of the MGF implies the existence of every moment. However, the reverse is not true: for example, the Pareto distribution of the first kind with parameters $a=3$ and $k=1$ does not have an MGF but its variance is $\nicefrac{3}{4}$~\cite[Chapter 20]{johnson1994continuous}, and the lognormal distribution does not have an MGF but all its moments are finite~\cite[Chapter 14]{johnson1994continuous}~\cite{asmussen2016laplace}.

In this manuscript, we will not deal with the MGFs of random variables, but with their logarithm, which is referred to as the \emph{cumulant-generating function} (CGF). More precisely, we will abuse notation and refer to the CGF of a random variable $X$ to mean the CGF of the centered random variable $X - \bE \big[ X \big]$ to simplify the calculations and adhere to the standard usage in the machine learning theory community. In practice, this is not problematic since $\log M_{X-a}(\lambda) = \log M_X(\lambda) - \lambda a$ for every $\lambda, a \in \bR$.

\begin{definition}
\label{def:cgf}
The \emph{cumulant-generating function} of a random variable $X$ is $\Lambda_X(\lambda) \coloneqq  \log \bE \big[ e^{\lambda (X - \bE[X])} \big]$ for all $\lambda \in \bR$. We say that it exits if it is bounded in $(-b,b)$ for some $b > 0$.
\end{definition}
 
If it exists, the CGF is a convex, continuously differentiable function such that $\Lambda_X(0) = \Lambda_X'(0) = 0$~\cite{banerjee2021information, zhang2006information}. For the purpose of this manuscript, there are other two properties of the CGF that we will consider:
\begin{enumerate}
	\item Scaling a random variable scales the argument of the CGF, that is, $\Lambda_{\alpha X}(\lambda) = \Lambda_X(\alpha \lambda)$ for all $\lambda \in \bR$ and all $\alpha \in \bR$ such that $\alpha \lambda \in (-b,b)$.
	\item The CGF of the sum of independent random variables is the sum of their CGFs, that is, $\Lambda_{X + Y}(\lambda) = \Lambda_X(\lambda) + \Lambda_Y(\lambda)$ for all independent $X$ and $Y$.
\end{enumerate}

\section{Information Theory}
\label{sec:information_theory}

Information theory deals with the intangible concept of information. Informally, the occurrence of a rare event gives us more information than the occurrence of common ones. Although information theory was originated to deal with problems from communications' theory, it has grown into a field of its own and has applications to probability theory, statistics, computer science, and even physics. 

This section describes certain quantities paramount to information theory, presents some of their properties, and gives some intuition about them. These quantities and intuition will be useful later in the manuscript to better understand some conditions under which learning algorithms generalize. As in the previous section, the coverage of the topic is really scarse and the interested reader may refer to the wonderful books \cite{Cover2006,polianskyi2022,mackay2003information} and papers~\cite{dwork2006calibrating,dwork2014algorithmic,dwork2016concentrated, mironov2017renyi, bun2016concentrated, balle2020hypothesis, issa2020operational,saeidian2023apointwise,saeidian2023bpointwise} for further reading. 

\subsection{Entropy}
\label{subsec:ent}

Perhaps the most important quantity in information theory is the \emph{entropy}. It is a measure of the uncertainty of a random object $X$. The larger the entropy of the random object $X$, the more uncertain we are about its outcome. 

\begin{definition}
\label{def:entropy}
The \emph{entropy} of a random object $X$ with density $\sfp_X$ is
\begin{equation*}
	\ent(X) \coloneqq - \bE \big[ \log \sfp_X(X) \big].
\end{equation*}
\end{definition}

For multiple random variables $X_1, \ldots, X_n$ this concept is readily extended and $\ent(X_1, \ldots, X_n) = - \bE \big[ \log \sfp_{X_1, \ldots, X_n} (X_1, \ldots, X_n) \big]$. Usually, the entropy is considered for either discrete or absolutely continuous random variables~\cite{Cover2006,polianskyi2022}, and thus that $\sfp_X$ is a pmf or a pdf is understood from the context. 

In information theory, a Markov kernel $\bP_Y^X$ is often called a \emph{channel} and it is understood as a mechanism that processes a random object $X$ into another random object $Y$ (see \Cref{fig:channel} above). In this manuscript, a common channel will be an algorithm $\bA$ that transforms a dataset $S$ into a hypothesis $W = \bA(S)$. It is natural to wonder what is the uncertainty of the processed object $Y$ \emph{given} the input $X$. This concept is referred to as the \emph{disintegrated entropy of $Y$ given $X$} and it is a random object depending on $X$~\cite{negrea2019information}. The expectation of this random object is referred to as the \emph{conditional entropy of $Y$ given $X$.}

\begin{definition}
\label{def:conditional_entropy}
The \emph{disintegrated entropy of $Y$ given $X$} is 
\begin{equation*}
	\ent^X(Y) \coloneqq - \bE^X \big[ \log \sfp_Y^X(Y) \big]
\end{equation*}
and the \emph{conditional entropy of $Y$ given $X$} is 
\begin{equation*}
	\ent(Y|X) \coloneqq \bE \big[ \ent^X (Y) \big] = - \bE \big[ \log \sfp_Y^X(X) \big].
\end{equation*}

\end{definition}

\begin{figure}[t]
\centering
\begin{tikzpicture}[->, line width=0.5pt, node distance=2cm, scale=1]]

	\node (X) at (0,0) {$X$};
	\node[draw, rectangle, minimum width=1cm] (channel) at (2,0) {$\bP_Y^X$};
	\node (Y) at (4,0) {$Y$};
	
	\path (X) edge (channel);
	\path (channel) edge (Y);
\end{tikzpicture}
\caption{Illustration of a channel processing $X$ to $Y$.}
\label{fig:channel}
\end{figure}
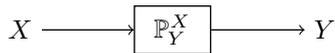

Two important properties of the entropy are summarized below~\cite{Cover2006, polianskyi2022}. The first one states that the knowledge of the input of a Markov kernel can only decrease its uncertainty, and the second one describes how to decompose the entropy of a collection of random objects.

\begin{proposition}[\!\! {\cite[Theorems 1.4 and 2.6]{polianskyi2022}}]
\label{prop:entropy_properties_general}
Consider the random objects $X$, $Y$, and $X_1, \cdots, X_n$. Then,
\begin{itemize}
	\item Conditioning reduces entropy: $\ent(Y|X) \leq \ent(Y)$.
	\item Chain rule: $\ent(X_1, \cdots, X_n) = \sum_{i=1}^n \ent(X_i | X_1, \cdots, X_{i-1}) \leq \sum_{i=1}^n \ent(X_i)$.
\end{itemize}
\end{proposition}

\paragraph{Shannon's Entropy.} Originally, \citet{shannon1948mathematical} introduced the concept of entropy for discrete random variables. For this reason, for discrete random variables, we will refer to the entropy as Shannon's entropy. Shannon introduced it to describe the amount of information contained in a random variable, which we now measure in \emph{bits} if the logarithm is in base 2 or \emph{nats} if it is in base $e$. Shannon showed that the minimum description length (on average) of realizations of a discrete random variable $X$ is within one bit of the entropy $\ent(X)$. Therefore, the entropy is a fundamental limit of \emph{compression}. With this interpretation, there are a couple important properties of Shannon's entropy.

\begin{proposition}[\!\! {\cite[Theorem 1.4]{polianskyi2022}}]
\label{prop:entropy_properties_discrete}
Consider a discrete random variable $X$, a bijective function $f: \cX \to \cX$, and a one-to-one function $g: \cX \to \cY$. Then, 
\begin{itemize}
	\item Shannon's entropy is non-negative: $\ent(X) \geq 0$. 
	\item The uniform distribution maximizes the Shannon's entropy: $\ent(X) \leq \log |\cX|$.
	\item Shannon's entropy is invariant to relabeling: $\ent(X) = \ent(f(X))$.
	\item Shannon's entropy is monotone: $\ent(X, Y) \geq \ent(X)$ with equality if and only if $Y = g(X)$.
\end{itemize}
\end{proposition}

\subsection{Relative Entropy}
\label{subsec:relative_entropy}

\looseness=-1 The \emph{relative entropy} $\relent(\bP \Vert \bQ)$ of a distribution $\bP$ with respect to another distribution $\bQ$ (or \emph{Kullback--Leibler divergence}) is a measure of the dissimilarity between the distributions $\bP$ and $\bQ$. 
It is important to note that the relative entropy is asymmetric, i.e. $\relent(\bP \Vert \bQ) \neq \relent(\bQ \Vert \bP)$, and therefore it is not a distance. This is natural in certain situations like hypothesis testing. For instance, consider that we want to study the bias of a die and we consider the distribution of a fair die $\bP$ and of a biased die $\bQ$ that always lands on even numbers. If we observe the sequence ``$\mathrm{2,4,4,2,6}$'', then we may believe that the distribution is $\bQ$, but we cannot be absolutely sure. On the other hand, if we observe the sequence ``$\mathrm{2, 4, 4, 2, 3}$'', then we are absolutely certain that the distribution is $\bP$. The relative entropy describes this imbalance as $\relent(\bQ \Vert \bP) = \log 2$ whereas $\relent(\bP \Vert \bQ) \to \infty$ \cite[Section 2.1]{polianskyi2022}.

From an information-theoretic perspective, it describes the inefficiency of assuming that a random variable $X$ has distribution $\bQ$, when in reality it has distribution $\bP_X = \bP$. That is, if we constructed a code for a discrete random variable $X$ assuming it has a distribution $\bQ$, then we would need \emph{at least} $\ent(X) + \relent(\bP \Vert \bQ)$ bits (on average) to describe the said random variable~\cite[Section 2.3]{Cover2006}. The example with the die above describes a pathological worst-case situation. Since the odd outcomes have probability zero under $\bQ$, the code constructed assuming this distribution may not even have codewords to describe them, and therefore the outcome cannot be described with this code even if we were to have access to infinite extra bits.

\begin{definition}
\label{def:relative_entropy}
The \emph{relative entropy} of a distribution $\bP$ with respect to another distribution $\bQ$ is
\begin{equation}
\label{eq:relative_entropy}
	\relent(\bP \Vert \bQ) \coloneqq \bE_{x \sim \bP} \bigg[ \log \frac{\rmd \bP}{\rmd \bQ} (x) \bigg] = \int \log \Big(\frac{\rmd \bP}{\rmd \bQ} \Big) \rmd \bP
\end{equation}
if $\bP \ll \bQ$ and $\relent(\bP \Vert \bQ) \to \infty$ otherwise.
\end{definition}

Another formulation of the relative entropy comes from defining it as an instance of an $f$-divergence (see \Cref{subsec:f-divergences} below). More precisely, \Cref{eq:relative_entropy} can be equivalently written using a change of measure (\Cref{th:radon-nikodym}) as
\begin{equation*}
	\relent(\bP \Vert \bQ) \coloneqq \bE_{x \sim \bQ} \bigg[ \frac{\rmd \bP}{\rmd \bQ} (x) \log \frac{\rmd \bP}{\rmd \bQ} (x) \bigg] = \int \Big(\frac{\rmd \bP}{\rmd \bQ} \Big) \log \Big(\frac{\rmd \bP}{\rmd \bQ} \Big) \rmd \bQ.
\end{equation*}

When both $\bP$ and $\bQ$ have a density with respect to a common dominating measure $\mu$ (e.g. $\mu = \bP + \bQ$), then it is convenient to write the relative entropy in terms of these densities using the \emph{likelihood ratio $\nicefrac{\sfp}{\sfq}$}.
\begin{equation*}
	\relent(\bP \Vert \bQ) = \bE_{x \sim \bP} \bigg[ \log \frac{\sfp(x)}{\sfq(x)} \bigg].
\end{equation*}

As with the entropy, we sometimes want to study the relative entropy of a channel $\bP_Y^X$ with respect to another channel $\bQ_Y^X$, or of a channel with respect to a distribution. Moreover, if we consider two channels $\bP_Y^X$ and $\bQ_Y^X$ and a common input distribution $\bP_X$, then the relative entropy of $\bP_Y^X$ with respect to $\bQ_Y^X$ is a continuous random variable depending on $X$ and its expectation is the \emph{conditional relative entropy given $X$}.

\begin{definition}
\label{def:conditional_relative_entropy}
The conditional relative entropy of the channel $\bP_Y^X$ with respect to the channel $\bQ_Y^X$ given $X$ is
\begin{equation*}
	\relent(\bP_Y^X \Vert \bQ_Y^X | \bP_X) \coloneqq \bE \big[ \relent(\bP_Y^X \Vert \bQ_Y^X) \big].
\end{equation*}
\end{definition}

The relative entropy enjoys several properties that are paramount in many important results in information theory. Some of these properties are summarized below~\cite{Cover2006, polianskyi2022}.

\begin{proposition}
\label{prop:properties_relative_entropy}
Consider two distributions $\bP_X$ and $\bQ_X$ on $\cX$, two channels $\bP_Y^X$ and $\bQ_Y^X$ from $\cX$ to $\cY$, and a one-to-one function $f$ from $\cX$ to $\cY$. Moreover, let $\bP_{X,Y} = \bP_{Y}^X \otimes \bP_X$ and $\bQ_{X,Y} = \bQ_{Y}^X \otimes \bQ_X$ and also consider the joint distributions $\bP_{X^n} = \bP_{X_1, \ldots X_n}$ and $\bQ_{X^n} = \bQ_{X_1, \ldots, X_n}$. Then, 
\begin{itemize}
	\item The relative entropy is non-negative: $\relent(\bP_X \Vert \bQ_X) \geq 0$.
	\item The relative entropy $\relent(\bP_X \Vert \bQ_X)$ is convex and lower semicontinuous with respect to $\bP_X$ for every fixed $\bQ_X$.
	\item The relative entropy is monotone: $\relent(\bP_{X,Y} \Vert \bQ_{X,Y}) \geq \relent(\bP_Y \Vert \bQ_Y)$.
	\item Chain rule. $\relent(\bP_{X^n} \Vert \bQ_{X^n}) = \sum_{i=1}^n \relent(\bP_{X_i | X^{i-1}} \Vert \bQ_{X_i | X^{i-1}} | \bP_{X^{i-1}})$. 
	
	\item The relative entropy tensorizes: If $Q_{X^n} = Q_{X_1} \otimes \cdots \otimes Q_{X_n}$, then 
\begin{align*}	
	\relent(\bP_{X^n} \Vert \bQ_{X^n}) &= \relent(\bP_{X^n} \Vert \bP_{X_1} \otimes \cdots \otimes \bP_{X_n}) + \sum_{i=1}^n \relent(\bP_{X_i} \Vert \bQ_{X_i}) \\
	&\geq \sum_{i=1}^n \relent(\bP_{X_i} \Vert \bQ_{X_i})
\end{align*}
	with equality if and only if $\bP_{X^n} = \bP_{X_1} \otimes \cdots \bP_{X_n}$.

	\item Conditioning increases the relative entropy (\Cref{fig:channels_to_aid_relative_entropy_properties} left). Assume that $\bP_Y$ and $\bQ_Y$ are both obtained after passing a random object $X$ with distribution $\bP_X$ through the channels $\bP_Y^X$ and $\bQ_Y^X$, i.e. $\bP_Y = \bP_Y^X \circ \bP_X$ and $\bQ_Y = \bQ_Y^X \circ \bP_X$. Then $\relent(\bP_Y \Vert \bQ_Y) \leq \relent(\bP_Y^X \Vert \bQ_Y^X | \bP_X)$.
	\item Data processing inequality (\Cref{fig:channels_to_aid_relative_entropy_properties} right). Now consider that $\bP_Y$ and $\bQ_Y$ are obtained after passing a random object $X$ and a random object $\bar{X}$ with distributions $\bP_X$ and $\bQ_X = \bP_{\bar{X}}$ through the channel $\bP_Y^X$, i.e. $\bP_Y = \bP_Y^X \circ \bP_X$ and $\bQ_Y = \bP_Y^X \circ \bQ_X$. Then, $\relent(\bP_{Y} \Vert \bQ_Y) \leq \relent(\bP_X \Vert \bQ_Y)$ with equality if and only if $\bP_Y^X$ is a deterministic and losseless channel described by $f$.
\end{itemize}
\end{proposition}

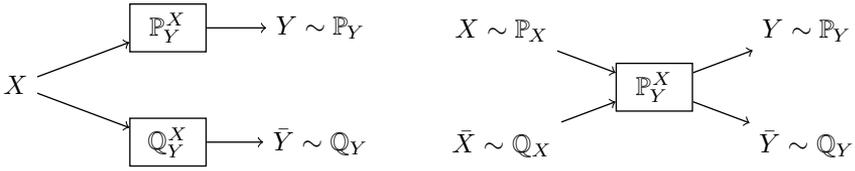
\begin{figure}[t]
	\centering
	\begin{subfigure}{0.4\textwidth}
		\centering
		\begin{tikzpicture}[->, line width=0.5pt, node distance=2cm, scale=1]]
			\node (X) at (0,0) {$X$};
			\node [draw, rectangle, minimum width=1cm] (channel_1) at (2,0.75) {$\bP_Y^X$};
			\node [draw, rectangle, minimum width=1cm] (channel_2) at (2,-0.75) {$\bQ_Y^X$};
			\node (Y_1) at (4,0.75) {$Y \sim \bP_Y$};
			\node (Y_2) at (4,-0.75) {$\bar{Y} \sim \bQ_Y$};
			
			\path (X) edge (channel_1);
			\path (X) edge (channel_2);
			\path (channel_1) edge (Y_1);
			\path (channel_2) edge (Y_2); 
		\end{tikzpicture}
	\end{subfigure}
	\hspace{2em} 
	\begin{subfigure}{0.4\textwidth}
		\centering
		\begin{tikzpicture}[->, line width=0.5pt, node distance=2cm, scale=1]]
			\node (X_1) at (0, 0.75) {$X \sim \bP_X$};
			\node (X_2) at (0, -0.75) {$\bar{X} \sim \bQ_X$};
			\node[draw, rectangle, minimum width=1cm] (channel) at (2,0) {$\bP_Y^X$};
			\node (Y_1) at (4,0.75) {$Y \sim \bP_Y$};
			\node (Y_2) at (4,-0.75) {$\bar{Y} \sim \bQ_Y$};
			
			\path (X_1) edge (channel);
			\path (X_2) edge (channel);
			\path (channel) edge (Y_1);
			\path (channel) edge (Y_2); 
		\end{tikzpicture}
		\end{subfigure}
	\caption{Illustrations of a single random object $X$ processed with different channels (left) and two different random objects processed with the same channel (right).}
	\label{fig:channels_to_aid_relative_entropy_properties}
\end{figure}

Finally, an important result concerning the relative entropy that will be repeatedly used and referred to during this manuscript is that it enjoys a variational representation due to Donsker and Varadhan~\cite{donsker1975asymptotic} and Gibbs~\cite{gibbs1902elementary}.

\begin{lemma}[{Donsker and Varadhan lemma \cite[Lemma
2.1]{donsker1975asymptotic} and Gibbs variational principle \cite[Theorem III on Chapter XI]{gibbs1902elementary}, respectively}]
\label{lemma:dv_and_gvp}
Let $\cX$ be a measurable space and $\bQ$ be a probability measure on $\cX$.
\begin{itemize}
	\item Let $\bP$ be also a probability measure on $\cX$ such that $\bP \ll \bQ$ and  $\cG_\bQ$ be the set of all measurable functions on $\cX$ such that $\bE_{x \sim \bQ} \Big[e^{g(x)} \Big] < \infty$. Then,
	\begin{equation}
	\label{eq:dv}
		\relent(\bP \Vert \bQ) = \sup_{g \in \cG_\bQ} \Big\{ \bE_{x \sim \bP} \big[ g(x) \big] - \log \bE_{x \sim \bQ} \Big[e^{g(x)} \Big] \Big\}.
	\end{equation}
	
	\item Let $g$ be a measurable function on $\cX$ such that $\bE_{x \sim \bQ} \Big[e^{g(x)} \Big] < \infty$ and $\cP_{\bQ}(\cX)$ be the set of all probability measures $\bP$ on $\cX$ such that $\bP \ll \bQ$. Then, 
	\begin{equation}
	\label{eq:gvp}
		\log \bE_{x \sim \bQ} \Big[e^{g(x)} \Big] = \sup_{\bP \in \cP_\bQ(\cX)} \Big\{ \bE_{x \sim \bP} \big[ g(x) \big] - \relent(\bP \Vert \bQ) \Big\}.
	\end{equation}
\end{itemize}

\end{lemma}

Consider two fixed random variables $X$ and $Y$ on a measurable space $\cX$ distributed according to $\bP_X = \bP$ and $\bP_{Y}= \bQ$. The Donsker and Varadhan lemma~\citep{donsker1975asymptotic} tells us that the relative entropy of $\bP$ with respect to $\bQ$ is no smaller than the difference of the expected value of $g(X)$ and the CGF of $g(Y)$ for all functions $g$ such that the CGF exists. The Gibbs variational principle~\citep{gibbs1902elementary} is its dual dual formulation. It says that, for a fixed function $g$ and a fixed random variable $X'$ with distribution $\bP_{Y} = \bQ$, the CGF of $g(Y)$ is no smaller than the difference between the expected value of $g(X)$ and the relative entropy of $\bP$ with respect to $\bQ$ for all random variables $X$ distributed according to $\bP_X = \bQ$. As we will see in~\Cref{ch:expected_generalization_error,ch:pac_bayesian_generalization}, these results are very useful to bound the expectation of functions of random variables with complicated distributions when we know that the CGF of those functions with respect to simpler distributions is bounded. We also found these results helpful to characterize the regret of the Thompson sampling in different Markov decision problems~\citep{gouverneur2022information,gouverneur2023thompson,gouverneur2024chained}.

\subsection{Mutual Information}
\label{subsec:mutual_information}

The \emph{mutual information} $\minf(X ; Y)$ between two random objects $X$ and $Y$ represents the information that they contain about each other. More precisely, it is defined as the reduction of uncertainty of one random object when the other is known. Mathematically, it is also equal to the relative entropy of the joint distribution $\bP_{X,Y}$ with respect to the product of the priors $\bP_{X} \otimes \bP_{Y}$. That is, it can be interpreted as the inefficiency of describing the two random objects independently instead of jointly.

\begin{definition}
\label{def:mutual_information}
\looseness=-1 The \emph{mutual information} between two random variables $X$ and $Y$ is 
\begin{equation*}
	\minf(X;Y) \coloneqq \ent(X) - \ent(X | Y) = \relent(\bP_{X,Y} \Vert \bP_X \otimes \bP_Y).
\end{equation*}
\end{definition}

Consider a situation like the one in~\Cref{fig:channel} where a random object $X$ is passed through a channel $\bP_Y^X$ to generate $Y$. The mutual information can also be written as the conditional relative entropy of the channel $\bP_Y^X$ with respect to the output distribution $\bP_Y$ given the input distribution $\bP_X$, i.e. 
\begin{equation*}
	\minf(X;Y) = \relent(\bP_Y^X \Vert \bP_Y | \bP_X).
\end{equation*}
This suggests that the mutual information is the average inefficiency of describing the output of a channel $Y$ without using the input $X$ when it is known. This interpretation coincides with the original definition of uncertainty reduction. However, it is especially useful in communications as the \emph{capacity} of a channel $\rmC(\bP_Y^X)$, or maximum rate rate at which we can send information over the channel and recover it at the output with a vanishingly low probability of error, is exactly the supremum over all input distributions of the mutual information, that is, $\rmC(\bP_Y^X) = \sup_{\bP_X} \relent(\bP_Y^X \Vert \bP_Y | \bP_X)$. This interpretation will also be useful to understand some results concerning the generalization error and mutual information later in the manuscript.

It is often useful to consider the mutual information between two random objects when a third one is known. This is quantified with the \emph{disintegrated mutual information}~\cite{negrea2019information} and its expectation, the \emph{conditional mutual information}.

\begin{definition}
\label{def:conditional_minf}
\looseness=-1 The \emph{disintegrated mutual information between $X$ and $Y$ given $Z$} is 
\begin{equation*}
	\minf^Z(X;Y) \coloneqq \ent^Z(X) - \ent^Z(X | Y) = \relent( \bP^Z_{X,Y} \Vert \bP^Z_X \otimes \bP^Z_Y)
\end{equation*}
and the \emph{conditional mutual information between $X$ and $Y$ given $Z$} is 
\begin{equation*}
	\minf(X;Y|Z) \coloneqq \bE \big[ \minf^Z(X;Y) \big] = \ent(X | Z) - \ent(X | Y, Z) =  \relent( \bP^Z_{X,Y} \Vert \bP^Z_X \otimes \bP^Z_Y | \bP_Z) .
\end{equation*}
\end{definition}

Some properties of the mutual information are described below~\cite{Cover2006, polianskyi2022}.

\begin{proposition}[\!\! {\cite[Theorems 3.2, 3.7, and 4.1]{polianskyi2022}}]
\label{prop:properties_minf}
Consider the random objects $X$, $Y$, $Z$, and $X_1, \cdots, X_n$. Then,
\begin{itemize}
	\item The mutual information is symmetric: $\minf(X;Y) = \minf(Y;X)$.
	\item The mutual information is non-negative: $\minf(X;Y) \geq 0$ with equality if and only if $X$ and $Y$ are independent.
	\item More data, more information: $\minf(X; Y,Z) \geq \minf(X;Y)$.
	\item Data processing inequality: If $X$, $Y$, and $Z$ form a Markov chain $X - Y - Z$, then $\minf(X;Y) \geq \minf(X;Z)$ with equality if and only if they also form the Markov chain $Y - Z - X$. Moreover $\minf(X;Y|Z) \leq \minf(X;Y)$ with equality if and only if $X$ is independent of $Z$.
	\item Permutation invariance: $\minf(X;Y) \geq \minf(f(X);g(Y))$, with equality if and only if $f: \cX \to \cA$ and $g: \cY \to \cB$ are one-to-one functions with a measurable inverse.
	\item Chain rule: $\minf(X^n ; Y) = \sum_{i=1}^n \minf(X_i; Y | X^{i-1})$.
	\item Golden formula: $\minf(X;Y) = \inf_{\bQ \in \cP(\cY)} \relent(\bP_Y^X \Vert \bQ | \bP_X)$.
\end{itemize}
\end{proposition}

\subsection{$f$-divergences}
\label{subsec:f-divergences}

The relative entropy $\relent$ is a special case of an \emph{$f$-divergence}, which is a family of dissimilarity measures between probability distributions. They are defined as an expectation of a convex function of the likelihood ratio of the probability densities. 

\begin{definition}
\label{def:f-divergence}
Consider two probability distributions $\bP$ and $\bQ$ on $\cX$ with densities $\sfp$ and $\sfq$. Further consider a convex function $f: \bR_+ \to \bR$ such that $f(1) = 0$ and let $f(0) \coloneqq \lim_{x \to 0^+} f(x)$. Then, the $f$-\emph{divergence} of $\bP$ with respect to $\bQ$ is
\begin{equation*}
	\fdiv(\bP \Vert \bQ) \coloneqq \bE_{x \sim \bQ} \bigg[ f \bigg( \frac{\rmd \bP}{\rmd \bQ}(x) \bigg) \bigg] = \bE_{x \sim \bQ} \bigg[ f \bigg( \frac{\sfp(x)}{\sfq(x)} \bigg) \bigg]
\end{equation*}
if $\bP \ll \bQ$ and $\fdiv(\bP \Vert \bQ) \to \infty$ otherwise.
\end{definition}

Like the relative entropy, there are many common dissimilarity measures that are instances of $f$-divergences. Consider two probability distributions $\bP$ and $\bQ$ on $\cX$ with densities $\sfp$ and $\sfq$. Below, we define some of the ones that will be considered in this manuscript and give some interpretation about them.
\begin{itemize}
	\item  The \emph{total variation} is defined as an $f$-divergence with $f(x) = \frac{1}{2} | x - 1|$, that is,
	\begin{equation*}
		\tv(\bP , \bQ) \coloneqq \frac{1}{2} \bE_{x \sim \bQ} \mleft[ \Big| \frac{\sfp(x)}{\sfq(x)} -1 \Big| \mright] = \sup_{\cA \subseteq \cX} \Big \{ \bP \big[ \cA \big] - \bQ \big[ \cA \big ] \Big\} \in [0,1].
	\end{equation*}
	
	The total variation has a direct interpretation via \emph{binary hypothesis testing}~\cite{polianskyi2022, wainwright2019high}. Consider a random object $X$ with an unknown distribution. We know that $X$ can either be distributed according to $\bP$ or $\bQ$, and therefore its distribution is $\bP_X = \bP \cdot \bP_U[0] + \bQ \cdot \bP_U[1]$, where $U$ is a binary random variable indicating if the distribution of $X$ is $\bP$ (when $U = 0$) or $\bQ$ (when $U = 1$). If we have no prior knowledge of which distribution is more likely we let $\bP_U[0] = \bP_U[1] = \nicefrac{1}{2}$. Then, in binary hypothesis testing one tries to design a decision function $\phi: \cX \to \{0, 1\}$ that accurately predicts the indicating variable $U$ (or hypothesis). The error probability achieved with the best decision function is completely characterized by the total variation
	\begin{equation*}
		\inf_{\phi} \bP \Big[ \phi(X) \neq U \Big] = \frac{1}{2} \Big(1 - \tv(\bP , \bQ) \Big),
	\end{equation*}
that is, the more different are the distributions $\bP$ and $\bQ$, the smaller error probability can be achieved by the best decision function.

	Another interesting property of the total variation is that it is a \emph{distance} between probability distributions. More precisely, it is a \emph{Wasserstein distance} with respect to the discrete metric (as shown in~\Cref{prop:wasserstein_and_total_variation} below).

	\item The \emph{squared Hellinger distance} (or \emph{Hellinger squared divergence}) is defined as an $f$-divergence with $f(x) = (1 - \sqrt{x})^2$, that is,
	\begin{equation*}
		\hel^2(\bP , \bQ) \coloneqq \bE_{x \sim \bQ} \left[ \Bigg( \sqrt{\frac{\sfp(x)}{\sfq(x)}} - 1\Bigg)^2 \right] \in [0,2].
	\end{equation*}
	
	The Hellinger distance, like the total variation, is a distance between probability measures. One of its main interests is that it (weakly) characterizes the total variation (see \cite[Section 7.3]{polianskyi2022}) and that it tensorizes, that is, $\hel^2(\bP , \bQ) = 2 - 2 \prod_{i=1}^n \big(1 - \nicefrac{1}{2} \cdot \hel^2(\bP_i , \bQ_i) \big)$ for all $\bP = \bigotimes_{i=1}^n \bP_i$ and $\bQ = \bigotimes_{i=1}^n \bQ_i$. This makes this divergence very useful for the problem of composite hypothesis testing and therefore a well studied quantity. %
	
	\item The \emph{$\chi^2$-divergence} is defined as an $f$-divergence with $f(x) = (x-1)^2$, that is,
	\begin{equation*}
		\chi^2(\bP \Vert \bQ) \coloneqq  \bE_{x \sim \bQ} \bigg[ \Big( \frac{\sfp(x)}{\sfq(x)} - 1 \Big)^2 \bigg].
	\end{equation*}
	
	Most $f$-divergences locally behave like the $\chi^2$-divergence \cite[Section 7.10]{polianskyi2022}. This is one of the reasons for which this quantity is well understood in the information theory community. 
\end{itemize}

It is often useful to know the relationships between different $f$-divergences. For instance, sometimes we can prove a result with the total variation, but we are really interested in the properties of the relative entropy. In this case, Pinsker's inequality~\cite{pinsker1964information} or the Bretagnolle--Huber's inequality~\cite{bretagnolle1978estimation} are useful.

\begin{lemma}[{\citet{pinsker1964information} with constants from \cite{kullback1967lower, csiszar1967information, kemperman1969optimum}} \!\!]
\label{lemma:pinsker-inequality}
	Let $\bP$ and $\bQ$ be two probability distributions on $\cX$. Then,
	\begin{equation*}
		\tv(\bP , \bQ) \leq \sqrt{\frac{1}{2} \cdot \relent(\bP \Vert \bQ)}.
	\end{equation*}
\end{lemma}

\begin{lemma}[{\citet[Lemma 2.1]{bretagnolle1978estimation}}]
\label{lemma:bretagnolle-huber-inequality}
	Let $\bP$ and $\bQ$ be two probability distributions on $\cX$. Then,
	\begin{equation*}
		\tv(\bP , \bQ) \leq \sqrt{1 - e^{- \relent(\bP \Vert \bQ)}}.
	\end{equation*}
\end{lemma}

These relationships exist between any two $f$-divergences as shown by Harremoës
and Vajda with their joint range technique~\cite{harremoes2011pairs}. The joint range theorem gives the sharpest possible inequalities between any two $f$-divergences, although sometimes these are cumbersome. Some other good relationships involving the total variation, the relative entropy, the Hellinger distance, and the $\chi^2$ divergence are~\cite[Sections 7.3 to 7.6]{polianskyi2022}
\begin{align}
	\label{eq:relation_tv_hel}
	\tv(\bP , \bQ) &\leq \hel(\bP , \bQ) \sqrt{ 1 - \frac{1}{4} \cdot \hel^2(\bP , \bQ)}, \\
	\label{eq:relation_tv_chi}
	\tv(\bP , \bQ) &\leq \frac{1}{2} \sqrt{\chi^2 ( \bP \Vert \bQ)}, \textnormal{ and} \\
    \label{eq:relation_relent_chi}
    \relent(\bP , \bQ) &\leq \log \mleft( 1 + \chi^2 ( \bP \Vert \bQ) \mright) \leq \chi^2 ( \bP \Vert \bQ)
\end{align}

All $f$-divergences share some important properties with the relative entropy, like the non-negativity, the data processing inequality, or the fact that conditioning increases their value~\cite[Theorems 7.4 and 7.5]{polianskyi2022}. However, in this manuscript, we will only employ that they have a variational representation.

\begin{lemma}[{Variational representation of $f$-divergences~\cite[Theorem 7.24]{polianskyi2022}}]
\label{lemma:variational_representation_f_divergences}
	Let $\bP$ and $\bQ$ be two probability distributions on $\cX$ such that $\bP \ll \bQ$, $f_*$ be the convex conjugate of $f$ (see~\Cref{def:convex_conjugate}), and $\cG_{\bQ, f}$ be the set of all measurable functions on $\cX$ such that $\bE_{x \sim \bQ} \big[f_* \circ g (x) \big] < \infty$ . Then,
	\begin{equation*}
		\fdiv(\bP \Vert \bQ) = \sup_{g \in \cG_{\bQ, f}} \Big \{ \bE_{x \sim \bP} \big[ g(x) \big] - \bE_{x \sim \bQ} \big[ f_* \circ g (x) \big] \Big\}.
	\end{equation*}
\end{lemma}

In particular, we will employ the corollaries stating a variational characterization of the $\chi^2$-divergence and the relative entropy.

\begin{corollary}[{Variational representation of the $\chi^2$ divergence}]
\label{cor:variational_representation_chi_2}
Let $\bP$ and $\bQ$ be two probability distributions on $\cX$ such that $\bP \ll \bQ$ and $\cG_\bQ$ be the set of all measurable functions on $\cX$ such that $\bE_{x \sim \bQ} \big[ e^{g(x)}\big] < \infty$. Then,
	\begin{equation*}
		\chi^2(\bP \Vert \bQ) = \sup_{g : \cX \to \bR} \bigg \{ \frac{\big( \bE_{x \sim \bP}[g(x)] - \bE_{x \sim \bQ}[g(x)] \big)^2}{\var_{x \sim \bQ}[g(x)]} \bigg \}.
	\end{equation*}
\end{corollary}

\begin{corollary}[{Variational representation of the relative entropy}]
\label{cor:variational_representation_relative_entropy}
Let $\bP$ and $\bQ$ be two probability distributions on $\cX$ such that $\bP \ll \bQ$. Then, 
	\begin{equation*}
		\relent(\bP \Vert \bQ) = 1 + \sup_{g \in \cG_\bQ} \Big\{ \bE_{x \sim \bP} \big[ g(x) \big] - \bE_{x \sim \bQ} \Big[e^{g(x)} \Big] \Big\}
	\end{equation*}
\end{corollary}

The variational representation of the relative entropy borrowed from $f$-divergences (\Cref{cor:variational_representation_relative_entropy}) is looser than the one by Donsker and Varadhan (\Cref{lemma:dv_and_gvp}).\footnote{A slight modification of the definition of $f$-divergences can recover the Donsker and Varadhan lemma, as shown in \cite[Section 7.13]{polianskyi2022}.} However, it is still useful, as \Cref{cor:variational_representation_relative_entropy} is \emph{still an equivalence} of the relative entropy which sometimes is easier to manipulate as shown in~\Cref{subsec:a_fast_rate_bound_mi}.

\subsection{Rényi Divergence}
\label{subsec:renyi_divergence}

To conclude with this section, we describe another family of dissimilarity measures between probability distributions: the \emph{Rényi divergences of order $\alpha$}. In particular, the relative entropy $\relent$  is the Rényi divergence of order 1.

\begin{definition}
\label{def:renyi_divergence}
Consider two probability distributions $\bP$ and $\bQ$ on $\cX$ with densities $\sfp$ and $\sfq$. If $\bP \ll \bQ$, for every $\alpha \in (0,1) \cup (1,\infty)$, the \emph{Rényi divergence} of order $\alpha$ of $\bP$ with respect to $\bQ$ is
\begin{equation*}
	\renyidiv{\alpha}(\bP \Vert \bQ) \coloneqq \frac{1}{\alpha - 1} \log \bE_{x \sim \bQ} \bigg[ \bigg( \frac{\rmd \bP}{\rmd \bQ}(x) \bigg)^{\alpha} \bigg] = \frac{1}{\alpha - 1} \log \bE_{x \sim \bQ} \bigg[ \bigg( \frac{\sfp (x)}{\sfq(x)} \bigg)^{\alpha} \bigg].
\end{equation*}
For the orders $\alpha=0$, $\alpha = 1$, and $\alpha \to \infty$, the definition is obtained by continuity. To be precise, if $\alpha = 0$, then
\begin{equation*}
	\renyidiv{0}(\bP \Vert \bQ) \coloneqq - \log \bQ \big[ \supp(\bP) \big] = - \log \bQ \big[ \{ x \in \cX:  \sfp(x) > 0 \}\big];
\end{equation*}
if $\alpha = 1$, then by continuity it is equivalent to the relative entropy, that is, $\renyidiv{1}(\bP \Vert \bQ) \coloneqq \relent(\bP \Vert \bQ)$; and if $\alpha \to \infty$, then
\begin{equation*}
	\renyidiv{\infty} (\bP \Vert \bQ) \coloneqq \log \Big( \esssup_{\bP} \frac{\rmd \bP}{\rmd \bQ}   \Big) = \log \Big( \esssup_{\bP} \frac{\rmd \sfp}{\rmd \sfq}   \Big) .
\end{equation*}
If $\bP \not \ll \bQ$, then $\renyidiv{\alpha}(\bP \Vert \bQ) \to \infty$ for all $\alpha \geq 0$.
\end{definition}

\sloppy The Rényi divergences also share many valuable properties with the $f$-divergences, like the non-negativity and the data processing inequality~\cite[Theorems 8 and 9]{van2014renyi}. Moreover, it is non-decreasing with respect to the order $\alpha$, that is, $\renyidiv{\alpha} \leq \renyidiv{\alpha'}$ for all $0 \leq \alpha \leq \alpha'$.

The Rényi divergences have several interpretations. For example, let $X$ be a random object with some private information we wish to conceal. Consider a mechanism $\bM$ that processes the input $X$ and generates a sanitized version $Y$, and that can be described by the Markov kernel $\bP_Y^X$. This kind of mechanisms are referred to as \emph{privacy mechanisms}. The essential supremum of the Rényi divergence of order $\infty$ of the mechanism's channel with respect to the marginal distribution of the sanitized output is equivalent to the \emph{maximal leakage} $\cL(X \to Y)$ of $X$ into $Y$, that is
\begin{equation*}
	\cL(X \to Y) \coloneqq \esssup_{\bP_X} \renyidiv{\infty}(\bP_Y^X \Vert \bP_Y).
\end{equation*}
The maximal leakage describes the largest amount of information that the sanitized output $Y$ leaks about $X$ to adversaries seeking to guess arbitrary (possibly randomized) functions of $X$, or equivalently, aiming to maximize arbitrary gain functions~\cite{issa2020operational}. Similarly, a version describing the leakage of every \emph{realization} $y$ of the mechanism is the \emph{pointwise maximal leakage}, which is equal to the Rényi divergence of order $\infty$ of the Bayesian posterior of the input $\bP_X^{Y=y}$ with respect to its real distribution $\bP_X$, that is
\begin{equation*}
	\ell(X \to y) \coloneqq \renyidiv{\infty}(\bP_X^{Y=y} \Vert \bP_X).
\end{equation*}
The pointwise maximal leakage describes the largest amount of information that a realization $y$ leaks about $X$ to adversaries
seeking to guess arbitrary (possibly randomized) functions of $X$, or equivalently, aiming to maximize arbitrary gain functions~\cite{saeidian2023apointwise,saeidian2023bpointwise}.

Now consider that $X$ is a dataset with $n$ entries, each corresponding to a person, that is, $\cX = (\cX')^n$ for some space $\cX'$. Further consider that the private information in $X$ is the membership of each person to the dataset. The Rényi divergences are also related to the amount of information that an adversary can obtain about each individual as described by the \emph{differential privacy} framework~\cite{dwork2006calibrating, dwork2014algorithmic}. More precisely, an algorithm is $\varepsilon$-differentially private if the Rényi divergence of order infinity of the output of the mechanism with the dataset $x$ with respect to the output with the dataset $x'$ is smaller than $\varepsilon$, that is, $\renyidiv{\infty}(\bP_Y^{X=x} \Vert \bP_Y^{X=x'}) \leq \varepsilon$ for every two datasets $x,x' \in \cX$ differing only in one element. The $\varepsilon$ parameter quantifies the amount of information obtainable about each individual~\cite{balle2020hypothesis}. Furthermore, this definition can be extended to Rényi divergences with other orders to simplify the practical calculations of the mechanisms' guarantees~\cite{dwork2016concentrated, mironov2017renyi, bun2016concentrated, balle2020hypothesis}. 

From an information-theoretic perspective, the Rényi divergence of order $\infty$ has a similar interpretation to the relative entropy, but in the worst case sense. More precisely, if we constructed a code for a discrete random variable $X$ assuming it has a distribution $\bQ$, when in reality it has a distribution $\bP_X = \bP$, then we would need at least $\renyidiv{\infty}(\bP \Vert \bQ)$ more bits to describe the said random variable than if we would have considered the real distribution $\bP$~\cite[Chapter 6]{grunwald2007minimum}.

\section{Miscellaneous}
\label{sec:miscellaneous}

This section compiles the treatment of miscellaneous topics that will be useful to understand the rest of the manuscript. These topics belong to fields that do not directly intersect with the development of the text and do not require a separate section. More precisely, an understanding of the convex conjugate and the Wasserstein distance is necessary to follow some of the arguments henceforth, but a primer on convex (or even functional) optimization or on optimal transport would be excessive for this manuscript. 

\subsection{Convex Conjugate}
\label{subsec:convex_conjugate}

The convex conjugate is an important concept in the fields of convex (and functional) analysis and in optimization theory. As mentioned above, we will not cover these fields here but the interested reader is referred to the books~\cite{boyd2004convex, hiriart2004fundamentals}.

\begin{definition}
\label{def:convex_conjugate}
The \emph{convex conjugate} (or just \emph{conjugate} or \emph{Fenchel--Legendre}'s dual) of a function $f : \bR \to \bR$ is defined as
\begin{equation}
	\label{eq:convex_conjugate}
	f_*(y) \coloneqq \sup_{x \in \mathrm{dom}(f)} \big\{ x y - f(x) \big\}.
\end{equation}
\end{definition}

The convex conjugate $f_*$ of a function $f$ is convex, regardless if $f$ is convex or not. The \emph{Fenchel--Young inequality} follows immediately from the definition, and it states that for all $x, y\in \bR$,
\begin{equation*}
	f_*(y) + f(x) \geq xy.
\end{equation*}

When the function $f$ is convex and differentiable, the convex conjugate is also known as the \emph{Legendre's transform}.\footnote{There are different terminologies for this restriction: some state that it restricts to differentiable functions~\cite[Section 3.2.2]{boyd2004convex}, some to differentiable functions with an invertible derivative~\cite[Remark 0.1]{hiriart2004fundamentals}} %
More specifically, when the function $f$ represents or dominates a CGF, then the convex conjugate is known as the \emph{Cramér's transform}. In this manuscript, we will primarily deal with the inverse of the convex conjugate of CGFs. A particularly useful result in this scenario states an expression of the inverse of the convex conjugate of a smooth convex function. This result is often used to obtain the classical \emph{Chernoff's inequality} via the \emph{Cramér--Chernoff} method~\cite[Sections 2.2. and 2.3]{boucheron2003concentration}.

\begin{lemma}[\!\!{\cite[Lemma 2.4]{boucheron2003concentration}}]
\label{lemma:boucheron_convex_conjugate_inverse}
Let $f$ be a convex and continuously differentiable function defined on $[0,b)$, where $0 < b \leq \infty$. Assume that $f(0) = f'(0) = 0$. Then, the convex conjugate $f_*$ is a non-negative, convex, and non-decreasing function on $[0, \infty)$. Moreover, for every $z \geq 0$, the set $\{ y \geq 0: f_*(y) > z \}$ is non-empty and the generalized inverse of $f_*$, defined as $f_*^{-1}(z) \coloneqq \inf \{ y \geq 0 : f_*(y) > z \}$ is concave and can be also written as
\begin{equation*}
	f_*^{-1}(z) = \inf_{x \in (0,b)} \bigg \{ \frac{z + f(x)}{x} \bigg\}.
\end{equation*}
\end{lemma}

\subsection{Wasserstein Distance}
\label{subsec:wasserstein_distance}

The Wasserstein distance was introduced by Kantorovich as a generalization of the ``déblais'' and ``remblais'' problem from Gaspard Monge. This problem assumes that there is an amount of soil to be extracted from the ground (``déblai'') and transported to places where it should be incorporated in a construction (``remblai''). The soil's extraction and placement locations were known from the start. The goal was to find the optimal assignment of each piece of soil from an extraction $x$ to a placement $y$ location (or \emph{optimal transport}) taking into account that the transport is costly. More precisely, he considered that the cost to transport a mass $m$ for a distance $d(x,y)$ amounted to $d(x,y) \ m$. 

In this generalization from Kantorovich, the positions of the ``deblai'' and ``remblai'' are described by two distributions $\bP$ and $\bQ$ on a Polish space $(\cX, \rho)$. Larger values of the distribution on a subset of $\cX$ represent larger amounts of soil in a position. The cost to transport pieces of soil from an extraction $x$ to a placement $y$ position is represented by the separation of the positions given by the metric $\rho$ and the density of these elements with respect to a coupling $\mu$ between the ``déblai'' $\bP$ and ``remblai'' $\bQ$ distributions, that is, $\rho(x,y) \rmd \mu(x,y)$. The field of optimal transport is wide and well-studied, and a further discussion about it is outside of the scope of this manuscript. For a further introduction to the field and its history, we kindly refer the reader to the book~\cite{villani2009optimal}.

\begin{definition}
\label{def:wasserstein_distance}
Let $(\cX,\rho)$ be a Polish space and let $p \in [1, \infty)$. Then, the \emph{Wasserstein distance} of order $p$ with respect to the metric $\rho$ between two probability distributions $\bP$ and $\bQ$ on $\cX$ is 
\begin{equation*}
	\bW_{p,\rho}(\bP, \bQ) \coloneqq \Big( \inf_{\mu \in \Pi(\bP, \bQ)} \bE_{(x,y) \sim \mu} \big[ \rho(x,y)^p \big] \Big)^{\frac{1}{p}}.
\end{equation*}
\end{definition}

Hölder's inequality implies that the Wasserstein distance is non-decreasing with respect to the order, that is, $\bW_{p,\rho} \leq \bW_{q, \rho}$ for all $p \leq q$ \cite[Remark 6.6]{villani2009optimal}. In this manuscript, we mainly use the Wasserstein distance of order 1, also known as the \emph{Kantorovich--Rubinstein distance}. Therefore, in the interest of brevity, we define $\bW_\rho \coloneqq \bW_{1,\rho}$. To be precise, the Wasserstein distance is not a distance in $\cP(\cX)$ as there are distributions $\bP, \bQ \in \cP(\cX)$ such that $\bW_\rho(\bP, \bQ) \to \infty$. This is solved by working on the Wasserstein space
\begin{equation*}
	\cP_{\rho}(\cX) \coloneqq \{ \bP \in \cP(\cX) : \bE_{x \sim \cP} \big[ \rho(x,y) \big] < \infty \textnormal{ for all } y \in \cX \}.
\end{equation*}

Similarly to the relative entropy and other $f$-divergences, the Wasserstein distance also enjoys a variational representation that will be employed later in the manuscript. This result comes from the \emph{Kantorovich--Rubinstein} duality, which states that the minimization over couplings in the Wasserstein distance definition is equivalent to a maximization over \emph{Lipschitz continuous} functions.
Intuitively, a function $f$ cannot change rapidly if it is Lipschitz continuous. That is, the distance between the images $f(x)$ and $f(y)$ of two points $x$ and $y$ is bounded by how separated these points are with respect to the metric $\rho$.

\begin{definition}
\label{def:lipschitz}
A function $f: \cX \to \bR$ is \emph{$L$-Lipschitz continuous} with respect to the metric $\rho$, or simply $f \in L\textnormal{-}\mathrm{Lip}(\rho)$, if $|f(x) - f(y)| \leq L \rho(x,y)$ for all $x, y \in \cX$.
\end{definition}

\begin{lemma}[{Kantorovich--Rubinstein duality~\cite[Remark 6.5]{villani2009optimal}}]
\label{lemma:kantorovich_rubinstein_duality}
Let $\bP$ and $\bQ$ be two probability distributions in $\cP_\rho(\cX)$. Then,
\begin{equation}
\label{eq:kantorovich_rubinstein_duality}
	\bW_\rho(\bP, \bQ) \coloneqq \sup_{f \in 1\textnormal{-}\mathrm{Lip}(\rho)} \Big \{ \bE_{x \sim \bP} \big[ f(x) \big] - \bE_{x \sim \bQ} \big[ f(x) \big] \Big\}.
\end{equation}
\end{lemma}

This lemma can also be written considering two random variables $X$ and $Y$ distributed according to $\bP$ and $\bQ$ in $\cP_\rho(\cX)$. Then,
\begin{equation}
\label{eq:kantorovich_rubinstein_duality_rv}
	\bW_\rho(\bP, \bQ) \coloneqq \sup_{f \in 1\textnormal{-}\mathrm{Lip}(\rho)} \Big \{ \bE \big[ f(X) \big] - \bE \big[ f(Y) \big] \Big\}.
\end{equation}

From~\Cref{lemma:kantorovich_rubinstein_duality} we can see that the total variation is a Wasserstein distance with respect to the discrete metric $\rho_\rmH(x,y) = \bI_{\{x \neq y\}} (x,y)$. Only bounded functions with range 1 are $1$-Lipschitz with respect to the discrete metric, and the function that maximizes the variational representation~\eqref{eq:kantorovich_rubinstein_duality} is $\bI_\cA$ for some subset $\cA \subseteq \cX$, which recovers the deviation of the total variation. 
In fact, the Wasserstein distance is generally dominated by the total variation, as shown below.

\begin{proposition}[{\cite[Theorem 6.15 and discussion thereafter]{villani2009optimal}}]
\label{prop:wasserstein_and_total_variation}
The Wasserstein distance of order 1 is dominated by the total variation. That is, if $\bP$ and $\bQ$ are two distributions on $\cX$, then $\bW_\rho(\bP, \bQ) \leq \diam_\rho(\cX) \tv(\bP , \bQ)$, where $\diam_\rho(\cX) \coloneqq \sup \{ \rho(x,y) : x,y \in \cX \} $ is the diameter of $\cX$. In particular, the inequality holds with equality with respect to the discrete metric $\rho_\rmH$, that is, $\bW_{\rho_\rmH}(\bP, \bQ) = \tv(\bP , \bQ)$.
\end{proposition}

With this relationship between the Wasserstein distance and the total variation, we can recover another probabilistic interpretation of the total variation.
Consider two distributions $\bP$ and $\bQ$ and their couplings $\Pi(\bP, \bQ)$. The total variation is the smallest probability that elements drawn from the joint distribution of a coupling are different. That is, 
\begin{equation*}
	\tv(\bP, \bQ) = \inf_{\bP_{X,Y} \in \Pi(\bP, \bQ)} \bP_{X,Y} \big[ X \neq Y \big].
\end{equation*}

Finally, we conclude with an extension of the Bobkov--Götze's theorem~\citep[Theorem~4.8]{van2014probability} relating the Wasserstein distance with the relative entropy. This result will be useful later in~\Cref{ch:expected_generalization_error}.

\begin{lemma}[{Extension of the Bobkov--Götze's theorem~\citep[Theorem~4.8]{van2014probability}}]
\label{lemma:bobkov_gotze}
    Consider a Polish space $(\cX, \rho)$ and a probability distribution $\bP$ on $\cX$ with a finite first moment. Let $X$ be a random object distributed according to $\bP$. Then, the following two are equivalent:
    \begin{enumerate}
        \item \emph{For every} 1-Lipschitz function $f$, the CGF $\Lambda_{f(X)}(\lambda)$ is bounded from above by a function $\psi(\lambda) < \infty$ for all $\lambda \in [0,b)$ and some $b \in \bR_+$, where the function $\psi$ is convex, continuously differentiable, and $\psi(0) = \psi'(0) = 0$.
        \item $\bW_\rho(\bQ, \bP) \leq \psi_*^{-1} \mleft( \relent(\bQ \Vert \bQ )\mright)$ for all distributions $\bQ$ on $\cX$.    
    \end{enumerate}
\end{lemma}

This extension follows in a straightforward manner from the proof given in~\citep[Section 4.1]{van2014probability}. The first property of the lemma states that for all $\lambda \in [0,b)$ and all $f \in \mathrm{1-Lip}(\rho)$
\begin{equation*}
    \log \bE \mleft[ e^{\lambda(X - \bE[X])} \mright] \leq \psi(\lambda) .
\end{equation*}

From the Gibbs variational principle of~\Cref{lemma:dv_and_gvp}, this is equivalent to 
\begin{equation*}
    \sup_{\lambda \in [0,b]} \sup_{f \in \mathrm{1-Lip}(\rho)} \sup_{\bQ \in \bP_\cP(\cX)} \mleft \{ \lambda \big( \bE[f(X)] - \bE_{x \sim \bQ} [f(x)] \big) - \relent(\bQ \Vert \bP) - \psi(\lambda)  \mright\} \leq 0.
\end{equation*}

Exchanging the order of the suprema and evaluating the suprema over $\lambda$ and $f$ explicitly using the definition of the convex conjugate (\Cref{def:convex_conjugate}) and the Kantorovich--Rubinstein duality of~\Cref{lemma:kantorovich_rubinstein_duality} respectively yields that the above expression is equivalent to
\begin{equation*}
    \sup_{\bQ \in \bP_\cP(\cX)} \mleft \{ \psi_* \mleft( \bW_\rho(\bP, \bQ) \mright) - \relent(\bQ \Vert \bP) \mright \} \leq 0 ,
\end{equation*}
which is a reformulation of the second property in the lemma noting that $\psi_*^{-1}$ is the generalized inverse of $\psi_*$.

%% file: chapters/primer_generalization_theory.tex
In this chapter, we introduce the notation and definitions related to generalization theory. We also put into context the framework of information-theoretic generalization with respect to other frameworks and motivate its adoption. 

In \Cref{sec:formal_model_learning}, we start describing the notation commonly employed in learning theory. After that, in \Cref{sec:pac_learning,sec:rademacher_complexity}, we review classical approaches to generalization based on the study of the complexity of the hypothesis class. These sections include results based on the uniform convergence (or Glivenko--Cantelli) property of the class, and its Rademacher complexity. These frameworks provide the whole hypothesis class with generalization guarantees, which often is a strong requirement. This leads us to the introduction of the minimum description length (MDL) principle and the associated parsimonious philosophical postulate of the Occam's razor in \Cref{sec:mdl_and_occams_razor}. 

The MDL principle provides guarantees for every hypothesis in a finite class, but these guarantees do not explicitly take into account the algorithm that returns the said hypothesis. Hence, we discuss algorithmic stability next in \Cref{sec:algorithmic_stability}, which gives guarantees for algorithms with a certain stability property. In particular, we first discuss the uniform stability framework, as it is very successful in providing deterministic algorithms with generalization guarantees. After that, we also discuss how privacy can be interpreted as a stability property of an algorithm, and therefore be linked to generalization. We conclude the section by discussing how one can employ different information measures to define an algorithm's stability. In this last subsection, we note that the ``privacy as stability'' interpretation can also be formalized in terms of information measures and how this particular way of understanding stability gives guarantees to particular algorithms and particular classes of data distributions. 

Equipped with this knowledge, in \Cref{sec:information_theoretic_generalization}, we describe the information-theoretic framework for generalization. This framework can be seen as encompassing the framework of algorithmic stability with information measures and, therefore, also takes into account both the algorithm and the data distribution. Therefore, the framework relaxes the requirements for the guarantees, and they do not need to hold for every hypothesis in a class and every data distribution. This allows the information-theoretic framework for generalization to yield tighter guarantees than previous frameworks, although at the price of being more individualized, and it motivates its adoption in this manuscript. We conclude this section with a description of the different kinds of information-theoretic generalization guarantees based on their specificity level: either guarantees in expectation or in probability (both PAC-Bayesian or single-draw PAC-Bayesian). 

Finally, in~\Cref{sec:parameterized_models}, the chapter concludes with a justification of why, for parameterized models, one can express the results based on the parameters of the models instead of their resulting hypothesis. In particular, this justification is important for this manuscript as the results hereafter will be presented in this form.

The intention of this chapter is to give a brief introduction to the field of generalization theory and the notation that we will employ to refer to the topics within the field. We also wanted to discuss the different theoretical frameworks to study generalization, along with their trade-offs, in order to clarify our reasoning to study the information-theoretic generalization framework and how we position it with respect to other frameworks. Experts on the topic may choose to skip this chapter. More novice readers can benefit from the digestion of the concepts and the intuitions given to better understand the following chapters and to get a taste of the field and the relevant literature before delving deeper into it.

\section{A Formal Model of Learning}
\label{sec:formal_model_learning}

Consider a problem we want to solve and a set of hypotheses $\cH$ to solve the problem. The \emph{instances} (or \emph{examples}) $z$ of the problem lie in a space $\cZ$ and are distributed according to a distribution $\bP_Z$. The performance of a hypothesis $h \in \cH$ on an instance $z \in \cZ$ can be evaluated with a \emph{loss function} $\ell: \cH \times \cZ \to \bR_+$. Larger values of the loss function denote a worse performance and a value of zero means perfect performance. To solve the problem, we want to find a hypothesis $h$ that minimizes the \emph{population risk} $\ccR(h)$. The population risk is defined as the expected value of the loss of the hypothesis $h$ on instances of the problem, namely
\begin{equation*}
	\poprisk(h) \coloneqq \bE \big[ \ell(h, Z) \big].
\end{equation*}
For example, a common type of problems are \emph{supervised learning} problems. Here, the instances $z = (x,y)$ are a tuple formed by a \emph{feature} and a \emph{label} (or \emph{target}), and the hypotheses $h : \cX \to \cY$ are functions that return a label when given a feature.  Then, the loss functions are often of the type $\ell(h, z) = \rho(h(x), y)$, where $\rho$ is a measure of dissimilarity between the predicted $h(x)$ and true $y$ label associated to the feature $x$. Traditionally, a large volume of the theory has focused on \emph{classification tasks} where the labels' set consists of $k$ classes $\cY = [k]$ and the loss function is the 0--1 loss $\ell(h, z) = \bI_{\cE_h}(x,y)$, where $\cE_h = \{ (x,y) \in \cX \times \cY : h(x) \neq y \}$.

In a \emph{learning problem}, we assume that we have access to a sequence of $n$ data instances $s \coloneqq (z_1, \cdots, z_n) \in \cZ^n$, or \emph{training set}. Usually, we assume that these instances are all independent and identically distributed (i.i.d.); that is, they are sampled from $\bP_Z^{\otimes n}$. Unless we explicitly state it otherwise, this will also be assumed throughout the manuscript. Then, a \emph{learning algorithm} $\bA$ is a (possibly randomized) mechanism that generates a hypothesis $H$ of the solution of a problem given a training set $s$. The algorithm $\bA$ is characterized by the Markov kernel $\bP_H^S$ that, for a given training dataset $s$, returns a distribution on the hypothesis space $\bP_H^{S=s}$. An algorithm \emph{generalizes} if the population risk of the hypothesis that it generates is small. %
Then, a quantity of interest is the \emph{excess risk} of a hypothesis $h$, which is defined as the difference between its population risk and the smallest population risk of hypothesis in that class, that is
\begin{equation*}
	\excess(h, \cH) \coloneqq \poprisk(h) - \inf_{h^\star \in \cH} \poprisk(h^\star).
\end{equation*} 
Therefore, small values of $\excess(h, \cH)$ guarantee small values of the population risk of a hypothesis $h$ relative to the optimal risk achievable in that class.

Often, we do not have full knowledge of the distribution $\bP_Z$, so calculating the population risk or the excess risk is not feasible. However, a good proxy of the population risk is the \emph{empirical risk} $\emprisk(h, s)$, which is defined as the average loss of a hypothesis $h$ on the samples from the training set $s$, namely
\begin{equation*}
	\emprisk(h,s) \coloneqq \frac{1}{n} \sum_{i=1}^n \ell(h,z_i).
\end{equation*}
Indeed, for every \emph{fixed} hypothesis $h$, the empirical risk is an unbiased estimator of the population risk, that is, $\bE \big[ \emprisk(h,S) \big] = \poprisk(h)$. For this reason, many learning algorithms attempt to return a hypothesis minimizing the empirical risk, also known as performing \emph{empirical risk minimization} (ERM). This is still a complicated task that only results in a low population risk if the difference between the population and empirical risks is small. Consequently, many generalization bounds, or bounds on the population risk, are obtained by bounding the \emph{generalization error} (or \emph{generalization gap}) 
\begin{equation*}
	\gen(h, s) \coloneqq \poprisk(h) - \emprisk(h,s).
\end{equation*}
More precisely, consider the decomposition
\begin{equation}
\label{eq:gen_decomposition}
	\poprisk(h) %
	=  \emprisk(h,s) + \gen(h,s).
\end{equation}
Then, a bound on the generalization error $\gen(h,s)$ directly gives a computable bound on the population risk. Hence, if a hypothesis $h$ has a small empirical risk and a small generalization error, we can say that the hypothesis $h$ generalizes.

The generalization error is also useful to bound the excess risk from above. Note that the excess risk can be decomposed as
\begin{align}
\label{eq:excess_risk_decomposition}
	&\excess(h, \cH) = \nonumber \\
	& \quad \gen(h, s) + \big( \underbrace{\emprisk(h,s) - \inf_{\tilde{h}_s \in \cH} \emprisk(\tilde{h}_s, s)}_{\textnormal{optimization error}} \big) + \big( \underbrace{\inf_{\tilde{h}_s \in \cH} \emprisk(\tilde{h}_s, s) - \inf_{h^\star \in \cH} \poprisk(h^\star)}_{\textnormal{approximation error}}\big).
\end{align}
Generally, the approximation error is non-positive as the ERM achieves a better empirical risk than the smallest possible population risk. Therefore, if the optimization error can be sufficiently controlled, bounding the generalization error is sufficient to bound the excess risk.

\section{Probably Approximately Correct Learning}
\label{sec:pac_learning}

The \emph{Probably Approximately Correct} (PAC) framework from \citet{valiant1984theory} was originally formulated for the 0--1 loss for classification tasks in the \emph{realizable setting}. This setting assumes that there exists a hypothesis $h^\star \in \cH$ such that $\poprisk(h^\star) = 0$. In this setting, every ERM $h_{\mathrm{ERM}}$ achieves zero empirical risk, that is, $\emprisk(h_{\mathrm{ERM}},S) = 0$ almost surely.

Informally, the framework states that a hypothesis class $\cH$ is \emph{PAC learnable} if, for all $\alpha, \beta \in (0,1)$, there exists a function $n_\cH : (0,1)^2 \to \bN$ and a learning algorithm $\bA$ such that, after observing $n \geq n_\cH(\alpha, \beta)$ samples, it returns a hypothesis that is \emph{approximately correct} (has a population risk smaller than $\alpha$) with \emph{probability} at least $1 - \beta$. If such an algorithm exists, then it is called a \emph{PAC learning algorithm} for $\cH$. PAC learnability can be extended outside the realizable setting and to general loss functions~\cite[Chapter 3]{shalev2014understanding}. %

\begin{definition}
\label{def:agnostic_pac_learnable}
A hypothesis class $\cH$ is \emph{agnostic PAC learnable} if, for all $\alpha \in \bR_+$ and $\beta \in (0,1)$, and for every distribution $\bP_Z$, there exists a function $n_\cH: \bR_+ \times (0,1) \to \bN$ and a learning algorithm $\bA$ such that, after observing $n \geq n_\cH(\alpha, \beta)$ samples from $\bP_Z$, it returns a hypothesis $h$ such that, with probability at least $1 - \beta$, 
\begin{equation*}
	\excess(h,\cH) \leq \alpha.
\end{equation*}
\end{definition}

The function $n_\cH: \bR_+ \times (0,1) \to \bN$ is called the \emph{sample complexity} of learning $\cH$ and determines how many examples are required to guarantee a PAC solution. To be precise, for an agnostic PAC learnable hypothesis class $\cH$, there are multiple functions $n_\cH : \bR_+ \times (0,1)$ that satisfy the requirements given in \Cref{def:agnostic_pac_learnable}. %
Hence, the sample complexity generally refers to the minimal of these functions. 

Generally, PAC learning theory focuses on the study and bounding from above of the sample complexity of hypothesis classes.  However, we often have access to a fixed, finite number of samples $n$ and want to understand the amount of error that our algorithm will suffer. Therefore, in order to have a better comparison between the generalization guarantees of the different frameworks, instead of focusing on the sample complexity, we will study a complementary concept: the \emph{error complexity} $\alpha_\cH : \bN \times (0,1) \to \bR_+$, which is defined as the generalized inverse of the \emph{sample complexity} for a fixed confidence $\beta$, that is 
\begin{equation*}
	\alpha_\cH(n, \beta) = \inf \{ \alpha \in \bR_+ : n_\cH(\alpha, \beta) \leq n \}.
\end{equation*}

In this way, the generalization guarantees %
from the PAC learning framework will be formulated as follows: for all $n \in \bN$ and $\beta \in (0,1)$, there exists a function $\alpha_\cH: \bN \times (0,1) \to \bN$ and a learning algorithm $\bA$ such that, after observing $n$ samples, it returns a hypothesis $h$ such that, with probability at least $1 - \beta$, 
\begin{equation*}
	\excess(h,\cH) \leq \alpha_\cH(n,\beta).
\end{equation*} 
That is, the error complexity is an upper bound on the excess risk. If the \emph{error complexity} $\alpha_\cH(n,\beta)$ is non-increasing in $n$ and $\lim_{n \to \infty} \alpha_\cH(n,\beta) = 0$ for all $\beta \in (0,1)$, then the hypothesis class $\cH$ is agnostic PAC learnable. 

\subsection{Uniform Convergence}
\label{subsec:uniform_convergence}

Above, we saw that for a hypothesis class $\cH$ to be agnostic PAC learnable, there needs to exist an algorithm such that, for every data distribution $\bP_Z$, it attains an arbitrarily small population risk when given a sufficiently large number of instances.

Similarly, the \emph{uniform convergence property} of a hypothesis class $\cH$ states that, for every data distribution $\bP_Z$, the absolute generalization error  of every hypothesis in the class $\cH$ is arbitrarily small for a sufficiently large number of instances~\cite[Chapter 4]{shalev2014understanding}. 

\begin{definition}
\label{def:uniform_convergence}
A hypothesis class $\cH$ has the \emph{uniform convergence property} (or is a \emph{uniformly Glivenko--Cantelli class}~\cite{dudley1991uniform}) if, for all $n \in \bN$ and $\beta \in (0,1)$, for every distribution $\bP_Z$, for every dataset $S$ of $n$ instances sampled from $\bP_Z$, and for every hypothesis $h \in \cH$, there exists a function $\alpha_\mathrm{UC} : \bN \times (0,1) \to \bR_+$ such that $\lim_{n \to \infty} \alpha_\mathrm{UC}(n, \beta) = 0$ and that, with probability at least $1 - \beta$,
\begin{equation*}
	\big| \gen(h,S) \big| \leq \alpha_\mathrm{UC}(n, \beta).
\end{equation*}
\end{definition}

Therefore, based on the decomposition~\eqref{eq:gen_decomposition}, the uniform convergence property guarantees that the population risk will be close to the observed empirical risk regardless of the learning algorithm used and the underlying data distribution. %
Moreover, based on the decomposition~\eqref{eq:excess_risk_decomposition}, controlling the optimization error is sufficient to ensure that the uniform convergence property also gives guarantees on the excess risk.
In fact, if a hypothesis class $\cH$ has the uniform convergence property, then the class is agnostic PAC learnable and the ERM algorithm is an agnostic PAC learner for $\cH$~\cite[Corollary 4.4]{shalev2014understanding}. Moreover, for binary classification problems, if a hypothesis class is agnostic PAC learnable, then it has the uniform convergence property~\cite[Theorem 6.7]{shalev2014understanding}, although this is not true for more general problems~\cite[Section 4]{shalev2010learnability}. 

For instance, if the loss function has a range contained in $[a,b]$ and the hypothesis class is finite $|\cH| < \infty$, then it can be shown that $\cH$ has the uniform convergence property by the union bound and Hoeffding's inequality~\cite[Chapter 4]{shalev2014understanding}. More precisely, in this case
\begin{equation}
\label{eq:uniform_convergence_finite_hypothesis}
	\alpha_\mathrm{UC}(n,\beta) \leq (b-a) \sqrt{\frac{2 \log \frac{2 |\cH|}{\beta}}{n}}.
\end{equation}

Practically, this can be extended to real valued parameterized hypothesis classes. Namely, if the class is parameterized with $d$ parameters and the computer uses a 64 bit precision, the set of hypotheses is upper bounded by $2^{64 d}$ and $\log |\cH| \leq 64 d \log 2$. Unfortunately, the requirements for this property are too strict for current learning algorithms such as parameterized, differentiable networks, where the number of parameters $d$ is very large~\cite{zhang2021understanding, nagarajan2019uniform}. However, it seems that for interpolating algorithms, that is, algorithms that output a hypothesis $h$ such that $\emprisk(h,S) = 0$ a.s., it is still possible to avoid the problems outlined by \citet{nagarajan2019uniform}. The idea is to consider a surrogate hypothesis $\tilde{h}$ that belongs to a \emph{structural Glivenko--Cantelli hypothesis class} (a weaker notion of uniform convergence for sequences of learning problems) and study its deviation from from the original hypothesis~\cite{negrea2020defense}. 

\subsection{Vapnik--Chervonenkis and Natarajan Dimensions}
\label{subsec:vc_and_natarajan_dimensions}

In the previous subsection, we established that finite hypothesis classes $\cH$ are agnostic PAC learnable and that their generalization gap is in $\cO \big( \sqrt{\nicefrac{1}{n} \cdot \log \nicefrac{|\cH|}{\beta}} \big)$ with probability at least $1 - \beta$. However, we know that there are infinite hypothesis classes that also have a vanishing generalization gap. The \emph{Vapnik--Chervonenkis} (VC) and the \emph{Natarajan} dimensions give a sharp characterization of this kind of hypothesis classes for classification problems.

Let us focus first on the binary classification setting. In this setting, a hypothesis corresponds to a function $h$ from $\cX$ to $\cY = \{0,1\}$, where we recall that $\cZ = \cX \times \cY$. The maximum number of distinct ways of classifying $n$ instances using hypotheses in a hypothesis class $\cH$ is known as the \emph{growth function} $\growth_\cH(n)$ of that class, and provides us with a measure of the richness of the class $\cH$.

\begin{definition}
\label{def:growth_function}
The \emph{growth function} $\growth_\cH : \bN \to \bN$ of a hypothesis class $\cH$ is 
\begin{equation*}
	\growth_\cH(n) \coloneqq \max_{\{ x_1, \ldots, x_n \} \in \cX^n} \Big| \big\{ \big( h(x_1), \ldots, h(x_n) \big) : h \in \cH \big\} \Big|.
\end{equation*}
\end{definition}

The growth function $\Pi_\cH(n)$ of a hypothesis class $\cH$ can be used to bound the absolute generalization error and, hence, establish a uniform convergence property of the class~\cite[Section 3.2]{mohri2018foundations}. More precisely, if a hypothesis class $\cH$ has a growth function $\Pi_\cH(n)$, then
\begin{equation}
\label{eq:generalization_growth_function}
	\alpha_\mathrm{UC}(n, \beta) = \sqrt{\frac{8 \log \frac{4 \Pi_\cH(2n)}{\beta}}{n}}.
\end{equation} 

However, calculating the growth function is cumbersome, as it requires the calculation %
of $\Pi_\cH(n)$ for each $n \in \bN$ by definition. The growth function has a trivial bound of $\growth_\cH(n) \leq 2^n$. This bound does not provide us with any interesting results on generalization as~\eqref{eq:generalization_growth_function} becomes vacuous. Nonetheless, it is useful to define the concept of \emph{shattering}. A set $\{ x_1, \ldots, x_n \}$ of $n$ instances is said to be shattered by a hypothesis class $\cH$ if
the instances of the set can be classified by elements in the class $\cH$ in every possible way, that is, if $\growth_\cH(n) = 2^n$. The VC dimension of a hypothesis class $\cH$ is defined as the largest size of a set shattered by $\cH$. More precisely,
\begin{equation*}
	\vcdim(\cH) = \sup \big \{ n \in \bN : \growth_\cH(n) = 2^n \big \}.
\end{equation*}

The Sauer--Shelah--Perles lemma~\cite[Lemma 6.10]{shalev2014understanding} bounds the growth function of a hypothesis class $\cH$ in terms of its VC dimension. Formally, it states that
\begin{equation*}
\label{eq:sauer_lemma}
	\growth_\cH(n) \leq \sum_{i=0}^{\vcdim(\cH)} \binom{n}{i} \leq 
	\begin{cases}
		2^{\vcdim(\cH)} & \textnormal{ if } n \leq \vcdim(\cH) \\
		\bigg( \frac{en}{\vcdim(\cH)} \bigg)^{\vcdim(\cH)} & \textnormal{ otherwise}
	\end{cases}.
\end{equation*}
Hence, the VC dimension characterizes the uniform convergence of a hypothesis class and, since we are considering the binary classification problem, also its agnostic PAC learnability. To be more exact, there exist absolute constants $c_1$ and $c_2$ such that~\cite[Theorem 6.8]{shalev2014understanding}
\begin{equation*}
	c_1 \cdot \sqrt{\frac{\vcdim(\cH) + \log \frac{1}{\beta}}{n}} \leq \alpha_\mathrm{UC}(n,\beta) \leq c_2 \cdot \sqrt{\frac{\vcdim(\cH) + \log \frac{1}{\beta}}{n}}.
\end{equation*}

The VC dimension is particularly useful since it allows us to obtain uniform convergence generalization bounds for infinite hypothesis classes. A canonical example is the class of threshold functions $\cH = \{ \bI_{ \{x \leq a \}} : a \in \bR \}$. Clearly, there are an infinite number of thresholds on the real line. However, their VC dimension is 1, and therefore they have the uniform convergence property. Other classical examples of hypotheses classes with finite VC dimension are the intervals, hyperplanes, or axis aligned rectangles~\cite{shalev2014understanding, mohri2018foundations}.

The notion of VC dimension can be extended to non-binary classification problems with $k$ classes~\cite[Chapter 29]{shalev2014understanding}. A set $\cA \subseteq \cX$ is shattered by the hypothesis class $\cH$ if there exist two functions $f_0, f_1 : \cA \to [k]$ such that
\begin{itemize}
	\item For every $x \in \cA$, $f_0(x) \neq f_1(x)$ .
	\item For every subset $\cB \subseteq \cA$, there exists a function $h \in \cH$ such that $h(x) = f_0(x)$ for all $x \in \cB$ and $h(x) = f_1(x)$ for all $x \in \cA \setminus \cB$.
\end{itemize}
Then, the Natarajan dimension $\ndim(\cH)$ of a hypothesis class $\cH$ is defined as the maximal size of a set shattered by $\cH$. Moreover, the Natarajan dimension also characterizes the uniform convergence and the agnostic PAC learnability of a general classification problem. More precisely, there exist absolute constants $c_1$ and $c_2$ such that~\cite[Theorem 29.3]{shalev2014understanding}
\begin{equation*}
	c_1 \cdot \sqrt{\frac{\ndim(\cH) + \log \frac{1}{\beta}}{n}} \leq \alpha_\mathrm{UC}(n,\beta) \leq c_2 \cdot \sqrt{\frac{\ndim(\cH) \log k + \log \frac{1}{\beta}}{n}}.
\end{equation*}

In the realizable setting, both the VC and the Natarajan dimension of a hypothesis class $\cH$ characterize more strongly the uniform convergence and the PAC learnability of a general classification problem. More precisely there exist absolute constants $c_1$ and $c_2$ such that~\cite[Theorem 29.3]{shalev2014understanding}
\begin{equation*}
	c_1 \cdot \frac{\ndim(\cH) + \log \frac{1}{\beta}}{n} \leq \alpha_\mathrm{UC}(n,\beta) \leq c_2 \cdot \frac{\ndim(\cH)}{n} \cdot \mathtt{W} \left[ \frac{\big(\frac{1}{\beta}\big)^{\ndim(\cH)} k n}{c_2 \ndim(\cH)} \right],
\end{equation*}
\looseness=-1 where $\mathrm{W}$ is the Lambert function, and the 0 branch can be bounded from above as $\mathrm{W}(x) \leq \log(x + 1)$~\cite[Theorem 2.3.]{hoorfar2008inequalities}. Therefore, in the realizable setting the uniform convergence of a class is of the order $\nicefrac{(\ndim(\cH) + \log \nicefrac{1}{\beta})}{n}$ up to logarithmic terms.

Unfortunately, the VC or the Natarajan dimensions still do not help us to characterize the generalization behavior of deep learning. For a feed-forward network with sigmoid activation functions of $p$ parameters and $c$ connections between the parameters, the VC dimension is lower bounded in $\Omega(c^2)$ and upper bounded in $\cO(p^2 c^2)$~\cite[Chapter 20]{shalev2014understanding}.

\section{Rademacher Complexity}
\label{sec:rademacher_complexity}

The requirements needed for the PAC learning and uniform convergence based guarantees are very strong. They need to hold simultaneously for all hypotheses in a hypothesis class $\cH$ and all data distributions $\bP_Z$, independently of the draw of the training data $S$. A first relaxation of these requirements come from the \emph{empirical Rademacher complexity} of a hypothesis class $\cH$ with respect to a training dataset $S$~\cite{koltchinskii2000rademacher, bartlett2002rademacher}.

\begin{definition}
\label{def:rademacher_complexity}
The \emph{empirical Rademacher complexity} of a hypothesis class $\cH$ with respect to a fixed training set $s = (z_1, \ldots, z_n)$ of $n$ instances and a loss function $\ell$ is
\begin{equation*}
	\rad(\ell \circ \cH, s) \coloneqq \frac{1}{n} \bE \bigg[ \sup_{h \in \cH} \sum_{i=1}^n R_i \cdot \ell(h,z_i) \bigg],
\end{equation*}
where $R_i$ are independent and identically distributed random variables such that $\bP_{R_i} \big[ -1 \big] = \bP_{R_i} \big[1 \big] = \nicefrac{1}{2}$ for all $i \in [n]$. This kind of variables are known as \emph{Rademacher random variables}.
\end{definition}

The Rademacher complexity can be understood in two different, but equally valid ways:
\begin{enumerate}
	\item As a measure of richness of the hypothesis class by measuring the degree to which it can fit random noise~\cite[Section 3.1]{mohri2018foundations}. More precisely, let $R \coloneqq (R_1, \cdots, R_n)$ and $\ell_s(h) \coloneqq (\ell(h,z_1), \ldots, \ell(h,z_n))$ be the vectors describing the uniform noise from the Rademacher random variables and the losses of a hypothesis $h$ on the instances of the training set $s$. Then, the Rademacher complexity can be written as
	\begin{equation*}
		\rad(\ell \circ \cH, s) = \frac{1}{n} \bE \bigg[ \sup_{h \in \cH} R^\intercal \ell_s(h) \bigg],
	\end{equation*}
	where the inner product $R^\intercal \ell_s(h)$ measures the correlation between the random noise and the losses. Thus, the Rademacher complexity measures how well, on average, the hypothesis class correlates with random noise. Richer or more complex hypothesis classes $\cH$ can generate more losses and correlate better with random noise. This understanding is not unique to Rademacher random variables, and there is an analogue to the Rademacher complexity with standard Gaussian random variables named \emph{Gaussian complexity}, and the two are equivalent up to constants~\cite[Excercise 5.5]{wainwright2019high}\cite{bartlett2002rademacher}. As in~\Cref{subsec:uniform_convergence}, more complex hypothesis classes will have a larger generalization error than simpler ones.
	\item As a measure of the discrepancy between fictitious training and test sets~\cite[Section 26.1]{shalev2014understanding}. Indeed, for a fixed training set $s$, one can construct a fictitious training set as $S_1 = \{ z_i \in s: R_i = 1 \}$ and a fictitious test set as $S_2 = \{ z_i \in s: R_i = -1 \}$, where the randomness of the sets only comes from their construction using the Rademacher random variables. Then, the Rademacher complexity calculates what is, on average, the worst difference between these two sets. Namely, we may write
	\begin{equation*}
		\rad(\ell \circ \cH, s) = \frac{1}{n} \bE \bigg[ \sup_{h \in \cH} \bigg\{ \sum_{z \in S_1} \ell(h,z) - \sum_{z \in S_2} \ell(h,z) \bigg\} \bigg].
	\end{equation*}
	In this way, the connection with generalization becomes apparent.
\end{enumerate}

In both interpretations it is clear that the Rademacher complexity depends on the training dataset $s$, and that the larger the number of samples, the better it represents either (i) the richness or complexity of the hypothesis class, or (ii) the discrepancy between fictitious training and test sets. This is reflected in the following theorem relating the Rademacher complexity with the generalization error of hypotheses from a class $\cH$.

\begin{theorem}[\!\!{\cite[Theorem 3.3]{mohri2018foundations}}]
\label{th:generalization_rademacher_complexity}
Consider a loss function with a range contained in $[a,b]$. Then, for all $\beta \in (0,1)$ and all $h \in \cH$, with probability no less than $1 - \beta$
\begin{equation*}
\label{eq:generalization_rademacher_complexity}
	\gen(h,S) \leq 2 \rad(\ell \circ \cH, S) + (b-a) \sqrt{\frac{9 \log \frac{2}{\beta}}{2n}}
\end{equation*}
\end{theorem}

The Rademacher complexity can be used to derive non-vacuous generalization bounds for important hypothesis classes such as support vector machines (SVMs) and other kernel based hypotheses~\cite[Chapter 26]{shalev2014understanding}~\cite[Chapters 5 and 6]{mohri2018foundations}. However, even if the empirical Rademacher complexity is data dependent and could in theory be calculated with the training dataset, the expectation with respect to the Rademacher random variables $R_i$ requires performing $2^n$ empirical risk minimizations. This is computationally hard for some hypothesis classes and often one resorts to results either using the expected Rademacher complexity or bounding this empirical measure with the growth function or the VC dimension, hence losing the advantage of having a data dependent measure.

\looseness=-1 The usage of the Rademacher complexity to explain the generalization of deep learning models seems to be complicated. Even if the measure takes advantage of the collected training dataset, it is still a quantity that holds uniformly for all hypotheses in a class $\cH$, which is a very strong requirement. For example, although there are works that obtain upper bounds on the Rademacher complexity of feed-forward networks that depend on different norms of the weights~\cite{bartlett1996sample, neyshabur2015norm, bartlett2017spectrally, golowich2018size, liang2019fisher}, these are still not tight enough to characterize their generalization. Moreover, there are reasons to believe that this measure will not be sufficient for this task. Indeed, the first interpretation of the Rademacher complexity of the class $\cH$ describes how well the hypotheses from the class can fit random binary label assignments, and it has been shown that parameterized, differentiable networks can perfectly fit random labels~\cite{zhang2021understanding}. Hence, it is expected that the Rademacher complexity is close to 1 for binary classification with the 0--1 loss, resulting in a trivial generalization bound.

\section{Minimum Description Length and Occam's Razor}
\label{sec:mdl_and_occams_razor}

So far, all frameworks to guarantee the generalization of a hypothesis $h$ only depended on the hypothesis class $\cH$ and held simultaneously for all hypotheses in the class. When the hypothesis class is ``simple'', this leads to good generalization guarantees, but for more ``complex'' classes the guarantees are vacuous. 

A first step towards considering a more ``per hypothesis'' specialized measure of complexity is done in \emph{structural risk minimization} (SRM)~\citep{vapnik1999nature,shawe1996framework}. The idea behind SRM is to decompose a complex hypothesis class $\cH$ into a countable union of hypothesis classes $\cH = \bigcup_{k = 1}^\infty \cH_k$ such that $\cH_k \subseteq \cH_{k+1}$ for all $k \in \bN$ and where the complexity of each class $\cH_k$ is non-decreasing in $k$. The complexity of the class $\cH_k$ can be measured with any of the previous methods, either uniform convergence~\cite[Section 7.2]{shalev2014understanding} or the Rademacher complexity~\cite[Section 4.3]{mohri2018foundations}. Then, the bounds on the population risk of a hypothesis $h$ depend on the class $\cH_k$ they belong: they trade-off a smaller empirical risk for more complex classes (or larger $k$) for a larger uniform convergence or Rademacher complexity.

A further step is given by the \emph{minimum description length} (MDL) \emph{principle}~\cite{rissanen1978modeling, grunwald2007minimum}~\cite[Section 7.3]{shalev2014understanding}, where a different generalization guarantee is given to each hypothesis depending on the ``preference'' given to it. To be more precise, consider a countable hypothesis class such that $\cH = \bigcup_{k = 1}^\infty \{ h_k \}$. Furthermore, assume that we have a preference to each hypothesis $h_k$ determined by a probability distribution $\bQ$: larger values of $\bQ \big[ h_k \big]$ mean a larger preference for that hypothesis. Then, an application of Hoeffding's inequality and the union bound yields that, for all losses with a range contained in $[a,b]$ and all hypotheses $h_k \in \cH$, with probability no less than $1 - \beta$~\citep[Theorem 2]{mcallester1998some}\citep{barron1991complexity, barron1991minimum, kearns1995experimental,mcallester1999pac,mcallester2003pac},
\begin{equation}
\label{eq:mdl_generalization}
	\gen(h_k, S) \leq \sqrt{\frac{-\log \bQ \big[ h_k \big] + \log \frac{1}{\delta}}{2n}}.
\end{equation}

A difficult question is \emph{how} to select the preference over the hypotheses in the class $\cH$ when no prior knowledge is available. For instance, if the hypothesis class is finite $|\cH| < \infty$ and we have no preference for any hypothesis (that is, we consider a uniform preference distribution $\bQ \big [ h_k \big] = \nicefrac{1}{|\cH|}$), then the guarantee is equal for each hypothesis and~\eqref{eq:mdl_generalization} is reduced to the same guarantee (up to constants) received by the uniform convergence property in~\eqref{eq:uniform_convergence_finite_hypothesis}. A simple way to express a preference is to favour hypotheses that are simpler to describe, or that have a smaller \emph{description length}. A binary string is a finite sequence of zeros and ones. A description language for a hypothesis class $\cH$ is a function mapping each hypothesis $h \in \cH$ to a string $\desc(h)$ with length $|\desc(h)|$, where $\desc(h)$ is called the \emph{description of $h$} and $|\desc(h)|$ is its \emph{description length}. %
If the language is prefix-free, then Kraft's inequality holds~\cite[Theorem 5.2.1]{Cover2006}, namely $\sum_{h \in \cH} 2^{- |\desc(h)| } \leq 1$. %
In this way, one may consider the simple prior $\bQ \big[ h_k \big] \propto 2^{- |\desc(h)|}$ resulting in the bound
\begin{equation}
\label{eq:mdl_generalization_compression}
	\gen(h_k,S) \leq \sqrt{\frac{ |\desc(h_k)| \log 2 + \log \frac{1}{\delta}}{2n}},
\end{equation}
where the description language could be any prefix-free compression algorithm and $|\desc(h_k)|$ is its length (or bits). When the description length of a hypothesis $h$ is the smallest possible it is known as its \emph{Kolmogorov complexity}~\cite{grunwald2007minimum} and the prior $\bQ$ is referred to as the \emph{universal prior}~\cite[Section 14.6]{Cover2006}.

Then, following the MDL principle leads to selecting hypotheses that trade off a good empirical performance (small empirical risk) and a small complexity or description length (small generalization error). This is aligned with the \emph{Occam's razor}, which states that ``it is futile to do with more, what can be done with fewer''~\cite{de1974summa,ball2016tyranny}. This has been incorporated into the methodology of science in the following form~\cite{blumer1987occam, shalev2014understanding}: ``given two explanations of the data, all other things being equal, the simpler explanation is preferable''. Indeed, following the MDL principle, given two hypotheses $h_k$ and $h_l$ with the same empirical risk $\emprisk(h_k,s) = \emprisk(h_l,s)$ for a fixed dataset $s$, the one that is easier to describe (or has a smaller description length) is preferable.

\sloppy The parsimonious philosophical principle of the Occam's razor will also resonate with the rest of the generalization guarantees described in this manuscript, although with other characterizations of the hypotheses complexity and not necessarily their description length.

The MDL principle has big connections with other theoretical frameworks to generalization such as PAC Bayesian theory~\cite{shawe1997pac,mcallester1998some,mcallester1999pac, mcallester2003pac}. More precisely, the two are equivalent when the considered algorithms are deterministic and the hypothesis class is discrete, as shown below in~\Cref{subsec:levels_of_specificity}. In fact, under the PAC Bayesian umbrella, the MDL has been employed to obtain non-vacuous bounds for deep learning algorithms~\cite{lotfi2022pac}. The idea is to describe the parameters of the networks with a tunable  prefix-free variable-length code that acts as the description language $\desc$. Then, both the quantized parameters and the quantization levels of the code are learnt simultaneously to minimize a variant of the MDL generalization guarantee in~\eqref{eq:mdl_generalization_compression}.

\section{Algorithmic Stability}
\label{sec:algorithmic_stability}

The MDL principle takes a step away from previous frameworks like PAC learning and uniform convergence and establishes generalization guarantees that are specific for each hypothesis in the hypothesis class $\cH$. However, it requires that the hypothesis class is discrete (or belongs to a subclass $\cH_k$ from a countable set of subclasses that cover the whole class in the case of SRM) and it does not take into account the algorithm that selects the hypothesis. Even though one, in theory, could study specific algorithms and give bounds on the description length of their selected hypotheses, this is not embedded into the definition of the framework.

\emph{Algorithmic stability}~\cite{rogers1978finite, devroye1979adistribution, devroye1979bdistribution}, on the other hand, makes the algorithm the central object of its framework. To put it simply, an algorithm is stable if, when presented with similar input training datasets, it outputs similar hypotheses. The connection to generalization is then simple. If the hypothesis returned by an algorithm does not change much for similar training datasets, given a training dataset sampled from the data distribution, then the output hypothesis would be similar to other training datasets from the same distribution. Hence, the generalization error will depend on how well the training dataset represents the data distribution. This framework still follows the Occam's razor principle if we measure the complexity of the hypothesis returned by an algorithm by how unstable the algorithm is.

Precisely describing what it means for the input training datasets and for the output hypotheses to be ``similar'' is complicated, and there are different notions of stability with different generalization guarantees. Below, we will describe a particularly successful notion named \emph{uniform stability}~\cite{bousquet2002stability}. Then, we will discuss how privacy can be understood as a stability notion: if an algorithm is private, then it is stable. This will serve as a connection point to~\Cref{ch:privacy_and_generalization}. Finally, we will mention how the stability of an algorithm can be described with information-theoretic measures of dependence of the output hypothesis on the training data. This final subsection will be the link to~\Cref{sec:information_theoretic_generalization} discussing information-theoretic generalization and, more broadly, to the rest of the manuscript. We omit the discussion around other notions of stability, such as the \emph{leave-one-out}, the \emph{leave-one-out cross validation}, or \emph{hypothesis stability}, among others~\cite{bousquet2002stability, kutin2002almost, rakhlin2005stability, mukherjee2006learning, shalev2010learnability} due to space and scope constraints.

\subsection{Uniform Stability}
\label{subsec:uniform_stability}

Consider a \emph{deterministic} algorithm $\bA: \cZ^n \to \cH$. \citet{bousquet2002stability} described the stability of the algorithm as the largest difference in performance of hypotheses generated by the algorithm when presented to \emph{neighbouring datasets}. Therefore, for them, two input training datasets are ``similar'' if they are neighbours, that is, if they differ in at most one element; and two output hypotheses are ``similar'' if the absolute difference in their performance is bounded.

\begin{definition}
\label{def:uniform_stability}
A deterministic algorithm $\bA$ is \emph{uniformly stable} with parameter $\gamma$ if for every training dataset $s = (z_1, \ldots, z_n) \in \cZ^n$, every neighbouring dataset $s^{(i)} = (z_1, \ldots, z_{i-1}, z_i', z_{i+1}, \ldots, z_n) \in \cZ^n$ for all $i \in [n]$, and every sample $z \in \cZ$
\begin{equation*}
	\big| \ell \big( \bA(S), z\big) - \ell \big(\bA(S^{(i)}), z \big) \big| \leq \gamma.
\end{equation*}
\end{definition}

By the definition of uniform stability, and the fact that $\bE \big[ \gen(\bA(S),S) \big] = \bE \big[ \ell \big( \bA(S), z\big) - \ell \big(\bA(S^{(i)}), z \big) \big]$~\citep[Lemma 7]{bousquet2002stability}, it directly follows that $\bE \big[ \gen(\bA(S),S) \big] \leq \gamma$.
There have been increasingly better bounds on the generalization error based on the clear intuition that if an algorithm is uniformly stable, then it generalizes~\cite{bousquet2002stability,feldman2018generalization, feldman2019high, bousquet2020sharper}. To our knowledge, the following one is the one that best characterizes the generalization error of uniformly stable algorithms, although there exists better characterizations of the excess risk~\cite{klochkov2021stability}.

\begin{theorem}[{\citet[Corollary 8]{bousquet2020sharper}}]
\label{th:generalization_uniform_stability}
Consider a loss function with a range contained in $[a,b]$.  Let $h = \bA(S)$, where $\bA$ is a deterministic, uniformly stable algorithm with parameter $\gamma$. Then, there exist universal constants $c_1, c_2 \in \bR_+$ such that, for all $\beta \in (0,1)$, with probability no less than $1 - \beta$,
\begin{equation*}
	|\gen( h, S) | \leq c_1 \cdot \gamma \log (n) \log \bigg(\frac{1}{\beta} \bigg) + c_2 \cdot (b-a) \cdot \sqrt{\frac{\log \frac{1}{\beta}}{n}}.
\end{equation*}
\end{theorem}

The bound in \Cref{th:generalization_uniform_stability} tells us that if an algorithm is uniformly stable, as long as the stability parameter $\gamma$ decreases with respect to the number of samples $n$ faster than logarithmically, then the algorithm will generalize for a large enough training dataset. The framework of uniform stability has been relatively successful in providing us with generalization error guarantees for known algorithms, that is, we know that many known algorithms are uniformly stable. For example, the ERM solution to convex learning problems with a strongly convex regularization term~\cite{shalev2010learnability} is uniformly stable with $\gamma \in \Theta( \nicefrac{1}{\sqrt{n}})$. Moreover, there is a line of work establishing the uniform stability of SGD under different combinations of conditions such as the Bernstein or Polyak--Łojasiewicz conditions, or smoothness, convexity, or Lipschitzness, among others~\cite{hardt2016train, kuzborskij2018data, lei2020fine, bassily2020stability, klochkov2021stability, charles2018stability, lei2023stability}, which is an encouraging direction to better understand the generalization in deep learning.

\subsection{Privacy as a Stability Measure}
\label{subsec:privacy_as_stability}

Recall from \Cref{subsec:renyi_divergence} that %
a \emph{privacy mechanism} is an algorithm that answers queries about a given dataset in a way that is informative about the queries themselves but not the specific records (or instances) in the dataset. Private learning algorithms are a special class of learning algorithms that employ a privacy mechanism to analyze the training dataset. Hence, these algorithms produce hypotheses that are uninformative about the dataset with which they were trained. In this way, private learning algorithms are stable, as ideally the hypotheses generated by the algorithm when presented with similar training datasets are also similar. 

To visualize this concept more clearly, let us use Dwork et al.'s definition of \emph{differential privacy}~\cite{dwork2006calibrating, dwork2014algorithmic}. A \emph{randomized} algorithm $\bA$ is $(\varepsilon, \delta)$-\emph{differentially private} if for all subsets of hypotheses $\cA \subseteq \cH$ and all neighbouring datasets $s$ and $s'$
\begin{equation}
\label{eq:differential_privacy}
	\bP \big[ \bA(s) \in \cA \big] \leq e^\varepsilon  \bP \big[ \bA(s') \in \cA \big] + \delta.
\end{equation}
If $\delta = 0$, then the algorithm is just $\varepsilon$-\emph{differentially private}. This definition clarifies the idea that private algorithms are stable: like for uniform stability, two input training datasets are ``similar'' if they are neighbours, and two output hypotheses are ``similar'' if their distributions are close as measured by~\eqref{eq:differential_privacy}.

Since differentially private algorithms are stable, under certain requirements on the privacy parameters $(\varepsilon, \delta)$, they should generalize. The connection between differential privacy (and other privacy notions like maximal leakage) and generalization has been previously studied~\cite{dwork2015preserving, dwork2015generalization, raginsky2016information, oneto2017differential, steinke2020reasoning, hellstrom2020generalization, rodriguez2021upper, esposito2021generalization, kulynych2022you, bun2023stability} and will be further discussed in more detail in~\Cref{ch:privacy_and_generalization}.

\subsection{Stability via Information Measures}
\label{subsec:stability_via_information_measures}

\begin{figure}[t]
\centering
\begin{tikzpicture}[->, line width=0.5pt, node distance=2cm, scale=1]]

	\node (S) at (0,0) {$S$};
	\node[draw, rectangle, minimum width=1cm] (algo) at (2,0) {$\bA \equiv \bP_H^S$};
	\node (H) at (4,0) {$H$};algo
	
	\path (S) edge (algo);
	\path (algo) edge (H);
\end{tikzpicture}
\caption{Illustration of a learning algorithm $\bA$ viewed as a channel processing a dataset $S$ to obtain a hypothesis $H$.}
\label{fig:algorithm_as_channel}
\end{figure}
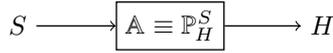

The stability notion borrowed from differential privacy in the previous subsection is inherently information-theoretic. To illustrate this, consider a (possibly randomized) algorithm $\bA$ characterized by the Markov kernel $\bP_H^S$, that is, given a fixed input dataset $s$, the output of the algorithm is a random hypothesis $H \coloneqq \bA(s)$ distributed according to $\bP_{H}^{S=s}$. If the algorithm is deterministic, then $\bP_H^{S=s} = \bI_{\{h = h_s \}}(h)$ and $\bA(s) = h_s$. From an information-theoretic perspective, an algorithm is a channel processing a dataset into a hypothesis (\Cref{fig:algorithm_as_channel}). Then, recall from \Cref{subsec:renyi_divergence} that an algorithm is $\varepsilon$-differentially private if, for every two neighbouring training datasets $s$ and $s'$, the output distributions of the hypotheses resulting from the channel processing are close in Rényi divergence of order $\infty$, that is, 
\begin{equation*}
	\renyidiv{\infty}(\bP_H^{S=s} \Vert \bP_H^{S=s'}) \leq \varepsilon.	
\end{equation*}

In this way, an algorithm is \emph{$\varepsilon$-$\mathrm{IM}$ stable} if, given two neighbouring input datasets, the difference between the two output hypotheses, as measured by some information measure $\mathrm{IM}$ like the ones presented in \Cref{sec:information_theory}, is smaller than $\varepsilon$. This definition, like uniform stability, considers two input datasets to be ``similar'' if they are neighbours, and understands that two output hypotheses are ``similar'' if their distributions are close according to some information measure. Other common examples of this information-theoretic stability are total variation, relative entropy, and Wasserstein distance stability~\cite{bassily2016algorithmic, raginsky2016information}. As for the previous definitions of stability, if an algorithm satisfies any of these notions, its generalization error is also bounded~\cite{bassily2016algorithmic, raginsky2016information, steinke2020reasoning, rodriguez2021upper}. 

So far, the stability definitions were agnostic to the data distribution $\bP_Z$. \citet{raginsky2016information} introduced the concept of \emph{$(\varepsilon, \bP_Z)$-$\mathrm{IM}$ stability} to take into account the effect of the data distribution into the stability of an algorithm. Given a dataset $s = (z_1, \ldots, z_n)$, define $s^{-i} = (z_1, \ldots, z_{i-1}, z_{i+1}, \ldots, z_n)$ as the dataset obtained by removing the $i$-th sample from $s$. In this way, the distribution 
\begin{equation*}
	\bP_H^{S^{-i}=s^{-i}} = \bP_H^{S = (z_1, \ldots, z_{i-1}, Z_i, z_{i+1}, \ldots, z_n)} \circ \bP_Z
\end{equation*}
represents the expected distribution of the output hypothesis after processing a dataset $S$ where $Z_i$ is distributed according to the data distribution $\bP_Z$ and the other instances are fixed $S^{-i} = s^{-i}$. Then, $(\varepsilon, \bP_Z)$-$\mathrm{IM}$ stability guarantees that, on average, the output hypothesis does not change much when each instance of the training dataset is changed by another instance from the data distribution. For example, an algorithm is $(\varepsilon, \bP_Z)$-\emph{total variation stable} if\footnote{The original definition in~\cite{raginsky2016information} was written in a different, but equivalent form. We choose this presentation to ease the understanding of the concepts.}
\begin{equation*}
	\frac{1}{n} \sum_{i=1}^n \bE \Big[ \tv\big( \bP_H^{S} , \bP_H^{S^{-i}} \big) \Big] \leq \varepsilon.
\end{equation*}

Taking this argument to the extreme, one may consider how much, on average, the distribution of the output hypothesis $\bP_H^{S=s}$ for a given realization $S=s$ of the training dataset changes with respect to the prototypical distribution on samples from the data distribution $\bP_H = \bP_H^S \circ \bP_S$. For example, this was the rationale considered in~\cite{alabdulmohsin2015algorithmic,alabdulmohsin2017information} to define the stability of an algorithm and to prove generalization guarantees in terms of the total variation. This argument can also be interpreted as to how much the algorithm's output distribution depends on the input training dataset. Both of these interpretations are paramount in the information-theoretic generalization framework introduced in the next section. Therefore, the stability framework to characterize the generalization of learning algorithms, when the stability notion is defined via privacy or information measures, can be understood as belonging to the information-theoretic generalization framework and \emph{vice versa}. The lines between the different frameworks are blurry and, in the end, the terminology barely depends on personal taste.

\section{Information-Theoretic Generalization}
\label{sec:information_theoretic_generalization}

The MDL principle provides us guarantees that are specific for each hypothesis $h$ in the class $\cH$, unlike those given to us by uniform convergence or the Rademacher complexity. Algorithmic stability takes the algorithm into account and gives us generalization guarantees that are specific to the algorithm $\bA$ used to find the hypothesis. Information-theoretic generalization also provides us with guarantees that are specific to the learning algorithm and, depending on the level of specificity, that are also specific to each hypothesis. Not needing to provide guarantees that hold uniformly for all elements in the hypothesis class $\cH$ allows these frameworks to attain tighter characterizations of the generalization of learning algorithms.

\looseness=-1 Often, the guarantees the information-theoretic generalization framework provides us do not hold for every data distribution $\bP_Z$, and they are specialized to different classes of data distributions. These classes are chosen depending of the behavior of the loss random variable $\ell(h,Z)$ for hypotheses in $h \in \cH$. This consideration allows the framework to derive bounds on the population risk for potentially unbounded losses and separates it from frameworks like uniform stability. 

The information-theoretic framework, like algorithmic stability via information measures, considers the (possibly randomized) algorithm as a channel processing a dataset into a hypothesis (\Cref{fig:algorithm_as_channel}). Essentially, this framework encompasses all generalization guarantees that depend on information measures like those presented in \Cref{sec:information_theory} involving the algorithm's Markov kernel $\bP_H^S$ and the data distribution $\bP_Z$ in some capacity. A common theme in the guarantees obtained in this framework is that they can be interpreted with classically information-theoretic concepts like \emph{information} or \emph{compression}. 

Intuitively, the more information the algorithm's output hypothesis captures about the dataset that it used for training, the worse it will generalize.  The reason is closely related to the concept of \emph{overfitting}, which is a phenomenon that occurs when the output hypothesis describes very well the training data but fails to describe the underlying distribution. The idea is that, in order to perfectly describe the training data, the algorithm adapted to the sampling noise for the specific data it observes, which may differ from future observations. This can be quantified, for example, with the dissimilarity between the distribution of the algorithm's output $\bP_H^{S=s}$ and the smoothed, prototypical distribution on samples from the data distribution $\bP_H = \bP_H^S \circ \bP_S$, for instance $\relent(\bP_H^S \Vert \bP_H)$~\cite{shawe1997pac,mcallester1998some,mcallester1999pac,mcallester2003pac,zhang2006information,raginsky2016information}. This example is purposely chosen to highlight again the connection between stability via information measures and information-theoretic stability.

This intuition also follows the parsimonious philosophical principle of the Occam's razor from \Cref{sec:mdl_and_occams_razor}: given two algorithms that output hypotheses with the same empirical risk, the one that extracts the least information, or needs the least bits to be compressed, is preferred.

In practice, the foundational generalization bounds from this framework are usually obtained combining a \emph{change of measure} (or a \emph{decoupling lemma}) and a \emph{concentration inequality}. For randomized algorithms, the population risk $\poprisk(H)$ is a random variable that depends on the joint distribution $\bP_S \otimes \bP_H^S$, where $S$ is the random training dataset and $\bP_S = \bP_Z^{\otimes n}$. This dependence between the hypothesis $H$ and the dataset $S$ makes the usage of standard concentration inequalities around the empirical risk impossible. For this reason, a common first step is to use a change of measure to consider the population risk $\poprisk(H')$ of an auxiliary random hypothesis $H' \sim \bQ$  that is independent (or \emph{decoupled}) of the training data $S$. Ideally, the chosen distribution $\bQ$ is one that allows us to control the new population risk using standard concentration inequalities like those in~\cite{boucheron2003concentration}. The penalty for studying the risk of the auxiliary hypothesis $H'$ instead of the real one $H$ is captured by an information measure that describes the dissimilarity between their distributions. When the chosen distribution is the prototypical distributions on samples from the distribution $\bP_H$, then we recover the intuition and interpretation given above. We will go deeper into this in \Cref{ch:expected_generalization_error,ch:pac_bayesian_generalization,ch:privacy_and_generalization}, but the reader is also referred to the review papers~\cite{alquier2021user, hellstrom2023generalization}.

The main criticism of this framework is that in order to find computable generalization guarantees it is necessary to evaluate or further bound the information measures between distributions. This is often difficult or impractical, leading to crude bounds that may be vacuous~\cite{bassily2018learners,livni2020limitation,haghifam2023limitations,livni2023information}. However, they can still lead to non-vacuous generalization statements in deep learning~\cite{dziugaite2017computing,dziugaite2018data,rivasplata2019pac,perez2021tighter,zhou2019non,lotfi2022pac,rodriguez2023morepac}, to recover from below known results for hypothesis classes with a bounded VC or Natarajan dimension~\cite{steinke2020reasoning,grunwald2021pac,hellstrom2022new}, or helps us better understand noisy, iterative algorithms like stochastic gradient descent~(SGD)~\cite{neu2021information}, among other upsides. More examples of the benefits obtained from guarantees derived from this framework are described in the following chapters.

\subsection{Levels of Specificity}
\label{subsec:levels_of_specificity}

There are different levels of specificity when it comes to the guarantees provided to us by information-theoretic generalization. The specificity levels come depending if we require the results to hold for a specific hypothesis returned from the algorithm given a specific training dataset, or if we only need them to hold on average for the outputs of the algorithm given a specific training dataset, or, finally, if we need them to hold on average for the outputs of the algorithm given training datasets from a certain distribution. These specificity levels are also referred to as ``flavours'' of generalization~\cite{hellstrom2020generalization, hellstrom2023generalization} and are further described below from less, to more specific together with a classical example. All guarantees are provided to a given algorithm described by a Markov kernel $\bP_H^S$ and a data distribution $\bP_Z$.

\begin{itemize}
	\item Generalization guarantees \emph{in expectation} (or ``PAC-Bayesian guarantees in expectation'', ``expectedly approximately correct (EAC) guarantees''~\cite{dalalyan2008aggregation,salmon2011optimal,dalalyan2012sharp}, or ``mean approximately correct (MAC) guarantees''~\cite{grunwald2021pac}). This is the least specific level. This kind of guarantee states that
	\begin{equation*}
		\bE[ \gen(H,S) ] \leq \alpha_{\mathrm{exp}},
	\end{equation*}
	where the expectation is taken with respect to the marginal, prototypical distribution $\bP_H = \bP_H^S \circ \bP_S$, where we recall that $\bP_S = \bP_Z^{\otimes n}$.
	
	For a loss with a range contained in $[a,b]$, the classical example of this guarantee using the Donsker and Varadhan~\Cref{lemma:dv_and_gvp} is~\cite{raginsky2016information}
	\begin{equation*}
		\alpha_{\mathrm{exp}} = (b-a) \sqrt{\frac{\minf(H;S)}{2n}}.
	\end{equation*}

	\item \emph{PAC-Bayesian} guarantees. These guarantees are specific to the observed realization of the training dataset $S$. More precisely, this kind of guarantee states that, for all $\beta \in (0,1)$, with probability no less than $1 - \beta$,
	\begin{equation*}
		\bE^S [ \gen(H,S) ] \leq \alpha_{\mathrm{PAC\textnormal{-}Bayes}},
	\end{equation*}
	where the probability is taken with respect to the draw of the training data from $\bP_S$ and the expectation is taken with respect to the conditional distribution of the output hypothesis for that training data $\bP_H^S$.
	
	For a loss with a range contained in $[a,b]$, the classical example of this guarantee using the Donsker and Varadhan~\Cref{lemma:dv_and_gvp} is~\cite{mcallester1998some,mcallester1999pac,mcallester2003pac} 
	\begin{equation}
		\label{eq:mc_allester_prev}
		\alpha_{\mathrm{PAC\textnormal{-}Bayes}} = (b-a) \sqrt{\frac{\relent(\bP_H^S \Vert \bQ) + \log \frac{\xi(n)}{\beta}}{2n}},
	\end{equation}
	where $\xi(n) \in [\sqrt{n}, 2 + \sqrt{2n}]$~\cite{maurer2004note,germain2015risk,rodriguez2023morepac} and $\bQ$ is any distribution on $\cH$ such that $\bP_H^S \ll \bQ$. Note that compared to the guarantee in expectation, this includes a penalty for the confidence $1-\beta$ of the statement.
	
	Note that if the learning algorithm is deterministic (that is, if the algorithm's Markov kernel is $\bP_H^S(h) = \bI_{\{ h = h_S \}}(h)$) and the hypothesis class is discrete, then $\relent(\bP_H^S \Vert \bQ) = - \log \bQ[h_S]$ and the PAC-Bayesian guarantees are equivalent to the guarantees from the MDL framework (see \Cref{sec:mdl_and_occams_razor} above).
	
	\item \emph{Single-draw PAC-Bayesian} guarantees (or ``pointwise or derandomized PAC-Bayesian guarantees''~\cite{alquier2013sparse,guedj2013pac})\footnote{We choose this name as we feel it is the one that better clarifies that the bounds are on a single draw of the hypothesis rather than on the algorithm's distribution. This is also the name employed in~\cite{hellstrom2020generalization,hellstrom2023generalization}}. These guarantees are specific to the observed realization of the training dataset $S$ and the particular output hypothesis returned by the algorithm. To be precise, this kind of guarantee states that, for all $\beta \in (0,1)$, with probability no less than $1 - \beta$,
	\begin{equation*}
		\gen(H,S) \leq \alpha_{\mathrm{sdPAC}},
	\end{equation*}
	where the probability is taken with respect to the draw of the training data from $\bP_S$ and the later draw of the algorithm's output hypothesis from $\bP_H^S$.
	
	For a loss with a range contained in $[a,b]$, the classical example\footnote{The bound was not originally given in this form in \cite{rivasplata2020pac}, but one can recover it after routine manipulations.} of this guarantee using the change of measure stemming from the Radon--Nikodym theorem~\eqref{eq:change_of_measure} is~\cite{rivasplata2020pac}
	\begin{equation*}
		\alpha_{\mathrm{sdPAC}} = (b-a) \sqrt{\frac{\log \frac{\rmd \bP_H^S}{\rmd \bQ}(H) + \log \frac{\xi(n)}{\beta}}{2n}},
	\end{equation*}
	where we recall that $\nicefrac{\rmd \bP_H^S}{\rmd \bQ}$ is the Radon--Nikodym derivative between the algorithm's output distribution and the reference distribution $Q$. Note that compared to standard PAC-Bayesian and the expectation guarantees, this changes depending on the particular realization of the output hypothesis.
\end{itemize}

\looseness=-1 The more specific guarantees give us more practical information. For example, the single-draw PAC-Bayesian guarantees hold for the particular realization of the hypothesis that an algorithm returns for a particular dataset. However, they are often harder to calculate or control. On the other hand, the less specific guarantees gives us more abstract information and are often easier to calculate and control. For instance, the mutual information can be upper bounded by other quantities relevant to a particular algorithm, giving us a concrete understanding of the important elements for that algorithm to generalize. For example, for the stochastic gradient Langevin dynamics (SGLD)~\cite{gelfand1991recursive,welling2011bayesian} algorithm, we learnt that the gradient incoherence between samples important in determining its generalization~\cite{negrea2019information,haghifam2020sharpened,rodriguez2020randomsubset} (see \Cref{sec:noisy_iterative_learning_algos}). More examples of these bounds' usefulness are given in \Cref{ch:expected_generalization_error}.  

Before moving to some bibliographic remarks we clarify some notation. We say that a bound has a \emph{fast rate} if it decreases linearly with the number of samples, that is, if it is in $\cO(\nicefrac{1}{n})$. For example, the guarantees for the realizable setting in classification problems provided by uniform convergence. Similarly, we say that a bound has a \emph{slow rate} if it is in $\cO(\nicefrac{1}{\sqrt{n}})$. If a bound has a rate different than these two, we will refer to the rate in comparison to them. For probabilistic bounds, we say that a bound is of \emph{high probability} if its dependence with the confidence parameter $\beta$ is logarithmic, that is, it depends on $\log \nicefrac{1}{\beta}$ like~\eqref{eq:mc_allester_prev}. Other dependencies, such as the linear one $\nicefrac{1}{\beta}$, will simply be referred to as not of high probability.\footnote{There are other notions of ``high probability''. For example, a common definition states that a bound is of high probability if the probability of failure $\beta$ can be made arbitrarily close to 0 as the number of samples tend to infinity~\citep{wainwright2019high}. We choose this nomenclature, similarly to~\citet{hellstrom2020generalization}, since it helps us distinguish between an exponential decay in the probability of failure (when the dependence is logarithmic) and weaker decays such as polynomial.}

\subsection{Bibliographic Remarks}
\label{subsec:bibliographic_remarks_primer}

Arguably, the first relation of information-theoretic concepts and generalization guarantees is due to \citet[Section 4.6]{vapnik1999nature} prior to 1995. He formalized a relationship between the MDL principle from~\Cref{sec:mdl_and_occams_razor} and compression to a generalization bound similar to~\eqref{eq:mdl_generalization_compression}. After that, Shawe-Taylor, Bartlett, and Williamson~\citep{shawe1996framework,shawe1997pac} introduced the first PAC-Bayesian bounds using a luckiness factor in 1996-1997, and those were further developed by \citet{mcallester1998some,mcallester1999pac,mcallester2003pac} in 1998-2003. These latter results first obtained another formal relationship between the MDL principle and generalization, now in the form of~\eqref{eq:mdl_generalization}, and then evolved into the more general PAC-Bayesian bound from~\eqref{eq:mc_allester_prev}. Interestingly, McAllester did not use neither the Donsker and Varadhan nor the Gibbs~\Cref{lemma:dv_and_gvp} to obtain the result and, to our knowledge, the first to do so in the context of generalization was \citet{seeger2002pac} in 2002, who actually re-discovered it. After that, the usage of this result became customary in the PAC-Bayesian literature, popularized by \citet{audibert2004better}, \citet{catoni2003pac,catoni2007pac}, \citet{zhang2006information}, and Germain and others~\citep{germain2009pac,begin2016pac}, although some re-derived it again and gave it a different name for the context of generalization, such as the \emph{information exponential inequality} from \citet{zhang2006information} in 2006.

Information-theoretic generalization bounds in expectation were derived in the PAC-Bayesian community, at least, since 2006~\cite{alquier2006transductive,catoni2007pac,juditsky2008learning}. These were popularized in 2016 after \citet{xu2017information} extended the results from \citet{russo2016controlling} from the bias of adaptive data analysis to the generalization error or learning algorithms. The popularity of these results came, in part, for the simplicity of the proofs that directly used the Donsker and Varadhan~\Cref{lemma:dv_and_gvp} and the fact that the bound explicitly featured the mutual information. The combination of these two facts allowed for the development of many results and interpetations after that, some of which will be discussed later in \Cref{ch:expected_generalization_error}.

\begin{remark}
Another, previous conceptual connection between information-theoretic concepts and generalization guarantees came from Dudley's \emph{metric entropy}~\cite{dudley2010universal} to bound the rate of uniform convergence~\cite{pollard1984convergence} in 1984. For a metric hypothesis space $(\cH, \rho)$, an $\varepsilon$-net is a set of hypotheses $\cN(\varepsilon,\cH,\rho)$ that cover the hypothesis class with $\varepsilon$ precision, that is, that for all hypothesis $h \in \cH$ there is a hypothesis $h' \in \cN(\varepsilon,\cH,\rho)$  such that $\rho(h,h') \leq \varepsilon$. The metric entropy is defined as the logarithm of the covering number $\log |\cN^\star(\varepsilon,\cH,\rho,)|$, which is defined as the smallest cardinality of an $\varepsilon$-net. This concept measures the spread of the hypotheses in the class $\cH$, relating it with the entropy. Moreover, it coincides with the entropy of a uniform distribution on the elements of the cover.
\end{remark}

\section{Parameterized Models}
\label{sec:parameterized_models}

Modern learning algorithms are often a combination of a parameterized model and an optimization algorithm, where the parameters are usually referred to as the \emph{weights}. For example, a parameterized, differentiable network and stochastic gradient descent (SGD). 

Consider for example a supervised learning problem where we recall that the instances $z = (x,y)$ are a tuple formed by a feature and a label, and the hypotheses $h: \cX \to \cY$ are functions that return a label when given a feature. In a parameterized model, the hypotheses $h_w : \cX \to \cY$ are completely described by the weights $w \in \cW$. In this way, the hypothesis space is
\begin{equation}
	\label{eq:hypothesis_class_weights}
	\cH = \{ h_w : w \in \cW \}.
\end{equation}
Therefore, each weight $w$ uniquely determines a hypothesis $h_w \in \cH$, but the reverse is not necessarily true: there can be some hypothesis $h$ that results from multiple weights. For example, consider the simple parameterized model $h_w(x) = \sign (x w)$ with $\cX = \cW = \bR$. Here, every choice of $w \in \cW$ uniquely determines the hypothesis $h_w$. On the other hand, the hypothesis $h_w(x) = \sign(x)$ results from every positive weight $w > 0$ and the hypothesis $h_w(x) = - \sign(x)$ results from every negative one $w < 0$. In other words, the relationship between weights and hypotheses is \emph{onto} but not \emph{one-to-one}.

This onto relationship between parameters and weights allows us to find generalization guarantees considering the weights instead of the hypotheses themselves.
\begin{itemize}
	\item Recall the results stemming from uniform convergence or the Rademacher complexity from \Cref{sec:pac_learning,sec:rademacher_complexity}, these results needed to hold for every hypothesis $h$ in the hypothesis class $\cH$. If, instead, we ensure that they hold for every weight $w$ in the weight space $\cW$, they will necessarily hold for all hypotheses in $\cH$ as per~\eqref{eq:hypothesis_class_weights}.
	\item The reasoning is similar for the guarantees coming from the MDL principle in \Cref{sec:mdl_and_occams_razor}, as one can describe every hypothesis $h_w$ by describing the weights $w$. When multiple weights result in the same hypothesis, then the one that is more easily described will be selected as per the Occam's razor.
	\item For deterministic algorithms and uniform stability (\Cref{subsec:uniform_stability}), the algorithm is at the center and therefore considering the weights that uniquely determine they hypothesis or the hypothesis itself is irrelevant.
	\item \looseness=-1 When privacy is considered as a stability measure (\Cref{subsec:privacy_as_stability}) the rationale is slightly different. When the model is parameterized, usually the private algorithm returns the weights $W$, and these are the ones that enjoy the privacy guarantees. Luckily, most privacy frameworks like differential privacy or maximal leakage have \emph{post-processing} guarantees~\cite{dwork2014algorithmic,issa2020operational,saeidian2023apointwise}. This guarantee states that no amount of post-processing can degrade the privacy guarantees. Therefore, the hypothesis $h_W$ maintains the same privacy guarantees and therefore also the same generalization guarantees derived from them.
	\item Finally, for algorithmic stability based on information measures or information-theoretic generalization (\Cref{subsec:stability_via_information_measures,sec:information_theoretic_generalization}), the maintainance of the generalization guarantees comes from the data processing inequality (see~\Cref{prop:properties_relative_entropy} and the rest of \Cref{sec:information_theory}). Consider, for example, the relative entropy: for two weights $W$ and $W'$, this inequality guarantees that $\relent(\bP_{h_W} \Vert \bP_{h_{W'}}) \leq \relent(\bP_W \Vert \bP_{W'})$. Therefore, the guarantees obtained for the weights also translate to the hypotheses.
\end{itemize}

\looseness=-1 For this reason, a good volume of the information-theoretic generalization literature presents their results abusing notation and describing the random hypothesis returned by the algorithm with the letter $W$ and using it indistinguishably from the model's weights. Henceforth, in this manuscript, we will also present our results in this way. However, it is important to mention that, until the weights are employed to find specific guarantees for some particular algorithm like the SGLD in~\Cref{sec:noisy_iterative_learning_algos}, the results hold for any hypothesis and not only for parameterized models.

\begin{remark}
	\looseness=-1 In information theory, the letter $H$ is historically reserved to the Shannon entropy $\ent$ and expressions like $\ent(H)$ can be confusing. Pragmatically, this is also a reason why in information-theoretic generalization the letter $W$ is employed instead of the classical $H$ in learning theory.
\end{remark}

%% file: chapters/generalization_expectation.tex
In this chapter, we discuss generalization guarantees in expectation within the framework of information-theoretic generalization introduced in~\Cref{sec:information_theoretic_generalization}. This generalization guarantee is the least specific within the three levels of specificity described in~\Cref{subsec:levels_of_specificity}. To be precise, the guarantees of this kind describe what is the smallest expected difference between the population risk and the empirical risk. In other words, if we consider the empirical risk as an estimator of the population risk, these guarantees look for an expression that characterizes (or bounds from above) the bias of the said estimator. 

The first question we should ask ourselves before continuing the study is: ``\emph{What do we desire from guarantees in expectation?}''. If the purpose is to find an estimator with a small bias, we could evaluate the loss on a single sample and obtain an unbiased estimator. The objective with these kinds of bounds is to gain \emph{understanding}. The reason why the estimator taken is the empirical risk is no coincidence, as most learning algorithms return a hypothesis trying to minimize this objective. Therefore, we seek to know what drives apart the performance of an algorithm on its objective (the population risk) with respect to the proxy it uses to return a hypothesis (the empirical risk). Bounds in expectation, due to their simplicity, allow us to find the average performance discrepancy more easily and get a better idea of a problem. For example, through the lens of the uniform stability of~\Cref{subsec:uniform_stability}, we know that, under certain conditions on the loss and the parameter's space, the parameters of gradient descent can be chosen so that the expected generalization rate is in $\cO(\nicefrac{1}{n})$ (see \Cref{subsec:gd_in_clb_setting}). 

A second question should be: \emph{``Why do we study information-theoretic bounds?''}. The reason has been mostly outlined in~\Cref{ch:generalization_theory}, particularly in~\Cref{sec:information_theoretic_generalization}, and it is because this framework can give us guarantees that are specific to the learning algorithm and the data distribution. This greatly differs from previous frameworks where the guarantees are given \emph{uniformly} to a whole class of hypothesis classes and data distributions like the ones in~\Cref{sec:pac_learning,sec:rademacher_complexity}. As we will see shortly in~\Cref{sec:bounds_using_mutual_information}, a common information-theoretic bound states that the generalization error of an algorithm evaluated on a loss with a bounded range is bounded from above by $$\bE \big[ \gen(W;S) \big] \leq (b-a) \sqrt{\frac{\minf(W;S)}{2n}},$$ where $n$ is the number of samples and $\minf(W;S)$ is the mutual information that the algorithm's output hypothesis $W$ has about the training data $S$ (see \Cref{fig:algorithm_as_channel_and_backward_channel}). If, in our problem, the data distribution $\bP_Z$ belongs to a well-behaved class, we may work out how much information the algorithm will capture about the data and find a \emph{specific} bound in that situation, even though the bound does not hold \emph{uniformly} over all algorithms and all data distributions. Moreover, a bound of this kind gives a better understanding of what can hinder the generalization performance of an algorithm: the more information about the data that is encoded into the output of an algorithm, the worse it generalizes.

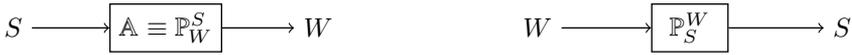
\begin{figure}[t]
\centering
\begin{subfigure}[b]{0.45\textwidth}
    \centering
    \begin{tikzpicture}[->, line width=0.5pt, node distance=2cm, scale=1]]

		\node (S) at (0,0) {$S$};
		\node[draw, rectangle, minimum width=1cm] (algo) at (2,0) {$\bA \equiv \bP_W^S$};
		\node (W) at (4,0) {$W$};
	
	    \path (S) edge (algo);
	    \path (algo) edge (W);
    \end{tikzpicture}
\end{subfigure}
\hfill
\begin{subfigure}[b]{0.45\textwidth}
    \centering
    \begin{tikzpicture}[->, line width=0.5pt, node distance=2cm, scale=1]]

		\node (W) at (0,0) {$W$};
		\node[draw, rectangle, minimum width=1cm] (reverse) at (2,0) {$\bP_S^W$};
		\node (S) at (4,0) {$S$};
	
	    \path (W) edge (reverse);
	    \path (reverse) edge (S);
    \end{tikzpicture}
\end{subfigure}

\caption{\looseness=-1 Illustration of a learning algorithm $\bA$ viewed as a channel processing a dataset $S$ to obtain a hypothesis $W$ (left), and of the backward channel describing the processing of a hypothesis $W$ to obtain the dataset with which it was trained (right).}
\label{fig:algorithm_as_channel_and_backward_channel}
\end{figure}

With this motivation in mind, the chapter starts in~\Cref{sec:bounds_using_mutual_information} describing bounds using mutual information as their main contributor. This section is based on the results from~\citet{xu2017information} and~\citet{bu2020tightening} in~\Cref{subsec:slow_rate_cmi}, and our results from~\citep{rodriguez2023morepac,rodriguez2024moments} in~\Cref{subsec:a_fast_rate_bound_mi,subsec:losses_with_bounded_moment}, including some new reflections that are not published. Then, the chapter continues in~\Cref{sec:bounds_using_conditional_mutual_information}, describing bounds using conditional mutual information as the main element and it is based on the results from~\citet{steinke2020reasoning} and \citet{hellstrom2022new}. This conditional mutual information shifts the focus from how much information the algorithm's output has about the training data to how much we can identify the training data based on the algorithm's output. This change in perspective will prove relevant.

After that, in~\Cref{sec:random_subset_and_single_letter}, we describe some variants of the bounds in the previous two sections that have proven useful for gaining insights into the elements that hinder generalization: single-letter and random-subset bounds. The first kind notes that we can focus on the information that the algorithm's output has \emph{individually} about each data instance, and the second kind allows us to leverage some of the problem's information to design problem-specific bounds later on. This section is based mostly on the works from~\citet{bu2020tightening}, \citet{negrea2019information}, \citet{haghifam2020sharpened}, and our results in~\citep{rodriguez2020randomsubset,rodriguez2021tighter}. 

In \Cref{sec:bounds_using_wasserstein_distance}, we abstract the formula to obtain the bounds so far and present bounds using the Wasserstein distance. These bounds serve two purposes. First, they allow us to consider the geometry of the hypothesis space \emph{and} the dependence between the algorithm's output and the training set. Second, they allow us to re-discover most of the bounds in the previous \Cref{sec:bounds_using_mutual_information,sec:bounds_using_conditional_mutual_information,sec:random_subset_and_single_letter} with different assumptions and to prove that these new bounds are tight for non-trivial situations. This section is mainly based on our results in~\citep{rodriguez2021tighter} and part of our results in~\citep{haghifam2023limitations}. 

Then, in \Cref{sec:noisy_iterative_learning_algos}, we describe an application of the mutual information and conditional mutual information-based bounds in the context of noisy, iterative learning algorithms. We show how the presented bounds can be used to find factors that may drive the generalization performance of the stochastic gradient Langevin dynamics and stochastic gradient descent. This chapter is based on the results from~\citet{pensia2018generalization}, \citet{bu2020tightening}, \citet{negrea2019information}, \citet{haghifam2020sharpened}, our results from~\citep{rodriguez2020randomsubset}, and \citet{neu2021information}. 

In \Cref{sec:futher_advances_bounds_using_information_measures}, we describe further advances in information-theoretic generalization error bounds. This section is mainly based on the results from many other works~\citep{steinke2020reasoning,haghifam2021towards,hellstrom2022new,harutyunyan2021information,Wang2023TighterIG,wang2024sample,asadi2018chaining,asadi2020chaining}, with the exception of~\Cref{th:ecmi}, which our result from~\citep{haghifam2023limitations}. Finally, in~\Cref{sec:limitations_bounds_using_mutual_information}, we challenge the bounds described throughout the chapter and show that, despite being tight, there are problems where frameworks like uniform stability from~\Cref{subsec:uniform_stability} can guarantee that an algorithm generalizes, but that the information-theoretic bounds previously discussed fail. This last section is based on our results from~\citep{haghifam2023limitations}.

The chapter intends to be didactic and convey the key messages and results from the current information-theoretic bounds on the expected generalization error. This is why many other works are included together with our own. This is also the reason why many of both the other works and ours are  (i) written in a different way than what can be found in the papers, (ii) contain different, more pedagogical proofs, and/or (iii) are more general and extend the published results. We hope the reader can benefit from this chapter, which complements the recent survey from~\citet {hellstrom2023generalization}.

\section{Bounds Using Mutual Information}
\label{sec:bounds_using_mutual_information}

In this section, we will consider the problem of bounding the generalization error from above. To do so, we will consider the analogous problem of bounding the bias of the empirical risk in estimating the population risk when the hypothesis is obtained using the same training set used to calculate the empirical risk. This framework inspired \citet{xu2017information} to adapt the work from \citet{russo2016controlling} studying the bias in generic adaptive data analysis that later, to the context of generalization theory.\footnote{Interestingly, essentially the same building result appeared previously in~\citep[Exercise 4.13]{boucheron2003concentration} in an abstract manner, in the context of information inequalities.} Their derived bounds depend on the amount of information that the hypothesis returned by the algorithm $W$ has about the training set $S$ that it received, that is, $\minf(W'S)$. Throughout the section, we will discuss the implications of such an information measure and different generalization bounds that can be derived depending on the generality of the assumptions on the loss function $\ell$.

In~\Cref{subsec:a_slow_rate_bound_mi}, after understanding the ideas from~\citet{xu2017information} for losses with a bounded range in~\Cref{subsec:a_slow_rate_bound_mi}, we will see~\citet{bu2020tightening}' generalization the results to losses with a bounded CGF. These bounds are characterized by their slow rate in $\cO(\sqrt{\nicefrac{\minf(W;S)}{n}})$. Then, we will present our~\citep{rodriguez2024moments} tighter bounds when the loss has a bounded range along with \citet{catoni2007pac}'s results in this setting in~\Cref{subsec:a_fast_rate_bound_mi}. These bounds are characterized by their fast rate in $\cO(\nicefrac{\minf(W;S)}{n})$. Finally, in~\Cref{subsec:losses_with_bounded_moment}, we will conclude discussing our results from~\citep{rodriguez2024moments} showing how we can interpolate between fast rate and slow-rate bounds depending on which moments of the loss function are bounded. In particular, when only the second moment is bounded, we can achieve a slow-rate bound, and in the limit, when all moments are bounded, we recover a fast-rate bound.

\subsection{A Slow-Rate Bound: Losses With a Bounded CGF}
\label{subsec:a_slow_rate_bound_mi}

Consider a random hypothesis $W'$ distributed according to the data-independent distribution $\bP_{W'} = \bQ$. The expected generalization error of this hypothesis $W'$ is $\bE [ \gen(W',S) ] = 0$ since the instances are identically distributed and hence the empirical risk is an unbiased estimator of the population risk, namely
\begin{equation*}
	\bE [ \emprisk(W',S) ] = \frac{1}{n} \sum_{i=1}^n \bE_{w' \sim \bQ} \big[ \bE[ \ell(w',Z_i) ] \big] = \bE[\poprisk(W')].
\end{equation*}

However, by construction, a learning algorithm returns a hypothesis sampled from the conditional distribution $\bP_W^S$, and therefore, it depends on the training data $S$. If the empirical risk is considered an estimator of the population risk, then this dependence makes it potentially biased.

A way around this bias problem is to do a change of measure and study the generalization error of the data-independent hypothesis $W'$. A change of measure is a result that relates the expectation of two random objects with distinct distributions and a correction term that determines how far apart these two distributions are. The canonical example is the change of measure~\eqref{eq:change_of_measure} resulting from the Radon--Nikodym~\Cref{th:radon-nikodym}. Nonetheless, now we will employ the Donsker and Varadhan~\Cref{lemma:dv_and_gvp}. In~\eqref{eq:dv}, let the measurable space be $\cW \times \cZ^{\otimes n}$, the distributions be $\bP_{W,S}$ and $\bQ \otimes \bP_S$, and the function be $\lambda \gen(w,s)$ for some $\lambda \in \bR$; then
\begin{equation}
	\label{eq:dv_to_lambda_gen_prev}
	\relent(\bP_{W,S} \Vert \bQ \otimes \bP_S) \geq \lambda \bE \big[ \gen(W,S) \big] - \log \bE \Big[ e^{\lambda \gen(W',S)} \Big].
\end{equation}
Re-arranging the equation, we obtain an upper bound of the expected generalization error that holds for every $\lambda > 0$, namely
\begin{equation}
	\label{eq:dv_to_lambda_gen}
	\bE \big[ \gen(W,S) \big] \leq \frac{1}{\lambda} \bigg( \relent(\bP_{W,S} \Vert \bQ \otimes \bP_S)  + \log \bE \Big[ e^{\lambda \gen(W',S)} \Big] \bigg).
\end{equation}

Dissecting~\eqref{eq:dv_to_lambda_gen}, we note that the Donsker and Varadhan~\Cref{lemma:dv_and_gvp} served effectively to decouple the hypothesis $W$ and the data $S$. The effects of this decoupling are manifested in the right-hand side of the equation:
\begin{enumerate}
    \item The dependence between the two random objects is now captured by the relative entropy $\relent(\bP_{W,S} \Vert \bQ \otimes \bP_S)$, where $\bP_{W,S}$ is the joint distribution (or coupling) between $W$ and $S$, and where $\bQ \otimes \bP_S$ is the distribution between the data $S$ and the independent (or decoupled) hypothesis $W'$. The larger the dependence, the less effective the decoupling lemma, and the worse the upper bound.
    \item The concentration of the decoupled function is now captured by $\log \bE \big[ \exp \big( \lambda \gen(W',S) \big) \big]$. Since under the decoupled distribution $\bE \big[ \gen(W',S) \big] = 0$, this term is equivalent to the CGF $\Lambda_{\gen(W',S)}(\lambda)$. The CGF of a random variable is intimately connected to its tails through the Cramér--Chernoff method~\citep[Section 2.2]{boucheron2003concentration}. The faster the CGF grows as $\lambda$ increases, the less the random variable $\gen(W',S)$ concentrates around zero. This makes this conversion less helpful and the upper bound larger.
\end{enumerate}

\subsubsection{Warm-Up: A Slow-Rate Bound for Losses With a Bounded Range}

A requirement for the Donsker and Varadhan~\Cref{lemma:dv_and_gvp} and~\eqref{eq:dv_to_lambda_gen} to hold is that the MGF of $\gen(W',S)$ is bounded. We will deal with this more generally below. For now, it suffices to know that if the loss has a range contained in $[a,b]$  and the instances are independent and identically distributed, then $\log \bE \Big[ e^{\lambda \gen(W',S)} \Big] \leq \nicefrac{\lambda^2 (b-a)^2}{8 n}$. Therefore, \Cref{eq:dv_to_lambda_gen} can be further bounded by
\begin{equation*}
	\bE \big[ \gen(W,S) \big] \leq \frac{\relent(\bP_{W,S} \Vert \bQ \otimes \bP_S)}{\lambda}  + \frac{\lambda (b-a)^2}{8n}
\end{equation*}
for all $\lambda > 0$. The right-hand side of this equation is convex with respect to $\lambda$; therefore, it has a minimum that tightens the bound. The optimization of the parameter $\lambda$ results in the following theorem due to~\citet{xu2017information}.

\begin{theorem}[{\citet[Theorem 1, adapted to bounded losses]{xu2017information}}]
\label{th:mi_bound_bounded}
	Consider a loss function with a range contained in $[a,b]$. Then, for every data-independent distribution $\bQ$ on $\cW$
	\begin{equation*}
		\bE \big[ \gen(W,S) \big] \leq (b-a) \sqrt{ \frac{\relent(\bP_{W,S} \Vert \bQ \otimes \bP_S)}{2n}}
	\end{equation*}
	In particular, choosing $\bQ$ to be the marginal distribution $\bP_W = \bP_W^S \circ \bP_S$ is optimal and
	\begin{equation*}
		\bE \big[ \gen(W,S) \big] \leq (b-a) \sqrt{ \frac{\minf(W;S)}{2n}}.
	\end{equation*}
\end{theorem}

The optimality of the choice of the prototypical distribution on samples from the data distribution $\bP_W = \bP_W^S \circ \bP_S$ comes from the golden formula from~\Cref{prop:properties_minf} and the theorem holds for distributions $\bQ$ such that $\bP_{W,S} \not \ll \bQ \otimes \bP_S$ by the convention that in that case $\relent(\bP_{W,S} \Vert \bQ \otimes \bP_S) \to \infty$. For this reason, after this subsection, all results in this chapter will be presented in terms of mutual information only. This result was particularly celebrated since it provides us with a simple relationship between the generalization error of a learning algorithm and the mutual information between the algorithm's input (the training data $S$) and its output hypothesis $W$. Moreover, this relationship resembles the classical relationships provided by the uniform convergence property of a hypothesis class $\cW$ where the upper bound on the generalization error was of the order $\sqrt{\nicefrac{\mathfrak{C}}{n}}$, where $\mathfrak{C}$ represents the complexity of the class $\cW$ measured by, for example, the VC or the Natarajan dimension (\Cref{sec:pac_learning}). In this case, however, the complexity term describes the algorithm and its relationship with the data and is represented by the mutual information, suggesting that the more the algorithm depends on the data, the less it generalizes.
If the hypothesis returned by the algorithm is independent of the data and $\bP_W^S = \bQ$ a.s.; then, the mutual information is 0 and the generalization error is also 0, as we expected from the start. On the other hand, if the returned hypothesis is a deterministic function of the training set $\bP_W^S = \delta_{f(S)}$ and both the hypothesis $W$ and the instances $Z_i$ are continuous, then the mutual information tends to infinity and the bound is vacuous.

If we consider the algorithm as a channel processing a training set $S$ to obtain a hypothesis $W$ (\Cref{fig:algorithm_as_channel_and_backward_channel} left), the mutual information describes the average inefficiency in trying to encode the hypothesis $W$ without employing the training set $S$, when this set is known. That is, it describes how much knowing the training set $S$ reduces the uncertainty about the hypothesis $W$ on average with respect to the data distribution. 

Due to the symmetry of the mutual information (\Cref{prop:properties_minf}), the reverse is also true. Consider the backward channel $\bP_S^W$ that is naturally defined from the joint distribution $\bP_{W,S}$. Then, the mutual information describes the average inefficiency in trying to encode the dataset $S$ without employing the hypothesis $W$ when this hypothesis is known. That is, how much does knowing the hypothesis $W$ reduces the uncertainty of the dataset $S$ on average with respect to the prototypical hypothesis distribution. In particular, for discrete data and hypotheses, it amounts to the average extra bits needed to encode the training set if the returned hypothesis is not used. Therefore, the hypothesis contains information on the training set. Considering this rather artificial channel already elucidates some connections between privacy and generalization: if an algorithm is private, the information that an algorithm contains on the training set should be small, and therefore $\minf(W;S)$ should be bounded from above. This intuition will be formalized later in~\Cref{ch:privacy_and_generalization}.

\subsubsection{Losses With a Bounded CGF}

The technique described above is not restricted to bounded losses. As mentioned before, a requirement for the Donsker and Varadhan~\Cref{lemma:dv_and_gvp} and~\eqref{eq:dv_to_lambda_gen} to hold is that the MGF of $\gen(W',S)$ is bounded, that is, that $\bE \Big[ e^{\lambda \gen(W',S)} \Big] < \infty$. Since we are actually interested in controlling the CGF of $\gen(W',S)$ as per~\eqref{eq:dv_to_lambda_gen}, it seems logical to consider losses with a bounded CGF. This includes unbounded losses with sufficiently behaved tails, that is, losses that have a probability of taking large values that vanishes exponentially fast.

\begin{definition}[Bounded CGF] 
\label{def:bounded_cgf}
	A loss function $\ell: \cW \times \cZ$ is of \emph{bounded CGF} with respect to the data distribution $\bP_Z$ (or just of bounded CGF, if the data distribution is clear from the context) if, for every hypothesis $w \in \cW$, there is a convex and continuously differentiable function $\psi(\lambda)$ defined on $[0, b)$ for some $b \in \bR_+$ such that $\psi(0) = \psi'(0) = 0$ and $\Lambda_{- \ell(w,Z)} \leq \psi(\lambda)$ for all $\lambda \in [0,b)$. 
\end{definition}

For a fixed hypothesis $w \in \cW$, using the algebra of CGFs from~\Cref{subsec:moments_and_generating_functions} yields that 
\begin{equation*}
	\Lambda_{\gen(w,S)}(\lambda) = \sum_{i=1}^n \Lambda_{-\ell(w,Z_i)}\Big(\frac{\lambda}{n}\Big) \leq n \psi \Big(\frac{\lambda}{n}\Big).
\end{equation*}
for all $\lambda \in [0,n b)$. Therefore, the CGF of $\gen(W',S)$ can be bounded by
\begin{equation}
\label{eq:cgf_algebra}
	\log \bE \Big[ e^{\lambda \gen(W',S)} \Big] = \log \bE_{w \sim \bQ} \bigg[ \bE \Big[ e^{\lambda \gen(w,S)} \Big] \bigg] \leq n \psi \Big(\frac{\lambda}{n}\Big)
\end{equation} 
for all $\lambda \in [0,nb)$. Combining this upper bound and the change of variable $\lambda = \lambda' n$ with~\eqref{eq:dv_to_lambda_gen} results in 
\begin{equation*}
	\bE \big[ \gen(W,S) \big] \leq \frac{1}{\lambda'} \bigg( \frac{\relent(\bP_{W,S} \Vert \bQ \otimes \bP_S)}{n}  -  \psi(\lambda') \bigg),
\end{equation*}
which holds for all $\lambda' \in [0,b)$. Finally, optimizing this equation using~\Cref{lemma:boucheron_convex_conjugate_inverse} from~\Cref{subsec:convex_conjugate} gives us the following guarantee in expectation.

\begin{theorem}[{\citet[Theorem 2 and Proposition 2]{bu2020tightening}}]
\label{th:mi_bound_general_cgf}
	Consider a loss function with a bounded CGF (\Cref{def:bounded_cgf}). Then, for every data-independent distribution $\bQ$ on $\cW$ 
	\begin{equation*}
		\bE \big[ \gen(W,S) \big] \leq \psi_*^{-1} \bigg( \frac{\relent(\bP_{W,S} \Vert \bQ \otimes \bP_S)}{n}  \bigg).
	\end{equation*}
	In particular, choosing $\bQ$ to be the marginal distribution $\bP_W = \bP_W^S \circ \bP_S$,
	\begin{equation*}
		\bE \big[ \gen(W,S) \big] \leq \psi_*^{-1} \bigg( \frac{\minf(W;S)}{n}  \bigg).
	\end{equation*}
\end{theorem}

The inverse of the convex conjugate of the function dominating the CGF $\psi_*^{-1}$ characterizes the tail behavior of the loss function through the Cramér--Chernoff method~\citep[Section 2.2]{boucheron2003concentration} as we will see in~\Cref{sec:classical_concentration_inequalities}. For now, it is enough to know that $\psi_*^{-1}$ is a non-decreasing function such that $\psi_*^{-1}(0) = 0$ and $\lim_{y \to \infty} \psi_*^{-1}(y) \to \infty$, and that, for a fixed hypothesis $w$, for all $\beta \in (0,1)$, with probability no smaller than $1 - \beta$
\begin{equation*}
	\poprisk(w) \leq \emprisk(w,S) + \psi_*^{-1} \bigg( \frac{\log \frac{1}{\beta}}{n} \bigg) .
\end{equation*}
That is, the rate at which the empirical risk of a fixed hypothesis concentrates around the population risk is at least $\psi_*^{-1}(\nicefrac{1}{n})$. With this in mind, \Cref{th:mi_bound_general_cgf} becomes clearer. The bias of the empirical risk estimator of the population risk decreases at the rate of concentration of the said estimator for a fixed hypothesis, inflated by the dependence between the hypothesis and the training set, that is $\psi_*^{-1} \big( \nicefrac{\minf(W;S)}{n} \big)$. As previously, if the hypothesis is independent of the training data, then $\minf(W;S) = 0$ and the theorem indicates that the empirical risk is an unbiased estimator of the population risk. On the other hand, for continuous hypotheses obtained deterministically from continuous training data $\minf(W;S) \to \infty$ and the bound is vacuous.

There are several tail behaviors that are sufficiently common to have a proper name. Below, we discuss three of them: \emph{sub-gamma}, \emph{sub-exponential}\footnote{Here, we are considering the sub-exponential characterization of random variables from~\citet[Theorem 2.13]{wainwright2019high} and not the one given by~\cite[Excercise 2.22]{boucheron2003concentration}.}, and \emph{sub-Gaussian}. Loosely speaking, they represent losses whose tail is lighter than a gamma, exponential, and Gaussian random variable respectively. As a result, these three behaviors have in common that the generalization error bound has a dominating term that vanishes at a slow rate $\sqrt{\nicefrac{\minf(W;S)}{n}}$.

\paragraph{Sub-Gamma Losses.} A loss is \emph{$(\sigma^2, c)$-sub-gamma} if it has a CGF bounded by $\psi(\lambda) = \nicefrac{\lambda^2 \sigma^2}{2(1-c\lambda)}$ for all $\lambda \in (0, \nicefrac{1}{c})$ in the sense of~\Cref{def:bounded_cgf}. Then, since $\psi_*^{-1}(y) = \sqrt{2 \sigma^2 y} + cy$~\citep[Section 2.4]{boucheron2003concentration}, a corollary of~\Cref{th:mi_bound_general_cgf} is that if the loss is $(\sigma^2, c)$-sub-gamma, then
\begin{equation*}
	\bE \big[ \gen(W,S) \big] \leq \sqrt{2 \sigma^2  \cdot \frac{\minf(W;S)}{n}} + c \cdot \frac{\minf(W;S)}{n},
\end{equation*}
combining a slow and a fast rate of bias decrease.

\paragraph{Sub-Exponential Losses.} A loss is \emph{$(\sigma^2, c)$-sub-exponential} if it has a CGF bounded by $\psi(\lambda) = \nicefrac{\lambda^2 \sigma^2}{2}$ for all $\lambda \in (0, \frac{1}{c})$ in the sense of~\Cref{def:bounded_cgf}. Therefore,
\begin{equation}
\psi_*^{-1}(y) = 
\begin{cases}
	\sqrt{2 \sigma^2 y} & \textnormal{if } \lambda = \sqrt{\frac{2 y}{\sigma^2}} \leq \frac{1}{c} \\
	cy + \frac{\sigma^2}{2c^2} & \textnormal{otherwise}
\end{cases}.
\end{equation}
As we show in~\cite[Appendix B.4]{rodriguez2023morepac}, the condition for the first case may be re-written as $y \leq \nicefrac{\sigma^2}{2c^2}$ and similarly the condition for the second case as $y > \nicefrac{\sigma^2}{2c^2}$, which simplifies the equation to
\begin{equation}
\psi_*^{-1}(y) = 
\begin{cases}
	\sqrt{2 \sigma^2 y} & \textnormal{if } y \leq \nicefrac{\sigma^2}{2c^2} \\
	(c+1) y & \textnormal{otherwise}
\end{cases}.	
\end{equation}
In this way, a corollary of~\Cref{th:mi_bound_general_cgf} is that if the loss is $(\sigma^2, c)$-sub-exponential, then
\begin{equation*}
	\bE \big[ \gen(W,S) \big] \leq 
	\begin{cases}
		\sqrt{2 \sigma^2 \cdot \frac{\minf(W;S)}{n}} & \textnormal{if } \minf(W;S) \leq \frac{n \sigma^2}{2 c^2} \\
		(c+1) \cdot \frac{\minf(W;S)}{n} & \textnormal{otherwise}
	\end{cases},
\end{equation*}
which also combines a slow and a fast rate of bias decrease. For smaller values of the dependence $\minf(W;S)$, the slow rate dominates; while for larger values the fast rate dominates.

\paragraph{Sub-Gaussian Losses.} A loss is \emph{$\sigma^2$-sub-Gaussian} if it has a CGF bounded by $\psi(\lambda) = \nicefrac{\lambda^2 \sigma^2}{2}$ for all $\lambda \in \bR$ in the sense of~\Cref{def:bounded_cgf}. Then, since $\psi_*^{-1}(y) = \sqrt{2 \sigma^2 y}$, a corollary of~\Cref{th:mi_bound_general_cgf} is that if the loss is $\sigma^2$-sub-Gaussian, then
\begin{equation}
\label{eq:mi_bound_subgaussian}
	\bE \big[ \gen(W,S) \big] \leq \sqrt{2 \sigma^2  \cdot \frac{\minf(W;S)}{n}}.
\end{equation}
This was, in fact, the theorem presented by~\citet{xu2017information}. \Cref{th:mi_bound_bounded} is a corollary of this result using the property that if a loss has a range contained in $[a,b]$ a.s., then the loss is $\nicefrac{(b-a)^2}{4}$-sub-Gaussian.

\subsection{A Fast-Rate and a Mixed-Rate Bound: Losses With a Bounded Range} 
\label{subsec:a_fast_rate_bound_mi}

In the subsection above, we noted that for losses with a range contained in $[a,b]$, the bias of the empirical risk estimator of the population risk vanishes at a slow rate $\sqrt{\nicefrac{\minf(W;S)}{n}}$. However, we can benefit from using a different estimator. If we consider a linear function of the empirical risk as our estimator, we can find a bound that vanishes at a fast rate $\nicefrac{\minf(W;S)}{n}$. This advantage comes from considering a tighter characterization of the CGF of bounded losses than the one provided by their sub-Gaussian nature. Without loss of generality, in this subsection, we will consider losses with a range bounded in $[0,1]$. If, instead, the loss has a range bounded in $[a,b]$, the bound can be shifted by $-a$ and scaled by $\nicefrac{1}{(b-a)}$.

\begin{theorem}
\label{th:mi_bound_bounded_fast_rate}
Consider a loss function with a range contained in $[0,1]$. Then, for every $c  \in (0,1]$ and every $\gamma > 1$,
\begin{equation*}
	\bE \big[ \poprisk(W) \big] \leq c \gamma \log \left( \frac{\gamma}{\gamma - 1} \right) \cdot \bE \big[ \emprisk(W, S) \big] + c \gamma \cdot \frac{\minf(W;S)}{n} + \gamma \kappa(c),
\end{equation*}
where $\kappa(c) \coloneqq 1 - c(1 - \log c)$.
In particular, for interpolating algorithms (that is, when $\bE^S [ \emprisk(W,S) ] = 0$), the optimal parameters are $\gamma \to 1$ and $c = \exp\big({-\frac{\minf(W;S)}{n}}\big)$ and then
\begin{equation*}
	\bE \big[ \poprisk(W) \big] \leq 1 - e^{- \frac{\minf(W;S)}{n}} \leq \frac{\minf(W;S)}{n}.
\end{equation*}
\end{theorem}

This bound shows that the population risk can be bounded from above by a linear combination of the empirical risk and the normalized complexity measured by $\nicefrac{\minf(W;S)}{n}$, where the coefficients for each term and the bias can change depending of the values of each term and ensure that the bound is always in $[0,1]$. 

The parameter $\gamma$ controls the influence of the empirical risk compared to the normalized complexity: if the empirical risk is large relative to the normalized complexity, then $\gamma$ is larger and the normalized complexity coefficient increases, if instead the empirical risk is small or even close to interpolation, then $\gamma$ is close to $1$ and the empirical risk coefficient increases. In particular, for a fixed value of $c$, the optimal value of $\gamma$ is
\begin{align*}
	\gamma &= 1 + \left[ -1 - \mathtt{W} \left( - \exp \left(-1 - \frac{c\cdot \frac{\minf(W;S)}{n} + \kappa(c)}{c \cdot \bE \big[ \emprisk(W,S) \big]} \right) \right) \right]^{-1} \\
	&\approx 1 + \left[ \sqrt{2 \cdot \frac{c\cdot \frac{\minf(W;S)}{n} + \kappa(c)}{c \cdot \bE \big[ \emprisk(W,S) \big]}} +  \frac{5}{6} \cdot \frac{c\cdot \frac{\minf(W;S)}{n} + \kappa(c)}{c \cdot \bE \big[ \emprisk(W,S) \big]}\right],
\end{align*}
\looseness=-1 where $\mathtt{W}$ is the Lambert W function and the -1 branch is approximated following~\citep{chatzigeorgiou2013bounds}.

The parameter $c \in (0,1]$ controls how much weight is given to the empirical risk and the normalized complexity terms compared to the bias. For larger values of the empirical risk and the normalized complexity term, the value of $c$ is small, decreasing their contribution to the bound and increasing the contribution of the bias $\kappa(c) \in [0,1)$. If the empirical risk and the normalized complexity term are smaller, then the value of $c$ approaches $1$, where the contribution of these two terms is only controlled by $\gamma$ and the bias is $0$. In fact, a weaker version of~\Cref{th:mi_bound_bounded_fast_rate} can be obtained considering this small empirical risk and small normalized complexity regime by letting $c = 1$.

\begin{corollary}
\label{cor:mi_bound_bounded_fast_rate_weaker}
Consider a loss function with a range contained in $[0,1]$. Then, for every  $\gamma > 1$,
\begin{equation*}
	\bE \big[ \poprisk(W) \big] \leq \gamma \log \Big( \frac{\gamma}{\gamma - 1} \Big) \cdot \bE \big[ \emprisk(W, S) \big] + \gamma \cdot \frac{\minf(W;S)}{n}.
\end{equation*}
\end{corollary}

This weaker version of the result was discovered by~\citet[Theorem 1.2.1]{catoni2007pac} and presented in another form using a parameter $\lambda > 0$ such that $\lambda = n \log \nicefrac{\gamma}{(\gamma - 1)}$ around a decade prior to the \citet{xu2017information}'s \Cref{th:mi_bound_bounded}, but it did not receive as much attention. In part, this was because the implications of such a result and the nice properties of the mutual information $\minf(W;S)$ to describe the complexity of an algorithm were not emphasized. As in~\Cref{th:mi_bound_bounded_fast_rate}, for interpolating algorithms, the optimal parameter is $\gamma \to 1$ and the bound has a fast rate $\nicefrac{\minf(W;S)}{n}$. This is reminiscent of the classical generalization guarantees from uniform convergence in the realizable setting for classification problems (see~\Cref{subsec:vc_and_natarajan_dimensions}).

\citet[Theorem 1.2.1]{catoni2007pac} also presented an equivalent bound to \Cref{th:mi_bound_bounded_fast_rate}, although in a different form that does not make evident its fast-rate nature. Throughout this subsection, we will show that the two results are equivalent and how to obtain one from the other. We will focus the exposition following our proof of the result as it follows naturally from  an application of a variant of~\Cref{th:mi_bound_general_cgf}.
To prove \Cref{th:mi_bound_bounded_fast_rate}, we need to first introduce the following result~\citep{maurer2004note}.

\begin{lemma}[{\citet[Lemma 3]{maurer2004note}}]
\label{lemma:convex_dominated_by_bernoulli}
	Let $X$ be a random variable taking values in $[0,1]$ and $Y$ be the unique Bernoulli random variable such that $\bP_Y[1] = 1 - \bP_Y[0] = \bE[Y] = \bE[X]$. Then, if $f: [0,1] \to \bR$ is convex, then $\bE \big[ f(X) \big] \leq \bE \big[ f(Y) \big]$.
\end{lemma}

Consider a loss $\ell(W',Z)$ with a range in $[0,1]$ and let $L$ be the unique Bernoulli random variable with $\bP_{L}[1] = \bE[\ell(W',Z)] = \bE[L]$, where $W'$ is distributed according to $\bP_W$. Now, let $f(l;\lambda) = e^{- \lambda (l - \bE[L])}$ be a convex function on $l$ for all $\lambda \in \bR$. Then, it follows from~\Cref{lemma:convex_dominated_by_bernoulli} that the CGF of $-\ell(W',Z)$ is bounded from above by the mirror image of the CGF of $L$, that is, for all $\lambda \in \bR$
\begin{equation*}
	\Lambda_{- \ell(W',Z)}(\lambda) = \log \bE \left[ e^{- \lambda (\ell(W',Z) - \bE[\ell(W',Z)])} \right] \leq \log \bE \left[ e^{- \lambda (L - \bE[L])} \right] = \Lambda_{L}(- \lambda).
\end{equation*}

In this way, since the CGF of a Bernoulli random variable is known~\citep[page 23]{boucheron2003concentration}, it follows that $\Lambda_{- \ell(W',Z)}(\lambda) \leq \psi(\lambda) \coloneqq \lambda \bE[\poprisk(W)] + \log \big(1 - \bE[\poprisk(W)](1- e^{-\lambda}) \big)$ for all $\lambda \in \bR$, where we noted that $\bE[L] = \bE[\poprisk(W)]$. This means that the CGF of the loss is bounded in a weaker sense than in~\Cref{def:bounded_cgf}, as the random loss considered takes into account the random hypothesis $W$ and the upper bound depends on its marginal distribution $\bP_W$. However, it is this weakening of the condition which allows us to obtain a tighter guarantee following a similar procedure.
In the previous section, we mentioned that requiring that the CGF of an unbounded loss $-\ell(W',Z)$ is, for example, sub-Gaussian is complicated since we need to know the distribution of the marginal distribution of the hypothesis $\bP_W$. Hence, we resorted to the stronger requirement that the CGF of the unbounded loss $- \ell(w,Z)$ is sub-Gaussian for all $w \in \cW$. On the other hand, now we can make use of this weaker assumption since it follows from the boundedness of the loss.

Similarly to before, using the algebra of the CGFs to bound the CGF of $\gen(W',S)$ as in~\eqref{eq:cgf_algebra} and combining it with~\eqref{eq:dv_to_lambda_gen_prev} with the change of variable $\lambda = \lambda' n$ results in
\begin{equation*}
	\psi_* \Big( \bE \big[ \gen(W,S) \big] \Big) \leq \frac{\minf(W;S)}{n},
\end{equation*}
where the left hand side is obtained by optimizing $\lambda \bE \big[ \gen(W,S) \big] - \psi(\lambda)$ and by the  definition of the convex conjugate~\Cref{def:convex_conjugate}. Operating similarly to \cite[page 23]{boucheron2003concentration} we see that for every $t \in \big[0, \bE[\poprisk(W)]\big]$,
\begin{equation*}
	\psi_*(t) = \relentber \Big( \bE \big[\poprisk(W) \big]- t  \big \Vert \bE \big[\poprisk(W) \big] \Big),
\end{equation*}
\looseness=-1 where $\relentber(\hat{r} \Vert r) \coloneqq \relent \big( \mathrm{Ber}(\hat{r}) \Vert \mathrm{Ber}(r) \big)$, leading us the next intermediate result.

\begin{lemma}
\label{lemma:small_kl_mi}
Consider a loss function with a range contained in $[0,1]$. If $\bE \big[ \emprisk(W,S) \big] \leq \bE \big[ \poprisk(W) \big]$, then
\begin{equation*}
	\relentber \Big( \bE \big[\emprisk(W,S)\big] \big \Vert \bE \big [\poprisk(W) \big] \Big) \leq \frac{\minf(W;S)}{n}.
\end{equation*}
\end{lemma}

This lemma is the ``in expectation'' analogue to the Seeger--Langford~\cite{seeger2002pac,langford2001bounds} PAC-Bayes bound. This suggests that the difference between the empirical risk estimation of the population risk decreases at a fast rate $\nicefrac{\minf(W;S)}{n}$ when measured with the metric $\relentber$. However, from this result, it is hard to obtain a bound on the population risk that we can interpret, as the generalized inverse of $\relentber$ does not have an analytical form.  This issue can be resolved by employing the variational representation of the relative entropy borrowed from $f$-divergences (\Cref{cor:variational_representation_relative_entropy}) to $\relentber (\hat{r}, r)$. Let $g_0 = g(0)$ and $g_1 = g(1)$, then for all $g_0$ and $g_1$ in $(-\infty, \infty)$,
\begin{align}
	\relentber \big( \hat{r} \Vert r \big) &\geq 1 + \hat{r} g_1 +  ( 1- \hat{r} ) g_0 - r e^{g_1} - ( 1 - r ) e^{g_0} \nonumber \\
	\label{eq:relentber_f_div_before_change_variable}
	&\geq 1 + g_0 + (g_1 - g_0) \hat{r} - e^{g_0} + r (e^{g_0} - e^{g_1})
\end{align}

For~\eqref{eq:relentber_f_div_before_change_variable} to be relevant to us, we require that $g_0 \geq g_1$, as otherwise we would obtain a lower bound instead of an upper bound. To simplify the equations, let $\gamma = \nicefrac{e^{g_0} }{(e^{g_0} - e^{g_1})} \geq 1 $, which implies that $g_0 - g_1 = \log \big(\gamma / (\gamma - 1)\big)$ and therefore
\begin{align}
\label{eq:relentber_f_div_after_change_variable}
	\relentber \big( \hat{r} \Vert r \big) \geq 1 + g_0 - e^{g_0} - \log \Big( \frac{\gamma}{\gamma - 1} \Big) \hat{r} + \gamma^{-1} e^{g_0} r.
\end{align}

To finalize the proof, note that the optimal value of the parameter $g_0$ is $\log \big( \nicefrac{\gamma}{(\gamma - r)} \big)$ and therefore since $\gamma > 1$ and $r \in [0,1]$, then $g_0 \geq 0$. Finally, letting $c \coloneqq e^{-g_0} \in (0,1]$, using the inequality from~\eqref{eq:relentber_f_div_after_change_variable} in~\Cref{lemma:small_kl_mi}, and re-arranging the terms recovers~\Cref{th:mi_bound_bounded_fast_rate}. Note that \Cref{lemma:small_kl_mi} only proves \Cref{th:mi_bound_bounded_fast_rate} when the empirical risk is smaller than the population risk, however, in the other scenario the bound holds trivially.

Before moving to the next subsection, let us recover \citet[Theorem 1.2.1] {catoni2007pac}'s stronger bound from~\Cref{lemma:small_kl_mi}. This bound is also obtained employing a variational representation of the relative entropy, although this time it follows from the Donsker and Varadhan~\Cref{lemma:dv_and_gvp}. Let $g_0 = g(0)$ and $g_1 = g(1)$, then for all $g_0$ and $g_1$ in $(-\infty,\infty)$,
\begin{align}
	\relentber \big( \hat{r} \Vert r \big) &\geq   g_1 \hat{r} + g_0 (1 - \hat{r}) - \log \left( e^{g_1} r + e^{g_0} (1 - r) \right) \nonumber \\
	&= g_0 + (g_1 - g_0) \hat{r} - \log \left( e^{g_0} + \left( e^{g_1} - e^{g_0} \right) r \right) \nonumber \\
	\label{eq:relentber_dv_before_change_variable}
	&= -(g_0 - g_1) \hat{r} - \log \left( 1 - r \big( 1 - e^{-(g_0 - g_1)} \big) \right).
\end{align}

As before, the bound resulting from~\eqref{eq:relentber_dv_before_change_variable} is only relevant to us when $g_0 \geq g_1$, as otherwise it results in a lower bound instead of an upper bound. The equations can be simplified by letting $\lambda = n(g_0 - g_1) \geq 0$. After this change of variable, using the inequality from~\eqref{eq:relentber_dv_before_change_variable} in~\Cref{lemma:small_kl_mi}, and re-arranging the terms yields the desired result.

\begin{theorem}[{\cite[Theorem 1.2.1]{catoni2007pac}}]
\label{th:catoni_mi}
Consider a loss function with a range contained in $[0,1]$. Then, for every $\lambda > 0$
\begin{equation*}
	\bE \big[ \poprisk(W) \big] \leq \frac{1}{1 - e^{- \frac{\lambda}{n}}} \left[1 - e^{- \frac{\lambda}{n} \bE \big[ \emprisk(W,S) \big] - \frac{\minf(W;S)}{n}} \right].
\end{equation*}
\end{theorem}

\Cref{th:mi_bound_bounded_fast_rate,th:catoni_mi} are equivalent in the sense that they are both obtained via a variational representation of the relative entropy from~\Cref{lemma:small_kl_mi}. That is, for every value of the empirical risk and the mutual information, there exists a value of $\lambda > 0$, $\gamma > 1$, and $c \in (0,1]$ such that the population risk upper bounds in \Cref{th:mi_bound_bounded_fast_rate,th:catoni_mi,} and \Cref{lemma:small_kl_mi} are the same. 

Starting from \Cref{th:catoni_mi}, one could have also easily obtained \Cref{th:mi_bound_bounded_fast_rate} as follows. Let $\lambda = n \gamma \log \big(\nicefrac{\gamma}{\gamma - 1} \big)$, then $\gamma > 1$ and
\begin{equation*}
	\bE \big[ \poprisk(W) \big] \leq \gamma \left[1 - e^{- \gamma \log \left( \frac{\gamma}{\gamma - 1} \right) \bE \big[ \emprisk(W,S) \big] - \frac{\minf(W;S)}{n}} \right].
\end{equation*}
Finally, one can note that the function $1 - e^{-x}$ is a non-decreasing, concave, continuous function for $x > 0$ and therefore can be upper bounded by its envelope, that is, $1 - e^{-x} = \inf_{a > 0} \{ e^{-a} x + 1 - e^{-a} (1 + a) \}$. Using the envelope in the equation above and letting $c \coloneqq e^{-a} \in (0,1]$  completes the equivalence.

\subsubsection{A Mixed-Rate Bound}

To get a better understanding of~\Cref{th:mi_bound_bounded_fast_rate}, consider the small empirical risk and small normalized complexity regime, that is, let $c = 1$ and consider~\Cref{cor:mi_bound_bounded_fast_rate_weaker}. Now, one can relax the bound and optimize the parameter $\gamma$ to derive a parameter-free bound.

When the empirical risk is smaller than the normalized complexity and $\gamma$ is small, a good upper bound for $\log \big( \nicefrac{\gamma}{(\gamma - 1)} \big)$ is $\nicefrac{1}{2}(x + \nicefrac{1}{x})$. Using this approximation in~\Cref{cor:mi_bound_bounded_fast_rate_weaker} yields that 
\begin{equation*}
	\bE \big[ \poprisk(W) \big] \leq  \frac{2\gamma - 1}{\gamma - 1} \cdot \bE \big[ \emprisk(W, S) \big] +  \gamma \cdot \frac{\minf(W;S)}{n}.
\end{equation*}
Optimizing this equation with respect to  $\gamma$ yields the following result. 

\begin{theorem}
\label{th:mi_bound_bounded_mixed_rate}
Consider a loss function with a range contained in $[0,1]$. Then, 
\begin{align*}
	\bE \big[ \poprisk(W) \big] \leq \bE \big[ \emprisk(W,S) \big] + \frac{\minf(W;S)}{n} + \sqrt{2 \bE \big[ \emprisk(W,S) \big]  \cdot \frac{\minf(W;S)}{n}}.
\end{align*}
\end{theorem}

The bound above is symmetric and gives the same weight to the empirical risk and the normalized complexity measured by $\nicefrac{\minf(W;S)}{n}$. Moreover, it has a mixed rate: on the one hand, it has the linear dependence on the normalized complexity of a fast-rate bound, but on the other hand, it also contains a squared root dependence typical of a slow-rate bound. However, the squared root dependence on the normalized complexity is weighted by the empirical risk. Hence, if the empirical risk is smaller than the normalized complexity, then the fast rate dominates.

\subsection{Interpolating Between the Slow and the Fast Rate: Losses With a Bounded Moment}
\label{subsec:interpolating_between_slow_and_fast_rate}

In the previous subsections we saw that when the CGF of the loss is bounded by some function $\psi$, then the population risk can be bounded from above by a function of the empirical risk and the ratio between the mutual information between the algorithm's output and the training set and the number of data instances. The rate of the bound, or the speed at which it vanishes with respect to the number of samples, is determined by the convex conjugate of the dominating function $\psi_*$. For example, if the loss is bounded, we can obtain a fast rate (\Cref{th:mi_bound_bounded_fast_rate}), and if it is sub-Gaussian, we can obtain a slow rate~\eqref{eq:mi_bound_subgaussian}.

However, as discussed in~\Cref{subsec:moments_and_generating_functions}, there are situations where the loss may not even have a CGF. If the loss has a bounded $p$-th moment (that is, $\bE \big[ \ell(w,Z)^p \big] \leq m_p$ for all $w \in \cW$) for some $p > 1$, but it does not have a CGF,  then  we say that the loss has a \emph{heavy tail}. A reasonable question is then: ``For losses with a heavy tail, can we find generalization bounds? If so, at which rate?''. It turns out that we can, and the bound's rate interpolates between a slow rate when $p =  2$ and a fast rate when $p \to \infty$. 

\begin{theorem}
\label{th:mi_bound_moments}
Consider a loss function $\ell(w,Z)$ with a $p$-th moment bounded by $m_p$ for all $w \in \cW$. Then, for every $c \in  (0,1]$ and every $\gamma > 1$, 
\begin{equation*}
	\bE \big[ \poprisk(W) \big] \leq c \gamma \log \left( \frac{\gamma}{\gamma - 1} \right) \bE \big[ \emprisk_{\leq t^\star} \big] + m_p^{\frac{1}{p}} \left[ \frac{p}{p-1} \right]  \left[ c \gamma \cdot \frac{\minf(W;S)}{n} + \gamma \kappa(c) \right]^{\frac{p-1}{p}},
\end{equation*}
where $\kappa(c) \coloneqq  1 - c( 1 - \log c)$ and
\begin{equation*}
	t^\star \coloneqq m_p^{\frac{1}{p}} \cdot \left[ c \gamma \cdot \frac{\minf(W;S)}{n} + \gamma \kappa(c) \right]^{- \frac{1}{p}} 
\end{equation*}
\end{theorem}

To simplify the interpretation, let $c = 1$, then the rate is $m_p^{\nicefrac{1}{p}} \cdot \left( \nicefrac{\minf(W;S)}{n} \right)^{\nicefrac{(p-1)}{p}}$. 
The term $m_p^{\nicefrac{1}{p}}$ controls the weight given to the complexity term and is equal to the $\cL_p$ norm of the loss. Note that by the non-decreasing nature of the $\cL_p$ norms a trivial bound would be $\bE \big[ \poprisk(W) \big] \leq m_p^{\nicefrac{1}{p}}$. Hence, the presented bound improves upon this as for increasing number of samples, the contribution of $m_p^{\nicefrac{1}{p}}$ vanishes. 
Then, the term $ \left( \nicefrac{\minf(W;S)}{n} \right)^{\nicefrac{(p-1)}{p}}$ shows how the rate is interpolating between a slow rate when $p = 2$ and a fast rate when $p \to \infty$. When $p = 2$, the tail of the loss could be as heavy as a Pareto distribution without a 3rd moment (or with a parameter $a \in [2,3)$~\citep[Chapter 20]{johnson1994continuous}), while when $p \to \infty$, the loss is essentially bounded, that is, $\esssup \ell(w,Z) \leq m_\infty$ for all $w \in \cW$. For even heavier tails, that is, when $p \in (1,2)$, the rate is slower than the standard slow rate since we are considering a scenario where the loss does not even have a finite variance. The term $\emprisk_{\leq t^\star}$ represents a truncated version of the empirical risk and comes from a refinement of~\citet{alquier2006transductive}'s truncation method, which is discussed next.

\subsubsection{Alquier's Truncation Method}

We derived~\Cref{th:mi_bound_moments} in~\citep{rodriguez2024moments} borrowing the \emph{truncation method} developed by~\citet{alquier2006transductive} in his Ph.D. thesis in the context of PAC-Bayesian bounds. This method consisted of decomposing the loss as $\ell \leq \ell^-_{\nicefrac{n}{\lambda}} +\ell^+_{\nicefrac{n}{\lambda}}$, where 
\begin{equation*}
	\ell^-_{\nicefrac{n}{\lambda}} \coloneqq \min \left\{ \ell, \frac{n}{\lambda} \right\} \textnormal{ and } \ell^+_{\nicefrac{n}{\lambda}} \coloneqq \left[ \ell - \frac{n}{\lambda} \right]_+
\end{equation*}
represent a truncated version of the loss (hence the name) and its unbounded remainder, where $[x]_+ \coloneqq \max \{ x, 0 \}$. With this decomposition, one may bound the population risk $\poprisk^-_{\nicefrac{n}{\lambda}}$ associated with the truncated loss using the techniques described above, and bound the population risk $\poprisk^+_{\nicefrac{n}{\lambda}}$ associated with the unbounded remainder using standard tail inequalities. Then, an appropriate selection of the truncation point $\nicefrac{n}{\lambda}$, that trades off the tightness of the generalization bounds for bounded losses and the likelihood (or lack thereof) of the tail event, results in a bound with the desired rate. 

To be precise, the population risk of the unbounded remainder $\poprisk^+_{\nicefrac{n}{\lambda}}$ can be bounded by the relationship between moments and tails~\citep[Lemma 4.4]{kallenberg1997foundations}. That is
\begin{align}
	\bE \big[ \poprisk^+_{\nicefrac{n}{\lambda}}(W) \big] &= \int_{0}^\infty \bP \big[ \poprisk^+_{\nicefrac{n}{\lambda}}(W) \geq t \big] \rmd t \nonumber \\ 
	&\leq \int_{0}^\infty \bP \left[ \poprisk(W) \geq t + \frac{n}{\lambda} \right] \rmd t \nonumber \\
	&= \int_{\frac{n}{\lambda}}^\infty \bP \left[ \poprisk(W) \geq t \right] \rmd t. \nonumber 
\end{align}

In~\cite{rodriguez2024moments}, we employ a refinement of this truncation method, which was already suggested by~\citet{alquier2021user} itself. Consider instead the decomposition of the loss $\ell = \ell_{< \nicefrac{n}{\lambda}} + \ell_{\geq \nicefrac{n}{\lambda}}$, where 
\begin{equation*}
	\ell_{< \nicefrac{n}{\lambda}} \coloneqq  \ell \cdot \bI_{\{ \ell < \nicefrac{n}{\lambda} \}} \textnormal{ and } \ell_{\geq \nicefrac{n}{\lambda}} \coloneqq  \ell \cdot \bI_{\{ \ell \geq \nicefrac{n}{\lambda} \}}
\end{equation*}
again represent a truncated version of the loss and its unbounded reminder. The technique follows analogously, with the main advantage that $\ell_{\leq \nicefrac{n}{\lambda}}$ has the potential to be much smaller than $\ell^{-}_{\nicefrac{n}{\lambda}}$ since all losses larger than $\nicefrac{n}{\lambda}$ contribute $\nicefrac{n}{\lambda}$ to the former and $0$ to the latter, while the population risk of the unbounded remainder $\poprisk_{\geq \nicefrac{n}{\lambda}}$ can be bounded by the exact same quantity as $\poprisk_{\nicefrac{n}{\lambda}}^+$, that is
\begin{align}
	\bE \big[ \poprisk_{\geq \nicefrac{n}{\lambda}}(W) \big] &= \int_{0}^\infty \bP \big[ \poprisk_{\geq \nicefrac{n}{\lambda}}(W) \geq t \big] \rmd t \nonumber \\ 
	\label{eq:tail_bound_remainder}
	&= \int_{\frac{n}{\lambda}}^\infty \bP \left[ \poprisk(W) \geq t \right] \rmd t. 
\end{align}
\looseness=-1 Hence, we  continue with this version of the method for the rest of the subsection.

\subsubsection{Optimal Truncation Point}

If the loss $\ell(w,Z)$ has a $p$-th moment bounded by $m_p$ for all $w \in \cW$, combining Markov's inequality with~\eqref{eq:tail_bound_remainder} yields
\begin{equation}
    \label{eq:unbounded_remainder_moment_tail_bound}
	\bE \big[ \poprisk_{\geq \nicefrac{n}{\lambda}}(W) \big] \leq \int_{\frac{n}{\lambda}}^{\infty} \frac{m_p}{t^p} \cdot \rmd t = \frac{m_p}{p-1} \left( \frac{\lambda}{n} \right)^{p-1}.
\end{equation}
Then, the above bound on the unbounded remainder together with~\Cref{th:mi_bound_bounded_fast_rate} yield the following intermediate result.

\begin{lemma}
\label{lemma:truncation_method_intermediate}
Consider a loss function $\ell(w,Z)$ with a $p$-th moment bounded by $m_p$ for all $w \in \cW$. Then, for every $c \in (0,1]$ and every $\gamma > 1$
\begin{equation*}
	\bE \big[ \poprisk(W) \big] \leq c \gamma \log \left( \frac{\gamma}{\gamma - 1} \right) \cdot \bE \big[ \emprisk_{\leq \frac{n}{\lambda}}(W, S) \big] + c \gamma \cdot \frac{\minf(W;S)}{\lambda} + \gamma \kappa(c) \cdot \frac{n}{\lambda} + \frac{m_p}{p-1} \left( \frac{\lambda}{n} \right)^{p-1},
\end{equation*}
where $\kappa(c) \coloneqq 1 - c(1 - \log c) $.
\end{lemma}

Finally, choosing the optimal parameter $\lambda$ trades off (i) the penalty of the loss' tail after the truncation point $\nicefrac{n}{\lambda}$ for (ii) the penalty of that range range $\nicefrac{n}{\lambda}$ while exploiting the existing sharp bounds for losses with a bounded range. This optimization results in~\Cref{th:mi_bound_moments}.

\subsubsection{A Bound for Losses With a Bounded Variance}

A particularly important case of the results above is the one for losses with a bounded second moment, as it presents the weaker assumption that still recovers a slow rate of $\cO(\nicefrac{1}{\sqrt{n}})$. However, the raw second moment can be much larger than the variance as for every random variable $X$ it holds that $\bE[X^2] = \mathrm{Var}(X) + \bE[X]^2$. It is of interest, then, to derive bounds that only assume that the loss has a bounded variance. In the following theorem, we present a bound of this type that essentially substitutes the second moment dependence by the variance. However, the bound has an extra multiplicative factor that prevents it to have exactly a slow rate. When the mutual information is small, the rate is essentially $\cO(\nicefrac{1}{\sqrt{n}})$, but this degenerates as it increases towards $n$, where the bound is essentially vacuous.

\begin{theorem}
\label{th:mi_bounded_variance}
Consider a loss function $\ell(w,Z)$ with a variance bounded by $\sigma^2$ for all $w \in \cW$. Then, for every $c \in (0,1]$ and every $\gamma > 1$
\begin{equation*}
	\bE \big[ \poprisk(W) \big] \leq \Bigg[1 - 2 \sqrt{\kappa_2 \cdot \frac{\minf(W;S)}{n} + \kappa_3} \Bigg]_+^{-1} \Bigg[ \kappa_1 \cdot \bE \big[ \emprisk(W,S) \big] + 2 \sigma \sqrt{\kappa_2 \cdot \frac{\minf(W;S)}{n} + \kappa_3} \Bigg],
\end{equation*}
where $\kappa_1 \coloneqq c \gamma \log \big( \nicefrac{\gamma}{(\gamma -1 )} \big)$, $\kappa_2 \coloneqq c \gamma$, and $\kappa_3 \coloneqq \gamma \big( 1 - c (1 - \log c) \big)$.
\end{theorem}

Note that there is nothing preventing us to show the equivalent of~\Cref{th:mi_bound_moments} only assuming that the $p$-th moment of the loss $\bE[\ell(W',Z)]$ is bounded with respect to the product distribution of the data and the algorithm's hypothesis $\bP_{W} \otimes \bP_Z$. Although this condition is weaker, it is harder to guarantee as it requires some knowldege of the data distribution $\bP_Z$ \emph{and} the algorithm's Markov kernel $\bP_W^S$, a knowledge that could have been used instead to directly find a bound on the population risk $\bE \big[ \poprisk(W) \big]$. However, considering a result of this type is useful to prove~\Cref{th:mi_bounded_variance}.

The proof of this theorem follows by first considering the equivalent of~\Cref{th:mi_bound_moments} for $p=2$, where instead of $m_2$ one directly has $\bE \big[ \ell(W',Z)^2 ]$. Then, we may employ the equality $\bE \big[ \ell(W',Z)^2 ] \leq \mathrm{Var}( \ell(W',Z) ) + \bE[\poprisk(W)]^2$ together with the inequality $\mathrm{Var}(\ell(W',Z')) \leq \sup_{w \in \cW} \mathrm{Var}(\ell(w,Z)) = \sigma^2$ obtain an expression that only depends on the parameter $\sigma^2$ bounding the algorithm-independent variance. Finally, using the inequality $\sqrt{x + y} \leq \sqrt{x} + \sqrt{y}$ and re-arranging the equation while accepting the convention that $1/0 \to \infty$ completes the proof.

\section{The randomized-subsample Setting}
\label{sec:bounds_using_conditional_mutual_information}

As mentioned previously, the bounds presented in~\Cref{sec:bounds_using_mutual_information} become vacuous as the mutual information diverges, that is $\minf(W;S) \to \infty$. This can happen in situations where the hypothesis is a deterministic function of the training set $\bP_W^S = \delta_{f(S)}$ and both the hypothesis and the problem instances $Z_i$ are continuous. An example of this situation would be when the hypothesis returned by the algorithm is the empirical average of the data instances and those are Gaussian. Other pathological examples where the generalization error is small but the mutual information is large, leading to vacuous bounds, are presented in~\citep{bassily2018learners}.

For this reason, \citet{steinke2020reasoning}, inspired by \citet{vapnik1971uniform}, introduced the \emph{randomized-subsample setting}\footnote{In~\citep{hellstrom2020generalization}, this is called the random subset setting. However, this may cause confusion with the random-subset bounds in this monograph and in the literature.}. In this setting, it is considered that the training set $S$ is obtained through the following mechanism. First, a super sample
\begin{equation*}
    \tilde{S} \coloneqq 
    \begin{pmatrix}
        \tilde{Z}_{1,0} & \tilde{Z}_{2,0} & \cdots & \tilde{Z}_{n,0} \\
        \tilde{Z}_{1,1} & \tilde{Z}_{2,1} & \cdots & \tilde{Z}_{n,1}
    \end{pmatrix}^{\intercal}
\end{equation*}
of $2n$ i.i.d. instances is obtained by sampling from the data distribution $\bP_Z$. Then, an independent sequence of indices $U \coloneqq (U_1, \ldots, U_n)$ is generated, where each index is distributed as a Bernoulli random variable, that is, $\bP_{U_i}[0] = \bP_{U_i}[1] = \nicefrac{1}{2}$. Finally, the training set is sub-sampled from the superset so that $Z_i = \tilde{Z}_{i,U_i}$. Unused samples, or \emph{ghost samples} $S_\textnormal{ghost} = \tilde{S} \setminus S$, are virtual and only exist for the purpose of the analysis. To disambiguate, in the following, we will often refer to the \emph{standard setting} to consider the setting where $S$ is collected by sampling $n$ samples directly from $\bP_Z$.

Imagine that we were to have access to a realization of the super sample $\tilde{s}$ and the indices $u$. Then, for a hypothesis $w$, an estimate of the generalization error is the \emph{empirical generalization error}
\begin{equation*}
	\empgen(w,\tilde{s},u) \coloneqq \emprisk(w,s_{\textnormal{ghost}}) - \emprisk(w, s),
\end{equation*}
where the empirical risk on the ghost samples $\emprisk(w,s_{\textnormal{ghost}})$ acts as a proxy to the population risk. In fact, they have the same expected value since the hypothesis returned by the algorithm is independent of the ghost samples, that is $\bE[\emprisk(W,S_{\textnormal{ghost}})] = \bE[\poprisk(W)]$. Therefore, studying the bias of the empirical risk on the training data with respect to the empirical risk on the ghost samples is equivalent to studying its bias with respect to the population risk, that is $\bE[\empgen(W,\tilde{S},U)] = \bE[\gen(W,S)]$.

\paragraph{Connection to the Rademacher Complexity.} The empirical generalization error is connected to the Rademacher complexity discussed in~\Cref{sec:rademacher_complexity}. For example, recall that the Rademacher complexity may be understood as a measure of the discrepancy between fictitious training and test sets~\citep[Section 26.1]{shalev2014understanding}. Similarly, the empirical generalization error is the discrepancy between a fictitious test and the real training set. In fact, ghost samples were used prior to \citet{steinke2020reasoning}'s work to prove generalization bounds based on the Rademacher complexity~\citep{vapnik1971uniform, devroye1996vapnik}, and these works inspired them in their formalism. As an example of this connection, let $R_i = -1^{U_i}$ and $R_i' = R_i$, and note that $R_i$ and $R_i'$ are Rademacher random variables. Then, for a fixed hypothesis $w$ and a fixed supersample $\tilde{s}$, we have that
\begin{align*}
    \bE \big[ \empgen(w,\tilde{s}, U) \big] &= \frac{1}{n} \bE \mleft[ \sum_{i=1}^n R_i \big( \ell(w,z_{i,1}) - \ell(w, z_{i,0}) \big) \mright] \\
    &= \frac{1}{n} \bE \mleft[ \sum_{i=1}^n R_i \ell(w,z_{i,1}) \mright]  + \frac{1}{n} \bE \mleft[ \sum_{i=1}^n R'_i \ell(w, z_{i,0}) \mright] \\
    &\leq 2 \rad(\ell \circ \cW, \tilde{s}),
\end{align*}
\looseness=-1 where the inequality comes from taking the supremum over all hypotheses $w \in \cW$.

Similarly to what we did in~\Cref{subsec:a_slow_rate_bound_mi}, we can bound from above this bias through a change of measure employing the  Donsker and Varadhan~\Cref{lemma:dv_and_gvp}. However, the structure introduced by the super sample and the indices, allows us to consider the dependence with respect to the indices instead of the one with respect to the training data directly. To see this, consider a fixed super sample $\tilde{s} \in \cZ^{n \times 2}$. Now, we can consider a random hypothesis $W'$ that depends on the super sample in some capacity $\bP_{W'}^{\tilde{S}=\tilde{s}} = \bQ(\tilde{s})$. In this case, in~\eqref{eq:dv}, the measurable space is $\cW \times \{ 0, 1 \}^n$, the distributions are $\bP^{\tilde{S}=\tilde{s}}_{W,U}$ and $\bQ(\tilde{s}) \otimes \bP_U$, and the function is $\lambda \empgen(w,\tilde{s},u)$ for some $\lambda \in \bR$. In this way, for a fixed super sample $\tilde{s}$, the equivalent to~\eqref{eq:dv_to_lambda_gen} is
\begin{equation}
	\label{eq:dv_to_lambda_emp_gen}
	\bE^{\tilde{S}=\tilde{s}} \big[ \empgen(W,\tilde{s}, U) \big] \leq \frac{1}{\lambda} \bigg( \relent(\bP^{\tilde{S}=\tilde{s}}_{W, U} \Vert \bQ(\tilde{s}) \otimes \bP_{U})  + \log \bE^{\tilde{S}=\tilde{s}} \Big[ e^{\lambda \empgen(W',\tilde{s},U)} \Big] \bigg).
\end{equation}
Finally, considering the super sample as a random variable and taking the expectation to both sides completes the analogy with~\eqref{eq:dv_to_lambda_gen}.

\subsection{A Slow Rate for Losses With a Bounded Range}
\label{subsec:slow_rate_cmi}

Consider that the loss has a range bounded in $[a,b]$. Now, note that for a fixed hypothesis $w \in \cW$ and a fixed super sample $\tilde{s}$, we may write the empirical generalization error as
\begin{equation*}
	\empgen(w,\tilde{s},U) = \frac{1}{n} \sum_{i=1}^n \Big \{ \ell(w,\tilde{z}_{i,1-U_i}) - \ell(w,\tilde{z}_{i,U_i}) \Big\}.
\end{equation*}
When written in this way, we may focus on the random variable $\ell(w,\tilde{z}_{i,1-U_i}) - \ell(w,{z}_{i,U_i})$ and note that it has mean zero and that it is contained in $[a-b,b-a]$.  Therefore, this random variable is $(b-a)^2$-sub-Gaussian and by the algebra of CGFs from~\Cref{subsec:moments_and_generating_functions} it follows that
\begin{equation*}
	\Lambda_{\empgen(w,\tilde{s},U)} \leq \frac{\lambda^2 (b-a)^2}{2n},
\end{equation*}
which in turn means that $\log \bE \Big[ e^{\lambda \empgen(W',\tilde{s},U)} \Big] \leq \nicefrac{\lambda^2 (b-a)^2}{2 n}$ for all $\tilde{s} \in \cZ^{n \times 2}$. Therefore, \Cref{eq:dv_to_lambda_emp_gen} can be further bounded by
\begin{equation*}
	\bE^{\tilde{S}=\tilde{s}} \big[ \empgen(W,\tilde{s},U) \big] \leq \frac{\relent(\bP^{\tilde{S}=\tilde{s}}_{W,U} \Vert \bQ(\tilde{s}) \otimes \bP_{U})}{\lambda}  + \frac{\lambda (b-a)^2}{2n}
\end{equation*}
for all $\lambda > 0$ and all $\tilde{s} \in \cZ^{n \times 2}$. As before, the right-hand side of this equation is convex with respect to $\lambda$ and hence it has a minimum that tightens the bound. The optimization of the parameter $\lambda$ results in the following extension of the theorem due to~\citet{steinke2020reasoning}.

\begin{theorem}[{\citet[Theorem 2, extended]{steinke2020reasoning}}]
\label{th:cmi_bound_bounded}
	Consider a loss function with a range contained in $[a,b]$. Then, for every index-independent Markov kernel $\bQ$ from $\cZ^{n \times 2}$ to distributions on $\cW$
	\begin{align*}
		\bE \big[ \gen(W,S) \big] &\leq (b-a) \bE \left[ \sqrt{ \frac{2\relent(\bP^{\tilde{S}}_{W,U} \Vert \bQ(\tilde{S}) \otimes \bP_{U})}{n}} \right] \\
		&\leq (b-a) \sqrt{ \frac{2\relent(\bP_{W,\tilde{S},U} \Vert \bQ(\tilde{S}) \otimes \bP_{\tilde{S}} \otimes \bP_{U})}{n}}.
	\end{align*}
	In particular, choosing $\bQ(\tilde{S})$ to be the marginal Markov kernel $\bP_{W}^{\tilde{S}} = \bP_{W}^{\tilde{S},U} \circ \bP_U$ is optimal and
	\begin{equation*}
		\bE \big[ \gen(W,S) \big] \leq (b-a) \bE \left[ \sqrt{\frac{2 \minf^{\tilde{S}}(W;U)}{n}} \right] \leq (b-a) \sqrt{ \frac{2 \minf(W;U|\tilde{S})}{n}}.
	\end{equation*}
\end{theorem}

As previously, the optimality of the choice of marginal distribution comes from the golden formula from~\Cref{prop:properties_minf} and the theorem also holds for all Markov kernels $\bQ(\tilde{S})$ such that $\bP^{\tilde{S}}_{W,U} \not \ll \bQ(\tilde{S}) \otimes \bP_{U}$ a.s. by the convention that in that case $\relent(\bP^{\tilde{S}}_{W,U} \Vert \bQ(\tilde{S}) \otimes \bP_{U}) \to \infty$ for every realization $\tilde{s}$ in which the absolute continuity condition does not hold. Also, the conditional mutual information arises from the \emph{chain rule} $\minf(W,\tilde{S};U) = \minf(W;U|\tilde{S}) + \minf(\tilde{S};U)$ and the fact that the indices are independent of the super sample. For this reason, we will also focus only on the result concerning the conditional mutual information. The extension with respect to \citep[Theorem 2]{steinke2020reasoning} comes from the first inequality, where the expected value is pulled outside of the square root, tightening the bound due to Jensen's inequality.

This result is very similar in form to \Cref{th:mi_bound_bounded}. It maintains the classical relationship provided by the uniform convergence property of a hypothesis class $\cW$, where the upper bound on the generalization error is of the order $\sqrt{\nicefrac{\mathfrak{C}}{n}}$ and $\mathfrak{C}$ represents the complexity of the class (\Cref{sec:pac_learning}). However, now the conditional mutual information $\minf(W;U|\tilde{S})$ captures the information that the hypothesis has about the \emph{identity} of the data used for training, instead of \emph{the whole training set}. As before, if the hypothesis returned by the algorithm is independent of the data and $\bP_W^S = \bQ$ a.s., then the conditional mutual information is 0 as expected. Since $U$ is a discrete random variable, this conditional mutual information is bounded from above by $\ent(U|\tilde{S}) \leq \log n$. This change in perspective is important, as this better captures our intuition about what makes models generalize poorly. Informally, the more the hypothesis can be used to detect which instances were employed for training, when presented with equally distributed samples, the more it depends on these particular training samples and the worse it generalizes.

Since the indices $U$ are a discrete random variable, the conditional mutual information can be directly interpreted as the average extra bits needed to encode the indices corresponding to the training data when the hypothesis and the super sample are known, that is $\minf(W;U|\tilde{S}) = \ent(U) - \ent(U|W,\tilde{S})$, where we employed that $\ent(U|\tilde{S}) = \ent(U)$ since the indices are independent of the super sample. Again, this elucidates connections between privacy and generalization: if an algorithm is private, given samples from the same distribution, the information that the algorithm contains about which of the samples where used for training should be small, and therefore $\minf(W;U|\tilde{S})$ should be bounded from above. This improves upon the intuition gained from \Cref{th:mi_bound_bounded} and, as mentioned previously, will be formalized later in~\Cref{ch:pac_bayesian_generalization}.

\subsection{Beyond Bounded Losses and the Slow Rate}
\label{subsec:fast_rate_cmi}

\citet[Theorem 2]{steinke2020reasoning} proved a fast rate bound for losses with a bounded range featuring the conditional mutual information $\minf(W;U|\tilde{S})$. This bound was later improved by~\citet[Corollary 1]{hellstrom2021fast} to obtain better constants. The bound is given below for completeness, although the proof is outside of the scope of this monograph.

\begin{theorem}
    \label{th:cmi_bound_fast_rate}
    Consider a loss function with a range contained in $[a,b]$. Then, let us consider the set of positive parameters $\cG = \{ (\gamma_1, \gamma_2) \in \bR_+^2 : \gamma_1(1-\gamma_2) + (e^{\gamma_1}-1-\gamma_1)(1+\gamma_2^2) \leq 0 \}$. Then,
    \begin{equation*}
        \bE\big[ \poprisk(W) \big] \leq (b-a) \inf_{(\gamma_1, \gamma_2) \in \cG} \mleft \{ \gamma_2 \bE \big[ \emprisk(W,S) \big] + \frac{\minf(W;U|\tilde{S})}{\gamma_1 n} \mright \}.
    \end{equation*}
\end{theorem}

These bounds maintain the same interpretation as the bounds from~\Cref{subsec:slow_rate_cmi}, but providing a fast rate. In particular, in the realizable case, the bound can be slightly modified to achieve that $\gamma_1 = \log 2$ and therefore
\begin{equation*}
    \bE\big[ \poprisk(W) \big] \leq \frac{\minf(W;U|\tilde{S})}{n \log 2},
\end{equation*}
which is again, similarly to the bounds from~\Cref{subsec:a_fast_rate_bound_mi}, reminiscent of the classical generalization guarantees from uniform convergence in the realizable setting for classification problems (see~\Cref{subsec:vc_and_natarajan_dimensions}). 

To be precise, the bounds from~\citet{steinke2016upper} do not necessarily require that the loss is bounded. Instead, they require that for every two instances $z$ and $z'$, the loss is Lipschitz with respect to the sample space and $|\ell(w,z) - \ell(w,z')| \leq \rho(z,z')$. Then, instead of the range $(b-a)$, in \Cref{th:cmi_bound_bounded}, one could have $\bE[ \rho(Z,Z') ]$ for two independent instances distributed according to $\bP_Z$. Another attempt at circumventing the bounded range assumption was taken by~\citet{zhou2022individually}, where they engineered a conditional version of the CGF to obtain similar bounds to those in~\Cref{th:mi_bound_general_cgf}.

\section{Single-Letter and Random-Subset Bounds}
\label{sec:random_subset_and_single_letter}

\looseness=-1 In~\Cref{sec:bounds_using_mutual_information}, we derived bounds on the expected generalization error for an algorithm returning a hypothesis $W$ after observing a training set $S$ consisting of $n$ samples. Essentially, these bounds exploited the Donsker and Varadhan~\Cref{lemma:dv_and_gvp} to \emph{change the measure} from the joint distribution $\bP_{W,S}$ that we wanted to study to a product distribution where the generalization error is zero. In other words, they \emph{decoupled} the hypothesis and the training data. After that, the extra terms appearing in the change of measure can be bounded using \emph{concentration inequalities} such as those following the Cramér--Chernoff method~\citep[Section 2.2]{boucheron2003concentration}.

\subsection{Single-Letter Bounds}
\label{subsec:single_letter_bounds}

Note that, for a fixed hypothesis $w \in \cW$ and a fixed training set $s \in \cZ^n$, we may write the generalization error as
\begin{equation}
	\gen(w,s) = \frac{1}{n} \sum_{i=1}^{n} \Big\{ \bE \big[ \poprisk(w) \big] - \ell(w,z_i) \Big \} = \frac{1}{n} \sum_{i=1}^n \gen(w,z_i),
\end{equation}
where $\gen(w,z_i) \coloneqq \bE[\poprisk(w)] - \ell(w,z_i)$ is the \emph{single-letter generalization error} on sample $z_i$. Then, due to the linearity of the expectation, the expected generalization error is equal to the average of the expected single-letter generalization errors. That is, for each random instance $Z_i$, we may consider the expected single-letter generalization error $\bE[\gen(W,Z_i)]$, where the expectation is taken with respect to the joint distribution $\bP_{W,Z_i} = \bP_{W}^{Z_i} \otimes \bP_{Z_i}$, where the Markov kernel $\bP_{W}^{Z_i}$ represents the expected kernel obtained with an instance $Z_i$, that is $\bP_{W}^{Z_i} = \bP_W^{S} \circ \bP_{S^{-i}}$, where we recall that given a dataset $s = (z_1, \ldots, z_n)$, the dataset obtained by removing the $i$-th instance from $s$ is defined as $s^{-1} = (z_1, \ldots, z_{i-1}, z_{i+1}, \ldots, z_n)$. In this way, all the results obtain in~\Cref{sec:bounds_using_mutual_information} for the expected generalization error with respect to the mutual information $I(W;S)$ can now be replicated for the single-letter expected generalization error with respect to the single-letter mutual information $I(W;Z_i)$.

For example, considering \Cref{th:mi_bound_general_cgf} for each single-letter generalization error results in the following bound from~\citet{bu2020tightening}, where we directly write the results in terms of the single-letter mutual information with the understanding that equivalent results follow with a choice of an arbitrary data-independent distribution $\bQ$ on $\cW$.

\begin{theorem}[{\citet[Theorem 2]{bu2020tightening}}]
\label{th:single_letter_mi_bound_general_cgf}
	Consider a loss function with a bounded CGF (\Cref{def:bounded_cgf}). Then, 
	\begin{equation*}
		\bE \big[ \gen(W,S) \big] \leq \frac{1}{n} \sum_{i=1}^n \psi_*^{-1} \big( \minf(W;Z_i) \big).
	\end{equation*}
\end{theorem}

\looseness=-1 Similarly to~\citet{bu2020tightening}, we also extended our fast-rate results for bounded losses and the results for losses with a bounded moment to the single-letter setting in~\citep{rodriguez2024moments}.

\begin{theorem}
\label{th:single_letter_mi_bound_bounded_fast_rate}
Consider a loss function with a range contained in $[0,1]$. Then, for every $i \in  [n]$ and every $c_i  \in (0,1]$ and every $\gamma_i > 1$,
\begin{equation*}
	\bE \big[ \poprisk(W) \big] \leq \bar{\kappa}_1 \cdot \bE \big[ \emprisk(W, S) \big] + \frac{\bar{\kappa}_2}{n} \sum_{i=1}^n \minf(W;Z_i) + \bar{\kappa}_3
\end{equation*}
where $\bar{\kappa}_1 \coloneqq \frac{1}{n} \sum_{i=1}^n c_i \gamma_i \log \big( \nicefrac{\gamma_i}{(\gamma_i-1)} \big)$, $\bar{\kappa}_2 \coloneqq \frac{1}{n} \sum_{i=1}^n c_i \gamma_i$, and $\bar{\kappa}_3 \coloneqq \frac{1}{n} \sum_{i=1}^n \gamma_i ( 1 - c_i(1 - \log c_i))$.
In particular, for interpolating algorithms (that is, when $\bE^S [ \emprisk(W,S) ] = 0$), the optimal parameters are $\gamma_i \to 1$ and $c_i = \exp\big({-\minf(W;Z_i)}\big)$ and then
\begin{equation*}
	\bE \big[ \poprisk(W) \big] \leq \frac{1}{n} \sum_{i=1}^n \Big( 1 - e^{- \minf(W;Z_i) } \Big) \leq \frac{1}{n} \sum_{i=1}^n \minf(W;Z_i).
\end{equation*}
\end{theorem}

\begin{theorem}
\label{th:single_letter_mi_bounded_moments}
Consider a loss function $\ell(w,Z)$ with a $p$-th moment bounded by $m_p$ for all $w \in \cW$. Then, for every $i \in  [n]$ and every $c_i  \in (0,1]$ and every $\gamma_i > 1$,
    \begin{equation*}
        \poprisk \leq \Bar{\kappa}_1 \cdot \emprisk_{\leq t^\star} + m_p^{\frac{1}{p}} \Big(\frac{p}{p-1}\Big) \Big( \Bar{\kappa}_2 \cdot \frac{1}{n} \sum_{i=1}^n \minf(W;Z_i) + \Bar{\kappa}_3 \Big)^{\frac{p-1}{p}},
    \end{equation*}
    where $$t^\star \coloneqq m_p^{\frac{1}{p}} \Big( \Bar{\kappa}_2 \cdot \frac{1}{n} \sum_{i=1}^n \minf(W;Z_i) + \Bar{\kappa}_3 \Big)^{-\frac{1}{p}}$$
    and where $\bar{\kappa}_1 \coloneqq \frac{1}{n} \sum_{i=1}^n c_i \gamma_i \log \big( \nicefrac{\gamma_i}{(\gamma_i-1)} \big)$, $\bar{\kappa}_2 \coloneqq \frac{1}{n} \sum_{i=1}^n c_i \gamma_i$, and $\bar{\kappa}_3 \coloneqq \frac{1}{n} \sum_{i=1}^n \gamma_i ( 1 - c_i(1 - \log c_i))$.
\end{theorem}

\begin{theorem}
\label{th:single_letter_mi_bounded_variance}
Consider a loss function $\ell(w,Z)$ with a variance bounded by $\sigma^2$ for all $w \in \cW$. Then, for every $i \in [n]$ and every $c_i \in (0,1]$ and every $\gamma_i > 1$
\begin{align*}
	\bE \big[ \poprisk(W) \big] \leq &\left[1 - 2 \sqrt{\bar{\kappa}_2 \cdot \frac{1}{n} \sum_{i=1}^n \minf(W;Z_i) + \bar{\kappa}_3} \right]_+^{-1} \cdot  \\ 
	&\left[ \bar{\kappa}_1 \cdot \bE \big[ \emprisk(W,S) \big] + 2 \sigma \sqrt{\bar{\kappa}_2 \cdot \frac{1}{n} \sum_{i=1}^n \minf(W;Z_i) + \bar{\kappa}_3} \right],
\end{align*}
where $\bar{\kappa}_1 \coloneqq \frac{1}{n} \sum_{i=1}^n c_i \gamma_i \log \big( \nicefrac{\gamma_i}{(\gamma_i-1)} \big)$, $\bar{\kappa}_2 \coloneqq \frac{1}{n} \sum_{i=1}^n c_i \gamma_i$, and $\bar{\kappa}_3 \coloneqq \frac{1}{n} \sum_{i=1}^n \gamma_i ( 1 - c_i(1 - \log c_i))$.
\end{theorem}

The motivation behind this kind of result is that the single-letter mutual information $I(W;Z_i)$ often does not go to infinity even when the mutual information $\minf(W;S)$ does. Consider the following example.

\begin{example}[Gaussian location model problem]
\label{ex:glm}
In this problem, we observe $n$ instances of a $d$-dimensional Gaussian, that is $Z_i \sim \cN(\mu, \sigma^2 I_d)$. Then, the algorithm returns the average of these instances $W = \frac{1}{n} \sum_{i=1}^n Z_i$ with the objective of minimizing the distance $\Vert w - z \rVert_2$ between the returned hypothesis and the instances.
In this setting, we established that $I(W;S) \to \infty$. On the other hand, the single-letter mutual information has a bounded closed form expression that decreases with the number of samples, namely $I(W;Z_i) = \frac{d}{2} \log \frac{n}{n-1} \in \Theta(\nicefrac{d}{n})$.
\end{example}
In the particular setting of~\Cref{ex:glm}, if the considered loss is the squared distance $\ell(w,z) = \lVert w - z \rVert_2^2$, then the CGF is bounded and $\psi_*^{-1}$ can be calculated leading to the expected generalization bound $\bE \big[ \gen(W,S) \big] \leq \sigma^2 d \sqrt{\nicefrac{8}{n-1}}$~\citep[Section IV.A]{bu2020tightening}. Although the bound is sub-optimal, since the expected generalization error is in $\cO(\nicefrac{\sigma^2 d}{n})$, it has the same rate as the bounds obtained from uniform stability~\citep{boucheron2005theory} and algorithmic stability~\citep{bousquet2002stability} and showcases an improvement with respect to bounds using the mutual information $I(W;S)$.

Moreover, these results also expand the intuition obtained from~\Cref{sec:bounds_using_mutual_information}. Although these results do not have the classical shapes from~\Cref{sec:pac_learning} with a complexity term divided by the number of samples, they showcase that the generalization error degrades as the returned hypothesis keeps more information about \emph{each} of the training samples. In contrast, the results from~\Cref{sec:bounds_using_mutual_information} told us that the generalization error degrades as the returned hypothesis has more information about the \emph{whole} training set. This seemingly innocuous difference is important, as the mutual information between the hypothesis and the training set also contemplates artificial, spurious dependencies between the instances generated by the hypothesis. Mathematically, this can be understood through the \emph{chain rule} of the mutual information from~\Cref{prop:properties_minf}, namely
\begin{align}
	\minf(W;S) &= \sum_{i=1}^n \minf(W;Z_i | Z^{i-1}) \nonumber \\
	&= \sum_{i=1}^n \big\{ \minf(W,Z^{i-1};Z_i) - I(Z_i ; Z^{i-1}) \big\} \nonumber \\
	\label{eq:mi_larger_sum_single_letter_mi}
	&= \sum_{i=1}^n \big\{ \minf(W; Z_i) + \minf(Z_i;Z^{i-1} | W) \big\} \geq \sum_{i=1}^n \minf(W;Z_i),
\end{align}
where the artificial dependence is captured by the conditional mutual information $\minf(Z_i;Z^{i-1} | W)$ and where $\minf(Z_i;Z^{i-1}) = 0$ as the samples are independent of each other. For example, consider the Gaussian location model from above and only two instances $Z_0$ and $Z_1$. These two instances are independent and therefore $\minf(Z_0;Z_1) = 0$. However, once the average of the two is known, they become dependent and $\minf \big(Z_0;Z_1 \big| \nicefrac{(Z_0 + Z_1)}{2} \big) > 0$.

This extension is not unique to the bounds from~\Cref{sec:bounds_using_mutual_information}, and it can also be applied considering the randomized-subsample setting from~\Cref{sec:bounds_using_conditional_mutual_information}. Similarly to before, for a fixed hypothesis $w \in \cW$, a fixed supersample $\tilde{s} \in \cZ^n$, and a fixed sequence of indices $u \in \{0,1\}^n$ we may write the empirical generalization error as
\begin{equation}
\label{eq:empgen_decomposition}
	\empgen(w,\tilde{s},u) = \frac{1}{n} \sum_{i=1}^{n} \Big\{ \ell(w,\tilde{z}_{i,1-u_i} - \ell(w,\tilde{z}_{i,u_i}) \Big \} = \frac{1}{n} \sum_{i=1}^n \empgen(w,\tilde{z}_{i,0}, \tilde{z}_{i,1},u_i),
\end{equation}
where $\empgen(w,\tilde{z}_{i,0},\tilde{z}_{i,1},,u_i) \coloneqq \ell(w,\tilde{z}_{i,1-u_i}) - \ell(w,\tilde{z}_{i,u_i})$ is the \emph{single-letter empirical generalization error} on samples $\tilde{z}_{i,0}$ and $\tilde{z}_{i,1}$. Again, due to the linearity of the average of the expectation the expected single-letter empirical generalization errors is equal to the expected generalization error. Then, for each random index $U_i$, we may consider the expected single-letter empirical generalization error $\bE [\empgen(W,\tilde{Z}_{i,0},\tilde{Z}_{i,1},U_i)]$ where the expectation is taken with respect to the joint distribution $\bP_{W,\tilde{Z}_i,U_i}$. Hence, all the results obtained in~\Cref{sec:bounds_using_conditional_mutual_information} for the expected generalization error with respect to the conditional mutual information $\minf(W;U|\tilde{S})$ can be replicated for the single-letter empirical generalization error with respect to the single-letter conditional mutual information $\minf(W;U_i|\tilde{Z}_{i,0},\tilde{Z}_{i,1})$. This is what we did in~\citep{rodriguez2020randomsubset} to obtain the following result, which was re-discovered later by \citet{zhou2022individually}.

\begin{theorem}
\label{th:single_letter_cmi_bound_bounded}
	Consider a loss function with a range contained in $[a,b]$. Then, 
	\begin{equation*}
		\bE \big[ \gen(W,S) \big] \leq \frac{b-a}{n} \sum_{i=1}^n \sqrt{2 \minf(W;U_i|\tilde{Z}_{i,0}, \tilde{Z}_{i,1})}.
	\end{equation*}
\end{theorem}

\Cref{th:single_letter_cmi_bound_bounded} also expands the intuition obtained from~\Cref{sec:bounds_using_conditional_mutual_information}. This single-letter bound tells us that the generalization error degrades the more we can identify \emph{a single} training sample when presented with another sample from the same distribution on average. This is in contrast to the bounds in~\Cref{sec:bounds_using_conditional_mutual_information}, where the bounds state that the generalization error degrades the more we can identify the training samples when presented with an additional training set sampled from the same distribution. Similarly to what we discussed above, the conditional mutual information of the hypothesis and the indices given the super sample considers artificial, spurious dependencies between the instances and the indices generated by the hypothesis. Mathematically, this can also be understood using the \emph{chain rule} of the mutual information from~\Cref{prop:properties_minf}, namely
\begin{align}
	\minf(W;U|\tilde{S}) &= \sum_{i=1}^n \minf(W;U_i | \tilde{S}, U^{i-1} ) \nonumber \\
		&= \sum_{i=1}^n \big\{ \minf(W, \tilde{S}^{-i}, U^{i-1}; U_i | \tilde{Z}_{i,0}, \tilde{Z}_{i,1}) - \minf( \tilde{S}^{-i}, U^{i-1}; U_i | \tilde{Z}_{i,0}, \tilde{Z}_{i,1}) \big\} \nonumber \\
		&= \sum_{i=1}^n \big \{ \minf(W;U_i | \tilde{Z}_{i,0}, \tilde{Z}_{i,1} ) + \minf( \tilde{S}^{-i}, U^{i-1}; U_i | \tilde{Z}_{i,0}, \tilde{Z}_{i,1}, W)  \big\} \nonumber \\
		\label{eq:cmi_larger_sum_single_letter_cmi}
		&\geq \sum_{i=1}^n  \minf(W;U_i | \tilde{Z}_{i,0}, \tilde{Z}_{i,1} ),
\end{align}
where $\tilde{S} = S \setminus \{ \tilde{Z}_{i,0}, \tilde{Z}_{i,1} \}$, where the artificial dependence is captured by the conditional mutual information $\minf( \tilde{S}^{-i}, U^{i-1}; U_i | \tilde{Z}_{i,0}, \tilde{Z}_{i,1}, W)$, and where $\minf( \tilde{S}^{-i}, U^{i-1}; U_i | \tilde{Z}_{i,0}, \tilde{Z}_{i,1}) = 0$ as the instances and the indices are all independent of each other. As before, consider the Gaussian location model as an example with only two instances. In the randomized-subsample setting, we consider a super sample consisting of four instances $\tilde{Z}_{0,0}, \tilde{Z}_{0,1}, \tilde{Z}_{1,0}$, and $\tilde{Z}_{1,1}$ and two indices $U_0$ and $U_1$ that select the training instances as $Z_0 = \tilde{Z}_{0,U_0}$ and $Z_1 = \tilde{Z}_{1,U_1}$. The samples and the indices are independent and hence $\minf(\tilde{Z}_{0,0}, \tilde{Z}_{0,1}, U_0 ; U_1 | \tilde{Z}_{1,0}, \tilde{Z}_{1,1}) = 0$. However, once the average between the two training instances is known, they become dependent and 
\begin{align*}
	\minf\Big(\tilde{Z}_{0,0}, \tilde{Z}_{0,1},& U_0 ; U_1 \Big| \tilde{Z}_{1,0}, \tilde{Z}_{1,1}, \frac{\tilde{Z}_{0,U_0} - \tilde{Z}_{1,U_1}}{2} \Big) \geq \\
	& \minf\Big(\tilde{Z}_{0,U_0} ; U_1 \Big| \tilde{Z}_{1,0}, \tilde{Z}_{1,1}, \frac{\tilde{Z}_{0,U_0} - \tilde{Z}_{1,U_1}}{2} \Big) \geq 0,
\end{align*}
where the first inequality comes from the \emph{permutation invariance} and the more data more information properties of the mutual information from~\Cref{prop:properties_minf}.

\subsection{Random-Subset Bounds}
\label{subsec:random_subset_bounds}

Often, it is useful to consider conditional versions of the mutual information where some samples are known in order to derive comprehensible results for known algorithms. An example of such a situation is given later in~\Cref{sec:noisy_iterative_learning_algos}. To obtain this kind of bounds one may want to note that, for a fixed hypothesis $w \in \cW$ and a fixed training set $s \in \cZ^n$, the generalization error can be written as
\begin{equation*}
	\gen(w,s) = \poprisk(w) - \bE \big[\emprisk(w,s_J) \big],
\end{equation*}
where $J$ is a uniformly distributed random subset of $[n]$ such that $|J| = m$ and $s_J = \{z_i \}_{i \in J}$. To note this, we need to realize that since $J$ is uniformly distributed, there are $\binom{n}{m}$ possible subsets of size $m$ in $[n]$, and that each sample $z_i \in s$ belongs to $\binom{n-1}{m-1}$ of those subsets. Hence, 
\begin{align*}
	\bE \big[\emprisk(w,s_J) \big] &= \frac{1}{\binom{n}{m}} \sum_{j \in \cJ} \frac{1}{m} \sum_{i \in j} \ell(w,z_i)  \\
		&= \frac{\binom{n-1}{m-1}}{\binom{n}{m}} \cdot \frac{1}{m} \sum_{i=1}^n \ell(w,z_i) \\
		&= \frac{1}{n} \sum_{i=1}^n \ell(w,z_i) = \emprisk(w,s).
\end{align*}

\begin{remark}
\label{rem:extra_randomness}
Similarly, it is often useful to consider that a source of randomness in the algorithm is known. This can be modeled assuming that the hypothesis depends on a random variable $R$ that is known to us and that is independent of all other random variables in the system. Incorporating this known variable into the bounds on the expected generalization error is simple as we can always write $\bE[\gen(W,S)] = \bE[ \bE^R [ \gen(W,S) ] ]$ and, since $R$ only depends on the hypothesis $W$, then we may proceed to bound $\bE^R [ \gen(W,S) ]$ exactly as before in \Cref{sec:bounds_using_mutual_information,sec:bounds_using_conditional_mutual_information} substituting the mutual information $\minf$ by the disintegrated mutual information $\minf^R$ and taking the expectation with respect to $R$ afterward. However, as in this monograph we will only use the extra random variable $R$ for random-subset bounds, we chose not to include it previously to simplify the exposition.
\end{remark}

With this in mind, consider a fixed realization $r$ of the random variable $R$, a fixed subset $j \subseteq [n]$, and a fixed set of samples $s_{j^c}$. Then, the expected generalization error $\bE^{R=r, J=j, S_{j^c} = s_{j^c}} [ \gen(W,S_j)]$ can be bounded from above using the techniques described in \Cref{subsec:single_letter_bounds,sec:bounds_using_mutual_information}. To see that, if the loss has a CGF bounded by $\psi$, then
\begin{align*}
	\bE^{R=r, J=j, S_{j^c} = s_{j^c}} \big[ \gen(W,S_j) \big] &\leq \frac{1}{m} \sum_{i \in j} \psi_*^{-1} \big( \minf^{R=r, J=j, S_{j^c} = s_{j^c}}(W;Z_i) \big) \textnormal{ and }\\
	\bE^{R=r, J=j, S_{j^c} = s_{j^c}} \big[ \gen(W,S_j) \big] & \leq \psi_*^{-1} \Big( \frac{\minf^{R=r, J=j, S_{j^c} = s_{j^c}}(W;S_j)}{m} \Big).
\end{align*}
Finally, taking the expectation to both sides leads to the following result, which is a generalization of our result~\citep[Proposition 2]{rodriguez2020randomsubset} extending~\citep[Theorem 2.4]{negrea2019information}.

\begin{theorem}[{Generalization of the our extension of \citep[Theorem 2.4]{negrea2019information}}]
\label{th:random_subset_mi_bound_general_cgf}
	Consider a loss function with a bounded CGF (\Cref{def:bounded_cgf}). Also consider a random subset $J \subseteq [n]$ such that $|J| = m$, which is uniformly distributed and independent of $W$ and $S$, and a random variable $R$ that is independent of $S$. Then, 
	\begin{align*}
		\bE \big[ \gen(W,S) \big] &\leq \frac{1}{m} \bE \left[  \sum_{i \in J} \psi_*^{-1}  \big( \minf^{R, J, S_{J^c}}(W;Z_i) \big)\right] \textnormal{ and} \\
		\bE \big[ \gen(W,S) \big] &\leq \bE \bigg[ \psi_*^{-1} \Big( \frac{\minf^{R, J, S_{j^c}}(W;S_j)}{m} \Big) \bigg].
	\end{align*}
\end{theorem}

\begin{remark}
	As we did in~\Cref{subsec:single_letter_bounds} in \Cref{th:single_letter_mi_bound_bounded_fast_rate,th:single_letter_mi_bounded_moments}, the random-subset result from \Cref{th:random_subset_mi_bound_general_cgf} can be replicated to obtain analogue results for losses with a bounded range and with bounded moments. As these results are neither the tighter (as we will see later in \Cref{subsec:comparison_of_the_bounds}) nor are used later in the monograph, they are not explicitly written.
\end{remark}

As previously, we can consider a similar situation in the randomized-subsample setting. In this case, consider again a fixed hypothesis $w \in \cW$, a fixed super sample $\tilde{s} \in \cZ^{n \times 2}$, and a fixed sequence of indices $u \in \{0,1\}^n$. Then, the empirical generalization error can be written as 
\begin{equation*}
	\empgen(w,\tilde{s}, u) = \bE \big[ \emprisk(w,s_{\textnormal{ghost},J})\big] - \bE \big[ \emprisk(w,s_J) \big]
\end{equation*}
where again $J$ is a uniformly distributed random subset of $[n]$ such that $|J| = m$. Using the same argument as before we observe that $\bE \big[ \emprisk(w,s_{\textnormal{ghost},J})\big] =  \emprisk(w,s_{\textnormal{ghost}})$ and $\bE \big[ \emprisk(w,s_J) \big] =  \emprisk(w,s)$. Therefore, consider a fixed realization $r$ of a random variable $R$ that is only dependent to the hypothesis $W$ and independent of all other variables of the system, a fixed subset $j \subseteq [n]$, a fixed super set $\tilde{s}$, and a fixed sequence of indices $u_{j^c}$. Then, the expected empirical generalization error $\bE^{R=r, J=j, \tilde{S}=\tilde{s}, U_{j^c} = u_{j^c}} [ \empgen(W,\tilde{s}_j,U_j) ]$ can be bounded from above using the techniques described in~\Cref{sec:bounds_using_conditional_mutual_information,subsec:single_letter_bounds}. To see that, if the loss has a range bounded in $[a,b]$, then
\begin{align*}
	\bE^{R=r, J=j, \tilde{S}=\tilde{s}, U_{j^c} = u_{j^c}} \big[ &\empgen(W,\tilde{s}_j, U_j) \big] \leq \\
    &\frac{b-a}{m} \sum_{i \in j}  \sqrt{ 2 \minf^{R=r, J=j, \tilde{S}=\tilde{s}, U_{j^c} = u_{j^c}}(W;U_i)}
\end{align*}
and
\begin{align*}
	\bE^{R=r, J=j, \tilde{S}=\tilde{s}, U_{j^c} = u_{j^c}} \big[ &\empgen(W,\tilde{s}_j, U_j) \big] \leq \\
    &  (b-a)  \sqrt{ \frac{2 \minf^{R=r, J=j, \tilde{S}=\tilde{s}, U_{j^c} = u_{j^c}}(W;U_j)}{m}}
\end{align*}
Finally, taking the expectation to both sides leads to our~\citep[Proposition 4]{rodriguez2020randomsubset}, although a similar result could be directly obtained combining a slight modification of~\citep[Theorem 3.1]{haghifam2020sharpened} (where $R$ is included) and~\citep[Lemma 3.6]{haghifam2020sharpened}.

\begin{theorem}
\label{th:random_subset_cmi}
	Consider a loss function with a range contained in $[a,b]$. Also consider a random subset $J \subseteq [n]$ such that $|J| = m$, which is uniformly distributed and independent of $W$, $\tilde{S}$, and $U$, and a random variable $R$ that is indepdendent of $\tilde{S}$, and $U$. Then, 
	\begin{align*}
		\bE \big[ \gen(W,S) \big] &\leq \frac{b-a}{m} \bE\left[ \sum_{i \in J} \sqrt{ 2 \minf^{R, J, \tilde{S}, U_{j^c}}(W;U_i)} \right] \textnormal{ and} \\
		\bE \big[ \gen(W,S) \big] &\leq (b-a) \bE \left[ \sqrt{ \frac{2 \minf^{R, J, \tilde{S}, U_{j^c}}(W;U_j)}{m}} \right].
	\end{align*}
\end{theorem}

As mentioned at the start of this subsection, the appeal of the random-subset bounds is that they present \emph{disintegrated, conditional} versions of the mutual information, where some elements such as large subsets of the data (or the indices in the randomized-subsample setting) and other sources of the randomness of the algorithm are known. 
In~\citep{rodriguez2020randomsubset}, we realize that the approach presented so far is limited, in the sense that there are no more random objects that can be conditioned in the relative entropy. That is, \Cref{th:random_subset_mi_bound_general_cgf,th:random_subset_cmi} contain the maximum conditioning possible for the approach of   finding random variables whose expectation is equal to the expected generalization error, and then bounding these random variables using the results from ~\Cref{sec:bounds_using_mutual_information,sec:bounds_using_conditional_mutual_information}.
However, for bounded losses, it is possible to derive tighter bounds where even more of the random elements are conditioned. One way to do so is with a dedicated analysis of the Donsker and Varadhan~\Cref{lemma:dv_and_gvp} as shown by~\citet[Theorem 2.5]{negrea2019information} and~\citet[Theorem 3.7]{haghifam2020sharpened}. A simpler way to achieve similar bounds, without sacrificing the understanding of the generalization error random variable and the distinct random variables with equal expectation, is to first bound the generalization error with the Kantorovich--Rubenstein duality from~\Cref{lemma:kantorovich_rubinstein_duality}, realize that the Wasserstein distance is dominated by the total variation, and finally bound the total variation with Pinsker's and Bretagnolle--Huber's inequalities from~\Cref{lemma:pinsker-inequality,lemma:bretagnolle-huber-inequality}. An extra benefit from this approach is that the obtained bound ensures that heavily influential samples (from which the hypothesis has a high amount of information) do not contribute too negatively to the bound, which is ensured to be non-vacuous. The derivation of similar bounds and the discussion of the method will appear shortly in~\Cref{sec:bounds_using_wasserstein_distance}. For now, we only introduce the bounds.

\begin{theorem}[{\citet[Theorem 2.5]{negrea2019information}}]
\label{th:tighter_random_subset_mi_bounded_loss}
	Consider a loss function with a range contained in $[a,b]$. Also consider a random index $J \in [n]$, which is uniformly distributed and independent of $W$ and $S$, and a random variable $R$ that is independent of $S$. Then, for every Markov kernel $\bQ$ from $\cZ^{n-1} \otimes \cR \otimes J$ to distributions on $\cW$
	\begin{equation*}
		\bE \big[ \gen(W,S) \big] \leq \frac{b-a}{\sqrt{2}} \bE \mleft[ \sqrt{ \relent \big(\bP_W^{S, R} \Vert \bQ(S^{-J}, R, J) \big)}\mright].
	\end{equation*}
\end{theorem}

\begin{theorem}[{Extension of \citep[Theorem 3.7]{haghifam2020sharpened}}]
\label{th:tighter_random_subset_cmi_bounded_loss}
	Consider a loss function with a range contained in $[a,b]$. Also consider a random index $J \in [n]$, which is uniformly distributed and independent of $W$, $\tilde{S}$, and $U$, and a random variable $R$ that is independent of $\tilde{S}$, and $U$. Then, for every Markov kernel $\bQ$ from $\cZ^{n \times 2} \otimes \{ 0, 1 \}^{n-1} \otimes \cR \otimes \cJ$ to distributions on $\cW$
	\begin{equation*}
		\bE \big[ \gen(W,S) \big] \leq \sqrt{2} (b-a) \bE \mleft[ \sqrt{ \relent \big(\bP_W^{\tilde{S}, U, R} \Vert \bQ(\tilde{S}, U_{J^c}, R) \big)} \mright].
	\end{equation*}
\end{theorem}

Notably, the two equations in~\Cref{th:tighter_random_subset_mi_bounded_loss,th:tighter_random_subset_cmi_bounded_loss} are only given for $m=1$. This case, when $m = 1$, will be of interest in~\Cref{sec:noisy_iterative_learning_algos} as it is the setting in which the two distributions in the relative entropy terms of~\Cref{th:tighter_random_subset_mi_bounded_loss,th:tighter_random_subset_cmi_bounded_loss} are the closest to each other. In fact, \citet{negrea2019information} saw that the case $m=1$ is optimal.

\subsection{Comparison of the Bounds}
\label{subsec:comparison_of_the_bounds}

In~\Cref{subsec:single_letter_bounds}, we already noted that the single-letter mutual and conditional mutual information were tighter than the full mutual and conditional mutual information counterparts as they did not account for artificial dependencies between the instances or between the instances and the indices given the mutual information, c.f.~\eqref{eq:mi_larger_sum_single_letter_mi} and~\eqref{eq:cmi_larger_sum_single_letter_cmi}. This also implies that single-letter bounds (e.g.~\Cref{th:single_letter_mi_bound_general_cgf}) are tighter than the full mutual information counterparts (e.g.~\Cref{th:mi_bound_general_cgf}). This is clear in the fast-rate bounds for bounded losses, but it is also true for more general losses as one can first employ Jensen's inequality to pull the average of single-letter mutual information terms inside of the concave functions $\psi_*^{-1}$, $\sqrt{\cdot}$, $(\cdot)^{\frac{p-1}{p}}$ for $p > 1$. Only for the more complicated bounds for losses with a bounded variance the two bounds are not comparable.

Therefore, it is of interest to compare the different mutual information terms to get a mental map of the relationships between the different bounds. That was the purpose of our~\citep[Appendix J]{rodriguez2020randomsubset} and~\citep[Appendix D]{rodriguez2021tighter}. Since the random-subset bounds are incomparable when all the expectations are outside of the concave functions, we will consider weaker versions where the expectation is inside the concave function after applying Jensen's inequality. Furthermore, as the most relevant case is the one where $m=1$, this will be the one studied. Finally, we will drop the dependence on $R$ for the comparison, although the results also hold when it is present as per~\Cref{rem:extra_randomness}. Hence, for the random-subset bounds the quantities to study are $\bE[\minf(W;Z_J|S^{-J})]$ and $\bE[\minf(W;U_J|\tilde{S},U^{-J})]$. A summary of these relationships is given in~\Cref{fig:comparison_mutual_information_bounds}.

\begin{figure}[ht]
\centering
\begin{tikzpicture}[scale=0.8, transform shape]
		\node [align=center] at (3.5,4) {\Cref{th:single_letter_mi_bound_general_cgf}~\citep{bu2020tightening}\\ $\frac{(b-a)}{n}\sum_{i=1}^n \sqrt{\frac{\minf(W;Z_i)}{2}}$};	
		
		\node [align=center, rotate=-90] at (3.5,3) {$\leq$};
			
		\node [align=center] at (3.5,0) {\Cref{th:random_subset_mi_bound_general_cgf}~\citep{negrea2019information} (after Jensen's)\\ $(b-a) \sqrt{\frac{\bE[I(W;S_J|S^{-J})]}{2}}$};	
	
		\node [align=center, rotate=-90] at (3.5,1) {$\leq$};
	
		\node [align=center] at (3.5,2) {\Cref{th:mi_bound_bounded}~\citep{xu2017information} \\ $(b-a) \sqrt{\frac{I(W;S)}{2n}}$};
		
		\node [align=center] at (-4,4) {\Cref{th:single_letter_mi_bound_general_cgf} \\ $\frac{(b-a)}{n}\sum_{i=1}^n \sqrt{2 I(W;U_i|\tilde{Z}_{i,0}, \tilde{Z}_{i,1})}$};		
		
		\node [align=center] at (-0.25,4-0.25) {$\leq$};
		\node [align=center] at (-0.25,2-0.25-0.5){\footnotesize{$\big( \textnormal{if } 3 \minf(W;U|\tilde{S}) \leq \minf(W;\tilde{S}) \big)$}};
		
		\node [align=center, rotate=-90] at (-4,3) {$\leq$};	
		
		\node [align=center] at (-4,0) {\Cref{th:random_subset_cmi}~\citep{haghifam2020sharpened} (after Jensen's)\\ $(b-a) \sqrt{2 \bE[I(W;U_J|\tilde{S},U^{-J})]}$};	
		
		\node [align=center] at (-0.25,2-0.25) {$\leq$};
		\node [align=center] at (-0.25,4-0.25-0.5) {\footnotesize{$\big( \textnormal{if } 3 \minf(W;U_i|\tilde{Z}_{i,0}, \tilde{Z}_{i,1}) \leq \minf(W;\tilde{Z}_{i,0}, \tilde{Z}_{i,1}) \big)$}};
	
		\node [align=center, rotate=-90] at (-4,1) {$\leq$};
	
		\node [align=center] at (-4,2)  {\Cref{th:cmi_bound_bounded}~\citep{steinke2020reasoning}\\ $(b-a)\sqrt{\frac{2 \minf(W;U|\tilde{S})}{n}}$};
		
		\node [align=center] at (-0.25,0-0.25) {$\leq$};
		\node [align=center] at (-0.25,0-0.25-0.5) {\footnotesize{$\big( \textnormal{if } 3 \minf(W;U_j|\tilde{S},U^{-i}) \leq I(W;\tilde{Z}_{i,0}, \tilde{Z}_{i,1}|S^{-i}) \big)$}};
	\end{tikzpicture}
    \caption{Summary of the comparison between the full, the single-letter, and the random-subset bounds (weakened using Jensen's inequality) of both the standard and the randomized-subsample setting for losses with a bounded range.}\label{fig:comparison_mutual_information_bounds}
\end{figure}

\subsubsection{Comparison in the Standard Setting}

\looseness=-1 In the standard setting, the ordering of the bounds follows the proposition below.

\begin{proposition} 
\label{prop:comparison_minf}
	Let $J$ be a uniformly distributed random index of $[n]$ that is independent of $W$, $\tilde{S}$, and $U$. Then,
	\begin{equation*}
		\frac{1}{n} \sum_{i=1}^n \minf(W;Z_i) \leq \frac{\minf(W;S)}{n} \leq \bE \big[ \minf(W;Z_J|S^{-J}) \big].
	\end{equation*}
\end{proposition}

The first inequality of this proposition was proven in~\eqref{eq:mi_larger_sum_single_letter_mi} and previously in~\citep[Proposition 2]{bu2020tightening}. For the second inequality, note that $\bE \big[ \minf(W;Z_J|S^{-J}) \big] = \frac{1}{n} \sum_{i=1}^n \minf(W;Z_i|S^{-i})$. Then, the proof proceeds similarly to~\eqref{eq:mi_larger_sum_single_letter_mi}. More precisely,
\begin{align*}
	\sum_{i=1}^n \minf(W;Z_i|S^{-i}) &= \sum_{i=1}^n \big\{ \minf(W, S^{-i} ;Z_i) - \minf(Z_i; S^{-i}) \big\} \\
		&= \sum_{i=1}^n \big\{ \minf(Z^{i-1};Z_i) + \minf(W;Z_i|Z^{i-1}) + \minf(Z_{i+1}^n; Z_i | W, Z^{i-1}) \big\} \\
		&\geq \minf(W;Z_i|Z^{i-1}),
\end{align*}
where the first and the second inequalities follow from the \emph{chain rule} of the mutual information and the independence of the instances.
Therefore, the random-subset bounds are capturing more dependencies than the full mutual information bounds and therefore more than the single-letter bounds. 

\subsubsection{Comparison in the randomized-subsample Setting}

In the randomized-subsample setting, the bounds can be similarly ordered providing us with similar insights.

\begin{proposition}
\label{prop:comparison_minf_rs}
	Let $J$ be a uniformly distributed random index of $[n]$ that is independent of $W$ and $S$. Then,
	\begin{equation*}
		\frac{1}{n} \sum_{i=1}^n \minf(W;U_i|\tilde{Z}_{i,0},\tilde{Z}_{i,1}) \leq \frac{1}{n} \sum_{i=1}^n \minf(W;U_i|\tilde{S}) \leq \frac{\minf(W;U|\tilde{S})}{n} \leq \bE \big[ \minf(W;U_J|\tilde{S},U^{-J}) \big].
	\end{equation*}
\end{proposition}

The relationship between the first and the third element was proven previously in~\eqref{eq:mi_larger_sum_single_letter_mi}. The second element did not appear previously in the monograph as it will not be employed and it is looser than the single-letter bound shown here. Nonetheless, this was the state-of-the art bound prior to the development of our single-letter bound in~\citep{rodriguez2020randomsubset}. For completeness, we provide here a proof of each of these inequalities in a similar fashion to before, by repeatedly using the \emph{chain rule} of the mutual information and the fact that the instances and indices are mutually independent of each other. Namely, 

\begin{align*}
	\sum_{i=1}^n \minf(W;U_i | \tilde{S}, U^{-i}) 
	&= \sum_{i=1}^n \big \{ \minf(W,U^{-i}; U_i | \tilde{S}) - \minf(U^{-i};U_i | \tilde{S}) \big\} \\
	&= \sum_{i=1}^n \big \{ \minf(U^{i-1};U_i | \tilde{S}) + \minf(W;U_i | \tilde{S}, U^{i-1}) + \minf(U_{i+1}^n; U_i | \tilde{S}, U^{i-1}, W) \big\} \\
	&\geq \sum_{i=1}^n \big \{ \minf(W;U_i | \tilde{S}, U^{i-1}) \big\} = I(W;U | \tilde{S}) \\
	&= \sum_{i=1}^n \big \{ \minf(W,U^{i-1};U_i | \tilde{S}) - \minf(U^{i-1};U_i | \tilde{S}) \big\} \\
	&= \sum_{i=1}^n \big \{ \minf(W;U_i | \tilde{S}) + \minf(U^{i-1};U_i | \tilde{S}) \big\} \\
	&\geq \sum_{i=1}^n \minf(W;U_i | \tilde{S}) \\
	&= \sum_{i=1}^n \big \{ \minf(W, \tilde{S} ; U_i) - \minf(\tilde{S}; U_i) \big\} \\
	&= \sum_{i=1}^n \big \{ \minf(W, \tilde{Z}_{i,0}, \tilde{Z}_{i,1}; U_i) + \minf(\tilde{S}^{-i} ; U_i | \tilde{Z}_{i,0}, \tilde{Z}_{i,1})\big\} \\
	&\geq \sum_{i=1}^n \big \{ \minf(W; U_i | \tilde{Z}_{i,0}, \tilde{Z}_{i,1})  +  \minf(\tilde{Z}_{i,0}, \tilde{Z}_{i,1} ; U_i) \big\} \\
	&= \sum_{i=1}^n \minf(W; U_i | \tilde{Z}_{i,0}, \tilde{Z}_{i,1})
\end{align*}

\subsubsection{Comparison between the settings}

It is easy to see that $\minf(W;U|\tilde{S}) \leq \minf(W;S)$~\citep{haghifam2020sharpened}. To see that, note that the Markov chain $(\tilde{S}, U) - S - W$ holds. Then, by the data processing inequality from~\Cref{prop:properties_minf} it follows that $\minf(W;S) \geq \minf(W; \tilde{S}, U)$. Finally, by the \emph{chain rule} of the mutual information and its non-negativity it follows that $$\minf(W;\tilde{S}, U) = \minf(W;U|\tilde{S}) + \minf(W;\tilde{S}) \geq \minf(W;U|\tilde{S}).$$ 
An analogous reasoning reveals similar insights for the single-letter and random-subset terms. Nonetheless, the additional factor of two in the bounds of the randomized-subsample setting makes the comparison between the different settings harder.

\citet{hellstrom2020generalization} note that since $S$ is a deterministic function of $\tilde{S}$ and $U$, then $\minf(W;S) = \minf(W;U|\tilde{S}) + \minf(W;\tilde{S})$ holds with equality. Then, for bounded losses, the bound from the randomized-subsample setting from~\Cref{th:cmi_bound_bounded} is tighter than the one from the standard setting from~\Cref{th:mi_bound_bounded} if $3 \minf(W;U|\tilde{S}) \leq \minf(W;\tilde{S})$.

Similarly, we can show that $$\minf(W;Z_i) = \minf(W;U_i | \tilde{Z}_{i,0}, \tilde{Z}_{i,1}) + \minf(W;\tilde{Z}_{i,0}, \tilde{Z}_{i,1}).$$ To see that, note that $\minf(W;U_i) = 0$ as the index $U_i$ is independent of the hypothesis; that $\minf(W; \tilde{Z}_{i, 1-U_i} | U_i) = 0$ as, given the index $U_i$, the hypothesis is independent of the sample not used for training $\tilde{Z}_{i, 1-U_i}$; and that $\minf(W;\tilde{Z}_{i,U_i} | U_i, \tilde{Z}_{i, 1- U_i}) = \minf(W;Z_i)$ as, given the index $U_i$, the hypothesis is independent of the sample not used for training $\tilde{Z}_{i, 1-U_i}$ and the only dependence captured is the one of the hypothesis and the sample used for training $Z_i = \tilde{Z}_{i,U_i}$. With that in mind, using the \emph{chain rule} of the mutual information we have that 
\begin{align*}
    \minf(W;Z_i) &= \minf(W;U_i) + \minf(W; \tilde{Z}_{i, 1-U_i} | U_i) + \minf(W; \tilde{Z}_{i,U_i} | U_i, \tilde{Z}_{i, 1-U_i}) \\
    &= \minf(W; U_i, \tilde{Z}_{i,0}, \tilde{Z}_{i,1}).
\end{align*}
A further application of the \emph{chain rule} completes the proof. Therefore, \Cref{th:single_letter_cmi_bound_bounded} is tighter than the particularization of \Cref{th:single_letter_mi_bound_general_cgf} to bounded losses if $3 \minf(W;U_i | \tilde{Z}_{i,0}, \tilde{Z}_{i,1}) \leq \minf(W; \tilde{Z}_{i,0}, \tilde{Z}_{i,1}))$.

For the random-subset bounds, we can show that $$\minf(W;Z_i|S^{-i}) = \minf(W;U_i|\tilde{S}, U^{-i}) + \minf(W; \tilde{Z}_{i,0}, \tilde{Z}_{i,1} | S^{-i}).$$ Using the same argument as above, we can establish that $\minf(W;Z_i | S^{-i}) = \minf(W; \tilde{Z}_{i,0}, \tilde{Z}_{i,1}, U_i | S^{-i})$ since the conditioning on $S^{-i}$ does not change the fact that, given the index $U_i$, the hypothesis depends on the sample used for training $Z_i = \tilde{Z}_{i,U_i}$ and is independent of the one that is not used for training $\tilde{Z}_{i,1-U_i}$. Then, using the \emph{chain rule} of the mutual information we may establish that $$\minf(W;Z_i|S^{-i}) = \minf(W;U_i|\tilde{Z}_{i,0}, \tilde{Z}_{i,1}, S^{-i}) + \minf(W; \tilde{Z}_{i,0}, \tilde{Z}_{i,1} | S^{-i}).$$ All that is left to prove is that $\minf(W;U_i|\tilde{Z}_{i,0}, \tilde{Z}_{i,1}, S^{-i}) = \minf(W;U_i|\tilde{S}, U^{-i}) $. This can be seen by noting that, given the samples $\tilde{Z}_{i,0}$, $\tilde{Z}_{i,1}$, and the rest of the training set $\tilde{S}^{-i}$, then the hypothesis $W$ and the index $U_i$ are independent of the rest of the ghost samples $S_\textnormal{ghost}^{-i}$ and indices $U^{-i}$. Hence, $\minf(W;U_i|\tilde{Z}_{i,0}, \tilde{Z}_{i,1}, S^{-i}) = \minf(W;U_i|\tilde{Z}_{i,0}, \tilde{Z}_{i,1}, S^{-i}, S_\textnormal{ghost}^{-i}, U^{-i})$, which completes the equality. Therefore, \Cref{th:random_subset_cmi} is tighter than the particularization of \Cref{th:random_subset_mi_bound_general_cgf} to bounded losses if $3 \minf(W;U_i | \tilde{S}, U^{-i}) \leq \minf(W;\tilde{S}_i | S^{-i})$ for all indices $i \in [n]$. 

\section{Bounds Using the Wasserstein Distance}
\label{sec:bounds_using_wasserstein_distance}

From \Cref{sec:bounds_using_mutual_information,sec:bounds_using_conditional_mutual_information,sec:random_subset_and_single_letter}, we described how to obtain different bounds on the generalization error that depend on the relative entropy between different distributions. At their core, all these bounds come from a clever application of the Donsker and Varadhan~\Cref{lemma:dv_and_gvp} to decouple the hypothesis $W$ and the training data $S$ as shown in~\eqref{eq:dv_to_lambda_gen}.

More generally, this lemma can be employed to measure the difference between expectations. In~\eqref{eq:dv}, consider the measurable space $\cX$ and the random variables $X$ and $Y$ distributed according to $\bP$ and $\bQ$ respectively. Then, fix the function $g(x) = \lambda \big( f(x) - \bE[f(Y)] \big)$ for some auxiliary function $f$ and some $\lambda \in \bR$. In this way, the equivalent of~\eqref{eq:dv_to_lambda_gen} is summarized in the following lemma.
\begin{lemma}
\label{lemma:decoupling_relent}
Let $\cX$ be a measurable space and $X$ and $Y$ be two random variables distributed according to $\bP$ and $\bQ$ respectively. Consider a measurable function $f$ on $\cX$ such that the CGF $\Lambda_{g(Y)}$ exists, where $g(x) = \lambda\big(f(x) - \bE[f(Y)]\big)$. Then, for all $\lambda > 0$
\begin{equation}
    \label{eq:dv_to_lambda_f}
    \bE \big[ f(X) \big] - \bE \big[ f(Y) \big] \leq \frac{\relent(\bP \Vert \bQ) + \Lambda_{g(Y)}(\lambda)}{\lambda}.
\end{equation}
\end{lemma}
For example, \Cref{eq:dv_to_lambda_emp_gen}, from which most of the theorems presented so far stemmed from can be derived from~\eqref{eq:dv_to_lambda_f} by letting $X = (W,S)$, $Y = (W',S)$, and $f(w,s) = \gen(w,s)$, and by noting that $\bE[\gen(W',S)] = 0$. 

If instead of considering~\Cref{lemma:decoupling_relent}, we consider the Kantorovich--Rubinstein duality (\Cref{lemma:kantorovich_rubinstein_duality}), we may obtain a similar result with respect to the Wasserstein distance. 
Namely, consider some fixed dataset $s$. Then, let $X = W$, $Y = W'$, and $f(w;s) = \emprisk(w,s)$ in~\eqref{eq:kantorovich_rubinstein_duality_rv}, where $W$ and $W'$ are distributed according to the algorithm's distribution $\bP_W^{S=s}$ and the prototypical marginal distribution $\bP_W$ respectively. Then, if the loss $\ell(\cdot, z)$ is $L$-Lipschitz under some metric $\rho$ for all $z \in \cZ$, so is $\emprisk(w,s)$ for all $s \in \cZ^n$. Indeed, note that
\begin{equation*}
    \mleft| \frac{1}{n} \sum_{i=1}^n \ell(w, z_i) - \frac{1}{n} \sum_{i=1}^n \ell(w', z_i) \mright| \leq \frac{1}{n} \sum_{i=1}^n \mleft| \ell(w,z_i) - \ell(w, z_i) \mright| \leq L \rho(w, w'), 
\end{equation*}
where the first inequality comes from the triangle inequality and the second one from the Lipschitzness of $\ell$. Therefore, 
\begin{equation*}
    \bE^{S=s} \big[ \emprisk(W', s) \big] - \bE^{S=s} \big[ \emprisk(W, s) \big] \leq L \bW_\rho(\bP_W^{S=s}, \bP_W).
\end{equation*}
Finally, since $\bE \big[ \emprisk(W', S) \big] = \poprisk(W')$ as discussed in~\Cref{sec:bounds_using_mutual_information}, taking the expectation with respect to the training data recovers~\citep[Theorem 2]{wang2019information}. This effectively bounds the expected generalization error from above by the expected Wasserstein distance between the algorithm's output distribution $\bP_W^S$ and the hypothesis marginal distribution $\bP_W$. 

We can extend this reasoning to include any data-independent distribution $\bQ$. To do so, fix a dataset $s$ and let $X = W$ and $Y = W'$ be defined as before in~\eqref{eq:kantorovich_rubinstein_duality_rv}. However, now let $f(w,s) = \gen(w,s)$. If the loss $\ell(\cdot, z)$ is $L$-Lipschitz under some metric $\rho$ for all $z \in \cZ$, then so is $\poprisk(\cdot)$. Namely, 
\begin{equation*}
    \mleft| \poprisk(w) - \poprisk(w') \mright| \leq \bE \mleft[ \mleft| \ell(w,Z) - \ell(w',Z) \mright| \mright] \leq L  \rho(w, w'), 
\end{equation*}
where the first inequality comes from Jensen's inequality and the second one from the Lipschitzness of $\ell$. Therefore, by the triangle inequality $\gen(\cdot, s)$ is $2L$-Lipschitz for all $s \in \cZ^n$ and
\begin{equation*}
    \bE^{S=s} [ \gen(W,s) ] - \bE^{S=s} [ \gen(W',s)] \leq 2L  \bW_\rho(\bP_W^{S=s} , \bQ).
\end{equation*}
Finally, taking the expectation with respect to the training data and noting again that $\bE[\gen(W',S)] = 0$ extends upon~\citep[Theorem 2]{wang2019information} by allowing any data-independent distribution $\bQ$. This flexibility will prove useful, for example, in~\Cref{sec:noisy_iterative_learning_algos}. These two results are summarized below.

\begin{theorem}[{Extension of~\citep[Theorem 2]{wang2019information}}]
\label{th:wasserstein_full}
Consider an $L$-Lipschitz loss function $\ell(\cdot, z)$ with respect to some metric $\rho$ for all $z \in \cZ$. Then, for every data-independent distribution $\bQ$ on $\cW$
\begin{equation*}
    \bE \big[ \gen(W,S) \big] \leq L \bE \big[ \min \mleft \{ \bW_\rho(\bP_W^S, \bP_W), 2 \bW_\rho(\bP_W^{S} , \bQ) \mright \} \big].
\end{equation*}
\end{theorem}

The main appeal of bounds with the form of~\Cref{th:wasserstein_full} is that they consider the geometry of the space by construction. This is reminiscent of classical approaches considering the complexity of the hypothesis class such as the Rademacher complexity from~\Cref{sec:rademacher_complexity}, with the added benefit that they also capture the dependence of the algorithm and the training data. Note that $\bE \big[ \bW_\rho(\bP_W^S, \bP_W) \big]$ still measures the difference between the distribution of the hypothesis returned by the algorithm after observing a training set, and the prototypical distribution on samples from the data distribution $\bP_W = \bP_W^S \circ \bP_S$ as discussed in~\Cref{subsec:stability_via_information_measures,sec:information_theoretic_generalization,sec:bounds_using_mutual_information}. The consideration of the space geometry is lost in the bounds featuring the relative entropy unless the distributions are absolutely continuous to each other. For example, consider~\Cref{fig:example_wasserstein}, where different algorithm distributions $\bP_W^{S=s}$ in $\bR^2$ are shown for different datasets along with the data-independent distribution $\bQ$. Since $\bP_W^{S=s}$ is not absolutely continuous with respect to $\bQ$ for some training sets $s$, the relative entropy $\relent(\bP_W^{S=s} \Vert \bQ)$ will go to infinity, even though the two distributions are clearly not far from each other. On the other hand, the Wasserstein distance with some standard metric like the $\ell_2$ norm $\rho(x,y) = \rVert x - y \rVert$ captures this similarity.

\begin{figure}[ht]
    \centering
    \includegraphics[width=0.6\textwidth]{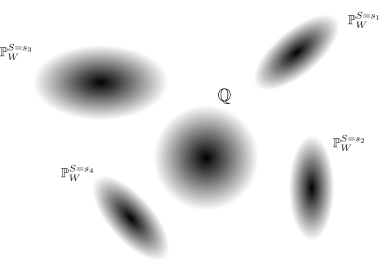}
    \caption{
    Example of four hypothesis distributions $\bP_W^{S=s_1}$, $\bP_W^{S=s_2}$, $\bP_W^{S=s_3}$, and $\bP_W^{S=s_4}$ after observing four different training sets, and a data-independent distribution $\bQ$ in $\bR^2$. The darker areas indicate regions with higher probability density and the white background indicates the region of probability zero. Even if the distributions characterizing the algorithm are not absolutely continuous with respect to the data-independent distribution $\bQ$, they are still close in $\bR^2$.}
    \label{fig:example_wasserstein}
\end{figure}

\looseness=-1 We recall from~\Cref{subsec:single_letter_bounds} that $\gen(w,s) = \frac{1}{n} \sum_{i=1}^n \gen(w,z_i)$ and that therefore $\bE[\gen(W,S)] = \frac{1}{n} \sum_{i=1}^n \bE[\gen(W,Z_i)]$ by the linearity of the expectation. Therefore, one may consider~\Cref{th:wasserstein_full} for each single-letter generalization error to obtain the following theorem, which generalizes the independently developed~\citep[Corollary 3]{zhang2021optimal} to algorithms that may consider the ordering of the samples, and extends our~\citep[Theorem 1]{rodriguez2021tighter} considering arbitrary data-independent distributions $\bQ$ instead of requiring it to be the prototypical marginal distribution $\bP_W$.

\begin{theorem}
\label{th:wasserstein_single_letter}
Consider an $L$-Lipschitz loss function $\ell(\cdot, z)$ with respect to some metric $\rho$ for all $z \in \cZ$. Then, for every data-independent distribution $\bQ$ on $\cW$
\begin{equation*}
    \bE \big[ \gen(W,S) \big] \leq \frac{L}{n} \sum_{i=1}^n \bE \mleft[ \min \mleft\{ \bW_\rho(\bP_W^{Z_i}, \bP_W), 2 \bW_\rho(\bP_W^{Z_i} , \bQ) \mright \} \mright].
\end{equation*}
\end{theorem}

Similarly, we recall from~\Cref{subsec:random_subset_bounds} that $\gen(w,s) = \poprisk(w) - \bE[ \emprisk(w,s_J) ]$, where $J$ is a uniformly distributed random subset of $[n]$ such that $|J| = m$ and $s_J = \{z_i \}_{i \in J}$. Therefore, one may fix a realization $r$ of the random variable $R$, a fixed subset $j \subseteq [n]$, and a fixed set of samples $s_{j^c}$ and employ~\Cref{th:wasserstein_full} and~\Cref{th:wasserstein_single_letter} to the expected generalization error $\bE^{R=r, J=j, S_{j^c}=s_{j^c}} [ \gen(W,S_j) ]$. Finally, taking the expectation with respect to $R$, $J$, and $S_{J^c}$ extends our~\citep[Theorem 2]{rodriguez2021tighter}.

\begin{theorem}
\label{th:wasserstein_random_subset}
Consider an $L$-Lipschitz loss function $\ell(\cdot, z)$ with respect to some metric $\rho$ for all $z \in \cZ$. Also consider a random subset $J \subseteq [n]$ such that $|J| = m$, which is uniformly distributed and independent of $W$ and $S$, and a random variable $R$ that is independent of $S$. Then, for every Markov kernel $\bQ$ from $\cZ^{n-m} \otimes \cR$ to distributions on $\cW$
\begin{align*}
    \bE \big[ \gen(W,S) \big] &\leq L \bE \mleft[ \min \mleft\{ \bW_\rho(\bP_W^{S,R}, \bP_W^{S_{J^c},R}), 2 \bW_\rho(\bP_W^{S,R}, \bQ(S_{J^c},R)) \mright \} \mright]  \textnormal{ and} \\
    \bE \big[ \gen(W,S) \big] &\leq \frac{L}{m} \bE \mleft[ \sum_{i \in J} \min \mleft\{ \bW_\rho(\bP_W^{S_{J^c},Z_i, R}, \bP_W^{S_{J^c}, R}), 2 \bW_\rho(\bP_W^{S_{J^c},Z_i, R} , \bQ(S_{J^c}, R)) \mright \} \mright].
\end{align*}
\end{theorem}

In particular, when $m=1$, both equations in~\Cref{th:wasserstein_random_subset} coincide and 
\begin{equation*}
    \bE \big[ \gen(W,S) \big] \leq L \bE \mleft[ \min \mleft\{ \bW_\rho(\bP_W^{S,R}, \bP_W^{S^{-J},R}), 2 \bW_\rho(\bP_W^{S,R}, \bQ(S^{-J},R)) \mright \} \mright].
\end{equation*}

After~\Cref{th:wasserstein_full}, we noted how the Wasserstein-based bounds resolved the relative entropy-based bounds' problem of non-absolutely continuous algorithm distributions $\bP_W^{S=s}$ with respect to the data independent distribution $\bQ$. Shortly, in~\Cref{subsec:implications_bounds_using_mutual_information}, we will show that even when the geometry of the space is ignored and the discrete metric is considered, the Wasserstein distance-based bounds improve upon the relative entropy-based bounds when the loss is bounded. Now, we will present a simple scenario where some Wasserstein-based bounds achieve the desired generalization error rate, while previous bounds based on the relative entropy fail to do so.

Recall the Gaussian location model from~\Cref{ex:glm} and let the loss function be the Euclidean distance $\ell(w,z) = \lVert w - z \rVert_2$. In this example, the expected generalization error can be calculated exactly (see \Cref{app:glm}):
\begin{equation*}
    \bE \big[ \gen(W,S) \big] = \sqrt{\frac{2 \sigma^2}{n}} \Big( \sqrt{n+1} - \sqrt{n-1} \Big) \frac{\Gamma\big(\frac{d+1}{2}\big)}{\Gamma\big(\frac{d}{2})} \in \cO \bigg(\frac{\sqrt{\sigma^2 d}}{n}\bigg).
\end{equation*}

As discussed in~\citep{bu2020tightening}, the bound from~\citep{xu2017information} is not applicable in this setting since $\minf(W;S) \to \infty$ and since $\ell(w,Z)$ is not sub-Gaussian given that $\mathrm{Var}[\ell(w,Z)] \to \infty$ as $\lVert w \rVert_2 \to \infty$.
When $d=1$, the loss $\ell(W',Z)$ is $1$-sub-Gaussian if $W'$ is distributed according to the marginal distribution $\bP_W$ and single-letter mutual information bound from~\citet{bu2020tightening} in~\Cref{th:single_letter_mi_bound_general_cgf} produces a bound in $\mathcal{O}\big(\sqrt{\nicefrac{\sigma^2}{n}}\big)$, which decreases slower than the true generalization error, see \Cref{fig:glm}. This happens since the bound grows as the square root of $\minf(W;Z_i)$, which is in $\mathcal{O}(\nicefrac{1}{n})$. 

In this scenario, the loss is 1-Lipschitz under $\rho(w,w') = \lVert w-w' \rVert_2$, and thus the bounds based on the Wasserstein distance are applicable.
Applying the bound from~\citet{wang2019information} in~\Cref{th:wasserstein_full} yields a bound in  $\mathcal{O}\big(\sqrt{\nicefrac{\sigma^2 d}{n}}\big)$, which decreases at the same sub-optimal rate as the single-letter mutual information bound.
However, both the single-letter and random-subset Wasserstein distance bounds from \Cref{th:wasserstein_single_letter,th:wasserstein_random_subset} produce bounds in $\mathcal{O}\big(\frac{\sqrt{\sigma^2 d}}{n} \big)$,  which decrease at the same rate as the true generalization error (see \Cref{fig:glm}). 

\begin{figure}[ht]
\centering
\includegraphics[width=0.495\textwidth]{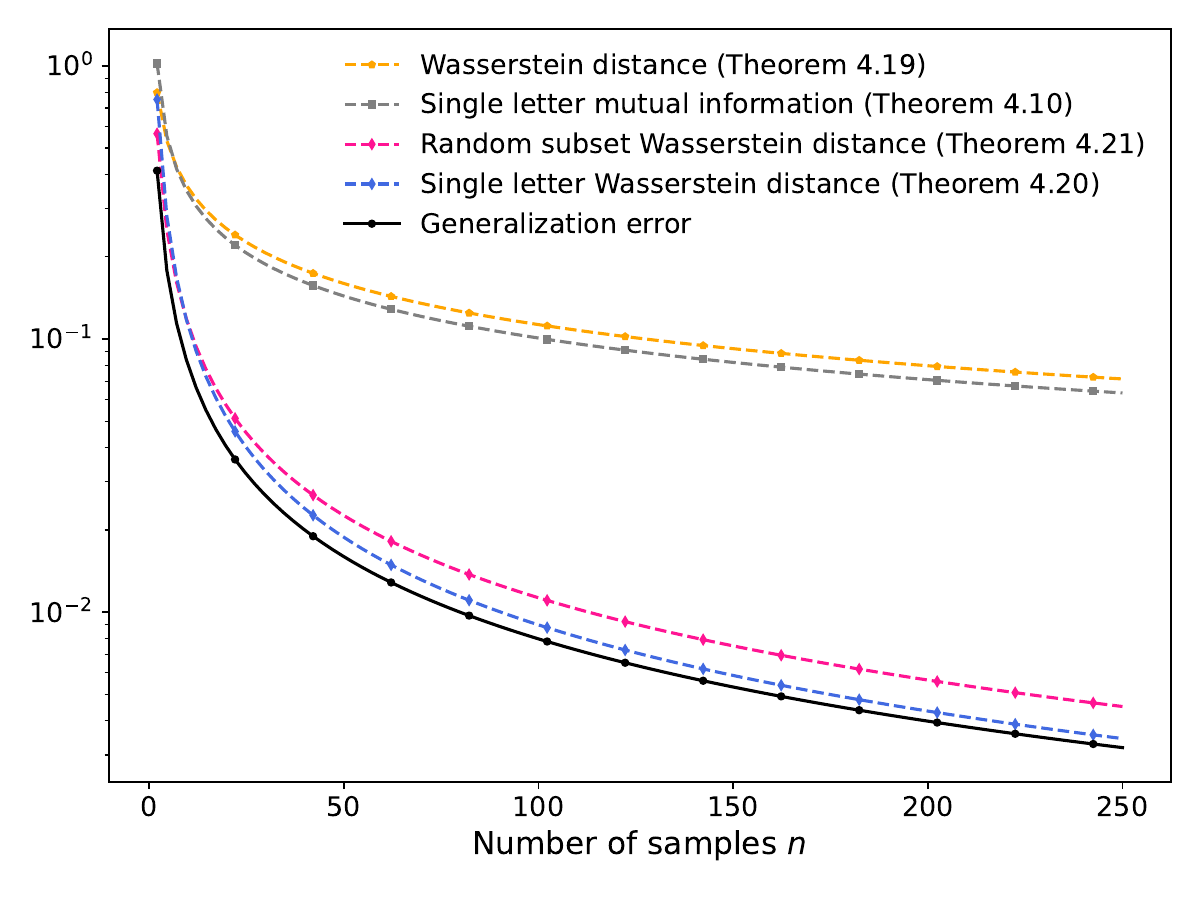}
\includegraphics[width=0.495\textwidth]{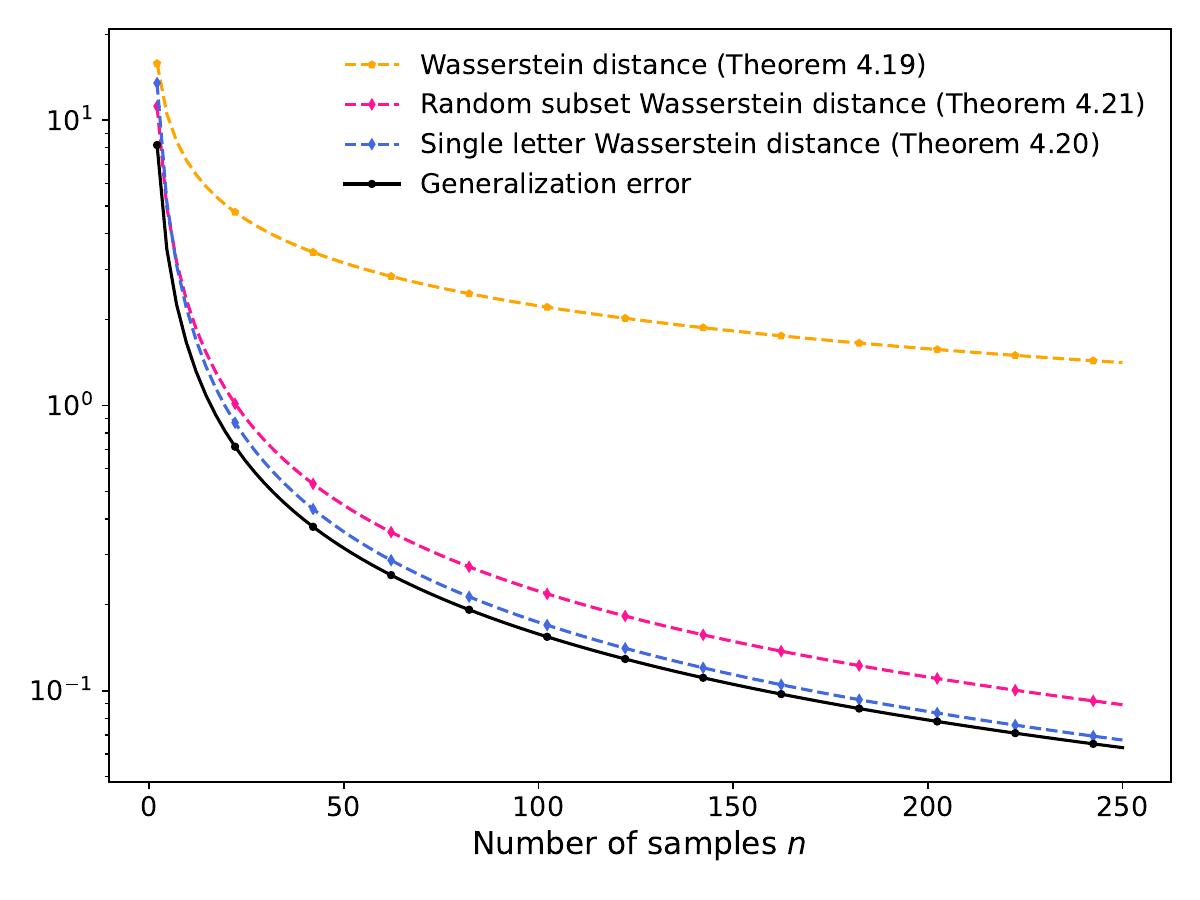}
\caption{Expected generalization error and generalization error bounds for the Gaussian location model from~\Cref{ex:glm} with $\mathcal{N}(\mu,1)$ (left) and $\mathcal{N}(\mu,I_{250})$ (right). See \Cref{app:glm} for the derivations and details.}
\label{fig:glm}
\end{figure}

The analysis that led to~\Cref{th:wasserstein_full,th:wasserstein_single_letter,th:wasserstein_random_subset} can be adapted to the randomized-subsample setting with a few modifications. Recall that the empirical generalization error is defined as $\empgen(w, \tilde{s}, u) = \emprisk(w,s_\textnormal{ghost}) - \emprisk(w,s)$. Then, even if it is true that
\begin{equation*}
    \bE[ \gen(W,S) ] = \bE[ \empgen(W, \tilde{S}, U) ] = \bE[ \emprisk(W,S_\textnormal{ghost})] - \bE[\emprisk(W,S)], 
\end{equation*}
the functions $\emprisk(\cdot, s_\textnormal{ghost})$ and $\emprisk(\cdot, s)$ are distinct and therefore the Kantorovich--Rubinstein duality from~\Cref{lemma:kantorovich_rubinstein_duality} cannot be invoked. Nonetheless, we may also note that
\begin{equation*}
    \empgen(W, \tilde{S}, U) = \emprisk(W,S_\textnormal{ghost}) - \emprisk(W,S) - \bE[ \emprisk(W', S_{\textnormal{ghost}}) - \emprisk(W', S)]
\end{equation*}
for every random variable $W'$ that is independent of the indices $U$ since the training data $S$ and the ghost samples $S_\textnormal{ghost}$ are identically distributed.

With this in mind, one may consider a fixed super sample $\tilde{s}$ and some fixed indices $u$. Then, as before, one may let $X = W$ and $Y = W'$ in~\eqref{eq:kantorovich_rubinstein_duality_rv}, where $W$ and $W'$ are distributed according to the algorithm's distribution $\bP_W^{S=s} = \bP_{W}^{\tilde{S}=\tilde{s}, U=u}$ a.s. and an indices-independent Markov kernel $\bQ(\tilde{s})$. Then, one may apply the Kantorovich--Rubinstein duality from~\Cref{lemma:kantorovich_rubinstein_duality} twice, first with $f(w;\tilde{s},u) = \emprisk(w,s_\textnormal{ghost})$ and then with $f(w;\tilde{s},u) = \emprisk(w,s)$. In this way, one has that
\begin{align*}
    \bE^{\tilde{S}=s,U=u}&[ \emprisk(W,s_\textnormal{ghost})] - \bE^{\tilde{S}=s,U=u}[ \emprisk(W',s_\textnormal{ghost})] \\
    &+ \bE^{\tilde{S}=s,U=u}[ \emprisk(W',s)] - \bE^{\tilde{S}=s,U=u}[ \emprisk(W,s)] 
    \leq 2 L \bW_\rho \big(\bP_W^{\tilde{S}=s, U=u}, \bQ(\tilde{s}) \big),
\end{align*}
\looseness=-1 where we used that the Wasserstein distance is a distance and therefore symmetric in its arguments. Finally, taking the expectation with respect to the super sample and the indices extends our~\citep[Equation (3)]{rodriguez2021tighter} by allowing any arbitrary Markov kernel $\bQ$.

\begin{theorem}
\label{th:wasserstein_full_rs}
Consider an $L$-Lipschitz loss function $\ell(\cdot, z)$ with respect to some metric $\rho$ for all $z \in \cZ$. Then, for every data-independent Markov kernel $\bQ$ from $\cZ^{n \times 2}$ to distributions on $\cW$
\begin{equation*}
    \bE \big[ \gen(W,S) \big] \leq 2L \bE \big[ \bW_\rho(\bP_W^{\tilde{S}, U} , \bQ(\tilde{S}))  \big]. 
\end{equation*}
\end{theorem}

We may recall again from~\Cref{subsec:single_letter_bounds} and~\eqref{eq:empgen_decomposition} that $\empgen(w,\tilde{s}, u) = \frac{1}{n} \sum_{i=1}^n \empgen(w,\tilde{z}_{i,0}, \tilde{z}_{i,1}, u_i)$, where $\empgen(w,\tilde{z}_{i,0}, \tilde{z}_{i,1}, u_i) = \ell(w,\tilde{z}_{i,{1-u_i}}) - \ell(w, \tilde{z}_{i,u_i})$ is the single-letter empirical generalization error. Hence, we may apply \Cref{th:wasserstein_full_rs} to each term to obtain a generalization of our~\citep[Theorem 3]{rodriguez2021tighter}.

\begin{theorem}
\label{th:wasserstein_single_letter_rs}
Consider an $L$-Lipschitz loss function $\ell(\cdot, z)$ with respect to some metric $\rho$ for all $z \in \cZ$. Then, for every data-independent Markov kernel $\bQ$ from $\cZ^{2}$ to distributions on $\cW$
\begin{equation*}
    \bE \big[ \gen(W,S) \big] \leq \frac{2L}{n} \sum_{i=1}^n \bE \mleft[ \bW_\rho(\bP_W^{\tilde{Z}_{i,0},\tilde{S}_{i,1}, U_i} , \bQ(\tilde{Z}_{i,0},\tilde{Z}_{i,1}))  \mright]. 
\end{equation*}
\end{theorem}

Similarly, we may recall from~\Cref{subsec:random_subset_bounds} how $\empgen(w,\tilde{s}, u) = \bE\big[ \emprisk(w, s_{\textnormal{ghost}, J}) \big] - \bE \big[ \emprisk(w,s_J) \big]$ for every hypothesis $w \in \cW$, super sample $\tilde{s} \in \cZ^{n \times 2}$, and indices $u \in \{ 0, 1 \}^n$ if $J$ is a uniformly distributed random subset of $[n]$ such that $|J| = m$. Therefore, we may apply \Cref{th:wasserstein_full_rs,th:wasserstein_single_letter_rs} to the expected generalization error $\bE^{R=r, J=j,\tilde{S}=s, U_{j^c}=u_{j^c}}[\gen(W,\tilde{s}_j, U_j)]$ and take the expectation with respect to the remaining random objects to obtain an extension of our~\citep[Theorem 4]{rodriguez2021tighter}.

\begin{theorem}
\label{th:wasserstein_random_subset_rs}
Consider an $L$-Lipschitz loss function $\ell(\cdot, z)$ with respect to some metric $\rho$ for all $z \in \cZ$. Also consider a random subset $J \subseteq [n]$ such that $|J| = m$, which is uniformly distributed and independent of $W$ and $S$, and a random variable $R$ that is independent of $S$. Then, for every Markov kernel $\bQ$ from $\cZ^{n \times 2} \otimes \{0, 1 \}^{n-m} \otimes \cR$ to distributions on $\cW$
\begin{align*}
    \bE \big[ \gen(W,S) \big] &\leq 2L \bE \mleft[ \bW_\rho(\bP_W^{\tilde{S},U,R}, \bQ(\tilde{S},U_{J^c},R))  \mright]  \textnormal{ and} \\
    \bE \big[ \gen(W,S) \big] &\leq \frac{2L}{m} \bE \mleft[ \sum_{i \in J} \bW_\rho(\bP_W^{\tilde{S}, U_{J^c}, U_i, R} , \bQ(\tilde{S}, U_{J^c}, R)) \mright].
\end{align*}
\end{theorem}

In particular, when $m=1$ both equations in~\Cref{th:wasserstein_random_subset_rs} coincide and 
\begin{equation*}
    \bE \big[ \gen(W,S) \big] \leq 2L \bE \mleft[ \bW_\rho(\bP_W^{\tilde{S},U,R}, \bQ(\tilde{S},U^{-J},R))  \mright] 
\end{equation*}

Now that we derived Wasserstein distance-based bounds analogous to those based in the relative entropy from~\Cref{sec:bounds_using_mutual_information,sec:bounds_using_conditional_mutual_information,sec:random_subset_and_single_letter}, it is important to understand how they relate to each other. Next, in~\Cref{subsec:comparison_of_wasserstein_bounds}, we compare the different Wasserstein distance bounds to each other; in~\Cref{subsec:implications_bounds_using_mutual_information}, we show the implications that these bounds have on the bounds based on the relative entropy (and therefore the mutual information) and their relationship with them; and in~\Cref{subsec:bounds_using_f_divergences}, we describe how these bounds spur new bounds based on other $f$-divergences.

\subsection{Comparison of the Bounds}
\label{subsec:comparison_of_wasserstein_bounds}

As with the mutual information-based bounds, we note that the single-letter Wasserstein distance-based bounds (for example~\Cref{th:wasserstein_single_letter}) are tighter than the bounds considering the full training set (for example~\Cref{th:wasserstein_full}) and the random-subset bounds (for example~\Cref{th:wasserstein_random_subset}). However, now we see that the random-subset bounds are tighter than twice the bounds considering the full training set. This discrepancy with the comparison of the mutual information bounds from~\Cref{subsec:comparison_of_the_bounds} may come from the fact that now all bounds can be compared more fairly, in the sense that now the random-subset bounds do not need to be weakened with Jensen's inequality for the comparison.
Comparing these bounds was the purpose of our~\citep[Appendix D]{rodriguez2021tighter}. A schematic of the relationships between the different bounds is given in~\Cref{fig:comparison_wasserstein_bounds}, albeit with the simplification that the Markov kernels $\bQ$ are assumed to follow from smoothing the algorithm's distribution, for example $\bQ = \bP_W^S \circ \bP_S$ or $\bQ(S_{J^c}) = \bP_W^S \circ \bP_{S_J}$.

\begin{figure}
    \centering
    \begin{tikzpicture}[scale=0.8, transform shape]
		\node [align=center] at (3.5,4) {\Cref{th:wasserstein_single_letter}\\ $\frac{L}{n}\sum_{i=1}^n \bE \big[ \bW_\rho(\bP_W^{Z_i}, \bP_W) \big]$};	
		
		\node [align=center, rotate=-90] at (3.5,3) {$\leq$};
			
		\node [align=center] at (3.5,0) {\Cref{th:wasserstein_full}~\citep{wang2019information} \\ $L \bE \big[ \bW_\rho(\bP_W^S, \bP_W) \big]$};
	
		\node [align=center, rotate=-90] at (3.5,1) {$\leq$};
        \node [align=center] at (4.1,1) {\footnotesize{$(\times 2)$}};
	
		\node [align=center] at (3.5,2) {\Cref{th:wasserstein_random_subset}\\ $L \bE \big[ \bW_\rho(\bP_W^{\tilde{S}}, \bP_W^{S_{J^c}}) \big]$};	

        \draw [densely dotted] (5.5,3.75) -- (7,3.75);
        \draw [densely dotted] (7,3.75) -- (7,2);
        \node [align=center, rotate=-90] at (7,2-0.25) {$\leq$};
        \draw [densely dotted] (7,1.5) -- (7,-0.25);
        \draw [densely dotted] (5,-0.25) -- (7,-0.25);
		
		\node [align=center] at (-4,4) {\Cref{th:wasserstein_single_letter_rs} \\ $\frac{2L}{n}\sum_{i=1}^n \bE \big[ \bW_\rho(\bP_W^{\tilde{Z}_{i,0}, \tilde{Z}_{i,1}, U_i}, \bP_W^{\tilde{Z}_{i,0}, \tilde{Z}_{i,1}}) \big]$};		
		
		\node [align=center] at (-0.25,4-0.25) {$\leq$};
		\node [align=center] at (-0.25,2-0.25-0.5){\footnotesize{$(\times 4)$}};
		
		\node [align=center, rotate=-90] at (-4,3) {$\leq$};	
		
		\node [align=center] at (-4,0) {\Cref{th:wasserstein_full_rs}\\ $2L \bE \big[ \bW_\rho(\bP_W^{\tilde{S}, U}, \bP_W^{\tilde{S}}) \big]$};
		
		\node [align=center] at (-0.25,2-0.25) {$\leq$};
		\node [align=center] at (-0.25,4-0.25-0.5) {\footnotesize{$(\times 4)$}};
	
		\node [align=center, rotate=-90] at (-4,1) {$\leq$};
        \node [align=center] at (-3.4,1) {\footnotesize{$(\times 2)$}};
	
		\node [align=center] at (-4,2)          {\Cref{th:wasserstein_random_subset_rs} \\ 
        $2L \bE \big[ \bW_\rho(\bP_W^{\tilde{S}, U}, \bP_W^{\tilde{S},U_{J^c}})\big]$};	
		
		\node [align=center] at (-0.25,0-0.25) {$\leq$};
		\node [align=center] at (-0.25,0-0.25-0.5) {\footnotesize{$(\times 4)$}};

        \draw [densely dotted] (-7,3.75) -- (-7.5,3.75);
        \draw [densely dotted] (-7.5,3.75) -- (-7.5,2);
        \node [align=center, rotate=-90] at (-7.5,2-0.25) {$\leq$};
        \draw [densely dotted] (-7.5,1.5) -- (-7.5,-0.25);
        \draw [densely dotted] (-5.75, -0.25) -- (-7.5,-0.25);
        
	\end{tikzpicture}
    \caption{Summary of the comparison between the full, the single-letter, and the random-subset bounds based on the Wasserstein distance of both the standard and the randomized-subsample setting for $L$-Lipschitz losses.}
    \label{fig:comparison_wasserstein_bounds}
\end{figure}
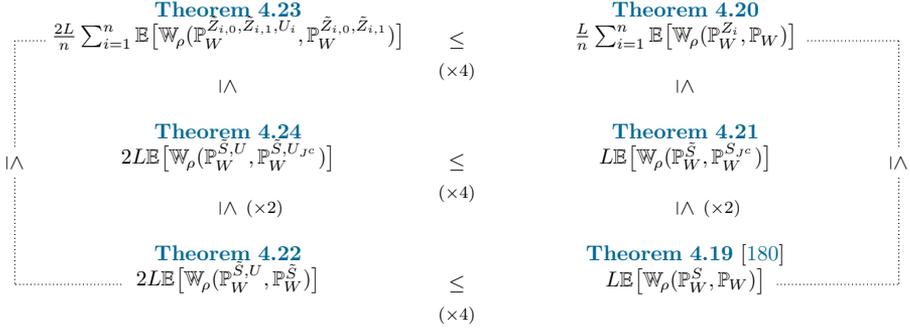

\subsubsection{Comparison in the Standard Setting}

In the standard setting, the relationship between the different bounds can be summarized with the following proposition.

\begin{proposition}
\label{prop:relationship_wasserstein}
Consider a random subset of indices $J \subseteq [n]$ such that $|J| = m$. Then, for every data independent distribution $\bQ$ on $\cW$, and every Markov kernel $\bQ(\cdot)$ from $\cZ^{n-m}$ to distributions on $\cW$ such that $\bQ = \bQ(S_{J^c}) \circ \bP_{S_{J^c}}$
\begin{align*}
    \frac{1}{n} \sum_{i=1}^n \bE \big[ \bW_\rho(\bP_W^{Z_i}, \bQ) \big] &\leq \bE \big[ \bW_\rho(\bP_W^{S}, \bQ) \big], \\
    \frac{1}{n} \sum_{i=1}^n \bE \big[ \bW_\rho(\bP_W^{Z_i}, \bQ) \big] &\leq \bE \big[ \bW_\rho(\bP_W^{S}, \bQ(S_{J^c}) \big], \textnormal{ and} \\
    \bE \big[ \bW_\rho(\bP_W^{S}, \bQ(S_{J^c}) \big] &\leq 2 \bE \big[ \bW_\rho(\bP_W^{S}, \bQ) \big] + \bE \big[ \bW_\rho(\bP_W, \bQ) \big].
\end{align*}
\end{proposition}

The first inequality in~\Cref{prop:relationship_wasserstein} is easy to prove by noting that $\bP_W^{Z_i}$ is a smooth version of $\bP_W^S$ for all $i \in [n]$, namely $\bP_{W}^{Z_i} = \bP_W^S \circ \bP_W^{S^{-i}}$. Let $W'$ be the random variable independent of the training data $S$  distributed according to $\bQ$. Then $\bE^S[f(W')] = \bE[f(W')]$. Therefore, employing the Kantorovich--Rubinstein duality from~\Cref{lemma:kantorovich_rubinstein_duality} we see that for every $i \in [n]$
\begin{align*}
    \bE \mleft[ \bW_\rho(\bP_W^S, Q) \mright] &= \bE \mleft[ \sup_{f \in 1\textnormal{-}\mathrm{Lip}(\rho)} \mleft \{ \bE^S[f (W)] - \bE[f(W')] \mright\} \mright] \\
    &\geq  \bE \mleft[ \sup_{f \in 1\textnormal{-}\mathrm{Lip}(\rho)} \mleft \{ \bE^{Z_i}[f (W)] - \bE[f(W')] \mright\} \mright] = \bE \mleft[ \bW_\rho(\bP_W^{Z_i}, Q) \mright],
\end{align*}
where the inequality follows from the fact that $\bE[\sup_f f(X)] \geq \sup_f \bE[f(X)]$.

The second inequality in~\Cref{prop:relationship_wasserstein} follows similarly. First, note that the statement can be re-written as
\begin{equation*}
    \frac{1}{n} \sum_{i=1}^n \bE \big[ \bW_\rho(\bP_W^{Z_i}, \bQ) \big] \leq \frac{1}{\binom{n}{m}} \sum_{j \in \cJ} \bE \big[ \bW_\rho(\bP_W^{S}, \bQ(S_{j^c})) \big]
\end{equation*}
writing the expectation with respect to $J$ explicitly. Then, we can prove this statement by proving the stronger statement that for all $i \in [n]$ and all $j \subseteq [n]$ such that $i \in j$
\begin{equation*}
    \bE \big[ \bW_\rho(\bP_W^{Z_i}, \bQ) \big] \leq \bE \big[ \bW_\rho(\bP_W^{S}, \bQ(S_{j^c})) \big].
\end{equation*}
This statement is stronger since, without loss of generality, we could consider the instances $Z_i$ ordered so that the sequence $\{ \bW_\rho(\bP_W^{Z_i}, \bQ) \}_{i=1}^n$ is non-increasing. Then, this statement would make sure that $\bE \big[ \bW_\rho(\bP_W^{Z_1}, \bQ) \big]$ is smaller than $\big[ \bW_\rho(\bP_W^{S}, \bQ(S_{j^c})) \big]$ for the $\binom{n-1}{m-1}$ subsets $j$ in which sample $Z_1$ appears; that $\bE \big[ \bW_\rho(\bP_W^{Z_2}, \bQ) \big]$ is smaller than $\big[ \bW_\rho(\bP_W^{S}, \bQ(S_{j^c})) \big]$ for the subsets $j$ in which sample $Z_2$ appears and $Z_1$ does not; and so on therefore proving the original statement. 

Now, we may proceed as before using that $\bP_W^{Z_i}$ is a smoothed version of $\bP_W^S$ and that $\bQ$ is a smoothed version of $\bQ(S_{j^c})$, namely $\bQ = \bQ(S_{j^c}) \circ \bP_{S_{j^c}}$. Employing the Kantorovich--Rubinstein duality from~\Cref{lemma:kantorovich_rubinstein_duality}
\begin{align*}
    \bE \big[ \bW_\rho(\bP_W^{S}, \bQ(S_{j^c})) \big] &= \bE  \mleft[ \sup_{f \in 1\textnormal{-}\mathrm{Lip}(\rho)} \mleft \{ \bE^S[f(W)] - \bE^{S_{j^c}}[f(W')]\mright \} \mright] \\
    &\geq \bE  \mleft[ \sup_{f \in 1\textnormal{-}\mathrm{Lip}(\rho)} \mleft \{ \bE^{Z_i}[f(W)] - \bE[f(W')] \mright \} \mright] = \bE \big[ \bW_\rho(\bP_W^{Z_i}, \bQ) \big],
\end{align*}
where we let $W'$ be the random variable distributed according to $\bQ$ and with joint distribution with $S_{j^c}$ equal to $\bQ(S_{j^c}) \otimes \bP_{S^j}$, and where we also used that  $\bE[\sup_f f(X)] \geq \sup_f \bE[f(X)]$.

Finally, the third inequality from~\Cref{prop:relationship_wasserstein} follows from employing the triangle inequality of the Wasserstein distances~\citep[Chapter 6]{villani2009optimal}. Note that we cannot employ the smoothing argument as above since while $\bQ$ is a smoothed version of $\bQ(S_{J^c})$, including the expectation with respect to $\bP_{S_{J^c}}$ inside the supremum of $\bW_\rho(\bP_W^S, \bQ(S_{J^c})$ would also smooth the first term $\bP_W^S$, leading to the inequality $\bE \big[\bW_\rho(\bP_W^S, \bQ(S_{J^c})) \big] \geq \bE \big[ \bW_\rho(\bP_W^{S_J}, \bQ) \big]$, while at the same time, and for the same arguments, $\bE \big[ \bW_\rho(\bP_W^S, \bQ) \big] \geq \bE \big[\bW_\rho(\bP_W^{S_J}, \bQ)  \big]$. Employing the triangle inequality we have that
\begin{equation*}
    \bW_\rho(\bP_W^S, \bQ(S_{J^c})  \leq  \bW_\rho(\bP_W^S, \bQ)  +  \bW_\rho(\bQ, \bP_W)  +  \bW_\rho(\bP_W, \bQ(S_{J^c}) ).
\end{equation*}
From the same arguments outlined above, $\bP_W$ is a smoothed version of $\bP_W^S$ and $\bQ$ is a smoothed version of $\bQ(S_{J^c})$ it follows that 
\begin{equation*}
    \bE \big[ \bW_\rho(\bP_W^S, \bQ(S_{J^c}) \big] \leq 2  \bE \big[ \bW_\rho(\bP_W^S, \bQ) \big] + \bE \big[ \bW_\rho(\bP_W, \bQ)  \big].
\end{equation*}

An interesting thing to notice is that if we choose a Markov kernel $\bQ(\cdot)$ such that $\bQ = \bP_W$, then the term $\bW_\rho(\bQ, \bP_W)$ disappears. For example, choosing $\bQ(S_{J^c}) = \bP_W^{S_{J^c}}$. Moreover, in case that the distance between $\bP_W$ and $\bQ$ cannot be quantified, since $\bP_W$ is a smoothed version of $\bP_W^S$, we can re-use the previous arguments to note that
\begin{equation*}
    \bE \big[ \bW_\rho(\bP_W^S, \bQ(S_{J^c}) \big] \leq 3  \bE \big[ \bW_\rho(\bP_W^S, \bQ) \big].
\end{equation*}

\subsubsection{Comparison in the Randomized-Subsample Setting}

In the randomized-subsample setting, the relationship between the different bounds can be summarized with a proposition analogous to~\Cref{prop:relationship_wasserstein}.

\begin{proposition}
\label{prop:relationship_wasserstein_rs}
Consider a random subset of indices $J \subseteq [n]$ such that $|J| = m$. Further consider a Markov kernel $\bQ$ from $\cZ^{n \times 2} \otimes \{ 0, 1 \}^{n-m}$ to distributions on $\cW$ and let us abuse notation to also denote by $\bQ$ the Markov kernels from $\cZ^{n \times 2}$ and from $\cZ^2$ to distributions on $\cW$ defined as $\bQ(\tilde{s}) = \bQ(\tilde{s},U_{J^c}) \circ \bP_{U_{J^c}}$ and $\bQ(\tilde{z}_{i,0},\tilde{z}_{i,1}) = \bQ(\tilde{S}) \circ \bP_{\tilde{S} \setminus \{ \tilde{z}_{i,0}, \tilde{z}_{i,1} \}}$. Then, for every such Markov kernel 
\begin{align*}
    \frac{1}{n} \sum_{i=1}^n \bE \big[ \bW_\rho(\bP_W^{\tilde{Z}_{i,0}, \tilde{Z}_{i,1}, U_i}, \bQ(\tilde{Z}_{i,0}, \tilde{Z}_{i,1})) \big] &\leq \bE \big[ \bW_\rho(\bP_W^{\tilde{S}, U}, \bQ(\tilde{S})) \big], \\
    \frac{1}{n} \sum_{i=1}^n \bE \big[ \bW_\rho(\bP_W^{\tilde{Z}_{i,0}, \tilde{Z}_{i,1}, U_i}, \bQ(\tilde{Z}_{i,0}, \tilde{Z}_{i,1})) \big] &\leq \bE \big[ \bW_\rho(\bP_W^{\tilde{S}, U}, \bQ(\tilde{S}, U_{J^c}) \big], \textnormal{ and} \\
    \bE \big[ \bW_\rho(\bP_W^{\tilde{S}, U}, \bQ(\tilde{S}, U_{J^c}) \big] &\leq 2 \bE \big[ \bW_\rho(\bP_W^{\tilde{S}, U}, \bQ(\tilde{S})) \big] + \bE \big[ \bW_\rho(\bP_W^{\tilde{S}}, \bQ(\tilde{S})) \big].
\end{align*}
\end{proposition}

Similarly to the standard setting, the first two inequalities of~\Cref{prop:relationship_wasserstein_rs} follow by noting that $\bP_W^{\tilde{Z}_{i,0}, \tilde{Z}_{i,1}, U_i}$ is a smoothed version of $\bP_W^{\tilde{S}, U}$, that is, $\bP_W^{\tilde{Z}_{i,0}, \tilde{Z}_{i,1}, U_i} = \bP_{W}^{\tilde{S}, U} \circ \bP_{\tilde{S} \setminus \{\tilde{Z}_{i,0}, \tilde{Z}_{i,1} \}, U^{-i}}$; and that the Markov kernel $\bQ(\tilde{Z}_{i,0}, \tilde{Z}_{i,1})$ is a smooth version of $\bQ(\tilde{S})$ and therefore also of $\bQ(\tilde{S}, U_{J^c})$. The proof follows by employing the Kantorovich--Rubinstein duality from~\Cref{lemma:kantorovich_rubinstein_duality} and the fact that $\sup_f \bE[X] \leq \bE[ \sup_f f(X)]$ and it is almost verbose to those shown previously.

The third inequality from~\Cref{prop:relationship_wasserstein_rs} follows from employing the triangle inequality of the Wasserstein distances~\citep[Chapter 6]{villani2009optimal} similarly to what we did to prove~\Cref{prop:relationship_wasserstein}. In this case, we have that
\begin{equation*}
    \bW_\rho(\bP_W^{\tilde{S}, U}, \bQ(\tilde{S}, U_{J^c}))  \leq \bW_\rho(\bP_W^{\tilde{S}, U}, \bQ(\tilde{S})) + \bW_\rho(\bQ(\tilde{S}), \bP_W^{\tilde{S}} ) + \bW_\rho( \bP_W^{\tilde{S}}, \bQ(\tilde{S}, U_{J^c})). 
\end{equation*}
Using the smoothing arguments from this section, noting that $\bP_W^{\tilde{S}}$ and $\bQ(\tilde{S}, U_{J^c})$ are both smoothed versions of $\bP_{W}^{\tilde{S}, U}$ and $\bQ(\tilde{S}, U_{J^c})$ with respect to $\bP_U$ yields that
\begin{equation*}
    \bE \big[ \bW_\rho(\bP_W^{\tilde{S}, U}, \bQ(\tilde{S}, U_{J^c})) \big]  \leq 2 \bE \big[ \bW_\rho(\bP_W^{\tilde{S}, U}, \bQ(\tilde{S})) \big] + \bE \big[ \bW_\rho(\bP_W^{\tilde{S}},\bQ(\tilde{S}) ) \big].
\end{equation*}
Finally, we may note that choosing the Markov kernel $\bQ$ from $\cZ^n$ to distributions on $\cW$ such that $\bQ(\tilde{S}) = \bP_W^{\tilde{S}}$ makes the term $\bW_\rho(\bP_W^{\tilde{S}},\bQ(\tilde{S}) )$ disappear. Also, another smoothing argument with respect to $\bP_U$ reveals that
\begin{equation*}
    \bE \big[ \bW_\rho(\bP_W^{\tilde{S}, U}, \bQ(\tilde{S}, U_{J^c})) \big]  \leq 3 \bE \big[ \bW_\rho(\bP_W^{\tilde{S}, U}, \bQ(\tilde{S})) \big].
\end{equation*}

\subsubsection{Comparison Between the Settings}

Comparing the Wasserstein distance-based bounds from the standard and the randomized-subsample setting is simple if we note that $\bP_{W}^{\tilde{S}, U} = \bP_{W}^S$ a.s., that is, the distribution of the hypothesis when the supersample and the indices of the instances used for training are given is almost surely equal to the distribution of the hypothesis when the training samples are given. With this observation, the comparisons follow a combination of the triangle inequality and the smoothing technique shown in the previous two subsections. Their relationship can be summarized with the following proposition.

\begin{proposition}
\label{prop:comparison_between_setting_wasserstein}
    Consider a random subset of indices $J \subseteq [n]$ such that $J = m$. Further consider a Markov kernel $\bQ$ from $\cZ^{n \times 2} \times \{0, 1 \}^{n-m}$ to distributions on $\cW$. Furthermore, let us also denote by $\bQ$ any smoothing, or marginalization, of this Markov kernel. Then, for every such Markov kernel,
    \begin{align*}
        \bE \big[ \bW_\rho( \bP_W^{\tilde{S},U} \bQ(\tilde{S}))] &\leq 2 \bE \big[ \bW_{\rho} ( \bP_W^S, \bQ) \big] + \bE \big[ \bW_\rho( \bP_W^{\tilde{S}}, \bQ(\tilde{S}))], \\
        \bE \big[ \bW_\rho(\bP_W^{\tilde{Z}_{i,0}, \tilde{Z}_{i,1}, U_i}, \bQ(\tilde{Z}_{i,0}, \tilde{Z}_{i,1})) \big] &\leq 2 \bE \big[ \bW_\rho(\bP_W^{Z_i}, \bQ) \big] + \bE \big[ \bW_\rho(\bP_W^{\tilde{Z}_{i,0}, \tilde{Z}_{i,1}}, \bQ(\tilde{Z}_{i,0}, \tilde{Z}_{i,1})) \big], \textnormal{ and} \\
        \bE \big[ \bW_\rho(\bP_W^{\tilde{S}, U}, \bQ(\tilde{S}, U_{J^c}) \big] &\leq 2\bE \big[ \bW_\rho(\bP_W^{S}, \bQ(S_{J^c}) \big]  + \bE \big[ \bW_\rho(\bP_W^{\tilde{S}, U_{J^c}}, \bQ(\tilde{S}, U_{J^c}) \big].
    \end{align*}
\end{proposition}

Before showing these inequalities, note that they are all analogous to each other and that the second term on the right-hand side always disappears if the Markov kernel $\bQ$ is chosen to be the Markov kernel associated with smoothing the algorithm's Markov kernel $\bP_W^{\tilde{S}, U}$.

Since the proof of the three statements is analogous and almost verbose using the techniques outlined previously, we only write explicitly the proof of the first inequality. Employing the triangle inequality of the Wasserstein distances~\citep[Chapter 6]{villani2009optimal} we have that
\begin{equation}
\label{eq:triangle_ineq_wasserstein_rs}
    \bW_\rho(\bP_W^{\tilde{S},U}, \bQ(\tilde{S})) \leq \bW_\rho(\bP_W^{\tilde{S},U}, \bQ) + \bW_\rho(\bQ, \bP_W^{\tilde{S}}) + \bW_\rho(\bP_W^{\tilde{S}}, \bQ(\tilde{S})).
\end{equation}
Then, noting that $\bP_W^{\tilde{S}}$ is a smoothed version of $\bP_W^{\tilde{S},U}$ with respect to $\bP_U$ and that $\bP_W^{\tilde{S},U} = \bP_W^S$ a.s. results in the desired inequality
\begin{equation*}
    \bE \big[ \bW_\rho(\bP_W^{\tilde{S},U}, \bQ(\tilde{S})) \big] \leq \bE \big[ \bW_\rho(\bP_W^{S}, \bQ) \big] + \bE \big[ \bW_\rho(\bQ, \bP_W^{\tilde{S}}) \big].
\end{equation*}

Rather than just stating the comparison of the different bounds, we may inspect them a little further. To get a better understanding, let $\bQ(\tilde{S}) = \bP_W^{\tilde{S}}$. Instead of applying the triangle inequality twice in~\eqref{eq:triangle_ineq_wasserstein_rs}, we may just apply it once to see that, after taking expectations to both sides,
\begin{equation*}
    \bE \big[ \bW_\rho(\bP_W^{\tilde{S},U}, \bP_W^{\tilde{S}}) \big] \leq \bE \big[ \bW_\rho(\bP_W^S, \bP_W) \big] + \bE \big[ \bW_\rho(\bP_W^{\tilde{S}}, \bP_W)\big].
\end{equation*}

The term on the left-hand side tells us what the difference is between the hypothesis distribution when the training set is known and when the supersample is known but not which indices are used for training. This describes how much the algorithm will change the distribution of its output when changing the samples used for training from a given set of identically distributed instances. The first term on the right-hand side is already known to us, it tells us the difference between the hypothesis distribution after observing the training data and the prototypical hypothesis distribution from samples of the data distribution $\bP_W = \bP_W^S \circ \bP_S$. We might expect that, similarly to what happened for the mutual information-based bounds (\Cref{subsec:comparison_of_the_bounds}), the inequality $\bE \big[ \bW_\rho(\bP_W^{\tilde{S},U}, \bP_W^{\tilde{S}}) \big] \leq \bE \big[ \bW_\rho(\bP_W^S, \bP_W) \big]$ holds, since the knowledge of the super sample $\tilde{S}$ restricts the smoothing of $\bP_W^{\tilde{S},U}$ to all $2^n$ possible training set arrangements described by $U$, whereas the smoothing of $\bP_W^{S}$ considers potentially infinite (if $\cZ$ is not finite) training sets. It is possible that such an inequality holds, although we have not been able to prove it. Nonetheless, if we focus on the extra term $\bW_\rho(\bP_W^{\tilde{S}}, \bP_W)$ that appears in the inequality, we see that it captures precisely how different the smoothed distributions $\bP_W^{\tilde{S}}$ and $\bP_W$ are. A final smoothing argument tells us that
\begin{equation*}
     \bE \big[ \bW_\rho(\bP_W^{\tilde{S},U}, \bP_W^{\tilde{S}}) \big] \leq 2\bE \big[ \bW_\rho(\bP_W^S, \bP_W) \big]
\end{equation*}
holds, capturing the essence of the expected relationship between the two terms.

\subsection{Implications for Bounds Using Mutual Information}
\label{subsec:implications_bounds_using_mutual_information}

At the beginning of this section, we noted how the relative entropy or mutual information-based bounds were the result of the decoupling~\Cref{lemma:decoupling_relent} that stems from the Donsker and Varadhan~\Cref{lemma:dv_and_gvp}, and how employing instead the decoupling that stems from the Kantorovich--Rubinstein duality of~\Cref{lemma:kantorovich_rubinstein_duality} led to similar results based on the Wasserstein distance. A benefit of the latter bounds is that they take into account the geometry of the space in an explicit way, which is not often the case for the relative entropy-based bounds. In this subsection, we will first show that if we disregard the geometry of the space, we can transform the bounds based on the Wasserstein distance into bounds on the total variation. Later, we will see that this allows us to extend the set of relative entropy and mutual information-based bounds to Lipschitz losses, which, as we will show, are tight. Moreover, this implies that, for losses with a bounded range, the Wasserstein distance-based bounds are always tighter than those based on the relative entropy. Finally, we will conclude by showing how the total variation bounds also help us derive new bounds based on other $f$-divergences.

\subsubsection{Disregarding the Geometry: Bounds Using Total Variation}
\label{subsubsec:bounds_using_total_variation}

Recall from~\Cref{prop:wasserstein_and_total_variation} that the Wasserstein distance is dominated by the total variation. Applying this consideration to the first result in this section results in the following corollary.

\begin{corollary}
\label{cor:total_variation_full}
    Consider an $L$-Lipschitz loss function $\ell(\cdot, z)$ with respect to some metric $\rho$ for all $z \in \cZ$. Further consider a bounded space $\cW$ with $\diam_\rho(\cW) = B$. Then, for every data-independent distribution $\bQ$ on $\cW$
    \begin{equation*}
        \bE \big[ \gen(W,S) \big] \leq L B \bE \mleft[ \min \mleft \{ \tv(\bP_W^S, \bP_W), 2 \tv(\bP_W^S, \bQ) \mright \} \mright].
    \end{equation*}
\end{corollary}

By the interpretation of the total variation as the limit accuracy in binary hypothesis testing from~\Cref{subsec:f-divergences}, we may strengthen our understanding that the less the final hypothesis depends on the data, the more it generalizes. More precisely, given a dataset $S$ and upon observing a hypothesis $W$, we may wonder if this hypothesis comes from the posterior distribution $\bP_W^S$ characterizing an algorithm $\bA$ or a data-independent prior distribution $\bQ$. Let $T$ denote a binary random variable describing if $W$ is sampled from the posterior distribution $\bP_W^S$ (when $T = 1$) or from the prior distribution $\bQ$ (when $T = 0$). Further consider the set $\cF$ of decision rules $f: \cW \to \{0, 1 \}$ that try to estimate $T$ and tell if the hypothesis $W$ comes from the prior or the posterior. Then, we have that
\begin{equation}
	\bE[\textnormal{gen}(W,S)] \leq 2LB \mleft( 1 - 2 \bE \Big[ \inf_{f \in \cF} \bP^S \big[ f(W) \neq T \big] \Big] \mright).
\end{equation}
In other words, the harder it is to distinguish hypotheses sampled from the algorithm after observing the training data $S$ and from the data-independent prior $\bQ$, the better the algorithm generalizes.

Note as well that by considering the total variation, we have essentially disregarded the geometrical information provided to us by the metric $\rho$. 
Indeed, if one considers the discrete metric $\rho_\mathrm{H}(x,y) = \bI_{\{x \neq y\}}(x,y)$, then $\bW_{\rho_\mathrm{H}}(\bP, \bQ) = \tv(\bP, \bQ)$~(\Cref{prop:wasserstein_and_total_variation}). Then, for example, if the distribution $\bP_W^S$ is not absolutely continuous with respect to $\bQ$ like in~\Cref{fig:example_wasserstein}, \Cref{cor:mi_bound_bounded_fast_rate_weaker} only guarantees that the generalization error is bounded by $2LB$ even if the supports of the distributions were close and $\bW_{\lVert \cdot \rVert_2}(\bP_W^S, \bQ) = \varepsilon$ for some positive $\varepsilon \to 0$.

The diameter term $B$ can be arbitrarily large as, for example, when the hypothesis are the weights of a differentiable network and the metric $\rho$ is the $\ell_2$ norm $\lVert \cdot \rVert_2$.
Nonetheless, this term can still be small and relevant for practical settings.
For instance, consider again that the hypothesis are the weights of a differentiable network. However, consider now that the metric $\rho$ is the infinity norm $\lVert \cdot \rVert_{\infty}$ and that each weight is enforced to be smaller than some small constant $c$. Then, the diameter of the space $B = \diam_{\lVert \cdot \rVert_\infty}(\cW)$ is (at most) equal to $c$.

A particularly interesting example is the one we obtained for bounded losses. In this setting, the loss is always $(b-a)$-Lipschitz under the discrete metric and therefore~\Cref{cor:total_variation_full} reduces to
\begin{equation}
    \label{eq:total_variation_full}
    \bE \big[ \gen(W,S) \big] \leq (b-a) \bE \mleft[ \min \mleft\{ \tv(\bP_W^S, \bP_W), 2 \tv(\bP_W^S, \bQ) \mright\} \mright].
\end{equation}
Therefore, the generalization bound scales with the total variation and, if the prior $\bQ$ is chosen to be $\bP_W$, it is always confined within the range $[0, b-a]$, which ensures it is never vacuous.

It is tempting to think that considering an $L$-Lipschitz loss and a hypothesis space with a diameter bounded by $B$ is equivalent to considering a loss with a bounded range. Indeed, if a function $f : \cW \to \bR$ is both $L$-Lipschitz and $\diam_\rho(\cW) = B$, we have that $|f(w) - f(w')| \leq LB$ for all $w,w' \in \cW$ and therefore there is some constant $c \in \bR$ such that $f(w) \in [c, c+LR]$. However, this is not necessarily true for loss functions $\ell : \cW \times \cZ \to \bR$ since they take two arguments: the hypothesis $w$ and the instances $z$. As a simple example, for some dimension $d$, consider the  $1$-Lipschitz loss function $\ell(w,z) = \langle w, z \rangle$, the unit ball as the hypothesis space $\cW = \{ w \in \bR^d : \lVert w \rVert_2 = 1\}$, and the reals as the instance space $\cZ = \bR^d$. Then, the diameter of the hypothesis space is bounded $(\diam_\rho(\cW) = 2)$ but the loss can be infinite, that is $\sup_{z \in \bR^d} \langle w, z \rangle \to \infty$ for all $w \in \cW \setminus \{ 0 \}$.

This reasoning extends analogously for all the presented Wasserstein distance-based bounds from~\Cref{th:wasserstein_single_letter,th:wasserstein_random_subset,th:wasserstein_full_rs,th:wasserstein_single_letter_rs,th:wasserstein_random_subset_rs}. In what follows, we will refer to these corollaries without explicitly writing them to avoid verbosity.

\subsubsection{From Total Variation to Relative Entropy and Mutual Information}
\label{subsubsec:from_total_variation_to_relative_entropy}

Consider the analogue of~\Cref{cor:total_variation_full} for~\Cref{th:wasserstein_single_letter}. Moreover, for simplicity, let $\bQ = \bP_W$. Applying Pinsker's and the Bretagnolle-Huber inequalities from~\Cref{lemma:pinsker-inequality,lemma:bretagnolle-huber-inequality} leads to the following result.

\begin{corollary}
    \label{cor:total_variation_relent_single_letter}
    Consider an $L$-Lipschitz loss function $\ell(\cdot, z)$ with respect to some metric $\rho$ for all $z \in \cZ$. Further consider a bounded space $\cW$ with $\diam_\rho(\cW) = B$. Then,
    \begin{equation*}
        \bE \big[ \gen(W,S) \big] \leq \frac{L B}{n} \sum_{i=1}^n \bE \mleft[ \Psi\mleft( \relent(\bP_W^{Z_i} \Vert \bP_W)\mright) \mright] \leq \frac{L B}{n \sqrt{2}} \sum_{i=1}^n \sqrt{ \minf(W;Z_i) },
    \end{equation*}
    where $\Psi(x)^2 = \min \{ \nicefrac{x}{2}, 1 - \exp(-x) \}$.
\end{corollary}

If the loss is bounded, this corollary improves upon~\Cref{th:single_letter_mi_bound_general_cgf} in two different ways. First, it pulls the expectation with respect to $\bP_{Z_i}$ outside of the concave square root, thus strengthening the result via Jensen's inequality. Second, the addition of the Bretagnolle--Hubert inequality ensures that heavily influential samples with high $\minf(W;Z_i)$ do not contribute too negatively to the bound, which is ensured to be non-vacuous. The main difference between the two results are their assumptions: \Cref{th:single_letter_mi_bound_general_cgf} needs losses with a bounded CGF and \Cref{cor:total_variation_relent_single_letter} needs them to be Lipschitz with respect to some metric on a space with a bounded diameter. When the loss is bounded, the two coincide with the aforementioned improvements of~\Cref{cor:total_variation_relent_single_letter} with respect to \Cref{th:single_letter_mi_bound_general_cgf}. Moreover, since $\Psi$ is a concave function, a further application of Jensen's inequality and~\Cref{prop:comparison_minf} yields the following result. 
\begin{corollary}
    \label{cor:total_variation_relent_full}
    Consider an $L$-Lipschitz loss function $\ell(\cdot, z)$ with respect to some metric $\rho$ for all $z \in \cZ$. Further consider a bounded space $\cW$ with $\diam_\rho(\cW) = B$. Then,
    \begin{equation*}
    \bE \big[ \gen(W,S) \big] \leq LB \Psi \mleft( \frac{\minf(W;S)}{n} \mright) \leq LB \sqrt{\frac{\minf(W;S)}{2n}},
    \end{equation*}
    where $\Psi(x)^2 = \min \{ \nicefrac{x}{2}, 1 - \exp(-x) \}$.
\end{corollary}
These results prove that that~\Cref{th:wasserstein_single_letter} improves upon~\Cref{th:mi_bound_bounded,th:single_letter_mi_bound_general_cgf} when the loss has a bounded range.

The same logic can be applied to the analogue of~\Cref{cor:total_variation_full} for~\Cref{th:wasserstein_single_letter_rs}, which results in the following Corollary. However, in these cases, the Bretagnolle--Huber inequality from~\Cref{lemma:bretagnolle-huber-inequality} is not considered. The reason is that the relative entropy terms in these equations are bounded by values in the region where the Pinsker's inequality from~\Cref{lemma:pinsker-inequality} is tighter than the Bretagnolle--Huber inequality. The details of this consideration are not relevant to the main text and are delegated to \Cref{app:rs_and_bh}.

\begin{corollary}
    \label{cor:total_variation_relent_single_letter_rs}
    Consider an $L$-Lipschitz loss function $\ell(\cdot, z)$ with respect to some metric $\rho$ for all $z \in \cZ$. Further consider a bounded space $\cW$ with $\diam_\rho(\cW) = B$. Then,
    \begin{align*}
        \bE \big[ \gen(W,S) \big] &\leq \frac{ L B}{n} \sum_{i=1}^n \bE \mleft[ \sqrt{2 \relent(\bP_W^{\tilde{Z}_{i,0}, \tilde{Z}_{i,1}, U_i} \Vert \bP_W^{\tilde{Z}_{i,0}, \tilde{Z}_{i,1}})} \mright] \\
        &\leq \frac{ L B}{n} \sum_{i=1}^n \sqrt{ 2 \minf(W;U_i | \tilde{Z}_{i,0}, \tilde{Z}_{i,1})}.
    \end{align*}
\end{corollary}

In this case, this corollary improves upon~\Cref{th:single_letter_cmi_bound_bounded} in the ways described above and in its generality, as now it holds for all Lipschitz losses and not only those with a bounded range. Again, since $\sqrt{\cdot}$ is concave, a further application of Jensen's inequality and~\Cref{prop:comparison_minf} yields the following result, which indicates that \Cref{th:wasserstein_single_letter_rs} improves upon~\Cref{th:cmi_bound_bounded,th:single_letter_cmi_bound_bounded}.

\begin{corollary}
    \label{cor:total_variation_relent_full_rs}
    Consider an $L$-Lipschitz loss function $\ell(\cdot, z)$ with respect to some metric $\rho$ for all $z \in \cZ$. Further consider a bounded space $\cW$ with $\diam_\rho(\cW) = B$. Then,
    \begin{equation*}
    \bE \big[ \gen(W,S) \big] \leq LB \sqrt{ \frac{2 \minf(W;U|\tilde{S})}{n} } \leq LB \sqrt{\frac{2 \minf(W;U|\tilde{S})}{n}}.
    \end{equation*}
\end{corollary}

To conclude this part of the subsection, we note that we can apply the same reasoning to the analogues of \Cref{cor:total_variation_full} for the random subset~\Cref{th:wasserstein_random_subset,th:wasserstein_random_subset_rs}. However, let us do so considering an arbitrary $\bQ$. Then, after the application of Pinsker's and the Bretagnolle--Huber inequalities from~\Cref{lemma:pinsker-inequality,lemma:bretagnolle-huber-inequality} we essentially recover~\Cref{th:tighter_random_subset_mi_bounded_loss,th:tighter_random_subset_cmi_bounded_loss} as promised in~\Cref{subsec:random_subset_bounds}.For~\Cref{cor:total_variation_relent_random_subset}, we gain a factor of two compared to~\Cref{th:tighter_random_subset_mi_bounded_loss}, while \Cref{cor:total_variation_relent_random_subset_rs} maintains the constants. However, we benefit by extending the bound to general Lipschitz losses. We present the results for $m=1$ to make the relationship with the bounds in the previous section evident, although the results can be obtained for a general $m \in [n]$.

\begin{corollary}
\label{cor:total_variation_relent_random_subset}
	Consider an $L$-Lipschitz loss function $\ell(\cdot, z)$ with respect to some metric $\rho$ for all $z \in \cZ$. Further consider a bounded space $\cW$ with $\diam_\rho(\cW) = B$. Also consider a random index $J \in [n]$, which is uniformly distributed and independent of $W$ and $S$, and a random variable $R$ that is independent of $S$. Then, for every Markov kernel $\bQ$ from $\cZ^{n-1} \otimes \cR \otimes J$ to distributions on $\cW$
	\begin{equation*}
		\bE \big[ \gen(W,S) \big] \leq 2LB \bE \mleft[ \Psi\mleft( \relent \big(\bP_W^{S, R} \Vert \bQ(S^{-J}, R, J) \big) \mright)\mright],
	\end{equation*}
    where $\Psi(x)^2 = \min \{ \nicefrac{x}{2}, 1 - \exp(-x) \}$.
\end{corollary}

\begin{corollary}
\label{cor:total_variation_relent_random_subset_rs}
	Consider an $L$-Lipschitz loss function $\ell(\cdot, z)$ with respect to some metric $\rho$ for all $z \in \cZ$. Further consider a bounded space $\cW$ with $\diam_\rho(\cW) = B$. Also consider a random index $J \in [n]$, which is uniformly distributed and independent of $W$, $\tilde{S}$, and $U$, and a random variable $R$ that is independent of $\tilde{S}$, and $U$. Then, for every Markov kernel $\bQ$ from $\cZ^{n \times 2} \otimes \{ 0, 1 \}^{n-1} \otimes \cR \otimes \cJ$ to distributions on $\cW$
	\begin{equation*}
		\bE \big[ \gen(W,S) \big] \leq 2LB \bE \mleft[ \sqrt{ \relent \big(\bP_W^{\tilde{S}, U, R} \Vert \bQ(\tilde{S}, U^{-J}, R) \big)} \mright].
	\end{equation*}
\end{corollary}

\subsubsection{Generality of the Mutual Information Bounds}
\label{subsubsec:generality_mutual_information_bounds}

For losses with a bounded CGF such as sub-Gaussian, we know that in the trivial case where the output of a learning algorithm is independent of the training set, \citet{xu2017information}'s result from~\Cref{th:mi_bound_bounded} is tight. However, this result is not so interesting as, in that case, the algorithm did not learn anything from the data. Our following theorem~\citep{haghifam2023limitations} states that for Lipschitz losses, the bounds are tight even when the learning algorithm depends on the training set.

\begin{theorem}
    \label{th:minf_bounds_tight}
    There exists an $L$-Lipschitz loss function $\ell(\cdot, z)$ with respect to some metric $\rho$ for all $z \in \cZ$, a bounded space $\cW$ with $\diam_\rho(\cW) = B$, a data distribution $\bP_S$, and a learning algorithm $\bA$ such that: (i) the expected generalization error satisfies that $\bE \big[ \gen(W,S) \big] \geq \frac{LB}{2\sqrt{2n}}$ and (ii) the upper bounds from~\Cref{cor:total_variation_relent_full} is $\bE \big[ \gen(W,S) \big] \leq LB \sqrt{\frac{\log 2}{2n}}$.
\end{theorem}
The theorem immediately implies that the bounds from \Cref{th:wasserstein_single_letter,th:wasserstein_single_letter_rs} and~\Cref{cor:total_variation_relent_single_letter,cor:total_variation_relent_full_rs,cor:total_variation_relent_single_letter_rs} are also tight since the Wasserstein distance bounds led to the mutual information bounds, and due to the relationships between the mutual information bounds from~\Cref{subsec:comparison_of_the_bounds}.

\Cref{th:minf_bounds_tight} also shows that there exists a learning algorithm for which the bounds are tight. This implies that the bounds cannot be improved \emph{simultaneously} (or \emph{uniformly}) for every learning algorithm. Note, however, that there may exist a tighter bound for some \emph{specific} learning algorithms. The proof of this theorem is inspired by~\citep{orabona2019modern}. Since it does not give any additional insight into the discussion, we delegate it to~\Cref{app:minf_bounds_tight}.

Finally, we note that the results from this~\Cref{subsec:implications_bounds_using_mutual_information} can be extended to losses that are Lipschitz and have a bounded CGF, without the need for the hypothesis space $\cW$ to be bounded. Therefore, the results from~\Cref{cor:total_variation_relent_full,cor:total_variation_relent_single_letter,cor:total_variation_relent_full_rs,cor:total_variation_relent_single_letter_rs,cor:total_variation_relent_random_subset,cor:total_variation_relent_random_subset_rs} are valid in a more general setting. The result follows from our extension of the Bobkov--Götze theorem~\citep[Theorem 4.8]{van2014probability} from~\Cref{lemma:bobkov_gotze}. To realize this, consider a random variable $X$ distributed according to $\bP$. Also consider an $L$-Lipschitz function $f$ that has a CGF $\Lambda_{f(X)}(\lambda)$ bounded by $\varphi(\lambda)$, then $\nicefrac{f(X)}{L}$ is $1$-Lipschitz and has a CGF bounded by $\varphi(\nicefrac{\lambda}{L})$. Hence, if $f(X) - \bE[X] \leq L \bW_\rho(\bQ, \bP)$ for some distribution $\bQ$, it follows that $f(X) - \bE[X] \leq L \varphi_*^{-1} \big( \relent(\bQ \Vert \bP) \big)$.

As an example, assume that the loss $\ell(\cdot, z)$ is $L$-Lipschitz and that $\ell(W', z)$ has a CGF bounded by $\varphi$ for every $z \in \cZ$, where $W'$ is distributed according to $\bQ$. Then, we may obtain the following result as a corollary of \Cref{th:wasserstein_single_letter} due to our extension of the Bobkov--Götze's theorem in~\Cref{lemma:bobkov_gotze}:

\begin{equation*}
   \bE \big[ \gen(W,S) \big] \leq  \frac{2L}{n} \sum_{i=1}^n \varphi_*^{-1} \big( I(W;Z_i) \big) \leq 2L \varphi_*^{-1} \mleft( \frac{I(W;S)}{n} \mright).
\end{equation*}
This result encompasses the case where the loss is bounded in $[a,b]$ since if a random variable is bounded in $[a,b]$ it is $\nicefrac{(b-a)}{2}$-sub-Gaussian. 

Note, however, that in this case the
``bounded CGF'' condition is different from the one in~\Cref{def:bounded_cgf} and~\Cref{sec:bounds_using_conditional_mutual_information,sec:random_subset_and_single_letter}. There, the loss $\ell(w,Z)$ is supposed to have a CGF bounded by $\psi$ for all $w \in \cW$, while now we are considering that the loss $\ell(W',z)$ has a CGF bounded by $\varphi$ for all $z \in \cZ$. Therefore, instead of considering that the loss is concentrated around its expectation with respect to the data, we can engineer a prior distribution $\bQ$ such that the loss is concentrated around its expectation with respect to hypotheses sampled from the said prior.

\subsection{Bounds Using Other $f$-divergences}
\label{subsec:bounds_using_f_divergences}

In~\citep{rodriguez2021tighter}, we note how using the joint range technique from Harremoës
and Vajda~\cite{harremoes2011pairs} discussed in~\Cref{subsec:f-divergences} allows us to find expected generalization bounds based on different $f$-divergences.

In particular, as an example, starting from~\Cref{cor:total_variation_full} we may obtain the corollaries
\begin{align*}
    \bE \big[ \gen(W,S) \big] &\leq LB \bE \mleft[ \hel(\bP_W^S, \bP_W) \sqrt{1 - \frac{\hel^2 (\bP_W^S , \bP_W)}{4} }\mright]  \textnormal{ and} \\
    \bE \big[ \gen(W,S) \big] &\leq \frac{LB}{2} \bE \mleft[ \sqrt{ \chi^2 (\bP_W^S \Vert \bP_W)
    } \mright],
\end{align*}
although this extends more broadly to any $f$-divergence and it holds for the single-letter, random subset, and randomized subset versions of~\Cref{cor:total_variation_full} as well, including those with a generic distribution or Markov kernel $\bQ$. We only present these two as the Hellinger distance and the $\chi^2$ divergence are special due to their connections to hypothesis testing and the fact that most $f$-divergences locally behave like the $\chi^2$ divergence as described in~\Cref{subsec:f-divergences}.

Moreover, recall from the start of this \Cref{sec:bounds_using_wasserstein_distance}, that the results based on the relative entropy and the Wasserstein distance are obtained through decoupling lemmas, either~\Cref{lemma:decoupling_relent} powered by the Donsker and Varadhan~\Cref{lemma:dv_and_gvp} or the Kantorovich--Rubinstein duality from~\Cref{lemma:kantorovich_rubinstein_duality}. This idea can be extended to other $f$-divergences using the variational representation of $f$-divergences from~\Cref{lemma:variational_representation_f_divergences}. As an example, considering the variational representation of the $\chi^2$ divergence from~\Cref{cor:variational_representation_chi_2}, the following result stems directly by noting that if $\ell(w,Z)$ has a variance bounded by $\sigma^2$, then $\emprisk(w)$ has a variance bounded by $\nicefrac{\sigma^2}{n}$ due to the Independence of the samples.

\begin{proposition}
\label{prop:gen_chi_2}
Consider a loss function $\ell(w,Z)$ with a variance bounded by $\sigma^2$ for all $w \in \cW$. Then,
\begin{equation*}
    \bE \big[ \gen(W,S) \big] \leq \bE \mleft[ \sqrt{ \frac{\sigma^2 \chi^2(\bP_W^S \Vert \bP_W)}{n}}\mright].
\end{equation*}
\end{proposition}

This proposition offers an alternative to~\Cref{th:mi_bounded_variance} in terms of the $\chi^2$ divergence. Note that since $\relent \leq \log(1 + \chi^2)$, although more complicated, \Cref{th:mi_bounded_variance} is often tighter than~\Cref{prop:gen_chi_2}. A similar discussion will appear later in~\Cref{subsec:losses_with_bounded_moment} in the context of PAC-Bayes bounds, where these ideas are refined by~\citet{ohnishi2021novel} and~\citet{esposito2021generalization}.

Finally, we may also mention that since the total variation is symmetric, applying Pinsker's inequality with the distributions in the opposite order to \Cref{cor:total_variation_relent_full,cor:total_variation_relent_single_letter,cor:total_variation_relent_random_subset,cor:total_variation_relent_full_rs,cor:total_variation_relent_single_letter_rs,cor:total_variation_relent_random_subset_rs} and further applying Jensen's inequality yields bounds based on the lautum information $\texttt{L}$ \citep{palomar2008lautum}.
For instance, a corollary of \Cref{th:wasserstein_single_letter} is
\begin{equation*}
    \bE \big[ \gen(W,S) \big]
    \leq \frac{LB}{n}\sum_{i=1}^n  \Psi\big( \textnormal{\texttt{L}}(W;Z_i) \big),
\end{equation*}
where the lautum information $\textnormal{\texttt{L}}(W;Z_i)$ also measures the level of dependence between the hypothesis $W$ and the samples $Z_i$, has ties to independence testing, and has similar properties to the mutual information~\citep{palomar2008lautum}.

\subsection{The Choice of the Metric and the Backward Channel}
\label{subsec:choice_of_metric_and_backward_channel}

In~\citep{lopez2018generalization}, the authors study the characterization of the expected generalization error in terms of the discrepancy between the data distribution $\bP_S$ and the backward channel distribution $\bP_{S}^{W}$ from~\Cref{fig:algorithm_as_channel_and_backward_channel} motivated by its connection to rate-distortion theory, see, for example, \citep[Chapters 24--27]{polianskyi2022} or \citep[Chapter~10]{Cover2006}.

More concretely, they proved that the generalization error is bounded from above by the discrepancy of these distributions, where the discrepancy is measured by the Wasserstein distance of order $p$ with the Minkowski distance of order $p$ as a metric, that is, $\rho(x,y) = \lVert x-y \rVert_p$. Namely, 
\begin{equation*}
    \bE \big[ \gen(W,S) \big] \leq \frac{L}{n^{1/p}} \bE \big[\bW_{p,\lVert \cdot \rVert}^p(\bP_S,\bP_{S}^{W})]^{1/p}.
\end{equation*}
Similarly, every result in this~\Cref{sec:bounds_using_wasserstein_distance} can be replicated considering the backward channel instead of the forward channel, for example, using $\bP_{S}^{W}$ instead of $\bP_W^S$ in~\Cref{th:wasserstein_full} or $\bP_{Z_i}^{W}$ instead of $\bP_{W}^{Z_i}$ in \Cref{th:wasserstein_single_letter}.
However, in this case, the loss $\ell$ would be required to be Lipschitz with respect to the instances in the sample space $\cZ$ for hypotheses $w \in \cW$ and not  for all the hypothesis in the space $\cW$, thus exploiting the geometry of the samples' space and not the hypotheses' one.

As an example, noting that $$\bE \big[ \gen(W,S) \big] = \bE \big[\emprisk(W,S') - \emprisk(W,S) \big],$$ where $S'$ is an independent copy of the training set $S$ such that $\bP_{W,S'} = \bP_W \otimes \bP_S$ produces the bound
\begin{equation*}
    \bE \big[ \gen(W,S) \big] \leq L \bE\big[\bW_\rho(\bP_{S},\bP_{S}^{W}) \big].
\end{equation*}
Compared to~\citep{lopez2018generalization}, these results (i) are valid for any metric $\rho$ as long as the loss  $\ell$ is Lipschitz under $\rho$,
 and (ii) have single-letter and random-subset versions, and (iii) have variants in both the standard and randomized-subsample settings.

The choice of the metric can be decisive for a tight analysis of the presented bounds, and there are times when a loss function can be Lipschitz under several metrics. For example, a bounded loss function represented as a norm is Lipschitz with respect to that norm and the discrete metric. Nonetheless, in many situations, the metric of choice becomes apparent based on the loss function. For example, if we consider the forward channel bounds and samples of the type $z = (x,y)$ and the following two common supervised tasks:
\begin{itemize}
    \item Regression. If a norm is used as the loss function $\ell(w,z) = \lVert w - y \rVert$, then such a norm is also a good choice for a metric since by the reverse triangle inequality the loss is 1-Lipschitz under that metric: $\big \lvert \lVert w - y \rVert - \lVert w' - y \rVert \big \rvert \leq \lVert w - w' \rVert$ for all $w, w' \in \mathcal{W}$.
    \item Classification. If the 0-1 loss is used as the loss function $\ell(w,z) = \bI_{\{f_w(w) \neq y \}}(w,z)$ for some parameterized model $f_w$, then the discrete metric is a good choice since the loss is also 1-Lipschitz under this metric: $\lvert \bI_{\{ f_w(x) \neq y \}}(w,z) - \bI_{\{ f_{w'}(x) \neq y\}}(w',y) \big \rvert \leq \bI_{\{w \neq w'\}}(w,w')$.
\end{itemize}

Similarly, for the backward channel bounds, it is known that the logistic loss, the softmax loss,\footnote{This result can be derived from~\citep[Proposition 3]{gao2018properties} and the $L_1 - L_2$ inequality.} the Hinge loss, and many distance-based losses like norms, the Huber, $\epsilon$-insensitive, and pinball losses, are Lipschitz under the $L_1$ norm metric $\rho(z,z') = \lvert z - z' \rvert$~\citep[Chapter 2]{steinwart2008support}\citep{gao2018properties}. 

\section{Application: Noisy, Iterative Learning Algorithms}
\label{sec:noisy_iterative_learning_algos}

\looseness=-1 This section discusses an application of the previously presented generalization error bounds for noisy, iterative algorithms. This application and line of work was initiated by~\citet{pensia2018generalization}, who were the first to note that a combination of the \emph{chain rule} of the mutual information (and the relative entropy) and the \emph{more data, more information} (or the \emph{monotonicity} of the relative entropy) made these kind of bounds very suitable to study iterative algorithms (see \Cref{prop:properties_minf,prop:properties_relative_entropy} to recall these properties). In their paper, \citeauthor{pensia2018generalization} studied the stochastic gradient Langevin dynamics (SGLD)~\cite{gelfand1991recursive,welling2011bayesian} and the stochastic gradient Hamiltonian Monte Carlo (SGHMC)~\citep{chen2014stochastic} algorithms. Later works, including ours, mostly focused on improving the bounds for SGLD~\citep{li2019generalization,bu2020tightening,negrea2019information, haghifam2020sharpened,rodriguez2020randomsubset, wang2021analyzing, wang2023generalization,futami2024time}  and developing new bounds for stochastic gradient descent (SGD)~\citep{neu2021information, wang2021generalization,wang2024sample}.

\subsection{Stochastic Gradient Langevin Dynamics}
\label{subsec:sgld}

The SGLD algorithm is an iterative, optimization procedure to learn the parameters of a parameterized model, often a differentiable network, from a training set $s \in \cZ^n$~\cite{gelfand1991recursive,welling2011bayesian}. As discussed previously in~\Cref{sec:parameterized_models}, since the parameters returned by the algorithm $w \in \bR^d$ completely characterize the hypothesis, we may employ the two terms without distinction.

The method originates from the Bayesian inference community and combines ideas from stochastic optimization and Langevin dynamics. This combination helps SGLD to explore the parameter space more effectively. 
The algorithm is simple:
\begin{itemize}
    \item First, it starts with a random initialization $W_0$ of the model's parameters.
    \item Then, for each iteration $t \in [T]$, it samples a random batch $s_{V_t}$ from the dataset $s$; updates the previous parameter $W_{t-1}$ with a scaled $(-\eta_t)$ version of the gradient of the empirical risk of that parameter in the random batch, that is, $\nabla_w \emprisk (W_{t-1}, s_{V_t})$; and adds a scaled $(\sigma_t)$ isotropic Gaussian random noise $\Xi_t \sim \cN(0,I_d)$. More precisely, the update rule is
    \begin{equation}
    	\label{eq:sgld}
    	W_t = W_{t-1} - \eta_t \nabla_w \emprisk(W_{t-1}, s_{V_t}) + \sigma_t \Xi_t,
    \end{equation}
    where $W_t$ are the parameters (hypothesis) at iteration $t$.
    \item Finally, it returns the final hypothesis $W_T$.
\end{itemize}

When the batch is composed of all samples, that is, there is no stochasticity in the sample selection, the algorithm is called simply 
Langevin dynamics (LD).

\subsubsection{Unconditional Bounds}

\citet{pensia2018generalization} leveraged the generalization bounds from~\Cref{th:mi_bound_bounded}~\citep{xu2017information} and analyzed the information captured by the final hypothesis returned by SGLD about the training set. As discussed in the section opening, the development of their bounds follows from a sequential application of the \emph{more data, more information} and the \emph{chain rules} of the mutual information (\Cref{prop:properties_minf}), that is
\begin{align*}
	\minf(W_T; S) &\leq \minf(W^T, W_0; S) \\
	&= \minf(W_0; S) + \sum_{t=1}^T \minf(W_t; S | W^{t-1}) \\
	&= \sum_{t=1}^T \minf(W_t; S | W_{t-1}),
\end{align*}
where in the last equation we used that the initial parameter $W_0$ is independent of the data $S$ and the Markovian nature of the problem~\eqref{eq:sgld}. Then, we may note that $\minf(W_t; S | W_{t-1}) = \ent(W_t | W_{t-1}) - \ent(W_t | S, W_{t-1})$. Now, given $S$ and $W_{t-1}$, the random variable $W_t$ follows a Gaussian distribution with covariance $\sigma_t^2 I_d$ and mean given by~\eqref{eq:sgld}, and therefore $\ent(W_t | S, W_{t-1}) = \frac{d}{2} \log ( 2 \pi e \sigma_t^2)$. On the other hand, the distribution of $W_t$ given $W_{t-1}$ is unknown. Since the differential entropy is invariant to translations~\citep[Theorem 8.6.3]{Cover2006}, we may instead study the random variable $W_t - W_{t-1}$ given $W_{t-1}$. Now, we may bound this variable's expected squared norm as
\begin{equation*}
	\bE^{W_{t-1}} \Big[ \lVert W_t - W_{t-1} \rVert_2^2 \Big] \leq \eta_t^2 \bE^{W_{t-1}} \Big [ \lVert \nabla_w \emprisk(W_{t-1}, S_{V_t}) \rVert^2 \Big] + \sigma_t^2 \bE^{W_{t-1}} \Big[ \lVert \Xi_t \rVert_2^2 \Big].
\end{equation*}

When studying gradient-based, iterative algorithms, it is common to consider Lipschitz-continuous losses (\Cref{def:lipschitz}).
The main reason for this assumption is that if a loss is $L$-Lipschitz-continuous, then it has a bounded gradient, that is $\sup_{w \in \cW, z \in \cZ} \lVert \ell(w,z) \rVert_2 \leq L$. Under this assumption, the expected squared norm is bounded by $\eta_t^2 L^2 + d \sigma^2$. Notably, the entropy of a variable with an expected squared norm bounded by $a$ is bounded by the entropy of a Gaussian random variable with an isotropic covariance with scale $\nicefrac{a}{d}$. Combining this upper bound with the prior calculation of the entropy $\ent(W_t | W_{t-1}, S)$ gives us an upper bound for the mutual information
\begin{align*}
	\minf(W_t; S | W^{t-1}) &\leq \frac{d}{2} \log \left( \frac{2 \pi e (\eta_t^2 L^2 + d \sigma_t^2)}{d} \right) - \frac{d}{2} \log \left( 2 \pi e \sigma_t^2 \right) \\
	&= \frac{d}{2} \log \left( 1 + \frac{\eta_t^2 L^2 }{d \sigma_t^2} \right).
\end{align*}

The sum of all these terms bounds the mutual information and, as per~\Cref{sec:bounds_using_mutual_information}, it also bounds the expected generalization error.

\begin{proposition}[{\cite[Theorem 1]{pensia2018generalization}}]
\label{prop:pensia_mi_bound}
Consider an $L$-Lipschitz continuous loss. Then, the mutual information $\minf(W_T;S)$ for the SGLD algorithm is bounded by
\begin{equation*}
	\minf(W_T; S) \leq \frac{d}{2} \sum_{t=1}^T \log \left( 1 + \frac{\eta_t^2 L^2 }{d \sigma_t^2} \right).
\end{equation*}
\end{proposition}

Let the noise variance be $\sigma_t^2 = \eta_t$, the batch size be fixed $|v_t| = |v|$ for all $t \in [T]$, and the algorithm run for $k$ epochs so that $T \approx \nicefrac{k n}{|v|}$. If we let the step size decrease as $\eta_t = \nicefrac{c}{t}$ for some constant $c > 0$, then, a corollary of this proposition is that the mutual information is bounded by
\begin{equation*}
	\minf(W_T; S) \leq \frac{L^2 c}{2} \mleft (1 + \log \frac{nk}{|v|} \mright).
\end{equation*}
This follows by first applying the bound $\log(1+x) \leq x$ for all $x > 0$ and then the common bound on the Harmonic numbers $\sum_{t=1}^T \nicefrac{1}{t} \leq 1 + \log T$. \citet{pensia2018generalization} considered the bound from~\Cref{th:mi_bound_bounded}~\citep{xu2017information} to establish that if the loss was sub-Gaussian, then the SGLD algorithm had a generalization error with a vanishing rate in $\cO(L \sqrt{\nicefrac{\log (\frac{nk}{|v|})}{n}} )$. At the time, this result compared negatively with respect to \citep{mou2018generalization} in terms of the relationship with the data, as they showed that the generalization error was in $\cO(L \nicefrac{ \sqrt{\log n k}}{n} )$ for a single sample batch size $|v|=1$.\footnote{The logarithmic dependence in \citep{mou2018generalization} is actually better, although an exact portrait of this dependence goes beyond the scope of this exposition.} In hindsight, considering a fast-rate bound like our~\Cref{th:mi_bound_bounded_fast_rate}, \citet{pensia2018generalization}'s result implies that the population risk deviates from an estimate based on the empirical error at a rate in $\cO(L^2 \nicefrac{\log (\frac{nk}{|v|})}{n} )$, essentially matching the rate from \citep{mou2018generalization}.

\citet{bu2020tightening} refined \citet{pensia2018generalization}'s analysis by considering the single-letter mutual information. The analysis follows similarly to the one above employing the \emph{more data, more information} and the \emph{chain rules} of the mutual information properties from~\Cref{prop:properties_minf}. Although for a particular iterate $t$, the resulting mutual information did not improve substantially, as they showed that
\begin{equation*}
	\minf(W_t; Z_i | W_{t-1}) \leq \frac{d}{2} \log \left( 1 + \frac{\eta_t^2 L^2 }{d \sigma_t^2} \right);
\end{equation*}
\citet{bu2020tightening} introduced an innovation that improved the logarithmic dependence $\log(nk)$ in the final rate. Their trick was to consider a random variable $R$ that is only dependent on the parameters (or hypothesis) $W_t$ and condition on this random variable as described in~\Cref{rem:extra_randomness}. In this case, the randomness was chosen to be the trajectory of batches used for each of the steps in~\eqref{eq:sgld}, that is, $R = V^T$. In this way, the mutual information $\minf(W_t; Z_i | W_{t-1}, V^T = v^T)$ would only be included in the final summation for the iterations $t \in [T]$ in which $i \in v_T$, which are denoted as $\cT_i(v^T)$. This observation led to the following result.

\begin{proposition}[{\cite[Proposition 3, modified]{bu2020tightening}}]
\label{prop:bu_single_letter_mi_bound}
Consider an $L$-Lipschitz continuous loss. Then, the single-letter mutual information $\minf(W_T;Z_i)$ for the SGLD algorithm for a particular set of batches $v^T$ is bounded by
\begin{equation*}
	\minf(W_T;Z_i| V^T = v^T) \leq \frac{d}{2} \sum_{t \in \cT_i(v^T)} \log \left( 1 + \frac{\eta_t^2 L^2 }{d \sigma_t^2} \right).	
\end{equation*}
\end{proposition}

Under the same assumptions from before, that is, the noise variance be $\sigma_t^2 = \eta_t$, the algorithm runs for $k$ epochs so that $T \approx \nicefrac{nk}{|v|}$, and the step size decrease as $\eta_t = \nicefrac{c}{t}$ for some constant $c > 0$, it follows from~\Cref{prop:bu_single_letter_mi_bound} that
\begin{equation*}
	\minf(W_T;Z_i| V^T = v^T) \leq \frac{L^2 c}{2} \sum_{t \in \cT(v^T)} \frac{1}{t}.
\end{equation*}
Under the assumption that each sample is only used once per epoch, combining this bound on the individual mutual information with a slow-rate single-letter bound like~\Cref{th:single_letter_mi_bound_general_cgf} ensures that the generalization error rate is in $\cO(L \sqrt{\nicefrac{\log k}{n}} )$, which improves the previous analysis by a factor of $\cO(\log(\frac{n}{|v|}))$~\citep{bu2020tightening}, with the maximal benefit for a batch size of $|v|=1$. The key is to take the expectation with respect to the random indices outside of the square root for bounded or sub-Gaussian losses. Unfortunately, this technique does not offer a benefit with respect to the prior analysis when using fast-rate bounds since
\begin{align*}
	\frac{1}{n} \sum_{i=1}^n \minf(W_T; Z_i | V^T) &\leq  \frac{L^2 c}{2n} \sum_{i=1}^n \left \{ \frac{1}{i} + \sum_{t= 1}^{k-1} \frac{1}{nt} \right \} \\
	&\leq \frac{L^2 c}{2n} \big( \log(n) + 1 + \log(k-1) + 2 \big) \in \cO \left( L^2 \cdot \frac{\log nk}{n} \right).
\end{align*}

\subsubsection{Random-Subset Bounds}

\citet{negrea2019information} considered a similar analysis to the one from~\citet{pensia2018generalization} and \citet{bu2020tightening}. The main difference is that they considered a version of the random-subset bound from~\Cref{th:tighter_random_subset_mi_bounded_loss} with $m=1$. The main object to consider in this bound is the relative entropy $\relent( \bP_W^{S, R} \Vert \bQ(S^{-J}, R))$. The main appeal to considering this object instead of the mutual information or the single-letter mutual information is that it is data-dependent. To simplify the notation, let $\bQ_{W_T}^{S^{-J}, R} = \bQ(S^{-J}, R)$ be the distribution of a random variable $W_{T}'$. \citet{negrea2019information}'s analysis also begins considering the \emph{monotonicity} and the \emph{chain rules} of the relative entropy from ~\Cref{prop:properties_relative_entropy}. More precisely
\begin{align*}
	\relent( \bP_{W_T}^{S, V^T} &\Vert \bQ_{W_T}^{S^{-J}, V^T} ) \\
	&\leq \relent( \bP_{W^T}^{S, V^T} \Vert \bQ_{W^T}^{S^{-J}, V^T} ) \\
	&= \relent( \bP_{W_0}^{S, V^T} \Vert \bQ_{W_0}^{S^{-J} V^T} ) + \sum_{t = 1}^T \relent( \bP_{W_t}^{W^{t-1}, S,  V^T} \Vert \bQ_{W_t}^{W^{t-1}, S^{-J}, V^T}) \\
	&= \sum_{t = 1}^T \relent( \bP_{W_t}^{W_{t-1}, S, V^T} \Vert \bQ_{W_t}^{W_{t-1}, S^{-J}, V^T}),
\end{align*}
where we also considered the randomness $R$ to be the trajectory of batches used for each of the steps in~\eqref{eq:sgld}. The last inequality follows by choosing $\bQ_{W_0}^{S^{-J}, V^T} = \bP_{W_0}$ since $W_0$ is independent of the data, making the relative entropy zero; and also choosing $\bQ_{W_t}^{W^{t-1}, S^{-J}, V^T}$ to only depend on the past iterate. 

In this way, one may focus on bounding the relative entropy $\relent( \bP_{W_t}^{W_{t-1}, S, V^T} \Vert \bQ_{W_t}^{W_{t-1}, S^{-J}, V^T})$ by designing an appropriate data-depending distribution $\bQ_{W_t}^{W_{t-1}, S^{-J}, V^T}$. For this endeavor, consider $W_t$ and $W_t'$ be the random variables distributed according to $\bP_{W_t}^{W_{t-1}, S, V^T}$ and $\bQ_{W_t}^{W_{t-1}, S_{-J}, V^T}$ for some fixed $J = j$, $S^{-J} = s^{-j}$, and $V^T = v^T$. The random variable $W_t$ is Gaussian with mean
\begin{equation}
    \label{eq:sgld_rv_mean}
    w_{t-1} - \eta_t \nabla_w \emprisk (w_{t-1}, s_{v_t})
\end{equation}
and covariance $\sigma_t I_d$. Similarly, as in the analysis from~\citet{bu2020tightening}, when the index $j$ is not chosen in the batch $v_t$, we may choose $W_t'$ to be equal to $W_t$, thus having a zero relative entropy. When, $j$ belongs to $v_t$, a sensible choice for the design random variable $W_t'$ is another Gaussian with mean
\begin{equation}
    \label{eq:random_variable_design_sgld_standard_setting}
    W' = w_{t-1} - \eta_t \Big( \frac{|v|-1}{|v|} \nabla_w \emprisk(w_{t-1}, s_{v_t \setminus \{j\}}) + \frac{1}{|v|} \nabla_w \emprisk (w_{t-1}, s^{-j}) \Big) 
\end{equation}
and also with covariance $\sigma_t I_d$, where $|v| = |v_t|$ is the size of the batch, which is assumed to be equal for all iterations. In this way, as the relative entropy between Gaussian distributions is known, we have that
\begin{align*}
    \relent( \bP_{W_t}^{W_{t-1}, S, V^T} &\Vert \bQ_{W_t}^{W_{t-1}, S^{-J}, V^T}) = \\
    &\frac{\eta_t^2}{2 |v|^2 \sigma_t^2} \Big\lVert \nabla_w \ell(W_{t-1}, Z_J) - \nabla_w \emprisk(W_{t-1}, S^{-j}) \Big \rVert_2^2.
\end{align*}

\citet{negrea2019information} defined the random variable $\Gamma_t \coloneqq \nabla_w \ell(W_{t-1}, Z_J) - \nabla_w \emprisk(W_{t-1}, S^{-j})$ as the \emph{gradient incoherence}, which measures how different the gradient at a sample is from the average gradient on all the other samples. This quantity can be much smaller than the gradient norms or a Lipschitz constant. This is often the case, as confirmed empirically in~\citep{negrea2019information}. %
Combining the particularization of~\Cref{th:tighter_random_subset_mi_bounded_loss} for $m=1$ with the above analysis results in the following result.

\begin{theorem}[{\citet[Theorem 3.1]{negrea2019information}}]
\label{th:negrea_sgld}
    Consider a loss with a range bounded in $[a,b]$. Then, the expected generalization error of SGLD is bounded from above by
    \begin{equation*}
        \bE \big[ \gen(W_T; S) \big] \leq \frac{b-a}{\sqrt{2} |v|} \bE \left[ \sqrt{\sum_{t \in \cT_J(V^T)} \frac{\eta_t^2}{\sigma_t^2} \bE^{S^{-J}, J, V^T} \big[ \lVert \Gamma_t \rVert_2^2 \big] } \right]
    \end{equation*}
\end{theorem}

\Cref{th:negrea_sgld} has essentially the same form as the generalization bound obtained by combining \Cref{prop:bu_single_letter_mi_bound} and \Cref{th:single_letter_mi_bound_general_cgf} for a loss with a bounded range. The main difference is the appearance of the gradient incoherence instead of the norm of the gradients or the Lipschitz constant. Therefore, assuming that the loss is $L$-Lipschitz results again in a bound in $\cO(\frac{L}{|v|} \sqrt{\nicefrac{\log k}{n}} )$ with the assumption that the noise variance is $\sigma_t^2 = \eta_t$, that the algorithm runs for $k$ epochs and $T \approx \nicefrac{k n}{b}$, and that the step size decreases as $\eta_t = \nicefrac{c}{t}$ for some constant $c > 0$. Moreover, it has the extra benefit of having the $\nicefrac{1}{|v|}$ factor. However, note that for larger batch sizes $|v|$, one typically needs a larger number of epochs $k$.\footnote{Other, slightly improved rates can be obtained with different choices of the learning rate $\eta_t$~\citep[Appendix E]{negrea2019information}. This is, however, outside of the scope of this exposition.}

\citet{haghifam2020sharpened} and we~\citep{rodriguez2020randomsubset} continued this reasoning considering the randomized-subsample setting and bounds similar to~\Cref{th:tighter_random_subset_cmi_bounded_loss} with $m=1$. \citet{haghifam2020sharpened} provided a result valid for full batch SGLD, also called just Langevin dynamics (LD). We extended such a result for the general SGLD and improved it to be resilient to losses with large gradient norms or large Lipschitz constants.

The main object in~\Cref{th:tighter_random_subset_cmi_bounded_loss} is the relative entropy $\relent(\bP_W^{\tilde{S}, U, R} \Vert \bQ(\tilde{S}, U^{-J}, R))$. The main advantage of this quantity with respect to the quantity considered by \citet{negrea2019information} is that now, for our analysis, we have access to all the super samples $\tilde{S}$ and the only information missing in the design Markov kernel $\bQ(\tilde{S}, U^{-J}, R))$ is the index $J$. As we will see shortly, this extra knowledge will help us to design better random variables associated with the data-dependent distributions.

As previously, let $\bQ_{W_T}^{\tilde{S}, U^{-J}, R} = \bQ(\tilde{S}, U^{-J}, R))$ be the distribution of a random variable $W_T'$ to simplify the notation. By the \emph{monotonicity} and the \emph{chain rules} of the relative entropy from~\Cref{prop:properties_relative_entropy}
\begin{align*}
    \relent(&\bP_{W_T}^{\tilde{S}, U, V^T} \Vert \bQ_{W_T}^{\tilde{S}, U^{-J}, V^T}) \\
    &\leq \relent(\bP_{WT}^{\tilde{S}, U, V^T} \Vert \bQ_{W^T}^{\tilde{S}, U^{-J}, V^T}) \\
    &= \relent(\bP_{W_0}^{\tilde{S}, U, V^T} \Vert \bQ_{W_0}^{\tilde{S}, U^{-J}, V^T}) + \sum_{t=1}^T \relent(\bP_{W_t}^{W_{t-1},\tilde{S}, U, V^T} \Vert \bQ_{W_t}^{W^{t-1}, \tilde{S}, U^{-J}, V^T}) \\
    &= \sum_{t \in \cT_J(V^T)} \relent(\bP_{W_t}^{W_{t-1},\tilde{S}, U, V^T} \Vert \bQ_{W_t}^{W^{t-1}, \tilde{S}, U^{-J}, V^T}),
\end{align*}
\looseness=-1 where similarly to before we considered the randomness $R$ to be the trajectory of batches used for each of the steps in~\eqref{eq:sgld}, and we chose $\bQ_{W_0}^{\tilde{S}, U^{-J}, V^T} = \bP_{W_0}$ due to the independence of $W$ with the data and the indices. Moreover, we directly noted that we can choose $\bQ_{W_t}^{W^{t-1}, \tilde{S}, U^{-J}, V^T} = \bP_{W_t}^{W_{t-1}, \tilde{S}, U, V^T}$ for the iterations where $J$ is not included in the batch since, at those iterations, $W_t$ is independent of $\tilde{Z}_{J,0}$ and $\tilde{Z}_{J,1}$ given $W_{t-1}$. Finally, the main change to previous analyses is that we kept the information of all previous hypotheses available in the design Markov kernels.

Now we may focus on the terms $\relent(\bP_{W_t}^{W_{t-1},\tilde{S}, U, V^T} \Vert \bQ_{W_t}^{W^{t-1}, \tilde{S}, U^{-J}, V^T})$ separately. Consider, as before, $W_t$ and $W_t'$ be the random variables distributed according to $\bP_{W_t}^{W_{t-1},\tilde{S}, U, V^T}$ and $\bQ_{W_t}^{W^{t-1}, \tilde{S}, U^{-J}, V^T}$ for some fixed $J=j$, $\tilde{S}=s$, $U^{-J} = u^{-j}$, and $V^T = v^T$. We already established that $W_t$ is a Gaussian random variable with mean~\eqref{eq:sgld_rv_mean} and covariance $\sigma_t^2 I_d$. However, it is useful to re-write its mean using $u_j$ and $\tilde{z}_{j,0}$ and $\tilde{z}_{j,1}$, namely the mean is
\begin{align*}
    w_{t-1} - \frac{\eta_t}{|v|} \Big[(|v|-1) &\nabla_w \emprisk(w_{t-1}, s_{v_t \setminus \{j \}}) + \\
    &(1-u_j) \nabla_w \ell(w_{t-1},\tilde{z}_{j,0}) + u_j \nabla_w \ell(w_{t-1}, \tilde{z}_{j,1}) \Big],
\end{align*}
where $|v| = |v_t|$ is the size of the batch, which is assumed to be constant for all iterations.

A first approach to designing the random variable $W_t'$, proposed by~\citet{haghifam2020sharpened} in the context of LD, is to design $W_t'$ also as a Gaussian random variable with the same covariance as $W_t$ and a mean as close as possible to the mean of $W_t$. Since $U_j$ is unknown, we can substitute it by an estimate $\pi_{j,t}$ of the probability that $U_j = 1$. This estimator uses the information available at that iteration, namely all the data $\tilde{s}$, all the other indices $u^{-j}$, all the previous weights $w^{t-1}$, and all the indices of samples employed up to (and including) that iteration $v^{t}$. In the end of this subsection, we will show an example of such an estimator. For the time being, it suffices to understand that it will start with no information and $\pi_{j,0} = \nicefrac{1}{2}$, and it will become closer to $U_j$ as the iterations advance if the information of the index $U_j$ leaks into the weights $w^t$. In this way, the mean of $W_t'$ is
\begin{align*}
    w_{t-1} - \frac{\eta_t}{|v|} \Big[(|v|-1) &\nabla_w \emprisk(w_{t-1}, s_{v_t \setminus \{j \}}) + \\
    &(1-\pi_{j,t}) \nabla_w \ell(w_{t-1},\tilde{z}_{j,0}) + \pi_{j,t} \nabla_w \ell(w_{t-1}, \tilde{z}_{j,1}) \Big].
\end{align*}

Finally, since the relative entropy between Gaussian distributions is known we have that
\begin{equation*}
    \relent(\bP_{W_t}^{W_{t-1},\tilde{S}, U, V^T} \Vert \bQ_{W_t}^{W^{t-1}, \tilde{S}, U^{-J}, V^T}) = \frac{\eta^2}{2 |v|^2 \sigma^2} \bE^{\tilde{S}, U, V^T, J} \Big[ (U_J - \pi_{J,t})^2 \lVert \Gamma_{J,t} \rVert_2^2 \Big],
\end{equation*}
where $\Gamma_{J,t} \coloneqq \nabla_w \ell(w_{t-1}, \tilde{z}_{j,0}) - \nabla_w(w_{t-1}, \tilde{j,1})$ is the \emph{two-sample incoherence} at iteration $t$. Now, combining the particularization of~\Cref{th:tighter_random_subset_cmi_bounded_loss} for $m=1$ with the above analysis results in the following result.

\begin{theorem}
    \label{th:cmi_sgld_gaussian}
    Consider a loss with a range bounded in $[a,b]$. Then, the expected generalization error of SGLD is bounded from above by
    \begin{equation*}
        \bE \big[ \gen(W_T,S) ] \leq \frac{\sqrt{2}(b-a)}{|v|} \bE \left[ \sqrt{\sum_{t \in \cT_J(V^T)} \frac{\eta_t^2}{ \sigma^2_t}\bE^{\tilde{S}, U, V^T, J} \Big[ (U_J - \pi_{J,t})^2 \lVert \Gamma_{J,t} \rVert_2^2 \Big] }\right].
    \end{equation*}
\end{theorem}

This result is similar to~\Cref{th:negrea_sgld}. The incoherence terms $\Gamma_t$ and $\Gamma_{J,t}$ are similar and both represent how the gradient of a sample differs from other gradients. The main innovation is the addition of $(U_j - \pi_{J,t})^2$, which takes into account how well one can predict which sample is used in the training set at iteration $t$. In this way, samples $j$ that are very informative at early iterations, do not contribute negatively to the generalization error in later iterations as $(U_j - \pi_{j,t})^2$ will be close to zero. On the other hand, if a sample is close to indistinguishable, then $\pi \approx 0.5$ and $(U_j - \pi_{j,t})^2 \approx \nicefrac{1}{4}$, essentially recovering~\Cref{th:negrea_sgld}.

Even with these advances, in the early iterations, the estimator $\pi_{j,t}$ is still not a good estimate of the generalization error. In this case, if the gradient incoherence is large, then the bound may be affected by this. As an alternative, we proposed to design $W_t'$ as a random variable distributed as a mixture of two Gaussian distributions with means
\begin{align*}
    \mu_{j,t,0} =&w_{t-1} - \frac{\eta_t}{|v|} \Big[ (|v|-1) \nabla_w \emprisk(w_{t-1}, s_{v_t \setminus \{j\}}) + \nabla_w \ell(w_{t-1}, \tilde{z}_{j,0}) \Big] \textnormal{ and} \\
    \mu_{j,t,1} =&w_{t-1} - \frac{\eta_t}{|v|} \Big[ (|v|-1) \nabla_w \emprisk(w_{t-1}, s_{v_t \setminus \{j\}}) + \nabla_w \ell(w_{t-1}, \tilde{z}_{j,1}) \Big]
\end{align*}
and with weights, or responsibilities $(1-\pi_{j,t})$ and $\pi_{j,t}$ respectively, that is $W_t' \sim (1-\pi_{j,t}) \cN(\mu_{j,t,0}, \sigma_t^2 I_d) + \pi_{j,t} \cN(\mu_{j,t,1}, \sigma_t^2 I_d)$. Then, we may employ~\Cref{lemma:mixture_kl_ub} to see that $\relent(\bP_{W_t}^{W_{t-1},\tilde{S}, U, V^T} \Vert \bQ_{W_t}^{W^{t-1}, \tilde{S}, U^{-J}, V^T})$ is bounded from above by
\begin{align}
    - \log \Bigg[ (1-\pi_{J,t}) &e^{- \frac{\eta^2}{2 |v|^2 \sigma^2} \bE^{\tilde{S}, U, V^T, J} \Big[ U_J^2 \lVert \Gamma_{J,t} \rVert_2^2 \Big]} + \nonumber \\
    &\quad \pi_{J,t} e^{- \frac{\eta^2}{2 |v|^2 \sigma^2} \bE^{\tilde{S}, U, V^T, J} \Big[ (1-U_J)^2 \lVert \Gamma_{J,t} \rVert_2^2 \Big]} \Bigg].
    \label{eq:mini-batch-ld-mixture-intermediate}
\end{align}

We can make this equation more compact if we note that when $U_J = 0$, then it reduces to
\begin{equation*}
    - \log \left[ (1-\pi_{J,t}) + \pi_{J,t} e^{- \frac{\eta^2}{2 |v|^2 \sigma^2} \bE^{\tilde{S}, U, V^T, J} \Big[ (1-U_J)^2 \lVert \Gamma_{J,t} \rVert_2^2 \Big]} \right]
\end{equation*}
and that when $U_J = 1$, then it reduces to
\begin{equation*}
    - \log \left[ (1-\pi_{J,t}) e^{- \frac{\eta^2}{2 |v|^2 \sigma^2} \bE^{\tilde{S}, U, V^T, J} \Big[ U_J^2 \lVert \Gamma_{J,t} \rVert_2^2 \Big]} + \pi_{J,t} \right].
\end{equation*}
Hence, since $|U_J - \pi_{J,t}|$ is equal to $\pi_{J,t}$ when $U_J = 0$, and to $(1-\pi_{J,t})$ when $U_J = 1$, we may write the right-hand side of~\eqref{eq:mini-batch-ld-mixture-intermediate} as 
\begin{equation*}
    -\log \left[|U_J - \pi_{J,t}| \exp \left(- \frac{\eta_t^2 \lVert \Gamma_{J,t} \lVert ^2}{2\sigma_t^2|v|^2} \right) + |1 - U_J - \pi_{J,t}| \right].
\end{equation*}

Again, combining the particularization of~\Cref{th:tighter_random_subset_cmi_bounded_loss} for $m=1$ with the above analysis results in the following result.

\begin{theorem}
    \label{th:cmi_sgld_mixture}
    Consider a loss with a range bounded in $[a,b]$. Then, the expected generalization error of SGLD is bounded from above by
    \begin{align*}
        \bE \big[ &\gen(W_T,S) ] \leq \\
        &\sqrt{2}(b-a) \bE \left[ - \sum_{t \in \cT_J(V^T)} \log \left[|U_J - \pi_{J,t}| \exp \left(- \frac{\eta_t^2 \lVert \Gamma_{J,t} \lVert ^2}{2\sigma_t^2|v|^2} \right) + |1 - U_J - \pi_{J,t}| \right] \right].
    \end{align*}
\end{theorem}

For example, if $\pi_{J,t} = \nicefrac{1}{2}$, which is a reasonable value during the first iterations, the bound from~\Cref{th:cmi_sgld_mixture} is tighter than the one from \Cref{th:cmi_sgld_gaussian} for $\| \Gamma_{J,t} \| \gtrsim 2.21 \nicefrac{\sigma_t |v|}{\eta_t}$. This may likely be the case when $\nicefrac{\eta_t}{\sigma_t} \in \Theta(n^{\alpha/2})$ for $\alpha \in (0,1)$~\citep[Appendix E]{negrea2019information}.
In fact, for large values of ${\frac{\eta_t^2 \|\Gamma_{J,t}\|^2}{2\sigma_t^2|v|^2}}$, for instance, when the Lipschitz condition is not met, the terms inside the square root in~\Cref{th:cmi_sgld_mixture} tend to $ {-}\log|1-U_J-\pi_{J,t}|$ from below, which is upper bounded by $\log 2$ for a random estimator $\pi_{J,t} = \nicefrac{1}{2}$.

Moreover, we note that the only terms that change between~\Cref{th:cmi_sgld_gaussian} and~\Cref{th:cmi_sgld_mixture} are the summands inside the square root.
These terms are obtained by considering different data-dependent distributions $Q_{\smash{W_t| W^{t-1}, U_{J^c}, \tilde{S}, V^t}}$, which we can choose arbitrarily at each iteration. Hence, we may choose the tightest form for each summand, as it is stated next.

\begin{corollary}
    Consider a loss with a range bounded in $[a,b]$. Furthermore, consider the random variables
    \begin{align*}
        F_{J,t} &= \frac{\eta^2}{2 |v|^2 \sigma^2} \bE^{\tilde{S}, U, V^T, J} \Big[ (U_J - \pi_{J,t})^2 \lVert \Gamma_{J,t} \rVert_2^2 \Big] \textnormal{ and } \\
        G_{J,t} &= -\log \left[|U_J - \pi_{J,t}| \exp \left(- \frac{\eta_t^2 \lVert \Gamma_{J,t} \lVert ^2}{2\sigma_t^2|v|^2} \right) + |1 - U_J - \pi_{J,t}| \right].
    \end{align*}
    Then, the expected generalization error of SGLD is bounded from above by
    \begin{equation*}
        \bE \big[ \gen(W_T, S) \big] \leq \sqrt{2} (b-a) \bE \left[\sqrt{\sum_{t \in \cT_{J}(V^T)} \min \{F_{J,t}, G_{J,t} \}} \right].
    \end{equation*}
\end{corollary}

To conclude, we may write an explicit estimator $\pi_{J,t}$. An example of how to build such an estimate, based on binary hypothesis testing, is presented in~\citep{haghifam2020sharpened} for LD. Adapted to SGLD, we let the estimate $\pi_{J,t}$ be a function $\phi: \bR \rightarrow [0,1]$ of the log-likelihood ratio between the probability that $U_J = 1$ and $U_J = 0$, based on $(W^{t-1},J,\tilde{S},U_{J^c},V^{t})$. That is, 
\begin{align*}
 \pi_{J,t} 
 \coloneqq \phi \left( \log \frac{\bP_{\smash{U_J| W^{t-1}, \tilde{S}, U_{J^c},V^{t}}}(1)}{\bP_{\smash{U_J| W^{t-1}, \tilde{S}, U_{J^c},V^{t}}}(0)} \right).
\end{align*}

We can calculate this log-likelihood exactly since density functions $\sfp_{W^{t-1}}^{\tilde{S}, U_{J^c},V^{t}, U=u}$ exist and are Gaussian. That is,
\begin{align*}
    &\sfp_{W^{t-1}}^{\tilde{S}, U_{J^c},V^{t}, U=u}(w^{t-1}) = \prod_{t'=1}^{t-1} \left[ \frac{1}{2 \pi \sigma_t^2} \right]^{\frac{d}{2}} \exp \left( - \frac{1}{2 \sigma_t^2} Y_{J,t',u} \right),
\end{align*}
where
\begin{equation*}
    Y_{J,t',u} \coloneqq \left \lVert w_{t'} - w_{t'-1} +  \frac{\eta_t(|v|-1)}{|v|} \nabla_w \emprisk(w_{t'-1}, s_{v_t \setminus \{ j \}}) + \frac{\eta_t}{|v|} \nabla_w \ell(w_{t'-1}, \tilde{z}_{J,U})\right \rVert_2^2 .
\end{equation*}

Therefore, the estimate is 
\begin{equation}
    \label{eq:u_estimate_sgld}
    \pi_{J,t} = \phi \left( \sum_{t' = 1}^{t-1} (Y_{J,t',0} - Y_{J,t',1}) \right).
\end{equation}

\looseness=-1 Note how $Y_{J,t',u}$ is close to zero when $U=u$ since the only difference between the weights $w_{t'}$ and the update comes from the inherent randomness $\Xi_{t}$ from SGLD. On the other hand $Y_{J, t', u}$ is larger when $U = 1-u$ since now the difference comes both from the SGLD randomness $\Xi_t$ as well as choosing the wrong gradient in the estimate. Hence, the value inside the function $\phi$ in~\eqref{eq:u_estimate_sgld} is expected to be positive when $U=1$ and negative when $U=0$ with a larger absolute value after each iteration. Then, a good choice for the function $\phi$ could be a sigmoid function, which will bring positive values closer to 1 and negative values closer to 0. In fact, the sigmoid function is often used to bring log-likelihood ratios to probabilities.

After these developments, new studies on the generalization error of SGLD using different assumptions and new information-theoretic bounds appeared~\citep{wang2021analyzing,wang2023generalization,futami2024time}. Those reached similar results and focused on eliminating the time dependence in their results.

\subsection{Beyond SGLD}
\label{subsec:beyond_sgld}

\citet{pensia2018generalization} considered the study of the stochastic gradient Hamiltonian Monte Carlo (SGHMC)~\citep{chen2014stochastic}. This algorithm is similar to SGLD with the addition of a ``velocity'' vector to include momentum in the parameter updates. Namely, the parameters $W_0 \in \cW$ are still initialized randomly and the velocity parameters are initialized to all zeros $v_0 \in \cW$. After that, at each iteration $t \in [T]$, a random batch $s_{V_t}$ from the dataset $s$ is collected and the parameters are updated according to
\begin{align*}
    V_t &= \alpha_t V_{t-1} + \eta_t \nabla_w \emprisk(W_{t-1}, s_{V_t}) \\
    W_t &= W_{t-1} - \alpha_t V_{t-1} - \eta_t \nabla_w \emprisk(W_{t-1}, s_{V_t}) + \sigma_t \Xi_t.
\end{align*}
Instead of directly studying this algorithm, they considered a slight variant where, at each iteration, the velocity parameters are also perturbed with some noise $\sigma_t \Xi'_t$ that is independent of all other random variables. In this way, considering the tuple $(W_t, V_t)$ as a vector in $\bR^d$, they could replicate their results for SGLD within a factor of $\sqrt{2}$.

After that, \citet{neu2021information} adapted \citet{pensia2018generalization}'s analysis to stochastic gradient descent (SGD). The challenge for studying SGD is that, at each iteration $t \in [T]$, given a fixed dataset $s$, the indices of the batch $v_t$ and a fixed value of the parameters at the previous iteration $w_{t-1}$, the parameters at that iteration are deterministic and hence their distribution is $$\bP_{W_t}^{S=s,V_t=v_t, W_{t-1}=w_{t-1}}(w) = \bI_{\{w = w_{t-1} - \eta_t \nabla_w \emprisk(w_{t-1},s_{v_ t})\} }(w).$$ Therefore, the relative entropy $\relent(\bP_{W_t}^{S=s, V_t=v_t W_{t-1}=w_{t-1}}\Vert \bP_{W_t}^{V_t=v_t,W_{t-1}=w_{t-1}})$ cannot be evaluated. \citet{neu2021information} circumvented this problem by considering an artificial perturbed trajectory of the weights and analyzing the said trajectory. 

To be precise, let $W_T$ be the final iterate of SGD and consider some noise $\Xi_T \sim \cN(0, \sigma_{1:T}^2 I_d)$. Then, one may construct a noisy surrogate $\tilde{W}_T = W_T + \Xi_T$ and write
\begin{align}
    \bE \big[ &\gen(W_T,S) \big] \nonumber  \\
    &= \bE \big[ \gen(\tilde{W}_T, S) \big] + \bE \big[ \poprisk(W_T) - \poprisk(\tilde{W}_T) \big] + \bE \big[ \emprisk(\tilde{W}_T,S) + \emprisk(W_T,S) \big] \nonumber \\
    &= \bE \big[ \gen(\tilde{W}_T, S) \big] + \bE \big[ \Delta_{\sigma_{1:T}}(W_T,S') - \Delta_{\sigma_{1:T}}(W_T, S) \big],
    \label{eq:surrogate_decomposition}
\end{align}
where $\Delta_{\sigma}(w,s) = \bE[ \emprisk(w,s) - \emprisk(w+\Xi, s)]$ with $\Xi \sim \cN(0, \sigma I_d)$ is referred to as the \emph{ local value sensitivity } of the loss around $w \in \cW$ to perturbations at a level $\sigma$. This term evaluates how stable is the loss of the final iterate to perturbations. Intuitively, this term becomes small if SGD outputs a parameter vector in a flat area of the loss surface. This connects with the widely held belief  that algorithms that find ``wide
optima'' of the loss landscape generalize well~\citep{hochreiter1997flat,keskar2016large}, a hypothesis with contradicting theories~\citep{dinh2017sharp,izmailov2018averaging,he2019asymmetric,chaudhari2019entropy}.

The decomposition from~\eqref{eq:surrogate_decomposition} permits the analysis of the generalization error of the noisy version of the last iterate (or surrogate) $\tilde{W}_T$. Due to the fact that if $X$ and $Y$ are independent Gaussian random variables with variances $\sigma_x^2$ and $\sigma_y^2$, then $X+Y$ is also Gaussian with variance $\sigma_x^2 + \sigma_y^2$, we may construct an artificial iterative procedure to generate the said surrogate. Namely, the process starts by generating $\tilde{W}_0 = W_0$ and then, at each iteration $t \in [T]$, the surrogate parameters are updated with the rule
\begin{equation*}
    \tilde{W}_t = \tilde{W}_{t-1} - \eta_t \nabla_w \emprisk(W_{t-1}, S_{V_t}) + \Upsilon_t,
\end{equation*}
where $\Upsilon_t \sim \cN(0, \sigma_t^2)$. In this case, we can verify that $\tilde{W}_T = W_T + \Xi$ as desired, where $\sigma_{1:T}^2 = \sum_{t=1}^T \sigma_t^2$. In this way, \citet{neu2021information} adapted the techniques from~\citet{pensia2018generalization} to study the generalization error of the surrogate leading to the following result. Their result depends on the \emph{gradient sensitivity} of the iterations at each step, which they define as
\begin{equation*}
    \Gamma_\sigma(w) = \bE \mleft[ \mleft \lVert \bE[\nabla_w \ell(w,Z) ] - \bE^\Xi[\nabla_w \ell(w+\Xi, Z)] \mright \rVert^2 \mright]
\end{equation*}
and which measures the sensitivity of the gradients of the expected loss function to perturbations around $w$; and the \emph{gradient variance} of the loss of the iterations at each step, which they define as
\begin{equation*}
    V_t(w) = \bE^{W_t = w} \mleft[ \mleft \lVert  \nabla_w \emprisk(w, S_{V_t}) - \bE[\nabla_w \ell(w, Z) ] \mright \rVert^2 \mright].
\end{equation*}
\begin{theorem}[{\citet[Theorem 1, adapted to bounded losses]{neu2021information}}]
    \label{th:neu_sgd_bound}
    Consider a loss with a bounded range in $[a,b]$. For every sequence of positive numbers $(\sigma_t^2)_{t=1}^T$, let $\sigma_{1:t}^2 = \sum_{t'=1}^t \sigma_t^2$. Then, the generalization of SGD is bounded by
    \begin{align*}
        \bE \big[ \gen(W_T,S) \big] \leq &\sqrt{\frac{(b-a)^2}{n} \sum_{t=1}^T \frac{\eta_t^2}{\sigma_t^2} \bE \big[ \Gamma_{\sigma_{1:t}}(W_t) + V_t(W_t) \big]} \\
        &+ \bE \big[ \Delta_{\sigma_{1:T}}(W_T,S') - \Delta_{\sigma_{1:T}}(W_T, S) \big].
    \end{align*}
\end{theorem}

This result was celebrated because it provided the first information-theoretic analysis of SGD, and gave us insights into some of the elements that contributed to its generalization performance. For example, consider a fixed learning rate $\eta_t= \eta$, a fixed batch size $|V_t|=|v|$, and assume that the batches are chosen so that each sample $i$ is chosen in a batch exactly once. Consider a fixed noise parameter $\sigma = \sigma_t$. Assume that the gradient variance $\bE \big[\lVert \nabla_w \ell(w,Z) - \bE\big[\nabla_w \ell(w,Z) \big] \rVert^2 \big]  $ is bounded by $\nu$ and that the loss is globally $\mu$-smooth in the sense that $\lVert \nabla_w \ell(w,z) - \nabla_w(w+u,z) \rVert \leq \mu \lVert u \rVert$ for all $w,u \in \cW$ and all $z \in \cZ$. Then, \Cref{th:neu_sgd_bound} implies that
\begin{equation}
    \label{eq:sgd_neu}
    \bE \big[ \gen(W_T, S) \big] \in \cO \mleft( \frac{(b-a)^2 \eta^2 T}{n} \mleft(\mu^2 d T + \frac{\nu}{|v| \sigma^2} \mright) + \mu \sigma^2 d T \mright) .
\end{equation}

The rate becomes better as the number of iterations $T$ or the learning rate $\eta$ are decreased, which may decrease the training performance. Also, the rate gets better as $|v|$ increases until the term $\mu^2 d T$ dominates the expression. The noise parameter $\sigma$ has a complicated trade-off in the expression: smaller values improve the terms coming from the loss sensitivity, while larger values improve those coming from the variance to noise ratio $\nicefrac{\nu}{\sigma^2}$. Starting from~\eqref{eq:sgd_neu}, \citet{neu2021information} discussed the achievable rates for two different, common settings:
\begin{itemize}
    \item \textit{Small batch SGD}. In this setting, $T \in \cO(n)$ and $|v| \in \cO(1)$. The first observation they make is that choosing $\eta \in \cO(\nicefrac{1}{\sqrt{n}})$ could not guarantee a vanishing generalization error, even though it is known it does~\citep{hardt2016train,lei2020fine}. However, choosing $\eta \in \cO(\nicefrac{1}{n})$ and the noise $\sigma \in \cO(n^{-\nicefrac{4}{3}})$ still guarantees that $\bE[\gen(W,S)] \in \cO(n^{-\nicefrac{1}{3}})$.
    \item \textit{Large batch SGD}. In this setting, $T \in \cO(\sqrt{n})$ and $b \in \Omega(\sqrt{n})$. Choosing the learning rate $\eta \in \cO(\nicefrac{1}{\sqrt{n}})$ and the noise parameter $\sigma \in \Theta(\nicefrac{1}{\sqrt{n}})$ guarantees that $\bE[\gen(W,S)] \in \cO(\nicefrac{1}{\sqrt{n}})$.
\end{itemize}

Although the rates depend on the dimension $d$ and are not tight, especially for small-batch SGD, they show the promise that an information-theoretic analysis can shed light on the problem with more refined analyses. After that, further analyses of SGD appeared~\citep{wang2021generalization,wang2024sample}, improving the analysis but still not achieving tight rates. 

\section{Further Advances in Bounds Using Information Measures}
\label{sec:futher_advances_bounds_using_information_measures}

Both the mutual information $\minf(W;S)$ in the standard setting and the conditional mutual information $\minf(W;U|\tilde{S})$ in the randomized-subsample setting depend on the joint distribution of the algorithm's output and other variables. In contrast, the generalization error depends on the algorithm's output only through the losses it incurs. Therefore, it is possible to increase both these mutual information and conditional mutual information by \emph{embedding} information about the training set in the output of a learning algorithm without affecting the algorithm's statistical properties~\citep{livni2020limitation,bassily2018learners}. As a remedy, \citet{steinke2020reasoning} propose an alternative framework, the \emph{evaluated} conditional mutual information, that considers the information about the data captured by the \emph{incurred loss} rather than the output itself.

Consider the indices $U \in \{ 0, 1 \}^n$ and the supersample $\tilde{S} \in \cZ^{n \times 2}$ from the randomized-subsample setting of~\Cref{sec:bounds_using_conditional_mutual_information}. Then, define the loss vector $\mathbf{L} \in \bR^{n \times 2}$ as the array with entries $\mathbf{L}_{i,u} = \ell(W,\tilde{Z}_{i,u})$ for all $i \in [n]$ and all $u \in \{ 0, 1 \}^n$, where $W$ is the output of the algorithm. The \emph{evaluated conditional mutual information} is defined as the conditional mutual information of the loss vector $\mathbf{L}$ and the indices $U$ given the supersample $\tilde{S}$, that is $\minf(\mathbf{L};U|\tilde{S})$. Intuitively, this describes how much information does the output of the algorithm have about the identities of the samples used for training. Moreover, since for a given supersample $\tilde{S}$, the Markov chain $U - W - \mathbf{L} | \tilde{S}$ holds, by the data processing inequality from~\Cref{prop:properties_minf} we have that $\minf(\mathbf{L};U|\tilde{S}) \leq \minf(W;U|\tilde{S})$~\citep{steinke2020reasoning}.

\citet{steinke2020reasoning} proved both slow-rate and fast-rate upper bounds on the expected generalization error for losses with a bounded range based on this information measure. Later, \citet{haghifam2021towards} showed that these bounds can provide a sharp characterization of generalization in the realizable setting with the 0--1 loss. \citet{hellstrom2020generalization} extended these results proving bounds for the evaluated conditional mutual information similar to~\Cref{lemma:small_kl_mi}, and also proved that bounds featuring could guarantee the generalization of classes with a bounded Natarajan dimension $\ndim(\cW)$, thus recovering the classical results from~\Cref{subsec:vc_and_natarajan_dimensions}. That is, if a hypothesis class has a finite Natarajan dimension, then $$\minf(\mathbf{L};U|\tilde{S}) \leq \ndim(\cW) \log \mleft( \binom{|\cY|}{2} \frac{2en}{\ndim(\cW)} \mright) $$ and therefore the generalization error bounds vanish as the number of samples increases.
After that, in~\citep{haghifam2023limitations}, we extended the slow-rate bounds from~\citep{steinke2020reasoning} to Lipschitz losses.

\begin{theorem}
    \label{th:ecmi}
    Consider an $L$-Lipschitz loss function $\ell(\cdot, z)$ with respect to some metric $\rho$ for all $z \in \cZ$. Further consider a bounded space $\cW$ with $\diam_\rho(\cW) = B$. Then,
    \begin{equation*}
        \bE \big[ \gen(W,S) \big] \leq LB \sqrt{\frac{8 \minf(\mathbf{L};U|\tilde{S})}{n}}.
    \end{equation*}
\end{theorem}

\looseness=-1 The theorem follows, like most of the previous results, from the decoupling based on the Donsker and Varadhan~\Cref{lemma:dv_and_gvp} from~\Cref{lemma:decoupling_relent}. In this case, let $X = (\mathbf{L}, \tilde{S}, U)$ and $Y = (\mathbf{L}', \tilde{S}, U)$, where $\mathbf{L}'$ is a copy of $\mathbf{L}$ independent of $U$ such that $\bP_{\mathbf{L}',\tilde{S}, U} = \bP_{\mathbf{L},\tilde{S}} \otimes \bP_U$. Furthermore, let the function $f$ from the lemma be $$f(\mathbf{l}, \tilde{s}, u) = \frac{\lambda}{n} \sum_{i=1}^n \big( \mathbf{l}_{i,1-u} - \mathbf{l}_{i,u} \big).$$ Therefore, with these definitions we can see how $\bE[ f(\mathbf{L}, \tilde{S}, U) ] = \bE[ \gen(W,S)]$ and $\bE[ f(\mathbf{L}', \tilde{S}, U)] = 0$. Hence, according to~\Cref{lemma:decoupling_relent} we have that
\begin{equation*}
    \bE \big[ \gen(W,S) \big] \leq \frac{1}{\lambda} \mleft( \minf(\mathbf{L},\tilde{S};U) - \log \bE \mleft[ e^{\frac{\lambda}{n} \sum_{i=1}^n \big( \mathbf{L}'_{i,1-U_i} - \mathbf{L}'_{i,U_i}\big)}\mright]\mright),
\end{equation*}
where we noted that $\relent(\bP_{\mathbf{L},\tilde{S}, U} \lVert \bP_{\mathbf{L},\tilde{S}} \otimes \bP_U) = \minf(\mathbf{L},\tilde{S}; U)$ and where we wrote the CGF explicitly.

\looseness=-1 Then, we can see how by the independence of the supersamples and the indices and the \emph{chain rules} of the mutual information $\minf(\mathbf{L},\tilde{S}; U) = \minf(\tilde{S};U) + \minf(\mathbf{L};U|\tilde{S}) = \minf(\mathbf{L};U|\tilde{S})$. 
Finally, the theorem follows by noting how, for all $i \in [n]$, the random variables $\mathbf{L}'_{i,1-U_i} - \mathbf{L}'_{i,U_i}$ are independent of each other and have a range bounded in $[-2LB,2LB]$. Since they have a bounded range, we may note that they are also $2(LB)^2$-sub-Gaussian and then we may proceed like we did in~\Cref{th:mi_bound_bounded,th:cmi_bound_bounded}.

To prove that these random variables have a bounded range, note that we may re-write the random variables as
\begin{align*}
    \mathbf{L}'_{i,1-U_i} - \mathbf{L}'_{i,U_i} &= \ell(W', \tilde{Z}_{i,1-U_i}) - \ell(W', \tilde{Z}_{i,U_i}) \\
    &= \big( \ell(W', \tilde{Z}_{i,1-U_i}) - \ell(w, \tilde{Z}_{i,1-U_i})\big) + \big(\ell(W', \tilde{Z}_{i,U_i}) -  \ell(W', \tilde{Z}_{i,U_i})\big),
\end{align*}
where $W'$ is a random variable such that, for a fixed supersample $\tilde{S}=s$ and some fixed indices $U=u$ is distributed according to the distribution $\bP_W'^{\tilde{S}=\tilde{s}, U=u} = \bP_W^{\tilde{S}=\tilde{s}}$, and where $w$ is any fixed hypothesis. Then, since the loss is $L$-Lipschitz and the hypothesis space has diameter $B$, by the triangle inequality
\begin{equation*}
    |\mathbf{L}'_{i,1-U_i} - \mathbf{L}'_{i,U_i}| \leq 2L\rho(W',w) \leq 2 L B,
\end{equation*}
and therefore $\mathbf{L}'_{i,1-U_i} - \mathbf{L}'_{i,U_i} \in [-2LB, 2LB]$. 

Further developments came later, when~\citet{Wang2023TighterIG} combined the ideas from the evaluated conditional mutual information from~\citep{steinke2020reasoning} above, with the single-letter bounds from~\citep{bu2020tightening,rodriguez2020randomsubset} of~\Cref{subsec:single_letter_bounds}, and the Wasserstein distance bounds from~\citep{wang2019information,rodriguez2021tighter} of~\Cref{sec:bounds_using_wasserstein_distance} to find bounds based on the \emph{loss difference}. To be precise, they considered the loss vector $\mathbf{L}$ from above and defined the loss difference $\Delta \mathbf{L}_i$ as the difference of the losses on the $i$-th coordinate of the loss vector, that is, $\Delta \mathbf{L}_i \coloneqq \mathbf{L}_{i,0} - \mathbf{L}_{i,1}$. In this way, they proved generalization bounds based on the average of the mutual information between each loss difference $\Delta \mathbf{L}_i$ and the index $U_i$ given the supersamples. That is, they showed that
\begin{equation*}
    \bE \big[ \gen(W,S) \big] \leq \frac{1}{n} \sum_{i=1}^n \minf(\Delta \mathbf{L}_i; U_i | \tilde{S}),
\end{equation*}
among other similar bounds. Again, since given the supersample $\tilde{S}$, the Markov chain $U_i - (\mathbf{L}_{i,0}, \mathbf{L}_{i,1}) - \Delta \mathbf{L}_i | \tilde{S}$ holds, this information measure is smaller than the previous ones, rendering the bounds tighter. Using these bounds, \citet{Wang2023TighterIG} could also relate a notion of ``sharpness'' to generalization, finding similar results to those specific to SGD found by~\citet{neu2021information} and~\citet{wang2021generalization} that we saw in~\Cref{subsec:beyond_sgld}.

Another advance from~\citet{wang2024sample} came by combining stability notions similar to those in~\Cref{subsec:uniform_stability} with the mutual information bounds that we have seen throughout this chapter. Simplifying their results, they obtained bonds of the type $\gamma \cdot \mathrm{IM}$, where $\gamma$ represented a stability parameter similar to those in~\Cref{subsec:uniform_stability}, and $\mathrm{IM}$ represents an information measure. In this way, they unified both frameworks, when the uniform stability framework yields a valid generalization for a problem, the parameter $\gamma$ ensures it as $\mathrm{IM} \leq c$ for some $c$; while when the information-theoretic generalization framework yields a valid generalization for a problem, then $\mathrm{IM}$ ensures a vanishing rate. The only problem with this framework is that it is complicated, which considers $n+1$ hypotheses generated with different combinations of the supersample and uses them to define both the stability parameter and the information measure.

\begin{remark}
\label{rem:functional-cmi}
A similar approach was taken by~\citet{harutyunyan2021information}, where instead of the evaluated conditional mutual information, they considered the \emph{functional conditional mutual information}. Consider a supervised learning setting where $\cZ = \cX \times \cY$ and a parameterized model $f_w : \cX \to \cY$. Then, as discussed in~\Cref{sec:formal_model_learning}, the loss function may be written as $\ell(w,z) = \rho(f_w(x),z)$ for some measure of dissimilarity between the predicted $f_w(x)$ and the true $y$ target associated with the feature $x$. In this setup, for an output hypothesis $W$, \citet{harutyunyan2021information} considered the \emph{functional vector} to be the random variable $\mathbf{F} \in \cY^{n \times 2}$ with entries $\mathbf{F}_{i,u} = f_W(\tilde{X}_{i,u})$ for all $i \in [n]$ and all $u \in \{0,1\}$, where $\tilde{X}_{i,u}$ is the feature of the instance $\tilde{Z}_{i,u}$. Note that given a supersample $\tilde{S}$, the Markov chain $U - \mathbf{F} - \mathbf{L} | \tilde{S}$, and therefore the functional conditional mutual information is looser than the evaluated mutual information and every generalization error bound using the latter can also be employed with the former.
\end{remark}

Finally, a different line of work considered the chaining technique to bound the suprema of random processes~\citep[Chapter 5]{van2014probability} to bound the generalization error~\citep{asadi2018chaining,asadi2020chaining}. In this framework, the authors consider that $\{ \gen(w,S) \}_{w \in \cW}$ is a separable, sub-Gaussian process in the metric space $(\cW, \rho)$. Similarly to the separability of a space from~\Cref{subsec:polish_borel_spaces}, the separability assumption in a process requires that for every $w \in \cW$, there is a countable subset $\cW_0 \subseteq W$ such that $\gen(w,S) = \lim_{w' \in \cW_0 \to w} \gen(w'S)$ almost surely. The sub-Gaussian assumption is a statistical relaxation of the Lipschitzness assumption and requires that the generalization error does not vary much for close hypotheses, that is, that for all $\lambda > 0$ and all $w, w' \in \cW$ the following holds:
\begin{equation*}
    \log \bE \mleft[ e^{\lambda\big(\gen(w,S) - \gen(w',S)\big)} \mright] \leq \frac{\lambda^2 \rho(w,w')^2}{2}.
\end{equation*}

Before showing their results, let us recall what an $\varepsilon$-net is: A set $\cN$ is called an \emph{$\varepsilon$-net} for the metric space $(\cW,\rho)$ if, for every $w \in \cW$, there is a map $\pi(w) \in \cN$ such that $\rho(w,\pi(w)) \leq \varepsilon$.

With these assumptions, \citet{asadi2018chaining} showed that the expected generalization error is bounded from above by
\begin{equation}
    \label{eq:chaining_minf}
    \bE \big[ \gen(W,S) \big] \leq 3 \sqrt{2} \sum_{k = k_0 + 1}^\infty 2^{-k} \sqrt{\minf(W_k;S)},
\end{equation}
where $W_k$ are finer and finer quantizations of the hypothesis such that $W_k = \pi_k(W)$. To be precise, the quantizations were built by considering a series of $2^{-k}$-nets of the hypothesis space with an associated mapping $\pi_k : \cW \to \cN_k$ restricted to the form $\pi_k = \pi_k' \circ \pi_{k+1}$, where $\pi_k' : \cN_{k+1} \to \cN_k$. The initial net $\cN_{k_0}$ is constructed such that $2^{-k_0} \geq \diam_\rho(\cW)$ and therefore the mutual information for that coarse quantization is zero, that is $\minf(W_{k_0};S) = 0$.

With these developments, the intention was to simultaneously exploit the geometrical dependencies of the hypothesis and the statistical dependencies between the hypothesis and the training data, similarly to our results from~\Cref{sec:bounds_using_wasserstein_distance}. Although the bound from~\eqref{eq:chaining_minf} eludes the problems of the absolute continuity from the original mutual information bound from~\citet{xu2017information} in~\Cref{th:mi_bound_bounded}, its complicated nature makes it difficult to apply compared to the other, posterior alternatives presented above.

\section{Limitations of the Bounds Using Mutual Information}
\label{sec:limitations_bounds_using_mutual_information}

As discussed in~\Cref{sec:information_theoretic_generalization} and seen throughout this monograph, information-theoretic bounds are, by their nature, distribution- and algorithm-dependent. As seen in~\Cref{sec:noisy_iterative_learning_algos}, these key properties enable the bounds from this framework to achieve numerically non-vacuous generalization guarantees for SGLD with modern deep-learning tasks and architectures. For SGD, they can also obtain non-vacuous guarantees after employing a surrogate. Moreover, as shown in~\Cref{th:minf_bounds_tight} from~\Cref{subsec:implications_bounds_using_mutual_information}, these bounds are shown to be tight for Lipschitz losses and bounded hypothesis spaces, in the sense that they cannot be improved \emph{simultaneously} for all algorithms. It is natural to wonder, then, if these techniques can ever achieve the minimax rates. That is, if they can guarantee the generalization of whole hypothesis classes.

\citet{bassily2018learners} showed that the mutual information bound from~\citet{xu2017information} of~\Cref{th:mi_bound_bounded} provably fail to characterize the learnability of hypothesis classes with a finite VC dimension, even though they are known to generalize as seen in~\Cref{subsec:vc_and_natarajan_dimensions}. However, as discussed above, \citet{haghifam2021towards} and \citet{hellstrom2022new}, showed, respectively, that the evaluated conditional mutual information from~\Cref{sec:futher_advances_bounds_using_information_measures} (i) achieves the minimax rate for binary classification in the realizable setting, and that (ii) it recovers the generalization guarantees for hypothesis classes with a bounded Natarajan dimension from~\Cref{subsec:vc_and_natarajan_dimensions} .

In light of these exciting advancements, in \citep{haghifam2023limitations}, we studied if the bounds based on mutual information and conditional mutual information-based bounds from~\Cref{subsec:implications_bounds_using_mutual_information,sec:futher_advances_bounds_using_information_measures}, despite being tight for non trivial problems, cannot achieve the minimax rates for convex, Lipschitz losses with a bounded parameter space. To do so, we considered the gradient descent (GD) algorithm, since it is known to generalize in this setting at a rate in $\cO(\nicefrac{LB}{\sqrt{n}})$. Then, we showed that there is a convex, Lipschitz loss and a data distribution for which all of these bounds have a rate in $\Omega(1)$, proving that they cannot achieve the desired minimax rate. Furthermore, we also prove that considering a Gaussian surrogate as in~\eqref{eq:surrogate_decomposition} from~\Cref{subsec:beyond_sgld} still fails, casting doubt on this approach to enhance this framework.

After this result, \citet{livni2023information} and~\citet{attias2024information} continued this line of research proving a stronger statement. Namely, they show that there is a convex, Lipschitz function such that, for each of the mutual information and conditional mutual information bounds from~\Cref{subsec:implications_bounds_using_mutual_information}, for every algorithm that generalizes, there exists a data distribution under which the bounds are loose. In other words, for every algorithm that generalizes there exists a data distribution for which the existing bounds will fail. These results are based on the ``fingerprinting lemma''~\citep{bun2014fingerprinting,steinke2016upper,kamath2019privately} from the differential privacy community and differ from our proof techniques.

Our proofs for the failure of GD with Gaussian surrogates are quite technical and do not provide further insights, except from the result itself, which is now a corollary of the results in \citep{livni2023information,attias2024information}. Therefore, in light of these new developments, we restrict the rest of the section to our proofs of the failure of these bounds to characterize the generalization of GD. Although some of these results are also included in the theorems from~\citep{livni2023information,attias2024information}, the construction of the counterexample is still useful. First, it showcases how the information of the data can be leaked into the algorithm. This pathological construction can be useful for future privacy research. Moreover, our negative results for GD also extend to the evaluated (and therefore also to the functional) conditional mutual information bound from~\Cref{th:ecmi}, which is, at the moment, still not covered by the fingerprinting results from~\citep{livni2023information,attias2024information}. 

\begin{remark}
\label{rem:open_question_limitation}
An open question to both our results and those from~\citep{livni2020limitation,attias2024information} is if the new bounds from~\citet{Wang2023TighterIG}, which include a stability parameter $\gamma$, have counter-examples that make them vacuous in non-trivial settings. These bounds escape the presented counter-example since GD can be shown to have a vanishing stability parameter with techniques similar to those shown above in~\Cref{subsec:gd_in_clb_setting}, and therefore the bound in~\citep{wang2019information} generalize in this situation.
\end{remark}

\subsection{Gradient Descent in the Convex-Lipschitz-Bounded Setting}
\label{subsec:gd_in_clb_setting}

The field of stochastic convex optimization (SCO)~\citep{shalev2009stochastic} studies the problem of minimizing the population risk, that is 
\begin{equation*}
    \argmin_{w \in \cW} \poprisk(w),
\end{equation*}
where either the loss $\ell(\cdot, z)$ or the population risk itself $\poprisk$ are assumed to be convex. In this field, researchers try to design algorithms that return hypotheses $w \in \cW$ that achieve a small excess risk on large classes of problems. 

A common class of problems are those within the so-called convex-Lipschitz-bounded (CLB) setting. This setting considers loss functions $\ell(\cdot, z)$ that are convex and $L$-Lipschitz for all $z \in \cZ$ and that have a convex, bounded domain, that is, where the parameter space has a bounded diameter under some metric $\diam_\rho(\cW) = B$. For simplicity, we will consider the space to be a subset of the $d$-dimensional reals $\cW \subseteq \bR^d$ and the metric to be the Euclidean distance $\rho = \lVert \cdot \rVert_2$. A problem in this class is often described as the triple $\SCOprob$ and the set of all these triples is $\clb$.

In this setting, the (projected) GD is guaranteed to generalize and to have a small excess risk. This algorithm has been studied for a long time~\citep{cauchy1847methode,bubeck2015convex}. The algorithm is equivalent to the standard GD, with an additional projection operator at every step to ensure that the iterates are within the bounded parameter space $\cW$. Namely, let $\Pi_\cW(w) = \min_{w' \in \cW} \lVert w - w' \rVert_2$ be the projection operator from $\bR^d$ to $\cW$. Then, GD starts initializing the parameters $W_0$ and, at each time step $t \in [T]$, it updates the parameters following the rule
\begin{equation}
    \label{eq:gd_update}
    W_t = \Pi_\cW \mleft( W_{t-1} - \eta_t \nabla_w \emprisk(W_{t-1}, s) \mright) 
\end{equation}
for some learning rate $\eta_t$ and some training set $s \in \cZ^n$.

For simplicity, we restrict the discussion to GD with a constant learning rate $\eta_t = \eta$ for all iterations $t \in [T]$. In \citep{lastiterate}, the optimization error of the final iterate of GD in the CLB setting is shown to satisfy
\begin{equation}
\label{eq:opt-error-gd}
    \sup_{\SCOprob \in \clb} \sup_{\bP_Z \in \cP(\cZ)}\bE \big[ \emprisk(W_T, S) - \emprisk(W^\star_{S}\big) \big] \leq \frac{B^2}{2\eta T}+ \frac{\eta L^2 (2 + \log T)}{2}.
\end{equation}
(See \Cref{lem:gd-last-iterate} for a re-statement of this result in the context of this monograph). A similar result also appears in \citep[Theorem 5.3]{zhang2004solving}. Recently, \citet[Theorem 3.2 and Section 3.4]{bassily2020stability} proved a generalization bound for GD, 
\begin{equation}
    \label{eq:gen-error-gd}
    \sup_{\SCOprob \in \clb} \sup_{\bP_Z \in \cP(\cZ)} \bE \big[ \gen(W_T, S) \big] \leq  4L^2 \sqrt{T} \eta + \frac{4L^2T\eta}{n}.
\end{equation}
Recalling the excess risk decomposition from~\eqref{eq:excess_risk_decomposition}, we observe that combining \eqref{eq:opt-error-gd} and \eqref{eq:gen-error-gd} yields the following bound
\begin{align}
    \sup_{\SCOprob \in \clb} \sup_{\bP_Z \in \cP(\cZ)} \bE\big[ &\excess(W_T, \cW) \big] \leq \nonumber \\
    &\min  4L^2 \eta\left(\sqrt{T} +  \frac{T}{n}\right) + \frac{B^2}{2\eta T}+ \frac{(2 + \log T)\eta L^2}{2}.
\label{eq:excess-risk-gd}
\end{align}
For all $\alpha \geq 2$, \Cref{eq:excess-risk-gd} guarantees that GD achieves an excess risk in $\cO \mleft( \nicefrac{LB}{\sqrt{n}} \mright)$ for a number of iterations $T \in \Theta \mleft( n^\alpha \mright)$ and a step-size $\eta \in \Theta \mleft( \nicefrac{B\sqrt{n}}{Ln^\alpha} \mright)$. 
This, in fact, is the best achievable excess risk rate for the class $\clb$ in the distribution-free setting \citep{bubeck2015convex}, that is, the best rate that holds \emph{uniformly} for all distributions $\bP_Z$. In \citep{amir2021sgd,sekhari2021sgd}, it is shown that GD cannot attain this excess risk rate when the number of iterations satisfies $T \in o\mleft(n^2\mright)$. 

To establish our negative results, what matters is that choosing the number of iterations to be $T = n^2$ and the learning rate to be $\eta = \nicefrac{B}{(8L n \sqrt{n})}$ ensures that the expected generalization error is bounded as $\bE \big[ \gen(W_T, S) \big] \leq \nicefrac{LB}{\sqrt{n}}$ for all CLB problems and all data distributions.

\subsection{Failure of the Bounds Using Mutual Information}
\label{subsec:failure_bounds_using_minf}

In~\Cref{subsec:comparison_of_the_bounds}, we noted how the conditional mutual information measures from the randomized-subsample setting are smaller than the mutual information measures from the standard setting, and how the single-letter conditional mutual information bounds are tighter than the random-subset bounds or the bounds considering the full training set. Moreover, when describing the evaluated conditional mutual information in~\Cref{sec:futher_advances_bounds_using_information_measures}, we saw that it is tighter than the conditional mutual information bound from~\Cref{cor:total_variation_relent_full_rs}. Therefore, proving the failure of the single-letter conditional mutual information bound from~\Cref{cor:total_variation_relent_single_letter_rs} and the evaluated conditional mutual information from~\Cref{th:ecmi} also proves the failure of the mutual information versions of~\Cref{cor:total_variation_relent_full,cor:total_variation_relent_random_subset,cor:total_variation_relent_single_letter,cor:total_variation_relent_full_rs,cor:total_variation_relent_random_subset_rs}. 

\begin{theorem}
    \label{th:impossibility_minf}    
    
    Let $T = n^2$ and $\eta = \nicefrac{1}{n \sqrt{n}}$. There is a CLB problem $\SCOprob \in \cC_{1,2}$ and a data distribution $\bP_Z$ for which the information-theoretic bounds from~\Cref{cor:total_variation_relent_single_letter_rs} and~\Cref{th:ecmi} do not generalize. More precisely, let $W_T$ be the final iterate of gradient descent, then $\bE[\gen(W_T,S)] \in \cO(\nicefrac{1}{\sqrt{n}})$ while
    \begin{equation*}
        \frac{1}{n} \sum_{i=1}^n \minf(W_T;U|\tilde{Z}_{i,0},\tilde{Z}_{i,1}) \in \Omega(1) \ \textnormal{ and } \  \minf(\textbf{L};U|\tilde{S}) \in \Omega(1).
    \end{equation*}
\end{theorem}

Although the problem is stated for a problem in $\cC_{1,2}$, just scaling the loss $\ell$ and the parameter space $\cW$ directly converts it into a problem in $\clb$. Therefore, \Cref{th:impossibility_minf} maintains its full generality. Moreover, the generalization rate of $\cO(\nicefrac{1}{\sqrt{n}})$ is guaranteed by the choice of the number of iterations and the learning rate as discussed above in~\Cref{subsec:gd_in_clb_setting}.

The proof of~\Cref{th:impossibility_minf} is constructive. First, we construct the instance space as the set of all coordinate vectors in $\{ 0, 1 \}^d$ for some dimension $d \in \bN$. That is, $\cZ = \{ \coorvec{k} : k \in [d]\}$, where the coordinate vectors are defined as $$\coorvec{k} \coloneqq (\underbrace{0, \ldots, 0}_{k-1 \textnormal{ times}}, 1, \underbrace{0, \ldots, 0}_{d-k \textnormal{ times}}).$$ Then, we consider the data distribution on this space to be uniform, that is $\bP_Z[\coorvec{k}] = \nicefrac{1}{d}$ for all $k \in [d]$. As a loss function, we consider the convex, $1$-Lipschitz function $\ell(w,z) = - \langle w, z \rangle$, and as a parameter's space, the Euclidean unit ball in $\bR^d$, that is, $\cW = \{ w \in \bR^d : \lVert w \rVert_2 \}$.  Therefore, the problem $\SCOprob$ belongs to the CLB class $\cC_{1,2}$. Throughout this example, $w(k)$ denotes the $k$-th coordinate of $w \in \cW$.

Under this construction, the dynamics of GD are simple to analyze. Note that, for some fixed hypothesis $w$, the empirical risk is $\emprisk(w) = - \langle w, \bar{Z} \rangle$, where $\bar{Z} = \frac{1}{n} \sum_{i=1}^n Z_i$ is the mean of the instances of the training set. Since the initial point of GD is known, without loss of generality (as we will see shortly), assume that it is $0$. Moreover, the gradient of the empirical risk is $\nabla_w \emprisk(w) = - \bar{Z}$ for all $w \in \cW$. Considering the update rule of GD from~\eqref{eq:gd_update}, we see from induction that
\begin{equation}
    \label{eq:dynamic-gd-linear}
    W_t = 
    \begin{cases}
        \eta t \bar{Z}  & \eta t \lVert \bar{Z} \rVert_2 \leq 1 \\
        \nicefrac{\bar{Z}}{\lVert \bar{Z} \rVert_2} &\text{ otherwise}
    \end{cases}.
\end{equation}

Now, consider the $\sigma(\tilde{S})$-measurable random variable $E$ that is equal to one if and only if all the data instances in the supersample are distinct. That is
\begin{equation}
\label{eq:event-distinct}
    E = \bI_{\mleft \{\tilde{z}_{i,u} \neq \tilde{z}_{j,v} \textnormal{ for all } i,j \in [n] \textnormal{ and all } u,v \in \{ 0, 1 \} \mright \}}(\tilde{S}).
\end{equation}
As in the \emph{birthday paradox problem}~\citep[Section 5]{mitzenmacher2017probability}, we may bound the probability that $E=1$ as follows
\begin{equation}
    \bP[E=1] = \prod_{k = 0}^{2n -1} \Big(1 - \frac{k}{d} \Big) \geq \Big(1 - \frac{2n-1}{d} \Big)^{2n-1}.
    \label{eq:birthday_dim}
\end{equation}
This way, we may engineer a dimension $d$ for which $\bP[E=1] \geq c$ for all $n \geq 1$, where $c$ is a constant probability, independent of $n$. Solving for~\eqref{eq:birthday_dim} results in
\begin{equation*}
    d \geq \frac{2n - 1}{1 - c^{\nicefrac{1}{(2n-1)}}}.
\end{equation*}
For instance, for $c = 0.1$, a dimension $d = 2n^2$ suffices, and therefore $\bP[E=0] \leq 0.9$.
Now, we are ready to study what happens to both the single-letter conditional mutual information $\minf(W_T;U_i|\tilde{Z}_{i,0},\tilde{Z}_{i,1})$ and the evaluated conditional mutual information $\minf(\textbf{L},U|\tilde{S})$ for this particular example.

The idea behind the proofs is that, given the event $E=1$, we can completely determine the indices of the samples used for training either by looking (i) at the non-zero coordinates of the parameters $W_T$ returned by the GD, since if $W_T(k) \neq 0$ and $\tilde{Z}_{i,u} = \coorvec{k}$, then $U_i = u$; or by looking (ii) at the non-zero entries of the loss vector $\mathbf{L}$, since if $\mathbf{L}_{i,u} \neq 0$, then $U_i = u$.  Therefore, given the event $E = 1$, the single-letter conditional mutual information and the evaluated conditional mutual information are maximized. Finally, the fact that the event $E=1$ occurs with constant probability ensures that neither these information measures nor their respective generalization bounds decrease with respect to the number of samples.

Finally, let us show how the choice of the known initial parameter $w_0$ is not an issue for the analysis and can be assumed to be $0$ without loss of generality. After the GD update rule from~\eqref{eq:gd_update}, from induction we see that
\begin{equation*}
    W_t = 
    \begin{cases}
        w_0 + \eta t \bar{Z}  & \lVert  \eta t  \bar{Z} \rVert_2 \leq 1 \\
        \alpha w_0 + \beta \bar{Z} &\text{ otherwise}
    \end{cases},
\end{equation*}
where $\alpha =  \nicefrac{1}{\lVert w_0 + \eta t \tilde{Z} \rVert_2}$ and $\beta = \nicefrac{\eta t}{\lVert w_0 + \eta t \tilde{Z} \rVert_2}$. Then, if $\lVert W_t \rVert_2 \leq 1$, one may just subtract $w_0$ and extract the same conclusions. On the other hand, if $\lVert W_t \rVert_2 > 1$, one may find the unique vector $\bar{Z}$ such that $\Pi_\cW(w_0 + \eta t \bar{Z}) = W_t$, calculate $\alpha$, subtract $w_0$, and again extract the same conclusions.

\subsubsection{Failure of the Single-Letter Conditional Mutual Information}

Note that the single-letter conditional mutual information may be written as follows
\begin{align}
    \minf(W_T;U_i|\tilde{Z}_{i,0},\tilde{Z}_{i,1})  &= \ent( U_i \vert \tilde{Z}_{i,0},\tilde{Z}_{i,1}) - \ent (U_i \vert W_T,\tilde{Z}_{i,0},\tilde{Z}_{i,1}) \nonumber \\
    &= \ent(U_i) - \ent(U_i \vert W_T,\tilde{Z}_{i,0},\tilde{Z}_{i,1})\nonumber \\
    &= \log 2  - \ent( U_i \vert W_T,\tilde{Z}_{i,0},\tilde{Z}_{i,1}), \label{eq:birthday_icmi_bound}
\end{align}
where the second and third equations follow from the independence of the indices $U_i$ of the data instances $\tilde{Z}_{i,0}$ and $\tilde{Z}_{i,1}$ and from the fact that $\ent(U_i) = \log 2$, respectively. 
Then, we may use Fano's inequality to bound $\ent(U_i \vert W_T,\tilde{Z}_{i,0},\tilde{Z}_{i,1})$ and obtain the desired result. More precisely, Fano's inequality states that 
    \begin{align*}
        \ent(U_i \vert W_T,\tilde{Z}_{i,0},\tilde{Z}_{i,1})
        &\leq \entber \big(\bP [U_i \neq \hat{U}_i] \big),
        \label{eq:birthday_ent_icmi_bound}
    \end{align*}
    for every estimator $\hat{U}_i(W_T, \tilde{Z}_{i,0}, \tilde{Z}_{i,1})$ and where $\entber$ is the binary entropy~\citep[Theorem 3.12]{polianskyi2022}. The binary entropy is defined as $\entber(p) \coloneqq -\log p - (1-p) \log (1-p)$ for all $p \in [0,1]$ and it is maximized at $p = 0.5$ for which $\entber(0.5) = \log 2$.  Notice that $\hat{U}_i$ is a function of $W_T, \tilde{Z}_{i,0},$ and $\tilde{Z}_{i,1}$. Therefore, showing that $\bP[U_i \neq \hat{U}_i] < 0.5$ is a constant, independent of $n$ ensures that
    \begin{equation}       
        \minf(W_T;U_i|\tilde{Z}_{i,0},\tilde{Z}_{i,1})  \geq \log 2  - \entber \mleft( \bP[U_i \neq \hat{U}_i] \mright) \in \Omega(1)
        \label{eq:icmi_birthday_bounded}
    \end{equation}
    and completes the proof.

From~\eqref{eq:dynamic-gd-linear}, we can see that the non-zero coordinates of the last iterate $W_T$ are precisely the coordinates of the training samples. That is, if $\tilde{Z}_{i,U_i} = \coorvec{k}$, then $W_T(k) \neq 0$. 
Therefore, under the event $E = 1$ defined in \eqref{eq:event-distinct}, one can precisely determine if sample $\tilde{Z}_{i,0}$ or sample $\tilde{Z}_{i,1}$ was used for training after observing the parameters $W_T$ returned by GD since the samples are all distinct. 
In other words, one can completely determine the index $U_i$ from the triple $(W_T, \tilde{Z}_{i,0}, \tilde{Z}_{i,1})$. 

More precisely, consider a realization in which $\tilde{Z}_{i,0} = \coorvec{k}$ and $\tilde{Z}_{i,1} = \coorvec{l}$. Then, the estimator $\hat{U}_i(W_T, \tilde{Z}_{i,0}, \tilde{Z}_{i,1})$ is defined  as $\hat{U}_i(W_T, \tilde{Z}_{i,0}, \tilde{Z}_{i,1}) = 0$ if $W_T(k) \neq 0$ and $W_T(l) = 0$; and as $\hat{U}_i(W_T, \tilde{Z}_{i,0}, \tilde{Z}_{i,1}) = 1$ if $W_T(k) = 0$ and $W_T(l) \neq 0$. In the case that both $W_T(k) \neq 0$ and $W_T(l) \neq 0$, let  $\hat{U}_i(W_T, \tilde{Z}_{i,0}, \tilde{Z}_{i,1})$ be a Bernoulli random variable with parameter $\nicefrac{1}{2}$ independent of the supersample $\tilde{S}$ and the indices $U$. This estimator has a probability of error equal to $0$ given the event $E=1$. Therefore, by the law of total probability, the probability of error is
\begin{align*}
    \bP [U_i \neq \hat{U}_i] &= \bP[E = 0] \cdot \bP^{E=0}[U_i \neq \hat{U}_i] + \bP[E = 1] \cdot  \bP^{E=1}[U_i \neq \hat{U}_i] \\
    &= \bP[E = 0] \cdot \bP^{E=0}\bP[U_i \neq \hat{U}_i] \\
    &\leq 0.9 \cdot \bP^{E=0}[U_i \neq \hat{U}_i],
\end{align*}
where the last line follows from the construction. Next, consider the following random variables for all $i \in [n]$
\begin{equation*}
    G_i = \bI_{\mleft\{ \tilde{z}_{i,0} \neq \tilde{z}_{i,1} \textnormal{ and } \tilde{z}_{i,1-u_i} \neq \tilde{z}_{i,u_j} \textnormal{ for all } j \neq i \in [n] \mright\}} (\tilde{S}, U)
\end{equation*}
which describe the situation where the given samples $\tilde{Z}_{i,0}$ and $\tilde{Z}_{i,1}$ are distinct and the sample that is not chosen is also distinct from all other samples in the training set $S$, even when some of these samples are equal between themselves or to the chosen sample $\tilde{Z}_{i,U_i}$ (for example, when $E=0$). Therefore, given the event $E=0$ and $G_i = 1$, the estimator $\hat{U}_i$ still has a probability of error equal to zero. Hence, similarly to before we may bound the probability of error of the estimator as
\begin{align*}
    \bP[U_i \neq \hat{U}_i] &\leq 0.9 \cdot \Big( \bP^{E=0} [G_i = 0] \cdot \bP^{E = 0, G_i = 0}[U_i \neq \hat{U}_i] \\
    &\qquad \qquad \qquad + \bP^{E = 0}[G_i = 1] \cdot \bP^{E = 0, G_i = 1}[U_i \neq \hat{U}_i]\Big) \\
    &\leq 0.9 \cdot \bP^{E = 0, G_i = 0}[U_i \neq \hat{U}_i].
\end{align*}

Next, we claim that under the event where $E=0$ and $G_i=0$, the estimator $\hat{U}_i$ is a Bernoulli random variable with parameter $\nicefrac{1}{2}$. Consider a realization in which $U_i = u$, $\tilde{Z}_{i,0} = \coorvec{k}$, and $\tilde{Z}_{i,1} = \coorvec{l}$. Then, we claim that under the event $E=0$ and $G_i=0$, $W_T(k) \neq 0$ and $W_T(l) \neq 0$. The reason is that, under this event, the following cases may happen: either (i) $\tilde{Z}_{i,0} = \tilde{Z}_{i,1}$, or (ii) $\tilde{Z}_{i,0} \neq \tilde{Z}_{i,1}$ but there exists another sample in the training set which is equal to $\tilde{Z}_{i,1-u}$. It is easy to see that in these two cases $W_T(k) \neq 0$ and $W_T(l)\neq 0$.

Thus, we conclude that, given $E=0$ and $G_i=0$, we have that $\bP^{E=0,G_i=0}[U_i \neq \hat{U}_i] = \nicefrac{1}{2}$. This is true since $\hat{U}_i$ is a Bernoulli random variable with parameter $\nicefrac{1}{2}$ independent of the indices $U$ and the supersample $\tilde{S}$. Therefore, we have that $\bP[U_i \neq \hat{U}_i] \leq 0.45$, which completes the proof as per \eqref{eq:icmi_birthday_bounded}. In particular, we have that $\minf(W;U_i | \tilde{Z}_{i,0}, \tilde{Z}_{i,1}) \geq 0.005$.

\subsubsection{Failure of the Evaluated Conditional Mutual Information}

Similarly to before, note that the evaluated conditional mutual information may be written as
\begin{align}
    \minf(\mathbf{L}; U | \tilde{S})
    &= \ent( U | \tilde{S}) - \ent(U | \mathbf{L}, \tilde{S}) \\
    &= \ent(U) - \ent(U | \mathbf{L}, \tilde{S}) \\
    &= n \log 2 - \ent(U | \mathbf{L}, \tilde{S}), \label{eq:birthday_ecmi_bound}
\end{align}
where the second and third equations follow from the independence of the indices $U$ and the supersample $S$ and the fact that $\ent(U) = n \log 2$, respectively.

Then, as previously, the proof relies upon the fact that $U$ can be completely determined by the loss vector $\mathbf{L}$ given the event $E = 1$. More precisely, note that $\mathbf{L}_{i,u} = \ell(W_T, \tilde{Z}_{i,u}) = - \langle W_T, \tilde{Z}_{i,u} \rangle$. Also, remember from the above that the non-zero coordinates of $W_T$ are precisely the non-zero coordinates of the samples that are used for training. Therefore, given the event $E=1$, $\mathbf{L}_{i,u} = 0$ if and only if $\tilde{Z}_{i,u}$ was not used for training and therefore the training instance is $Z_i = \tilde{Z}_{i,1-u}$.  Hence, one can completely determine $U$ from $\mathbf{L}$ or, equivalently, $\ent^{E=1}(U| \mathbf{L}, \tilde{S}) = 0$. We may use this fact to bound $\ent(U \vert \mathbf{L}, \tilde{S})$ and obtain the desired result. Namely, if we recall that the entropy is defined as an expectation (\Cref{def:entropy}), using the law of total expectation
\begin{align}
    \ent(U | \mathbf{L}, \tilde{S}) &= \ent(U | \mathbf{L}, \tilde{S}, E) \nonumber \\
    &= \bP[E=0] \cdot \ent^{E=0}(U | \mathbf{L}, \tilde{S}) + \bP[E=1] \cdot \ent^{E=1}(U | \mathbf{L}, \tilde{S}) \nonumber \\
    &= \bP[E=0] \cdot \ent^{E=0}(U | \mathbf{L}, \tilde{S}) \nonumber \\
    &\leq n \cdot 0.9 \log 2 \label{eq:birthday_ent_ecmi_bound}
\end{align}
where the first line follows since $E$ is a function of the supersample $\tilde{S}$. The last inequality follows from the bound $\ent^{E=0}(U | \mathbf{L}, \tilde{S}) \leq n \log n$, which comes from the fact that $|\cU| = |\{ 0, 1 \}^n| = 2^n$ and that $\ent(U) \leq \log |\cU|$ (\Cref{prop:entropy_properties_discrete}), and the fact that $\bP[E=0] \leq 0.9$.

Finally, combining~\eqref{eq:birthday_ecmi_bound} and~\eqref{eq:birthday_ent_ecmi_bound} results in
\begin{equation*}
    \minf(\mathbf{L};U|\tilde{S})  \geq n \log 2  - n \cdot 0.9 \log 2 \in \Omega(n),
\end{equation*}
and completes the proof. In particular, we have that $\minf(\mathbf{L};U|\tilde{S}) \geq 0.069n$.

\begin{subappendices}

\section{Derivations for the Gaussian Location Model Example}
\label[appendix]{app:glm}

The problem considered in~\Cref{ex:glm} in~\Cref{sec:bounds_using_wasserstein_distance} is the estimation of the mean $\mu$ of a $d$-dimensional Gaussian distribution with known covariance matrix $\sigma^2 I_d$. Furthermore, there are $n$ samples $S = (Z_1, \ldots, Z_n)$ available, the loss is measured with the Euclidean distance $\ell(w,z) = \lVert w-z \rVert_2$, and the estimation is their empirical mean $W = \frac{1}{n} \sum_{i=1}^n Z_i$.

\looseness=-1 To calculate the expected generalization error and derive different bounds, it is convenient to know how the random variables are distributed.
For example, in this setting $\bP_Z = \cN(\mu,\sigma^2 I_d)$, $\bP_{W} = \cN\Big(\mu,\frac{\sigma^2}{n} I_d\Big)$, $\bP_{W}^{Z_i} = \cN\Big(\frac{(n-1)\mu + Z_i}{n},\frac{\sigma^2(n-1)}{n^2} I_d\Big)$, $\bP_{W}^{S^{-j}} = \cN\Big(\frac{\mu}{n} + \frac{1}{n} \sum_{i \neq j} Z_i,\sigma^2 I_d \Big)$, and $\bP_{W}^{S} = \delta\Big(\frac{1}{n} \sum_{i=1}^n {Z_i} \Big)$.
Another important feature of this problem is that the loss function is $1$-Lipschitz under $\rho(w,w') = \lVert w - w' \rVert_2$.

\subsection{Expected Generalization Error}
\label{subapp:glm_ege}

In order to derive an exact expression of the generalization error, it is suitable to write it in the following explicit form:
\begin{equation*}
    \bE \big[ \gen(W,S) \big] = \bE[\ell(W,Z)] - \frac{1}{n} \sum_{i=1}^n \bE[\ell(W,Z_i)],
\end{equation*}
where $Z \sim \bP_Z$ is independent of $W$. Then, the two terms can be evaluated independently.
The first term is equivalent to 
\begin{equation*}
    \bE[\ell(W,Z)] = \bE[\lVert W-Z \rVert_2] = \sqrt{2 \sigma^2 \Big(1 + \frac{1}{n}\Big)} \frac{\Gamma\big(\frac{d+1}{2}\big)}{\Gamma\big(\frac{d}{2}\big)},
\end{equation*}
where the first equality follows from the definition of the loss function. The second equality follows from noting that $(W-Z) \sim \cN\big(0,\sigma^2 \big(1+\frac{1}{n}\big)I_d\big)$ and therefore $\lVert W-Z \rVert_2 = \sqrt{ \sigma^2 \big(1+\frac{1}{n}\big)} X$, where $X$ is distributed according to the $\chi$ distribution with $d$ degrees of freedom.

Similarly, the summands of the second term are equivalent to
\begin{equation*}
    \bE[\ell(W,Z_i)] = \bE[\lVert W-Z_i \rVert_2] = \sqrt{2 \sigma^2 \Big(1 - \frac{1}{n}\Big)} \frac{\Gamma\big(\frac{d+1}{2}\big)}{\Gamma\big(\frac{d}{2}\big)},
\end{equation*}
where, as before, the first equality follows from the definition of the loss function.
The second equality follows from noting that $(W-Z_i) \sim \cN\big(0,\sigma^2 \big(1-\frac{1}{n}\big)I_d\big)$ and therefore $\lVert W-Z_i \rVert_2 = \sqrt{ \sigma^2 \big(1-\frac{1}{n}\big)} X$, where $X$ is distributed according to the $\chi$ distribution with $d$ degrees of freedom.
In this case, $W$ and $Z_i$ are not independent random variables. In fact, $(W,Z_i)$ is normally distributed with covariance matrix
\begin{equation*}
    \begin{pmatrix}
    \frac{\sigma^2}{n} I_d & \frac{\sigma^2}{n}I_d \\
    \frac{\sigma^2}{n} I_d & \sigma^2 I_d 
    \end{pmatrix},
\end{equation*}
from which the distribution of $W-Z_i$ is deduced.

Finally, subtracting both terms results in
\begin{equation*}
    \bE \big[ \gen(W,S) \big] = \sqrt{\frac{2 \sigma^2}{n}} \Big(\sqrt{n+1} - \sqrt{n-1} \Big) \frac{\Gamma\big(\frac{d+1}{2}\big)}{\Gamma\big(\frac{d}{2}\big)} \leq \frac{\sqrt{2\sigma^2 d}}{n}.
\end{equation*}
where the inequality follows from the following two bounds: (i) $\sqrt{n+1} - \sqrt{n-1} \leq \sqrt{\frac{2}{n}}$, which is obtained by multiplying and dividing by $\sqrt{n+1} + \sqrt{n-1}$ and noting that $\sqrt{n+1} + \sqrt{n-1} \geq \sqrt{2 n}$, and (ii) the upper bound on the ratio of gamma distributions by $\sqrt{\frac{d}{2}}$ using the series expansion at $d \to \infty$.

\subsection{Wasserstein Distance Bound}
\label{subapp:glm_was}

The bound from~\citep{wang2019information} can be calculated exactly since $\bP_{W}^{S}$ is a delta distribution, that is
\begin{equation*}
    \bE\big[\bW_{\lVert \cdot \rVert_2}(P_{W}^{S},P_W)\big] = \bE\bigg[ \Big\lVert \frac{1}{n} \sum_{i=1}^n Z_i - \frac{1}{n} \sum_{i=1}^n Z_i' \Big\rVert_2 \bigg] = \sqrt{\frac{4\sigma^2}{n}} \frac{\Gamma\big(\frac{d+1}{2}\big)}{\Gamma\big(\frac{d}{2}\big)} \leq \sqrt{\frac{2\sigma^2 d}{n}}
\end{equation*}
where $Z_i' \sim \bP_Z$ are independent copies of $Z_i$. Hence, the difference is distributed as a normal distribution with mean 0 and covariance $\frac{2\sigma^2}{n} I_d$, which means that the norm is $\sqrt{\frac{2\sigma^2}{n}} X$, where $X$ is a $\chi$ random variable with $d$ degrees of freedom. 

\subsection{Single-Letter Wasserstein Distance Bound}
\label{subapp:glm_iwas}

An exact calculation of the bound from \Cref{th:wasserstein_single_letter} is cumbersome. However, the Wasserstein distance of order one can be bounded from above by the Wasserstein distance of order two~\citep[Remark 6.6]{villani2009optimal}, that is $\bW_{\lVert \cdot \rVert_2} \leq \bW_{2,\lVert \cdot \rVert_2}$, which has a closed form expression for Gaussian distributions. More specifically, 
\begin{equation*}
    \bE \big[\bW_{2,\lVert \cdot \rVert_2}(\bP_{W}^{Z_i},\bP_W)\big] \leq \frac{\sqrt{2\sigma^2}}{n} \frac{\Gamma\big(\frac{d+1}{2}\big)}{\Gamma\big(\frac{d}{2}\big)} + \sqrt{\frac{\sigma^2 d}{n^3}} \leq \frac{\sqrt{\sigma^2 d}}{n} + \sqrt{\frac{\sigma^2 d}{n^3}}.
\end{equation*}

The second inequality follows from the closed-form expression for the squared Wasserstein distance of order 2, namely
\begin{equation*}
    \bW_{2,\lVert \cdot \rVert_2}(\bP_{W}^{Z_i},\bP_W)^2 = \frac{1}{n^2} \lVert \mu - Z_i \rVert^2 + \frac{\sigma^2 d}{n} \Big(1 + \frac{n-1}{n}-2\sqrt{\frac{n-1}{n}} \Big),
\end{equation*}
where the term $(1+\frac{n-1}{n} - 2\sqrt{\frac{n-1}{n}})$ is a perfect square that is bounded from above by $\frac{1}{n^2}$. Then, the expression results from employing the inequality $\sqrt{x + y} \leq \sqrt{x} + \sqrt{y}$ and noting that $\lVert \mu - Z_i \rVert_2 = \sigma X$, where $X$ is a $\chi$ distributed random variable with $d$ degrees of freedom.

\subsection{Random-Subset Wasserstein Distance Bound}
\label{subapp:glm_rswas}

As in~\Cref{subapp:glm_was}, since $\bP_{W}^{S}$ is a delta distribution, the bound from \Cref{th:wasserstein_random_subset} can be calculated exactly. In particular, the bound assuming that $|J|=1$ is
\begin{align*}
    \bE\big[\bW_{\lVert \cdot \rVert_2}(\bP_{W}^{S},\bP_{W}^{S^{-J}})\big] &= \bE\bigg[ \Big\lVert \frac{1}{n} \sum_{i=1}^n Z_i - \Big(\frac{Z_J'}{n} + \frac{1}{n} \sum_{i \neq J} Z_i\Big) \Big\rVert_2 \bigg] \\
    &= \frac{\sqrt{4\sigma^2}}{n} \frac{\Gamma\big(\frac{d+1}{2}\big)}{\Gamma\big(\frac{d}{2}\big)} \\
    &\leq \frac{\sqrt{2\sigma^2 d}}{n},
\end{align*}
where $Z_J' \sim \bP_Z$ is an independent copy of $Z_J$. Hence, the norm is $\frac{\sqrt{2\sigma^2}}{n} X$, where $X$ is a $\chi$ random variable with $d$ degrees of freedom.

\subsection{Single-Letter Mutual Information Bound}
\label{subapp:glm_ismi}

The single-letter mutual information is $\minf(W;Z_i) = \frac{d}{2} \log \big(\frac{n}{n-1}\big)$ for all $i \in [n]$~\citep{bu2020tightening}. Nonetheless, in order to employ the bound from~\citet{bu2020tightening} in~\Cref{th:single_letter_mi_bound_general_cgf}, the loss function $\ell(W,Z)$ needs to have a CGF $\Lambda_{-\ell(W,Z)}(\lambda)$ bounded from above by a convex function $\psi(\lambda)$ such that $\psi(0)=\psi'(0)=0$ for all $\lambda \in [0,b)$ for some $b \in \bR_+$, where $Z \sim \bP_Z$ is independent of $W$.

The loss function $\ell(W,Z) = \lVert W - Z \rVert_2$ is $\sqrt{\sigma^2 (1 + \frac{1}{n})}X$, where $X$ is a $\chi$ random variable with $d$ degrees of freedom. The moment generating function $M_{\ell(W,Z)}(\lambda)$ of such a random variable is
\begin{equation*}
    M_{\ell(W,Z)}(\lambda) = \Bar{M}\Big(\frac{d}{2},\frac{1}{2},\frac{\lambda^2}{2}\Big) +  \frac{\lambda \sqrt{2} \Gamma\big(\frac{d+1}{2}\big)}{\Gamma\big(\frac{d}{2}\big)} \Bar{M}\Big(\frac{k+1}{2},\frac{3}{2},\frac{\lambda^2}{2} \Big),
\end{equation*}
where $\Bar{M}$ is the Kummer's confluent hypergeometric function.

The expression of this moment generating function is too convoluted to study for $d > 1$. Nonetheless, for $d=1$ it has a closed-form expression, namely
\begin{equation*}
    M_{\ell(W,Z)}(\lambda) = e^{\frac{\lambda^2}{2}} \Big(1 + \textnormal{erf}\Big(\frac{\lambda}{\sqrt{2}}\Big) \Big).
\end{equation*}

Therefore, the CGF $    \Lambda_{ \ell(W,Z)}(\lambda) = \frac{\lambda^2}{2} + \log ( 1 + \textnormal{erf}(\frac{\lambda}{\sqrt{2}}))$, which is bounded from above by the convex function $\psi(\lambda) = \frac{\lambda^2}{2}$ for all $\lambda \in (-\infty,0]$. Noting that $\Lambda_{X}(\lambda) = \Lambda_{-X}(-\lambda)$ from the algebra of CGFs from~\Cref{subsec:moments_and_generating_functions}, we have that the bound from~\citet{bu2020tightening} in~\Cref{th:single_letter_mi_bound_general_cgf} can be applied yielding
\begin{align*}
    \bE \big[ \gen(W,S) \big] &\leq \frac{1}{n} \sum_{i=1}^n \sqrt{2 \sigma^2 \Big(1 + \frac{1}{n}\Big)  I(W;Z_i)} \\
    &\leq \sqrt{\sigma^2 \Big(1 + \frac{1}{n}\Big)  \log \Big(\frac{n}{n-1}\Big)} \\ 
    &\leq \sqrt{\frac{2 \sigma^2}{n-1}},
\end{align*}
where the last inequality stems from noting that $\frac{n}{n-1} = 1 + \frac{1}{n-1}$, the fact that $\log(1+x) \leq x$, and bounding $(1+\frac{1}{n})$ from above by $2$.

\section{Tightness of the Mutual Information-Based Bounds}
\label[appendix]{app:minf_bounds_tight}

In this section of the appendix, we develop the proof of~\Cref{th:minf_bounds_tight}.

\begin{proof}
Consider $\cW$ be a ball of radius $\nicefrac{B}{2}$ in $\bR^d$, and therefore $\diam(\cW) = B$. Further consider an arbitrary $z_0 \in \cW$ such that $\lVert z_0 \rVert_2 = \nicefrac{B}{2}$. Let the input space be $\cZ = \{ \nicefrac{2z_0}{B}, - \nicefrac{2z_0}{B} \}$, the data distribution $\bP_Z$ be $\bP_Z[\nicefrac{2z_0}{B}] = \bP_Z[-\nicefrac{2z_0}{B}] = \nicefrac{1}{2}$, and the loss function be $\ell(w,z) = - L \langle w, z \rangle$. It is easy to see that the loss function is convex and $L$-Lipschitz (see~\citep{orabona2019modern} for similar constructions).

Define a Rademacher random variable $R_i$ such that $R_i = 1$ if $Z_i = \nicefrac{2z_0}{B}$ and $R_i = -1$ if $Z_i = - \nicefrac{2z_0}{B}$. We can, therefore, represent each training sample as $Z_i = R_i \cdot \nicefrac{2z_0}{B}$. The empirical risk for a hypothesis $w \in \cW$ is $\emprisk(w,S) = - \frac{2L}{B n} \mleft \langle w, z_0 \sum_{i=1}^n R_i \mright \rangle$. In this case, it is straightforward that the empirical risk minimizer for this problem is
\begin{equation*}
    \argmin_{w\in \cW}\emprisk(w,S) = \bA_\mathrm{ERM}(S) =
    \begin{cases}
        z_0 & \text{if}~\text{sign}\mleft(\sum_{i=1}^{n} R_i \mright)=1\\
        -z_0 & \text{if}~\text{sign}\mleft(\sum_{i=1}^{n} R_i \mright)=-1
    \end{cases},
\end{equation*}
where $\sign(x) = 1$ if $x \geq 0$ and $\sign(x) = -1$ otherwise.

First, we provide a lower bound on the expected generalization error. The expected empirical risk of $\bA_\mathrm{ERM}$ is
\begin{align*}
    \bE \mleft[ \min_{w \in \cW} \emprisk(w,S) \mright] &= \bE \mleft[ \min_{w \in \cW} - \frac{2L}{B n} \mleft \langle w, z_0 \sum_{i=1}^n R_i \mright \rangle \mright] \\
    &= - \frac{2L}{B n} \bE \mleft[ \max_{w \in \{ z_0, - z_0\}} \mleft \langle w, z_0 \sum_{i=1}^n R_i \mright \rangle \mright] \\
    &= - \frac{2L}{B n} \bE \mleft[ \mleft| \mleft \langle z_0, z_0 \sum_{i=1}^n R_i \mright \rangle \mright| \mright] \\
    &=- \frac{L B}{2n} \bE \mleft[ \mleft| \sum_{i=1}^n R_i \mright| \mright],
\end{align*}
where we have used that $\max(x,y) = \frac{x+y}{2} + \frac{|x-y|}{2}$ for all $x,y \in \bR$. Observing that $\poprisk(w) = 0$ for all $w \in \cW$ we observe that the generalization error ot the hypothesis $W = \bA_\mathrm{ERM}(S)$ is lower bounded by
\begin{equation*}
    \bE \big[ \gen(W, S) \big] = \frac{L B}{2n} \bE \mleft[ \mleft| \sum_{i=1}^n R_i \mright| \mright] \geq \frac{LB}{2\sqrt{2n}},
\end{equation*}
\looseness=-1 where the inequality follows from the Khintchine--Kahane inequality \citep[Theorem D.9]{mohri2018foundations}.

Next, we analyze the upper bounds based on~\Cref{cor:total_variation_relent_full}. Observe that the Markov chain $S - \sign \mleft( \sum_{i=1}^n R_i \mright) - W$ holds. Then, by the data processing inequality~\Cref{prop:properties_minf} and the fact that the Shannon's entropy is non-negative
\begin{equation*}
    \minf(W;S) \leq \minf \mleft( W; \sign \mleft( \sum_{i=1}^n R_i \mright) \mright) \leq \ent \mleft( \sum_{i=1}^n R_i \mright) \leq \log 2,
\end{equation*}
where the last inequality holds since the $\sign(\cdot)$ can only take two values. Therefore, \Cref{cor:total_variation_relent_full} guarantees that $\bE \big[ \gen(W,S)\big] \leq LB \sqrt{\nicefrac{\log 2}{2n}}$.
\end{proof}

\section{Randomized-Subsample Setting and the Bretagnolle--Huber Inequality}
\label[appendix]{app:rs_and_bh}

In \Cref{cor:total_variation_relent_full_rs,cor:total_variation_relent_random_subset_rs,cor:total_variation_relent_single_letter_rs}, the immediate bound that stems from the use of the Bretagnolle--Huber inequality from~\Cref{lemma:bretagnolle-huber-inequality} is not included.
The reason for this is that the relative entropy terms $\relent(\bP_{W}^{\tilde{Z}_{i,0}, \tilde{Z}_{i,1}, U_i} \Vert \bP_{W}^{\tilde{Z}_{i,0}, \tilde{Z}_{i,1}})$ and $\relent(\bP_W^{\tilde{S},U, R} \Vert \bP_W^{\tilde{S}, U^{-J}, R})$, when $|J|=1$, are never greater than $\log 2$ as shown in ~\Cref{lemma:kl_smaller_than_log2} below. For simplicity, for the randomized subset bounds we will consider $\bQ(\tilde{S},U^{-J},R) = \bP_W^{\tilde{S},U^{-J},R}$. 
Hence, the range of these relative entropy terms is inside the range where Pinsker's inequality is tighter than the Bretagnolle--Huber inequality.
\begin{lemma}
\label{lemma:kl_smaller_than_log2}
Let $\bP_{X}^{A,B}$ be a Markov kernel from $\cA \to \cB$ to distributions on $\cX$, where $B$ is a Bernoulli random variable with probability $\nicefrac{1}{2}$ and $A \in \mathcal{A}$ is a random variable independent of $B$.
Let also $\bP_{X}^{A} = \bP_X^{A,B} \circ \bP_B$. Then, $\relent(\bP_X^{A,B} \Vert \bP_X^A) \leq \log 2$.
\end{lemma}
\begin{proof}
In this situation $\bP_X^A$ dominates $\bP_X^{A,B}$ and $\bP_X^{A,(1-B)}$, that is, $\bP_X^{A,B} \ll \bP_X^A$ and $\bP_X^{A,(1-B)} \ll \bP_X^A$ since $\bP_X^A = \frac{1}{2}( \bP_X^{A, B} + \bP_X^{A, (1-B)})$. 
Therefore, 
\begin{align*}
    \relent(\bP_X^{A,B} \Vert \bP_X^A) &= \bE^{A,B} \Bigg[ \log \Bigg( \frac{\mathrm{d} \bP_{X}^{A,B}}{\mathrm{d}\big(\frac{1}{2} \bP_{X}^{A,B} + \frac{1}{2} \bP_{X}^{A,(1-B)}\big)} \Bigg)  \Bigg] \\
    &= -\bE^{A,B} \Bigg[ \log \Bigg( \frac{\mathrm{d}\big(\frac{1}{2} \bP_{X}^{A,B} + \frac{1}{2} \bP_{X}^{A,(1-B)}\big)}{\mathrm{d}\bP_{X}^{A,B}} \Bigg) \Bigg] \\
    &= \log 2 - \bE^{A,B} \bigg[ \log \bigg( 1 + \frac{\mathrm{d}\bP_{X}^{A,(1-B)}}{\mathrm{d} \bP_{X}^{A,B}} \bigg) \bigg] \\
    &\leq \log 2,
\end{align*}
where the second equality stems from~\citep[Exercise 9.27]{mcdonald1999course}, the third one follows from the linearity $\bP_{X}^{A,B}$-a.e.\ of the Radon--Nikodym derivative, and last inequality is due to the fact that $\log(1+x) \geq 0$ for all $x \geq 0$ and the fact that $\nicefrac{\mathrm{d} \bP_{X}^{A,(1-B)}}{\mathrm{d} \bP_{X}^{A,B}}$ is always positive.
Also, each step is possible since the expectation integrates over the support of $\bP_{X}^{A,B}$, avoiding the problems of absolute continuity.
\end{proof}

\Cref{lemma:kl_smaller_than_log2} can be easily extended to the case where $B$ is a sequence of $k$ Bernoulli random variables $B_i$, noting that $\bP_{X}^{A} = 2^{-k} \sum_{j=1}^{2^k} \bP_{X}^{A,\mathcal{B}_j}$, where $\mathcal{B}$ are all the $2^k$ random sequences $\mathcal{B}_j$ where the $i$-th element can be either $B_i$ or $(1-B_i)$.
In that case, we have that $\relent(\bP_X^{A,B} \Vert \bP_X^A) \leq k \log 2$.

Then, note that $\relent(\bP_W^{\tilde{Z}_{i,0}, \tilde{Z}_{i,1}, U_i} \Vert \bP_W^{\tilde{Z}_{i,0}, \tilde{Z}_{i,1}}) \leq \log 2$ if $A = (\tilde{Z}_i,\tilde{Z}_{i+n})$, $B = U_i$, and $X = W$.
Similarly, for $|J|=1$, note that $\relent(\bP_W^{\tilde{S},U, R} \Vert \bP_W^{\tilde{S}, U^{-J}, R}) \leq \log(2)$ if $A = (\tilde{S},U_{J^c},R)$, $B = U_J$, and $X=W$.

If we were to replicate~\Cref{cor:total_variation_relent_full_rs} when $|J| > 2$, it is not guaranteed that the inequality obtained from Pinsker's inequality is tighter than the one obtained with the Bretagnolle-Huber inequality.
For instance, as discussed above, $\relent(\bP_W^{\tilde{S},U, R} \Vert \bP_W^{\tilde{S}, U_{J^c}, R})$ could be as large as $|J| \log 2$, which is already larger than $1.6$ for $|J| = 3$, which is the threshold where the Bretagnolle--Huber inequality starts to be tighter. Hence, for $|J| > 2$, one should also consider that inequality if one desires the tightest bound. However, the bound derived from the Bretagnolle--Huber inequality was not included since this kind of bounds are usually employed for $|J| = 1$ as we will see shortly in~\Cref{sec:noisy_iterative_learning_algos} and since they are shown to be tightest for $|J|=1$~\citep{negrea2019information}.

\section{Optimization Error Rate of Gradient Descent's Last Iterate}
\label[appendix]{app:helper_lemmata}

\begin{lemma}
\label[lemma]{lem:gd-last-iterate}
Let $\ell(\cdot, z)$ be a convex and $L$-Lipschitz loss function for every training instance $z \in \cZ$, and $\cW$ be a convex and compact parameter space with bounded diameter $B$. Let $\{w_t\}_{t\in [T]}$ denote the output of the GD algorithm with a constant step size $\eta$. Then,  for every training set $s \in \cZ^n$
\begin{align*}
    \emprisk(w_T, s) - \min_{w\in \cW}  \emprisk(w, s)  \leq \frac{B^2}{2\eta T}+ \frac{\eta L^2 (2 + \log T)}{2}.
\end{align*}
\end{lemma}
\begin{proof}
Let $g_t$ be a member of the sub-differential $\partial \emprisk(w_t, s)$ of $\emprisk(w_t, s)$. Since $\ell(\cdot, z)$ is convex and $L$-Lipschitz for all $z \in \cZ$, then $\emprisk(\cdot, s)$ is convex and $L$-Lipschitz as well. Hence, \citep[Theorem 2]{lastiterate} guarantees that
\begin{align*}
    \emprisk(w_T, s) &- \min_{w \in \cW} \emprisk(w, s) \leq \\
    &\frac{1}{T} \sum_{t=1}^{T} \big(\emprisk(w_t, s) -\min_{w\in \cW} \emprisk(w, w) \big) +\frac{1}{2} \sum_{k=1}^{T-1}\frac{1}{k(k+1)}\sum_{t=T-k}^{T} \eta \lVert g_t \rVert_2^2.
\end{align*}
Since $\lVert g_t \rVert_2 \leq L$, the second term can be upper bounded by $\frac{\eta L^2}{2}\sum_{k=1}^{T-1}\frac{1}{k}$. Then, by the well-known bounds on the Harmonic numbers  $\frac{\eta L^2}{2}\sum_{k=1}^{T-1}\frac{1}{k} \leq \frac{\eta L^2}{2}(1+ \log(T-1))\leq \frac{\eta L^2}{2}(1 + \log T)$. Finally, the first term is bounded by $ \frac{1}{T}\sum_{t=1}^{T} \big(\emprisk(w_t, s)-\min_{w\in \cW} \emprisk(w,s) \big) \leq \frac{B^2}{2\eta T} + \frac{\eta L^2}{2}$~\citep[Theorem 3.2]{bubeck2015convex}. Combining these two upper bounds proves the lemma.
\end{proof}

\end{subappendices}

%% file: chapters/pac_bayes.tex
\looseness=-1 In this chapter, we discuss PAC-Bayesian and single-draw PAC-Bayesian generalization guarantees within the framework of information-theoretic generalization introduced in~\Cref{sec:information_theoretic_generalization}. These generalization guarantees are the most specific within the three levels of specificity described in~\Cref{subsec:levels_of_specificity}.
To be precise, this kind of guarantee ensures that the generalization error is bounded by a certain quantity with at least a prescribed level of probability $1-\beta$. In a sense, they are measuring how much the empirical risk concentrates around the population risk. This is the reason why they have the classical term PAC (\Cref{sec:pac_learning}) in their name.

PAC-Bayesian guarantees describe what is the smallest expected difference between the average population risk on hypotheses returned by the algorithm and the average empirical risk that holds with at least probability $1 - \beta$. As we can see, this is more specific than guarantees on expectation, as now we know that the guarantee holds with a certain probability for the \emph{observed} training set. On the other hand, single-draw PAC-Bayesian guarantees describe what is the smallest difference between the population risk and the empirical risk that holds with probability $1 - \beta$. Therefore, they are again a step above standard PAC-Bayesian guarantees in the specificity ladder as they hold with a certain probability for the \emph{observed} training set and hypothesis returned. In the case that the algorithm is deterministic, then both bounds coincide.

As we did in the introduction of~\Cref{ch:expected_generalization_error}, we may ask ourselves, ``\emph{What do we desire from PAC-Bayesian guarantees?''} There are mainly three desiderata that we want from this kind of bounds:
\begin{enumerate}
    \item Like for bounds in expectation, a first objective of this kind of bounds is to gain understanding. That is, we seek to know what drives apart the performance of an algorithm on its objective  (the population risk) with respect to the proxy it uses to return a hypothesis (the empirical risk). The probabilistic nature of these bounds sometimes makes this understanding a little harder than for bounds in expectation. For this reason, often it is convenient to prove an intuitive result in expectation, and later try to translate it into a PAC-Bayesian version.
    \item \looseness=-1 A unique feature from these bounds is that, for the observed training data $s$, if we can evaluate the posterior distribution $\bP_W^{S=s}$, we can evaluate the bound. This makes these bounds very attractive to design algorithms that have a small population risk by optimizing the PAC-Bayesian bound directly. For this reason, an important feature of these bounds is that either the expression for the optimal posterior is available, or that it can be optimized in some way. The algorithms that design the posterior under this criterion are often referred to \emph{self-bounding algorithms}~\citep{freund1998self,langford2003microchoice} or \emph{self-certified learning methods}~\citep{perez2021tighter,rivasplata2022pac}.
    \item \looseness=-1 For single-draw PAC-Bayesian guarantees, or for standard PAC-Bayesian if the algorithm is deterministic, the guarantees hold for the observed training set $s$ and the returned hypothesis $w$. Therefore, if we can evaluate the posterior density with respect to some measure $\sfp_W^{S=s}(w)$, the main objective now is to have numerically sharp bounds to evaluate the performance of the algorithm.
\end{enumerate}

To the question ``Why do we study information-theoretic bounds?'', the answer is the same as we gave for the bounds in expectation from~\Cref{ch:expected_generalization_error}. That is, we choose this framework because these bounds are specific to the learning algorithm and the data distribution.

Above we mentioned that PAC-Bayesian bounds, in a sense, measure how much the population risk concentrates around the empirical risk. In other words, PAC-Bayesian guarantees are bounds on the tail of the generalization error random variable. In fact, if the hypothesis returned by the algorithm was completely independent of the algorithm, these bounds would reduce to classical concentration inequalities. For this reason, we start the chapter in~\Cref{sec:classical_concentration_inequalities} reviewing such classical inequalities. After that, we review the origin of this kind of bounds in \Cref{sec:origin_pac_bayesian_bounds} to lay down the notation and motivate the different paths of research in this field studied in this monograph. Here we also briefly discuss the limitations of the relative entropy as the dependency measure and point to works substituting by different metrics.

In~\Cref{sec:bounds_bounded_losses}, we focus our attention on PAC-Bayesian bounds for losses with a bounded range. We start by discussing two of the most important bounds in this setting, due to Seeger and Langford~\citep{seeger2002pac,langford2001bounds} and \citet{catoni2007pac}, along with their shortcomings. Then, we prove that they are related and develop a fast-rate equivalent to the bound from Seeger and Langford. This bound is important since it is equally tight to their bound, their linear nature makes it more interpretable, and it has a closed form for the optimal posterior distribution $\bP_W^S$. This bound will be the foundation for later developments in the chapter. This section is based on our results from~\citep{rodriguez2023morepac}.

After that, in~\Cref{sec:bounds_unbounded_losses}, we shift our attention to PAC-Bayesian bounds for unbounded losses. In particular, we develop bounds for losses with a bounded CGF and with bounded moments. For losses with a bounded CGF, we developed a PAC-Bayes Chernoff analogue. For losses with a bounded raw moment or a bounded variance, we employ our fast-rate bound from~\Cref{sec:bounds_bounded_losses} to derive new bounds with a rate that interpolates between a slow rate when only the second moment is bounded to a fast rate when the loss is bounded almost surely. In order to prove all these bounds we derive a new technique to optimize parameters in probabilistic statements that can be of independent interest. This part of the section is mostly based on our results from~\citep{rodriguez2023morepac,rodriguez2024moments}. Then, \Cref{subsec:related_approaches,sec:pac_bayes_unbounded_bibliographic_remarks} discuss related bibliography on other approaches to optimize parameters in probabilistic statements~\citep{langford2001not,catoni2003pac,seldin2012pac,tolstikhin2013pac,kuzborskij2019efron,de2007pseudo,kakade2008complexity} and on other PAC-Bayesian bounds for unbounded losses~\citep{catoni2004statistical,hellstrom2020generalization,guedj2021still,esposito2021generalization,alquier2006transductive,alquier2018simpler, holland2019pac, kuzborskij2019efron, haddouche2023pacbayes, haddouche2021pac, chugg2023unified}.

Later, in~\Cref{sec:single-draw-pac-bayes} we show how mostly all the results presented can be extended to single-draw PAC-Bayesian bounds. Some of the results follow from our article~\citep{rodriguez2024moments}, but most of them are new and could have been derived from our article~\citep{rodriguez2023morepac}. Crucially, we employ an extension of \citet{rivasplata2020pac}'s probabilistic version of the Donsker and Varadhan lemma.

Finally, in \Cref{sec:anytime_validity}, we reflect on anytime-valid PAC-Bayesian bounds, which are bounds that hold \emph{simultaneously} for every time step in interactive learning algorithms. We show that for the cases presented in this monograph, extending any high-probability bound to their anytime-valid version is possible incurring only in a logarithmic increment to the bound.

The intention of this chapter, like in~\Cref{ch:expected_generalization_error}, is to be didactic and to convey the key messages and results from PAC-Bayesian information-theoretic bounds. As mentioned previously, we hope the reader can get a benefit from this chapter. Other available resources surveying and introducing this field are given by~\citet{guedj2019primer}, \citet{alquier2021user}, and~\citet{hellstrom2023generalization}.

\section{Classical Concentration Inequalities}
\label{sec:classical_concentration_inequalities}

Recall from~\Cref{ch:expected_generalization_error} that, for every fixed hypothesis $w \in \cW$, the empirical risk is an unbiased estimator of the population risk, that is, $\bE[\emprisk(w,S)] = \poprisk(w)$. The bias in the estimation comes from the dependence between the hypothesis and the training data. 

\subsection{Essential, Markov-Based Concentration Inequalities}
\label{subsec:essential_concentration_inequalities}

In this chapter, we seek to understand how much the empirical risk of the algorithm's output hypothesis $W = \bA(S)$ concentrates around the population risk. For this reason, we believe that it is relevant to first review how much it concentrates for a fixed hypothesis $w \in \cW$. In this way, we can have an idea of which rates to expect under different assumptions. Ideally, the bounds derived in later sections will replicate the classical concentration bounds with an extra term describing the dependence between the algorithm and the data.

Arguably the simplest tail inequality is \emph{Markov's inequality}~\citep[Section 2.1.1]{wainwright2019high}, which states that for every non-negative random variable $X$ and every $t > 0$
\begin{equation}
    \label{eq:markov}
    \bP \mleft[ X \geq t \mright] \leq \frac{\bE \mleft[ X \mright]}{t}.
\end{equation}
This simple inequality reflects well that the larger the value of $t$, the less likely is that a realization of $X$ surpasses it, taking into account its expected value $\bE[X]$. 

Markov's inequality is the foundation of many other concentration (or tail) inequalities. For example, consider $X = \gen(w,S)^2 = (\poprisk(w) - \emprisk(w,S))^2$. Then, since the empirical risk $\emprisk(w,S)$ is an unbiased estimator of the population risk, we have that $\bE[X] = \mathrm{Var}\big(\emprisk(w,S)\big)$. Moreover, since the instances $Z_i$ are i.i.d. it follows that $\mathrm{Var}\big(\emprisk(w,S)\big) = \frac{1}{n} \mathrm{Var}\big(\ell(w,Z) \big)$. Combining these two things we obtain \emph{Chebyshev's inequality}~\citep[Section 2.1.1]{wainwright2019high}
\begin{equation*}
    \bP \mleft[ \gen(w,S)^2 \geq t \mright] \leq \frac{\var\mleft( \ell(w,Z) \mright)}{n t}.
\end{equation*}
If we consider that the loss function has a variance bounded by $\sigma^2$ for all hypotheses $w \in \cW$, re-arranging this inequality results in the following proposition.

\begin{proposition}[{Adaptation of Chebyshev's inequality~\citep{wainwright2019high}}]
    \label{prop:chebyshev}
    Consider a loss function $\ell(w,Z)$ with a variance bounded by $\sigma^2$ for all $w \in \cW$. Then, for all $\beta \in (0,1)$ and all $w \in \cW$, with probability no smaller than $1-\beta$
    \begin{equation*}
        \gen(w,S) \leq \sqrt{\frac{\sigma^2}{n \beta}}.
    \end{equation*}
\end{proposition}

The same exercise can be done considering $X = |\gen(w,S)|^p = | \poprisk(w) - \emprisk(w)|^p$. Then, $\bE[X]$ is the $p$-th central moment of $\emprisk(w,S)$. Consider that the $p$-th central moment of $\ell(w,Z)$ is bounded by $c_p$ for all $w \in \cW$. Then, similarly to what happened for $p=2$, the $p$-th central moment of $\emprisk(w,S)$ can be bounded as
\begin{equation*}
    \bE \mleft[ \mleft | \poprisk(w) - \frac{1}{n} \sum_{i=1}^n \ell(w,Z) \mright|^p \mright] \leq \frac{1}{n^p} \bE \mleft[ \mleft( \sum_{i=1}^n \mleft| \poprisk(w) - \ell(w,Z_i) \mright| \mright)^p \mright] \leq \frac{c_p}{n^{p-1}}, 
\end{equation*}
where the first inequality comes from the fact that $|\sum_{i=1}^n x| \leq \sum_{i=1}^n |x|$ for all $x \in \bR$, and the second comes from the fact that the samples $Z_i$ are i.i.d. Therefore, we may obtain a general version of~\Cref{prop:chebyshev}.

\begin{proposition}[{Adaptation of the moment's inequality~\citep{wainwright2019high}}]
    \label{prop:bounded_moments}
    Consider a loss function $\ell(w,Z)$ with a central $p$-th moment bounded by $c_p$ for all $w \in \cW$. Then, for all $\beta \in (0,1)$ and all $w \in \cW$, with probability no smaller than $1 - \beta$
    \begin{equation*}
        \gen(w,S) \leq \mleft( \frac{c_p}{n^{p-1} \beta} \mright)^{\frac{1}{p}}.
    \end{equation*}
\end{proposition}

\Cref{prop:chebyshev,prop:bounded_moments} teach us that for losses with a bounded $p$-th moment, we expect that the rate, with respect to the number of samples, is $\cO(n^{-\frac{p}{p-1}})$ for all $p > 1$. This is quite revealing: for losses with a bounded second moment, we may expect a slow rate of $\cO(\nicefrac{1}{\sqrt{n}})$ while as the number of moments increases, and $p \to \infty$, we may expect a fast rate of $\cO(\nicefrac{1}{n})$. This last situation corresponds to losses that are essentially bounded, that is, for which $\esssup |\ell(w,Z) - \bE[\ell(w,Z)]| \leq c < \infty$ for all $w \in \cW$.

The main issue with~\Cref{prop:chebyshev,prop:bounded_moments} is that they are not bounds of high probability, since the dependence with the probability parameter $\beta$ is polynomial $\cO(\beta^{-\frac{1}{p}})$ and not logarithmic. As we will see in~\Cref{sec:bounds_bounded_losses}, this can be resolved through Alquier's method in a way similar to what we saw in~\Cref{subsec:interpolating_between_slow_and_fast_rate}, although at the expense of trading a bounded central moment for a bounded raw moment in the bounds.

Markov's inequality~\eqref{eq:markov} is also the source of the Cramér--Chernoff method~\citep[Section 2.3]{boucheron2003concentration}. Indeed, consider $X = e^{n \lambda \gen(w,S)} = e^{n\lambda(\poprisk(w) - \emprisk(w,S))}$ for some $\lambda \in \bR$, then $\bE[X] =\exp \big( \Lambda_{-\emprisk(w,S)}(n\lambda) \big)$. Since the instances $Z_i$ are i.i.d., by the algebra of CGFs from~\Cref{subsec:moments_and_generating_functions} this means that $\bE[X] = \exp \big( n \Lambda_{-\ell(w,Z)} (\lambda ) \big) $. Therefore, if we consider a loss with a CGF bounded by $\psi$ in the sense of~\Cref{def:bounded_cgf}, it follows that
\begin{equation*}
    \bP \mleft[ e^{n \lambda \gen(w,S) } \geq t \mright] \leq \frac{e^{n \psi(\lambda)}}{t}.
\end{equation*}
Re-arranging this equation and letting $\beta = \nicefrac{e^{n \psi(\lambda)}}{t}$, we have that for all $\beta \in (0,1)$ and all $\lambda > 0$
\begin{equation*}
    \bP \mleft[ \gen(w,S) \geq \frac{1}{\lambda} \mleft( \psi(\lambda) + \frac{\log \frac{1}{\beta}}{n} \mright) \mright] \leq \beta.
\end{equation*}
Finally, noting that the optimal value of $\lambda$ does not depend on the random object $S$, we may optimize it using~\Cref{lemma:boucheron_convex_conjugate_inverse} from~\Cref{subsec:convex_conjugate} and recover the \emph{Chernoff inequality}~\citep[Section 2.2]{boucheron2003concentration}.

\begin{proposition}[{Adaptation of Chernoff inequality~\citep[Section 2.2]{boucheron2003concentration}}]
    \label{prop:chernoff}
    Consider a loss with a bounded CGF (\Cref{def:bounded_cgf}). Then, for all $\beta \in (0,1)$ and all $w \in \cW$, with probability no smaller than $1-\beta$
    \begin{equation*}
        \gen(w,S) \leq \psi_*^{-1} \mleft( \frac{\log \frac{1}{\beta}}{n} \mright).
    \end{equation*}
\end{proposition}

The function $\psi_*^{-1}$ is the inverse of the convex conjugate of the function that dominates the CGF of the loss. This function is non-decreasing and it satisfies that $\psi_*^{-1}(0) = 0$ and $\lim_{y \to \infty} \psi_*^{-1}(y) \to \infty$. If we write~\Cref{prop:chebyshev} as
\begin{equation*}
    \bP \mleft[ \gen(w,S) \geq \psi_*^{-1} \mleft( \frac{\log \frac{1}{\beta}}{n} \mright) \mright] \leq \beta,
\end{equation*}
we see how the function $\psi_*^{-1}$ gives us a good characterization of the tail of the generalization error random variable $\gen(w,S)$: for every value of $\beta \in (0,1)$, we can obtain a threshold $t = \psi_*^{-1} \mleft( \frac{1}{n} \log \frac{1}{\beta} \mright)$ for which the tail $\bP[\poprisk(w)-\emprisk(w,S)]$ has a probability smaller than $\beta$. 

As we saw in~\Cref{subsec:a_slow_rate_bound_mi}, for common behaviors like sub-gaussian, sub-exponential, or sub-gamma, the function $\psi_*^{-1}$ implies that the generalization error has essentially a slow rate $\cO(\nicefrac{1}{\sqrt{n}})$ with respect to $n$. However, the attractive property of this bound is that the dependence with the probability parameter $\beta$ is \emph{logarithmic}, and therefore it is a high-probability bound. A particularly important corollary of~\Cref{prop:chernoff} is the one for losses with a bounded range, which is known as \emph{Hoeffding's inequality}. Again, as discussed in~\Cref{subsec:a_slow_rate_bound_mi}, if the loss $\ell$ has a range in $[a,b]$, then it is $\nicefrac{(b-a)^2}{2}$-sub-Gaussian and therefore~\Cref{prop:chernoff} implies that
\begin{equation}
    \label{eq:hoeffding}
    \bP \mleft[ \gen(w,S) \geq (b-a) \sqrt{\frac{\log \frac{1}{\beta}}{2n}}\mright] \leq \beta.
\end{equation}

\subsection{Specialized Concentration Bounds for Bounded Losses}
\label{subsec:bounded_losses_classical}

More refined analyses can reveal better rates and/or better dependencies with the probability parameter $\beta$.
For example, instead of just considering that the $p$-th central moment of the loss is bounded by $c_p$, one may instead consider that the loss $\ell(w,Z)$ satisfies the \emph{Bernstein condition} with parameter $\alpha$. This condition requires that every moment is bounded by
\begin{equation*}
    \mleft| \bE \mleft[ \mleft( \ell(w,Z) - \poprisk(w) \mright)^p \mright] \mright| \leq \frac{1}{2} \cdot p! \sigma^2 \alpha^{p-2}
\end{equation*}
for all natural $p \geq 2$, where $\sigma^2$ is the variance of the loss~\citep[Section 2.1.3]{wainwright2019high}. In our case, since the losses are i.i.d, we may see that if the loss $\ell(w,Z)$ satisfies the Bernstein condition with parameter $\alpha$, so does the empirical risk $\emprisk(w,S)$ with parameter $\nicefrac{\alpha}{n}$. Namely,
\begin{align*}
    \mleft| \bE \mleft[ \mleft( \frac{1}{n} \sum_{i=1}^n \ell(w,Z_i) - \poprisk(w) \mright)^p \mright] \mright| &= \frac{1}{n^{p}} \mleft| \sum_{i=1}^n  \bE \mleft[ \mleft(\ell(w,Z) - \poprisk(w) \mright)^p \mright] \mright| \\
    &\leq \frac{1}{2n^{p-1}} p! \sigma^2 \alpha^{p-2} \\
    &= \frac{1}{2} \cdot p! \frac{\sigma^2}{n} \mleft(\frac{\alpha}{n} \mright)^{p-2},
\end{align*}
where the first equality follows from the independence of the losses and where in the last equality we are using that the variance of the empirical risk is $\nicefrac{\sigma^2}{n}$. Therefore, the \emph{Bernstein inequality}~\citep[Proposition 2.10]{wainwright2019high} states that 
\begin{equation*}
    \bP \mleft[ \mleft| \gen(w,S) \mright| \geq t \mright] \leq 2 e^{\frac{-nt^2}{2 (\sigma^2 + \alpha t)}}.
\end{equation*}

A particularly interesting case is that of losses with a range bounded in $[a,b]$, as in that case they always satisfy the Bernstein condition with $\alpha = (b-a)$~\citep[Section 2.1.3]{wainwright2019high}. In this situation, re-arranging this equation and letting $\beta = 2 e^{\frac{-nt^2}{2 (\sigma^2 + \alpha t)}}$ we obtain the following result.

\begin{proposition}[{Adaptation of Bernstein's inequality~\citep[Proposition2.10]{wainwright2019high}}]
    \label{prop:bernstein_inequality}
    Consider a loss $\ell(w,Z)$ with a range bounded in $[a,b]$ and a variance bounded by $\sigma^2$ for all $w \in \cW$. Then, for all $\beta \in (0,1)$ and all $w \in \cW$, with probability no smaller than $1 - \beta$
    \begin{equation*}
        \gen(w,S) \leq \frac{2(b-a) \log \frac{2}{\delta}}{n} + \sqrt{\frac{2 \sigma^2 \log \frac{2}{\beta}}{n}}.
    \end{equation*}
\end{proposition}

There are two terms that explain the concentration of the empirical risk around the population risk in Bernstein's inequality. The first one is a fast-rate decay that depends on the range of the loss. The second one is a slow-rate decay that depends on the variance of the loss. Compared to the moment's inequality from~\Cref{prop:bounded_moments}, this bound achieves a logarithmic dependence with respect to the probability parameter, while maintaining the slow rate for a bounded variance. However, it also requires that the loss is bounded, which would have achieved a fast rate in the moment's bound. Then, compared with the Chernoff bound from~\Cref{prop:chernoff}, using the fact that bounded losses are sub-Gaussian, we would have also achieved a high-probability bound, but then the constant in the slow-rate term would have been $(b-a)^2$ instead of $\sigma^2$, and it is often the case that $\sigma^2 \ll (b-a)^2$.

Bernstein's inequality can be slightly improved. This improved version is referred to as \emph{Bennett's inequality}~\citep[Exercise 2.7]{wainwright2019high} and states that
\begin{equation}
    \label{eq:bennet}
    \bP \mleft[ \mleft| \gen(w,S) \mright| \geq t \mright] \leq 2 e^{\frac{-n\sigma^2}{(b-a)^2} \phi\mleft( \frac{(b-a)t}{\sigma^2} \mright) },
\end{equation}
where $\phi(u) = (1+u)\log(u) - u$, and essentially recovers Bernstein's inequality by noting that $\phi(u) \geq \frac{u^2}{2(1+\nicefrac{u}{3})}$~\citep[Exercise 2.8]{boucheron2003concentration}. Another improvement of this inequality is its empirical version, the \emph{empirical Bernstein inequality}~\citep[Exercise 7.6]{lattimore2020bandit}, where the variance is estimated and is not assumed to be known.

\begin{proposition}[{Adaptation of the empirical Bernstein's inequality~\citep[Exercise 7.6]{lattimore2020bandit}}]
    \label{prop:empirical_bernstein_inequality}
    Consider a loss $\ell(w,Z)$ with a range bounded in $[a,b]$ and let $\hat{\Sigma}^2 = \frac{1}{n} \sum_{i=1}^n (\ell(W,Z_i) - \emprisk(W,S))^2$ be an estimate of the loss' variance. Then, for all $\beta \in (0,1)$ and all $w \in \cW$, with probability no smaller than $1 - \beta$
    \begin{equation*}
        \gen(w,S) \leq \frac{3(b-a) \log \frac{3}{\delta}}{n} + \sqrt{\frac{2 \hat{\Sigma}^2 \log \frac{3}{\beta}}{n}}.
    \end{equation*}
\end{proposition}

Finally, we conclude this section on classical concentration inequalities with the \emph{small-kl} inequality. For this bound, let us consider, without loss of generality, as we can always re-scale and center, losses with a range bounded in $[0,1]$. Then, employing Markov's inequality~\eqref{eq:markov} once again with $X = e^{n \relentber(\emprisk(w,S) \Vert \poprisk(w))}$ results in 
\begin{equation*}
    \bP \mleft[ e^{n \relentber(\emprisk(w,S) \Vert \poprisk(w)} \geq t \mright] \leq \frac{1}{t} \cdot \bE \mleft[ e^{n \relentber(\emprisk(w,S) \Vert \poprisk(w))} \mright],
\end{equation*}
where we recall that $\relentber(\hat{r} \Vert r) = \relent(\mathrm{Ber}(\hat{r}) \Vert \mathrm{Ber}(r))$ was defined in~\Cref{subsec:a_fast_rate_bound_mi}. \citet[Theorem 1]{maurer2004note} characterized the second factor in the right-hand-side of the inequality, and we noted that following an analysis from~\citet[Lemma 19]{germain2015risk} it could be confined in the range $[\sqrt{n}, \sqrt{2n+2}]$. Therefore, re-organizing the terms, letting $\xi(n)$ be the said therm, and letting $\beta = \nicefrac{\xi(n)}{t}$ results in the following result.

\begin{proposition}[{Adaptation of the 
small-kl inequality~\citep[Section 2.5]{seldinNotes}}]
    \label{prop:small_kl}
    Consider a loss with a range bounded in $[0,1]$. Then, for all $\beta \in (0,1)$ and all $w \in \cW$, with probability no smaller than $1 - \beta$
    \begin{equation*}
        \relentber \big( \emprisk(w,S) \Vert \poprisk(w) \big) \leq \frac{\log \frac{\xi(n)}{\beta}}{n}.
    \end{equation*}
\end{proposition}

Although the bound may appear complicated, it states that the empirical risk and the population risk get ``close'' at a fast rate, where the ``closeness'' is measured by $\relentber$. Applying Pinsker's inequality from~\Cref{lemma:pinsker-inequality} or different relaxation of the stronger \citet{marton1996measure}'s bound (see \citep[Corollaries 2.19 and 2.20]{seldinNotes}) can give us a better idea of the rate that this bound implies. In this case, we can use our developments from~\Cref{subsec:a_fast_rate_bound_mi} to note that the small-kl inequality is equivalent to the following result.

\begin{proposition}
    \label{prop:classical_fast_rate}
    Consider a loss with a range bounded in $[0,1]$. Then, for all $\beta \in (0,1)$ and all $w \in \cW$, with probability no smaller than $1 - \beta$
    \begin{equation*}
        \poprisk(w) \leq c \gamma \log \mleft( \frac{\gamma}{ \gamma - 1} \mright) \cdot \emprisk(w,S) + \frac{c \gamma \log \frac{\xi(n)}{\beta}}{n} + \gamma \kappa(c),
    \end{equation*}
    \emph{simultaneously} for all $c \in (0,1]$ and all $\gamma > 1$, where $\kappa(c) \coloneqq 1 - c (1 - \log c)$.
\end{proposition}

\looseness=-1 \Cref{prop:classical_fast_rate} tells us that, when the loss is bounded, we can expect a \emph{high-probability fast-rate bound}. The caveat is that now the empirical risk is not used as the estimator of the population risk, but a scaled and re-centered version of it.

\section{The Origin of PAC-Bayesian Bounds}
\label{sec:origin_pac_bayesian_bounds}

PAC-Bayesian bounds were originally introduced by~\citet{shawe1996framework} in the context of \emph{binary classification and SRM in the realizable setting}.\footnote{To be precise, results of this type can be traced back even to the MDL principle from~\citet{rissanen1978modeling} and its connection to the population risk from~\citep{barron1991complexity,barron1991minimum,kearns1995experimental}.} Recall from~\Cref{sec:mdl_and_occams_razor}, that in SRM we decompose the hypothesis class $\cW$ into a countable union of hypothesis classes $\cW = \cup_{k=1}^\infty \cW_k$ such that $\cW_k \subseteq \cW_{k+1}$ for all $k \in \bN$ and where the complexity of each class is non-decreasing in $k$. In particular, \citet{shawe1996framework} considered that hypothesis classes with an increasing VC dimension, that is, $\vcdim(\cW_k) = k$.  Under this setting, the authors assigned a ``preference'' to each of the hypothesis classes determined by a probability distribution $\bQ$: larger values of $\bQ[k]$ meant a larger preference for the hypothesis class $\cW_k$. Then, they could show that for every hypothesis $w \in \cW_k$, with probability no less than $1 - \beta$
\begin{equation*}
    \poprisk(w,S) \leq 2 \emprisk(w,S) + \frac{ - \log \bQ[k] + 4k \log \frac{2en^2}{k}}{n}.\footnote{The original bound also considers a secondary distribution over the number of errors made by the hypothesis $w$. We simplified the equation assuming a uniform distribution.}
\end{equation*}

This allowed the learner to enjoy the generalization guarantees from learning with a class of a finite VC dimension $k$ while searching through a much larger class of hypotheses at the expense of the extra term $- \nicefrac{\log \bQ[k]}{n}$ coming from the preference given to classes with the said VC dimension. 

This result was later extended by~\citet{mcallester1998some,mcallester1999pac,mcallester2003pac} in several ways. First~\citep{mcallester1998some}, he noted that as long as the loss has a bounded range and hypothesis class is discrete, that is, $\cW = \cup_{k=1}^\infty \{ w_k \}$, one can give a preference $\bQ[w_k]$ to each of the hypotheses $w$ and combine Hoeffding's inequality~\eqref{eq:hoeffding} and the union bound to recover the results from MDL principle of~\eqref{eq:mdl_generalization} in~\Cref{sec:mdl_and_occams_razor}. Later~\citep{mcallester1999pac,mcallester2003pac}, he realized that one does not need to consider only the realization of the algorithm's output, but can consider the algorithm as a Markov kernel $\bP_W^S$ (very much like we did throughout~\Cref{ch:expected_generalization_error}, although this happened years prior to those developments). With these realizations, one can consider the expected population and empirical risks of the algorithm's output \emph{after} observing the training data, namely $\bE^S[\poprisk(W)]$ and $\bE^S[\emprisk(W,S)]$ respectively. Moreover, in this way, one does not need to consider a countable hypothesis space. It is enough to consider a ``preference'' distribution $\bQ$ on the hypothesis space $\cW$ and a randomized algorithm with an associated Markov kernel $\bP_W^S$ that returns distributions on the hypothesis space after observing the data. Then, the generalization error, instead of increasing by the inverse of the probability of the selected hypothesis as in~\eqref{eq:mdl_generalization}, increases by the discrepancy of the algorithm's output hypothesis distribution with respect to the preference distribution, that is, $\relent(\bP_W^S \Vert \bQ)$. However, for discrete hypothesis classes and deterministic algorithms, that is, $\bP_W^{S=s} = \bI_{\{ w = w_s \}}(w)$, this formulation still recovers the results from the MDL principle of~\eqref{eq:mdl_generalization} since $\relent(\bI_{\{ w = w_s \}} \Vert \bQ) = - \log \bQ[w_s]$.

\begin{theorem}[{Adaptation of~\citet[Theorem 2]{mcallester2003pac}}]
    \label{th:mcallester}
    Consider a loss with a range bounded in $[a,b]$ and let $\bQ$ be distribution on $\cW$ independent of $S$. Then, for all $\beta \in (0,1)$, with probability no smaller than $1 - \beta$
    \begin{equation*}
        \bE^S \big[ \gen(W,S) \big] \leq (b-a) \sqrt{\frac{\relent(\bP_W^S \Vert \bQ) + \log \frac{\xi(n)}{\beta}}{2n}}
    \end{equation*}
    holds \emph{simultaneously} for all posteriors $\bP_W^S$.
\end{theorem}

The term $\xi(n) \in [\sqrt{2n}, \sqrt{2n} + 2]$ is the same discussed previously in~\Cref{prop:small_kl}. This version of the theorem is slightly tighter than the original formulation, and we will see how to obtain it next in~\Cref{sec:bounds_bounded_losses}.

The term \emph{PAC-Bayesian} originated from~\citet{mcallester1998some,mcallester1999pac,mcallester2003pac}. He noted that the preference distribution $\bQ$ had a resemblance to the \emph{prior} distribution in Bayesian statistics. It has to be chosen before observing the training data and running the algorithm, and based on our \emph{prior} knowledge of which might be the hypothesis achieving a small risk. Similarly, the algorithm's output distribution $\bP_W^S$ was reminiscent to the \emph{posterior} distribution in that framework since it is the algorithm's distribution \emph{after} observing the training data. However, there are important, key differences between the two frameworks, which sometimes cause confusion. As clarified by~\citet{rivasplata2022pac}: ``The PAC-PAC-Bayes prior acts as an analytical device and may or may not be used by the algorithm, and the PAC-Bayes posterior is unrestricted and so it may be different from the posterior that would be obtained through Bayesian inference''. Despite these differences, in the PAC-Bayes community, the term $\bP_W^S$ is still referred to as the \emph{posterior} and term $\bQ$ as the \emph{prior}, and we will continue to use these terms throughout the monograph. Similarly, we will often refer to the relative entropy $\relent(\bP_W^S \Vert \bQ)$ as the \emph{dependency term} and to the ratio $\frac{1}{n} \big( \relent(\bP_W^S \Vert \bQ) + \log \frac{\xi(n)}{\beta} \big)$ as the (normalized) \emph{complexity term}, paying tribute
to the resemblance of the term to the hypothesis class complexity term from the classical generalization theory (\Cref{sec:pac_learning}).\footnote{In~\citep{rodriguez2023morepac}, we refer to the term as the \emph{dependence-confidence term}, to disambiguate and clarify better its components.}

Another important thing to note from~\Cref{th:mcallester} is that it holds \emph{simultaneously} for all posteriors $\bP_W^S$. Therefore, after observing a training set $s \in \cZ^n$, we may readily optimize the algorithm (which is characterized by its posterior) to minimize the PAC-Bayesian bound. That is, we may optimize both the empirical risk and the generalization error bound simultaneously.

Note how~\Cref{th:mcallester} is essentially a PAC-Bayesian analogue of Hoeffding's inequality~\eqref{eq:hoeffding}. The main difference is the appearance of the relative entropy term $\relent(\bP_W^S \Vert \bQ)$, which measures the discrepancy between the posterior distribution and the prior. Recalling the intuition behind the relative entropy from~\Cref{subsec:relative_entropy}, for a fixed training set $s$, this term measures the inefficiency of assuming that the hypothesis $W$ will be distributed according to the prior distribution $\bQ$, when in reality it is distributed according to the posterior $\bP_W^{S=s}$. Since the prior distribution can be chosen arbitrarily, assume that the prototypical distribution on samples from the data distribution $\bQ = \bP_W^S \circ \bP_S$ is chosen. Then, this term tells us how much the algorithm's output depends on the particular realization of the training set, connecting the generalization error with \emph{overfitting}, a phenomenon occurring when the output hypothesis describes very well the training data but fails to describe the underlying distribution. 

The similarity to the generalization bounds from~\Cref{ch:expected_generalization_error}, particularly to~\Cref{th:mi_bound_bounded}, is natural. Both results can be obtained using the Donsker and Varadhan lemma or the Gibbs variational principle from~\Cref{lemma:dv_and_gvp}. However, to obtain PAC-Bayesian bounds there are a couple of subtleties that one needs to take into account, as we will see shortly in the following sections.

Since the development of this result, many works in the field have focused on two main tasks: (i) refining the bound to better characterize the population risk for bounded losses and (ii) extending this bound relaxing their assumptions or their setting. \Cref{sec:bounds_bounded_losses} will focus on the first task, while \Cref{sec:bounds_unbounded_losses,sec:anytime_validity} will focus on the second one.

\paragraph{Limitations of the Relative Entropy as the Dependency Measure.} The relative entropy is a convenient measure of the dependency of the hypothesis returned by the algorithm and the training set. It has interesting information-theoretic interpretations, it has connections to classical approaches to generalization, and it is easy to manipulate. However, it has its shortcomings. First, if the hypothesis space is continuous and the prior distribution does not have a pmf, then it cannot accommodate deterministic algorithms. A solution to this problem is to quantize the hypothesis class and reduce the framework to the MDL as done by~\citet{lotfi2022pac}. However, similarly to what happened with the mutual information, there are other situations where the generalization error is small and the relative entropy is large, leading to vacuous guarantees. Some examples of this issue, including some developed by us, are given in~\citep{bassily2018learners,livni2020limitation,haghifam2023limitations}. To deal with these problems, a relatively new research path consists of substituting the relative entropy as the dependency measure by different metrics like other $f$-divergences~\citep{esposito2021generalization,ohnishi2021novel,kuzborskij2024better}, Rényi divergences~\citep{begin2016pac,esposito2021generalization,hellstrom2020generalization}, or integral probability metrics like the Wasserstein distance~\citep{amit2022integral,haddouche2023wasserstein,viallard2024learning,viallard2024tighter}. 

\paragraph{Focus on High-Probability Bounds.} In this chapter we focus on high-probability bounds, that is, bounds whose dependence on the probability parameter $\beta$ is logarithmic. One could easily translate the bounds in expectation from~\Cref{ch:expected_generalization_error} using Markov's inequality~\eqref{eq:markov} or other techniques like the monitor technique from~\citet{bassily2016algorithmic}. For example, applying Markov's inequality using~\Cref{th:wasserstein_full} ensures that, if the loss $\ell(\cdot, z)$ is $L$-Lipschitz for every $z \in \cZ$, then, for every $\beta \in (0,1)$, with probability no smaller than $1-\beta$
\begin{equation*}
    \bE^S \big[ \gen(W,S) \big] \leq \frac{2L}{\beta} \cdot \bW_\rho (\bP_W^S , \bQ).
\end{equation*}
\looseness=-1 However, these bounds are not of high probability and thus are not considered.

\paragraph{Focus on the Standard Setting.} Like for the guarantees in expectation from~\Cref{ch:expected_generalization_error}, one can also consider the randomized-subsample setting from~\Cref{sec:bounds_using_conditional_mutual_information} to derive new PAC-Bayesian bounds. This has been studied, for example, by~\citet{hellstrom2020generalization,hellstrom2022new} and~\citet{grunwald2021pac}. However, these bounds hold for a realization of the supersample $\tilde{s} \in \cZ^{n \times 2}$ and the indices $u \in \{ 0, 1 \}^n$. This is often impractical for the desiderata 2. and 3. of PAC-Bayesian bounds discussed in the introduction of this chapter. The reason is that if $n$ test instances are available, then classical concentration inequalities will always be tighter than PAC-Bayesian bounds derived from the randomized subsample setting. This is why we only consider PAC-Bayesian bounds derived from the standard setting in this chapter. Nonetheless, it is worth mentioning that these kinds of bounds are still useful to gain some understanding: \citet{grunwald2021pac} showed these bounds guarantee the generalization of hypothesis classes with a finite VC dimension and for compression schemes and \citet{hellstrom2022new} made a similar remark for hypothesis classes with a finite Natarajan dimension.

\section{PAC-Bayesian Bounds for Bounded Losses}
\label{sec:bounds_bounded_losses}

There are many important inequalities in the PAC-Bayesian literature, especially in the case where the loss is bounded. These bounds are often presented for losses with a range in $[0,1]$, which includes the interesting 0--1 loss for classification tasks. For the rest of this section, we will keep this convention, always having in mind that one can scale and center the losses to lay in the desired range.

The Seeger--Langford~\citep{langford2001bounds, seeger2002pac} and \citet[Theorem 1.2.6]{catoni2007pac}'s bounds are known to be (two of) the tightest bounds in this setting, see, for example, the discussion in~\citep{foong2021tight}. Both of them can be derived from \citet[Theorem 2.1]{germain2009pac}'s convex function bound. Below we state the extension from \citet{begin2014pac} that lifts the double absolute continuity requirement from the original statement noted by \citet{haddouche2021pac}. The original statement from~\citet{begin2014pac} only considers bounded losses and convex functions from $[0,1]^2 \to \bR$. Since these assumptions are not needed in the proof, we will state the theorem in a more general form, which will be useful later in~\Cref{sec:bounds_unbounded_losses}. 

\begin{theorem}[{Extension of \citet[Theorem 4]{begin2016pac}}]
\label{th:germain_convex_pac_bayes}
    Let $\bQ$ be any prior independent of $S$ and let $W'$ be distributed according to $\bQ$. Then, for every convex function $f: \bR^2 \to \bR$ such that $\bE \big[ \exp \big( n f \big( \emprisk(W',s), \poprisk(W') \big) \big) \big] < \infty$ for all $s \in \cZ^n$, and every $\beta \in (0,1)$, with probability no smaller than $1-\beta$
    \begin{align*}
        f \big( \bE^S \big[\emprisk(W,S)\big] , &\bE^S \big[ \poprisk(W)\big] \big)  \\
        &\leq \frac{1}{n} \bigg( \relent \big( \bP_W^S \Vert \bQ \big) + \log \frac{1}{\beta} + \log \bE \Big[e^{n f \big(  \emprisk(W',S) , \poprisk(W') \big)} \Big]\bigg)
    \end{align*}
    holds \emph{simultaneously} for all posteriors $\bP_W^S$.
\end{theorem}

To reach this generic bound one can consider the Gibbs variational principle from~\Cref{lemma:dv_and_gvp} (or the dual version of the Donsker and Varadhan lemma) once again. 
For a fixed training set $s$, in~\eqref{eq:gvp}, let the measurable space be $\cW$, the measurable function be $n f(\emprisk(w,s), \poprisk(w))$ and $\cP_\bQ(\cW)$ be the set of all probability measures $\bP_W^{S=s}$ such that $\bP_W^{S=s} \ll \bQ$. Then, 
\begin{align}
    \log \bE^{S=s} &\mleft[ e^{n f(\emprisk(W',s), \poprisk(W'))} \mright] = \nonumber \\
    \label{eq:germain_intermediate}
    &\sup_{\bP_W^{S=s} \in \cP_\bQ(\cW)} \mleft \{ n \bE^{S=s} \mleft[ n f(\emprisk(W,s), \poprisk(w)) \mright] - \relent(\bP_W^{S=s} \Vert \bQ) \mright\},
\end{align}

This equation becomes the result from~\Cref{th:germain_convex_pac_bayes} after two considerations.
First, since the equation holds for all $s \in \cZ^n$, then it holds for a random training set $S$ almost surely. Therefore, we may apply Markov's inequality~\eqref{eq:markov} to the term inside the logarithm of the left-hand-side of~\eqref{eq:germain_intermediate}. More precisely, let $X = \bE^S [ \exp(n f(\emprisk(W',S), \poprisk(W')))]$ and $t = \nicefrac{1}{\beta} \cdot \bE [ \exp(n f(\emprisk(W',S), \poprisk(W')))]$, then
\begin{equation*}
    \bP \mleft[ \bE^S \mleft[e^{n f(\emprisk(W',S), \poprisk(W'))} \mright] \geq \frac{1}{\beta} \cdot \bE \mleft[e^{n f(\emprisk(W',S), \poprisk(W'))} \mright]\mright] \leq \beta.
\end{equation*}
In this way, we have that with probability no larger than $\beta$
\begin{align*}
    \sup_{\bP_W^{S=s} \in \cP_\bQ(\cW)} \Big \{ n \bE^{S=s} \big[ n f(\emprisk(W,s), \poprisk(w)) &\big] - \relent(\bP_W^{S=s} \Vert \bQ) \Big\} \\
    &\geq 
    \log \frac{1}{\beta} + \log \bE \Big[e^{n f \big(  \emprisk(W',S) , \poprisk(W') \big)} \Big]. 
\end{align*}
In other words, there exists a posterior $\bP_W^S \in \cP_\bQ(\cW)$ a.s. such that, with probability no larger than $\beta$
\begin{align*}
     n \bE^{S=s} \mleft[ n f(\emprisk(W,s), \poprisk(w)) \mright] &- \relent(\bP_W^{S=s} \Vert \bQ) \\
     &\geq 
    \log \frac{1}{\beta} + \log \bE \Big[e^{n f \big(  \emprisk(W',S) , \poprisk(W') \big)} \Big]. 
\end{align*}

Second, we may include the expectation with respect to $\bP_W^S$ inside of the convex function $f$ maintaining the inequality due to Jensen's inequality. Finally, we may re-arrange the equation to prove~\Cref{th:germain_convex_pac_bayes}. Moreover, by the convention that $\relent(\bP_W^S \Vert \bQ) \to \infty$ when $\bP_W^S \not \ll \bQ$, the bound holds trivially in that case.

This general-purpose bound is useful as an appropriate choice of the convex function $f$ can be used to recover different PAC-Bayesian bounds. For example, considering $f(\hat{r},r) = \frac{(2n-1)}{n} \cdot (\hat{r}-r)^2$ recovers~\citet{mcallester2003pac}'s bound from~\Cref{th:mcallester} up to constants.\footnote{It is often mentioned in the literature that one can recover~\Cref{th:mcallester} exactly by choosing $f(\hat{r},r) = 2 (\hat{r}-r)^2$. However, this does not allow us to use \citet{mcallester2003pac}'s proof technique.} Similarly, choosing $f(\hat{r},r)= \relentber(\hat{r} \Vert r)$ for $(\hat{r},r) \in [0,1]^2$ recovers the improved Seeger--Langford bound~\citep{langford2001bounds, seeger2002pac, maurer2004note} and choosing  $f(\hat{r},r) = - \log \big(1 - r(1-e^{\nicefrac{-\lambda}{n}})\big) - \nicefrac{\lambda \hat{r}}{n}$ for $(\hat{r},r) \in [0,1]^2$ recovers \citet[Theorem 1.2.6]{catoni2007pac}'s bound. 

\begin{theorem}[{Improved Seeger--Langford bound~\citep{langford2001bounds,seeger2002pac,maurer2004note}}]
\label{th:seeger_langford_pac_bayes}
    Consider a loss with a range bounded in $[0,1]$ and let $\bQ$ be any prior independent of $S$. Then, for every $\beta \in (0,1)$, with probability no smaller than $1-\beta$
    \begin{equation*}
        \label{eq:seeger_langford_pac_bayes}
        \relentber \big( \bE^S [\emprisk(W,S)] \Vert \bE^S [\poprisk(W)] \big) \leq \frac{\relent(\bP_W^S \Vert \bQ) + \log \frac{\xi(n)}{\beta}}{n}
    \end{equation*}
    holds \emph{simultaneously} for all posteriors $\bP_W^S$.
\end{theorem}

\begin{theorem}[{\citet[Theorem 1.2.6]{catoni2007pac}}]
\label{th:catoni_pac_bayes}
    Consider a loss with a range bounded in $[0,1]$ and let $\bQ$ be any prior independent of $S$. Then, for every $\lambda > 0$ and every $\beta \in (0,1)$, with probability no smaller than $1-\beta$
    \begin{equation*}
        \label{eq:catoni_pac_bayes}
        \bE^S \big[ \poprisk(W) \big] \leq \frac{1}{1 - e^{- \frac{\lambda}{n}}} \Bigg( 1 - e^{- \frac{\lambda \bE^S \big[ \emprisk(W,S) \big]}{n}   - \frac{\relent(\bP_W^S \Vert \bQ) + \log \frac{1}{\beta}}{n}} \Bigg).
    \end{equation*}
    holds \emph{simultaneously} for all posteriors $\bP_W^S$.
\end{theorem}

Recall~\Cref{prop:small_kl}. The Seeger--Langford bound~\citep{seeger2002pac,langford2001bounds} is a PAC-Bayes analogue of the \emph{small-kl} concentration inequality. However, it is hindered by its lack of interpretability. Moreover, it is difficult to optimize the posterior to minimize the bound and find an appropriate posterior $\bP_{W}^{S}$. This is due to the non-convexity of the bound with respect to the posterior $\bP_W^S$~\citep{thiemann2017strongly}, as well as the fact that it cannot be expressed explicitly as a function of the empirical risk $\bE^S \big[ \emprisk(W,S) \big]$ and the
dependency term $\relent(\bP_W^S \Vert \bQ)$~\citep{germain2009pac}. On the other hand, while \citet{catoni2007pac}'s bound is minimized by the Gibbs posterior  $\sfp_W^S(w) \propto \sfq(w) e^{- \lambda \bE^S \big[ \emprisk(w,S) \big]}$, it still lacks interpretability and depends on an arbitrary parameter $\lambda$ that has to be chosen \emph{before} the draw of the data.

To remedy these issues, several works relax the Seeger--Langford bound~\citep{langford2001bounds,seeger2002pac} using lower bounds on the relative entropy~\citep{tolstikhin2013pac,thiemann2017strongly,rivasplata2019pac}. For example, one may recover~\citet{mcallester2003pac}'s bound from~\Cref{th:mcallester} exactly employing Pinsker's inequality (\Cref{lemma:pinsker-inequality}). Then, as discussed in~\Cref{subsec:bounded_losses_classical}, one may instead use relaxations of the stronger \citet{marton1996measure}'s bound like \citep[Corollaries 2.19 and 2.20]{seldinNotes}). \citet{tolstikhin2013pac} use \citep[Corollary 2.20]{seldinNotes} and \citet{thiemann2017strongly} and \citet{rivasplata2019pac} use \citep[Corollary 2.19]{seldinNotes}. This relaxation results in an intractable PAC-Bayes bound and, for this reason, \citet{thiemann2017strongly} relax it further using the inequality $\sqrt{xy} \leq \frac{1}{2}(\lambda x + \nicefrac{y}{\lambda})$ for all $\lambda > 0$ to obtain a \emph{fast-rate} bound, and \citet{rivasplata2019pac} solve the resulting quadratic inequality for $\sqrt{\bE^S \big[ \poprisk(W) \big]}$ to obtain a \emph{mixed-rate} bound.

\begin{theorem}[{\citet[Theorem 3]{thiemann2017strongly}'s fast-rate bound}]
\label{th:thiemann_pac_bayes}
    Consider a loss function with a range bounded in $[0,1]$ and let $\bQ$ be any prior independent of $S$. Then, for every $\beta \in (0,1)$, with probability no smaller than $1-\beta$
    \begin{equation*}
        \label{eq:thiemann_pac_bayes}
        \bE^S \big[ \poprisk(W) \big] \leq \inf_{\lambda \in (0,2)} \Bigg \{ \frac{\bE^S \big[ \emprisk(W,S) \big]}{1 - \frac{\lambda}{2}} + \frac{\relent(\bP_W^S \Vert \bQ) + \log \frac{\xi(n)}{\beta}}{n \lambda (1- \frac{\lambda}{2})} \Bigg \}.
    \end{equation*}
    holds \emph{simultaneously} for all posteriors $\bP_W^S$.
\end{theorem}

\begin{theorem}[{\citet[Theorem 1]{rivasplata2019pac}'s mixed-rate bound}]
\label{th:rivasplata_pac_bayes}
    Consider a loss function with a range bounded in $[0,1]$ and let $\bQ$ be any prior independent of $S$. Then, for every  $\beta \in (0,1)$, with probability no smaller than $1-\beta$
    \begin{align*}
        \label{eq:rivasplata_pac_bayes}
        \bE^S \big[ &\poprisk(W) \big] \leq \bE^S \big[ \emprisk(W,S) \big] + \frac{\relent(\bP_W^S \Vert \bQ) + \log \frac{\xi(n)}{\beta}}{n} \nonumber \\ &+ \sqrt{2 \bE^S \big[ \emprisk(W,S) \big] \cdot \frac{\relent(\bP_W^S \Vert \bQ) + \log \frac{\xi(n)}{\beta}}{n } + \Bigg(\frac{\relent(\bP_W^S \Vert \bQ) + \log \frac{\xi(n)}{\beta}}{n } \Bigg)^2}.
    \end{align*}
    holds \emph{simultaneously} for all posteriors $\bP_W^S$.
\end{theorem}

Originally, \citet{rivasplata2019pac} present their bound in a different form, but this \emph{mixed-rate} form shows explicitly the combination of a \emph{fast-rate} term and an \emph{amortized slow-rate} term. Moreover, this form makes it easy to see that the bound is tighter than \citep[Equation (3)]{tolstikhin2013pac}, as their bound can be recovered using the inequality $\sqrt{x + y} \leq \sqrt{x} + \sqrt{y}$ on~\Cref{th:rivasplata_pac_bayes}.

In \citep{rodriguez2023morepac}, we noted that a strengthened version of \citet{catoni2007pac}'s bound that holds \emph{simultaneously} for all $\lambda > 0$ can be obtained from the Seeger--Langford~\citep{langford2001bounds,seeger2002pac} bound at the small cost of $\log \xi(n)$ in the complexity term. The proof follows applying the Donsker and Varadhan variational representation of the relative entropy from~\Cref{lemma:dv_and_gvp} to the binary relative entropy $\relentber$ exactly like we did to prove~\Cref{th:catoni_mi} in~\eqref{eq:relentber_dv_before_change_variable} of~\Cref{subsec:a_fast_rate_bound_mi}. This was also observed by~\citet[Proposition 2.1]{germain2009pac} and proved with different techniques to ours in~\citep[Chapter 20]{catoni2015pac} and~\citep[Lemmata E1 and E2]{foong2021tight}, although it was not stated explicitly as a PAC-Bayesian bound.

\begin{theorem}
\label{th:catoni_pac_bayes_uniform}
    Consider a loss function with a range bounded in $[0,1]$ and let $\bQ$ be any prior independent of $S$. Then, for every $\beta \in (0,1)$, with probability no smaller than $1-\beta$
    \begin{equation*}
        \bE^S \big[ \poprisk(W) \big] \leq \frac{1}{1 - e^{- \frac{\lambda}{n}}} \Bigg[1 - e^{- \frac{\lambda \bE^S \big[ \emprisk(W,S) \big]}{n}   - \frac{\relent(\bP_W^S \Vert \bQ) + \log \frac{\xi(n)}{\beta}}{n}} \Bigg] 
    \end{equation*}
    holds \emph{simultaneously} for all posteriors $\bP_W^S$ and all $\lambda > 0$.
\end{theorem}

Note that the fact that the bound holds \emph{simultaneously} for all $\lambda > 0$ implies that the bound can be equivalently written as: for all $\beta \in (0,1)$, with probability no smaller than $1 - \beta$
\begin{equation*}
    \bE^S \big[ \poprisk(W) \big] \leq  \inf_{\lambda > 0} \mleft \{ \frac{1}{1 - e^{- \frac{\lambda}{n}}} \Bigg[1 - e^{- \frac{\lambda \bE^S \big[ \emprisk(W,S) \big]}{n}   - \frac{\relent(\bP_W^S \Vert \bQ) + \log \frac{\xi(n)}{\beta}}{n}} \Bigg] \mright\}
\end{equation*}
holds \emph{simultaneously} for all posteriors $\bP_W^S$.
This bound is an explicit expression of the Seeger--Langford bound~\citep{langford2001bounds,seeger2002pac} in terms of $\bE^S \big[ \emprisk(W,S) \big]$ and $\relent(\bP_W^S \Vert \bQ)$. Compared to \citet{catoni2007pac}'s \Cref{th:catoni_pac_bayes}, this bound holds \emph{simultaneously} for all $\lambda>0$, making it useful for finding numerical population risk certificates without the need to pay an extra price for the parameter search. It also allows for an iterative procedure for obtaining a good posterior by updating the posterior $\bP_W^S$ and parameter $\lambda$ alternately. We note that, contrary to the statement from the Seeger--Langford bound in \Cref{th:seeger_langford_pac_bayes}, this statement tells us that the optimal posterior is given by the Gibbs distribution $\sfp_W^S(w) \propto \sfq(w) \cdot e^{- \lambda \bE^S \big[ \emprisk(w,S) \big]}$. However, finding the global optimum for the parameter $\lambda$ is tedious, and the function is not convex in that parameter.

Again, as we did to prove the fast-rate bound from~\Cref{th:mi_bound_bounded_fast_rate} of~\Cref{subsec:a_fast_rate_bound_mi}, we may apply the variational representation of the relative entropy borrowed from $f$-divergences (\Cref{cor:variational_representation_relative_entropy}) using~\eqref{eq:relentber_f_div_after_change_variable} to prove a PAC-Bayes analogue to the fast-rate bound we presented in~\Cref{prop:classical_fast_rate}. Similarly, we may follow the steps to show that~\Cref{th:catoni_mi} and~\Cref{th:mi_bound_bounded_fast_rate} are equivalent. Namely, let  $\lambda = n \log (\nicefrac{\gamma}{\gamma - 1})$ in~\Cref{th:catoni_pac_bayes_uniform}, which implies that $\gamma > 1$. Then, note that the function $1 - e^{-x}$ is a non-decreasing, concave, continuous function for $x > 0$ and therefore it can be upper bounded by its envelope, that is $1 - e^{-x} = \inf_{a > 0} \{ e^{-a} x + 1 - e^a (1+a) \}$. Using this envelope in the equation of~\Cref{th:catoni_pac_bayes_uniform} and letting $c = e^{-a} \in (0,1]$ results in the following theorem.

\begin{theorem}[{Fast-rate bound}]
\label{th:fast_rate_bound_strong}
    Consider a loss function with a range bounded in  $[0,1]$ and let $\bQ$ be any prior independent of $S$. Then, for every $\beta \in (0,1)$, with probability no smaller than $1-\beta$
    \begin{equation*}
        \bE^S \big[ \poprisk(W) \big] \leq c \gamma \log \Big(\frac{\gamma}{\gamma - 1} \Big) \cdot \bE^S \big[ \emprisk(W,S) \big] + c \gamma \cdot \frac{\relent(\bP_W^S \Vert \bQ) + \log \frac{\xi(n)}{\beta}}{n} + \kappa(c) \gamma,
    \end{equation*}
    holds \emph{simultaneously} for all posteriors $\bP_W^S$, all $\gamma > 1$, and all $c \in (0,1]$, where $\kappa(c) \coloneqq 1 - c(1 - \log c) $.
\end{theorem}

\begin{remark}
    \label{rem:fast_rate_strong_bounded_b}
    Note how if the loss $\ell$ has a range bounded in $[0,b]$, as will be the case in~\Cref{subsec:losses_with_bounded_moment}, then $\tilde{\ell} = \nicefrac{\ell}{b}$ has a range bounded in $[0,1]$. Then, the bound from~\Cref{th:fast_rate_bound_strong} can be trivially extended to this case first bounding the population risk of the scaled loss and then multiplying both sides of the inequality with $b$. To be precise, the right-hand side of~\Cref{th:fast_rate_bound_strong} equation would be
    \begin{equation*}
        c \gamma \log \Big(\frac{\gamma}{\gamma - 1} \Big) \cdot \bE^S \big[ \emprisk(W,S) \big] + b c \gamma \cdot \frac{\relent(\bP_W^S \Vert \bQ) + \log \frac{\xi(n)}{\beta}}{n} + b \kappa(c) \gamma.
    \end{equation*}
\end{remark}

Therefore, the Seeger--Langford bound (\Cref{th:seeger_langford_pac_bayes}), the strengthened Catoni's bound (\Cref{th:catoni_pac_bayes_uniform}), and this fast-rate bound (\Cref{th:fast_rate_bound_strong}) are equally tight. This is important since it means that the Seeger--Langford bound~\citep{langford2001bounds, seeger2002pac} can be exactly described with a linear combination of the empirical risk and the complexity term, where the coefficients of this combination and the bias vary depending on the data realization. This could have been hypothesized by observing the derivatives of the Seeger--Langford bound~\citep{langford2001bounds,seeger2002pac} from \citet[Appendix A]{reeb2018learning}, and a proof is now available. Furthermore, the optimal posterior of this bound is given by the Gibbs distribution $\sfp_W^S(w) \propto \sfq(w) \cdot e^{-n \log \big(\frac{\gamma}{\gamma -1}\big) \emprisk(w,S)}$, where the value of $\gamma$ depends on the dataset realization $s$.

We recall the influence of the parameters $\gamma$ and $c$ in the fast-rate bound from~\Cref{th:mi_bound_bounded_fast_rate}, as the discussion applies analogously here. The parameter $\gamma$ controls the influence of the empirical risk compared to the normalized complexity: if the empirical risk is large relative to the normalized complexity, then $\gamma$ is larger and the normalized complexity coefficient increases; if instead, the empirical risk is small or even close to interpolation, then $\gamma$ is close to $1$ and the empirical risk coefficient increases. In particular, for a fixed value of $c$, the optimal value of $\gamma$ is
\begin{align*}
	\gamma &= 1 + \left[ -1 - \mathtt{W} \left( - \exp \left(-1 - \frac{c\cdot \frac{\relent(\bP_W^S \Vert \bQ) + \log \frac{\xi(n)}{\beta}}{n} + \kappa(c)}{c \cdot \bE \big[ \emprisk(W,S) \big]} \right) \right) \right]^{-1} \\
	&\approx 1 + \left[ \sqrt{2 \cdot \frac{c\cdot \frac{\relent(\bP_W^S \Vert \bQ) + \log \frac{\xi(n)}{\beta}}{n} + \kappa(c)}{c \cdot \bE \big[ \emprisk(W,S) \big]}} +  \frac{5}{6} \cdot \frac{c\cdot \frac{\relent(\bP_W^S \Vert \bQ) + \log \frac{\xi(n)}{\beta}}{n} + \kappa(c)}{c \cdot \bE \big[ \emprisk(W,S) \big]}\right]^{-1},
\end{align*}
\looseness=-1 where $\mathtt{W}$ is the Lambert W function and the $-1$ branch is approximated following~\citep{chatzigeorgiou2013bounds}.

The parameter $c \in (0,1]$ controls how much weight is given to the empirical risk and normalized complexity terms compared to a bias. For larger values of the empirical risk and the normalized complexity term, the value of $c$ is small, decreasing their contribution to the bound and increasing the contribution of the bias $\kappa(c) \in [0,1)$. If the empirical risk and the normalized complexity term are smaller, then the value of $c$ approaches $1$, where the contribution of these two terms is only controlled by $\gamma$ and the bias is $0$. In fact, a weaker version of~\Cref{th:fast_rate_bound_strong} can be obtained considering this small empirical risk and small normalized complexity regime by letting $c = 1$.

\begin{corollary}
\label{cor:fast_rate_bound}
    Consider a loss function with a range bounded in $[0,1]$ and let $\bQ$ be any prior independent of $S$. Then, for every $\beta \in (0,1)$, with probability no smaller than $1-\beta$
    \begin{equation*}
        \label{eq:fast_rate_bound}
        \bE^S \big[ \poprisk(W) \big] \leq \gamma \log \Big(\frac{\gamma}{\gamma - 1} \Big) \cdot \bE^S \big[ \poprisk(W,S) \big] + \gamma \cdot \frac{\relent(\bP_W^S \Vert \bQ) + \log \frac{\xi(n)}{\beta}}{n} 
    \end{equation*}
    holds \emph{simultaneously} for all posteriors $\bP_W^S$ and all $\gamma > 1$.
\end{corollary}

This bound improves upon \citet{thiemann2017strongly}'s \Cref{th:thiemann_pac_bayes} as it is tighter for all values of the empirical risk and the dependency measure (see \Cref{app:comparison_fast_rate_bounds}). For instance, the value $\lambda = 1$ minimizes the multiplicative factor in the complexity term in \Cref{th:thiemann_pac_bayes}. Letting $\gamma = 2$ in \Cref{cor:fast_rate_bound} matches this factor and improves the multiplicative factor of the empirical risk from 2 to $2 \log 2 \approx 1.38$. Moreover, if we are in the \emph{realizable setting} and $\bE^S \big[ \emprisk(W,S) \big] = 0$ (that is, we are using an empirical risk minimizer), then letting $\gamma \to 1^+$ in this bound reveals that the fast rate can be achieved with multiplicative factor $1$, clarifying that the complexity term completely characterizes the population risk in this regime. Note that this is neither clear in \citet{thiemann2017strongly}'s nor \citet{rivasplata2019pac}'s bounds, where the multiplicative factor is 2.

However, substituting the value of the optimal $\gamma$ into~\Cref{th:fast_rate_bound_strong} or~\Cref{cor:fast_rate_bound} does not produce an interpretable bound. Nonetheless, following the same steps that we used to obtain~\Cref{th:mi_bound_bounded_mixed_rate} from~\Cref{cor:mi_bound_bounded_fast_rate_weaker}, the bound in \Cref{cor:fast_rate_bound} can be further relaxed to obtain a parameter-free mixed-rate bound that is tighter than \citet{rivasplata2019pac}'s mixed-rate and \citet{thiemann2017strongly}'s fast-rate bounds (see \Cref{app:comparison_fast_rate_bounds}).

\begin{theorem}[{Mixed-rate bound}]
\label{th:mixed_rate_bound}
    Consider a loss function with a range bounded in $[0,1]$ and let $\bQ$ be any prior independent of $S$. Then, for every $\beta \in (0,1)$, with probability no smaller than $1-\beta$
    \begin{align*}
        \label{eq:mixed_rate_bound}
        \bE^S \big[ \poprisk(W) \big] &\leq  \bE^S \big[ \emprisk(W,S) \big] +  \frac{\relent(\bP_W^S \Vert \bQ) + \log \frac{\xi(n)}{\beta}}{n} \\
        &+ \sqrt{2 \bE^S \big[ \emprisk(W,S) \big] \cdot \frac{\relent(\bP_W^S \Vert \bQ) + \log \frac{\xi(n)}{\beta}}{n}}
    \end{align*}
    holds \emph{simultaneously} for all posteriors $\bP_W^S$.
\end{theorem}

The mixed-rate bound presented in \Cref{th:mixed_rate_bound} provides a deeper insight into the relationship between the population risk, the empirical risk, and the complexity term. The bound grows linearly with both the empirical risk and the complexity term, with a correction term that reflects their interaction. Importantly, the bound is symmetric in these two terms, giving them equal importance. This may be beneficial for methods using PAC-Bayes bounds to optimize the posterior, such as PAC-Bayes with backprop~\citep{rivasplata2019pac, perez2021tighter}, where using the bound from \citet{thiemann2017strongly} or \citet{rivasplata2019pac} alone may cause the algorithm to disregard posteriors farther from the prior but that achieve lower population risk (see \Cref{app:pbb}).

\paragraph{Further Developments}

\citet{wu2022split} derived a ``split-kl'' inequality that competes with the Seeger--Langford bound~\citep{langford2001bounds, seeger2002pac} for ternary losses and  \citet{jang2023tighter} proved an even tighter bound via ``coin-betting''. However, their bounds are still neither easily interpretable nor directly aid in the selection of an appropriate posterior. Moreover, there are other advances on this front when further quantities are considered.  If the variance is known, \citet[Theorem 8]{seldin2012pac} and \citet[Theorem 9]{wu2021chebyshev} introduced, respectively, PAC-Bayes analogues to Bernstein and Bennett inequalities from~\Cref{prop:bernstein_inequality} and~\eqref{eq:bennet}. The PAC-Bayes Bernstein inequality was later improved by further bounding the variance using an empirical estimate of that quantity~\citep[Theorems 3 and 4]{tolstikhin2013pac} similarly to~\Cref{prop:empirical_bernstein_inequality}. Finally, \citet{mhammedi2019pac} derived a PAC-Bayes analogue to the unexpected Bernstein inequality where they use an empirical estimate of the second moment.

\section{PAC-Bayesian Bounds for Unbounded Losses}
\label{sec:bounds_unbounded_losses}

The goal of this section is to obtain PAC-Bayes analogues to classical concentration bounds or, at least, to find bounds with the same rates. First, in~\Cref{subsec:losses_with_bounded_CGF}, we will show how to obtain a PAC-Bayes analogue of the Chernoff inequality from~\Cref{prop:chernoff}. The proofs will follow what we did in~\Cref{subsec:a_slow_rate_bound_mi} with two main differences: first, we will adapt the proof to the PAC-Bayesian setting by employing Markov's inequality~\Cref{eq:markov} in order to have probabilistic statements with respect to the CGF, similarly to what we did above for the proof of~\Cref{th:germain_convex_pac_bayes}; and second, we will devise a technique to optimize data-dependent parameters in PAC-Bayesian bounds. This technique is of independent interest to optimize parameters that depend on random variables on any general probabilistic statement. Then, we will move on to losses with a bounded moment, where we can follow what we did in~\Cref{subsec:interpolating_between_slow_and_fast_rate} taking into account the same two considerations.

\subsection{Losses With a Bounded CGF}
\label{subsec:losses_with_bounded_CGF}

Recall from~\Cref{prop:chernoff}, that if the loss has CGF bounded by $\psi$ in the sense of \Cref{def:bounded_cgf} and the hypothesis $w$ is fixed, the empirical risk concentrates around the population risk at a rate $\psi_*^{-1} \big(\frac{1}{n} \log \nicefrac{1}{\beta}\big)$. More generally, this is true as long as the hypothesis $W'$ is distributed according to some distribution $\bQ$ and it is independent of the training data $S$. The reason is that, as we say in~\Cref{sec:bounds_using_mutual_information}, in this case, the empirical risk is an unbiased estimator of the population risk $\bE[\emprisk(W',S)] = \bE[\poprisk(W')]$ and we can always evaluate the CGF of the empirical risk as
\begin{equation*}
    \Lambda_{-\emprisk(W',S)}(n\lambda) = \bE \mleft[ e^{n \lambda \gen(W',S) )} \mright] = \bE_{w \sim \bQ} \mleft[ \bE  \mleft[ e^{n \lambda \gen(w,S) )} \mright] \mright],
\end{equation*}
and then $\Lambda_{-\emprisk(W',S)}(n\lambda) \leq n \psi(\lambda)$.

Therefore, it seems reasonable to decouple again the returned hypothesis $W$ from the training data $S$. To do so, we may either use the Donsker and Varadhan lemma or the Gibbs variational principle from~\Cref{lemma:dv_and_gvp} again, or just take advantage of the general-purpose~\Cref{th:germain_convex_pac_bayes}. In this case, consider the convex function $f(\hat{r},r) = \lambda(r - \hat{r})$ for some $\lambda > 0$. Then, for every $\beta \in (0,1)$ and every $\lambda > 0$, with probability no smaller than $1 - \beta$
\begin{equation*}
    \lambda \mleft( \bE^S \big[ \gen(w,S) \big] \mright) \leq \frac{1}{n} \mleft( \relent(\bP_W^S \Vert \bQ) + \log \frac{1}{\beta} + \log \bE \mleft[ e^{n \lambda \gen(W',S)} \mright] \mright)
\end{equation*}
holds \emph{simultaneously} for all posteriors $\bP_W^S$.

Finally, noting that if the loss has a CGF bounded by $\psi$, then $\Lambda_{-\emprisk(W',S)}(n\lambda) \leq n \psi(\lambda)$ and re-arranging the terms in the equation recovers the intermediate result from~\citet[Theorem 6]{banerjee2021information}.

\begin{lemma}[{\citet[Theorem 6]{banerjee2021information}}]
    \label{lemma:extension_banerjee}
    Consider a loss function $\ell$ with a bounded CGF (\Cref{def:bounded_cgf}). Let $\bQ$ be any prior independent of $S$. Then, for every $\beta \in (0,1)$ and every $\lambda \in (0,b)$, with probability no smaller than $1-\beta$
    \begin{equation*}
        \bE^S \big[ \gen(W,S) \big] \leq \frac{1}{\lambda} \Bigg( \frac{\relent(\bP_{W}^{S} \Vert \bQ) + \log \frac{1}{\beta}}{n} + \psi(\lambda) \Bigg)
    \end{equation*}
    holds \emph{simultaneously} for all posteriors $\bP_W^S$.
\end{lemma}

If we could optimize the parameter $\lambda$ in \Cref{lemma:extension_banerjee}, we would obtain a PAC-Bayes analogue to Chernoff's inequality (\Cref{prop:chernoff}). However, this is not possible since the optimal parameter depends on the data realization but needs to be selected \emph{before} the draw of this data~\citep[Remark 14]{banerjee2021information}.

To make this point clearer, fix a dataset $s$ and assume that the infimum
\begin{equation*}
    \inf_{\lambda \in (0,b) } \mleft\{ \frac{1}{\lambda} \Bigg( \frac{\relent(\bP_{W}^{S=s} \Vert \bQ) + \log \frac{1}{\beta}}{n} + \psi(\lambda) \Bigg) \mright\}
\end{equation*}
is achieved by $\lambda_s$. Then, \Cref{lemma:extension_banerjee} only guarantees that, with probability $1 - \beta$
\begin{equation*}
    \bE^S \big[ \gen(W,S) \big] \leq \frac{1}{\lambda_s} \Bigg( \frac{\relent(\bP_{W}^{S} \Vert \bQ) + \log \frac{1}{\beta}}{n} + \psi(\lambda_s) \Bigg)
\end{equation*}
for that \emph{fixed} parameter $\lambda_s$ and \emph{simultaneously} for all posteriors $\bP_W^{S=s}$. If the data realization is $s' \neq s$, this may result in a looser bound. In order to make sure that we have a bound that holds \emph{simultaneously} for all parameters $\lambda \in \{\lambda_s : s \in \cZ^n \}$ we need to take the union bound, which effectively renders the bound vacuous as the cardinality of the set $\cZ^n$ grows. 

Next, we will present a technique that allows us to bypass this subtlety for a small penalty. The idea is simple: separate the event space into a set of events where the optimization can be performed, and then pay the union bound price.  This can also be seen as optimizing over the set of parameters $\lambda$ that will yield \emph{almost optimal} bounds and paying the union bound price for considering that set. In this case, the event space is separated using a quantization based on the relative entropy $\relent(\bP_W^S \lVert \bQ)$.

There are other techniques to deal with these kinds of optimization problems, but we believe that they are not general enough or appropriate for this situation and are discussed in~\Cref{subsec:related_approaches}.

\subsubsection{Optimizing Parameters in Probabilistic Statements}
\label{subsubsec:optimizing_parameters_in_probabilistic_statements}

To start, note that the event $\{ \relent(\bP_W^S \Vert \bQ) > n \}$ is not interesting, given this event, the resulting bound from~\Cref{lemma:extension_banerjee} is non-decreasing with respect to the number of samples $n$. With this in mind, we may quantize the event $\cE = \{ \relent(\bP_W^S \Vert \bQ) \leq n \} $ into $n$ disjoint sub-events, find an almost optimal parameter per each sub-event, and then combine the solutions to obtain the following result.

\begin{theorem}[PAC-Bayes Chernoff analogue I]
\label{th:pac_bayes_chernoff_analogue}
Consider a loss function $\ell$ with a bounded CGF (\Cref{def:bounded_cgf}). Let $\bQ$ be any prior independent of $S$ and define the event $\cE \coloneqq \{ \relent(\bP_{W}^{S} \Vert \bQ) \leq n \}$. Then, for every $\beta \in (0,1)$, with probability no smaller than $1-\beta$
\begin{align}
    \bE^S \big[ \poprisk(W) \big] \leq &\bI_\cE \Bigg( \bE^S \big[ \emprisk(W,S) \big] + \psi_{*}^{-1} \bigg( \frac{\relent(\bP_{W}^{S} \Vert \bQ) + \log \frac{en}{\beta}}{n} \bigg) \Bigg) 
    \nonumber \\
    &+ \bI_{\cE^c} \cdot \esssup \bE^S \big[ \poprisk(W) \big]
\end{align}
holds \emph{simultaneously} for all posteriors $\bP_W^S$.
\end{theorem}

To prove this statement, let us first define $\cB_\lambda$ as the complement of the event in~\Cref{lemma:extension_banerjee}, that is,
\begin{equation*}
    \cB_\lambda \coloneqq \mleft\{\exists \ \bP_W^S : \bE^S \big[ \gen(W,S) \big] >  \frac{1}{\lambda} \Bigg( \frac{\relent(\bP_{W}^{S} \Vert \bQ) + \log \frac{1}{\beta}}{n} + \psi(\lambda) \Bigg) \mright\}.
\end{equation*}
Therefore, from~\Cref{lemma:extension_banerjee} we know that $\bP[\cB_\lambda] < \beta$ for all $\lambda \in (0,b)$. Now, let us quantize the event $\cE$ with the disjoint sub-events $\cE_1 \coloneqq \{ \relent(\bP_W^S \Vert \bQ) \leq 1 \}$ and $\cE_k \coloneqq \{ \lceil \relent(\bP_W^S \Vert \bQ) \rceil = k \}$ for all $k=2, \ldots, n$, thus forming a covering of the event $\cE$. In order to avoid measurability issues and conditioning to events with probability zero, we can identify these events defining  $\cK \coloneqq \{k \in \bN : 1 \leq k \leq n \textnormal{ and } \bP[\cE_k] > 0 \}$. 

In this way, for all $k \in \cK$, given the event $\cE_k$, with probability no more than $\bP[\cB_\lambda | \cE_k]$, there exists a posterior $\bP_W^S$ such that
\begin{equation}
    \label{eq:pac_bayes_cgf_with_lambda_and_k}
    \bE^S \big[ \gen(W,S) \big] > \frac{1}{\lambda} \bigg( \frac{k + \log \frac{1}{\beta}}{n} + \psi(\lambda) \bigg),
\end{equation}
for all $\lambda \in (0, b)$. The right-hand side of \eqref{eq:pac_bayes_cgf_with_lambda_and_k} can be minimized with respect to $\lambda$ \emph{independently of the training set $S$}. Let $\cB_{\lambda_k}$ be the event resulting from this minimization and note that $\bP[\cB_{\lambda_k}] \leq \beta$. According to~\Cref{lemma:boucheron_convex_conjugate_inverse}, this ensures that with probability no more than $\bP[\cB_{\lambda_k} | \cE_k]$, there exists a posterior $\bP_W^S$ such that
\begin{equation}
    \label{eq:pac_bayes_cgf_with_k}
    \bE^S \big[ \gen(W,S) \big] >\psi_{*}^{-1} \bigg( \frac{k + \log \frac{1}{\beta}}{n} \bigg),
\end{equation}
where $\psi_*$ is the convex conjugate of $\psi$ and where $\psi_*^{-1}$ is a non-decreasing concave function. Given $\cE_k$, since $k < \relent(\bP_W^S \Vert \bQ) + 1$, with probability no larger than $\bP[\cB_{\lambda_k} | \cE_k]$, there exists a posterior $\bP_W^S$ such that
\begin{equation*}
    \bE^S \big[ \gen(W,S) \big] >\psi_{*}^{-1} \bigg( \frac{\relent(\bP_W^S \Vert \bQ) + 1 + \log \frac{1}{\beta}}{n} \bigg).
\end{equation*}
Now, define $\cB'$ as the event stating that there exists a posterior $\bP_W^S$ such that
\begin{align*}
    \bE^S \big[ \poprisk(W) ] > &\bI_{\cE} \cdot \Bigg( \bE^S \big[ \emprisk(W,S) ] + \psi_{*}^{-1} \bigg( \frac{\relent(\bP_{W}^{S} \Vert \bQ) + \log \frac{e}{\beta}}{n} \bigg) \Bigg) \\
    &+ \bI_{\cE^c} \cdot \esssup \bE^S \big[ \poprisk(W) ]
\end{align*}
where $\bP[\cB' | \cE_k] \bP[\cE_k] \leq \bP[\cB_{\lambda_k} | \cE_k] \bP[\cE_k] \leq \bP[\cB_{\lambda_k}] \leq \beta$ for all $k \in \cK$ and where $\bP[\cB' \cap \cE^c] = 0$ by the definition of the essential supremum (see \Cref{sec:prob_theory}). Therefore, by the law of total probability, the probability of $\cB'$ is bounded as
\begin{equation*}
    \bP[\cB'] = \sum_{k \in \cK} \bP[\cB' | \cE_k] \bP[\cE_k] + \bP[\cB' \cap \cE^c] < n\beta.
\end{equation*}
Finally, the substitution $\beta \leftarrow \nicefrac{\beta}{n}$ recovers the statement in~\Cref{th:pac_bayes_chernoff_analogue}.

\begin{remark}
    \label{rem:infimum_psi}
    Note that while the optimization of~\eqref{eq:pac_bayes_cgf_with_lambda_and_k} always results in~\eqref{eq:pac_bayes_cgf_with_k}, the infimum is not always attained by a $\lambda_k \in (0,b)$. It is possible that it is attained by letting $\lambda \to b$, although never by letting $\lambda \to 0$ as $\psi(0) = 0$ and the term inside the infimum goes to $\infty$. In the case where the infimum is attained by letting $\lambda \to b$, by continuity, the desired inequality~\eqref{eq:pac_bayes_cgf_with_k} still holds and the event described by $\lim_{\lambda \to b} \cB_\lambda$ is still such that $\bP[\lim_{\lambda \to b} \cB_\lambda] \leq \beta$. 
\end{remark}

For sub-Gaussian losses (and therefore for bounded ones), this recovers \citet{mcallester2003pac}'s \Cref{th:mcallester} and \citet{hellstrom2021corrections}'s bound rates. For loss functions with heavier tails like \emph{sub-gamma} and \emph{sub-exponential}, the rates become a mixture of slow and fast rates with the same form as the mutual information bounds from~\Cref{subsec:a_slow_rate_bound_mi} that derive form~\Cref{th:mi_bound_general_cgf}, where instead of the mutual information we have the complexity PAC-Bayesian complexity $\relent(\bP_W^S \Vert \bQ) + \log \frac{en}{\beta}$ for the event $\cE$, and the essential supremum of the population risk $\bE^S [ \poprisk(W)]$ for the event $\cE^c$.

\paragraph{Smaller Union Bound Cost.} To obtain~\Cref{th:pac_bayes_chernoff_analogue}, we considered a covering of the event $\cE$ by considering uniform buckets on the value of the relative entropy. Similarly to \citep{langford2001bounds, catoni2003pac}, we can pay a multiplicative cost of $e$ to the relative entropy to reduce the union bound cost to $\log  \nicefrac{ (2 + \log n)}{\beta}$ by considering a geometric grid. As mentioned by \citet{maurer2004note}, however, these bounds are only useful when the dependency measure $\relent(\bP_W^S \Vert \bQ )$ grows slower than logarithmically. This procedure is almost equivalent to the one we just outlined, and therefore, it is delegated to \Cref{app:pac_bayes_smaller_union_bound_cost}.

\paragraph{Different Uninteresting Events.}
\Cref{th:pac_bayes_chernoff_analogue} considers the event $\{\relent(\bP_W^S \Vert \bQ) \leq n\}$ since the complementary event is uninteresting as, given this event, the bound's rate becomes $\Omega(1)$. However, if one is interested in a different event such as $\{\relent(\bP_W^S \Vert \bQ) \leq k_{\textnormal{max}}\}$, then the proofs may be replicated. The resulting bounds are equal to~\Cref{th:pac_bayes_chernoff_analogue}, where the factor inside the logarithm will be $\nicefrac{e k_\textnormal{max}}{\beta}$. Some examples would be to choose $k_{\textnormal{max}} = \lceil \log (d n) \rceil$ for parametric models such as the sparse single-index  \citep{alquier2013sparse} and sparse additive \citep{guedj2013pac} models, where $d$ is the dimension of the input data, or to choose $k_{\textnormal{max}} = \lceil \log (d p n) \rceil$ for the noisy $d \times p$ matrix completion problem \citep{mai2015bayesian}. 

Imagine that one is interested in a bound like those presented in \Cref{th:pac_bayes_chernoff_analogue} and does not consider any event to be uninteresting. This could happen in some regression applications where, even if $\relent(\bP_W^S \Vert \bQ) \geq n$ and the bound is in $\Omega(1)$, the particular value of the bound is necessary. In this case, working in the events' space is still beneficial. The idea is almost the same as before: separate the events' space into a countable set of events where the optimization can be performed and pay the union bound price. The main difference is that each of these events $\cE_k$ will be defined with a different value of $\beta_k$ so that price of the union bound is still finite $\sum_{k=1}^\infty \beta_k < \infty$. For instance, applying this approach to \Cref{lemma:extension_banerjee} results in the following theorem.

\begin{theorem}[PAC-Bayesian Chernoff analogue II]
\label{th:pac_bayes_chernoff_analogue_no_cutoff}
Consider a loss function $\ell$ with a bounded CGF (\Cref{def:bounded_cgf}). Let $\bQ$ be any prior independent of $S$. Then, for every $\beta \in (0,1)$, with probability no smaller than $1-\beta$
\begin{equation*}
    \bE^S \big[ \gen(W,S) \big] \leq  \psi_{*}^{-1} \Bigg( \frac{\relent(\bP_{W}^{S} \Vert \bQ) + \log \frac{e\pi^2\big(\relent(\bP_W^S \Vert \bQ) + 1\big)^2}{6\beta}}{n} \Bigg)
\end{equation*}
holds \emph{simultaneously} for all posteriors $\bP_W^S$.
\end{theorem}

Since $x + \log \nicefrac{e \pi^2 (x+1)^2}{6 \beta}$ is a non-decreasing, concave, continuous function for all $x > 0$, it can be upper bounded by its envelope. That is, $$x + \log \frac{e \pi^2 (x+1)^2}{6 \beta} \leq \inf_{a > 0} \mleft \{ \Big(\frac{a+3}{a+1}\Big) x + \log \frac{e\pi^2 (a+1)^2}{6 \beta} - \frac{2a}{a+1} \mright\}.$$ Taking $a=19$ leads to the following corollary, which effectively recovers the Chernoff inequality from~\Cref{prop:chernoff}.

\begin{corollary}
\label{cor:pac_bayes_chernoff_analogue_no_cutoff_linearized}
Consider a loss function $\ell$ with a bounded CGF (\Cref{def:bounded_cgf}). Let $\bQ$ be any prior independent of $S$. Then, for every $\beta \in (0,1)$, with probability no smaller than $1-\beta$
\begin{equation*}
    \bE^S \big[ \gen(W,S) \big] \leq  \psi_{*}^{-1} \Bigg( \frac{1.1 \relent(\bP_{W}^{S} \Vert \bQ) + \log \frac{10 e\pi^2}{\beta}}{n} \Bigg)
\end{equation*}
holds \emph{simultaneously} for all posteriors $\bP_W^S$.
\end{corollary}

As discussed above the theorem statement, the proof of \Cref{th:pac_bayes_chernoff_analogue_no_cutoff} follows similarly to the proof of~\Cref{th:pac_bayes_chernoff_analogue}. Let $\cE_1$ and $\cE_k$ be defined as above, but this time let $\cE_k$ be defined for all $k \in \bN$ such that $k > 2$. Similarly, the identifier of the events with zero probability is $\cK \coloneqq \{ k \in \bN : \bP[\cE_k] > 0 \}$. Now, instead of considering the event $\cB_{\lambda}$, we will consider the events $\cB_{\lambda,k}$ as the complements of the event in~\Cref{lemma:extension_banerjee} with probability no larger than $\beta_k$. That is, we consider the events
\begin{equation*}
    \cB_{\lambda,k} \coloneqq \mleft\{ \exists \ \bP_W^S : \bE^S \big[ \gen(W,S) \big] >  \frac{1}{\lambda} \Bigg( \frac{\relent(\bP_{W}^{S} \Vert \bQ) + \log \frac{1}{\beta_k}}{n} + \psi(\lambda) \Bigg) \mright\}.
\end{equation*}
Considering these events with different requirements for $\beta_k$ will be crucial for the proof.
Then, given the event $\cE_k$, with probability no more than $\bP[\cB_{\lambda,k} | \cE_k]$, there exists a posterior $\bP_W^S$ such that
\begin{equation}
    \label{eq:pac_bayes_cgf_with_lambda_and_k_no_cutoff}
    \bE^S \big[ \gen(W,S) \big] > \frac{1}{\lambda_k} \bigg[ \frac{k + \log \frac{1}{\beta_k}}{n} + \psi(\lambda_k) \bigg],
\end{equation}
for all $\lambda_k \in (0, b)$. As above, the right-hand side of \eqref{eq:pac_bayes_cgf_with_lambda_and_k_no_cutoff} can be minimized with respect to $\lambda$ \emph{independently of the training set $S$}. Let $\cB_{\lambda_k,k}$ be the event resulting from this minimization and note that $\bP[\cB_{\lambda_k}] \leq \beta_k$. Hence, with probability no larger than $\bP[\cB_{\lambda_k} | \cE_k]$, there exists a posterior $\bP_W^S$ such that
\begin{equation*}
    \bE^S \big[ \gen(W,S) \big] > \psi_{*}^{-1} \bigg( \frac{k + \log \frac{1}{\beta_k}}{n} \bigg),
\end{equation*}
where $\psi_*$ is the convex conjugate of $\psi$ and where $\psi_*^{-1}$ is a non-decreasing concave function. Now, let $\beta_k = \nicefrac{\beta}{k^2}$. Given $\cE_k$, since $k < \relent(\bP_W^S \Vert \bQ) + 1$, with probability no larger than $\bP[\cB_{\lambda_k} | \cE_k]$, there exists a posterior $\bP_W^S$ such that
\begin{equation}
    \label{eq:event_to_bound_no_cutoff}
    \bE^S \big[ \gen(W,S) \big] > \psi_{*}^{-1} \Bigg( \frac{\relent(\bP_W^S \Vert \bQ) + 1 + \log \frac{\big( \relent(\bP_W^S \Vert \bQ) + 1  \big)^2}{\beta}}{n} \Bigg).
\end{equation}
Now, define $\cB'$ as the event described in~\eqref{eq:event_to_bound_no_cutoff}, where  $\bP[\cB' | \cE_k] \bP[\cE_k] \leq \bP[\cB_{\lambda_k} | \cE_k] \bP[\cE_k] \leq \bP[\cB_{\lambda_k}] \leq \beta_k = \nicefrac{\beta}{k^2}$ for all $k \in \cK$. Therefore, the probability of $\cB'$ is bounded as
\begin{equation*}
    \bP[\cB'] = \sum_{k \in \cK} \bP[\cB' | \cE_k] \bP[\cE_k] < \sum_{k=1}^\infty \frac{\beta}{k^2} = \frac{\pi^2}{6} \cdot \beta.
\end{equation*}
Finally, the substitution $\beta \leftarrow \nicefrac{6\beta}{\pi^2}$ recovers the statement from~\Cref{th:pac_bayes_chernoff_analogue_no_cutoff}.

\begin{remark}
    \label{rem:choice_beta_k}
    The choice of $\beta_k = \nicefrac{\beta}{k^2}$ was arbitrary. In fact, any choice such that $\sum_{k=1}^\infty \beta_k < \infty$ would work. For example, there are better choices such as $\beta_k = \nicefrac{\beta}{k \log^2(6k)}$, which is employed in another context by \citet{kaufmann2016complexity}. However, this choice ultimately only changes a sub-logarithmic factor on the dependency measure, that is,  only incurs a penalty of $\cO( \frac{1}{n} \cdot \relent(\bP_W^S \Vert \bQ))$ and hence we decide to ignore this type of optimizations.
\end{remark}

\subsubsection{Implications to the Design of Posterior Distributions}
\label{subsubsec:implications_learning_posteriors}

We will focus on the discussion on the implications of having a parameter-free bound with more general assumptions with respect to the design of posterior distributions to \Cref{th:pac_bayes_chernoff_analogue_no_cutoff}. The discussion extends to~\Cref{th:pac_bayes_chernoff_analogue}, the theorems presented in~\Cref{subsec:losses_with_bounded_moment}, and other situations analogously.

The first consideration is that the parameter-free bound in~\Cref{th:pac_bayes_chernoff_analogue_no_cutoff} can always be transformed back into a parametric bound that holds \emph{simultaneously} for all parameters. In the case of \Cref{th:pac_bayes_chernoff_analogue_no_cutoff}, employing \Cref{lemma:boucheron_convex_conjugate_inverse}, we have that with probability no smaller than $1 - \beta$
\begin{equation*}
    \bE^S \big[ \gen(W,S) \big] \leq \inf_{\lambda \in (0,b)} \bigg \{ \frac{\relent(\bP_{W}^{S} \Vert \bQ) + \log \frac{en}{\beta}}{\lambda n} + \frac{\psi(\lambda)} {\lambda} \bigg \}
\end{equation*}
holds \emph{simultaneously} for all posteriors $\bP_W^S$.
This relaxation to a familiar structure tells us that the optimal posterior is the Gibbs distribution $\sfp_W^S(w) \propto \sfq(w) \cdot e^{\lambda n \emprisk(w,S)}$, where the value of $\lambda$ can now be chosen \emph{adaptively} for each dataset realization $s$.

The second consideration is if we are using some numerical estimation of the posterior using neural networks as with the PAC-Bayes with backprop~\citep{rivasplata2019pac,perez2021tighter} or other similar frameworks~\citep{dziugaite2017computing,lotfi2022pac}. Then, the posterior can be readily estimated as long as the inverse of the convex conjugate $\psi_*^{-1}$ is a differentiable function.

\subsection{Losses With a Bounded Moment}
\label{subsec:losses_with_bounded_moment}

In this section, we focus on losses with \emph{heavy tails}, which we recall from~\Cref{subsec:interpolating_between_slow_and_fast_rate} that are defined as losses with a finite $p$-th raw moment $\bE[\ell(w,Z)^p]$ for all $w \in \cW$ and some $p > 1$. However, by the end of the section, we will also find some results that hold for losses with a bounded variance or second central moment. Our results in this section will be the PAC-Bayesian equivalents to the ``in expectation'' results from~\Cref{subsec:interpolating_between_slow_and_fast_rate}, with the exception of an additional result for losses with a bounded variance.

Similarly to~\Cref{subsec:interpolating_between_slow_and_fast_rate}, this section employs our refinement of~\citet{alquier2006transductive}'s method. In this case, this method is combined with our fast-rate PAC-Bayesian bound from~\Cref{th:fast_rate_bound_strong}. This is in contrast to \citet{alquier2006transductive}, who developed the method using a bound similar to \citet{catoni2007pac}'s~\Cref{th:catoni_pac_bayes}. The reason for this choice is threefold. First, as shown in~\Cref{sec:bounds_bounded_losses}, our bound is as tight as the Seeger--Langford bound~\citep{seeger2002pac,langford2001bounds} from~\Cref{th:seeger_langford_pac_bayes}; second, our bound is more interpretable than bounds \emph{à la} Catoni; and third, our bound will allow us to optimize the parameters appearing in the bound using the event space optimization technique from~\Cref{subsec:losses_with_bounded_CGF} and the resulting bound will have a closed form for the optimal posterior.

\begin{remark}
    \label{rem:additonal_kappas}
    To alleviate the notation, throughout the section we will define $\kappa_1 \coloneqq c \gamma \log \big( \nicefrac{\gamma}{(\gamma -1)} \big)$, $\kappa_2 \coloneqq c \gamma$, and $\kappa_3 \coloneqq \gamma \big( 1 - c(1 - \log c)\big)$, with the understanding that they are functions of the parameters $c \in (0,1]$ and $\gamma > 1$ from~\Cref{th:fast_rate_bound_strong}.
\end{remark}

\subsubsection{Losses With a Bounded $p$-th Moment}

For losses with a $p$-th moment bounded by $m_p$, the main result of this section is the PAC-Bayes equivalent of~\Cref{th:mi_bound_moments}.

\begin{theorem}
    \label{th:alquier_truncation_method_refined_adaptive_lambda}
    Consider a loss function $\ell(w,Z)$ with a $p$-th moment bounded by $m_p$ for all $w \in \cW$. Then, for every $\beta \in (0,1)$, with probability no smaller than $1 - \beta$
    \begin{align*}
        &\bE^S \big[\poprisk(W) \big] \leq \\
        &\ \kappa_1 \cdot \bE^S \big[\emprisk_{\leq t^\star}(W,S)\big] + m_p^{\frac{1}{p}} \Big(\frac{p}{p-1}\Big) \Big( \kappa_2 \cdot \frac{1.1 \relent(\bP_W^S \Vert \bQ) + \log \frac{10 e \pi^2 \xi(n)}{\beta}}{n} + \kappa_3 \Big)^{\frac{p-1}{p}}
    \end{align*}
    holds \emph{simultaneously} for all posteriors $\bP_W^S$, all $c \in (0,1]$, and all $\gamma > 1$, where $$t^\star \coloneqq m_p^{\frac{1}{p}} \Big( \kappa_2 \cdot \frac{1.1 \relent(\bP_W^S \Vert \bQ) + \log \frac{10 e \pi^2 \xi(n)}{\beta}}{n} + \kappa_3 \Big)^{-\frac{1}{p}}.$$
\end{theorem}

Let us recall the discussion below~\Cref{th:mi_bound_moments}. Let $c = 1$, then the rate is $$m_p^{\nicefrac{1}{p}} \cdot \left( \frac{1.1 \relent(\bP_W^S \Vert \bQ) + \log \frac{10 e \pi^2 \xi(n)}{\beta}}{n} \right)^{\frac{(p-1)}{p}}.$$ 
The term $m_p^{\nicefrac{1}{p}}$ controls the weight given to the complexity term and is equal to the $\cL_p$ norm of the loss. By the non-decreasing nature of the $\cL_p$ norms a trivial bound would be $\bE \big[ \poprisk(W) \big] \leq m_p^{\nicefrac{1}{p}}$. Hence, the presented bound improves upon this as for increasing number of samples, the contribution of $m_p^{\nicefrac{1}{p}}$ vanishes. 
Then, the term $$ \left( \frac{1.1 \relent(\bP_W^S \Vert \bQ) + \log \frac{10 e \pi^2 \xi(n)}{\beta}}{n} \right)^{\frac{(p-1)}{p}}$$ shows how the rate is interpolating between a slow rate when $p = 2$ and a fast rate when $p \to \infty$. 

This interpolating behavior is exactly the one we expected from the classical moment's inequality from~\Cref{prop:bounded_moments}. To compare the two bounds, we may assume that the algorithm is independent of the data and that we choose $\bQ = \bP_W$, then the rate is $$m_p^{\nicefrac{1}{p}} \cdot \left( \frac{\log \frac{10 e \pi^2 \xi(n)}{\beta}}{n} \right)^{\frac{(p-1)}{p}}.$$ There are three main differences between the proposed bound and the classical one. First, the classical bound scales with the \emph{central} $p$-th moment, while the proposed bound scales with the \emph{raw} $p$-th moment, which can be much larger. Second, the presented bound is of \emph{high probability}, that is, it has a logarithmic dependence with the probability parameter $\beta$, while the classical bound has the inferior linear dependence. Third, the proposed bound estimates the population error with a \emph{truncated} version of the empirical risk, where the truncation point is selected to minimize the bound.

In order to prove~\Cref{th:alquier_truncation_method_refined_adaptive_lambda}, consider first our refinement of \citet{alquier2006transductive}'s truncation method. That is, let $\ell = \ell_{< \nicefrac{n}{\lambda}} + \ell_{\geq \nicefrac{n}{\lambda}}$ represent a decomposition of the loss into a truncated version of the loss and an unbounded reminder, where 
\begin{equation*}
    \ell_{< \nicefrac{n}{\lambda}} = \ell \cdot \bI_{\ell < \nicefrac{n}{\lambda}} \textnormal{ and } \ell_{\geq \nicefrac{n}{\lambda}} = \ell \cdot \bI_{\ell \geq \nicefrac{n}{\lambda}}.
\end{equation*}
Furthermore, let $\poprisk_{< \nicefrac{n}{\lambda}}$, $\poprisk_{\geq \nicefrac{n}{\lambda}}$, $\emprisk_{< \nicefrac{n}{\lambda}}$, and $\emprisk_{\geq \nicefrac{n}{\lambda}}$ represent, respectively, the population and empirical risks associated to the truncated loss and the unbounded reminder. Then, we may bound the population risk of the truncated version using our fast-rate bound from~\Cref{th:fast_rate_bound_strong} and the population risk of the unbounded reminder using standard tail inequalities. 

\begin{lemma}
    \label{lemma:alquier_truncation_method_refined}
    For all $\beta \in (0,1)$ and all $\lambda > 0$, with probability no smaller than $1 - \beta$
    \begin{align*}
       \bE^S \big[\poprisk(W)\big] \leq \kappa_1 &\cdot \bE ^S \big[\emprisk_{\leq \nicefrac{n}{\lambda}} (W,S) \big] \\
       &+ \kappa_2 \cdot \frac{\relent(\bP_W^S \Vert \bQ) + \log \frac{\xi(n)}{\beta}}{\lambda} + \kappa_3 \cdot \frac{n}{\lambda}+ \int_{\nicefrac{n}{\lambda}}^\infty \bP^S \big[ \ell(W,Z) > t \big] \rmd t
    \end{align*}
    holds \emph{simultaneously} for all posteriors $\bP_W^S$, all $c \in (0,1]$, and all $\gamma > 1$.
\end{lemma}

\looseness=-1 \Cref{lemma:alquier_truncation_method_refined} has two attractive properties. The first one is that it is a general-purpose lemma that holds for any type of loss. Letting $\lambda \in \Theta(\sqrt{n})$ guarantees that the bias of the truncated empirical risk at a truncation point in $\Theta(\sqrt{n})$ has a rate in $\cO(\nicefrac{1}{\sqrt{n}})$ plus the probability of the loss' tail from a point in $\Theta(\sqrt{n})$ onward. The second attractive property is that, if the tail of the loss is bounded by some function $f(n,\lambda)$, then the posterior optimizing the bound is the Gibbs posterior $\sfp_W^{S=s}(w) \propto \sfq(w) e^{- \lambda \cdot \frac{\kappa_1}{\kappa_2} \cdot \emprisk_{\leq \nicefrac{n}{\lambda}}(w,s)}$ and it is independent of the tail bound.

In particular, when the loss has a $p$-th moment bounded by $m_p$, the tail is bounded by~\eqref{eq:unbounded_remainder_moment_tail_bound} and 
\Cref{lemma:alquier_truncation_method_refined} can be written as follows.

\begin{lemma}
    \label{lemma:alquier_truncation_method_refined_bouned_moment}
    Consider a loss $\ell(w,Z)$ with a $p$-th moment bounded by $m_p$ for all $w \in \cW$. Then, for every $\beta \in (0,1)$ and all $\lambda > 0$, with probability no smaller than $1 - \beta$
    \begin{align*}
       \bE^S &\big[\poprisk(W)\big] \leq \\
       &\kappa_1 \cdot \bE^S \big[\emprisk_{\leq \nicefrac{n}{\lambda}} (W,S) \big] + \kappa_2 \cdot \frac{\relent(\bP_W^S \Vert \bQ) + \log \frac{\xi(n)}{\beta}}{\lambda} + \kappa_3 \cdot \frac{n}{\lambda}+ \frac{m_p}{p-1} \Big( \frac{\lambda}{n} \Big)^{p - 1}
    \end{align*}
    holds \emph{simultaneously} for all posteriors $\bP_W^S$, all $c \in (0,1]$ and all $\gamma > 1$.
\end{lemma}

\begin{remark}
    \label{rem:alternative_choice_of_loss}
    \citet{alquier2006transductive} discussed the possibility of modifying the truncation method specifically for losses with a bounded moment. When combined with our fast-rate bound from~\Cref{th:fast_rate_bound_strong}, this modification results in better constants while maintaining the rate. In the interest of the simplicity of the discussion, we defer this to~\Cref{app:alternative_choice_of_loss}, while noting that all the developments in this section can be applied in this situation almost verbatim. 
\end{remark}

Before studying the effect of the optimization of the parameter $\lambda$, it is important to reiterate that the term $\nicefrac{\kappa_3 n}{\lambda}$ does not affect the bound's rate as choosing $c = 1$ implies $\kappa_3 = 0$, and the bounds hold \emph{simultaneously} for all values of $c \in (0,1]$. Therefore, in all rate discussions henceforth, we may always assume that $\kappa_3 = 0$ for simplicity.

\citet{alquier2006transductive, alquier2021user} considered the \emph{data-independent} $\lambda = \sqrt{n}$. This gives a bound with a rate of $\nicefrac{1}{\sqrt{n}}$ for any loss with a bounded $p$-th moment where $p > 2$. A better choice is $\lambda = \big( \nicefrac{n^{p-1}}{m_p} \big)^{\nicefrac{1}{p}}$. This results in a bound with a rate of $n^{- \frac{p - 1}{p}}$.

\begin{theorem}
    \label{th:alquier_truncation_method_refined_fixed_lambda}
    Consider a loss $\ell(w,Z)$ with a  $p$-th moment bounded by $m_p$ for all $w \in \cW$. Then, for every $\beta \in (0,1)$, with probability no smaller than $1 - \beta$
    \begin{align*}
       \bE^S \big[\poprisk(W) \big] \leq \kappa_1 \cdot &\bE\big[ \emprisk_{\leq (m_p n)^{\frac{1}{p}}}(W,S)\big] \\
       &+ \Big( \frac{m_p}{n^{p-1}} \Big)^{\frac{1}{p}} \Big( \kappa_2 \cdot \Big( \relent(\bP_W^S \Vert \bQ) + \log \frac{\xi(n)}{\beta} \Big) + \kappa_3 \cdot n + \frac{1}{p-1} \Big)
    \end{align*}
    holds \emph{simultaneously} for all posteriors $\bP_W^S$, all $c \in (0,1]$ and all $\gamma > 1$.
\end{theorem}

In this way, the rate for $p = 2$ is exactly the same, a slow rate of $\nicefrac{1}{\sqrt{n}}$. However, as the order of the known bounded moment increases, that is $p \to \infty$, the rate becomes a fast rate of $\nicefrac{1}{n}$. Hence, this choice of $\lambda$ allows us to interpolate between a slow and a fast rate depending on how much knowledge about the tails is available to us. Furthermore, as we gain knowledge of the tails, the truncation of the loss $\ell_{\leq (m_p n)^{\nicefrac{1}{p}}}$ becomes less dependent on the number of training data $n$ and in the limit $p \to \infty$ only depends of the $\bP_Z$-a.s. boundedness of the loss, namely $\lim_{p \to \infty} (m_p n)^{\nicefrac{1}{p}} = \sup_{w \in \cW} \esssup \ell(w,Z)$. Unfortunately, in this way, the bound always depends on the dependency term $\relent(\bP_W^S \Vert \bQ)$ linearly.

 Instead of choosing a data-independent parameter $\lambda$, we can use the event space quantization technique from~\Cref{subsec:losses_with_bounded_CGF} to get a better dependence on the relative entropy. In particular, \Cref{th:pac_bayes_chernoff_analogue_no_cutoff} follows by not considering any ``uninteresting event'' and following the technique as outlined in~\Cref{th:pac_bayes_chernoff_analogue_no_cutoff}. Optimizing the parameter in this way readily results in the desired result from~\Cref{th:alquier_truncation_method_refined_adaptive_lambda}. In this way, the rate is maintained, while the dependence on the relative entropy changed from linear to polynomial of order $\nicefrac{(p-1)}{p}$. For order $p = 2$, this corresponds to the square root and only goes to the linear case when $p \to \infty$, when we also achieve a fast rate of $\nicefrac{1}{n}$.

\looseness=-1  Moreover, following the insights about the implications to the design of posterior distributions from~\Cref{subsec:losses_with_bounded_CGF}, we may use~\Cref{th:alquier_truncation_method_refined_adaptive_lambda} to obtain an equivalent result, but in the form of~\Cref{lemma:alquier_truncation_method_refined_bouned_moment} that holds \emph{simultaneously} for all $\lambda$.

 \begin{theorem}
    \label{th:alquier_truncation_method_refined_simultaneously_all_lambda}
    Consider a loss $\ell(w,Z)$ with a $p$-th moment bounded by $m_p$ for all $w \in \cW$. Then, for every $\beta \in (0,1)$, with probability no smaller than $1 - \beta$
    \begin{align*}
       \bE^S &\big[\poprisk(W)\big] \leq \\
       &\kappa_1 \cdot \bE \big[\emprisk_{\leq \nicefrac{n}{\lambda}}(W,S) \big] + \kappa_2 \cdot \frac{\relent(\bP_W^S \Vert \bQ) + \log \frac{10 e \pi^2 \xi(n)}{\beta}}{\lambda} + \kappa_3 \cdot \frac{n}{\lambda}+ \frac{m_p}{p-1} \Big( \frac{\lambda}{n} \Big)^{p - 1}
    \end{align*}
    holds \emph{simultaneously} for all posteriors $\bP_W^S$, all $\lambda > 0$, all $c \in (0,1]$, and all $\gamma > 1$.
\end{theorem}

 From \Cref{th:alquier_truncation_method_refined_simultaneously_all_lambda}, we understand that the posterior that optimizes both \Cref{th:alquier_truncation_method_refined_adaptive_lambda,th:alquier_truncation_method_refined_simultaneously_all_lambda} is the Gibbs posterior $\sfp_W^{S=s}(w) \propto \sfq(w) e^{- \lambda \cdot \frac{\kappa_1}{\kappa_2} \emprisk_{\leq \nicefrac{n}{\lambda}}(w,s)}$, where now $c$, $\lambda$, and $\gamma$ can be chosen \emph{adaptively after} observing the realization of the data $s$. This way, the choice of the parameter $\lambda$ can be made to optimize the bound emerging from that data realization. On the other hand, the Gibbs distribution emerging from the optimization of \Cref{lemma:alquier_truncation_method_refined} needs to commit to a \emph{fixed} parameter $\lambda$ \emph{before} observing the training data and is data-independent.

\subsubsection{The Case $p \to \infty$ and Essentially Bounded Losses}
\label{subsubsec:case_p_infinity}

So far, in this section, we only considered the algorithm-independent condition of losses with a bounded $p$-th moment $\bE \ell(w,Z)^p$ for all $w \in \cW$. This condition only depends on the loss and the problem distribution $\bP_Z$. Nonetheless, all the previous results can be replicated under the weaker condition that the loss has a bounded $p$-th moment with respect to the algorithm's output, that is, that $m_p' \coloneqq \bE^S \ell(W,Z)^p$ is bounded $\bP_S$-a.s.

As discussed in~\Cref{subsec:interpolating_between_slow_and_fast_rate}, although this condition is weaker, it is harder to guarantee as it requires some knowledge of the data distribution $\bP_Z$ \emph{and} the algorithm's Markov kernel $\bP_W^S$. This knowledge could instead be used to directly find a bound on $\poprisk = \bE^S \ell(W,Z)$. 
However, results under this condition can be useful in some situations. For example, they can be used to derive new results for losses with a bounded variance (as will be seen shortly and as shown in~\Cref{subsec:interpolating_between_slow_and_fast_rate}) and to obtain more meaningful findings when $p \to \infty$.

\Cref{th:alquier_truncation_method_refined_adaptive_lambda}, when specialized to $p \to \infty$, gives us a fast-rate result when the loss is $\bP_Z$-a.s. bounded, that is, when $\esssup \ell(w,Z) < \infty$ for all $w \in \cW$. This condition of the loss being $\bP_Z$-\emph{essentially bounded} can be a strong requirement, similar to the one of bounded losses.
However, when we have more information about the algorithm, then we can obtain a fast-rate result when the loss is $\bP_{W,S} \otimes \bP_Z$-a.s. bounded, that is, when $\esssup \ell(W,Z) < v$. 
This condition is much weaker than the previous essential boundedness or just boundedness of the loss. Namely, one needs to know that the algorithm is such that $\bP [ \ell(W,Z) < v ] = 1$. As an example, consider the squared loss $\ell(w,z) = (w - z)^2$ and some data that belongs to some interval of length 1 with probability 1, that is $\bP [ Z \in [c, c+1] ] =1$, but where we ignore the offset $c$. Consider $w \in \bR$, the simple algorithm that returns the average of the training instances $\bA(s) = \sum_{i=1}^n \nicefrac{z_i}{n}$ ensures that $\esssup \ell(W,Z) < 1$, while $\sup_{w \in \bR} \esssup \ell(w,Z) \to \infty$.

\subsubsection{Losses With a Bounded Variance}

A particularly important case is the one of losses with a bounded second moment. \Cref{th:alquier_truncation_method_refined_adaptive_lambda} obtains the expected slow rate of $\sqrt{\nicefrac{m_2\big(\relent(\bP_W^S \Vert \bQ) + \log \frac{\xi(n)}{\beta}\big)}{n}}$. Similarly, in~\citep{rodriguez2023morepac}, we employ again the events' space quantization technique to obtain a parameter-free PAC-Bayesian bound for losses with a bounded second moment optimizing the parameter in \citep[Theorem 2.4]{wang2015pac} or \citep[Theorem 2.1]{haddouche2023pacbayes}. The resulting bound is similar to the one of \citet[Corollary 1]{kuzborskij2019efron}. The proof and further details, including bounds for martingale sequences and non-i.i.d. data, are in \Cref{app:closed_form_parameter_free_wang}. 

\begin{theorem}
    \label{th:parameter_free_anytime_valid_bounded_2nd_moment}
    Let $\bQ$ be any prior independent of $S$ and define $\xi'(n) \coloneqq 2 e n (n+1)^2 \log(en)$ and the events $\cE_n \coloneqq \{\sigma^2_n \ \relent(\bP_W^S \Vert \bQ) \leq n \}$, where $\Sigma_n^2 \coloneqq \frac{1}{n} \sum_{i=1}^n \bE^S [ \ell(W,Z_i)^2 + 2 \ell(W,Z')^2 + 1]$ for all $n \in \bN$. Then, for every $\beta \in (0,1)$, with probability no smaller than $1-\beta$
    \begin{align*}
        \bE^S \big[ \poprisk(W) \big] \leq &\bI_{\cE_n} \cdot \left[ \bE^S \big[ \emprisk(W,S) \big]] + \frac{2}{\sqrt{6}} \cdot \sqrt{  \Sigma^2_n \Bigg( \frac{\relent(\bP_W^{S} \Vert \bQ) + \log \frac{\xi'(n)}{\beta} }{n}\Bigg)} \right] \\
        &+ \bI_{\cE_n^c} \cdot \esssup  \bE^S \big[ \poprisk(W) \big]
    \end{align*}
    holds \emph{simultaneously} for all posteriors $\bP_W^S$.
\end{theorem}

Note that in the bound from \Cref{th:parameter_free_anytime_valid_bounded_2nd_moment}, the terms $\bE^S \big[ \ell(W,Z_i)^2 \big]$ are fully empirical and the term $\bE^S \big[ \ell(W,Z')^2 \big]$ accounts for the assumption that the second moment of the loss is bounded. \Cref{th:parameter_free_anytime_valid_bounded_2nd_moment} is more general than \Cref{th:pac_bayes_chernoff_analogue}, as only the knowledge of one moment is required instead of the knowledge of a function dominating the CGF, which signifies information of all the moments.

As already mentioned previously, the raw second moment can be much larger than the \emph{variance}, or central second moment. Therefore, we will focus the rest of this section on the scenario when the variance is bounded, that is $\var(w,Z) \leq \sigma^2 < \infty$ for all $w \in \cW$.

In this scenario, \citet{alquier2018simpler} and~\citet{ohnishi2021novel} derived PAC-Bayesian bounds depending on the $\chi^2$ divergence. The result from~\citet{alquier2018simpler} is originally given considering the algorithm-dependent variance $\mathrm{Var}^S(\ell(W,Z))$, which, as discussed previously, usually requires too much knowledge on the algorithm and data distributions. We therefore presented it with the algorithm-independent variance $\sigma^2 \geq \mathrm{Var}^S(\ell(W,Z))$.

\begin{theorem}[{\citet[Theorem 1]{alquier2018simpler} and \citet[Corollary 2]{ohnishi2021novel}}]
    \label{th:bounded_variance_not_high_probability}
    Consider a loss $\ell(w,Z)$ with a variance bounded by $\sigma^2$. Then, for every $\beta \in (0,1)$, with probability no smaller than $1 - \beta$, each of the inequalities
    \begin{align}
        \label{eq:bounded_variance_alquier}
        \bE^S \big[ \poprisk(W) \big] &\leq \bE^S \big[ \emprisk(W,S) \big] + \sqrt{\frac{\sigma^2 (\chi^2(\bP_W^S \Vert \bQ) + 1)}{n \beta}}, \\
        \label{eq:bounded_variance_honorio_1}
        \bE^S \big[ \poprisk(W) \big] &\leq \bE^S \big[ \emprisk(W,S) \big] + \sqrt{\frac{\sigma^2 \sqrt{\chi^2(\bP_W^S \Vert \bQ) + 1}}{n \beta}}, \textnormal{ and } \\
        \label{eq:bounded_variance_honorio_2}
        \bE^S \big[ \poprisk(W) \big] &\leq \bE^S \big[ \emprisk(W,S) \big] + \sqrt{\frac{\chi^2(\bP_W^S \Vert \bQ) + \big(\frac{\sigma^2}{\beta}\big)^2}{2n}}
    \end{align}
    hold \emph{simultaneously} for all posteriors $\bP_W^S$
\end{theorem}

Although this bound still achieves the expected slow rate of $\nicefrac{\sigma^2}{\sqrt{n}}$ from the classical Chebyshev's inequality~\Cref{prop:chebyshev}, there are two main differences between this theorem and~\Cref{th:alquier_truncation_method_refined_adaptive_lambda} for $p = 2$. First, and most notable, the dependence with the confidence penalty $\nicefrac{1}{\beta}$ is not logarithmic, but polynomial. This can result in a loose bound when a high confidence is demanded: for example, for $\beta = 0.05$ we have that $\log \nicefrac{1}{\beta} \approx 3$ while $\nicefrac{1}{\beta} = 20$. Second, the dependency measure changed from the relative entropy $\relent$ to the chi-squared divergence $\chi^2$. The chi-squared divergence also measures the dissimilarity between the posterior $\bP_W^S$ and the prior $\bQ_W$, but it can be much larger. More precisely, recall from~\eqref{eq:relation_relent_chi} in~\Cref{subsec:f-divergences} that
\begin{equation*}
    \label{eq:relent_chi_squared_bound}
    0 \leq \relent \leq \log(1 + \chi^2) \leq \chi^2.
\end{equation*}
Moreover, no lower bound of the relative entropy $\relent$ is possible in terms of the chi-squared divergence $\chi^2$~\citep[Section 7.7]{polianskyi2022}.

Studying \Cref{th:alquier_truncation_method_refined_adaptive_lambda} with the weaker condition that $\bE^S \big[\ell(W,Z)^2 \big] \leq m_2'$ as discussed previously for the case where $p \to \infty$, we can obtain a high-probability PAC-Bayesian bound for losses with a bounded variance that has the relative entropy as the dependency measure. As discussed previously, the method and proof technique also extends to an analysis starting from \Cref{lemma:alquier_truncation_method_modified_bounded_moment} resulting in slightly different constants and using $\emprisk_{p, \nicefrac{n}{\lambda}}$ as an estimator instead of $\emprisk_{\leq \nicefrac{n}{\lambda}}$. Similarly, the method also extends to the semi-empirical bound from~\Cref{th:parameter_free_anytime_valid_bounded_2nd_moment}.

\begin{figure*}[ht]
    \centering
    \includegraphics[width=0.4\textwidth]{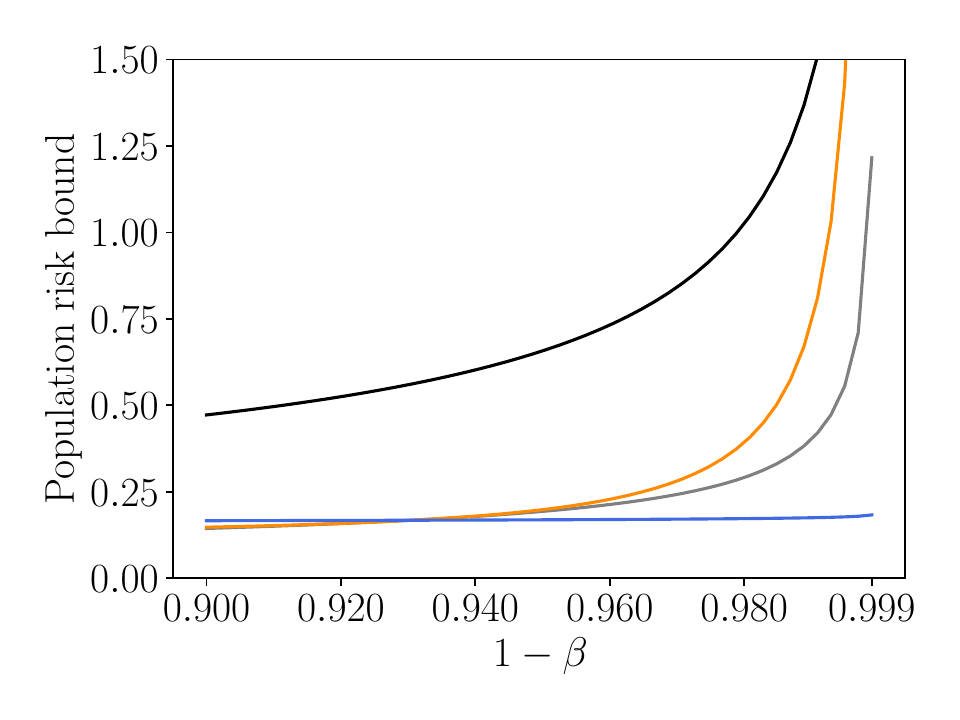}
    \includegraphics[width=0.4\textwidth]{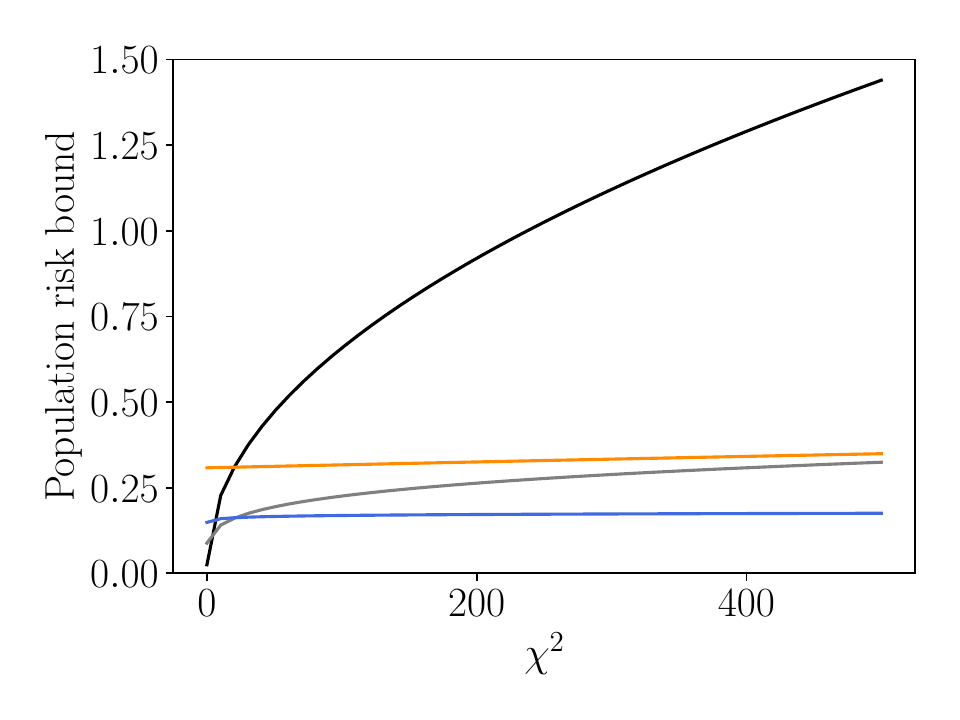}
    \includegraphics[width=0.4\textwidth]{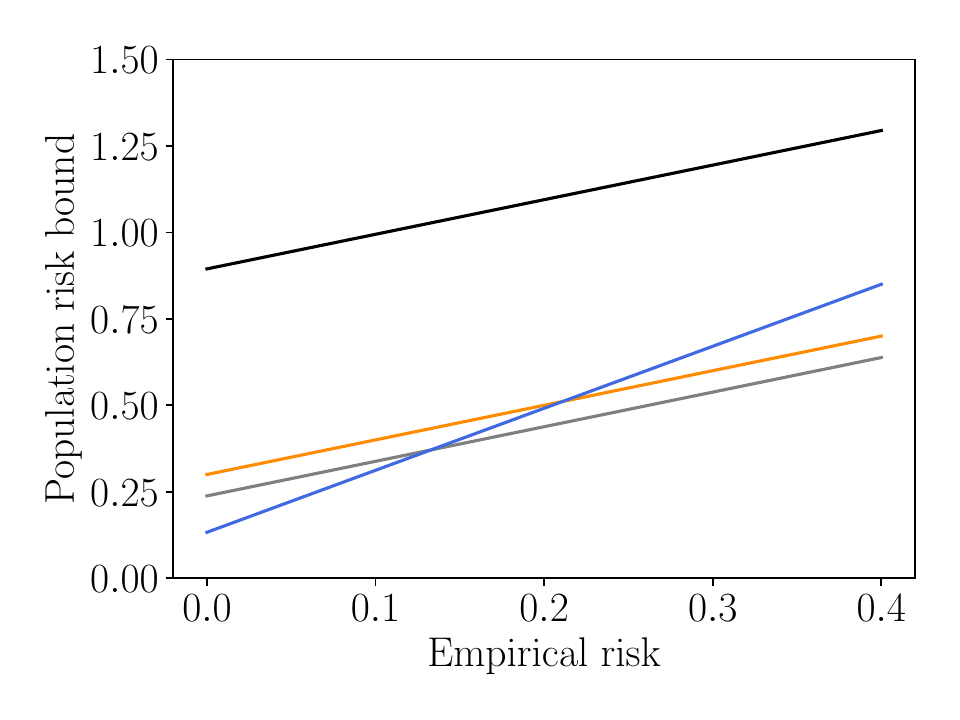}
    \includegraphics[width=0.4\textwidth]{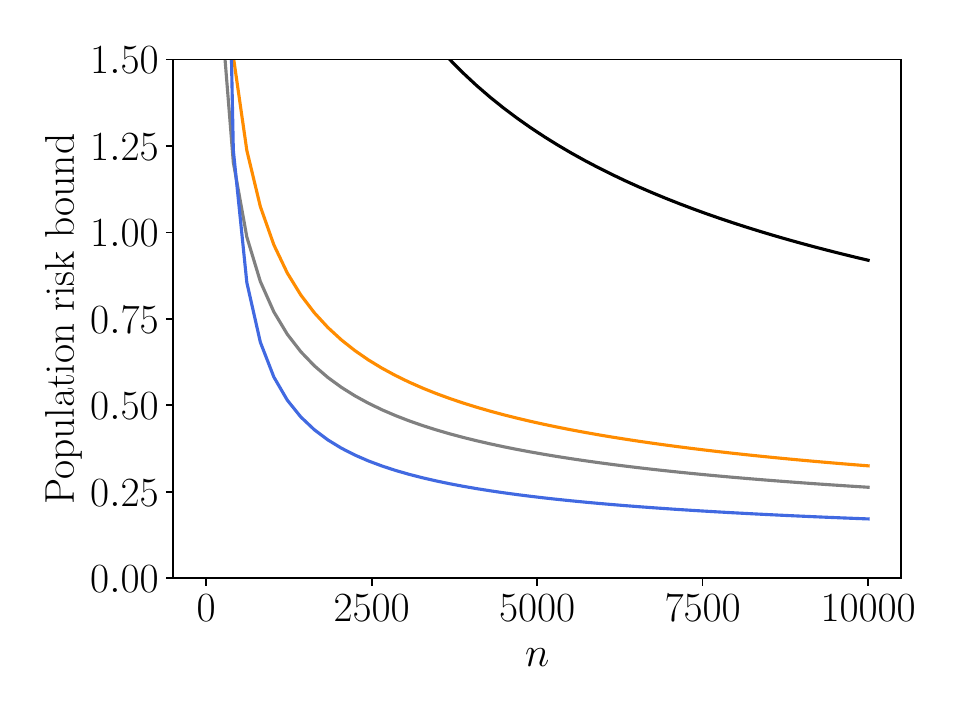}
    \caption{Illustration comparing \citep{alquier2018simpler, ohnishi2021novel} (\eqref{eq:bounded_variance_alquier} in black, \eqref{eq:bounded_variance_honorio_1} in gray, and \eqref{eq:bounded_variance_honorio_2} in orange) and our \Cref{th:bounded_variance_high_probability} (in blue) for varying values of the parameters $\beta$, $\chi^2$, $\emprisk$, and $n$, where here $\emprisk \coloneqq \bE^S[\emprisk(W,S)]$ and $\chi^2 \coloneqq \chi^2(\bP_W^S \Vert \bQ)$. To help the comparison, we actually use the upper bound relaxation~\eqref{eq:relaxation_of_bounded_variance_high_probability} of \Cref{th:bounded_variance_high_probability}. When they are not varying, the values of the parameters are fixed to $\beta = 0.025$, $\chi^2 = 200$, $\emprisk(W,S) = 0.025$, $n=10,000$, and $\sigma^2 = 1$.}
    \label{fig:comparison_bounds_bounded_variance}
\end{figure*}

\begin{theorem}
    \label{th:bounded_variance_high_probability}
    Consider a loss $\ell(w,Z)$ with a variance bounded by $\sigma^2$ for all $w \in \cW$. Then, for every $\beta \in (0,1)$, with probability no smaller than $1 - \beta$
    \begin{align*}
        \bE^S \big[\poprisk(W) \big] \leq \mleft( 1 - 2 \sqrt{\comp'} \mright)_+^{-1}  \mleft( \kappa_1 \cdot \bE^S \big[\emprisk(W,S)\big]  + 2 \sqrt{\sigma^2 \comp'} \mright)
    \end{align*}
    holds \emph{simultaneously} for all posteriors $\bP_W^S$, all $c \in (0,1]$, and all $\gamma > 1$, where
    \begin{equation*}
        \comp' \coloneqq \frac{1.1 \relent(\bP_W^S \Vert \bQ) + \log \frac{10 e \pi^2 \xi(n)}{\beta}}{n}.
    \end{equation*}
\end{theorem}

This theorem effectively captures the dependence on the variance and the rate from Chebyshev's inequality (\Cref{prop:chebyshev}) while additionally being a high-probability bound. The main disadvantage of this bound is the appearance of the multiplicative term $\mleft(1 - 2 \sqrt{\comp'}\mright)_+^{-1}$. 

To prove~\Cref{th:bounded_variance_high_probability}, consider as suggested above a relaxed version of \Cref{th:alquier_truncation_method_refined_adaptive_lambda} for $p=2$ with the weaker condition that $\bE^S \big[ \ell(W,Z)^2 \big] \leq m_2'$ and note that $m_2' = \var^S(\ell(W,Z)) + \bE^S \big[ \poprisk(W) \big] ^2$. Then, for every $\beta \in (0,1)$, with probability no smaller than $1 - \beta$
\begin{align*}
    \bE^S \big[ \poprisk(W) \big] \leq \kappa_1 \cdot \bE^S \big[\emprisk(W,S)\big] + 2 \sqrt{\big(\var^S(\ell(W,Z)) + \bE^S \big[\poprisk(W)\big]^2 \big) \cdot \comp'}
\end{align*}
holds \emph{simultaneously} for all posteriors $\bP_W^S$, all $c \in (0,1]$, and all $\gamma > 1$, where we also used that $\emprisk_{\leq t} \leq \emprisk$ for all $t \in \bR$.

Then, we may employ the inequality $\sqrt{x + y} \leq \sqrt{x} + \sqrt{y}$ to separate the square root and the inequality $\mathrm{Var}^S(\ell(W,Z)) \leq \sup_{w \in \cW} \mathrm{Var}(\ell(w,Z)) = \sigma^2$ to obtain our algorithm-independent variance. In this way, for all $\beta \in (0,1)$, with probability no smaller than $1 - \beta$
\begin{align*}
    \bE^S \big[ \poprisk(W) \big] \leq &\kappa_1 \cdot \bE^S \big[\emprisk(W,S)\big] + 2 \sqrt{\sigma^2 \comp'} + 2 \bE^S \big[\poprisk(W)\big] \sqrt{\comp'}
\end{align*}
holds \emph{simultaneously} for all posteriors $\bP_W^S$, all $c \in (0,1]$, and all $\gamma > 1$.

Re-arranging the equation proves \Cref{th:bounded_variance_high_probability}'s statement: when $1 \geq 2 \sqrt{\comp'}$, the theorem holds by the reasoning above, and when $1 \leq 2 \sqrt{\comp'},$ the theorem holds trivially by the convention that $\nicefrac{1}{0} \to \infty$.

Although the \Cref{th:bounded_variance_high_probability} is of high probability and considers the relative entropy, it is hard to compare \Cref{th:bounded_variance_not_high_probability} due to the first factor $( 1 -  2 \sqrt{\comp'})_+^{-1}$. This factor ensures the bound is only useful when $2 \sqrt{\comp'} < 1$, which is the range where the bound would be effective without the said factor anyway. To effectively compare the two bounds, we bound \Cref{th:bounded_variance_high_probability} from above using the relative entropy upper bound~\eqref{eq:relation_relent_chi}, that is, with probability no smaller than $1 - \beta$
\begin{equation}
\label{eq:relaxation_of_bounded_variance_high_probability}
     \bE^S \big[\poprisk(W) \big] \leq \mleft( 1 - 2 \sqrt{\mathfrak{C}_{n,\beta,S,\chi^2}'} \mright)_+^{-1}  \mleft( \kappa_1 \cdot \bE^S \big[\emprisk(W,S)\big]  + 2 \sqrt{\sigma^2 \mathfrak{C}_{n,\beta,S,\chi^2}'} \mright)
\end{equation}
holds \emph{simultaneously} for all posteriors $\bP_W^S$, all $c \in (0,1]$, and all $\gamma > 1$, where 
\begin{equation*}
    \mathfrak{C}_{n,\beta,S,\chi^2}' \coloneqq \kappa_2 \cdot \frac{1.1 \log (1 + \chi^2(\bP_W^S \Vert \bQ)) + \log \frac{10 e \pi^2 \xi(n)}{\beta}}{n} + \kappa_3.
\end{equation*}
Also, to simplify the comparison, we also fix $c = 1$ and $\gamma = \nicefrac{e}{(e - 1)}$. 
Even with this relaxation, the presented high-probability bound is tighter than \Cref{th:bounded_variance_not_high_probability} in many regimes (see \Cref{fig:comparison_bounds_bounded_variance}).

\subsection{Related Approaches to Optimize Parameters in Probabilistic Statements}
\label{subsec:related_approaches}

A related, but different technique to deal with these optimization problems is given by \citet{langford2001not} and \citet{catoni2003pac} to solve the bounded losses analogue of~\Cref{lemma:extension_banerjee}. They consider the optimization of $\lambda$ over a geometric grid $\cA = \{e^k : k \in \bN\} \cap [1,n]$ at the smaller union bound cost of $\log \nicefrac{(1+ \log n)} \beta $ at the price of a multiplicative constant of $e$. Using
rounding arguments similar to those in the proof of \Cref{th:pac_bayes_chernoff_analogue}, this translates into being able to optimize the parameter $\lambda$ in the region $[1,n]$. This technique generalizes to other countable families $\cA$ with a union bound cost of $\log |\cA|$~\citep[Section 2.1.4]{alquier2021user}. The downfall of this approach compared to the one presented here is that the optimal parameter $\lambda^\star$ is still dependent on the data drawn $S$, the probability parameter $\beta$, and the tail behavior captured by $\psi_*^{-1}$. It is, hence, uncertain if the optimal parameter will lie within the set $\cA$ in general, making a parameter-free expression for the bound impossible.

An extension of this technique is given by~\citet{seldin2012pac}. The idea is to construct a countably infinite grid $\cA$ over the parameters' space and then choose a parameter $\lambda$ from that grid. Then, they can give a closed-form solution by studying how far is the bound resulting from plugging the selected parameter from the grid and the optimal parameter. Their technique has been used for a bounded range and bounded variance setting in~\citep{seldin2012pac} and for a bounded empirical variance in~\citep{tolstikhin2013pac}. 

The main difference between these approaches and ours is that they design a grid $\cA$ over the parameters' space and optimize the parameter $\lambda$ in that grid. Then, the tightness of the resulting bound depends on how well that grid was crafted. This grid $\cA$ needs to be designed on a case-to-case basis and it can be cumbersome (see, for example, \citep[Appendix A]{tolstikhin2013pac}). Moreover, to design the said grid one requires an explicit expression for the optimal parameter. This may not be available in cases such as in \Cref{th:pac_bayes_chernoff_analogue}, where we only know that~\eqref{eq:pac_bayes_cgf_with_k} is the result of the optimization in~\eqref{eq:pac_bayes_cgf_with_lambda_and_k}. On the other hand, we consider a grid over the events' space and find the \emph{best} parameter for each cell (sub-event) in that grid. This gives three main advantages with respect to the previous techniques. First, the grid is the same for any situation, making the technique easier to employ (see, for example, \Cref{th:pac_bayes_chernoff_analogue,th:parameter_free_anytime_valid_bounded_2nd_moment,th:pac_bayes_chernoff_analogue_no_cutoff,th:paramter_free_anytime_valid_martingales,th:pac_bayes_chernoff_analogue_loglog}). For instance, it would be trivial to recover a result similar to the PAC-Bayes Bernstein analog of \citet[Theorem 8]{seldin2012pac} optimizing the parameter in \citep[Theorem 7]{seldin2012pac} with our approach. Second, to apply the technique, we do not need to know the explicit form of the optimal parameter, which may not exist like in \Cref{th:pac_bayes_chernoff_analogue}, we only need that the optimization is possible. Third, if the grid is made with respect to a random variable $X$, the resulting bound will be tight except from a logarithmic term and an offset changing $X$ by $X + 1$. Therefore, discretizing the events' space is essentially equivalent to crafting a subset $\cA'$ of the parameters' space (not necessarily with a grid structure) with the \emph{optimal} parameters for each region without the need to design this subset $\cA'$ in a case-to-case basis. 

Another possibility to deal with $\lambda$ is to integrate it with respect to an analytically integrable probability density with mass concentrated in its maximum. This is the method employed by~\citet{kuzborskij2019efron} and is known as \emph{the method of mixtures}~\citep[Section 2.3]{de2007pseudo}. Unfortunately, this method requires the existence of a canonical pair: two random variables $X$ and $Y$ satisfying that $\bE \exp \big( \lambda X + \nicefrac{\lambda^2 Y^2}{2} \big) \leq 1$ for all $\lambda$ in the domain of optimization~\cite[Equation (2.2)]{de2007pseudo}. This requirement may not necessarily hold in general settings like \Cref{th:pac_bayes_chernoff_analogue}. Moreover, often this method results in the introduction of a new parameter associated with the density used for integration, for example, the variance of a Gaussian as in \citep{kuzborskij2019efron}. Therefore, our proposed approach is still more general while resulting in essentially the same bound when restricted to the case where the method of mixtures can be employed.\footnote{The final bounds are not directly comparable due to differences in the logarithmic terms, but both are of the same order.}

Finally, \citet[Corollary 8]{kakade2008complexity} employed a similar technique to ours to prove a PAC-Bayes bound for bounded losses similar to \citet{mcallester1998some,mcallester1999pac,mcallester2003pac}'s \Cref{th:mcallester}. However, they did not employ the technique to optimize a parameter. Instead, they found a bound in terms of a threshold $a$ that held for every posterior $\bP_{W}^S$ such that $\relent(\bP_W^S \Vert \bQ) \leq a$. Then, they discretized the set of all posteriors into the sub-classes $\cP_{k} \coloneqq \{ \bP_W^S : 2^{k+1} < \relent(\bP_W^S \Vert \bQ) \leq 2^{k+2} \}$ and applied the union bound to find a uniform result. This technique is usually known as the \emph{peeling device}, \emph{stratification}, or \emph{slicing} in the probability theory and bandits communities~\citep[Section 13.7]{boucheron2003concentration}~\citep[Section 9.1]{lattimore2020bandit}. The similarity with our proof of~\Cref{th:pac_bayes_chernoff_analogue,th:pac_bayes_chernoff_analogue_no_cutoff,th:pac_bayes_chernoff_analogue_loglog} is clear by looking at our design of the events' discretization and their posterior's sub-classes. However, the nature of the two approaches is different: they have a natural constraint, and they discretize the posterior class space and apply the union bound to circumvent that; while we have a parameter whose optimal value is data-dependent, we discretize the events' space to find the optimal parameter in a data-independent way, and then we apply the union bound. Moreover, this technique is more general, as one can design the sub-events to include basically any random object that depends on the data as showcased in~\Cref{th:parameter_free_anytime_valid_bounded_2nd_moment,th:paramter_free_anytime_valid_martingales}. Nonetheless, one could consider our technique to be essentially equivalent to the peeling device since both techniques have the same idea and intention behind them.

\subsection{Further Bibliography About PAC-Bayesian Bounds for Unbounded Losses}
\label{sec:pac_bayes_unbounded_bibliographic_remarks}

\paragraph{Losses With a Bounded CGF.}
Prior to our development of the general PAC-Bayesian Chernoff analogue from~\Cref{th:pac_bayes_chernoff_analogue,th:pac_bayes_chernoff_analogue_loglog,th:pac_bayes_chernoff_analogue_no_cutoff} in~\citep{rodriguez2023morepac}, there were other attempts at developing PAC-Bayesian bounds for specific tail behaviors that are now a particularization of the presented theorems. More precisely,
\begin{itemize}
    \item \citet{catoni2004statistical} derived PAC-Bayes bounds for sub-exponential losses. However, these bounds are limited to the squared error loss in regression scenarios, where $\cZ = \cX \times \bR$  and the hypothesis $w$ represents the parameters of a regressor $\phi_w: \cX \to \bR$. The analysis assumes that the regressor is finite, i.e. $\lVert \phi_w \rVert_\infty < \infty$ for all $w \in \cW$. Additionally, the derived bounds also rely on a parameter that must be chosen \emph{before} the draw of the data.
    \item \citet{hellstrom2020generalization} and \citet{guedj2021still} obtained parameterized bounds for sub-Gaussian losses. They also provided a parameter-free version of the bound optimizing the parameter. However, the optimization contained a small mistake, as the parameter needs to be selected \emph{before} the draw of the data, and the value they chose depended on the data realization~\citep[Remark 14]{banerjee2021information}. \citet{hellstrom2021corrections} later resolved this issue to obtain analogous PAC-Bayes bounds employing properties unique to sub-Gaussian random variables~\citep[Theorem 2.6]{wainwright2019high}. \citet{esposito2021generalization} also derived PAC-Bayes bounds for this setting considering different dependency measures.
\end{itemize}

\paragraph{Losses With Heavy Tails.}
Similarly, before the development of our bounds for losses with heavy tails from~\Cref{subsec:losses_with_bounded_moment}, inspired by~\citet{alquier2006transductive}'s method, there were other people studying this kind of inequalities. 
\citet{alquier2018simpler} also developed PAC-Bayes bounds for losses with heavier tails that sometimes work for non i.i.d. data, although they are not of high probability and consider $f$-divergences as the dependency measure. \citet{holland2019pac} found PAC-Bayes bounds for losses with bounded second and third moments, but consider a different estimate than the empirical risk, and their bounds contain a term that may increase with the number of samples $n$. Finally, \citet{kuzborskij2019efron} and \citet{haddouche2023pacbayes} developed bounds for losses with a bounded second moment. The bound in~\citep{haddouche2023pacbayes} is anytime valid but depends on a parameter that needs to be chosen \emph{before} the draw of the training data.

\paragraph{Other Types of Unbounded Losses.}
\citet{haddouche2021pac} developed PAC-Bayes bounds under a different generalization, namely the hypothesis-dependent range (HYPE) condition, i.e., that there is a function $\kappa$ with positive range such that $\sup_{z \in \cZ} \ell(w,z) \leq \kappa(w)$ for all hypotheses $w \in \cW$, but their bounds decrease at a slower rate %
than the classical \citet{mcallester2003pac}'s \Cref{th:mcallester}
when they are restricted to the bounded case. Finally, \citet{chugg2023unified} also proved anytime-valid bounds for bounded CGFs and bounded moments, although their bounds contain parameters that need to be chosen \emph{before} the draw of the training data with other technical conditions.

\section{Single-Draw PAC-Bayesian Bounds}
\label{sec:single-draw-pac-bayes}

Single-draw PAC-Bayesian bounds try to bound the population risk of the hypothesis returned by the algorithm with high probability. They are one step above the standard PAC-Bayesian bounds in the specificity ladder from~\Cref{subsec:levels_of_specificity}. More precisely, instead of bounding the average population risk of hypotheses returned by the algorithm after the draw of the data, they bound the population risk of a hypothesis sampled from the algorithm's distribution after the draw of the data. The main disadvantage is that the dissimilarity between the posterior $\bP_W^S$ and the prior $\bQ$ is now a function that depends on both the drawn training set $s$ and the returned hypothesis $w$, making the bounds harder to optimize explicitly. As far as we know, these bounds were first introduced by~\citet{catoni2003pac,catoni2007pac}.

Almost every PAC-Bayesian bound in this chapter can be replicated into a single-draw PAC-Bayesian bound. Instead of starting from~\Cref{th:germain_convex_pac_bayes}, we may start with \citet[Theorem 1 (i)]{rivasplata2020pac}. This result is essentially a single-draw PAC-Bayes version of~\Cref{th:germain_convex_pac_bayes} without the requirement for the loss to be convex or to depend on the population and empirical risks. 

Similarly to~\citet{germain2009pac}'s result, \citet{rivasplata2020pac}'s result requires simultaneously that $\bP_W^S \ll \bQ$ and that $\bQ \ll \bP_W^S$ a.s. since at some point in their proof they use the equality $\nicefrac{\rmd \bP_W^S}{\rmd \bQ} = \big(\nicefrac{\rmd \bQ}{\rmd \bP_W^S}\big)^{-1}$, which only holds when this happens. Similarly to~\citet{begin2014pac}, who lifted the requirement that $\bQ \ll \bP_W^S$ a.s., we present below \citet{rivasplata2020pac}'s result without that extra requirement as well as a simple proof to avoid that requirement.

\begin{theorem}[{Extension of~\citet[Theorem 1 (i)]{rivasplata2020pac}}]
    \label{th:single_draw_general_theorem}
    Consider a measurable function $f: \cW \times \cS \to \bR$. Let $\bQ$ be a distribution on $\cW$ independent of $S$ such that $\bP_W^S \ll \bQ$ a.s. 
    and $W'$ be a random variable distributed according to $\bQ$. Then, for every $\beta \in (0,1)$, with probability no smaller than $1 - \beta$
    \begin{equation*}
        f(W,S) \leq \frac{1}{n} \mleft( \log \frac{\rmd \bP_W^S}{\rmd \bQ}(W) + \log \frac{1}{\beta} + \log \bE \mleft[ e^{n f(W',S)} \mright] \mright).
    \end{equation*}
\end{theorem}

To prove the theorem, consider the non-negative random variable
\begin{equation*}
    X = e^{n f(W,S) - \log \frac{\rmd \bP_W^S}{\rmd \bQ}(W)}.
\end{equation*}
Furthermore, note that by the change of measure~\eqref{eq:change_of_measure} from~\Cref{subsec:polish_borel_spaces} we have that
\begin{align*}
    \bE \mleft[ \bE^S \mleft[ e^{n f(W,S) - \log \frac{\rmd \bP_W^S}{\rmd \bQ}(W)} \mright] \mright] &= \bE \mleft[ \bE^S \mleft[ e^{n f(W',S) - \log \frac{\rmd \bP_W^S}{\rmd \bQ}(W')} \cdot \frac{\rmd \bP_W^S}{\rmd \bQ}(W') \mright] \mright] \\
    &= \bE \mleft[ e^{n f(W',S)} \mright].
\end{align*}
Then, applying Markov's inequality~\eqref{eq:markov} to the random variable $X$ and letting $t = \frac{1}{\beta} \cdot \bE \mleft[ e^{n f(W',S)} \mright]$ we have that
\begin{align*}
    \bP \mleft[ e^{n f(W,S) - \log \frac{\rmd \bP_W^S}{\rmd \bQ}(W)} \geq \frac{1}{\beta} \cdot  \bE \mleft[ e^{n f(W',S)} \mright]\mright] \leq \beta.
\end{align*}

Finally, since the logarithm is a non-decreasing, monotonic function we can take the logarithm to both sides of the inequality and re-arrange the terms to obtain the desired result.

In the following sub-sections, we will show how to recover all the presented results as a single-draw PAC-Bayesian bound.

\subsection{Losses With a Bounded Range} 

Let $f(w,s) = \relentber(\emprisk(w,s) \Vert \poprisk(w))$. In this way, we effectively recover the Seeger--Langford bound~\citep{seeger2002pac,langford2001bounds} with~\citet{maurer2004note}'s constant on the logarithm. 

\begin{theorem}
    \label{th:single-draw-seeger-langford}
    Consider a loss with a range bounded in $[0,1]$ and let $\bQ$ be any prior independent of $S$ such that $\bP_W^S \ll \bQ$ almost surely. Then, for every $\beta \in (0,1)$, with probability no smaller than $1 - \beta$
    \begin{equation*}
        \relentber(\emprisk(W,S) \Vert \poprisk(W)) \leq \frac{\log \frac{\rmd \bP_W^S}{\rmd \bQ}(W) + \log \frac{\xi(n)}{\beta}}{n}.
    \end{equation*}
\end{theorem}

\looseness=-1 Similarly, we obtain a single-draw PAC-Bayes analogue to~\citet[Theorem 1.2.6]{catoni2007pac}'s bound by letting $f(w,s) = - \log\big(1 - \poprisk(w)(1-e^{-\nicefrac{\lambda}{n}})\big) - \frac{\lambda}{n} \cdot \emprisk(w,s)$.

\looseness=-1 However, the existence of~\Cref{th:single-draw-seeger-langford} is more important since from there one can obtain the single-draw PAC-Bayesian equivalents of~\Cref{th:catoni_pac_bayes_uniform}, \citet{thiemann2017strongly}'s \Cref{th:thiemann_pac_bayes}, \citet{rivasplata2019pac}'s \Cref{th:rivasplata_pac_bayes}, and our fast-rate and mixed-rate bounds from~\Cref{th:fast_rate_bound_strong,th:mixed_rate_bound} and~\Cref{cor:fast_rate_bound}.

Since~\Cref{th:fast_rate_bound_strong} is the cornerstone of many of the bounds for unbounded losses from the next sub-section, we also write it explicitly below. However, this time, we already present it for losses that are bounded in $[0,b]$ from~\Cref{rem:fast_rate_strong_bounded_b}, as this is the version required to build the bounds for losses with a bounded moment.

\begin{theorem}
    \label{th:single-draw-fast-rate}
    Consider a loss with a range bounded in $[0,b]$ and let $\bQ$ be any prior independent of $S$ such that $\bP_W^S \ll \bQ$ almost surely. Then, for every $\beta \in (0,1)$, with probability no smaller than $1 - \beta$
    \begin{equation*}
        \poprisk(W) \leq c \gamma \log\mleft( \frac{\gamma}{\gamma - 1} \mright) \cdot \emprisk(W,S) + b c \gamma \cdot \frac{\log \frac{\rmd \bP_W^S}{\rmd \bQ}(W) + \log \frac{\xi(n)}{\beta}}{n} + b \kappa(c) \gamma
    \end{equation*}
    holds \emph{simultaneously} for all $\gamma > 1$ and all $c \in (0,1]$, where $\kappa(c) \coloneqq 1 - c (1 - \log c)$.
\end{theorem}

\subsection{Losses With a Bounded CGF}
\label{subsec:single_draw_bounded_cgf}

In order to obtain a single-draw PAC-Bayesian analogue to Chernoff's inequality, we may just let $f(w,s) = \lambda (\poprisk(w) - \emprisk(w,s)) = \lambda \gen(w,s)$ similarly to what we did to obtain \citet{banerjee2021information}'s~\Cref{lemma:extension_banerjee}. In this way, if the loss has a bounded CGF in the sense of~\Cref{def:bounded_cgf}, we have that for every $\beta \in (0,1)$ and every $\lambda \in (0,b)$, with probability no smaller than $1 - \beta$
\begin{equation}
    \label{eq:single-draw-chernoff-intermediate}
    \gen(W,S) \leq \frac{1}{\lambda} \mleft(\frac{\log \frac{\rmd \bP_W^S}{\rmd \bQ}(W) + \log \frac{1}{\beta}}{n} + \psi(\lambda) \mright).
\end{equation}

In this case, the optimal value of the parameter $\lambda$ from~\eqref{eq:single-draw-chernoff-intermediate} depends both on the draw of the training data $S$ and the hypothesis returned from the algorithm $W$. This is not a problem from the event's space optimization technique developed in~\Cref{subsec:losses_with_bounded_CGF}. We may consider the sub-events $\cE_1 = \{\log \nicefrac{\rmd \bP_W^S}{\rmd \bQ}(W) \leq 1 \}$ and $\cE_k = \{ k-1 < \log \nicefrac{\rmd \bP_W^S}{\rmd \bQ}(W) \leq k \}$ for all $k \in \bN$ such that $k \geq 2$. After that, we may proceed exactly as we did for the proof of~\Cref{th:pac_bayes_chernoff_analogue_no_cutoff} bounding the logarithm of the Radon--Nikodym derivative instead of the relative entropy. Following this procedure results in the desired single-draw PAC-Bayesian Chernoff analogue. 

\begin{theorem}[Single-draw PAC-Bayesian Chernoff analogue]
    \label{eq:single-draw-chernoff-analogue}
    Consider a loss function $\ell$ with a bounded CGF (\Cref{def:bounded_cgf}). Let $\bQ$ be any prior independent of $S$ such that $\bP_W^S \ll \bQ$ almost surely. Then, for every $\beta \in (0,1)$, with probability no smaller than $1 - \beta$
    \begin{equation*}
        \gen(W,S) \leq \psi_*^{-1} \mleft( \frac{1.1 \log \frac{\rmd \bP_W^S}{\rmd \bQ}(W) + \log \frac{10 e \pi^2}{\beta}}{n} \mright).
    \end{equation*}
\end{theorem}

\subsection{Losses With a Bounded $p$-th Moment}
\label{subsec:single_draw_bounded_moment}

Finally, after proving the single-draw PAC-Bayesian fast rate bound from~\Cref{th:single-draw-fast-rate} and describing how to use the events' space quantization technique to optimize parameters in the single-draw PAC-Bayesian setting, extending the results for losses with a bounded moment and a bounded variance from~\Cref{subsec:losses_with_bounded_moment} to this setting is routine.

The only result that does not extend directly to this setting is~\Cref{th:parameter_free_anytime_valid_bounded_2nd_moment}, as we did not derive the starting inequality which comes from a treatment of a certain martingale sequence (see~\Cref{app:closed_form_parameter_free_wang}).

As we did in~\Cref{subsec:losses_with_bounded_moment}, to alleviate the notation, we present the results defining $\kappa_1 \coloneqq c \gamma \log \big( \nicefrac{\gamma}{(\gamma -1)} \big)$, $\kappa_2 \coloneqq c \gamma$, and $\kappa_3 \coloneqq \gamma \big( 1 - c(1 - \log c)\big)$, with the understanding that they are functions of the parameters $c \in (0,1]$ and $\gamma > 1$ from~\Cref{th:single-draw-fast-rate}.

\begin{theorem}
    \label{sec:single-draw-bounded-moment}
    Consider a loss function $\ell(w,Z)$ with a $p$-th moment bounded by $m_p$ for all $w \in \cW$. Let $\bQ$ be any prior independent of $S$ such that $\bP_W^S \ll \bQ$ almost surely. Then, for every $\beta \in (0,1)$, with probability no smaller than $1 - \beta$
    \begin{equation*}
        \poprisk(W) \leq \kappa_1 \cdot \emprisk_{\leq t^\star} + m_p^{\frac{1}{p}} \Big(\frac{p}{p-1}\Big) \Big( \kappa_2 \cdot \frac{1.1 \log \frac{\rmd \bP_W^S}{\rmd \bQ}(W) + \log \frac{10 e \pi^2 \xi(n)}{\beta}}{n} + \kappa_3 \Big)^{\frac{p-1}{p}}
    \end{equation*}
    holds \emph{simultaneously} for all $c \in (0,1]$ and all $\gamma > 1$, where $$t^\star \coloneqq m_p^{\frac{1}{p}} \Big( \kappa_2 \cdot \frac{1.1 \log \frac{\rmd \bP_W^S}{\rmd \bQ}(W) + \log \frac{10 e \pi^2 \xi(n)}{\beta}}{n} + \kappa_3 \Big)^{-\frac{1}{p}}.$$
\end{theorem}

\begin{theorem}
    \label{sec:single-draw-bounded-variance}
    Consider a loss $\ell(w,Z)$ with a variance bounded by $\sigma^2$ for all $w \in \cW$. Let $\bQ$ be any prior independent of $S$ such that $\bP_W^S \ll \bQ$ almost surely. Then, for every $\beta \in (0,1)$, with probability no smaller than $1 - \beta$
    \begin{align*}
        \poprisk(W) \leq \mleft( 1 - 2 \sqrt{\mathfrak{C}_{n,\beta,S,W}'} \mright)_+^{-1}  \mleft( \kappa_1 \cdot \bE^S \big[\emprisk(W,S)\big]  + 2 \sqrt{\sigma^2 \mathfrak{C}_{n,\beta,S,W}'} \mright)
    \end{align*}
    holds \emph{simultaneously} for all $c \in (0,1]$ and all $\gamma > 1$, where
    \begin{equation*}
        \mathfrak{C}_{n,\beta,S,W}' \coloneqq \frac{1.1 \log \frac{\rmd \bP_W^S}{\rmd \bQ}(W) + \log \frac{10 e \pi^2 \xi(n)}{\beta}}{n}.
    \end{equation*}    
\end{theorem}

\section{Anytime Validity}
\label{sec:anytime_validity}

Interactive learning algorithms are mechanisms that take some data from the environment, produce a hypothesis that may or may not interact with the environment, then take more data from the environment, produce a new hypothesis, and so on. An example of these kinds of algorithms are simple iterative algorithms like the ones presented in~\Cref{sec:noisy_iterative_learning_algos} or bandits, online, or reinforcement learning algorithms. 

For these algorithms, it is often interesting to be able to assess the generalization performance of the algorithm throughout their training. For example, checking the generalization error of SGLD at the first epoch, and if we see that it is generalizing poorly, stop the training to avoid wasting resources. This situation is often referred to as ``\emph{peeking}''. However, the bounds presented so far only for a fixed time step, or a fixed number of samples. Therefore, it is of interest to consider bounds that are \emph{anytime valid}, that is, that hold \emph{simultaneously} for each time step.

There have been multiple works deriving anytime-valid bounds~\citep{jang2023tighter,wang2015pac,haddouche2023pacbayes,chugg2023unified}. Most of these works are rooted in the usage of \citet{ville1939etude}'s extension of Markov's inequality~\eqref{eq:markov} on a martingale sequence. \citet{chugg2023unified} derived a recipe to recover most of the usual parameterized PAC-Bayes bounds (including the Seeger--Langford~\citep{seeger2002pac,langford2001bounds} from~\Cref{th:seeger_langford_pac_bayes}, \citet{mcallester1998some,mcallester1999pac,mcallester2003pac}'s \Cref{th:mcallester}, and \citet{banerjee2021information}'s \Cref{lemma:extension_banerjee}) with an extension of Ville's inequality for both forward supermartingales and reverse submartingales. 

However, every standard PAC-Bayesian and single-draw PAC-Bayesian bound can be extended to an anytime-valid bound at a union bound cost, even if it does not have a suitable supermartingale or reverse submartingale structure. For high-probability PAC-Bayes bounds like the ones presented throughout this chapter, this extension comes at the small cost of adding $2 \log \frac{\pi t}{\sqrt{6}}$ to the complexity terms. This ``folklore'' result is formalized below for general probabilistic bounds. %

\begin{theorem}[{From standard to anytime-valid bounds}]
\label{th:standard_to_anytime_valid}
    Consider the probability space $(\cA, \ccA, \bP)$ and let $(\cE_t)_{t=1}^\infty$ be a sequence of event functions such that $\cE_t: (0,1) \to \ccA$. If $\bP[ \cE_t(\beta) ] \geq 1 - \beta$ for all $\beta \in (0,1)$ and all $t \geq 1$, then $\bP[ \cap_{t=1}^\infty \cE_t(\nicefrac{6 \beta}{\pi^2 t^2}) ] \geq 1 - \beta$ for all $\beta \in (0,1)$.
\end{theorem}

The proof of this statement is simple. Consider the equivalent statement: ``for every $\beta \in (0,1)$, if $\bP[ \cE_t^c (\beta) ] < \beta$ for all $t \geq 1$, then $\bP[ \cup_{t=1}^\infty \cE_t^c (\nicefrac{6\beta}{\pi^2 t^2})] < \beta$''. By the union bound, it follows that $\bP[ \cup_{t=1}^\infty \cE_t^c (\beta_t) ] < \sum_{t=1}^\infty \beta_t$. Let $\beta_t = \nicefrac{\beta}{t^2}$, then $\bP[ \cup_{t=1}^\infty \cE_t^c (\nicefrac{\beta}{t^2}) ] < \nicefrac{\pi^2 \beta}{6}$. Then, the substitution $\beta \leftarrow \nicefrac{6 \beta}{\pi^2}$ completes the proof.

There are better choices of $\beta_t$ such as $\beta_t = \nicefrac{\beta}{t \log^2(6t)}$~\citep{kaufmann2016complexity}, but all result in essentially the same cost $\cO(\log t)$ for high-probability PAC-Bayes bounds. In fact, 
most of these bounds are applied to bandits, online, or reinforcement learning algorithms where they only take one sample at each time step, and therefore the penalty $\cO(\log t)$ is of $\cO(\log n)$.
The main takeaway from this result is that the anytime-valid bounds obtained via martingales and \citet{ville1939etude}'s inequality only contribute to shaving a log factor for the PAC-Bayes high-probability bounds presented in this monograph. Hence, their main advantage is in describing online learning situations where the subsequent samples are dependent on each other, which is not inherently captured by statements like \Cref{th:standard_to_anytime_valid}.

\begin{remark}
\label{rem:specialized_bounds_anytime_valid}
\looseness=-1 \Cref{th:catoni_pac_bayes_uniform,th:fast_rate_bound_strong,th:mixed_rate_bound} and \Cref{cor:fast_rate_bound} follow verbatim as an anytime-valid bound substituting $\log \nicefrac{\xi(n)}{\beta}$ by $\log \nicefrac{\sqrt{\pi(n+1)}}{\beta}$ without needing \Cref{th:standard_to_anytime_valid}. The reason is that these results are derived from the Seeger--Langford bound~\citep{langford2001bounds, seeger2002pac, maurer2004note}, which is extended to an anytime-valid bound at this cost in \citep{jang2023tighter}.
\end{remark}

\begin{subappendices}
\section{Comparison Between the Fast- and Mixed-Rate Bounds}
\label[appendix]{app:comparison_fast_rate_bounds}

Just by inspecting their equations, it is apparent that the proposed mixed-rate bound of  \Cref{th:mixed_rate_bound} is tighter than those in~\citep{tolstikhin2013pac,rivasplata2019pac}. However, it is not directly obvious that the presented fast-rate bound of \Cref{th:fast_rate_bound_strong} is tighter than \citet{thiemann2017strongly}'s \Cref{th:thiemann_pac_bayes}. In fact, even~\Cref{cor:fast_rate_bound} is tighter than this result.

To show this, we will show the stronger statement that $f_{\textnormal{fr}}(r,c) \leq f_{\textnormal{th}}(r,c)$ for all $r, c \geq 0$, where 
\begin{align*}
    f_{\textnormal{th}}(r,c) &= \inf_{\lambda \in  (0,2)} \Bigg \{ \frac{r}{1 - \frac{\lambda}{2}} + \frac{c}{\lambda \big(1 - \frac{\lambda}{2} \big)} \Bigg \} \textnormal{ and } \\ f_{\textnormal{fr}}(r,c) &= \inf_{\gamma > 2} \Bigg \{ \gamma \log \Big(\frac{\gamma}{\gamma - 1} \Big) r +  \gamma \ c\Bigg \}.
\end{align*}
If this holds, then \Cref{cor:fast_rate_bound} is tighter than \Cref{th:thiemann_pac_bayes} as enlarging the optimization set in $f_{\textnormal{fr}}(r,c)$ from $\{\gamma>2\}$ to $\{\gamma>1\}$ will only improve the bound.

Note that with the change of variable $\gamma = \big( \lambda (1 - \nicefrac{\lambda}{2}) \big)^{-1}$, if $\lambda \in (0,2)$, then $\gamma > 2$. This way, we may re-write $f_{\textnormal{fr}}$ in terms of a minimization over $\lambda \in (0,2)$
\begin{align*}
    f_{\textnormal{fr}}(r,c) = \inf_{\lambda \in (0,2)} \Bigg \{ \frac{r}{\lambda \big(1 - \frac{\lambda}{2} \big)} \log \frac{2}{(\lambda - 2) \lambda + 2} + \frac{c}{\lambda \big(1 - \frac{\lambda}{2} \big)} \Bigg \}.
\end{align*}

Finally, noting that 
\begin{equation*}
    \frac{1}{\lambda} \log \frac{2}{(\lambda - 2) \lambda + 2} \leq 1
\end{equation*}
for all $\lambda \in (0,2)$ completes the proof.

Similarly, it can also be shown that the mixed-rate bound from \Cref{th:mixed_rate_bound}, which is itself a relaxation of the fast-rate bound of \Cref{th:fast_rate_bound_strong}, is also tighter than the fast-rate bound from \citep{thiemann2017strongly}. In this case, we will show the stronger statement that $f_{\textnormal{mr}}(r,c) \leq f_{\textnormal{th}}(r,c)$ for all $r, c \geq 0$, where
\begin{equation}
\label{eq:mixed_rate_relaxed}
        f_{\textnormal{mr}}(r,c) = \inf_{\gamma > 2} \Bigg \{ \frac{1}{2} \cdot \frac{2 \gamma - 1}{\gamma-1} r +  \gamma \ c\Bigg \}.
\end{equation}
As above, as \Cref{th:mixed_rate_bound} is the closed-form expression obtained optimizing the equivalent of~\eqref{eq:mixed_rate_relaxed} on the larger set $\{ \gamma > 1 \}$,  showing this statement suffices.

Again, letting $\gamma = \big( \lambda (1 - \nicefrac{\lambda}{2}) \big)^{-1}$, if $\lambda \in (0,2)$ allows us to write $f_{\textnormal{mr}}$ in terms of a minimization over $\lambda \in (0,2)$
\begin{align*}
    f_{\textnormal{mr}}(r,c) = \inf_{\lambda \in (0,2)} \Bigg \{ \frac{1}{2} \cdot \frac{\lambda^2 - 2 \lambda + 4}{\lambda^2 - 2 \lambda + 2} r + \frac{c}{\lambda \big(1 - \frac{\lambda}{2} \big)} \Bigg \}.
\end{align*}

Finally, noting that
\begin{equation*}
    \frac{1}{2} \cdot \frac{\lambda^2 - 2 \lambda + 4}{\lambda^2 - 2 \lambda + 2} \leq \frac{1}{1 - \frac{\lambda}{2}}
\end{equation*}
for all $\lambda \in (0,2)$ completes the proof.

\section{Example: PAC-Bayes With Backprop}
\label[appendix]{app:pbb}

In \Cref{sec:bounds_bounded_losses}, we mention that methods that use PAC-Bayes bounds to optimize the posterior, such as PAC-Bayes with backprop~\citep{rivasplata2019pac,perez2021tighter} could benefit from using the bounds from \Cref{th:fast_rate_bound_strong,th:mixed_rate_bound}. In this subsection of the appendix, we provide an example showcasing that this is the case. 

The PAC-Bayes with backprop method~\citep{rivasplata2019pac, perez2021tighter} considers a model parameterized by $w \in \bR^d$ and a prior distribution $\bQ$ over the parameters, for example $\bQ = w_0 + \sigma_0 \cN(0,I_d)$. Then, the parameters are updated using stochastic gradient descent on the objective
\begin{equation*}
    \emprisk(w,s) + f_{\textnormal{bound}}(w;\bQ),
\end{equation*}
where $\emprisk(w,s)$ is the empirical risk on the training data realization and $f_{\textnormal{bound}}(w;\bQ)$ is extracted from a PAC-Bayes bound evaluated on the parameters $w \in \bR^d$ with prior $\bQ$. With an appropriate choice of the posterior, the bound $f_{\textnormal{bound}}$ function is calculable and the said posterior can be constructed, e.g.  $\bP_{W}^S(w) = w + \sigma_0 \cN(0, I_d)$. After the iterative procedure is completed, the empirical risk is bounded using the Seeger--Langford bound~\citep{langford2001bounds,seeger2002pac} with a Monte Carlo estimate of the posterior parameters of $m$ samples with confidence $1-\beta'$, and the population risk is bounded also using the Seeger--Langford bound~\citep{langford2001bounds,seeger2002pac} with the number of training samples $n$ and a confidence $1- \beta$, amounting for a total confidence of $1 - (\beta' + \beta)$. For more details, please check~\citep{rivasplata2019pac, perez2021tighter}.

Using \citet{thiemann2017strongly}'s \Cref{th:thiemann_pac_bayes}, \citet{rivasplata2019pac}'s \Cref{th:rivasplata_pac_bayes}, or the classical \citet{mcallester2003pac}' bound from~\Cref{th:mcallester} as an objective can be harmful since they penalize too harshly the complexity term dominated by the normalized dependency $\nicefrac{\relent(\bP_W^S \Vert \bQ)}{n}$. Hence, SGD steers the parameters towards places too close to the prior, potentially avoiding other posteriors that achieve lower empirical error and have an overall better population risk. In this sense, it makes sense that bounds such as the proposed fast- and mixed-rate bounds from \Cref{th:fast_rate_bound_strong,th:mixed_rate_bound} or the Seeger--Langford bound~\citep{langford2001bounds,seeger2002pac} (with \citet{reeb2018learning}'s gradients), would lead to said posteriors. This is verified in \Cref{table:results} for a convolutional network and the MNIST dataset. For the fast-rate bound from \Cref{th:fast_rate_bound_strong} and \Cref{cor:fast_rate_bound}, at each iteration the approximately optimal $\gamma$ given after the theorem is employed, thus updating the posterior and the parameter alternately. We saw that the approximation of $\gamma$ is good both by comparing the results of the final posterior in \Cref{table:results} and the coefficients of the empirical risk and the complexity term in \Cref{fig:comparison_coefficients} with those obtained from the Seeger--Langford bound~\citep{langford2001bounds,seeger2002pac} with \citet[Appendix A]{reeb2018learning}'s gradients. After a few iterations, once the empirical risk is small and the \Cref{cor:fast_rate_bound} is a good approximation of \Cref{th:fast_rate_bound_strong}, the gradients are close to each other.

\begin{remark}
    \citet{lotfi2022pac} obtain even tighter population risk certificates for networks on the MNIST dataset (11.6 \%)    
    considering a compression approach to the PAC Bayes bound from~\citet{catoni2007pac}. To be precise, they considered a deterministic posterior that returns a quantized version of the network's parameters, thus reducing \Cref{th:catoni_pac_bayes} to the MDL formalism from~\Cref{sec:mdl_and_occams_razor}. They described the parameters of the networks with a tunable prefix-free variable-length code that acts as the description language $\desc$. Then, both the quantized parameters and the quantization levels of the code are learnt simultaneously to minimize a variant of the MDL generalization guarantee in the \Cref{th:catoni_pac_bayes} equivalent of~\eqref{eq:mdl_generalization_compression} and choosing an appropriate parameter $\lambda$ with a grid search and the use of the union bound. Therefore, their results could be tightened further using our strengthened version from~\Cref{th:catoni_pac_bayes_uniform}. Nonetheless, the goal of this example is not to propose a method that obtains state-of-the-art certificates, but to showcase that the tightness of the tractable bounds in \Cref{sec:bounds_bounded_losses} can improve methods that employ PAC-Bayes bounds to find a suitable posterior.
\end{remark}

\begin{table}[ht]
  \caption{Population risk certificate, empirical risk, and normalized dependency of the posterior obtained with PAC-Bayes with backprop~\citep{rivasplata2019pac,perez2021tighter} using Gaussian priors and different objectives. The best risk certificates are highlighted in bold face, and the second best is highlighted in italics. $^*$The gradients for the Seeger--Langford~\citep{langford2001bounds,seeger2002pac} bound are not calculated from the bound but hard-coded following \citep[Appendix A]{reeb2018learning}}
  \label{table:results}
  \centering
  \begin{tabular}{lccc}
    \toprule
    Objective & Risk certificate & Empirical risk & Normalized dependency \\
    \midrule
    \Cref{th:rivasplata_pac_bayes}~\citep{rivasplata2019pac} & 0.20870 & 0.11372 & 0.03117 \\
    \Cref{th:thiemann_pac_bayes}~\citep{thiemann2017strongly} & 0.21159 &  0.11053 &  0.03526 \\
    \Cref{th:mcallester}~\citep{mcallester2003pac} & 0.23658 & 0.23658 & 0.02715 \\
    \Cref{th:fast_rate_bound_strong} [ours] & \textbf{0.17354} & 0.07064 & 0.04556 \\
    \Cref{cor:fast_rate_bound} [ours] & \textbf{0.17501} & 0.07054 & 0.04649 \\
    \Cref{th:mixed_rate_bound} [ours] & \textit{0.19763} & 0.09214 & 0.04159 \\
    \midrule 
    \Cref{th:seeger_langford_pac_bayes}~\citep{langford2001bounds,seeger2002pac}$^*$ & \textbf{0.16922} & 0.06701 & 0.04594 \\
    \bottomrule
  \end{tabular}
\end{table}

\begin{figure}
    \centering
    \includegraphics[width=0.6\textwidth]{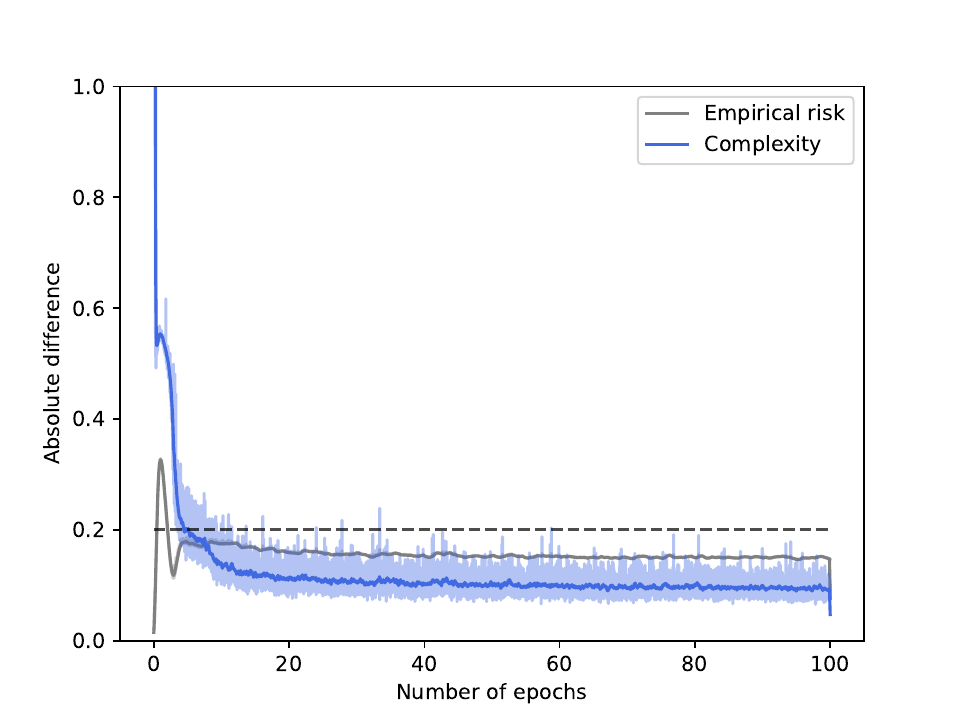}
    \caption{Absolute difference between the coefficients of the empirical risk (gray) and the complexity term (blue) of the gradients of the Seeger--Langford bound~\citep{langford2001bounds,seeger2002pac} from \citep{reeb2018learning} and the fast-rate bound (\Cref{cor:fast_rate_bound}) using the approximately optimal $\gamma$.}
    \label{fig:comparison_coefficients}
\end{figure}

\subsection{Experimental Details}

All calculations were performed using the original code from PAC-Bayes with backprop: \url{https://github.com/mperezortiz/PBB}. The file modified to include our bounds and the hard-coded gradients from \citet{reeb2018learning} is \texttt{bounds.py}. The convolutional network architecture consists of two convolutional layers with 32 and 64 filters, respectively, and a kernel size of 3. The last convolutional layer is followed by a max pooling layer with a kernel size of 2 and two linear layers with 128 and 10 nodes respectively. Between all layers, there is a ReLU activation function.

For all experiments, the standard deviation of the prior was $\sigma_0 = 0.1$. The learning rate was 0.01 for all experiments except for \citet{rivasplata2019pac}'s \Cref{th:rivasplata_pac_bayes} objective which was 0.005. The momentum was 0.99 for all objectives except from \citet{thiemann2017strongly}'s \Cref{th:thiemann_pac_bayes} which was 0.95. The number of Monte Carlo samples was $m=150,000$, the minimum probability $p_{\textnormal{min}}$ (see \citep{perez2021tighter} for the details) was $10^{-5}$, and the confidence parameters were $\beta' = 0.01$ and $\beta = 0.025$ respectively. The networks were trained for 100 epochs and a batch size of 250 to mimic the setting in \citep{perez2021tighter}.

To find the hyper-parameters, we used the same grid search as \citep{perez2021tighter}. That is, the standard deviation of the prior was selected over $\{ 0.005, 0.01, 0.02, 0.03, 0.04, 0.05, 0.1 \}$, the learning rate over $\{0.001, 0.005, 0.01 \}$, and the momentum over $\{0.95, 0.99 \}$. Therefore, the confidence parameters were updated to $\beta' \leftarrow \nicefrac{\beta'}{42}$ and $\beta \leftarrow \nicefrac{\beta}{42}$ respectively to comply with the union bound and maintain the guarantees.

All experiments were done on a TESLA V100 with 32GB of memory. Each full run takes approximately 110 hours with most of the time taken on the Monte Carlo sampling for the risk certificates calculation. For  $42 \cdot 5$ runs, this amounts to approximately 23,100 hours, which is around 32 months.  Since the time was prohibitive for us, the hyper-parameter search was done without the Monte Carlo sampling, where each run took around $25$ minutes amounting to a total of 87.5 hours or less than 4 days. Then, the final certificates were calculated using the full Monte Carlo sampling adding an extra 550 hours or around 23 days. In summary, the total amount of computing was approximately 27 days.

\section{PAC-Bayes Bounds With a Smaller Union Bound Cost}
\label[appendix]{app:pac_bayes_smaller_union_bound_cost}

As discussed in \Cref{subsec:losses_with_bounded_CGF}, the union bound cost of the PAC-Bayesian bounds of the type~\Cref{th:pac_bayes_chernoff_analogue} can be improved at the cost of a multiplicative factor of $e$ to the relative entropy. Below, we present the parallel of \Cref{th:pac_bayes_chernoff_analogue} with this improved union bound cost, and extending this result to the other theorems in~\Cref{sec:bounds_unbounded_losses} follows analogously almost verbatim.

\begin{theorem}
    \label{th:pac_bayes_chernoff_analogue_loglog}
    Consider a loss function $\ell$ with a bounded CGF(\Cref{def:bounded_cgf}). Let $\bQ$ be any prior independent of $S$ and define the event $\cE = \{ \relent(\bP_W^S \lVert \bQ) \leq n \}$. Then, for every $\beta \in (0,1)$ with probability no smaller than $1-\beta$
    \begin{align*}
        \bE^S \big[ \poprisk(W) \big] \leq &\bI_{\cE} \cdot \Bigg[ \bE^S \big[ \emprisk(W,S) \big] + \psi_*^{-1} \bigg( \frac{e \max \{\relent(\bP_{W}^S \lVert \bQ),1\} + \log \frac{2+\log n}{\beta}} {n} \bigg) \Bigg] \\
        &+ \bI_{\cE^c} \cdot \esssup \bE^S \big[ \poprisk(W) \big]
    \end{align*}
    holds \emph{simultaneously} for all posteriors $\bP_W^S$.
\end{theorem}

\begin{proof}
    Let $\cB_\lambda$ be the complement of the event in~\Cref{lemma:extension_banerjee} such that $\bP[\cB_\lambda] < \beta$ and consider the sub-events $\cE_0 \coloneqq \{\relent(\bP_{W}^S \lVert \bQ) \in [0,1] \}$, $\cE_1 \coloneqq \{ \relent(\bP_{W}^S \lVert \bQ) \in (1,e] \}$, and $\cE_k \coloneqq \{ \relent(\bP_W^S \lVert \bQ) \in (e^{k-1}, e^k] \}$ for all $k = 2, \ldots, n$, which form a covering of the event $\cE \coloneqq \{ \relent(\bP_W^S \lVert \bQ) \leq n \}$. Furthermore, define $\cK \coloneqq \{ k \in \bN \cup \{0\}: 0 \leq k \leq n \textnormal{ and } \bP[\cE_k] > 0 \}.$  For all $k \in \cK \setminus \{ 0 \}$, given the event $\cE_k$, with probability no more than $\bP[\cB_\lambda | E_k]$, there exists a posterior $\bP_W^S$ such that
    \begin{equation}
        \label{eq:pac_bayes_cgf_with_lambda_and_k_loglog}
        \bE^S \big[ \poprisk(W) \big] > \bE^S \big[ \emprisk(W,S) \big] + \frac{1}{\lambda} \bigg[ \frac{e^k + \log \frac{1}{\beta}}{n} + \psi(\lambda) \bigg],
    \end{equation}
    for all $\lambda \in (0, b)$. The right-hand side of \eqref{eq:pac_bayes_cgf_with_lambda_and_k_loglog} can be minimized with respect to $\lambda$ \emph{independently of the training set $S$}. Let $\cB_{\lambda_k}$ be the event resulting from this minimization and note that $\bP[\cB_{\lambda_k}] < \beta$. According to  \citep[Lemma 2.4]{boucheron2003concentration}, this ensures that, with probability no more than $\bP[\cB_{\lambda_k} | \cE_k]$, there exists a posterior $\bP_W^S$ such that
    \begin{equation*}
        \bE^S \big[ \poprisk(W) \big] > \bE^S \big[ \emprisk(W,S) \big] +\psi_{*}^{-1} \bigg( \frac{e^k + \log \frac{1}{\beta}}{n} \bigg),
    \end{equation*}
    where $\psi_*$ is the convex conjugate of $\psi$ and where $\psi_*^{-1}$ is a non-decreasing concave function. Given $\cE_k$, since $e^k < e\relent(\bP_W^S \Vert \bQ)$, with probability no larger than $\bP[\cB_{\lambda_k} | \cE_k]$, there exists a posterior $\bP_W^S$ such that
    \begin{equation*}
        \bE^S \big[ \poprisk(W) \big] > \bE^S \big[ \emprisk(W,S) \big] +\psi_{*}^{-1} \bigg( \frac{e\relent(\bP_W^S \Vert \bQ) + \log \frac{1}{\beta}}{n} \bigg).
    \end{equation*}
    Now, define $\cB'$ as the event stating that there exists a posterior $\bP_W^S$ such that
    \begin{align*}
        \bE^S \big[ \poprisk(W) \big] > &\bI_{\cE} \cdot \Bigg[ \bE^S \big[ \emprisk(W,S) \big] + \psi_{*}^{-1} \bigg( \frac{e\max \{ \relent(\bP_{W}^{S} \Vert \bQ), 1\} + \log \frac{e}{\beta}}{n} \bigg) \Bigg] \\
        &+ \bI_{\cE^c} \cdot \esssup \bE^S \big[ \poprisk(W) \big]
    \end{align*}
    where $\bP[\cB' | \cE_k] \bP[\cE_k] \leq \bP[\cB_{\lambda_k} | \cE_k] \bP[\cE_k] \leq \bP[\cB_{\lambda_k}] \leq \beta$ for all $k \in \cK$, and where $\bP[\cB' \cap \cE^c] = 0$ by the definition of the essential supremum. Note that, if $\{ 0 \} \in \cK$, the case for $k=0$ is handled by the addition of the maximum $\max \{  \relent(\bP_{W}^{S} \Vert \bQ), 1\}$ to the equation defining the event $\cB'$. Therefore, the probability of $\cB'$ is bounded as
    \begin{equation*}
        \bP[\cB'] = \sum_{k \in \cK} \bP[\cB' | \cE_k] \bP[\cE_k] + \bP[\cB' \cap \cE^c] < (2 + \log n) \beta.
    \end{equation*}
    Finally, the substitution $\beta \leftarrow \beta/ (2 + \log n)$ completes the proof.
\end{proof}

\section{Alquier's Alternative Truncated Loss for Losses With Bounded Moments}
\label[appendix]{app:alternative_choice_of_loss}

As discussed in~\Cref{rem:alternative_choice_of_loss}, \citet[Theorem 2.7]{alquier2006transductive} presented a result similar to \Cref{lemma:alquier_truncation_method_refined_bouned_moment} for losses with a bounded $p$-th moment. However, he did not obtain it with the straightforward technique outlined in~\Cref{subsec:losses_with_bounded_moment}. Instead, he considered the truncated loss function 
\begin{equation*}
    \ell_{p,\nicefrac{n}{\lambda}}(w,z) = \bigg[ \ell(w,z) - \frac{1}{p} \Big( \frac{p - 1}{p} \Big)^{p-1} \Big( \frac{\lambda}{n} \Big)^{p - 1} \cdot |\ell(w,z)|^p \bigg]_+.
\end{equation*}

Importantly, this loss function satisfies that $\ell_{p,\nicefrac{n}{\lambda}} \leq \nicefrac{n}{\lambda}$. Then, let $\poprisk_{p,\nicefrac{n}{\lambda}}$ be the population risk associated to $\ell_{p,\nicefrac{n}{\lambda}}$. It directly follows that
\begin{equation*}
    \bE^S \big[\poprisk(W) \big] \leq \bE^S \big[\poprisk_{p,\nicefrac{n}{\lambda}}\big] + \frac{1}{p} \Big(\frac{p-1}{p} \Big)^{p-1}  \Big( \frac{\lambda}{n} \Big)^{p - 1} \cdot \bE^S|\ell(W,Z)|^p.
\end{equation*}

\looseness=-1 In this way, like before, the term $\poprisk_{p,\nicefrac{n}{\lambda}}$ can be bounded using any standard PAC-Bayes bound for bounded losses and now the second term is bounded by construction. As before, we will present the result using our fast-rate bound from~\Cref{th:fast_rate_bound_strong} instead of a bound \emph{à la} Catoni like~\Cref{th:catoni_pac_bayes}. For this purpose, let $\emprisk_{p, \nicefrac{n}{\lambda}}$ be the empirical risk associated to the loss $\ell_{p, \nicefrac{n}{\lambda}}$.

\begin{lemma}[{Adaptation of \citep[Theorem 2.7]{alquier2006transductive}}]
    \label{lemma:alquier_truncation_method_modified_bounded_moment}
    Consider a loss $\ell(w,Z)$ with a $p$-th moment bounded by $m_p$ for all $w \in \cW$. Then, for every $\beta \in (0,1)$ and all $\lambda > 0$, with probability no smaller than $1 - \beta$
    \begin{align*}
    \label{eq:original_bound_to_alquier}
       \bE^S \big[ \poprisk(W) \big] 
       \leq \kappa_1 \cdot \bE^S \big[ \emprisk_{p, \nicefrac{n}{\lambda}}(W,S) \big] + \kappa_2 \cdot &\frac{\relent(\bP_W^S \Vert \bQ) + \log \frac{\xi(n)}{\beta}}{\lambda} \\
       &+ \kappa_3 \cdot \frac{n}{\lambda}+ \frac{m_p}{p} \Big(\frac{p-1}{p} \Big)^{p-1}  \Big( \frac{\lambda}{n} \Big)^{p - 1}
    \end{align*}
    holds \emph{simultaneously} for all posteriors $\bP_W^S$, all $c \in (0,1]$, and all $\gamma > 1$.
\end{lemma}

Comparing \Cref{lemma:alquier_truncation_method_modified_bounded_moment} to the truncation method with the straightforward \Cref{lemma:alquier_truncation_method_refined_bouned_moment}, we see that the result stemming from \citet{alquier2006transductive}'s modified construction improves the constant of the term associated to the tail from $\nicefrac{1}{p-1}$ to $(\nicefrac{p-1}{p})^{p-1} \cdot \nicefrac{1}{p}$. For $p=2$, the constant is $4$ times smaller changing from $1$ to $\nicefrac{1}{4}$; and for $p \to \infty$ the constant is $e$ times smaller, although both tend to 0. On the other hand, $\emprisk_{\leq \nicefrac{n}{\lambda}}$ has the potential to be smaller than $\emprisk_{p, \nicefrac{n}{\lambda}}$. The results derived in the rest of the paper use \Cref{lemma:alquier_truncation_method_refined_bouned_moment} as a starting point, but analogous results trivially follow from \Cref{lemma:alquier_truncation_method_modified_bounded_moment} with slightly different constants and changing $\emprisk_{\leq \nicefrac{n}{\lambda}}$ to $\emprisk_{p, \nicefrac{n}{\lambda}}$.

\section{A Parameter-Free Bound of a PAC-Bayesian Bound on Martingales}
\label[appendix]{app:closed_form_parameter_free_wang}

\citet{wang2015pac} and \citet{haddouche2023pacbayes} investigate the setting where the dataset $S$ is considered to be a sequence $S^* \coloneqq (Z_i)_{i \geq 1}$ such that $Z_i \in \cZ$, but where there is no restriction in the distribution of the samples $Z_i$, that is, every sample $Z_i$ can depend on all the previous ones. For every $n$, they let $S_n \coloneqq (Z_1, \ldots, Z_n)$ be the restriction of $S^*$ to its first $n$ points. Then, they consider the sequence of $\sigma$-fields $(\ccF_i)_{i \geq 1}$ to be a filtration adapted to $S^*$, for instance $\ccF_i = \sigma(Z_1, \ldots, Z_i)$. Finally, they consider a martingale difference sequence $(X_i(S_i, w))_{i \geq 1}$ indexed by a hypothesis $w \in \cW$ so that $\bE^{\ccF_{i-1}} \big[ X_i(S_i,w) \big] = 0$ for all $w \in \cW$. For instance, let $Y_0 = \sum_{i=1}^n \bE \bE\big[ \ell(w,Z_i) \big]$ and $Y_i(S_i, w) = \sum_{i=1}^n \bE^{\ccF_i} \big[ \ell(w, Z_i) \big]$ for all $i \geq 1$, then $X_i(S_i, w) = Y_i - Y_{i-1}$. Finally, for all $w \in \cW$, they define the martingale $M_n(w) \coloneqq \sum_{i=1}^n X_i(S_i,w)$ and follow \citet{bercu2008exponential} to also define
\begin{equation*}
    [M]_n(w) \coloneqq \sum_{i=1}^n X_i(S_i,w)^2 \textnormal{ and } \langle M \rangle_n(w) \coloneqq \bE^{\ccF_{i-1}} \bE\big[ X_i(S_i,w)^2 \big],
\end{equation*}
where $[M]_n(w)$ acts as an empirical variance term and $\langle M \rangle_n(w)$ as its theoretical counterpart~\citep{haddouche2023pacbayes}.
Then, their main anytime-valid bound for martingales is the following.

\begin{theorem}[{\citet[Theorem 2.4]{wang2015pac}}]
    \label{th:wang_main_th}
    Let $\bQ$ be any prior independent of $S_n$ and $(M_n(w))_{n \geq 1}$ be any sequence of martingales indexed by $w \in \cW$. Then, for all $\lambda > 0$, all $\beta \in (0,1)$ with probability no smaller than $1-\beta$
    \begin{equation}
        \label{eq:wang_main_th}
        \Big| \bE^{S_n} \big[M_n(W)\big] \Big| \leq \frac{\relent(\bP_W^{S_n} \Vert \bQ) + \log \frac{2(n+1)^2}{\beta}}{\lambda} + \frac{\lambda}{6} \cdot \bE^{S_n} \big[ [M]_n(W) + 2 \langle M \rangle_n(W) \big]
    \end{equation}
    holds \emph{simultaneously} for all posteriors $\bP_W^S$ and all $n \geq 1$.
\end{theorem}

\begin{theorem}[{\citet[Theorem 2.1]{haddouche2023pacbayes}}]
    \label{th:haddouche_main_th}
    Let $\bQ$ be any prior independent of $S_n$ and $(M_n(w))_{n \geq 1}$ be any sequence of martingales indexed by $w \in \cW$. Then, for all $\lambda > 0$, all $\beta \in (0,1)$, with probability no smaller than $1-\beta$
    \begin{equation*}
        \label{eq:haddouche_main_th}
        \Big| \bE^{S_n} \big[M_n(W)\big] \Big| \leq \frac{\relent(\bP_W^{S_n} \Vert \bQ) + \log \frac{2}{\beta}}{\lambda} + \frac{\lambda}{2} \cdot \bE^{S_n} \big[ [M]_n(W) + \langle M \rangle_n(W) \big]
    \end{equation*}
    holds \emph{simultaneously} for all posteriors $\bP_W^S$ and all  $n \geq 1$.
\end{theorem}

In what follows, we will focus on the result from \citet{wang2015pac} as it has the smaller constants. Taking a closer look at \Cref{th:wang_main_th}, we realize it has a similar shape to \Cref{lemma:extension_banerjee} for the particular case when the loss is sub-Gaussian, where the role of the sub-Gaussian parameter is taken by the sum of the ``variance'' terms $[M]_n(W) + 2 \langle M \rangle_n(W)$. Therefore, it appears we may directly employ the technique to derive the Chernoff analogue from the proof of \Cref{th:pac_bayes_chernoff_analogue}. However, one needs to take into account the fact that the ``optimal'' parameter $\lambda$ now depends on this ``variance'' terms, which are also dependent on the training set $S_n$ and on the number of samples $n$.

To optimize the bound from \Cref{th:wang_main_th} we will then proceed in two steps. The first step is to optimize the parameter $\lambda$ for a \emph{fixed} number of samples $n$ in a similar fashion to \Cref{th:pac_bayes_chernoff_analogue}, which results in \Cref{th:paramter_free_anytime_valid_martingales}. Then, the second step is to extend this result to an anytime-valid bound using \Cref{th:standard_to_anytime_valid} at a cost in the complexity term of $\cO(\nicefrac{\log n}{n})$.

For the first step,  define the event $\cB_{n,\lambda}$ as the complement of the event in~\eqref{eq:wang_main_th} for a \emph{fixed} number of samples $n$. Then, we can proceed similarly to the proof of \Cref{th:pac_bayes_chernoff_analogue} noticing that, for each number of samples $n$, the complement of the event 
\begin{equation*}
    \cE_n \coloneqq \Big \{ \bE^{S_n} \big[ [M]_n(W) + 2 \langle M \rangle_n(W) \big] \relent(\bP_W^S \Vert \bQ) \leq n^2 \Big\}
\end{equation*}
is uninteresting as the bound is non-vanishing given $\cE_n^c$. This produces the following PAC-Bayes bound for a fixed number of samples $n$.%

\begin{theorem}[{Parameter-free bound on martingales}]
    \label{th:paramter_free_anytime_valid_martingales}
    Let $\bQ$ be any prior independent of $S_n$ and $(M_n(w))_{n \geq 1}$ be any sequence of martingales indexed by $w \in \cW$. Further, define $\xi'(n) \coloneqq 2en(n+1)^2 \log(en) \leq 2e(n+1)^3$. Then, for every $\beta \in (0,1)$, with probability no smaller than $1-\beta$
    \begin{align*}
        \Big| \bE^{S_n} \big[&M_n(W)\big] \Big| \leq \\
        &\bI_{\cE_n} \cdot \frac{2}{\sqrt{6}} \cdot \sqrt{ \bE^{S_n} \big[ [M]_n(W) + \langle M \rangle_n(W) + 1 \big]\Big( \relent(\bP_W^{S_n} \Vert \bQ) + \log \frac{\xi'(n)}{\beta} \Big)} \\
        &+ \bI_{\cE_n^c} \cdot \esssup \Big| \bE^{S_n} \big[M_n(W)\big] \Big|
    \end{align*}
    holds \emph{simultaneously} for all posteriors $\bP_W^S$, where $\cE_n$ is the event $$\cE_n \coloneqq \Big\{ \bE^{S_n} \big[ [M]_n(W) + 2 \langle M \rangle_n(W) \big] \relent(\bP_W^S \Vert \bW_W) \leq n^2 \Big\}.$$
\end{theorem}

\begin{proof}
    Consider a fixed number of samples $n$. Let $\cB_{n,\lambda}$ be the complement of the event in~\eqref{eq:wang_main_th} such that $\bP[\cB_{n,\lambda}] < \beta$ and consider the sub-events 
    \begin{align*}
        \cE_{n,1,l} &\coloneqq \left \{ \relent(\bP_W^{S_n} \Vert \bQ) \leq 1 \textnormal{ and } \left \lceil \bE^{S_n} \big[ [M]_n(W) + 2 \langle M \rangle_n(W) \big] \right \rceil = l \right \}, \\
        \cE_{n,k, 1} &\coloneqq \left \{ \left \lceil \relent(\bP_W^{S_n} \Vert \bQ) \right \rceil = k \textnormal{ and }  \bE^{S_n} \big[ [M]_n(W) + 2 \langle M \rangle_n(W) \big] \leq 1 \right \}, \textnormal{ and} \\
        \cE_{n,k,l} &\coloneqq \left \{ \left \lceil \relent(\bP_W^{S_n} \Vert \bQ) \right \rceil = k  \textnormal{ and } \left \lceil \bE^{S_n} \big[ [M]_n(W) + 2 \langle M \rangle_n(W) \big] \right \rceil = l \right \},
    \end{align*}
    for all $k, l = 2, \cdots, n^2$ such that $k l \leq n^2$, which form a covering of $\cE_n$. Furthermore, define $\cK \coloneqq \{ (k,l) : 1 \leq k l \leq n^2 \textnormal{ and } \bP[\cE_{n,k,l}] > 0 \}$. For all pairs of indices $(k,l) \in \cK$, given the event $\cE_{n,k,l}$, with probability no more than $\bP[\cB_{n,\lambda} | \cE_{n,k,l}]$, there exists a posterior $\bP_W^{S_n}$ such that
    \begin{equation}
        \label{eq:pac_bayes_martingale_with_lambda}
        \Big| \bE^{S_n} \big[M_n(W)\big] \Big| > \frac{k + \log \frac{2(n+1)^2}{\beta}}{\lambda} + \frac{\lambda}{6} \cdot l.
    \end{equation}
    for all $\lambda \in (0, b)$. The parameter that optimizes the right-hand side of~\eqref{eq:pac_bayes_martingale_with_lambda} is
    \begin{equation*}
        \lambda = \lambda_{k,l} = \sqrt{\frac{6}{l} \Big(k + \log \frac{2(n+1)^2}{\beta} \Big)}.
    \end{equation*}
    Substituting the optimal $\lambda_{k,l}$ and using that $k \leq \relent(\bP_W^{S_n} \Vert \bQ) + 1$ and $l \leq \bE^{S_n} \big[ [M]_n(W) + 2 \langle M \rangle_n(W) + 1\big]$ yields that, given the event $\cE_{n,k,l}$, with probability smaller or equal than $\bP[\cB_{n,\lambda_{k,l}} | \cE_{n,k,l}]$, there exists a posterior $\bP_W^{S_n}$ such that
    \begin{align*}
        \Big| \bE^{S_n} \big[&M_n(W)\big] \Big| > \\
        &\frac{2}{\sqrt{6}} \cdot \sqrt{ \bE^{S_n} \big[ [M]_n(W) + \langle M \rangle_n(W) + 1 \big]\Big( \relent(\bP_W^{S_n} \Vert \bQ) + \log \frac{2e(n+1)^2}{\beta} \Big)}.
    \end{align*}

    Now, define $\cB'_n$ as the event stating that there exists a posterior $\bP_W^{S_n}$ such that
    \begin{align*}
       \Big| \bE^{S_n} \big[&M_n(W)\big] \Big| > \\
       &\bI_\cE \cdot \frac{2}{\sqrt{6}} \cdot \sqrt{ \bE^{S_n} \big[ [M]_n(W) + 2 \langle M \rangle_n(W) + 1 \big]\Big( \relent(\bP_W^{S_n} \Vert \bQ) + \log \frac{2e}{\beta} \Big)} \\
        &+ \bI_{\cE} \cdot \esssup | \bE^{S_n} M_n(W) |,
    \end{align*}
    where $\bP[\cB'_n | \cE_{n,k,l}] \bP[\cE_{n,k,l}] \leq \bP[\cB_{n,\lambda_{k,l}} | \cE_{n,k,l}] \bP[\cE_{n,k,l}] \leq \bP[\cB_{n,\lambda_{k,l}}] < \beta$ for all $(k,l) \in \cK$, and where $\bP[\cB'_n \cap \cE_n^c] = 0$ by the definition of the essential supremum. Therefore, by the law of total probability, the probability of $\cB'_n$ is bounded as
    \begin{equation*}
        \bP[\cB'_n] = \sum_{(k,l) \in \cK} \bP[\cB'_n | \cE_{n,k,l}] \bP[\cE_{n,k,l}] + \bP[\cB'_n \cap \cE_n^c] < n(1 + \log n) \beta = n \log (en) \beta.
    \end{equation*}
    Finally, let $\beta_n = n \log(en) \beta$ so that, with probability no larger than $\beta_n$, there exists a posterior $\bP_W^{S_n}$ such that
    \begin{align*}
        \Big| &\bE^{S_n} \big[M_n(W)\big] \Big| > \\
        &\bI_\cE \cdot \frac{2}{\sqrt{6}} \cdot \sqrt{ \bE^{S_n} \big[ [M]_n(W) + 2 \langle M \rangle_n(W) + 1 \big]\Big( \relent(\bP_W^{S_n} \Vert \bQ) + \log \frac{\xi'(n)}{\beta_n} \Big)}  \\
        &+ \bI_{\cE} \cdot \esssup | \bE^{S_n} M_n(W) |.
    \end{align*}
    Finally, the substitution $\beta_n \leftarrow \beta$ completes the proof.
\end{proof}

\looseness=-1 This technique can be extended to the corollary bound of \citet{haddouche2023pacbayes} for batch learning with i.i.d. data yielding~\Cref{th:parameter_free_anytime_valid_bounded_2nd_moment}, where we write $S_n = S$ to simplify the reading in the main text. Note again that we are using the particularization of \citet{haddouche2023pacbayes} with the constants from \citet{wang2015pac}.

Finally, for the second step, \Cref{th:paramter_free_anytime_valid_martingales,th:parameter_free_anytime_valid_bounded_2nd_moment} can be converted back to anytime-valid bounds using \Cref{th:standard_to_anytime_valid}. The resulting bound is exactly the same substituting $\log \xi'(n)$ for $\log \xi''(n)$, where $\xi''(n) \coloneqq \nicefrac{e \pi^2 (n+1)^2 n^3 \log (en)}{3}$.

In case that one desires to have a bound without a $\log n$ term, one may consider employing the technique outlined for~\Cref{th:pac_bayes_chernoff_analogue_no_cutoff} with~\citet{haddouche2023pacbayes}'s~\Cref{th:haddouche_main_th} instead of~\citet{wang2015pac}'s~\Cref{th:wang_main_th}.

\end{subappendices}

%% file: chapters/privacy_and_generalization.tex
In this chapter, we discuss further the relationships between privacy and generalization that we already pointed out in~\Cref{subsec:privacy_as_stability} and throughout the manuscript. To be precise, in~\Cref{ch:expected_generalization_error,ch:pac_bayesian_generalization}, we saw that the generalization error is bounded from above by a function that depends on the dependence of the algorithm's output hypothesis and the training data. Moreover, a learning algorithm is private if its output does not reveal much about the instances of the training set. Therefore, if an algorithm is private, the dependence of the algorithm's output and the training data is small and it generalizes.

In~\Cref{sec:review_privacy_frameworks}, we start reviewing two different frameworks for privacy: \emph{maximal leakage}~\citep{issa2020operational} and \emph{differential privacy}~\citep{dwork2006calibrating,dwork2014algorithmic}. We choose to review these frameworks since they have attractive operational meanings. Maximal leakage refers to the highest potential increase in an attacker's probability to guess a sensitive variable from your training set after observing the output hypothesis, as opposed to before observation. Differential privacy, on the other hand, captures how close is the probability of observing the output hypothesis if instead of using the real training set, we used any other set differing in one instance. 

Then, in~\Cref{sec:maximal_leakage_privacy}, we show how algorithms with a bounded maximal leakage generalize. To do so, we relate how a bounded maximal leakage implies that the quantities employed in~\Cref{ch:expected_generalization_error,ch:pac_bayesian_generalization} are also bounded. The results from this section are not published and we consider them folklore within the community. We see that these results coincide with other results in the community~\citep{esposito2021generalization,hellstrom2020generalization}. 

\looseness=-1 After that, in~\Cref{sec:differential_privacy_generalization}, we study how differentially private algorithms generalize. First, as before, we consider a simple relationship between the parameters of differentially private algorithms and the quantities employed in~\Cref{ch:expected_generalization_error,ch:pac_bayesian_generalization}. With this simple analysis, we obtain bounds that depend on the privacy parameter, but that do not decrease with the number of samples. That is, the bounds only guarantee that the algorithm generalizes if the privacy parameter decreases with the number of samples on the training set. These simple bounds are of the same order and similar to the bounds from the literature that we survey in~\Cref{subsec:other_results_dp_gen}. In \Cref{subsec:pure_dp_gdp_discrete}, we study the particular case of differentially private algorithms that are \emph{permutation invariant} and operate on \emph{discrete data}. In this case, we first show that any permutation invariant algorithm generalizes, and then we show that if the algorithm is also differentially private, then the generalization guarantees are tighter. Despite their advantage in the asymptotic regime, these bounds are only preferred to the standard ones when the ratio between the data cardinality and the number of samples is small, that is, when $|\cZ| \ll n$. This part of the section is mostly based on our paper~\citep{rodriguez2021upper}, although it is updated based on the (chronologically posterior) developments from~\Cref{ch:expected_generalization_error,ch:pac_bayesian_generalization}.

\looseness=-1 We wrote this chapter, like the previous ones, with a didactic intent. We wanted to strengthen the intuition we obtained throughout the manuscript that private algorithms generalize as well as survey some of the current understanding on the topic.

\section{A Short Review on Privacy Frameworks}
\label{sec:review_privacy_frameworks}

Recall from~\Cref{subsec:renyi_divergence} that a \emph{privacy mechanism} $\bM$ is a channel, described by the Markov kernel $\bP_Y^X$, that processes an input $X$ and generates a sanitized version $Y$ that is not informative about the input $X$. In the context of machine learning, we say that an algorithm $\bA$, described by the Markov kernel $\bP_W^S$, is a \emph{private algorithm} if the hypothesis is not informative about the training instances. That is, an algorithm $\bA$ is a private algorithm if it is a privacy mechanism.

Constructing a private algorithm is simple. The algorithm $\bA(S) = 0$ a.s. is maximally private, as the returned hypothesis $0$ contains no information about the training set. In general, any algorithm independent of the training data such that $\bP_W^S = \bQ$ has this property. However, these algorithms are not useful. The goal of a private algorithm is to balance the so-called \emph{privacy-utility trade-off}, where the returned hypothesis should be informative about the underlying structure of the training data related to the downstream task (high utility), while not being informative about the specific instances in the training set (high privacy).

While the definition of utility is task-specific, the definition of privacy can be agnostic to the downstream task. Therefore, there have been several frameworks that attempt to quantify the amount of privacy guaranteed by an algorithm~\citep{issa2020operational,saeidian2023apointwise,saeidian2023bpointwise,dwork2006calibrating,dwork2014algorithmic,dwork2016concentrated,mironov2017renyi}. In this manuscript, we will only discuss \emph{maximal leakage}~\citep{issa2020operational,saeidian2023apointwise,saeidian2023bpointwise} and \emph{differential privacy}~\citep{dwork2006calibrating,dwork2014algorithmic} due to their operational meaning.

\subsection{Maximal Leakage}
\label{subsec:maximal_leakage}

Consider that there is a sensitive, discrete variable $U$ that depends on the training data. If an algorithm is private, one would expect that the probability of correctly guessing the sensitive variable $U$ when having access to the returned hypothesis $W$ is close to the probability of correctly guessing it without access. Moreover, this should be true for every possible discrete variable $U$ that depends on the training data, as we may consider every such description as sensitive.

\citet{issa2020operational} formalized this idea and defined the \emph{maximal leakage} as the supremum of the logarithm of the ratio of these two probabilities with respect to every such variable. Namely,
\begin{equation*}
    \cL(S \to W) \coloneqq \sup_{\bP_{U}^{S} }  \log \frac{\sup_{\bP_{\hat{U}}^W} \bP \big[ U = \hat{U} \big]}{\max_{u \in \cU} \bP \big[U = u \big]},
\end{equation*}
where the supremum over the Markov kernels $\bP_{U}^{S}$ characterizes the worst-case sensitive variable $U$, and the supremum over the Markov kernels $\bP_{\hat{U}}^W$ characterizes the best attacker trying to guess the sensitive variable.

Instead of considering a sensitive, discrete variable $U$, one could think of an arbitrary gain function $g : \cX \times \cU \to \bR_+$ representing the adversary's objective. For example, as described by~\citet[Examples 1-4]{saeidian2023apointwise}, the gain function can describe multiple common situations such as:
\begin{itemize}
    \item \textit{Membership inference}. Imagine that the adversary wants to infer if an individual $i$ belongs to the training set $s$. Then, we may partition the training set according to the participation of the said individual. That is,  $s_{1} = \{ z \in s : z \textnormal{ contains the data of } i \}$ and $s_{0} = s \setminus s_{1}$. Then, the adversary does a binary guess $\cU = \{ 0, 1 \}$ and their gain function can be $g(s,u) = \bI_{\{ s = s_u \}}(s,u)$.
    \item \textit{Guessing a function of the training data}. Like in the maximal leakage described above, we may think of a sensitive variable $U$ that is related by a (possible randomized) function $f$ to the training data, that is $U = f(S)$. In this case, consider that $U$ lies in some metric space $(\cU,\rho)$. Then, the gain function can simply be $g(s,u) = \rho \big( f(s), u \big)$.
\end{itemize}
As for maximal leakage, we may consider that an algorithm is private if the expected gain when having access to the returned hypothesis is close to the expected gain without access. Also, this should be true for every possible gain function, as we want to protect the data from every possible attack.

\citet{m2012measuring} formalized this idea and defined the \emph{maximal $g$-leakage} as the supremum of the logarithm of the ratio of these two expected gains with respect to every such gain function. Namely,
\begin{equation*}
    \sup_g \cL_g(S \to W) \coloneqq \sup_g \log \frac{\bE \big[ \max_{u \in \cU} \bE^{W} \big[ g(S,u) \big]\big]}{\max_{u \in \cU} \bE[g(S,u)]}.
\end{equation*}

Interestingly, these two privacy measures are equivalent~\citep{issa2020operational,alvim2014additive} and are equal to the essential supremum of the Rényi divergence of order $\infty$ of the algorithm's channel with respect to the marginal distribution of the hypothesis, that is
\begin{equation*}
    \cL(S \to W) = \esssup_{\bP_S} \renyidiv{\infty} (\bP_W^S \Vert \bP_W).
\end{equation*}
If the data is discrete, the maximal leakage also coincides with other operational privacy notions as the \emph{min-capacity} of the \emph{min-entropy leakage}~\citep{braun2009quantitative, smith2009foundations} and the \emph{multiplicative Bayesian capacity} of the \emph{multiplicative Bayesian leakage}~\citep{alvim2020science}.

\subsection{Differential Privacy}
\label{subsec:dp}

Another operationalization of privacy is given by~\emph{differential privacy} (DP)~\citep{dwork2006calibrating,dwork2014algorithmic}. An algorithm is $(\varepsilon, \delta)$-DP if for every pair of neighbouring training sets $s$ and $s'$ and every event $\cA \subseteq \cW$
\begin{equation*}
    \bP_W^{S=s}[\cA] \leq e^\varepsilon \bP_W^{S=s'}[\cA] + \delta,
\end{equation*}
where we recall that two training sets are neighbours if they differ in at most one element. If $\delta=0$, then the algorithm is just $\varepsilon$-DP.
The latter setting, when $\delta = 0$, is often referred to as \emph{pure DP}, while the former, when $\delta > 0$, is referred to as \emph{approximate DP}. The reason is that, although $\delta$ cannot be exactly mapped to the probability of failure~\citep{meiser2018approximate}, it is often understood as such. That is, as the probability that the $\varepsilon$-DP guarantee does not hold.

As discussed in~\Cref{subsec:privacy_as_stability}, this notion of privacy states that an algorithm is private if it is stable. That is, an algorithm is DP if, when presented with two neighbouring training sets, the algorithm produces two hypotheses with ``similar'' distributions. Therefore, if we consider that each instance of the training set belongs to one individual, DP guarantees that the impact of the individual on the output hypothesis is small. In fact, consider the real training set $s$ and any other neighbour training set $s'$. Then, the Type I and Type II errors of every hypothesis test to distinguish if the outcome of the algorithm comes from $s$ or $s'$ are bounded from below by quantities depending on the privacy parameters $\varepsilon$ and $\delta$~\citep{balle2020hypothesis,wasserman2010statistical,kairouz2015composition}.

The hypothesis testing interpretation of DP led~\citet{dong2021gaussian} to define \emph{$f$-differential privacy} ($f$-DP). In this framework, an algorithm is said to be $f$-DP if the trade-off curve described by the Type-I and Type-II errors of the aforementioned hypothesis test lies above the non-increasing, convex function $f$. The function $f$ is thus the parameter that determines the privacy level of the algorithm.
A particular, and important, parameterization of these algorithms is given by \emph{$\mu$-Gaussian differentially private} ($\mu$-GDP) algorithms, where $f$ is defined as the trade-off function between two unit-variance Gaussian distributions with mean $0$ and $\mu$, respectively. That is, an algorithm is $\mu$-GDP if distinguishing between any two neighbouring input training sets is, at least, as hard as distinguishing between two Gaussian random variables with unit variance and means at a distance $\mu$.

Due to their operational, intuitive formulation, we will mainly focus on pure DP and GDP. Two important properties of $\varepsilon$-DP and $\mu$-GDP algorithms are:
\begin{enumerate}
\item \emph{They are KL-stable}\footnote{The definition of KL-stability in~\cite{bassily2016algorithmic} differs from~\eqref{eq:dp_distance_neighbour} and~\eqref{eq:gdp_distance_neighbour} in the explicit quantity on the right-hand side of the bounds. However, the idea that the relative entropy is bounded remains.}~\cite[Definition 4.2]{bassily2016algorithmic}.
That is, given any two (fixed) neighbouring training sets $s$ and $s'$, the relative entropy of their respective output distributions is bounded by~\cite[Lemma D.8]{dong2021gaussian}
 \begin{equation}
    \relent(\bP_W^{S=s} \Vert \bP_W^{S=s'}) \leq  \varepsilon\tanh \left( \frac{\varepsilon}{2} \right)
  \label{eq:dp_distance_neighbour}
 \end{equation}
if $\bA$ is $\varepsilon$-DP, and by~\cite[Theorem.~2.10]{dong2021gaussian}
 \begin{equation}
  \relent(\bP_W^{S=s} \Vert \bP_W^{S=s'}) \leq \frac{1}{2}\mu^2
  \label{eq:gdp_distance_neighbour}
 \end{equation}
if $\bA$ is $\mu$-GDP.

\item \emph{They possess group privacy}.
Given two training sets that differ in at most $k$ instances, an $\varepsilon$-DP algorithm $\bA$ is $k\varepsilon$-DP~\cite[Theorem.~2.2]{dwork2014algorithmic} and a $\mu$-GDP algorithm $\bA$ is $k\mu$-GDP~\cite[Theorem.~2.14]{dong2021gaussian}.
\end{enumerate}

A direct consequence of these two properties is condensed below.

\begin{proposition}[{\citep{dong2021gaussian,dwork2014algorithmic}}]
\label{claim:diff_priv_distance}
Given an algorithm $\mathbb{A}$, if two (fixed) training sets $s$ and $s'$ differ in at most $k$ instances, then the relative entropy of their respective output distributions is bounded from above as
\begin{equation}
	\relent(\bP_W^{S=s} \Vert \bP_W^{S=s'}) \leq k\varepsilon\tanh \left( \frac{k\varepsilon}{2} \right) \leq\min \left\{ \frac{1}{2}k^2\varepsilon^2, k\varepsilon \right\}
	\label{eq:dp_distance}
\end{equation}
if $\mathbb{A}$ is $\varepsilon$-DP, and as
\begin{equation}
	\relent(\bP_W^{S=s} \Vert \bP_W^{S=s'}) \leq \frac{1}{2} k^2 \mu^2
	\label{eq:gdp_distance}
\end{equation}
if $\mathbb{A}$ is $\mu$-GDP.
\end{proposition}

The last inequality of~\eqref{eq:dp_distance} is obtained using the first term of the Taylor expansion of $\tanh$ and the fact that $\tanh \leq 1$. We note that, if $\varepsilon \geq \nicefrac{2}{k}$, the linear approximation in~\eqref{eq:dp_distance} is tighter than the quadratic one; moreover, this approximation becomes increasingly accurate as $k$ increases, given a fixed $\varepsilon$. Incidentally, the linear upper bound may also be obtained through the bound on the max-information from~\citet[Remark~3.1]{dwork2014algorithmic}.

\section{Maximal Leakage and Generalization}
\label{sec:maximal_leakage_privacy}

Recall the definition of the Rényi divergence from~\Cref{def:renyi_divergence}. Then, since the logarithm is a monotonic non-decreasing function we may write the maximal leakage as
\begin{equation*}
    \cL(S \to W) = \esssup_{\bP_S} \esssup_{\bP_W^{S}} \log \frac{\rmd \bP_W^S}{\rmd \bP_W}.
\end{equation*}

Once we write the essential supremum in this form, we may note that if the maximum leakage is bounded by $\varepsilon$, then 
\begin{equation*}
    \log \frac{\rmd \bP_W^S}{\rmd \bP_W} \leq \varepsilon \textnormal { a.s.,} \hspace{2em} \relent(\bP_W^S \Vert \bP_W) \leq \varepsilon \textnormal{ a.s.,} \hspace{1em} \textnormal{and} \hspace{1em} \minf(W;S) \leq \varepsilon.
\end{equation*}

Therefore, if an algorithm is private in the sense that it has a maximal leakage bounded by $\varepsilon$, then it generalizes. Indeed, every expected generalization error and PAC-Bayesian bounds from~\Cref{ch:expected_generalization_error,ch:pac_bayesian_generalization} that involves these quantities can be bounded from above by changing these quantities by $\varepsilon$. 

This is one of the reasons why information-theoretic generalization bounds are attractive. We can determine that \emph{if an algorithm is private, then it generalizes} only by noting that the information captured by the algorithm is bounded if the algorithm is private. As an example, let us consider the most specific type of generalization bounds: single-draw PAC-Bayesian bounds~\Cref{subsec:levels_of_specificity,sec:single-draw-pac-bayes}. We will show how the bounds for losses with a bounded range, losses with a bounded CGF, and losses with a bounded moment can be employed to show that private algorithms generalize. In fact, for bounded losses and losses with a bounded CGF, the bounds will be almost identical to the classical ``small-kl'' and Chernoff inequalities from~\Cref{prop:small_kl,prop:chernoff}. The only difference will be that, where before there was only a cost $\frac{1}{n} \cdot \log \frac{1}{\beta}$ for the precision of the guarantee, now there will be a cost of $\frac{1}{n} \big( \varepsilon +  \log \frac{1}{\beta} \big)$. Therefore, we can clearly determine the effect of the information leakage in the generalization guarantee by the amount of maximal leakage of the algorithm.

\subsection{Losses With a Bounded Range}
\label{subsec:maximal_leakage_losses_bounded_range}

To start, we may consider the Seeger--Langford bound from~\Cref{th:seeger_langford_pac_bayes}. Then, if we choose the data independent prior $\bQ$ to be $\bP_W$ we obtain the following result.

\begin{theorem}
    \label{th:maximal_leakage_generalizes_bounded_bounded_small_kl}
    Consider a loss with a range bounded in $[0,1]$. If an algorithm $\bA$ has a maximal leakage bounded by $\varepsilon$, then, for every $\beta \in (0,1)$, with probability no smaller than $1 - \beta$
    \begin{equation*}
        \relentber \big( \emprisk(W,S) \Vert \poprisk(W) \big) \leq \frac{\varepsilon + \log \frac{\xi(n)}{\beta}}{n}.
    \end{equation*}
\end{theorem}

Then, we may employ our developments from~\Cref{subsec:a_fast_rate_bound_mi} to transform the ``small-kl'' inequality into a fast-rate bound, leading to the following result. For convenience, as the result will be employed in this form in~\Cref{subsec:maximal_leakage_losses_bounded_moment}, we present the result for losses with a range bounded in $[0,b]$, while noting that it can be generally extended for losses with a range bounded in $[a,b]$ by scaling and centering.

\begin{theorem}
    \label{th:maximal_leakage_generalizes_bounded_bounded_fast_rate}
    Consider a loss with a range bounded in $[0,b]$. If an algorithm $\bA$ has a maximal leakage bounded by $\varepsilon$, then, for every $\beta \in (0,1)$, with probability no smaller than $1 - \beta$
    \begin{equation*}
        \poprisk(W) \leq c \gamma \log \mleft( \frac{\gamma}{\gamma - 1} \mright) \cdot \emprisk(W,S) + b c \gamma \cdot \frac{\varepsilon + \log \frac{\xi(n)}{\beta}}{n} + b \kappa(c) \gamma.
    \end{equation*}
    holds \emph{simultaneously} for all $\gamma > 1$ and all $c \in (0,1]$, where $\kappa(c) \coloneqq 1 - c (1 - \log c)$.
\end{theorem}

\subsection{Losses With a Bounded CGF}
\label{subsec:maximal_leakage_losses_bounded_cgf}

Consider the bound from~\eqref{eq:single-draw-chernoff-intermediate} for losses with a bounded CGF in the sense of~\Cref{def:bounded_cgf}. Let us choose the data independent prior $\bQ$ to be $\bP_W$. Then, for all $\lambda \in (0,b)$, with probability no smaller than $1 - \beta$
\begin{equation*}
    \gen(W,S) \leq \frac{1}{\lambda} \mleft( \frac{\varepsilon + \log \frac{1}{\beta}}{n} + \psi(\lambda) \mright).
\end{equation*}

Note that the right-hand side of this equation does not contain any random element. Therefore, we can safely optimize the parameter $\lambda$ using~\Cref{lemma:boucheron_convex_conjugate_inverse} without the need for the events' space discretization technique from~\Cref{subsec:losses_with_bounded_CGF}. The optimization results in the following generic bound.

\begin{theorem}
    \label{th:maximal_leakage_generalizes_bounded_cgf}
    Consider a loss with a bounded CGF (\Cref{def:bounded_cgf}). If an algorithm $\bA$ has a maximal leakage bounded by $\varepsilon$, then, for every $\beta \in (0,1)$, with probability no smaller than $1 - \beta$
    \begin{equation*}
        \gen(W,S) \leq \psi_*^{-1} \mleft( \frac{\varepsilon + \log \frac{1}{\beta}}{n} \mright).
    \end{equation*}
\end{theorem}

\Cref{th:maximal_leakage_generalizes_bounded_cgf} is essentially the same bound from~\citet[Corollary 3]{esposito2021generalization}, but written in the style used throughout the manuscript. This result is a generalization of the specific theorem for sub-Gaussian losses from~\citep[Corollary 4]{esposito2021generalization}, which was later recovered as well in~\citep[Corollary 4]{hellstrom2020generalization}. For example, if a loss has a range bounded in $[a,b]$, with probability no smaller than $1 - \beta$
\begin{equation*}
    \gen(W,S) \leq (b-a) \sqrt{\frac{\varepsilon + \log \frac{1}{\beta}}{n}}.
\end{equation*}

\subsection{Losses With a Bounded Moment}
\label{subsec:maximal_leakage_losses_bounded_moment}

Consider~\Cref{th:maximal_leakage_generalizes_bounded_bounded_fast_rate} and the the single-draw version of~\Cref{lemma:alquier_truncation_method_refined_bouned_moment} that we can obtain using our refinement of~\citet{alquier2006transductive}'s truncation technique from~\Cref{subsec:interpolating_between_slow_and_fast_rate,subsec:losses_with_bounded_moment}. Then, we have that if the loss $\ell(w,Z)$ has a $p$-th moment bounded by $m_p$ for all $w \in \cW$ and the algorithm has a maximal leakage bounded by $\varepsilon$, then for every $\beta \in (0,1)$ and all $\lambda > 0$, with probability no smaller than $1 - \beta$
\begin{equation*}
    \poprisk(W) \leq \kappa_1 \cdot \emprisk_{\leq \nicefrac{n}{\lambda}}(W,S) + \kappa_2 \cdot \frac{\varepsilon + \log \frac{\xi(n)}{\beta}}{n} + \kappa_3 \cdot \frac{n}{\lambda} + \frac{m_p}{p-1} \mleft( \frac{\lambda}{n} \mright)^{p-1},
\end{equation*}
holds \emph{simultaneously} for all $c \in (0,1]$ and all $\gamma > 1$, where $\kappa_1 \coloneqq c \gamma \log \big( \nicefrac{\gamma}{(\gamma -1)} \big)$, $\kappa_2 \coloneqq c \gamma$, and $\kappa_3 \coloneqq \gamma \big( 1 - c(1 - \log c)\big)$.

Now, as in~\Cref{subsec:maximal_leakage_losses_bounded_cgf}, the optimal value of the parameter $\lambda > 0$ does not depend on any random element. Therefore, it can be safely optimized without the need for the events' space discretization technique from~\Cref{subsec:losses_with_bounded_CGF}. The optimization leads to the following result.

\begin{theorem}
    \label{th:maximal_leakage_generalizes_bounded_moment}
    Consider a loss $\ell(w,Z)$ with a $p$-th moment bounded by $m_p$ for all $w \in \cW$. If an algorithm $\bA$ has a maximal leakage bounded by $\varepsilon$, then, for every $\beta \in (0,1)$, with probability no smaller than $1 - \beta$
    \begin{equation*}
        \poprisk(W) \leq \kappa_1 \cdot \emprisk_{\leq t^\star}(W,S) + m_p^{\frac{1}{p}} \mleft(\frac{p}{p-1}\mright) \mleft( \kappa_2 \cdot \frac{\varepsilon + \log \frac{\xi(n)}{\beta}}{n} + \kappa_3 \mright)^{\frac{p-1}{p}}
    \end{equation*}
    holds \emph{simultaneously} for all $c \in (0,1]$ and all $\gamma > 1$, where $\kappa_1 \coloneqq c \gamma \log \big( \nicefrac{\gamma}{(\gamma -1)} \big)$, $\kappa_2 \coloneqq c \gamma$, $\kappa_3 \coloneqq \gamma \big( 1 - c(1 - \log c)\big)$, and
    \begin{equation*}
        t^\star \coloneqq m_p^\frac{1}{p} \mleft( \kappa_2 \cdot \frac{\varepsilon + \log \frac{\xi(n)}{\beta}}{n} \mright)^{\frac{p-1}{p}}.
    \end{equation*}
\end{theorem}

Finally, one could obtain the analogue results for~\Cref{th:bounded_variance_high_probability} for losses with a bounded variance.

\section{Differential Privacy and Generalization}
\label{sec:differential_privacy_generalization}

Like we did in~\Cref{sec:maximal_leakage_privacy} above, recall the definition of (pure) differential privacy and the group privacy property from~\Cref{subsec:dp}. Then, if an algorithm is $\varepsilon$-DP, for every two training sets $s$ and $s'$ that differ in every element
\begin{equation*}
    \bP_W^{S=s} \leq e^{n \varepsilon} \bP_W^{S=s'} \textnormal{ almost surely.}
\end{equation*}
Therefore, it follows that for any $\varepsilon$-DP algorithm
\begin{equation*}
    \log \frac{\rmd \bP_W^S}{\rmd \bQ} \leq n \varepsilon \textnormal { a.s.,} \hspace{2em} \relent(\bP_W^S \Vert \bQ) \leq n \varepsilon \textnormal{ a.s.,} \hspace{1em} \textnormal{and} \hspace{1em} \minf(W;S) \leq n\varepsilon,
\end{equation*}
where $\bQ = \bP_W^{S=s'}$ is chosen to be the distribution induced by the algorithm on a fictitious, fixed training set, and where the last inequality follows form the \emph{golden formula} from~\Cref{prop:properties_minf}. Note that the distribution $\bQ$ still does not depend on the data, as $s'$ can be a fictitious sequence of $n$ elements from the instance space $\cZ$ that is completely independent of the real training set $s$.

In this way, we may repeat the procedures from~\Cref{sec:maximal_leakage_privacy} substituting these information measures by $n \varepsilon$ and observing the new privacy guarantees. For example, %
the analogue of the single-draw PAC Bayesian guarantee for algorithms with a bounded maximal leakage from~\Cref{th:maximal_leakage_generalizes_bounded_cgf,th:maximal_leakage_generalizes_bounded_bounded_small_kl,th:maximal_leakage_generalizes_bounded_moment} to $\varepsilon$-DP algorithms is the following.

\begin{theorem}
    \label{th:dp_generalizes_bounded_cgf}
    Consider an $\varepsilon$-DP algorithm $\bA$.
    \begin{enumerate}
        \item If the loss has a range bounded in $[0,1]$, then for every $\beta \in (0,1)$, with probability no smaller than $1 - \beta$
        \begin{equation*}
            \relentber \big( \emprisk(W,S) \Vert \poprisk(W) \big) \leq \varepsilon  + \frac{\log \frac{\xi(n)}{\beta}}{n}.
        \end{equation*}
        \item If the loss has a range bounded in $[0,b]$, then for every $\beta \in (0,1)$, with probability no smaller than $1 - \beta$
        \begin{equation*}
            \poprisk(W) \leq c \gamma \log \mleft( \frac{\gamma}{\gamma - 1} \mright) \cdot \emprisk(W,S) + b c \gamma \cdot \mleft( \varepsilon + \frac{ \log \frac{\xi(n)}{\beta}}{n} \mright) + b \kappa(c) \gamma.
        \end{equation*}
        holds \emph{simultaneously} for all $\gamma > 1$ and all $c \in (0,1]$, where $\kappa(c) \coloneqq 1 - c (1 - \log c)$.
        \item If the loss has a bounded CGF (\Cref{def:bounded_cgf}), then for every $\beta \in (0,1)$, with probability no smaller than $1 - \beta$
        \begin{equation*}
            \gen(W,S) \leq \psi_*^{-1} \mleft( \varepsilon + \frac{\log \frac{1}{\beta}}{n} \mright).
        \end{equation*}
        \item If the loss $\ell(w,Z)$ has a $p$-th moment bounded by $m_p$ for all $w \in \cW$, then for every $\beta \in (0,1)$, with probability no smaller than $1 - \beta$
        \begin{equation*}
            \poprisk(W) \leq \kappa_1 \cdot \emprisk_{\leq t^\star}(W,S) + m_p^{\frac{1}{p}} \mleft(\frac{p}{p-1}\mright) \mleft( \kappa_2 \cdot \mleft( \varepsilon + \frac{ \log \frac{\xi(n)}{\beta}}{n}\mright) + \kappa_3 \mright)^{\frac{p-1}{p}}
        \end{equation*}
        holds \emph{simultaneously} for all $c \in (0,1]$ and all $\gamma > 1$, where $\kappa_1 \coloneqq c \gamma \log \big( \nicefrac{\gamma}{(\gamma -1)} \big)$, $\kappa_2 \coloneqq c \gamma$, $\kappa_3 \coloneqq \gamma \big( 1 - c(1 - \log c)\big)$, and
        \begin{equation*}
            t^\star \coloneqq m_p^\frac{1}{p} \mleft( \kappa_2 \cdot \mleft( \varepsilon + \frac{\log \frac{\xi(n)}{\beta}}{n} \mright)\mright)^{\frac{p-1}{p}}.
        \end{equation*}
        
    \end{enumerate}    
\end{theorem}

Comparing~\Cref{th:maximal_leakage_generalizes_bounded_cgf} and~\Cref{th:dp_generalizes_bounded_cgf}, we observe that while a bounded maximal leakage readily implies generalization, that an algorithm is $\varepsilon$-DP only implies that it generalizes if $\varepsilon$ is a decreasing function of $n$. In fact, this naive result is in line with the literature in generalization guarantees for DP algorithms that we survey in~\Cref{subsec:other_results_dp_gen}.

For this reason, we wonder if stronger generalization guarantees for pure DP algorithms are possible beyond the simple analysis from above. In~\citep{rodriguez2021upper}, we found an affirmative answer for both ``in expectation'' and PAC-Bayesian guarantees to this question for \emph{permutation invariant algorithms} operating on \emph{discrete data}. We extend the analysis to GDP algorithms as well. Our results in this regard are given in~\Cref{subsec:pure_dp_gdp_discrete}. 

\subsection{Generalization Bounds for Pure DP and GDP Algorithms}
\label{subsec:pure_dp_gdp_discrete}

As mentioned above, we consider \emph{permutation invariant algorithms} that operate on \emph{discrete data}. When we refer to discrete data, we mean data whose instance space $\cZ$ is countable. Moreover, when we refer to permutation invariant algorithms we refer to algorithms whose output does not depend on the order of the training instances. Below, we give a formal definition of this term for clarity.

\begin{definition}
    \label{def:permutation_invariant_algorithms}
    An algorithm $\bA : \cZ^n \to \cP(\cW)$ is said to be \emph{permutation invariant} if it operates on a \emph{set} of instances instead of on a sequence of instances. This means that the hypothesis generated by the algorithm with the training set $s = (z_1, z_2, \ldots, z_n)$ has the same distribution as the one generated with every other permutation $\mathrm{Per}(s)$ of the training set. That is, $\bP_W^{S=s} = \bP_W^{S=\mathrm{Per}(s)}$ for every permutation $\mathrm{Per}$.    
\end{definition}

Before moving on to the results, let us describe a situation motivating the study of the generalization error of DP permutation invariant algorithms operating on discrete data.

\begin{example}
    \label{example:icu}
    Consider a medical setting where the doctors want to determine if an ICU patient should enter a treatment $B$ or should remain with the current treatment $A$.
    To make their decision, they want the aid of a model $f_W$ that takes as input a set of hand-crafted features $X = (F_1, F_2, \dots, F_k)$ and predicts if the patient will survive a certain (fixed) time after entering treatment $B$. 
    An example of such a hand-crafted feature (which showcases its discrete nature) is the systolic blood pressure classified in the following intervals: lower than 120 mmMg, between 120 and 129 mmHg, between 130 and 139 mmHg, between 140 and 179 mmHg, or higher than 180 mmHg.
    With this purpose, the doctors collect a set of $n$ historical records $Z_i$ containing the aforementioned features $X_i$ and if the patient survived $Y_i$.
    Then, they design the model $f_W$ (that is, the hypothesis $W$) employing the historical records $\lbrace Z_1, \ldots, Z_n \rbrace$ (that is, the training set $S$) ensuring that it is $\varepsilon$-DP to preserve the patients' anonymity, for example, with private logistic regression~\cite[Section~4]{yu2014differentially}.
    The model achieves a certain accuracy $\alpha \in [0,1]$.
    Finally, they wonder how well, on average, the model will describe new patients' data for the $0-1$ loss function $\ell(w,z)=\bI_{\{f_w(x) \neq y\}}(w,z)$.
\end{example}

\looseness=-1 The selection of the features in~\Cref{example:icu} reflects some real choices made by physicians from Karolinska Institutet in internal studies. This kind of quantization is commonplace in the medical community, especially in studies related to blood pressure assessment, and often the total cardinality of the data is small~\citep{eguchi2009optimal,muntner2019blood,ma2022development,eguchi2009optimal}. 

Regarding the choice of the permutation invariant algorithm, we chose a privatized version of logistic regression since it is a simple, often interpretable algorithm liked by the medical community (see~\citep{boateng2019review} and the references therein). However, other common algorithms like gradient descent are also permutation invariant and can be privatized by applying some privacy mechanism on the result such as the randomized response or the Laplace mechanisms~\citep{dwork2014algorithmic}.

Our results are based on a mathematical formulation to bound from above the relative entropy $\relent(\bP_W^S \Vert \bQ)$, and therefore bounding the generalization error based on the theorems from~\Cref{ch:expected_generalization_error,ch:pac_bayesian_generalization}. In this formulation, we choose the prior $\bQ$ to be a variational approximation, in the form of a mixture, of the marginal distribution $\bP_W$. Since DP and GDP algorithms are KL-stable, the mixture elements are selected to be the hypotheses' conditional distributions that best cover the set of possible training sets. The rule for selecting such a cover of the space is determined by the level of stability of the algorithm.

We develop a strategy, based on the method of types, to find explicit upper bounds on the relative entropy using this formulation {when the data is discrete}. First, we show that permutation invariant algorithms have a mutual information in $\cO(\log n)$, and therefore, they generalize, regardless of whether they are private or not. Then, we leverage the stability property of $\varepsilon$-DP and $\mu$-GDP algorithms to obtain tighter bounds. To be precise, we obtain bounds on the relative entropy and the mutual information with a rate of $$\cO\mleft(|\cZ| \log \frac{\varepsilon n}{|\cZ|} \mright) \hspace{2em} \textnormal{and} \hspace{2em} \cO \mleft(|\cZ| \log (\varepsilon \sqrt{n \log n})\mright)$$ for $\varepsilon$-DP algorithms and of $$\cO\mleft(|\cZ| \log \mleft( |\cZ| \mu^2 n^2\mright) \mright) \hspace{2em} \textnormal{and} \hspace{2em} \cO\mleft(|\cZ| \log (|\cZ| \mu^2 n \log n)\mright)$$ for $\mu$-GDP algorithms.

To be precise, we derive each of our results considering that the cardinality of the data is finite, that is $|\cZ| < \infty$. The extension to general countable instance spaces $\cZ$ follows immediately as, in that case, the bounds are vacuous.

\subsubsection{A Review of the Method of Types}
\label{subsubsec:method_of_types}

Recall that a training set is defined as a sequence $S = (Z_1, Z_2, \ldots, Z_n)$ where each instance $Z_i$ is i.i.d. according to the probability distribution $\bP_Z$. Then, the training set $S$ is distributed according to $\bP_S = \bP_Z^{\otimes n}.$

\begin{definition}[{\hspace{1sp}\cite[Section 11.1]{Cover2006}}]
\label{def:type}
The \emph{type} $T_{s}$, or \emph{empirical probability distribution}, of a training set $s \in \cZ^n$ is the relative proportion of occurrences of each symbol from $\cZ$.
That is, $T_{s}(z) = \frac{N(z|s)}{n}$ for all $z \in \cZ$, where $$N(z|s) \coloneqq |\{ z' \in s: z = z' \}|$$ is the number of times the symbol $z$ appears in the training set $s$.
The set of all possible types of sequences of elements from $\cZ$ of length $n$ is denoted by $\cT_{\cZ,n}$.
\end{definition}

An interesting property of the types is that, although the number of elements in $\cZ^n$ scales exponentially in the number of samples $n$, the total number of types only scales polynomially in $n$.
It is known that $|\cT_{\cZ,n}| \leq (n+1)^{|\cZ|}$ \cite[Theorem~11.1.1]{Cover2006}.
This can be marginally improved by the following proposition, which is tighter for finite $n$ and alphabets of small cardinality, especially binary alphabets. The proof of this proposition is given in~\Cref{app:proof_claim_num_types} since it does not contribute to the discussion.

\begin{proposition}
\label{claim:num_types}
$|\cT_{\cZ,n}| \leq (n+1)^{|\cZ|-1}$, with equality if and only if $|\cZ|=2$.
\end{proposition}

Another interesting property of the type $T_{s}$ of a sequence $s = (z_1, z_2, \ldots, z_n)$ is that it is equal to the type of any permutation of the same sequence, that is, $T_{s} = T_{\mathrm{Per}(s)}$, where $\mathrm{Per}$ is a random permutation.
Therefore, the type $T_{s}$ uniquely identifies the training set $s$ with elements $\lbrace z_1, z_2, \ldots, z_n \rbrace$, where the order of the instances is irrelevant.
In this way, we can describe the distance between two training sets $s$ and $s'$ by means of their types $T_s$ and $T_{s'}$.

\begin{definition}
\label{def:dist}
 The \emph{distance} between two training sets $s$ and $s'$ is defined as the minimum number of instances that needs to be changed in $s$ to obtain $s'$.
 It is proportional to the total variation distance between the types $T_s$ and $T_{s'}$,
 \begin{equation*}
  d(s,s') \coloneqq \frac{1}{2} \sum_{z \in\cZ} \big| N(z|s) - N(z|s') \big|
  = \frac{n}{2} \cdot \tv(T_s , T_{s'}).
 \end{equation*}
\end{definition}

\subsubsection{An Auxiliary Lemma}

The following lemma, inspired by the variational approximation given in~\cite{hershey_2007_approx}, bounds the relative entropy between two probability distributions {$\bP$ and $\bQ$}, where the latter is a mixture probability distribution. As mentioned previously, this lemma will be the foundation of our formulation to bound the relative entropy $\relent(\bP_W^{S=s} \Vert \bQ)$ by considering the prior $\bQ$ to be a mixture of posterior distributions $\bP_W^{S=s'}$ with fictitious sets $s'$ covering the space $\cZ^n$.

\begin{lemma}
\label{lemma:mixture_kl_ub}
Let $\bP$ and $\bQ$ be two probability distributions such that $\bP \ll \bQ$.
Let also $\bQ$ be a finite mixture of probability distributions such that $\bQ = \sum_b \omega_b \bQ_b$, where $\sum_b \omega_b = 1$, $\omega_b \geq 0$, and $\bP \ll \bQ_b$ for all $b$.
Then, the following inequalities hold:
\begin{align}
	\relent(\bP \Vert \bQ) &\leq - \log \left( \sum\nolimits_b \omega_b \exp \big( {-} \relent(\bP \Vert \bQ_b) \big) \right) 
	\label{eq:lemma_ub_1} \\
	&\leq \min_b \big\lbrace \relent(\bP \Vert \bQ_b) - \log \omega_b \big\rbrace .
	\label{eq:lemma_ub_2} 
\end{align}
\end{lemma}

Given the element of the mixture $\bQ_b$, \Cref{eq:lemma_ub_2} depicts the trade-off between its similarity with the distribution $\bP$ and its responsibility $\omega_b$.
Intuitively, this equation tells us that the relative entropy between $\bP$ and $\bQ$ is bounded from above by the divergence between $\bP$ and the closest, most probable element of $\bQ$.
Particularly, if the weights $\omega_b$ are the same for every element of the mixture, then the bound~\eqref{eq:lemma_ub_2} is controlled by the element that is closest to $\bP$ in relative entropy.
This behavior of the bound will be useful in the proofs of the main results.
This bound is especially tight when one element of the mixture is either very probable or much closer to $\bP$ than the others.
In the scenario where neither of these two conditions is met, the bound~\eqref{eq:lemma_ub_2} is loose and~\eqref{eq:lemma_ub_1} is preferred.

This bound can be useful elsewhere as we show in~\Cref{subsec:sgld} to find tighter bounds for the SGLD algorithm. In~\citep[Section VI. A]{rodriguez2021upper}, we also describe how this lemma can be employed to sharpen the analysis of lower bounds on the minimax error with Fano's method.

\subsubsection{A Simple Upper Bound Through the Method of Types}

Recall that our objective is to find an upper bound for $\relent(\bP_W^{S=s} \Vert \bQ)$ for a fixed training set $s$, and where we may choose any data-independent prior $\bQ$. 

Let us consider the prior $\bQ$ to be a mixture of all the conditional distributions $\bP_W^{S=s'}$ of the hypothesis given a fictitious training set $s'$, with mixture probability (or responsibility) of $\omega_{s'}$; that is, 
\begin{equation*}
    \bQ = \sum_{s' \in \cS} \omega_{s'} \bP_{W}^{T_{s'}},
\end{equation*}
where we take into account that $\bP_{W}^{S=s'} = \bP_{W}^{T_{s'}}$ a.s. for all training sets $s' \in \cS$, since a training set $s'$ with finite elements is completely characterized by its type $T_{s'}$. Recall again that the distribution $\bQ$ is still \emph{data-independent}, since it does not use the knowledge of the specific training set $s$, but only uses the knowledge of the instance space $\cZ$ and the number of samples $n$ from the problem.

Then, we may leverage \Cref{lemma:mixture_kl_ub} to obtain a more tractable upper bound. That is, 
\begin{align}
 \relent(\bP_W^{S=s} \Vert \bQ) \leq \min_{s' \in \cS} \left \lbrace \relent \mleft( \bP_{W}^{S=s} \Vert \bP_{W}^{T_{\smash{s'}}} \mright) - \log \omega_{s'} \right \rbrace.
 \label{eq:prop1_bnd1}
\end{align}

A naive approach is to consider an equiprobable mixture $\bQ$, that is, $\omega_s = |\cS|^{-1} = |\cT_{\cZ,n}|^{-1}$ for all $s \in \cS$. Then, we have that $  \relent(\bP_W^{S=s} \Vert \bQ) \leq \log |\cS|$, which, combined with \Cref{claim:num_types} leads to the following proposition.

\begin{proposition}
\label{prop:simple}
    Consider a discrete instance space $|\cZ| < \infty$. Then, for all training sets $s \in \cZ^n$, there exists a distribution $\bQ$ on $\cW$ such that
    \begin{equation*}
         \relent(\bP_W^{S=s} \Vert \bQ) \leq (|\cZ|-1) \log(1+n).
    \end{equation*}
\end{proposition}

Note that in the proof of~\Cref{prop:simple} we do not leverage the properties of \Cref{lemma:mixture_kl_ub} at their fullest, since we do not take advantage of the combination of the relative entropy and the mixture probabilities for the minimization~\eqref{eq:prop1_bnd1}.
However, if we note again that $\minf(S;W) \leq \bE_{s \sim \bP_S} \big[ \relent(\bP_W^{S=s} \Vert \bQ) \big]$ for every distribution $\bQ \in \cP(\cW)$ (\Cref{prop:properties_minf}), we observe how \Cref{lemma:mixture_kl_ub} allows us, naively, to obtain the same bound on the mutual information that one would obtain through the following decomposition of the mutual information:
\begin{align}
    \minf(W;S) = \ent(S) - \ent(S|W) \leq \ent(S) \leq \log |\cS| \leq (|\cZ|-1) \log(1+n), \label{eq:simple_obvious}
\end{align}
where the first inequality  is due to the non-negativity of the entropy for discrete random variables (\Cref{prop:entropy_properties_discrete}).

\begin{remark}
\label{rem:multinomial}
We note that~\eqref{eq:simple_obvious} may be improved using the tighter bound on the entropy of a multinomial distribution from~\cite[Theorem~3.4]{kaji_bounds_2015}.
More precisely, assuming that $n$ is sufficiently large and after some  algebraic manipulations, the bound states that
\begin{equation*}
    \minf(S;W) \in \cO\mleft( \frac{|\cZ| - 1}{2} \log \left( \frac{n}{|\cZ|} \right) + 2 |\cZ|\mright), %
\end{equation*}
which grows more slowly than~\eqref{eq:simple_obvious} with respect to $n$, but has a more involved expression and interpretation.
\end{remark}

\Cref{prop:simple} states the fact that, if $|\cZ|$ is finite, the number of possible training sets grows polynomially in $n$, which is independent of how the algorithm works and which training set $s$ is at hand.
In other words, even if the algorithm tries to memorize each instance in the training set, $\relent(\bP_W^{S=s} \Vert \bQ)$ cannot grow faster than logarithmically in $n$. Therefore, according to the bounds in~\Cref{ch:expected_generalization_error,ch:pac_bayesian_generalization}, that depend on the ratio $\nicefrac{\relent(\bP_W^{S=s} \Vert \bQ)}{n}$, permutation invariant algorithms generalize.

In the sequel, we restrict ourselves to algorithms that are $\varepsilon$-DP or $\mu$-GDP, and thus we can improve upon this simple upper bound.

\subsubsection{Upper Bounds on the Relative Entropy of Private Algorithms}

As previously mentioned, $\varepsilon$-DP and $\mu$-GDP algorithms are smooth, or stable, in the sense that close input training sets produce similar output hypothesis distributions. Therefore, similarly to what we did above to obtain~\Cref{prop:simple}, we may employ~\Cref{lemma:mixture_kl_ub} to bound $\relent(\bP_W^{S=s} \Vert \bQ)$ using a mixture distribution of posterior distributions $\bP_W^{S=s'}$ as a prior. However, in this case, the mixture only considers a sub-collection $\cS' \subseteq \cS$ of all possible training sets. Then, due to the smoothness of the algorithm, the relative entropy $\relent(\bP_W^{S=s} \Vert \bP_W^{S=s'})$ is small if the training sets $s$ and $s'$ are close.

Specifically, we may define the collection $\cS'$ by partitioning the space of training sets $\cS$ in equal-sized hypercubes and taking their ``central'' training set. More precisely, since a training set $s$ is uniquely defined by its type $T_s$, we may note that all the types lie inside the unit hypercube $[0,1]^{|\cZ|-1}$, as the last dimension is completely determined by the first $|\cZ| -1$. Then, we may split the $[0,1]$ interval of each dimension in $1 \leq t \leq n$ parts, resulting in a hypercube cover of $t^{|\cZ|-1}$ smaller hypercubes. The type at the center of each smaller hypercube represents a training set of the collection $\cS'$.
In this way, the cover guarantees that the maximal distance inside a hypercube, that is, $\max_{s \in \cS} \min_{s' \in \cS'} d(s,s')$ is bounded. Then, if the mixture components are equiprobable and $\omega_{s'} = |\cS'|^{-1} = t^{-(|\cZ|-1)}$, the bound from \Cref{lemma:mixture_kl_ub} trades the relative entropy $\relent(\bP_W^{S=s} \Vert \bP_W^{S=s'})$, which decreases as the maximal distance decreases by increasing the number of hypercubes $|\cS'|$, for the logarithm of the responsibility $- \log \omega_{s'}$, which increases as the number of hypercubes $|\cS'|$ grows. 

The choice of the number of splits $t$ of each interval $[0,1]$ of each dimension is done based on the stability properties of each algorithm determined by their DP parameters $\varepsilon$ and $\mu$ and results in the following proposition. The complete proof is a bit technical and the details do not add to the discussion and therefore is deferred to~\Cref{app:proof_prop_mGDP_1}.

\begin{proposition}
\label{prop:mi_mGDP}
Consider a discrete instance space $|\cZ| < \infty$ and an algorithm $\bA$ characterized by a Markov kernel $\bP_W^S$. Then, for every training set $s \in \cZ^n$:
\begin{enumerate}
\item If $\bA$ is $\varepsilon$-DP and $\varepsilon \leq 1$, then there exists a distribution $\bQ$ on $\cW$ such that
\begin{equation}
	\relent(\bP_W^{S=s} \Vert \bQ) \leq (|\cZ|-1) \log \big(1 + e \varepsilon N \big).
	\label{eq:prop_eDP_1_main}
\end{equation}
\item If $\bA$ is $\mu$-GDP and $\mu \leq \nicefrac{1}{\sqrt{|\cZ|-1}}$, then there exists a distribution $\bQ$ on $\cW$ such that
\begin{equation}
	\relent(\bP_W^{S=s} \Vert \bQ) \leq \frac{1}{2}(|\cZ|-1) \log \big( 1+ e (|\cZ| - 1) \mu^2 N^2 \big).
 	\label{eq:prop_mGDP_1_main}
\end{equation}
\end{enumerate} 
For lower privacy guarantees, that is, $\varepsilon > 1$ or $\mu > \nicefrac{1}{\sqrt{|\cZ|-1}}$, the upper bounds on $\relent(\bP_W^{S=s} \Vert \bQ)$ are no better than the one in \Cref{prop:simple}.
\end{proposition}

Note that the upper bounds~\eqref{eq:prop_eDP_1_main} and~\eqref{eq:prop_mGDP_1_main} are tighter than the general one from~\Cref{prop:simple} once the algorithms have at least a privacy guarantee of $\varepsilon \leq \nicefrac{1}{e}$ or $\mu\leq \nicefrac{1}{\sqrt{e(|\cZ|-1)}}$ (since $\sqrt{1+n^2} < 1+n$ for $n\geq 1$).

The upper bounds from \Cref{prop:simple} and~\Cref{prop:mi_mGDP} can be tightened if a more accurate value for the number of hypercubes needed to cover the space of training sets is used. To be precise, since the types are empirical probability distributions they lay in the $(|\cZ|-1)$-simplex. In other words, if we consider the unit hypercube $[0,1]^{|\cZ|-1}$, there are only types inside the hypervolume comprised between the origin and the $(|\cZ|-2)$-simplex. Proceeding equally as before, we obtain the following proposition, whose complete proof is in~\Cref{app:proof_prop_mGDP_2}. Although the result is tighter, this improvement comes at the expense of losing simplicity in the final expressions.

\begin{proposition}
\label{prop:mi_mGDP_2}
Consider a discrete instance space $|\cZ| < \infty$ and an algorithm $\bA$ characterized by a Markov kernel $\bP_W^S$. Then, for every training set $s \in \cZ^n$:
\begin{enumerate}
\item If $\bA$ is $\varepsilon$-DP and $\varepsilon \leq \nicefrac{1}{n}$, then there exists a distribution $\bQ$ on $\cW$ such that
\begin{align}
	\relent (\bP_W^{S=s} \Vert \bQ)\leq (|\cZ|-1)(1+\varepsilon n) - \frac{1}{2} \log \mleft(2 \pi (|\cZ|-1)\mright),
	\label{eq:prop_eDP_2_main_0}
\end{align}
and if $\nicefrac{1}{n} < \varepsilon \leq 1$, then there exists a distribution $\bQ$ on $\cW$ such that
\begin{align}
	\relent (\bP_W^{S=s} \Vert \bQ)\leq (|\cZ|-1) &\log \left(1 + \frac{2}{|\cZ|-1} \varepsilon n  \right) \nonumber \\ &+(|\cZ|-1) \log \left( \frac{e^2}{2} \right) -\frac{1}{2} \log  \mleft( 2 \pi (|\cZ|-1) \mright).
	\label{eq:prop_eDP_2_main}
\end{align}

\item If $\bA$ is $\mu$-GDP and $\mu \leq \nicefrac{1}{\big(n\sqrt{|\cZ|-1}\big)}$, then there exists a distribution $\bQ$ on $\cW$ such that
\begin{align}
    \MoveEqLeft[1]
    \relent (\bP_W^{S=s} \Vert \bQ)\leq (|\cZ|-1) \left( 1+ \frac{|\cZ|-1}{2} \mu^2 n^2 \right) -\frac{1}{2} \log \mleft(2 \pi (|\cZ|-1)\mright),
    \label{eq:prop_mGDP_2_main_0}
\end{align}
and if $\nicefrac{1}{\big(n\sqrt{|\cZ|-1}\big)} < \mu \leq \nicefrac{1}{\sqrt{|\cZ|-1}}$, then there exists a distribution $\bQ$ on $\cW$ such that
\begin{align}
    \relent (\bP_W^{S=s} \Vert \bQ)\leq (|\cZ|-1) &\log \left(1 + \frac{2}{\sqrt{|\cZ| -1}} \mu n \right) \nonumber \\ &+(|\cZ|-1) \log \left( \frac{ e^\frac{3}{2} }{ 2 } \right) -\frac{1}{2} \log \mleft(2 \pi (|\cZ|-1)\mright).
	\label{eq:prop_mGDP_2_main}
\end{align}
\item For any algorithm $\bA$ there exists a distribution $\bQ$ on $\cW$ such that
\begin{align}
	\relent (\bP_W^{S=s} \Vert \bQ)\leq (|\cZ|-1) \log &\left(1 + \frac{n}{|\cZ|-1} \right) \nonumber \\ 
    &+(|\cZ|-1) -\frac{1}{2} \log \mleft(2 \pi (|\cZ|-1)\mright).
	\label{eq:prop_eDP_2_main_2}
\end{align}
\end{enumerate}
\end{proposition}

Note that, for large $\varepsilon n$ or $\mu n$, the gap between~\eqref{eq:prop_eDP_1_main} and~\eqref{eq:prop_eDP_2_main}, and the gap between~\eqref{eq:prop_mGDP_1_main} and~\eqref{eq:prop_mGDP_2_main} grows as
\begin{equation}
	(|\cZ|-1) \log \left( \frac{ |\cZ|-1 }{ e }\right) +\frac{1}{2} \log 2 \pi (|\cZ|-1).
	\label{eq:diff_prop2_prop3}
\end{equation}
This highlights the benefit of \Cref{prop:mi_mGDP_2} for alphabets of large cardinality.

We also note that~\eqref{eq:prop_eDP_2_main_2} is a valid upper bound on the relative entropy $\relent(\bP_W^{S=s} \Vert \bQ)$ for algorithms that are not stable. Similarly to \Cref{prop:simple}, the covering mixture used here takes every possible training set. However, this bound improves upon~\Cref{prop:simple} by using a more exact value for the total number of training sets. Consequently, the gap between~\Cref{prop:simple} and~\eqref{eq:prop_eDP_2_main_2} also scales as~\eqref{eq:diff_prop2_prop3} for large $n$.

To get an idea of how these bounds compare with each other, let us recover~\Cref{example:icu} and let us assume that the cardinality of the data, that is $|\cZ| = |\cX \times \cY|$, is fixed to $100$, that the probability required by the physicians is $\beta = 0.05$, and that they achieved an empirical risk of $0.05$. Then, consider two scenarios:
\begin{enumerate}
    \item In the first scenario, the privacy parameter $\varepsilon$ is fixed to $0.1$ and the number of samples $n$ varies from $1$ to $10,000$.
    \item In the second scenario, the number of samples $n$ is fixed to $5,000$ and the privacy parameter $\varepsilon$ varies from $10^{-4}$ to $1$.
\end{enumerate}

\Cref{fig:comparison_propositions_icu_kl} showcases how the proposed bounds achieve non-vacuous generalization error and depicts the trade-offs between privacy, the number of samples, and generalization. More specifically, we see how the more private (smaller $\varepsilon$) and the more samples $n$ we have, the better we can guarantee generalization.

\begin{figure}
    \centering
    \includegraphics[width=0.45\textwidth]{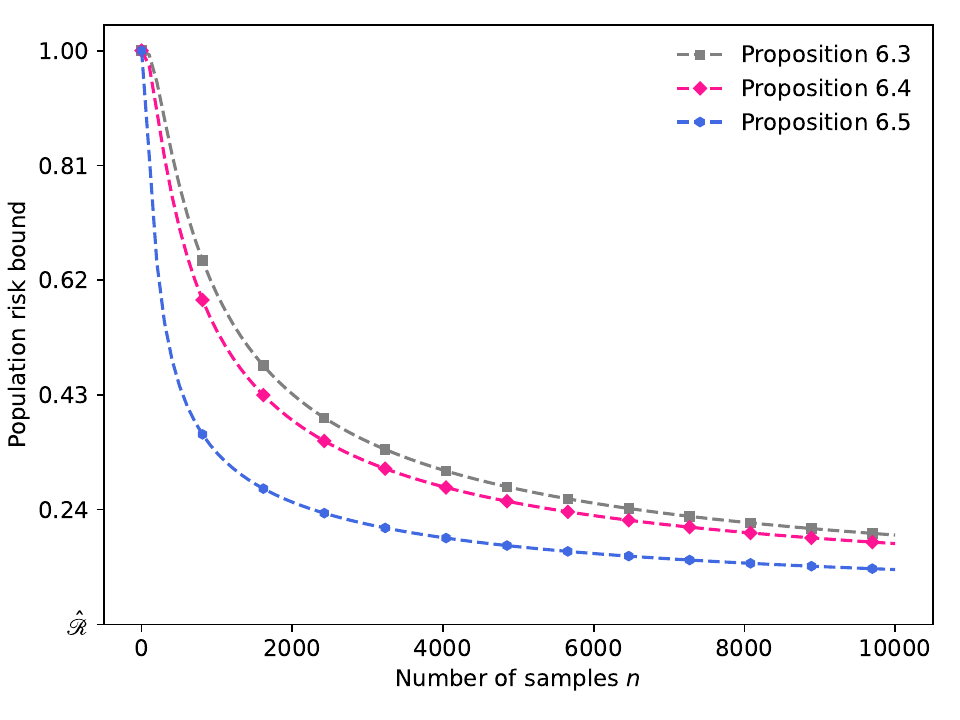}
    \includegraphics[width=0.45\textwidth]{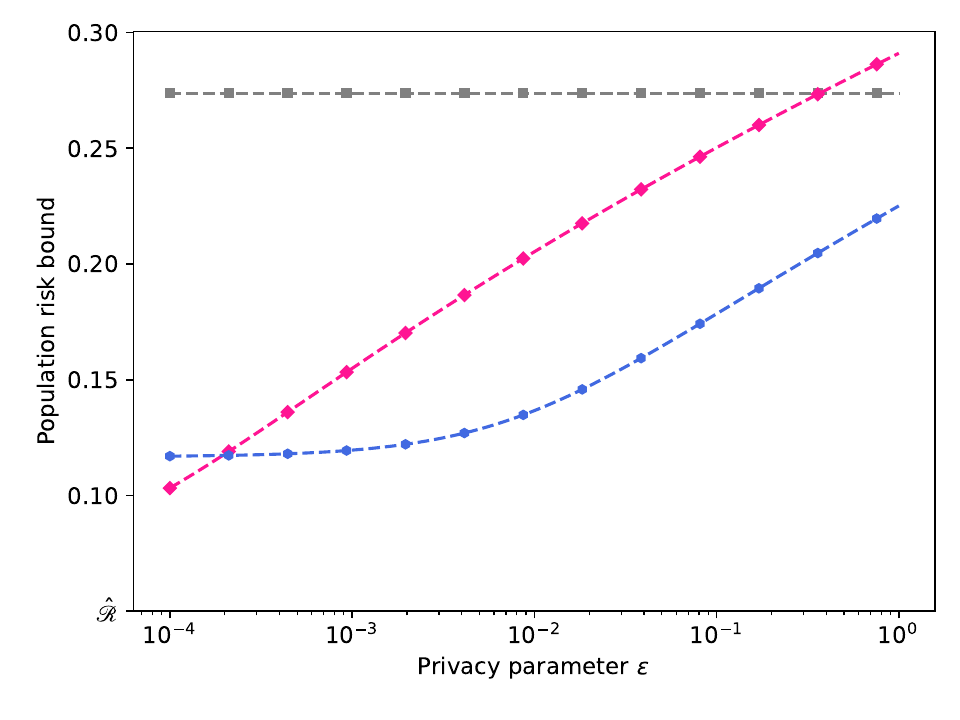}
    \caption{Comparison of the generalization guarantees combining \Cref{prop:simple,prop:mi_mGDP,prop:mi_mGDP_2} with the Seeger--Langford bound~\citep{seeger2002pac,langford2001bounds} from~\Cref{th:seeger_langford_pac_bayes} under the setting of~\Cref{example:icu}. The cardinality of the data is fixed to $|\cZ| = 100$, probability parameter is fixed to $\beta=0.05$, and the empirical risk is $\bE^{S}[\emprisk(W,S)] = 0.05$. In the left figure, the privacy parameter is fixed to $\varepsilon=0.1$ and the number of samples $n$ varies from $1$ to $10,000$. In the right figure, the number of samples is fixed to $n = 5,000$ and the privacy parameter $\varepsilon$ varies from $10^{-4}$ to $1$. In both figures, the $y$-axis starts at the value of the empirical risk.}
    \label{fig:comparison_propositions_icu_kl}
\end{figure}

So far, we have found upper bounds on the relative entropy $\relent(\bP_W^{S=s} \Vert \bQ)$ that hold for every $s \in \cS$, that is, they are \emph{uniform bounds}. However, we know from information theory that there is a collection, called the \emph{typical set} $\cT$, of types that are \emph{typical}. This means that the probability of randomly sampling a training set with a type in the typical set is close to one~\citep[Chapters 3 and 10]{Cover2006}. Therefore, noting that $\minf(W;S) \leq \bE \big[ \relent(\bP_W^{S} \Vert \bQ) \big]$ (\Cref{prop:properties_minf}), we may only cover the typical set $\cT$, and rely on the low probability of observing a training set from outside $\cT$ to bound the mutual information $\minf(W;S)$. More precisely, by the law of total expectation
\begin{equation}
    \label{eq:prop_mi_mGDP_3_minf_law_total_expectation}
    \minf(W;S) \leq \bE^{\{S \in \cT\}} \big[ \relent(\bP_W^S \Vert \bQ) \big] \cdot \bP \big[ S \in \cT \big] + \bE^{\{  S \not \in \cT \}} \big[ \relent(\bP_W^S \Vert \bQ) \big] \cdot \bP \big[ S \not \in \cT \big],
\end{equation}
where the first term is small due to the control of the relative entropy $\relent(\bP_W^S \Vert \bQ)$ using the marginal $\bQ$ tailored for the typical training sets, and the second term is small since the probability $\bP \big[ S \not \in \cT \big]$ vanishes.
In this way, we can provide sharper generalization guarantees in expectation. This procedure leads to the following proposition, whose proof is in~\Cref{app:proof_prop_mGDP_3}.

\begin{proposition}
\label{prop:mi_mGDP_3}

Consider a discrete instance space $|\cZ| < \infty$ and an algorithm $\bA$ such that $W = \bA(S)$.
\begin{enumerate}
\item If $\bA$ is $\varepsilon$-DP and $\varepsilon \leq 2$, then 
\begin{equation}
    \minf(W;S) \leq |\cZ| \log \big(1+ e \varepsilon \sqrt{n \log n} \big) + 2|\cZ|\frac{\varepsilon}{n},
 \label{eq:prop_eDP_3_main_1}
\end{equation}
while if $\varepsilon>2$, 
\begin{equation}
     \minf(W;S) \leq |\cZ| \log \big( 1+ 2\sqrt{n\log n} \big) + 2|\cZ|\frac{\varepsilon}{n}.
 \label{eq:prop_eDP_3_main_2}
\end{equation}
\item If $\bA$ is $\mu$-GDP and $\mu \leq \nicefrac{2}{\sqrt{|\cZ|}}$, then 
\begin{equation}
    \minf(W;S) \leq \frac{|\cZ|}{2} \log \left(1+ e |\cZ| \mu^2 n \log n\right) + |\cZ| \mu^2,
 \label{eq:prop_mGDP_3_main_1}
\end{equation}
while if $\mu>\nicefrac{2}{\sqrt{|\cZ|}}$,
\begin{equation}
    \minf(W;S) \leq |\cZ| \log \big( 1+ 2\sqrt{n\log n} \big) + |\cZ| \mu^2.
 \label{eq:prop_mGDP_3_main_2}
\end{equation}
\end{enumerate}
\end{proposition}

If we compare this asymptotic behavior with the one from~\Cref{prop:mi_mGDP,prop:mi_mGDP_2} we see now how the privacy coefficients $\varepsilon$ and $\mu$ multiply $\sqrt{n \log n}$ instead of $n$.
We note that $\sqrt{n \log n} < n$ for $n \geq 1$, which highlights the benefit of discriminating between typical and non-typical training sets. However, in practice, as we will see next in~\Cref{subsec:other_results_dp_gen}, \Cref{prop:mi_mGDP_2} is often tighter than this result in the common operating regions prior to the asymptotic regime.

\begin{remark}
We note that the common Laplace or Gaussian mechanisms that assure $\varepsilon$-DP and $\mu$-GDP add random noise of variance proportional to the inverse of the square of the parameter $\varepsilon$ or $\mu$~\cite[Section~3.3]{dwork2014algorithmic}, \cite[Theorem~2.7]{dong2021gaussian}.
If we compare the upper bounds obtained for $\varepsilon$-DP and $\mu$-GDP algorithms in \Cref{prop:mi_mGDP,prop:mi_mGDP_2,prop:mi_mGDP_3}, we observe that their asymptotic behaviors appear to be almost identical by letting $\mu = \varepsilon /\sqrt{|\cZ|-1}$.\footnote{In the case of \Cref{prop:mi_mGDP_3}, we would need $\mu = \nicefrac{\varepsilon}{\sqrt{|\cZ|}}$ to be more precise.}
This suggests that, even if both measures of privacy are not equivalent, $\mu$-GDP algorithms need to be ``noisier'' than $\varepsilon$-DP algorithms, as $|\cZ|$ grows, in order to obtain similar generalization performance according to our bounds.

The main reason for this behavior is that our results depend on the stability of the algorithm, measured by the relative entropy between the hypotheses obtained by two training sets at a distance $k$ (see \Cref{claim:diff_priv_distance}).
If we let $\varepsilon = \mu = \gamma$, then the tighter upper bound of the relative entropy for $\gamma$-DP algorithms is always smaller than the upper bound for $\gamma$-GDP algorithms.
Moreover, the looser upper bound of the relative entropy for $\gamma$-DP algorithms, that is, $k\gamma$, is smaller than the upper bound for $\gamma$-GDP algorithms once $\gamma > \nicefrac{2}{k}$.
Therefore, the hypothesis distribution of $\mu$-GDP algorithms changes more rapidly with the distance between two training sets, thus making these kinds of algorithms less stable (or smooth).
If we recall that the distance (in expectation) between two training sets grows with the size of the sample space $\cZ$, we can intuit the relationship between $\mu$ and $\varepsilon$ in our results for generalization.
\end{remark}

\subsubsection{Generalization Error of Other Stable Algorithms}

Although our main results deal with private algorithms, they can be adapted to other families of algorithms with stability or smoothness properties. 
That is, if the relative entropy between the distributions of the output hypotheses, produced by an algorithm fed with two training sets $s$ and $s'$ at a distance $k$, is bounded from above by some function $\phi(k)$, then \Cref{prop:mi_mGDP,prop:mi_mGDP_2,prop:mi_mGDP_3} can be replicated with this quantity in mind.
Note that when $k=1$, this property is known as $\sqrt{\phi(1)/2}$-KL-stability \cite[Definition 4.2]{bassily2016algorithmic}.

For example, the privacy framework of \emph{Rényi differential privacy} (RDP) states that an algorithm $\bA$ is $(\alpha,\varepsilon)$-RDP if $\renyidiv{\alpha}(\bP_W^{S=s} \Vert \bP_W^{S=s'} ) \leq \varepsilon$ for any two neighbouring datasets $s$ and $s'$~\cite{mironov2017renyi}.
Given the monotonicity of the Rényi divergence with respect to $\alpha$, we have that $\relent(\bP_W^{S=s} \Vert \bP_W^{S=s'} ) \leq \renyidiv{\alpha}(\bP_W^{S=s} \Vert \bP_W^{S=s'} )$ for every $\alpha \geq 1$. The case $0< \alpha <1$ is more challenging but still possible to bound.
Furthermore, this framework can also be extended for group privacy~\cite[Proposition~2]{mironov2017renyi}, which enables the use of our results.

\subsubsection{Particularization to Specific Private Algorithms}

The generalization error of specific learning algorithms that have been proved to be $\varepsilon$-DP or $\mu$-GDP can be characterized.
For instance, the Noisy-SGD and Noisy-Adam~\cite[Algorithms~1 and~2]{bu2019deep}, private, regularized variants of the common stochastic gradient descent (SGD) and Adam~\cite[Algorithm~1]{kingma2014adam} algorithms, are approximately $\frac{|v|}{n} \sqrt{T (e^{\nicefrac{1}{\sigma^2}}-1)}$-GDP, where $|v|$ is the batch size, $T$ is the number of iterations made by the algorithm, and $\sigma^2$ is the noise variance~\cite{bu2019deep}.
This example also highlights the benefit of private algorithms for generalization, given that the privacy parameter $\mu = \frac{|v|}{n} \sqrt{T (e^{\nicefrac{1}{\sigma^2}}-1)}$, decreases with the inverse of the number of samples $n$.

Similarly, there are common algorithms like Markov chain Monte Carlo (MCMC) that can be proved to be $(\alpha,\varepsilon)$-RDP~\cite{heikkila2019differentially} and others like classical logistic regression that can be adapted to be differentially private~\cite{yu2014differentially} with tunable privacy parameters.
Employing the aforementioned characterization of the generalization capabilities of Rényi differentially private algorithms, one would be able to obtain bounds on the generalization error of MCMC.

\subsection{Further Bibliography Connecting DP and Generalization}
\label{subsec:other_results_dp_gen}

The characterization of the generalization error for DP algorithms has predominantly been performed under the single-draw PAC-Bayesian framework. Currently, the expected generalization error bounds for DP algorithms are derived either from (i) results similar to those in~\Cref{ch:expected_generalization_error} using bounds on other measures of dependence between the hypothesis $W$ and the training set $S$, or from (ii) showing the relationship between privacy and stability (see~\Cref{subsec:privacy_as_stability} and~\citep{JMLR:v17:15-313}) and the bounds for stable algorithms (see \Cref{subsec:uniform_stability} and the references therein).

As we will see, these results are often of the same nature as the direct results we obtained in the beginning of~\Cref{sec:differential_privacy_generalization}. That is, the generalization is controlled by the privacy parameter but it does not decrease with the number of samples. Therefore, to guarantee the algorithm's generalization, they require that the privacy parameter decreases with the number of samples.

\subsubsection{Single-Draw PAC-Bayesian Bounds}

For pure DP algorithms, the first bounds on the generalization error are due to~\citet{dwork2015preserving}. More precisely, consider a loss with a range bounded in $[0,1]$, \citet[Theorem 9]{dwork2015preserving} stated that, for every $\tau > 0$, $\beta \in (0,1)$, and every $\varepsilon$-DP algorithm, as long as the number of samples is $n \geq \frac{12}{\tau^2} \log \nicefrac{4}{\beta}$ and $\varepsilon \leq \nicefrac{\tau}{2}$, then $\bP[\gen(W,S) > \tau] \leq \beta$. If we fix $\tau = 2\varepsilon$, we may get a statement with a terminology more similar to the one we have been using throughout the manuscript. That is, for every $\varepsilon$-DP, with probability no smaller than $1 - 4 \exp(- \nicefrac{n \varepsilon^2}{3})$
\begin{equation*}
    \gen(W,S) \leq 2 \varepsilon.
\end{equation*}

In that article, \citet[Theorem 11]{dwork2015preserving} also propose a generalization bound for pure DP algorithms, although this depends on the probability of not generalizing with the worst possible training set, and therefore we do not write it in this survey. Later, \citet[Corollary 7]{dwork2015generalization} improved upon that result and showed that for every $\varepsilon$-DP algorithm, with probability no smaller than $1 - 3 \exp(-\varepsilon^2 n)$
\begin{equation}
    \label{eq:dwork_pure_dp_gen}
    \gen(W,S) \leq \varepsilon L'
\end{equation}
for losses $\ell(w,\cdot)$ that are $L'$-Lipschitz for all $w \in \cW$. Actually, their result is given for the \emph{sensitivity} of the loss, but in the context of generalization, this is equivalent to $\nicefrac{1}{n}$ times the Lipschitz constant with respect to the instances. For bounded losses in $[0,1]$, then $L' = 1$. Therefore, they improved the result in~\citep{dwork2015preserving} by increasing the probability under which the generalization error is controlled by the parameter $\varepsilon$ and by reducing the constant in front of the privacy parameter.

Another result for pure DP algorithms comes from~\citet[Lemma 22]{jung2021new}, that states that for every $\varepsilon$-DP algorithm and every $\beta \in (0,1)$, with probability no smaller than $1 - \beta$
\begin{equation*}
    \gen(W,S) \leq e^\varepsilon - 1 + \sqrt{\frac{2 \log \frac{2}{\beta}}{n}}.
\end{equation*}
This result is generally looser than~\eqref{eq:dwork_pure_dp_gen} as $e^\varepsilon - 1 \geq \varepsilon$, but it works for every $\beta \in (0,1)$ and can be useful in the small data regime. The same is true for our simple result from~\Cref{th:dp_generalizes_bounded_cgf}, which is also looser than~\eqref{eq:dwork_pure_dp_gen} but works for every $\beta \in (0,1)$. However,  \Cref{th:dp_generalizes_bounded_cgf} is generally still stronger than~\citet[Lemma 22]{jung2021new} since (i) the dependence with $\varepsilon$ is linear and not exponential and (ii) the extra dependence with $\log \nicefrac{1}{\beta}$ is of a fast rate instead of a slow one.

For approximate DP, \citet[Theorem 10]{dwork2015preserving} also presented a generalization bound for losses with a range bounded in $[0,1]$. Adapted to the terminology employed throughout the manuscript, their results state that for every $(\varepsilon, \delta)$-DP algorithm such that $\delta \geq \exp \big( - \frac{\log 2}{\varepsilon} - \frac{n \varepsilon}{3}\big)$, with probability no smaller than $1 - 4 \exp(- \nicefrac{n \varepsilon^2}{3})$
\begin{equation*}
    \gen(W,S) \leq 4 \varepsilon.
\end{equation*}
 An alternative bound was given later by~\citet[Theorem 8]{dwork2015generalization}, where they show that for every $(\varepsilon, \delta)$-DP algorithm such that $\delta \geq 8 \exp \big( - \frac{n \varepsilon}{2}\big)$, with probability no smaller than $1 - \frac{16}{\varepsilon^2} \log \big(\frac{2}{\varepsilon}\big) \exp \big( - \frac{n \varepsilon^2}{2}\big)$
\begin{equation*}
    \gen(W,S) \leq 4 \varepsilon.
\end{equation*}
Then, \citet[Theorem 7.2]{bassily2016algorithmic} built on their work and proved that for every $(\varepsilon, \delta)$-DP algorithm such that $\varepsilon \in (0,\nicefrac{1}{3})$ and $\delta \in( 4 \varepsilon \exp (-n\varepsilon^2), \nicefrac{\varepsilon}{4})$, with probability no smaller than $1 - \frac{4 \varepsilon}{3}  \exp (-n\varepsilon^2)$
\begin{equation*}
    \gen(W,S) \leq 18 \varepsilon L'
\end{equation*}
for losses $\ell(w,\cdot)$ that are $L'$-Lipschitz for all $w \in \cW$.

Finally, \citet[Lemmata 7 and 15]{jung2021new} proved similar results as the ones they had for pure DP algorithms. For losses with a range bounded in $(0,1)$, for every $(\varepsilon,\delta)$-DP algorithm and every $\beta \in (0,1)$, with probability no smaller than $1 - \beta$
\begin{equation*}
    \gen(W,S) \leq e^\varepsilon - 1 + \frac{2 \delta}{\beta}.
\end{equation*}
Similarly, for losses $\ell(w,\cdot)$ that are $L'$-Lipschitz for all $w \in \cW$, they show that for every $(\varepsilon, \delta)$-DP algorithm and every $\beta \in (0,1)$, with probability no smaller than $1 - \beta$
\begin{equation*}
    \gen(W,S) \leq L' \mleft( e^\varepsilon - 1 + \frac{4\delta}{\beta} \mright).
\end{equation*}

\subsubsection{Expected Generalization Error Bounds}

\citet{dwork2014algorithmic} introduced the concept of \emph{max-information} $\minf_\infty$ and $\beta$-\emph{approximate max-information} $\minf_\infty^\beta$ in order to find bounds on the generalization error of algorithms. These quantities are defined as
\begin{equation*}
    \minf_\infty^\beta(W;S) \coloneqq \log \sup_{\cA \subseteq \cW \times \cS : \bP_{W,S}[\cA] > \beta}\frac{\bP_{W,S} \big[ \cA ] - \beta}{ (\bP_W \otimes \bP_S)[\cA]}
\end{equation*}
and $\minf_\infty(W;S) = \minf_\infty^0(W;S) = \renyidiv{\infty}(\bP_{W,S} \Vert \bP_W \otimes \bP_S)$ respectively. The relationship between approximate max information and the generalization error is not immediate. However, it is easy to see that $\minf(W;S) \leq \minf_\infty(W;S)$, and therefore any bound on this quantity can be readily employed in the bounds from~\Cref{ch:expected_generalization_error}. In particular, \citet[Theorem 7]{dwork2015generalization} noted that every $\varepsilon$-DP algorithm has a bounded max-information, reaching the same bound we gave at the start~\Cref{sec:differential_privacy_generalization}, namely that $\minf(W;S) \leq \minf_\infty(W;S) \leq n \varepsilon$. Recall that since the expected generalization bounds from~\Cref{ch:expected_generalization_error} depend on the ratio $\nicefrac{\minf(W;S)}{n}$, this kind of results only work if the privacy parameter $\varepsilon$ decreases with the number of samples.

A stronger statement was later proved by~\citet[Proposition 1.4 and Theorem 1.10]{bun2016concentrated}, although they did not show it in the context of bounding the generalization error. They showed that any $\varepsilon$-DP algorithm has a bounded mutual information, that is $\minf(W;S) \leq \frac{1}{2} n \varepsilon^2 $, improving upon the previous statement as long as $\varepsilon \leq 2$. 

Finally, \citet[Lemma 8]{JMLR:v17:15-313} showed that any $\varepsilon$-DP algorithm is uniformly stable with parameter $(e^\varepsilon - 1)$. Moreover, if $\varepsilon \leq 1$, then it is uniformly stable with parameter $2 \varepsilon$. Combing this with the fact that if an algorithm is uniformly stable with parameter $\gamma$, then $\bE \big[ \gen(W,S) \big] \leq \gamma$, we see that the findings from~\citet{JMLR:v17:15-313} also show expected generalization errors that do not decrease with the number of samples.

We may compare these results with the ones we developed in~\Cref{subsec:pure_dp_gdp_discrete}. At first glance, it is clear that our results are asymptotically tighter
since the expected generalization error vanishes as the number of samples grows, as opposed to the previous results, for which the generalization error bound remains constant. Nonetheless, our bounds depend on the cardinality of the input space, which is not the case for the other bounds;
those known bounds may thus provide a better characterization of the generalization error in the small sample regime. 
In order to better clarify the comparison between these results and those presented in this manuscript, we recover~\Cref{example:icu} and observe three different scenarios where three important parameters are varied:
\begin{enumerate}
    \item In the first scenario, the empirical risk $\bE[\emprisk(W,S)]$ is fixed to $0.05$, the cardinality of the data $|\cZ|$ is set to $100$, the privacy parameter $\varepsilon$ is set to $0.6$, and the number of samples $n$ vary from $1$ to $10,000$.
    \item In the second scenario, the cardinality of the data $|\cZ|$ and the empirical risk $\bE[\emprisk(W,S)]$ are still fixed to $100$ and $0.05$, respectively, but now the number of samples $n$ is fixed to $2,500$ and the privacy parameter $\varepsilon$ varies from $0$ to $2$.
    \item In the third scenario, the privacy parameter $\varepsilon$ is fixed to $0.6$, the number of data samples $n$ and the empirical risk $\bE[\emprisk(W,S)]$ are also fixed to $2,500$ and $0.05$, respectively, and the data cardinality $|\cZ|$ now varies from $2$ to $1,000$. 
\end{enumerate}

The results are presented in \Cref{fig:comparison_propositions_icu_minf}.
In the first scenario, the results show that, as mentioned before, the bounds from \citet{dwork2015generalization}, \citet{bun2016concentrated}, and~\citet{JMLR:v17:15-313} outperform the presented bounds in the small sample regime and are outperformed as the number of samples increases.
The results on the second scenario show how, for a moderate cardinality and number of samples, the presented bounds can outperform the bounds from \citet{dwork2015generalization}, \citet{bun2016concentrated}, and~\citet{JMLR:v17:15-313} when the privacy parameter is not very small.
Nonetheless, for very private algorithms ($\varepsilon \ll 1$), their bounds are generally tighter due to their linear (or even quadratic) dependency with $\varepsilon$.
Finally, the results on the third scenario show how, for a moderate number of samples, the results from this work are only applicable in the small cardinality regime; when the cardinality of the data increases, the bounds from \citet{dwork2015generalization}, \citet{bun2016concentrated}, and~\citet{JMLR:v17:15-313} are preferred.

\begin{figure}
    \centering
    \includegraphics[width=0.45\textwidth]{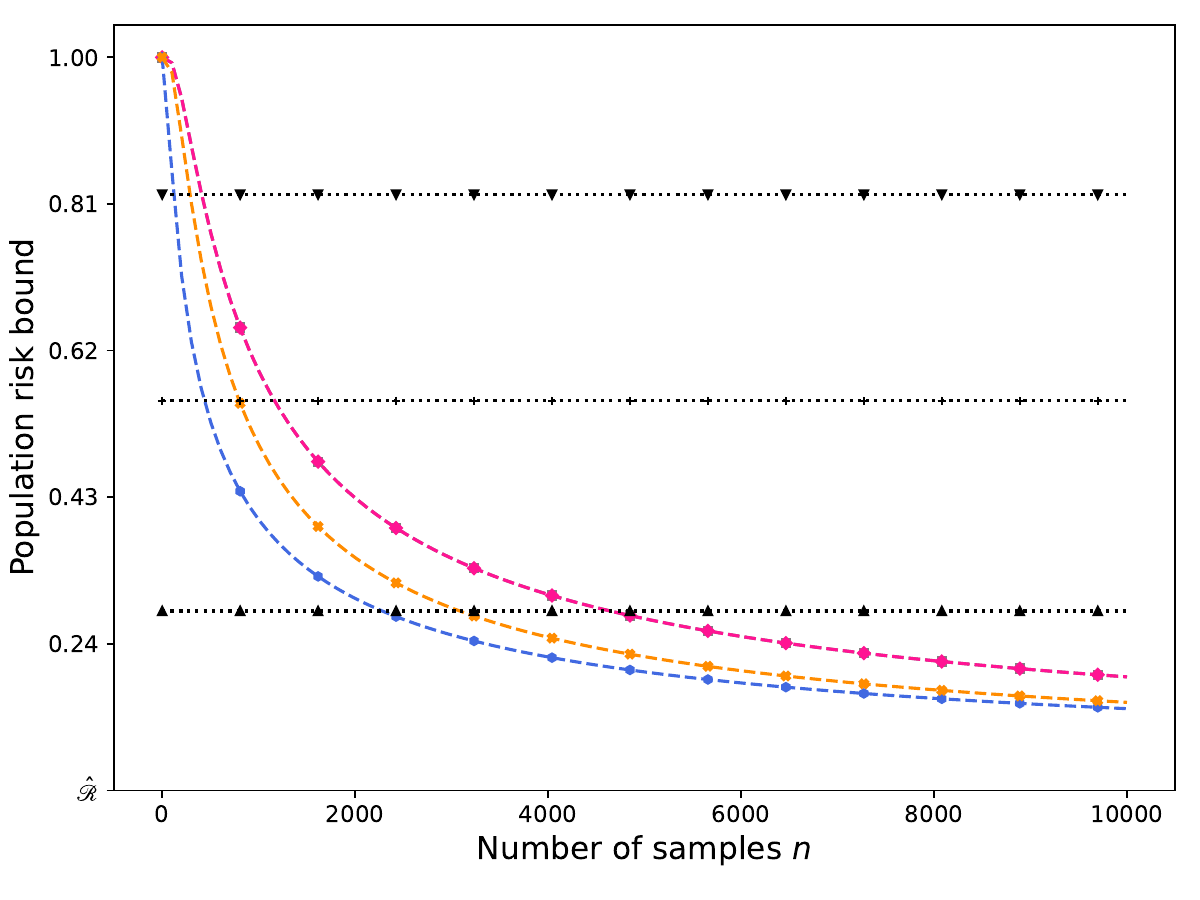}
    \includegraphics[width=0.45\textwidth]{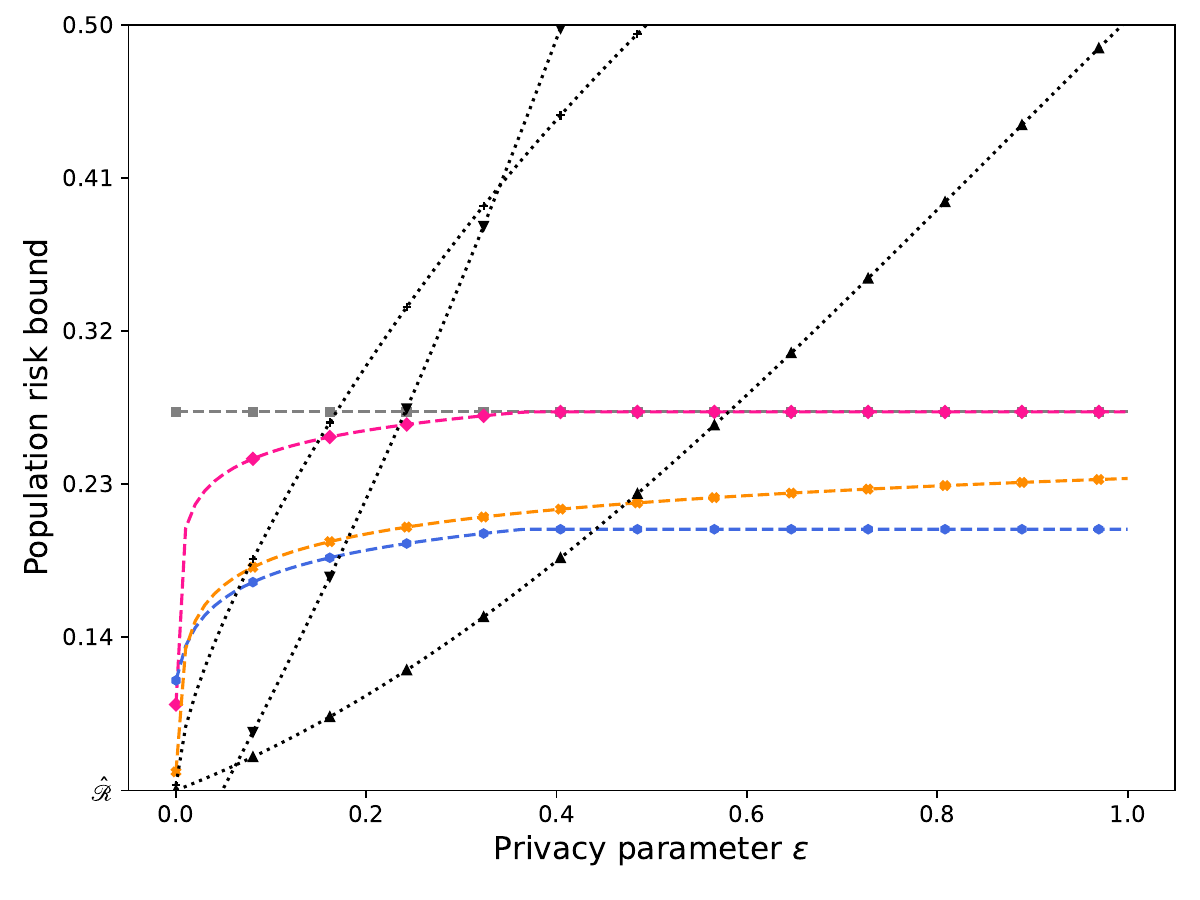}
    \includegraphics[width=0.45\textwidth]{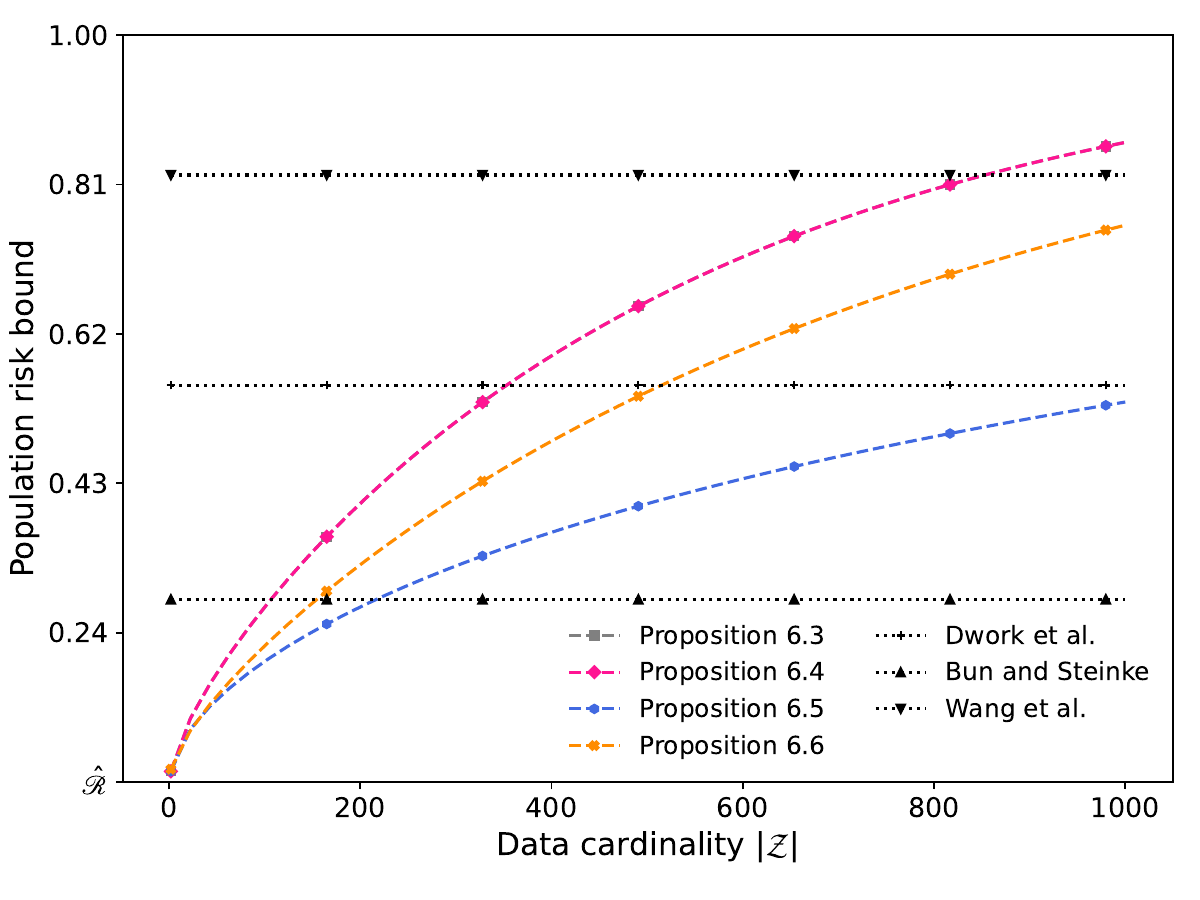}
    \caption{Comparison of the different expected generalization bounds for $\varepsilon$-DP algorithms in the setting of~\Cref{example:icu}. In particular, we obtain the guarantees combining \Cref{prop:simple,prop:mi_mGDP,prop:mi_mGDP_2,prop:mi_mGDP_3} and the results from~\citet{dwork2015generalization} and~\citet{bun2016concentrated} with \Cref{lemma:small_kl_mi},  and the result from~\citet{JMLR:v17:15-313} with the bounds in expectation from uniform stability of~\Cref{subsec:uniform_stability}. In the top-left figure, the cardinality of the data is fixed to $|\cZ| = 100$, the probability parameter is fixed to $\beta=0.05$, and the empirical risk is $\bE^{S}[\emprisk(W,S)] = 0.05$. In the left figure, the privacy parameter is fixed to $\varepsilon=0.1$ and the number of samples $n$ varies from $1$ to $10,000$. In the right figure, the number of samples is fixed to $n = 2,500$ and the privacy parameter $\varepsilon$ varies from $10^{-4}$ to $1$. In both figures, the $y$-axis starts at the value of the empirical risk.}
    \label{fig:comparison_propositions_icu_minf}
\end{figure}
 
\begin{subappendices}

\section{An Upper Bound on the Number of Types}
\label[appendix]{app:proof_claim_num_types}

The exact number of types is given by \cite[Problem~2.1]{csiszar_2011_information}
\begin{equation*}
	|\cT_{\cZ,n}| = \binom{n+|\cZ|-1}{|\cZ|-1}.
\end{equation*}
Then, we may bound $|\cT_{\cZ,n}|$ from above as follows,
\begin{align*}
	|\cT_{\cZ,n}| &= \frac{(n+|\cZ|-1)!}{(|\cZ|-1)!\,n!} \\
	&= \prod\nolimits_{i=1}^{|\cZ|-1} \left(1 + \frac{n}{i} \right)  \\
	&\leq \prod\nolimits_{i=1}^{|\cZ|-1} (1 + n)  \\
	&= (n+1)^{|\cZ|-1},
\end{align*}
with equality if and only if $|\cZ| = 2$.

\section{Upper Bound on the Relative Entropy Between a Distribution and a Mixture of Distributions}
\label{app:proof_lemma_mixture}

This section of the appendix is devoted to proving~\Cref{lemma:mixture_kl_ub}. 
We start the proof by defining the restricted measures $\tilde{\bQ}$ and $\tilde{\bQ}_b$ to be $\bQ$ and $\bQ_b$ in the support of $\bP$ and $0$ everywhere else. That is, 
\begin{equation*}
    \tilde{\bQ} \coloneqq \sum_b \omega_b \tilde{\bQ}_b = \sum_b \omega_b \bQ_b \cdot \bI_{\mathrm{supp}(\bP)}.
\end{equation*}

We note that these definitions ensure that $\tilde{\bQ} \ll \bP$ and $\tilde{\bQ}_b \ll \bP$ for all $b$, while maintaining the property that $\bP \ll \tilde{\bQ}$ and $\bP \ll \tilde{\bQ}_b$ for all $b$, since $\mathrm{supp}(\tilde{\bQ}) = \mathrm{supp}(\tilde{\bQ}_b) = \mathrm{supp}(\bP)$ for all $b$. In this way, we can manipulate the expression of the relative entropy as follows:
\begin{align*}
    \relent(\bP \Vert \bQ) &= \bE_{x \sim \bP} \bigg[ \log \frac{\rmd \bP}{\rmd \bQ}(x) \bigg]  \\
    &\stack{a}{=} \bE_{x \sim \bP} \Bigg[ \log \frac{\rmd \bP}{\rmd\tilde{\bQ}}(x) \Bigg]  \\ 
    &\stack{b}{=} - \bE_{x \sim \bP} \Bigg[ \log \frac{\rmd\tilde{\bQ}}{\rmd \bP}(x) \Bigg] &\\
    &= - \bE_{x \sim \bP} \Bigg[ \log \frac{\rmd \big( \sum_b \omega_b \tilde{\bQ}_b \big) }{\rmd \bP} (x)\Bigg]  \\
    &\stack{c}{=} - \bE_{x \sim \bP} \Bigg[ \log \Bigg( \sum_b \omega_b \frac{\rmd \tilde{\bQ}_b}{\rmd \bP}(x) \Bigg) \Bigg], 
\end{align*}
where (a) stems from the fact that the expectation will integrate $\nicefrac{\rmd \bP}{\rmd \bQ}$ over the union of sets of the support of $\bP$, where $\nicefrac{\rmd \bP}{\rmd \bQ} = \nicefrac{\rmd \bP}{\rmd \tilde{\bQ}}$ ($\bP$-a.s.), and (b) and (c) stem from \cite[Exercise 9.27]{mcdonald1999course}.

Now, if we consider a set of positive coefficients $\lbrace \phi_b \rbrace_b$ such that $\sum_b \phi_b = 1$ we have that
\begin{align}
    \relent(\bP \Vert \bQ) &= - \bE_{x \sim \bP} \Bigg[ \log \Bigg( \sum_b \phi_b \frac{\omega_b}{\phi_b} \frac{\rmd \tilde{\bQ}_b}{\rmd \bP}(x) \Bigg) \Bigg] \nonumber \\
    &\leq - \bE_{x \sim \bP} \Bigg[ \sum_b \phi_b \log \Bigg( \frac{\omega_b}{\phi_b} \frac{\rmd \tilde{\bQ}_b}{\rmd \bP} \Bigg) \Bigg] \nonumber \\
    &\coloneqq \hat{\rmD}_{\textnormal{KL}}(\bP\, \|\, \tilde{\bQ}\, ;\, \lbrace \phi_b \rbrace_b),
    \label{eq:approx_var_kl}
\end{align}
where the inequality stems from the convexity of $- \log$ and Jensen's inequality. 

We can now tighten the last inequality by minimizing the convex function $\hat{\rmD}_{\textnormal{KL}}(\bP\, \|\, \tilde{\bQ}\, ;\, \lbrace \phi_b \rbrace_b)$ over the linear {constraint} $\sum_b \phi_b = 1$ with the Lagrangian $\cL(\lbrace \phi_b \rbrace_b, \lambda) = \hat{\rmD}_{\textnormal{KL}}(\bP\, \|\, \tilde{\bQ}\, ;\, \lbrace \phi_b \rbrace_b) + \lambda(\sum_b \phi_b -1)$. The optimal value of the Lagrangian is given by
\begin{equation}
    \frac{\partial \cL(\lbrace \phi_b \rbrace_b, \lambda)}{\partial \phi_b} = 0 \Leftrightarrow \phi_b^\star = \frac{\omega_b}{\exp(\lambda + 1)} \exp \big( {-} \relent(\bP \Vert \tilde{\bQ}_b) \big), \nonumber
\end{equation}
where we use the fact that $\nicefrac{\rmd\tilde{\bQ}_b}{\rmd\bP} = \left(\nicefrac{\rmd\bP}{\rmd\tilde{\bQ}_b}\right)^{-1}$ \cite[Exercise 9.27]{mcdonald1999course} and that $\bP \ll \tilde{\bQ}_b$ for all $b$.
Then, each $\phi_b^\star$ satisfies the linear {constraint} when
\begin{equation}
    \phi_b^\star = \frac{\omega_b \exp \big({-} \relent(\bP \Vert \tilde{\bQ}_b) \big)}{\sum_{b'} \omega_{b'} \exp \big({-} \relent(\bP \Vert \tilde{\bQ}_{b'}) \big)}. \nonumber
\end{equation}
Therefore, we can recover~\eqref{eq:lemma_ub_1} as follows:
\begin{align}
    \MoveEqLeft[2]
    \hat{\rmD}_{\textnormal{KL}}(\bP\, \|\, \tilde{\bQ}\, ;\, \lbrace \phi_b^\star \rbrace_b) \nonumber \\
    &= -\bE_{x \sim \bP} \left[\sum_b \phi_b^\star \left( \log \left( \omega_b \frac{\rmd \tilde{\bQ}_b}{\rmd\bP}(x) \right) - \log \phi_b^\star \right) \right]  \nonumber \\
    &= - \bE_{x \sim \bP} \left[\sum_b \phi_b^\star \left( \log \left( \omega_b \frac{\rmd \tilde{\bQ}_b}{\rmd \bP}(x) \right) \right. \right. \nonumber \\
    & \qquad \qquad \qquad \left. \left. - \log \omega_b + \relent(\bP \Vert \tilde{\bQ}_b) +  \log \left(\sum_{b'}\omega_{b'} \exp \big( {-} \relent(\bP \Vert \tilde{\bQ}_b) \big) \right)  \right) \right]  \nonumber \\
    & = -\log \left(\sum_{b'}\omega_{b'} \exp \big( {-} \relent(\bP \Vert \tilde{\bQ}_{b'}) \big) \right), \label{eq:lemma_ub_1_proof}
\end{align}
where the last equality comes from the fact that $\sum_b \phi_b^\star = 1$, the fact that in $\relent(\bP \Vert \tilde{\bQ}_{b})$ the expectation with respect to $\bP$ will integrate $\log \frac{\rmd \bP}{\rmd\tilde{\bQ}_{b}}$ over the support of $\bP$, and the claim that
\begin{equation*}
    \bE_{x \sim \bP} \left[\sum_b \phi_b^\star \left( \log \left( \omega_b \frac{\rmd \tilde{\bQ}_b}{\rmd \bP}(x) \right) - \log \omega_b + \relent(\bP \Vert \tilde{\bQ}_b) \right ) \right] = 0,
\end{equation*}
or equivalently, 
\begin{align*}
    \bE_{x \sim \bP} \left[\sum_b \phi_b^\star \left( \log \left( \omega_b \frac{\rmd \tilde{\bQ}_b}{\rmd \bP}(x) \right) \right) \right] = \sum_b \phi_b^\star \left(  \log \omega_b - \relent(\bP \Vert \tilde{\bQ}_b) \right),
\end{align*}
which stems again from the fact that $\nicefrac{\rmd\tilde{\bQ}_b}{\rmd \bP} = \left(\nicefrac{\rmd\bP}{\rmd\tilde{\bQ}_b}\right)^{-1}$ \cite[Exercise 9.27]{mcdonald1999course} and that $\bP \ll \tilde{\bQ}_b$ for all $b$. 
Finally, we can leverage the log-sum-exp bounds in~\eqref{eq:approx_var_kl} and~\eqref{eq:lemma_ub_1_proof} to obtain a more comprehensible upper bound on the relative entropy:
\begin{align*}
	\relent(\bP \Vert \bQ) &\leq  - \log \left( \sum_{b} \omega_{b} \exp \big( {-} \relent(\bP \Vert \bQ_b) \big) \right)  \\
	&= - \log \left( \sum_{b} \exp \big( {-} \relent(\bP \Vert \bQ_b) + \log (\omega_b) \big) \right)  \\
	&\leq - \log \exp \max_b \big\lbrace {-} \relent(\bP \Vert \bQ_b) + \log \omega_b \big\rbrace \\
	&\leq \min_b \big\lbrace \relent(\bP \Vert \bQ_b) - \log \omega_b \big \rbrace,
\end{align*}
which is exactly the bound~\eqref{eq:lemma_ub_2}.

\section{Proof of \texorpdfstring{\Cref{prop:mi_mGDP}}{Proposition 6.4}}
\label[appendix]{app:proof_prop_mGDP_1}

\looseness=-1 The proof takes advantage of the property that a private algorithm, for example, $\varepsilon$-DP or $\mu$-GDP, produces statistically similar outputs for neighbouring training sets.
As in the proof of \Cref{prop:simple}, we employ \Cref{lemma:mixture_kl_ub} to obtain an upper bound on $\relent(\bP_W^{S=s} \Vert \bQ)$ but, this time, the mixture distribution $\bQ$ only uses some $\bP_W^{T_s}$.

We start by noting that all the types $T_s$ lie inside the unit hypercube $[0,1]^{|\cZ|-1}$. Although $T_s$ is a vector of $|\cZ|$ dimensions, since the last dimension is completely defined by the preceding $|\cZ|-1$, its possible values are located in a $(|\cZ|-1)$-dimensional subspace. This is the intuition behind \Cref{claim:num_types}. 

To simplify the analysis, instead of studying the types, we focus on the vector of counts $N_s$, which is a scaled version of the type, that is, $N_s(z) = N(z|s)$ for all $z \in \cZ$. As with the types, the first $|\cZ|-1$ dimensions completely define the vector of counts, and thus $N_s$ lies in a $[0,n]^{|\cZ|-1}$ hypercube.

We split the $[0,n]$ interval of each dimension in $1 \leq t\leq n$ parts, thus resulting in a $[0,n]^{|\cZ|-1}$ hypercube cover of $t^{|\cZ|-1}$ smaller hypercubes.\footnote{We note that there are only $n+1$ coordinates in every dimension to choose as the center of a small hypercube. If we were to choose $t = n+1$, then we would have one hypercube per type, and thus we would not exploit the fact that the algorithm is smooth.}
The types defined by the first $|\cZ|-1$ components of the vector of counts at the center of these smaller hypercubes are the ones selected to create the mixture (see \Cref{fig:hypercubes_simplex_prop2} for an illustration).
We note that the side of every small hypercube has length 
\begin{equation*}
	l \coloneqq \frac{n}{t}
\end{equation*}
and, more importantly, it has $l' \coloneqq \lfloor l \rfloor + 1$ atoms.

\begin{figure}[t]
 \centering
 \includegraphics{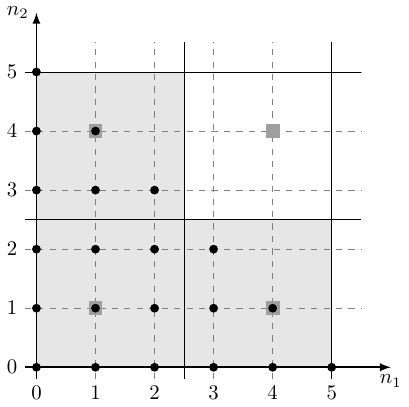}
 \caption{Example for $n=5$ samples, $|\cZ|=3$ dimensions, and $t=2$ parts per dimension.
 Here, $n_1$ and $n_2$ represent the counts of the first and second letters of the alphabet $\cZ$, that is, $n_i = N_s(z_i)$.
 The black dots represent the possible training sets and the gray squares, the considered training sets for the mixture.
 The highlighted area represents the hypercubes that contain possible training sets.
\label{fig:hypercubes_simplex_prop2}}
\end{figure}

If $l'$ is odd, we choose the central atom as the corresponding coordinate for the center of the small hypercube. On the other hand, if $l'$ is even, we enlarge the small hypercube in one unit and choose the center of this enlarged hypercube. Hence, the distance between the center and any other atom in the small hypercube, in only one of the dimensions is, at most,
\begin{equation*}
	|N_s(z_i) - N_{s'}(z_i)| \leq \frac{\lfloor l \rfloor + 1}{2} \leq \frac{n}{2t} + \frac{1}{2}
\end{equation*}
for all $i \in [|\cZ| - 1]$, where $z_i$ is the $i$-th element of $\cZ$. Therefore, for the first $|\cZ|-1$ components, the maximal distance between any vector of counts and the vector of counts at the center of the small hypercube is bounded by $(|\cZ|-1) \frac{\nicefrac{n}{t} + 1}{2}$. In the worst case, all the counts in these first dimensions are off the center with the same sign (for example, $N_s(z_i) - N_{s'}(z_i) = \frac{n/t + 1}{2}$ for all $i \in \big[ |\cZ|-1 \big]$) and the last dimension (not directly included, but existent in the vector of counts) has to compensate for it. Hence, the distance on the last dimension of the vector of counts is bounded from above by
\begin{equation*}
	|N_s(z_{|\cZ|}) - N_{s'}(z_{|\cZ|})| \leq \left(\frac{n}{2t} + \frac{1}{2}\right)(|\cZ|-1).
\end{equation*}

The covering devised this way makes sure that the distance $d(s,s_i)$, as described in \Cref{def:dist}, between any training $s$ belonging to the $i$-th hypercube and its center $s_i$ is upper-bounded as follows,
\begin{equation}
     d(s,s_i) \leq \left(\frac{n}{2t} + \frac{1}{2}\right)(|\cZ|-1) \leq \frac{n}{t}(|\cZ| - 1),
     \label{eq:d_max}
\end{equation}
\looseness=-1 since $t \leq n$.
We then replicate the proof of \Cref{prop:simple} from~\eqref{eq:prop1_bnd1} using this new mixture and the properties of $\varepsilon$-DP and $\mu$-GDP algorithms, that is, \Cref{claim:diff_priv_distance}.

\subsection{\texorpdfstring{$\varepsilon$}{e}-DP Algorithms}

We have that
\begin{align}
 \relent(\bP_W^{S=s} \Vert \bQ) &\leq \min_{s' \in \cS'} \left \lbrace \relent \big(\bP_W^{S=s} \Vert \bP_W^{T_{s'}}\big) - \log \omega_{s'} \right \rbrace \nonumber\\
 &\leq \frac{n}{t} (|\cZ|-1) \varepsilon + (|\cZ|-1) \log t,
 \label{eq:prop2_bnd1_dp}
\end{align}
where the second inequality follows from~\eqref{eq:dp_distance} in \Cref{claim:diff_priv_distance}, \eqref{eq:d_max}, and the fact that there are at most $t^{|\cZ|-1}$ smaller hypercubes.
The value of $t$ that minimizes~\eqref{eq:prop2_bnd1_dp} is
\begin{equation*}
 t = \varepsilon n,
\end{equation*}
and the following bound is obtained after replacing this value into~\eqref{eq:prop2_bnd1_dp},
\begin{equation}
	\relent(\bP_W^{S=s} \Vert \bQ) \leq (|\cZ|-1) \log \left( e \varepsilon n \right).
 \label{eq:prop2_bnd2_dp}
\end{equation}
However, this result is only meaningful if the condition $1 \leq t \leq n$, as mentioned earlier in the proof, holds true, i.e., $1/n \leq \varepsilon \leq 1$. 

On the one hand, if the optimal $t$ is such that $t > n$, we still need to choose $t=n+1$, as the maximum covering that can be designed selects one hypercube per type.
\Cref{prop:simple} has already addressed this situation and its result is tighter than~\eqref{eq:prop2_bnd1_dp}, with $t=n+1$, because of the loose upper bound~\eqref{eq:d_max}.
On the other hand, if the optimal $t$ is such that $t< 1$, we still need to choose $t=1$; a smaller $t$ implies that we are covering a volume larger than $[0,n]^{|\cZ|-1}$. The following upper bound is found in this way,
\begin{equation}
 \relent(\bP_W^{S=s} \Vert \bQ) \leq  (|\cZ| - 1) \varepsilon n
 \label{eq:prop2_bnd3_dp}
\end{equation}
and it is worse than the simple bound from the beginning of~\Cref{sec:differential_privacy_generalization}. 

We may further combine the bounds~\eqref{eq:prop2_bnd2_dp} and~\eqref{eq:prop2_bnd3_dp} into the more compact, albeit looser bound~\eqref{eq:prop_eDP_1_main}. We note that $x<\log(1+ex)$ if $x\leq 1$, which is equivalent to \eqref{eq:prop2_bnd3_dp}$<$\eqref{eq:prop_eDP_1_main} if $\varepsilon n < 1$, precisely the region where~\eqref{eq:prop2_bnd3_dp} is valid.

\subsection{\texorpdfstring{$\mu$}{m}-GDP Algorithms}

If we leverage~\eqref{eq:gdp_distance} from \Cref{claim:diff_priv_distance} instead of~\eqref{eq:dp_distance}, we now have that
\begin{align}
 \relent(\bP_W^{S=s} \Vert \bQ) 
 &\leq \min_{s' \in \cS'} \left \lbrace \relent\big(\bP_W^{S=s} \Vert \bP_W^{T_{s'}}\big) - \log \omega_{s'} \right \rbrace \nonumber\\
 &\leq \frac{1}{2} \frac{N^2}{t^2} (|\cZ|-1)^2 \mu^2 + (|\cZ|-1) \log t,
 \label{eq:prop2_bnd1_gdp}
\end{align}
where the value of $t$ that minimizes~\eqref{eq:prop2_bnd1_gdp} is now
\begin{equation*}
 t = \sqrt{|\cZ|-1}\, \mu n,
\end{equation*}
which yields
\begin{equation}
 \relent(\bP_W^{S=s} \Vert \bQ) \leq (|\cZ|-1) \log \left( \sqrt{ e (|\cZ| - 1)}\, \mu n \right).
 \label{eq:prop2_bnd2_gdp}
\end{equation}

If we operate analogously as with $\varepsilon$-DP algorithms, we obtain that when $\mu \leq \nicefrac{1}{\sqrt{|\cZ|-1}}$, then $\relent(\bP_W^{S=s} \Vert \bQ)$ is bounded from above by~\eqref{eq:prop_mGDP_1_main}, and that otherwise \Cref{prop:simple} is tighter.
This concludes the proof of \Cref{prop:mi_mGDP}.

\begin{remark}
A keen reader might notice that the upper bound on the relative entropy in~\eqref{eq:prop2_bnd1_dp} is not the tighter formula from~\eqref{eq:dp_distance} using $\tanh$.
Instead, we chose the linear upper bound for simplicity.

As mentioned in \Cref{subsec:dp}, this approximation is better than a quadratic one if $\varepsilon\geq \nicefrac{2}{k}$.
Had we chosen this other approximation, the result for $\varepsilon$-DP would be the same as the one for $\mu$-GDP.
Comparing~\eqref{eq:prop2_bnd2_dp} and~\eqref{eq:prop2_bnd2_gdp}, with $\varepsilon$ instead of $\mu$, we see that~\eqref{eq:prop2_bnd2_dp} is looser whenever $|\cZ|\leq 1+e$, that is, only for binary and ternary alphabets.

For typical training sets, where $|\cZ|$ could be in the hundreds or more, the linear approximation is thus very tight.
Therefore, in the remaining proofs, we use the linear approximation in~\eqref{eq:dp_distance} for the relative entropy between outputs of $\varepsilon$-DP algorithms.
\end{remark}

\section{Proof of \texorpdfstring{\Cref{prop:mi_mGDP_2}}{Proposition 6.5}}
\label[appendix]{app:proof_prop_mGDP_2}

In the proof of \Cref{prop:mi_mGDP}, we devised a covering of the whole space $[0,n]^{|\cZ|-1}$ with $t^{|\cZ|-1}$ small hypercubes.
However, we observe that many of the hypercubes we designed contain no vector of counts. Particularly, there are only counts inside the hypervolume comprised between the origin and the $(|\cZ|-2)$-simplex; in the example from \Cref{fig:hypercubes_simplex_prop2}, this volume is inside the highlighted area.
This simplex defines the manifold where the first $|\cZ|-1$ dimensions of the vector of counts sum up to $n$ in the $[0,n]^{|\cZ|-1}$ hypercube. Therefore, there are no possible vectors of counts above it since any training set is restricted to $n$ samples; however, there are many vectors below it since the $|\cZ|$-th dimension ``compensates'' for the unseen samples in the first $|\cZ|-1$ dimensions. 

For this reason, we only keep the hypercubes strictly needed to cover the hypervolume under the $(|\cZ|-2)$-simplex. This hypervolume is a hyperpyramid of height $n$ and its base is the hypervolume under the $(|\cZ|-3)$-simplex. Moreover, this hypervolume has $|\cZ|-1$ perpendicular edges of length $n$ which intersect at the origin.
For example, the hypervolume under the $1$-simplex is the right triangle with a vertex in the origin and two edges that go from $0$ to $n$ on both axes. The number of hypercubes needed to cover the hypervolume under the $(|\cZ|-2)$-simplex is given in the following lemma, whose proof is in \Cref{app:proof_lemma_hypercubes}.

\begin{lemma}
\label{lemma:number_of_hypercubes_under_simplex}
The minimum number of hypercubes of a regular $t^{\times k}$ grid on the $[0,n]^K$ hypercube that covers the hypervolume under the $(k-1)$-simplex is
\begin{equation}
    S_k(t) = \frac{1}{k!} \frac{(t+k-1)!}{(t-1)!} \leq \frac{1}{k!} \left(t+\frac{k-1}{2}\right)^{k}.
    \label{eq:number_of_hypercubes_under_simplex}
\end{equation}
\end{lemma}

From \Cref{lemma:number_of_hypercubes_under_simplex} we know that we only need 
\begin{align}
    S_{|\cZ|-1}(t) &= \frac{1}{(|\cZ|-1)!} \frac{(t + |\cZ| - 2)!}{(t-1)!} \nonumber \\
    &\leq \frac{1}{(|\cZ|-1)!} \left(t + \frac{|\cZ|-2}{2} \right)^{|\cZ|-1}
    \label{eq:min_number_hypercubes}
\end{align}
hypercubes to cover all the possible vectors of counts in the $[0,n]^{|\cZ|-1}$ hypercube instead of the $t^{|\cZ|-1}$ used in \Cref{prop:mi_mGDP}.
We then replicate the proof of \Cref{prop:simple} from~\eqref{eq:prop1_bnd1} using the properties of $\varepsilon$-DP and $\mu$-GDP algorithms, as we did in \Cref{prop:mi_mGDP}.

\subsection{\texorpdfstring{$\varepsilon$}{e}-DP Algorithms}

If we use the bound~\eqref{eq:dp_distance} in \Cref{claim:diff_priv_distance}, \eqref{eq:d_max}, and the fact that the number of smaller hypercubes is bounded by~\eqref{eq:min_number_hypercubes}, we obtain that
\begin{align}
\relent(\bP_W^{S=s} \Vert \bQ) &\leq \min_{s' \in \cS'} \left \lbrace \relent \big(\bP_W^{S=s} \Vert \bP_W^{T_{s'}} \big) - \log \omega_{s'} \right \rbrace \nonumber\\
 &\leq \frac{N}{t} (|\cZ|-1) \varepsilon + (|\cZ|-1) \log \left(t + \frac{|\cZ|-2}{2} \right)  - \log (|\cZ|-1)!.
 \label{eq:prop3_bnd1_dp}
\end{align}
The analytical expression for the value of $t$ that minimizes~\eqref{eq:prop3_bnd1_dp} is quite convoluted, so we study solutions of the form $t=\alpha \varepsilon N$, $\alpha > 0$, based on the results from \Cref{prop:mi_mGDP}.
Then, in order to meet the condition $1 \leq t \leq n$, we need that $\nicefrac{1}{(\alpha N)} \leq \varepsilon \leq \nicefrac{1}{\alpha}$. 
Furthermore, we bound from above the last term in~\eqref{eq:prop3_bnd1_dp} using Stirling's formula, that is, 
\begin{align}
-\log(|\cZ|-1)! \leq -\frac{1}{2} \log 2 \pi (|\cZ|-1) \quad- (|\cZ|-1) \log(|\cZ|-1) + (|\cZ|-1).
    \label{eq:prop3_stirling}
\end{align}
This way, we obtain the upper bound 
\begin{align}
   \relent(\bP_W^{S=s} \Vert \bQ) 
    &\leq (|\cZ|-1) \log \left( \frac{\alpha e^{\frac{1}{\alpha}+1}}{|\cZ|-1} \varepsilon N + \frac{1}{2} \frac{|\cZ|-2}{|\cZ|-1} e^{\frac{1}{\alpha}+1} \right)  -\frac{1}{2} \log 2 \pi (|\cZ|-1) \nonumber \\ 
    &\leq (|\cZ|-1) \log \left( \frac{e^{\frac{1}{\alpha}+1}}{2} \left(1 + \frac{2\alpha}{|\cZ| -1} \varepsilon N \right) \right)  -\frac{1}{2} \log 2 \pi (|\cZ|-1).
    \label{eq:prop3_alpha_bnd_dp}
\end{align}

Although it is possible to obtain an analytical formula for the optimal value of $\alpha$ in~\eqref{eq:prop3_alpha_bnd_dp}, we obtain no insights from it.
Instead, we provide a suboptimal value by analyzing the argument of the logarithm, in particular its derivative with respect to $\alpha$.
We observe that, as $\nicefrac{\varepsilon N}{(|\cZ|-1)}$ increases, the minimizing $\alpha$ tends to $1$ from above; if $\varepsilon n= |\cZ|-1$, then $\alpha \approx 1.37$ is optimal, which is already quite close to the limit.
Therefore, we obtain~\eqref{eq:prop_eDP_2_main} by setting $\alpha$ to 1.

Having set $\alpha=1$, if the optimal $t\ (=\varepsilon N)$ is such that $t < 1$, we fix $t = 1$ in~\eqref{eq:prop3_bnd1_dp} as explained in the proof of \Cref{prop:mi_mGDP}.
After some algebraic manipulations, we obtain the bound~\eqref{eq:prop_eDP_2_main_0}.

\subsection{\texorpdfstring{$\mu$}{m}-GDP Algorithms}

If we leverage~\eqref{eq:gdp_distance} in \Cref{claim:diff_priv_distance} instead of~\eqref{eq:dp_distance}, we have that
\begin{align}
\relent(\bP_W^{S=s} \Vert \bQ) 
 &\leq  \min_{s' \in \cS'} \left \lbrace \relent \big(\bP_W^{S=s} \Vert \bP_W^{T_{s'}} \big) - \log \omega_{s'} \right \rbrace \nonumber\\
 &\leq \frac{1}{2} \frac{N^2}{t^2} (|\cZ|-1)^2 \mu^2 + (|\cZ|-1) \log \left(t + \frac{|\cZ|-2}{2} \right)  - \log (|\cZ|-1)!.
 \label{eq:prop3_bnd1_gdp}
\end{align}
As for $\varepsilon$-DP algorithms, the analytical expression for the value of $t$ that minimizes~\eqref{eq:prop3_bnd1_gdp} is quite convoluted, so we study solutions of the form  $t=\alpha \sqrt{|\cZ|-1}\, \mu N$, based on the results from \Cref{prop:mi_mGDP}.
Then, in order to meet the condition $1 \leq t \leq n$, we need that
\begin{equation*}
    \frac{1}{\alpha n \sqrt{|\cZ|-1}} \leq \mu \leq \frac{1}{\alpha \sqrt{|\cZ|-1}},
\end{equation*}
and we obtain the upper bound
\begin{align}
\MoveEqLeft[1]
   \relent(\bP_W^{S=s} \Vert \bQ) \nonumber \\
    &\leq (|\cZ|-1) \log \left( \frac{\alpha e^{\frac{1}{2\alpha^2}+1}}{\sqrt{|\cZ|-1}} \mu n + \frac{1}{2} \frac{|\cZ|-2}{|\cZ|-1} e^{\frac{1}{2\alpha^2}+1} \right)  -\frac{1}{2} \log 2 \pi (|\cZ|-1) \nonumber \\ 
    &\leq (|\cZ|-1) \log \left( \frac{e^{\frac{1}{2\alpha^2}+1}}{2} \left( 1 + \frac{2\alpha}{\sqrt{|\cZ|-1}} \mu n \right) \right) -\frac{1}{2} \log 2 \pi (|\cZ|-1).
    \label{eq:prop3_alpha_bnd_gdp}
\end{align}

We analyze again the derivative with respect to $\alpha$ of the argument of the logarithm in~\eqref{eq:prop3_alpha_bnd_gdp}. We now observe that, as $\nicefrac{\mu n}{\sqrt{|\cZ|-1}}$ increases, the minimizing $\alpha$ also tends to 1 from above; if $\mu n = \sqrt{|\cZ| -1}$, $\alpha \approx 1.19$ is optimal, which is already quite close to the limit. Therefore, we obtain~\eqref{eq:prop_mGDP_2_main} by setting $\alpha$ to 1.

If the optimal value of $t$ is such that $t < 1$, we proceed similarly as with $\varepsilon$-DP algorithms and we obtain the remaining bound~\eqref{eq:prop_mGDP_2_main_0}.

\subsection{General Algorithms}

For both $\varepsilon$-DP and $\mu$-GDP, if the privacy guarantees are not good enough, the minimizations in~\eqref{eq:prop3_bnd1_dp} and~\eqref{eq:prop3_bnd1_gdp} result in an optimal value of $t> n$.
As explained in the proof of \Cref{prop:mi_mGDP}, we choose $t=n+1$ as, otherwise, we would have more smaller hypercubes than types.
In this case, there is one type per hypercube and we may replicate the proof of \Cref{prop:simple} and continue from~\eqref{eq:prop1_bnd1}.
That is,
\begin{align*}
\relent(\bP_W^{S=s} \Vert \bQ) 
 &\leq \min_{s' \in \cS} \left \lbrace \relent \big(\bP_W^{S=s} \Vert \bP_W^{T_{s'}} \big) - \log \omega_{s'} \right \rbrace \\
 &\stack{a}{=} \log S_{|\cZ|-1}(n+1) \\
 &\leq (|\cZ|-1) \log \left(n+1 + \frac{|\cZ|-2}{2} \right) - \log (|\cZ|-1)!,
\end{align*}
where $(a)$ is due to $\omega_s = |\cS|^{-1} = \big(S_{|\cZ|-1}(n+1)\big)^{-1}$ for all $s \in \cS$. 
After replacing the factorial with Stirling's approximation~\eqref{eq:prop3_stirling} and performing some algebraic manipulations, we obtain the bound~\eqref{eq:prop_eDP_2_main_2}.
This concludes the proof of \Cref{prop:mi_mGDP_2}.

\subsection{Proof of \texorpdfstring{\Cref{lemma:number_of_hypercubes_under_simplex}}{Lemma 6.2}}
\label[appendix]{app:proof_lemma_hypercubes}

\begin{figure}[t]
 \centering
 \includegraphics{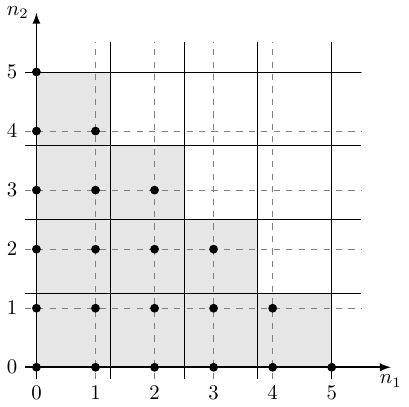}
 \caption{Example for $n=5$ samples, $k=2$ dimensions, and $t=4$ parts per dimension. Here, $n_1$ and $n_2$ are the dimensions of the $[0,5]^{2}$ hypercube. The number of highlighted blocks is $S_{2}(4)=10$. 
 \label{fig:hypercubes_simplex}}
\end{figure}

We start the proof by observing that, if the segment $[0,n]$ is divided in $t$ parts, we need all $t$ parts to cover it, that is, $S_1(t) = t$.
Then, we note that the number of squares from a regular $t \times t$ grid needed to cover the area under the $1$-simplex is given by
\begin{equation*}
    S_2(t) = \sum_{j=1}^t (t-j+1) = \frac{1}{2} t(t+1),
\end{equation*}
since we need all vertical squares in the first column, and one less for every additional column, see \Cref{fig:hypercubes_simplex} for a visual example.
We observe that it is possible to rewrite this equation as
\begin{equation*}
    S_2(t) = \sum_{j=1}^t j = \sum_{j=1}^t S_1(j).
\end{equation*}
This formulation is quite intuitive since the base of the pyramid has $S_1(t)$ blocks, the second level has $S_1(t-1)$ blocks, and so on, until we reach the level $t$ which has $S_1(1)$ blocks.
Furthermore, this recursion applies to any dimension, that is,
\begin{equation}
    S_k(t) = \sum_{j=1}^t S_{k-1}(j).
    \label{eq:hypercubes_simplex_2}
\end{equation}
For example, the number of cubes needed to cover the volume under the $2$-simplex is given by
\begin{equation*}
    S_3(t) = \sum_{j=1}^t \frac{1}{2} j(j+1) = \frac{1}{6} t(t+1)(t+2).
\end{equation*}

In what follows, we show that the recursion in~\eqref{eq:hypercubes_simplex_2} may be directly computed as $S_k(t)=f_k(t)$, where
\begin{equation}
    f_k(t) \coloneqq \frac{1}{k!} \frac{(t+k-1)!}{(t-1)!}.
 \label{eq:hypercubes_simplex_3}
\end{equation} 
We prove this by induction.
We have already shown that $S_k(t)=f_k(t)$ for $k=1$, $k=2$, and $k=3$. Now, assuming it is true for $k-1$, we show it also holds for $k$.
First, we note the following equality

\begin{equation}
 f_{k-1}(t) = \begin{cases}
               f_k(t) - f_k(t-1) & \textnormal{if } t>1,\\
               1 & \textnormal{if } t=1,
              \end{cases}
 \label{eq:hypercubes_simplex_4}
\end{equation}
which may be obtained by some simple algebraic manipulations.
Then, assuming $t>1$,
\begin{align*}
    S_k(t) &= \sum_{j=1}^t S_{k-1}(j) \\
    &\stack{a}{=} \sum_{j=1}^t f_{k-1}(j) \\
    &\stack{b}{=} 1+ \sum_{j=2}^t f_k(j) - f_k(j-1) \\
    &= f_k(t),
\end{align*}
where $(a)$ comes from the induction assumption, and $(b)$ is due to~\eqref{eq:hypercubes_simplex_4}.
The case $t=1$ follows trivially since $f_k(1)=1$ for any $k \in \mathbb{Z}_+$; thus, the first part of~\eqref{eq:number_of_hypercubes_under_simplex} is proved.

The upper bound on $S_k(t)$, the second part of~\eqref{eq:number_of_hypercubes_under_simplex}, stems from the following small result, whose proof is in~\Cref{app:proof_factorials_to_exponential}.

\begin{proposition}
\label{claim:factorials_to_exponential}
For any $t,m \in \mathbb{Z}_+$,
\begin{equation}
    \frac{(t+m)!}{(t-1)!} \leq \left(t+\frac{m}{2}\right)^{m+1}{.}
    \label{eq:claim_statement}
\end{equation}
\end{proposition}

We may then proceed to bound the corresponding
{factor} in~\eqref{eq:hypercubes_simplex_3} with~\eqref{eq:claim_statement}, where we set $m=k-1$. This concludes the proof of \Cref{lemma:number_of_hypercubes_under_simplex}.

\subsection{Proof of \Cref{claim:factorials_to_exponential}}
\label[appendix]{app:proof_factorials_to_exponential}

We start the proof by noting the following upper bound, that is valid for every $t,a,b\in\bR$:
\begin{equation}
    (t+a)(t+b)\leq \left(t +\frac{a+b}{2}\right)^{2},
    \label{eq:bound_prod_center}
\end{equation}
which is a simple consequence of the equality
\begin{equation*}
    (t+a)(t+b)+\frac{(a-b)^2}{4} = \left(t +\frac{a+b}{2}\right)^{2}.
\end{equation*}

Therefore, if $m = 2l$, where $l \in \mathbb{Z}_+$, there is an odd number of %
{factors} ($2l+1$) on the left-hand side of~\eqref{eq:claim_statement}. We may pair all the %
{factors} that are equidistant from the center, with the exception of the middle one, and bound them using~\eqref{eq:bound_prod_center} to obtain
\begin{equation*}
    (t+l)^{2l}(t+l) = \left(t+\frac{m}{2}\right)^{m+1}.
\end{equation*}
On the other hand, if $m = 2l + 1$, where $l \in \mathbb{Z}_+$, there is now an even number of %
{factors} ($2l+2$) on the left-hand side of~\eqref{eq:claim_statement}. 
We may nonetheless pair all the %
{factors} that are equidistant from the center and, by using~\eqref{eq:bound_prod_center}, also obtain the bound
\begin{equation*}
    \left(t + \frac{2l +1}{2}\right)^{2l+2} = \left(t+\frac{m}{2}\right)^{m+1}.
\end{equation*}

\section{Proof of \texorpdfstring{\Cref{prop:mi_mGDP_3}}{Proposition 6.6}}
\label[appendix]{app:proof_prop_mGDP_3}

In the proofs of \Cref{prop:mi_mGDP,prop:mi_mGDP_2}, we designed a covering of the whole space of training sets to approximate the marginal distribution of the output hypothesis.
However, as the number of samples $n$ increases, the training sets that are more likely to be chosen accumulate on a certain region of space.
Similarly to typical sequences, we may define typical training sets given the connection between training sets and types.
We recall the definition of the \emph{strong typical} set~\citep[Section 10.6]{Cover2006}:
\begin{align}
    \cT_\varepsilon^n = \big \{ &s \in \cZ^n : |T_s(z) - \bP_Z(z) | \leq \varepsilon \textnormal{ for every } z \in \cZ  \textnormal{ such that } \bP_Z(z) > 0 \big \} \nonumber \\
    & \cap \big \{ s \in \cZ^n : N(z | s) = 0  \leq \varepsilon \textnormal{ for every } z \in \cZ \textnormal{ such that } \bP_Z(z) = 0 \big \}
     \label{eq:prop4_typical_set}
\end{align}
where $T_{s}(z)$ and $N(z|s)$ are defined in \Cref{def:type}.
For simplicity, we denote the strong typical set as $\cT$ in the sequel.
We also assume that $\bP_Z(z)>0$ for all $z \in \cZ$; otherwise, we may eliminate the elements with zero probability and reduce the cardinality of $\cZ$.

Directly from the definition in~\eqref{eq:prop4_typical_set}, we have that, for all $s \in \cT$ and all $z \in \cZ$,
\begin{equation*}
 \big| N(z|s) - n \bP_Z(z) \big| \leq n \varepsilon,
\end{equation*}
and by Hoeffding's inequality~\eqref{eq:hoeffding} and the union bound,
\begin{equation*}
  \bP_{S}[\cZ^n \setminus \cT] \leq 2|\cZ| e^{-2n\varepsilon^2}.
\end{equation*}
If we choose $\varepsilon=\sqrt{\frac{\log n}{n}}$, we get that the typical vector of counts  $N_s$ is found inside a $|\cZ|$-dimensional hypercube\footnote{Here we note that the covering is performed in a $|\cZ|$-dimensional space, unlike the previous proofs where we covered $(|\cZ|-1)$-dimensional spaces. Although this leads to a looser bound, the final expression is more manageable.}
 of side
\begin{equation}
 l_{\cT} \leq 2\sqrt{n \log n},
 \label{eq:prop4_l_T}
\end{equation}
with a probability
\begin{equation}
 \bP_S[\cT] \geq 1 - \frac{2|\cZ|}{n^2}.
 \label{eq:prop4_prob_s}
\end{equation}

We may now devise a covering of the set $\cT$ by splitting each dimension in $t$ parts. Thus, the side of each small hypercube has length
\begin{equation*}
 l \coloneqq \frac{l_{\cT}}{t} \leq \frac{2\sqrt{n \log n}}{t},
\end{equation*}
where $t\leq2\sqrt{n \log n}$ or, otherwise, we have less than one type per hypercube.
Following a similar analysis as in the previous proofs, we have that, for every dimension $i \in \big[|\cZ|\big]$, the distance between the center and any other atom in the small hypercube is bounded as follows,
\begin{align*}
|N_s(z_i) - N_{s'}(z_i)| \leq \frac{\lfloor l \rfloor + 1}{2} \leq \frac{\sqrt{n \log n}}{t} + \frac{1}{2} \leq \frac{2\sqrt{n \log n}}{t},
\end{align*}
where the last inequality is due to the bound on $t$.
Therefore, the distance $d(s,s_i)$, as described in \Cref{def:dist}, between any training set $s$ belonging to the $i$-th hypercube and its center $s_i$ is bounded as
\begin{equation}
 d(s,s_i) \leq \frac{\sqrt{n \log n}}{t} |\cZ|.
 \label{eq:prop4_d_max}
\end{equation}

We now proceed to bound the desired mutual information using the \emph{golden formula} from~\Cref{prop:properties_minf} and the law of total expectation,
\begin{align}
    \minf(W;S) \leq \bE_{s \sim\bP_S^\cT} &\big[ \relent (\bP_W^{S=s} \Vert \bQ) \big] \cdot \bP_S[\cT] \nonumber\\
 & + \bE_{s \sim\bP_S^{\cS \setminus \cT}} \big[ \relent (\bP_W^{S=s} \Vert \bQ)  \big] \cdot \bP_S[\cS \setminus \cT],
\label{eq:prop4_total_exp}
\end{align}
which is a re-writing of~\eqref{eq:prop_mi_mGDP_3_minf_law_total_expectation} in a notation that better serves the purposes of this proof, and 
where the second term on the right-hand side of~\eqref{eq:prop4_total_exp} may be bounded from above using~\eqref{eq:prop4_prob_s} and~\eqref{eq:lemma_ub_1} from \Cref{lemma:mixture_kl_ub}.
More precisely, for all $s \notin \cT$,
\begin{align}
    \relent (\bP_W^{S=s} \Vert \bQ) & \stack{a}{\leq} - \log \left( \sum\nolimits_{s'} \omega_{s'} \exp \big( {-} \relent\big(\bP_W^{S=s} \Vert \bP_W^{S=s'} \big) \big) \right) \nonumber \\
	&\stack{b}{\leq} \sum\nolimits_{s'} \omega_{s'} \relent\big(\bP_W^{S=s} \Vert \bP_W^{S=s'} \big) \nonumber \\ 
	&\leq \max_{s'} \big\lbrace \relent\big(\bP_W^{S=s} \Vert \bP_W^{S=s'} \big) \big\rbrace,
	\label{eq:worst_case_div}
\end{align}
where $(a)$ is due to~\eqref{eq:lemma_ub_1} and considering the mixture $\bQ$ as the weighted sum of $\bP_W^{S=s'}$, where each $s'$ is the training set at the center of the covering hypercubes; and $(b)$ stems from Jensen's inequality.
Next, we may leverage \Cref{claim:diff_priv_distance} in order to find the upper bounds for $\varepsilon$-DP and $\mu$-GDP algorithms.

\subsection{\texorpdfstring{$\varepsilon$}{e}-DP Algorithms}

If we assume the worst case in bound~\eqref{eq:dp_distance} from \Cref{claim:diff_priv_distance}, we have that, for every $s \notin \cT$, \Cref{eq:worst_case_div} leads to
\begin{equation*}
	\relent (\bP_W^{S=s} \Vert \bQ) \leq \varepsilon n.
	\label{eq:worst_case_div_dp}
\end{equation*}
Therefore, since~\eqref{eq:worst_case_div} does not depend on $S$, and taking into account~\eqref{eq:prop4_prob_s}, we have that
\begin{equation}
	\bE_{s \sim\bP_S^{\cS \setminus \cT}} \big[ \relent (\bP_W^{S=s} \Vert \bQ) \big] \cdot \bP_S[\cS \setminus \cT] \leq  \frac{2 |\cZ| \varepsilon}{n}.
	\label{eq:second_term_rhs_prop4_dp}
\end{equation}

The first term on the right-hand side of~\eqref{eq:prop4_total_exp} may be analyzed similarly as in the previous proofs but considering only the covering of $\cT$; this reduced covering determines the maximum distance~\eqref{eq:prop4_d_max} inside each small hypercube.
In other words,
\begin{align}
\bE_{s \sim\bP_S^\cT} \big[ \relent (\bP_W^{S=s} \Vert \bQ) \big] \cdot \bP_S \big[ \cT \big] \leq \frac{\sqrt{n \log n}}{t} |\cZ| \varepsilon +|\cZ| \log t,
 \label{eq:prop4_partial_exp}
\end{align}
where we simply used that $\bP_S \big[ \cT \big] \leq 1$. The value of $t$ that minimizes this expression is 
\begin{equation}
t = \sqrt{n \log n} \varepsilon,
\label{eq:prop4_opt_t_dp}
\end{equation}
where we need to verify that $1 \leq t \leq 2 \sqrt{n \log n}$. If this condition holds, we may replace~\eqref{eq:prop4_opt_t_dp} in~\eqref{eq:prop4_partial_exp} and, jointly with~\eqref{eq:prop4_total_exp} and~\eqref{eq:second_term_rhs_prop4_dp}, obtain
\begin{equation}
    \minf(W;S) \leq |\cZ| \log(e \varepsilon \sqrt{n \log n}) + 2|\cZ| \frac{\varepsilon}{n}.
\label{eq:prop4_mut_inf_1_dp}
\end{equation}
Given the constraint on $t$, the range of $\varepsilon$ for which this bound is valid is 
\begin{equation*}
\frac{1}{\sqrt{n \log n}} \leq \varepsilon \leq 2.
\end{equation*}

If $\varepsilon < \nicefrac{1}{\sqrt{n \log n}}$, which corresponds to an optimal $t$ such that $t < 1$ according to~\eqref{eq:prop4_opt_t_dp}, we choose $t=1$ as it was noted in the proof of \Cref{prop:mi_mGDP}. In this case, the mutual information is bounded as follows,
\begin{equation}
	\minf(W;S) \leq |\cZ| \varepsilon \sqrt{n \log n} + 2|\cZ| \frac{\varepsilon}{n}.
	\label{eq:prop4_mut_inf_2_dp}
\end{equation}
We then proceed to combine~\eqref{eq:prop4_mut_inf_1_dp} and~\eqref{eq:prop4_mut_inf_2_dp} into a single, more compact bound. Using a similar argument as in the proof of \Cref{prop:mi_mGDP}, we obtain~\eqref{eq:prop_eDP_3_main_1}.

Finally, if $\varepsilon > 2$, which corresponds to $t > 2 \sqrt{n \log n}$, we fix $t = 2 \sqrt{n \log n} + 1$. In this case, we are including each training set in $\mathcal{T}$ into the mixture. Instead of replacing this value of $t$ in~\eqref{eq:prop4_partial_exp}, we note that a uniform covering of $\cT$ results in
\begin{align}
 \bE_{s \sim\bP_S^\cT} \big[ \relent (\bP_W^{S=s} \Vert \bQ) \big] \cdot \bP_S[\cT] &\leq \sum_{s \in \cT}\bP_S[s] \cdot \min_{s' \in \cT} \left \lbrace \relent \big( \bP_W^{S=s} \Vert \bP_{W}^{T_{s'}}\big) - \log \omega_{s'} \right \rbrace \nonumber \\
 &=\bP_S[\cT] \cdot \log |\cT| \nonumber\\
 &\leq |\cZ| \log \big( 1+ 2\sqrt{n\log n} \big),
 \label{eq:prop4_large_epsilon}
\end{align}
where the first inequality is due to \Cref{lemma:mixture_kl_ub}, and the last inequality follows from the fact that $\cT$ is contained by a $\cZ$-dimensional hypercube of side $2\sqrt{n\log n}$ as stated in~\eqref{eq:prop4_l_T}. Combining~\eqref{eq:prop4_total_exp}, \eqref{eq:second_term_rhs_prop4_dp}, and~\eqref{eq:prop4_large_epsilon}, we obtain~\eqref{eq:prop_eDP_3_main_2}.

\subsection{\texorpdfstring{$\mu$}{m}-GDP Algorithms}

If we assume now the worst case in bound~\eqref{eq:gdp_distance} from \Cref{claim:diff_priv_distance} and taking into account~\eqref{eq:worst_case_div}, we have that
\begin{equation*}
	\relent (\bP_W^{S=s} \Vert \bQ) \leq \frac{1}{2} \mu^2 n^2,
\end{equation*}
which in combination with~\eqref{eq:prop4_prob_s} means that
\begin{equation*}
	\bE_{s \sim\bP_S^{S \setminus \cT}} \big[ \relent (\bP_W^{S=s} \Vert \bQ) \big]\cdot \bP_S[\cS \setminus \cT] \leq |\cZ| \mu^2.
\end{equation*}

Taking into account \Cref{claim:diff_priv_distance} and that we are considering only the covering of $\cT$, the first term on the right-hand side of~\eqref{eq:prop4_total_exp} is now 
\begin{align*}
\bE_{s \sim\bP_S^\cT} \big[ \relent (\bP_W^{S=s} \Vert \bQ) \big] \cdot \bP_S [ \cT ] \leq \frac{1}{2} \frac{n \log n}{t^2} |\cZ|^2 \mu^2 +|\cZ| \log t,
\end{align*}
which is minimized when $t = \sqrt{|\cZ| n \log n} \mu$.
If we operate analogously as for $\varepsilon$-DP algorithms, we obtain~\eqref{eq:prop_mGDP_3_main_1} when $\mu \leq \nicefrac{2}{\sqrt{|\cZ|}}$, and~\eqref{eq:prop_mGDP_3_main_2} otherwise.
This concludes the proof of \Cref{prop:mi_mGDP_3}.
    
\end{subappendices}

%% file: chapters/discussion.tex
In this manuscript, we studied the information-theoretic framework for generalization. This framework draws from information theory to define measures of dependence between the algorithm and the training set, which are then used to characterize the algorithm's generalization. We delved deep into the framework analyzing existing results, coming up with new ones, uncovering some of its limitations, and unveiling some of its connections with the field of privacy.

In~\Cref{ch:expected_generalization_error}, we surveyed, analyzed, and compared most of the relevant literature on information-theoretic expected generalization guarantees. Here, we provided a comprehensive, structured, and systematic method to generate new guarantees considering different dependencies of the algorithm's output hypothesis: the whole training set, individual data instances, or random subsets of the training set. This method only requires the existence of a lemma decoupling two random variables of the form of the Donsker and Varadhan~\Cref{lemma:dv_and_gvp}, the Kantorovich--Rubinstein duality~\Cref{lemma:kantorovich_rubinstein_duality}, or the $f$-divergences decoupling~\Cref{lemma:variational_representation_f_divergences}. Therefore, it is agnostic to the measure of dependence considered. %

Furthermore, in this chapter, we collected and distilled some of the recent work employing the information-theoretic relative-entropy generalization error bounds to characterize the generalization capabilities of noisy, iterative algorithms. Also, we discuss how these techniques have been recently used to characterize non-noisy algorithms, such as stochastic gradient descent, by comparing them to a noisy version. Moreover, we also showed that these bounds are tight, in the sense that there always exist non-trivial problems whose generalization can be precisely characterized by them. Later, we uncovered some of the limitations of these bounds. To be precise, we noted that there is a convex, Lipschitz function and a bounded domain where virtually every relative entropy-based information-theoretic guarantee is vacuous for gradient descent. However, the framework of uniform stability ensures that gradient descent always generalizes in that setting. %

By the individualized nature of the information-theoretic framework to generalization, a complete characterization of the situations under which it is advantageous is impossible. Still, many other algorithms and situations may benefit from this kind of study outside of noisy iterative algorithms, such as the compression schemes studied in~\citep{steinke2020reasoning,haghifam2021towards}. For example, there has been little study of the generalization guarantees of non-parameterized algorithms, such as decision trees or random forests, where adapting these approaches could be beneficial.
Moreover, while there have been subsequent works to ours showing how, for every algorithm, virtually every relative entropy-based information-theoretic guarantee is vacuous~\citep{livni2023information,attias2024information}, this does not include some of the newer guarantees outlined in~\Cref{sec:futher_advances_bounds_using_information_measures}. Furthermore, these limitations only extend to situations where the loss is convex, Lipschitz, and the hypothesis space is bounded. Also, the limiting situations described both by us and in~\citep{livni2023information,attias2024information} are situations where the instance space dimension grows with the number of samples. Hence, a future research topic would be to have a more extensive characterization of the limitations of these guarantees. Namely, studying if they extend to other situations outside of convex, Lipschitz functions with bounded hypothesis space and studying if they extend to situations where the dimension of the instance space is fixed and independent of the number of samples.

To deal with some of the limitations of the relative entropy-based guarantees, we introduced and analyzed new bounds based on the Wasserstein distance. The Wasserstein distance (i) avoids the setback of the relative entropy (and other $f$-divergences and the Rényi divergence) of depending on the Radon--Nikodym derivative between the algorithm's output hypothesis distribution with respect to some reference, and (ii) takes into account the geometry of the space to compare these two distributions. %
Nonetheless, an exciting future avenue is to specialize these guarantees to specific situations to have concrete bounds that give us insight into the generalization of certain algorithms. For example, they could be used to study the generalization of noisy, iterative algorithms like we did with the relative entropy-based bounds.

Later, in~\Cref{ch:pac_bayesian_generalization}, we studied the literature and the sub-field of PAC-Bayesian generalization. First, we studied the conventional PAC-Bayesian bounds, and later we noted that every result of this kind could be extended to single-draw PAC-Bayesian bounds.%
For PAC-Bayesian bounds, we separated our study into bounded and unbounded losses:
\begin{itemize}
    \item For losses with a bounded range, we developed an equivalent expression for the Seeger--Langford~\citep{seeger2002pac,langford2001bounds} bound. This bound is also equivalent to a variant of the result from~\citet{catoni2007pac}. The main attractions of our new bound are (i) that it exhibits a linear dependence with the empirical risk of the algorithm and the dependence measure divided by the number of samples (and thus it has a fast rate of $\nicefrac{1}{n}$); and (ii) that it has an explicit form of the distribution of the optimal algorithm for that expression, that is, the algorithm that minimizes simultaneously the empirical risk and the generalization error. %

    A follow-up of this work would be to find situations where we can exploit this closed-form expression to derive algorithms with better generalization guarantees. 
    
    \item For losses with a potentially unbounded range, we developed bounds for losses with both light and heavy tails. For losses with light tails, that is, that have a bounded cumulant generating function, we obtained a PAC-Bayes analogue to the Chernoff inequality. Then, for losses with heavy tails, either with only a bounded $p$-th moment or a bounded variance, we also derived new bounds. An interesting property of the bounds for losses with a bounded $p$-th moment is that their relationship with the number of samples $n$ interpolates between a fast rate $\nicefrac{1}{n}$, when all the moments are bounded, and a slow rate $\nicefrac{1}{\sqrt{n}}$, when only the second moment is bounded. Moreover, to develop these results, we craft a method to optimize parameters in probabilistic statements. %

    A natural follow-up of this line of work is twofold. On one end, one may want to characterize particular situations where the loss has a known tail behavior with the bounds we developed. On the other end, one may wish to study different situations where a parameter needs to be optimized in probabilistic statements and seamlessly apply our technique to uncover new results.
\end{itemize}

As we discussed for the guarantees in expectation, the relative entropy, even though it is a convenient and tractable metric, has several shortcomings~\citep{livni2020limitation,bassily2018learners,haghifam2023limitations}. Therefore, a recent and exciting line of work is starting to substitute this measure of dependence in the PAC-Bayesian guarantees with other $f$-divergences~\citep{esposito2021generalization,ohnishi2021novel,kuzborskij2024better,hellstrom2020generalization}, Rényi divergences~\citep{esposito2021generalization,begin2016pac}, or integral probability metrics like the Wasserstein distance~\citep{amit2022integral,haddouche2023wasserstein,viallard2024tighter,viallard2024learning}. 

Another exciting line of research is to characterize the generalization error of interactive algorithms such as online and reinforcement learning. Here, the data instances are not independent of each other, and the results must be anytime valid. There has already been some research on the topic~\citep{haddouche2023pacbayes,seldin2012pac,chugg2023unified,flynn2023pac}. However, the field is still young and there are multiple interesting threads to pull.

Finally, in~\Cref{ch:privacy_and_generalization}, we studied the connection between privacy and generalization. More precisely, we considered the connection between generalization and the privacy frameworks of maximal leakage and differential privacy. First, we collected the folklore knowledge that algorithms with a bounded maximal leakage generalize. 
Then, we developed generalization error bounds for differentially private algorithms that, for a fixed privacy parameter $\varepsilon$, vanish as the number of samples increases. However, these guarantees are restricted to the case of permutation invariant algorithms that operate on discrete data.
A clear future avenue is to investigate further under which conditions such a feat is possible. This can be either (i) obtaining new bounds that decrease with the number of samples for a fixed privacy parameter in general, or in particular settings; or (ii) describing a set of situations under which this kind of bounds is impossible.

%% file: chapters/acknowledgements.tex
We thank Borja's PhD thesis opponent, \emph{Prof. Benjamin Guedj}; the grading committee members, \emph{Prof. Po-Ling Loh}, \emph{Prof. Giuseppe Durisi}, and \emph{Prof. Alexandre Proutiere}; and the advance reviewer, \emph{Prof. Magnus Jansson}, for their precious feedback.

We also extend our gratitude to the co-authors of some of the papers we discuss in this monograph based on Borja's PhD, \emph{Dr. Mahdi Haghifam}, \emph{Prof. Daniel Roy}, \emph{Dr. Gintare Karolina Dziugaite}, and \emph{Prof. Omar Rivasplata}. Similarly, we want to thank \emph{Prof. Gergely Neu}, \emph{Prof. Pierre Alquier}, \emph{Prof. Maria Pérez-Ortiz}, and \emph{Prof. Mario Díaz} for giving us valuable feedback on some of the papers we used to write this tutorial.

Finally, Borja Rodríguez-Gálvez and Mikael Skoglund were funded, in part, by the Swedish Research Council under contract 2019-03606. Ragnar Thobaben was also funded, in part, by the Swedish Research Council under contract 2021-05266.

%% file: main.bbl
\begin{thebibliography}{258}
\providecommand{\natexlab}[1]{#1}
\providecommand{\url}[1]{#1}
\csname url@samestyle\endcsname
\providecommand{\newblock}{\relax}
\providecommand{\bibinfo}[2]{#2}
\providecommand{\BIBentrySTDinterwordspacing}{\spaceskip=0pt\relax}
\providecommand{\BIBentryALTinterwordstretchfactor}{4}
\providecommand{\BIBentryALTinterwordspacing}{\spaceskip=\fontdimen2\font plus
\BIBentryALTinterwordstretchfactor\fontdimen3\font minus
  \fontdimen4\font\relax}
\providecommand{\BIBforeignlanguage}[2]{{%
\expandafter\ifx\csname l@#1\endcsname\relax
\typeout{** WARNING: IEEEtranN.bst: No hyphenation pattern has been}%
\typeout{** loaded for the language `#1'. Using the pattern for}%
\typeout{** the default language instead.}%
\else
\language=\csname l@#1\endcsname
\fi
#2}}
\providecommand{\BIBdecl}{\relax}
\BIBdecl

\bibitem[Jiang et~al.(2017)Jiang, Jiang, Zhi, Dong, Li, Ma, Wang, Dong, Shen,
  and Wang]{jiang2017artificial}
F.~Jiang, Y.~Jiang, H.~Zhi, Y.~Dong, H.~Li, S.~Ma, Y.~Wang, Q.~Dong, H.~Shen,
  and Y.~Wang, ``Artificial intelligence in healthcare: past, present and
  future,'' \emph{Stroke and vascular neurology}, vol.~2, no.~4, 2017.

\bibitem[Yurtsever et~al.(2020)Yurtsever, Lambert, Carballo, and
  Takeda]{yurtsever2020survey}
E.~Yurtsever, J.~Lambert, A.~Carballo, and K.~Takeda, ``A survey of autonomous
  driving: Common practices and emerging technologies,'' \emph{IEEE access},
  vol.~8, pp. 58\,443--58\,469, 2020.

\bibitem[Badue et~al.(2021)Badue, Guidolini, Carneiro, Azevedo, Cardoso,
  Forechi, Jesus, Berriel, Paixao, Mutz, et~al.]{badue2021self}
C.~Badue, R.~Guidolini, R.~V. Carneiro, P.~Azevedo, V.~B. Cardoso, A.~Forechi,
  L.~Jesus, R.~Berriel, T.~M. Paixao, F.~Mutz \emph{et~al.}, ``Self-driving
  cars: A survey,'' \emph{Expert Systems with Applications}, vol. 165, p.
  113816, 2021.

\bibitem[Bolton and Hand(2002)]{bolton2002statistical}
R.~J. Bolton and D.~J. Hand, ``Statistical fraud detection: A review,''
  \emph{Statistical science}, vol.~17, no.~3, pp. 235--255, 2002.

\bibitem[Xu and Raginsky(2017)]{xu2017information}
A.~Xu and M.~Raginsky, ``Information-theoretic analysis of generalization
  capability of learning algorithms,'' \emph{Advances in Neural Information
  Processing Systems (NeurIPS)}, vol.~30, 2017.

\bibitem[Russo and Zou(2016)]{russo2016controlling}
D.~Russo and J.~Zou, ``Controlling bias in adaptive data analysis using
  information theory,'' in \emph{Artificial Intelligence and Statistics}.\hskip
  1em plus 0.5em minus 0.4em\relax PMLR, 2016, pp. 1232--1240.

\bibitem[Bu et~al.(2020{\natexlab{a}})Bu, Zou, and
  Veeravalli]{bu2020tightening}
Y.~Bu, S.~Zou, and V.~V. Veeravalli, ``Tightening mutual information-based
  bounds on generalization error,'' \emph{IEEE Journal on Selected Areas in
  Information Theory}, vol.~1, no.~1, pp. 121--130, 2020.

\bibitem[Negrea et~al.(2019)Negrea, Haghifam, Dziugaite, Khisti, and
  Roy]{negrea2019information}
J.~Negrea, M.~Haghifam, G.~K. Dziugaite, A.~Khisti, and D.~M. Roy,
  ``Information-theoretic generalization bounds for {SGLD} via data-dependent
  estimates,'' \emph{Advances in Neural Information Processing Systems
  (NeurIPS)}, vol.~32, 2019.

\bibitem[Steinke and Zakynthinou(2020)]{steinke2020reasoning}
T.~Steinke and L.~Zakynthinou, ``Reasoning about generalization via conditional
  mutual information,'' in \emph{Conference on Learning Theory (COLT)}.\hskip
  1em plus 0.5em minus 0.4em\relax PMLR, 2020, pp. 3437--3452.

\bibitem[Esposito et~al.(2021)Esposito, Gastpar, and
  Issa]{esposito2021generalization}
A.~R. Esposito, M.~Gastpar, and I.~Issa, ``Generalization error bounds via
  r{\'e}nyi-, $f$-divergences and maximal leakage,'' \emph{IEEE Transactions on
  Information Theory}, vol.~67, no.~8, pp. 4986--5004, 2021.

\bibitem[Ohnishi and Honorio(2021)]{ohnishi2021novel}
Y.~Ohnishi and J.~Honorio, ``Novel change of measure inequalities with
  applications to {PAC}-{B}ayesian bounds and {M}onte {C}arlo estimation,'' in
  \emph{International Conference on Artificial Intelligence and Statistics
  (AISTATS)}.\hskip 1em plus 0.5em minus 0.4em\relax PMLR, 2021, pp.
  1711--1719.

\bibitem[Hellstr{\"o}m and Durisi(2020)]{hellstrom2020generalization}
F.~Hellstr{\"o}m and G.~Durisi, ``Generalization bounds via information density
  and conditional information density,'' \emph{IEEE Journal on Selected Areas
  in Information Theory}, vol.~1, no.~3, pp. 824--839, 2020.

\bibitem[Shawe-Taylor et~al.(1996)Shawe-Taylor, Bartlett, Williamson, and
  Anthony]{shawe1996framework}
J.~Shawe-Taylor, P.~L. Bartlett, R.~C. Williamson, and M.~Anthony, ``A
  framework for structural risk minimisation,'' in \emph{Conference on
  Computational Learning Theory (COLT)}, 1996, pp. 68--76.

\bibitem[Rissanen(1978)]{rissanen1978modeling}
J.~Rissanen, ``Modeling by shortest data description,'' \emph{Automatica},
  vol.~14, no.~5, pp. 465--471, 1978.

\bibitem[Barron(1991)]{barron1991complexity}
A.~R. Barron, ``Complexity regularization with application to artificial neural
  networks,'' in \emph{Nonparametric functional estimation and related
  topics}.\hskip 1em plus 0.5em minus 0.4em\relax Springer, 1991, pp. 561--576.

\bibitem[Barron and Cover(1991)]{barron1991minimum}
A.~R. Barron and T.~M. Cover, ``Minimum complexity density estimation,''
  \emph{IEEE transactions on information theory}, vol.~37, no.~4, pp.
  1034--1054, 1991.

\bibitem[Kearns et~al.(1995)Kearns, Mansour, Ng, and
  Ron]{kearns1995experimental}
M.~Kearns, Y.~Mansour, A.~Y. Ng, and D.~Ron, ``An experimental and theoretical
  comparison of model selection methods,'' in \emph{Conference on Computational
  Learning Theory (COLT)}, 1995, pp. 21--30.

\bibitem[McAllester(1998)]{mcallester1998some}
D.~A. McAllester, ``Some {PAC}-{B}ayesian theorems,'' in \emph{Conference on
  Computational Learning Theory (COLT)}, 1998, pp. 230--234.

\bibitem[McAllester(1999)]{mcallester1999pac}
------, ``{PAC}-{B}ayesian model averaging,'' in \emph{Conference on
  Computational Learning Theory (COLT)}, 1999, pp. 164--170.

\bibitem[McAllester(2003)]{mcallester2003pac}
------, ``{PAC}-{B}ayesian stochastic model selection,'' \emph{Machine
  Learning}, vol.~51, no.~1, pp. 5--21, 2003.

\bibitem[Seeger(2002)]{seeger2002pac}
M.~Seeger, ``{PAC}-{B}ayesian generalisation error bounds for {G}aussian
  process classification,'' \emph{Journal of machine learning research},
  vol.~3, no. Oct, pp. 233--269, 2002.

\bibitem[Langford and Seeger(2001)]{langford2001bounds}
J.~Langford and M.~Seeger, ``Bounds for averaging classifiers,'' School of
  Computer Science, Carnegie Mellon University, Tech. Rep., 2001.

\bibitem[Catoni(2003)]{catoni2003pac}
O.~Catoni, ``A {PAC}-{B}ayesian approach to adaptive classification,'' CNRS –
  Laboratoire de Probabilités et Modèles Aléatoires, Uni- versité Paris,
  Tech. Rep., 2003.

\bibitem[Catoni(2007)]{catoni2007pac}
------, ``{PAC}-{B}ayesian supervised classification: the thermodynamics of
  statistical learning,'' \emph{arXiv preprint arXiv:0712.0248}, 2007.

\bibitem[Tolstikhin and Seldin(2013)]{tolstikhin2013pac}
I.~O. Tolstikhin and Y.~Seldin, ``{PAC}-{B}ayes-empirical-{B}ernstein
  inequality,'' \emph{Advances in Neural Information Processing Systems
  (NeurIPS)}, vol.~26, 2013.

\bibitem[Thiemann et~al.(2017)Thiemann, Igel, Wintenberger, and
  Seldin]{thiemann2017strongly}
N.~Thiemann, C.~Igel, O.~Wintenberger, and Y.~Seldin, ``A strongly quasiconvex
  {PAC}-{B}ayesian bound,'' in \emph{International Conference on Algorithmic
  Learning Theory (ALT)}.\hskip 1em plus 0.5em minus 0.4em\relax PMLR, 2017,
  pp. 466--492.

\bibitem[Rivasplata et~al.(2019)Rivasplata, Tankasali, and
  Szepesv{\'a}ri]{rivasplata2019pac}
O.~Rivasplata, V.~M. Tankasali, and C.~Szepesv{\'a}ri, ``{PAC}-{B}ayes with
  backprop,'' \emph{arXiv preprint arXiv:1908.07380}, 2019.

\bibitem[Catoni(2004)]{catoni2004statistical}
O.~Catoni, \emph{Statistical learning theory and stochastic optimization:
  {\'E}cole d'Et{\'e} de Probabilit{\'e}s de Saint-Flour, Summer School
  XXXI-2001}.\hskip 1em plus 0.5em minus 0.4em\relax Springer Science \&
  Business Media, 2004, vol. 1851.

\bibitem[Guedj and Pujol(2021)]{guedj2021still}
B.~Guedj and L.~Pujol, ``Still no free lunches: the price to pay for tighter
  {PAC}-{B}ayes bounds,'' \emph{Entropy}, vol.~23, no.~11, p. 1529, 2021.

\bibitem[Alquier(2006)]{alquier2006transductive}
P.~Alquier, ``Transductive and inductive adaptative inference for regression
  and density estimation,'' \emph{University Paris 6}, 2006.

\bibitem[Alquier and Guedj(2018)]{alquier2018simpler}
P.~Alquier and B.~Guedj, ``Simpler {PAC}-{B}ayesian bounds for hostile data,''
  \emph{Machine Learning}, vol. 107, no.~5, pp. 887--902, 2018.

\bibitem[Holland(2019)]{holland2019pac}
M.~Holland, ``{PAC}-{B}ayes under potentially heavy tails,'' \emph{Advances in
  Neural Information Processing Systems (NeurIPS)}, vol.~32, 2019.

\bibitem[Kuzborskij and Szepesv{\'a}ri(2019)]{kuzborskij2019efron}
I.~Kuzborskij and C.~Szepesv{\'a}ri, ``Efron-{S}tein {PAC}-{B}ayesian
  inequalities,'' \emph{arXiv preprint arXiv:1909.01931}, 2019.

\bibitem[Haddouche et~al.(2021)Haddouche, Guedj, Rivasplata, and
  Shawe-Taylor]{haddouche2021pac}
M.~Haddouche, B.~Guedj, O.~Rivasplata, and J.~Shawe-Taylor, ``{PAC}-{B}ayes
  unleashed: {G}eneralisation bounds with unbounded losses,'' \emph{Entropy},
  vol.~23, no.~10, p. 1330, 2021.

\bibitem[Haddouche and Guedj(2023{\natexlab{a}})]{haddouche2023pacbayes}
M.~Haddouche and B.~Guedj, ``{PAC}-{B}ayes generalisation bounds for
  heavy-tailed losses through supermartingales,'' \emph{Transactions on Machine
  Learning Research}, 2023.

\bibitem[Wang et~al.(2015)Wang, Shen, Miao, Chen, and Xu]{wang2015pac}
Z.~Wang, L.~Shen, Y.~Miao, S.~Chen, and W.~Xu, ``{PAC}-{B}ayesian inequalities
  of some random variables sequences,'' \emph{Journal of Inequalities and
  Applications}, vol. 2015, no.~1, pp. 1--8, 2015.

\bibitem[Goldwasser and Micali(1984)]{Goldwasser1984ProbabilisticE}
S.~Goldwasser and S.~Micali, ``Probabilistic encryption,'' \emph{J. Comput.
  Syst. Sci.}, vol.~28, pp. 270--299, 1984.

\bibitem[Dwork and Naor(2010)]{dwork2010difficulties}
C.~Dwork and M.~Naor, ``On the difficulties of disclosure prevention in
  statistical databases or the case for differential privacy,'' \emph{Journal
  of Privacy and Confidentiality}, vol.~2, no.~1, 2010.

\bibitem[Smith(2009)]{smith2009foundations}
G.~Smith, ``On the foundations of quantitative information flow,'' in
  \emph{International Conference on Foundations of Software Science and
  Computational Structures}.\hskip 1em plus 0.5em minus 0.4em\relax Springer,
  2009, pp. 288--302.

\bibitem[Issa et~al.(2020)Issa, Wagner, and Kamath]{issa2020operational}
I.~Issa, A.~B. Wagner, and S.~Kamath, ``An operational approach to information
  leakage,'' \emph{IEEE Transactions on Information Theory}, vol.~66, no.~3,
  pp. 1625--1657, 2020.

\bibitem[Mário et~al.(2012)Mário, Chatzikokolakis, Palamidessi, and
  Smith]{m2012measuring}
S.~A. Mário, K.~Chatzikokolakis, C.~Palamidessi, and G.~Smith, ``Measuring
  information leakage using generalized gain functions,'' in \emph{2012 IEEE
  25th Computer Security Foundations Symposium}.\hskip 1em plus 0.5em minus
  0.4em\relax IEEE, 2012, pp. 265--279.

\bibitem[Alvim et~al.(2014)Alvim, Chatzikokolakis, McIver, Morgan, Palamidessi,
  and Smith]{alvim2014additive}
M.~S. Alvim, K.~Chatzikokolakis, A.~McIver, C.~Morgan, C.~Palamidessi, and
  G.~Smith, ``Additive and multiplicative notions of leakage, and their
  capacities,'' in \emph{2014 IEEE 27th Computer Security Foundations
  Symposium}.\hskip 1em plus 0.5em minus 0.4em\relax IEEE, 2014, pp. 308--322.

\bibitem[Dwork et~al.(2006)Dwork, McSherry, Nissim, and
  Smith]{dwork2006calibrating}
C.~Dwork, F.~McSherry, K.~Nissim, and A.~Smith, ``Calibrating noise to
  sensitivity in private data analysis,'' in \emph{Theory of Cryptography:
  Third Theory of Cryptography Conference, TCC 2006, New York, NY, USA, March
  4-7, 2006. Proceedings 3}.\hskip 1em plus 0.5em minus 0.4em\relax Springer,
  2006, pp. 265--284.

\bibitem[Dwork et~al.(2014)Dwork, Roth, et~al.]{dwork2014algorithmic}
C.~Dwork, A.~Roth \emph{et~al.}, ``The algorithmic foundations of differential
  privacy,'' \emph{Foundations and Trends{\textregistered} in Theoretical
  Computer Science}, vol.~9, no. 3--4, pp. 211--407, 2014.

\bibitem[Dwork et~al.(2015{\natexlab{a}})Dwork, Feldman, Hardt, Pitassi,
  Reingold, and Roth]{dwork2015preserving}
C.~Dwork, V.~Feldman, M.~Hardt, T.~Pitassi, O.~Reingold, and A.~L. Roth,
  ``Preserving statistical validity in adaptive data analysis,'' in \emph{ACM
  Symposium on Theory of Computing (STOC)}, 2015, pp. 117--126.

\bibitem[Dwork et~al.(2015{\natexlab{b}})Dwork, Feldman, Hardt, Pitassi,
  Reingold, and Roth]{dwork2015generalization}
C.~Dwork, V.~Feldman, M.~Hardt, T.~Pitassi, O.~Reingold, and A.~Roth,
  ``Generalization in adaptive data analysis and holdout reuse,''
  \emph{Advances in Neural Information Processing Systems (NeurIPS)}, vol.~28,
  2015.

\bibitem[Wang et~al.(2016)Wang, Lei, and Fienberg]{JMLR:v17:15-313}
Y.-X. Wang, J.~Lei, and S.~E. Fienberg, ``Learning with differential privacy:
  Stability, learnability and the sufficiency and necessity of {ERM}
  principle,'' \emph{Journal of Machine Learning Research}, vol.~17, no. 183,
  pp. 1--40, 2016.

\bibitem[Jung et~al.(2021)Jung, Ligett, Neel, Roth, Sharifi-Malvajerdi, and
  Shenfeld]{jung2021new}
C.~Jung, K.~Ligett, S.~Neel, A.~Roth, S.~Sharifi-Malvajerdi, and M.~Shenfeld,
  ``A new analysis of differential privacy’s generalization guarantees,'' in
  \emph{ACM Symposium on Theory of Computing (STOC)}, 2021, pp. 9--9.

\bibitem[Gray(2009)]{gray2009probability}
R.~M. Gray, \emph{Probability, random processes, and ergodic properties},
  2nd~ed.\hskip 1em plus 0.5em minus 0.4em\relax Springer, 2009.

\bibitem[Kallenberg(2002)]{kallenberg1997foundations}
O.~Kallenberg, \emph{Foundations of modern probability}, 2nd~ed.\hskip 1em plus
  0.5em minus 0.4em\relax Springer, 2002.

\bibitem[McDonald and Weiss(2013)]{mcdonald1999course}
J.~N. McDonald and N.~A. Weiss, \emph{A course in real analysis}, 2nd~ed.\hskip
  1em plus 0.5em minus 0.4em\relax Elsevier, 2013.

\bibitem[Norman L.~Johnson(1994)]{johnson1994continuous}
N.~B. Norman L.~Johnson, Samuel~Kotz, \emph{Continuous Univariate
  Distributions}.\hskip 1em plus 0.5em minus 0.4em\relax John Wiley \& Sons
  Inc., 1994, vol.~1.

\bibitem[Asmussen et~al.(2016)Asmussen, Jensen, and
  Rojas-Nandayapa]{asmussen2016laplace}
S.~Asmussen, J.~L. Jensen, and L.~Rojas-Nandayapa, ``On the {L}aplace transform
  of the lognormal distribution,'' \emph{Methodology and Computing in Applied
  Probability}, vol.~18, pp. 441--458, 2016.

\bibitem[Banerjee and Mont{\'u}far(2021)]{banerjee2021information}
P.~K. Banerjee and G.~Mont{\'u}far, ``Information complexity and generalization
  bounds,'' in \emph{2021 IEEE International Symposium on Information Theory
  (ISIT)}.\hskip 1em plus 0.5em minus 0.4em\relax IEEE, 2021, pp. 676--681.

\bibitem[Zhang(2006)]{zhang2006information}
T.~Zhang, ``Information-theoretic upper and lower bounds for statistical
  estimation,'' \emph{IEEE Transactions on Information Theory}, vol.~52, no.~4,
  pp. 1307--1321, 2006.

\bibitem[Cover and Thomas(2006)]{Cover2006}
T.~M. Cover and J.~A. Thomas, \emph{Elements of Information Theory},
  2nd~ed.\hskip 1em plus 0.5em minus 0.4em\relax Wiley-Interscience, 2006.

\bibitem[Polyanskiy and Wu(2023)]{polianskyi2022}
Y.~Polyanskiy and Y.~Wu, \emph{Information Theory: From Coding to Learning},
  1st~ed.\hskip 1em plus 0.5em minus 0.4em\relax Cambridge University Press,
  2023.

\bibitem[MacKay(2003)]{mackay2003information}
D.~J. MacKay, \emph{Information theory, inference and learning
  algorithms}.\hskip 1em plus 0.5em minus 0.4em\relax Cambridge university
  press, 2003.

\bibitem[Dwork and Rothblum(2016)]{dwork2016concentrated}
C.~Dwork and G.~N. Rothblum, ``Concentrated differential privacy,'' \emph{arXiv
  preprint arXiv:1603.01887}, 2016.

\bibitem[Mironov(2017)]{mironov2017renyi}
I.~Mironov, ``R{\'e}nyi differential privacy,'' in \emph{2017 IEEE 30th
  computer security foundations symposium (CSF)}.\hskip 1em plus 0.5em minus
  0.4em\relax IEEE, 2017, pp. 263--275.

\bibitem[Bun and Steinke(2016)]{bun2016concentrated}
M.~Bun and T.~Steinke, ``Concentrated differential privacy: Simplifications,
  extensions, and lower bounds,'' in \emph{Theory of Cryptography
  Conference}.\hskip 1em plus 0.5em minus 0.4em\relax Springer, 2016, pp.
  635--658.

\bibitem[Balle et~al.(2020)Balle, Barthe, Gaboardi, Hsu, and
  Sato]{balle2020hypothesis}
B.~Balle, G.~Barthe, M.~Gaboardi, J.~Hsu, and T.~Sato, ``Hypothesis testing
  interpretations and {R}{\'e}nyi differential privacy,'' in
  \emph{International Conference on Artificial Intelligence and Statistics
  (AISTATS)}.\hskip 1em plus 0.5em minus 0.4em\relax PMLR, 2020, pp.
  2496--2506.

\bibitem[Saeidian et~al.(2023{\natexlab{a}})Saeidian, Cervia, Oechtering, and
  Skoglund]{saeidian2023apointwise}
S.~Saeidian, G.~Cervia, T.~J. Oechtering, and M.~Skoglund, ``Pointwise maximal
  leakage,'' \emph{IEEE Transactions on Information Theory}, 2023.

\bibitem[Saeidian et~al.(2023{\natexlab{b}})Saeidian, Cervia, Oechtering, and
  Skoglund]{saeidian2023bpointwise}
------, ``Pointwise maximal leakage on general alphabets,'' in \emph{2023 IEEE
  International Symposium on Information Theory (ISIT)}, 2023, pp. 388--393.

\bibitem[Shannon(1948)]{shannon1948mathematical}
C.~E. Shannon, ``A mathematical theory of communication,'' \emph{The Bell
  system technical journal}, vol.~27, no.~3, pp. 379--423, 1948.

\bibitem[Donsker and Varadhan(1975)]{donsker1975asymptotic}
M.~D. Donsker and S.~S. Varadhan, ``Asymptotic evaluation of certain {M}arkov
  process expectations for large time, {I},'' \emph{Communications on Pure and
  Applied Mathematics}, vol.~28, no.~1, pp. 1--47, 1975.

\bibitem[Gibbs(1902)]{gibbs1902elementary}
J.~W. Gibbs, \emph{Elementary principles in statistical mechanics: developed
  with especial reference to the rational foundations of thermodynamics}.\hskip
  1em plus 0.5em minus 0.4em\relax C. Scribner's sons, 1902.

\bibitem[Gouverneur et~al.(2022)Gouverneur, Rodr{\'\i}guez-G{\'a}lvez,
  Oechtering, and Skoglund]{gouverneur2022information}
A.~Gouverneur, B.~Rodr{\'\i}guez-G{\'a}lvez, T.~J. Oechtering, and M.~Skoglund,
  ``An information-theoretic analysis of {B}ayesian reinforcement learning,''
  \emph{Allerton}, 2022.

\bibitem[Gouverneur et~al.(2023)Gouverneur, Rodr{\'\i}guez-G{\'a}lvez,
  Oechtering, and Skoglund]{gouverneur2023thompson}
------, ``Thompson sampling regret bounds for contextual bandits with
  sub-{G}aussian rewards,'' in \emph{IEEE International Symposium on
  Information Theory (ISIT)}.\hskip 1em plus 0.5em minus 0.4em\relax IEEE,
  2023.

\bibitem[Gouverneur et~al.(2024)Gouverneur, Rodr{\'\i}guez-G{\'a}lvez,
  Oechtering, and Skoglund]{gouverneur2024chained}
------, ``Chained information-theoretic bounds and tight regret rate for linear
  bandit problems,'' \emph{arXiv preprint arXiv:2403.03361}, 2024,
  \emph{{S}ubmitted to the Conference on Learning Theory (COLT).}

\bibitem[Wainwright(2019)]{wainwright2019high}
M.~J. Wainwright, \emph{High-dimensional statistics: A non-asymptotic
  viewpoint}, ser. Cambridge series in tatistical and probabilistic
  mathematics.\hskip 1em plus 0.5em minus 0.4em\relax Cambridge university
  press, 2019, vol.~48.

\bibitem[Pinsker(1964)]{pinsker1964information}
M.~S. Pinsker, ``Information and information stability of random variables and
  processes,'' \emph{Holden-Day}, 1964.

\bibitem[Bretagnolle and Huber(1978)]{bretagnolle1978estimation}
J.~Bretagnolle and C.~Huber, ``Estimation des densit{\'e}s: risque minimax,''
  \emph{S{\'e}minaire de probabilit{\'e}s de Strasbourg}, vol.~12, pp.
  342--363, 1978.

\bibitem[Kullback(1967)]{kullback1967lower}
S.~Kullback, ``A lower bound for discrimination information in terms of
  variation (corresp.),'' \emph{IEEE transactions on Information Theory},
  vol.~13, no.~1, pp. 126--127, 1967.

\bibitem[Csiszár(1967)]{csiszar1967information}
I.~Csiszár, ``Information-type measures of difference of probability
  distributions and indirect observations,'' \emph{Studia Sci. Math.
  Hungarica}, vol.~2, 1967.

\bibitem[Kemperman(1969)]{kemperman1969optimum}
J.~Kemperman, ``On the optimum rate of transmitting information,'' \emph{The
  Annals of Mathematical Statistics}, pp. 2156--2177, 1969.

\bibitem[Harremo{\"e}s and Vajda(2011)]{harremoes2011pairs}
P.~Harremo{\"e}s and I.~Vajda, ``On pairs of $ f $-divergences and their joint
  range,'' \emph{IEEE Transactions on Information Theory}, vol.~57, no.~6, pp.
  3230--3235, 2011.

\bibitem[Van~Erven and Harremos(2014)]{van2014renyi}
T.~Van~Erven and P.~Harremos, ``R{\'e}nyi divergence and {K}ullback--{L}eibler
  divergence,'' \emph{IEEE Transactions on Information Theory}, vol.~60, no.~7,
  pp. 3797--3820, 2014.

\bibitem[Gr{\"u}nwald(2007)]{grunwald2007minimum}
P.~D. Gr{\"u}nwald, \emph{The minimum description length principle}.\hskip 1em
  plus 0.5em minus 0.4em\relax MIT press, 2007.

\bibitem[Boyd and Vandenberghe(2004)]{boyd2004convex}
S.~P. Boyd and L.~Vandenberghe, \emph{Convex optimization}.\hskip 1em plus
  0.5em minus 0.4em\relax Cambridge university press, 2004.

\bibitem[Hiriart-Urruty and Lemar{\'e}chal(2004)]{hiriart2004fundamentals}
J.-B. Hiriart-Urruty and C.~Lemar{\'e}chal, \emph{Fundamentals of convex
  analysis}.\hskip 1em plus 0.5em minus 0.4em\relax Springer Science \&
  Business Media, 2004.

\bibitem[Boucheron et~al.(2003)Boucheron, Lugosi, and
  Bousquet]{boucheron2003concentration}
S.~Boucheron, G.~Lugosi, and O.~Bousquet, ``Concentration inequalities,'' in
  \emph{Summer school on machine learning}.\hskip 1em plus 0.5em minus
  0.4em\relax Springer, 2003, pp. 208--240.

\bibitem[Villani(2009)]{villani2009optimal}
C.~Villani, \emph{Optimal transport: old and new}.\hskip 1em plus 0.5em minus
  0.4em\relax Springer, 2009.

\bibitem[van Handel(2014)]{van2014probability}
R.~van Handel, ``Probability in high dimension,'' Princeton University, NJ,
  Tech. Rep., 2014.

\bibitem[Valiant(1984)]{valiant1984theory}
L.~G. Valiant, ``A theory of the learnable,'' \emph{Communications of the ACM},
  vol.~27, no.~11, pp. 1134--1142, 1984.

\bibitem[Shalev-Shwartz and Ben-David(2014)]{shalev2014understanding}
S.~Shalev-Shwartz and S.~Ben-David, \emph{Understanding machine learning: From
  theory to algorithms}.\hskip 1em plus 0.5em minus 0.4em\relax Cambridge
  university press, 2014.

\bibitem[Dudley et~al.(1991)Dudley, Gin{\'e}, and Zinn]{dudley1991uniform}
R.~M. Dudley, E.~Gin{\'e}, and J.~Zinn, ``Uniform and universal
  {G}livenko-{C}antelli classes,'' \emph{Journal of Theoretical Probability},
  vol.~4, no.~3, pp. 485--510, 1991.

\bibitem[Shalev-Shwartz et~al.(2010)Shalev-Shwartz, Shamir, Srebro, and
  Sridharan]{shalev2010learnability}
S.~Shalev-Shwartz, O.~Shamir, N.~Srebro, and K.~Sridharan, ``Learnability,
  stability and uniform convergence,'' \emph{Journal of Machine Learning
  Research}, vol.~11, pp. 2635--2670, 2010.

\bibitem[Zhang et~al.(2021{\natexlab{a}})Zhang, Bengio, Hardt, Recht, and
  Vinyals]{zhang2021understanding}
C.~Zhang, S.~Bengio, M.~Hardt, B.~Recht, and O.~Vinyals, ``Understanding deep
  learning (still) requires rethinking generalization,'' \emph{Communications
  of the ACM}, vol.~64, no.~3, pp. 107--115, 2021.

\bibitem[Nagarajan and Kolter(2019)]{nagarajan2019uniform}
V.~Nagarajan and J.~Z. Kolter, ``Uniform convergence may be unable to explain
  generalization in deep learning,'' \emph{Advances in Neural Information
  Processing Systems (NeurIPS)}, vol.~32, 2019.

\bibitem[Negrea et~al.(2020)Negrea, Dziugaite, and Roy]{negrea2020defense}
J.~Negrea, G.~K. Dziugaite, and D.~Roy, ``In defense of uniform convergence:
  Generalization via derandomization with an application to interpolating
  predictors,'' in \emph{International Conference on Machine Learning
  (ICML)}.\hskip 1em plus 0.5em minus 0.4em\relax PMLR, 2020, pp. 7263--7272.

\bibitem[Mohri et~al.(2018)Mohri, Rostamizadeh, and
  Talwalkar]{mohri2018foundations}
M.~Mohri, A.~Rostamizadeh, and A.~Talwalkar, \emph{Foundations of machine
  learning}.\hskip 1em plus 0.5em minus 0.4em\relax MIT press, 2018.

\bibitem[Hoorfar and Hassani(2008)]{hoorfar2008inequalities}
A.~Hoorfar and M.~Hassani, ``Inequalities on the {L}ambert {W} function and
  hyperpower function,'' \emph{J. Inequal. Pure and Appl. Math}, vol.~9, no.~2,
  pp. 5--9, 2008.

\bibitem[Koltchinskii and Panchenko(2000)]{koltchinskii2000rademacher}
V.~Koltchinskii and D.~Panchenko, ``Rademacher processes and bounding the risk
  of function learning,'' in \emph{High dimensional probability II}.\hskip 1em
  plus 0.5em minus 0.4em\relax Springer, 2000, pp. 443--457.

\bibitem[Bartlett and Mendelson(2002)]{bartlett2002rademacher}
P.~L. Bartlett and S.~Mendelson, ``Rademacher and {G}aussian complexities: Risk
  bounds and structural results,'' \emph{Journal of Machine Learning Research},
  vol.~3, no. Nov, pp. 463--482, 2002.

\bibitem[Barlett(1996)]{bartlett1996sample}
P.~Barlett, ``The sample complexity of pattern classification with neural
  networks,'' \emph{the size of the weights is more important than the size of
  the network, Technical report, Australian National University}, 1996.

\bibitem[Neyshabur et~al.(2015)Neyshabur, Tomioka, and
  Srebro]{neyshabur2015norm}
B.~Neyshabur, R.~Tomioka, and N.~Srebro, ``Norm-based capacity control in
  neural networks,'' in \emph{Conference on Learning Theory (COLT)}.\hskip 1em
  plus 0.5em minus 0.4em\relax PMLR, 2015, pp. 1376--1401.

\bibitem[Bartlett et~al.(2017)Bartlett, Foster, and
  Telgarsky]{bartlett2017spectrally}
P.~L. Bartlett, D.~J. Foster, and M.~J. Telgarsky, ``Spectrally-normalized
  margin bounds for neural networks,'' \emph{Advances in Neural Information
  Processing Systems (NeurIPS)}, vol.~30, 2017.

\bibitem[Golowich et~al.(2018)Golowich, Rakhlin, and Shamir]{golowich2018size}
N.~Golowich, A.~Rakhlin, and O.~Shamir, ``Size-independent sample complexity of
  neural networks,'' in \emph{Conference on Learning Theory (COLT)}.\hskip 1em
  plus 0.5em minus 0.4em\relax PMLR, 2018, pp. 297--299.

\bibitem[Liang et~al.(2019)Liang, Poggio, Rakhlin, and Stokes]{liang2019fisher}
T.~Liang, T.~Poggio, A.~Rakhlin, and J.~Stokes, ``Fisher-{R}ao metric,
  geometry, and complexity of neural networks,'' in \emph{International
  Conference on Artificial Intelligence and Statistics (AISTATS)}.\hskip 1em
  plus 0.5em minus 0.4em\relax PMLR, 2019, pp. 888--896.

\bibitem[Vapnik(1999)]{vapnik1999nature}
V.~Vapnik, \emph{The nature of statistical learning theory}.\hskip 1em plus
  0.5em minus 0.4em\relax Springer science \& business media, 1999.

\bibitem[de~Ockham et~al.(1974)de~Ockham, Boehner, G{\'a}l, and
  Brown]{de1974summa}
G.~de~Ockham, P.~Boehner, G.~G{\'a}l, and S.~Brown, \emph{Summa logicae}.\hskip
  1em plus 0.5em minus 0.4em\relax University of Notre Dame Press, 1974.

\bibitem[Ball(2016)]{ball2016tyranny}
\BIBentryALTinterwordspacing
P.~Ball, ``The tyranny of simple explanations,'' August 2016. [Online].
  Available:
  \url{https://www.theatlantic.com/science/archive/2016/08/occams-razor/495332/}
\BIBentrySTDinterwordspacing

\bibitem[Blumer et~al.(1987)Blumer, Ehrenfeucht, Haussler, and
  Warmuth]{blumer1987occam}
A.~Blumer, A.~Ehrenfeucht, D.~Haussler, and M.~K. Warmuth, ``Occam's razor,''
  \emph{Information processing letters}, vol.~24, no.~6, pp. 377--380, 1987.

\bibitem[Shawe-Taylor and Williamson(1997)]{shawe1997pac}
J.~Shawe-Taylor and R.~C. Williamson, ``A {PAC} analysis of a {B}ayesian
  estimator,'' in \emph{Conference on Computational Learning Theory (COLT)},
  1997, pp. 2--9.

\bibitem[Lotfi et~al.(2022)Lotfi, Finzi, Kapoor, Potapczynski, Goldblum, and
  Wilson]{lotfi2022pac}
S.~Lotfi, M.~Finzi, S.~Kapoor, A.~Potapczynski, M.~Goldblum, and A.~G. Wilson,
  ``{PAC}-{B}ayes compression bounds so tight that they can explain
  generalization,'' \emph{Advances in Neural Information Processing Systems
  (NeurIPS)}, vol.~35, pp. 31\,459--31\,473, 2022.

\bibitem[Rogers and Wagner(1978)]{rogers1978finite}
W.~H. Rogers and T.~J. Wagner, ``A finite sample distribution-free performance
  bound for local discrimination rules,'' \emph{The Annals of Statistics}, pp.
  506--514, 1978.

\bibitem[Devroye and Wagner(1979{\natexlab{a}})]{devroye1979adistribution}
L.~Devroye and T.~Wagner, ``Distribution-free inequalities for the deleted and
  holdout error estimates,'' \emph{IEEE Transactions on Information Theory},
  vol.~25, no.~2, pp. 202--207, 1979.

\bibitem[Devroye and Wagner(1979{\natexlab{b}})]{devroye1979bdistribution}
------, ``Distribution-free performance bounds with the resubstitution error
  estimate (corresp.),'' \emph{IEEE Transactions on Information Theory},
  vol.~25, no.~2, pp. 208--210, 1979.

\bibitem[Bousquet and Elisseeff(2002)]{bousquet2002stability}
O.~Bousquet and A.~Elisseeff, ``Stability and generalization,'' \emph{The
  Journal of Machine Learning Research}, vol.~2, pp. 499--526, 2002.

\bibitem[Kutin and Niyogi(2002)]{kutin2002almost}
S.~Kutin and P.~Niyogi, ``Almost-everywhere algorithmic stability and
  generalization error,'' in \emph{Conference on Uncertainty in Artificial
  Intelligence (UAI)}, 2002, pp. 275--282.

\bibitem[Rakhlin et~al.(2005)Rakhlin, Mukherjee, and
  Poggio]{rakhlin2005stability}
A.~Rakhlin, S.~Mukherjee, and T.~Poggio, ``Stability results in learning
  theory,'' \emph{Analysis and Applications}, vol.~3, no.~04, pp. 397--417,
  2005.

\bibitem[Mukherjee et~al.(2006)Mukherjee, Niyogi, Poggio, and
  Rifkin]{mukherjee2006learning}
S.~Mukherjee, P.~Niyogi, T.~Poggio, and R.~Rifkin, ``Learning theory: stability
  is sufficient for generalization and necessary and sufficient for consistency
  of empirical risk minimization,'' \emph{Advances in Computational
  Mathematics}, vol.~25, pp. 161--193, 2006.

\bibitem[Feldman and Vondrak(2018)]{feldman2018generalization}
V.~Feldman and J.~Vondrak, ``Generalization bounds for uniformly stable
  algorithms,'' \emph{Advances in Neural Information Processing Systems
  (NeurIPS)}, vol.~31, 2018.

\bibitem[Feldman and Vondrak(2019)]{feldman2019high}
------, ``High probability generalization bounds for uniformly stable
  algorithms with nearly optimal rate,'' in \emph{Conference on Learning Theory
  (COLT)}.\hskip 1em plus 0.5em minus 0.4em\relax PMLR, 2019, pp. 1270--1279.

\bibitem[Bousquet et~al.(2020)Bousquet, Klochkov, and
  Zhivotovskiy]{bousquet2020sharper}
O.~Bousquet, Y.~Klochkov, and N.~Zhivotovskiy, ``Sharper bounds for uniformly
  stable algorithms,'' in \emph{Conference on Learning Theory (COLT)}.\hskip
  1em plus 0.5em minus 0.4em\relax PMLR, 2020, pp. 610--626.

\bibitem[Klochkov and Zhivotovskiy(2021)]{klochkov2021stability}
Y.~Klochkov and N.~Zhivotovskiy, ``Stability and deviation optimal risk bounds
  with convergence rate $ o (1/n) $,'' \emph{Advances in Neural Information
  Processing Systems (NeurIPS)}, vol.~34, pp. 5065--5076, 2021.

\bibitem[Hardt et~al.(2016)Hardt, Recht, and Singer]{hardt2016train}
M.~Hardt, B.~Recht, and Y.~Singer, ``Train faster, generalize better: Stability
  of stochastic gradient descent,'' in \emph{International Conference on
  Machine Learning (ICML)}.\hskip 1em plus 0.5em minus 0.4em\relax PMLR, 2016,
  pp. 1225--1234.

\bibitem[Kuzborskij and Lampert(2018)]{kuzborskij2018data}
I.~Kuzborskij and C.~Lampert, ``Data-dependent stability of stochastic gradient
  descent,'' in \emph{International Conference on Machine Learning
  (ICML)}.\hskip 1em plus 0.5em minus 0.4em\relax PMLR, 2018, pp. 2815--2824.

\bibitem[Lei and Ying(2020)]{lei2020fine}
Y.~Lei and Y.~Ying, ``Fine-grained analysis of stability and generalization for
  stochastic gradient descent,'' in \emph{International Conference on Machine
  Learning (ICML)}.\hskip 1em plus 0.5em minus 0.4em\relax PMLR, 2020, pp.
  5809--5819.

\bibitem[Bassily et~al.(2020)Bassily, Feldman, Guzm{\'a}n, and
  Talwar]{bassily2020stability}
R.~Bassily, V.~Feldman, C.~Guzm{\'a}n, and K.~Talwar, ``Stability of stochastic
  gradient descent on nonsmooth convex losses,'' \emph{Advances in Neural
  Information Processing Systems (NeurIPS)}, vol.~33, pp. 4381--4391, 2020.

\bibitem[Charles and Papailiopoulos(2018)]{charles2018stability}
Z.~Charles and D.~Papailiopoulos, ``Stability and generalization of learning
  algorithms that converge to global optima,'' in \emph{International
  Conference on Machine Learning (ICML)}.\hskip 1em plus 0.5em minus
  0.4em\relax PMLR, 2018, pp. 745--754.

\bibitem[Lei(2023)]{lei2023stability}
Y.~Lei, ``Stability and generalization of stochastic optimization with
  nonconvex and nonsmooth problems,'' in \emph{The Thirty Sixth Annual
  Conference on Learning Theory (COLT)}.\hskip 1em plus 0.5em minus 0.4em\relax
  PMLR, 2023, pp. 191--227.

\bibitem[Raginsky et~al.(2016)Raginsky, Rakhlin, Tsao, Wu, and
  Xu]{raginsky2016information}
M.~Raginsky, A.~Rakhlin, M.~Tsao, Y.~Wu, and A.~Xu, ``Information-theoretic
  analysis of stability and bias of learning algorithms,'' in \emph{IEEE
  Information Theory Workshop (ITW)}.\hskip 1em plus 0.5em minus 0.4em\relax
  IEEE, 2016, pp. 26--30.

\bibitem[Oneto et~al.(2017)Oneto, Ridella, and Anguita]{oneto2017differential}
L.~Oneto, S.~Ridella, and D.~Anguita, ``Differential privacy and
  generalization: Sharper bounds with applications,'' \emph{Pattern Recognition
  Letters}, vol.~89, pp. 31--38, 2017.

\bibitem[Rodr{\'\i}guez-G{\'a}lvez
  et~al.(2021{\natexlab{a}})Rodr{\'\i}guez-G{\'a}lvez, Bassi, and
  Skoglund]{rodriguez2021upper}
B.~Rodr{\'\i}guez-G{\'a}lvez, G.~Bassi, and M.~Skoglund, ``Upper bounds on the
  generalization error of private algorithms,'' \emph{IEEE Transactions on
  Information Theory}, 2021.

\bibitem[Kulynych et~al.(2022)Kulynych, Yang, Yu, B{\l}asiok, and
  Nakkiran]{kulynych2022you}
B.~Kulynych, Y.-Y. Yang, Y.~Yu, J.~B{\l}asiok, and P.~Nakkiran, ``What you see
  is what you get: Principled deep learning via distributional
  generalization,'' \emph{Advances in Neural Information Processing Systems
  (NeurIPS)}, vol.~35, pp. 2168--2183, 2022.

\bibitem[Bun et~al.(2023)Bun, Gaboardi, Hopkins, Impagliazzo, Lei, Pitassi,
  Sivakumar, and Sorrell]{bun2023stability}
M.~Bun, M.~Gaboardi, M.~Hopkins, R.~Impagliazzo, R.~Lei, T.~Pitassi,
  S.~Sivakumar, and J.~Sorrell, ``Stability is stable: Connections between
  replicability, privacy, and adaptive generalization,'' in \emph{ACM Symposium
  on Theory of Computing (STOC)}, 2023, pp. 520--527.

\bibitem[Bassily et~al.(2016)Bassily, Nissim, Smith, Steinke, Stemmer, and
  Ullman]{bassily2016algorithmic}
R.~Bassily, K.~Nissim, A.~Smith, T.~Steinke, U.~Stemmer, and J.~Ullman,
  ``Algorithmic stability for adaptive data analysis,'' in \emph{ACM symposium
  on Theory of Computing (STOC)}, 2016, pp. 1046--1059.

\bibitem[Alabdulmohsin(2015)]{alabdulmohsin2015algorithmic}
I.~M. Alabdulmohsin, ``Algorithmic stability and uniform generalization,''
  \emph{Advances in Neural Information Processing Systems (NeurIPS)}, vol.~28,
  2015.

\bibitem[Alabdulmohsin(2017)]{alabdulmohsin2017information}
I.~Alabdulmohsin, ``An information-theoretic route from generalization in
  expectation to generalization in probability,'' in \emph{Artificial
  intelligence and statistics}.\hskip 1em plus 0.5em minus 0.4em\relax PMLR,
  2017, pp. 92--100.

\bibitem[Alquier(2021)]{alquier2021user}
P.~Alquier, ``User-friendly introduction to {PAC}-{B}ayes bounds,'' \emph{arXiv
  preprint arXiv:2110.11216}, 2021.

\bibitem[Hellstr{\"o}m et~al.(2023)Hellstr{\"o}m, Durisi, Guedj, and
  Raginsky]{hellstrom2023generalization}
F.~Hellstr{\"o}m, G.~Durisi, B.~Guedj, and M.~Raginsky, ``Generalization
  bounds: perspectives from information theory and {PAC}-{B}ayes,'' \emph{arXiv
  preprint arXiv:2309.04381}, 2023.

\bibitem[Bassily et~al.(2018)Bassily, Moran, Nachum, Shafer, and
  Yehudayoff]{bassily2018learners}
R.~Bassily, S.~Moran, I.~Nachum, J.~Shafer, and A.~Yehudayoff, ``Learners that
  use little information,'' in \emph{International Conference on Algorithmic
  Learning Theory (ALT)}.\hskip 1em plus 0.5em minus 0.4em\relax PMLR, 2018,
  pp. 25--55.

\bibitem[Livni and Moran(2020)]{livni2020limitation}
R.~Livni and S.~Moran, ``A limitation of the {PAC}-{B}ayes framework,''
  \emph{Advances in Neural Information Processing Systems (NeurIPS)}, vol.~33,
  pp. 20\,543--20\,553, 2020.

\bibitem[Haghifam et~al.(2023)Haghifam, Rodr{\'\i}guez-G{\'a}lvez, Thobaben,
  Skoglund, Roy, and Dziugaite]{haghifam2023limitations}
M.~Haghifam, B.~Rodr{\'\i}guez-G{\'a}lvez, R.~Thobaben, M.~Skoglund, D.~M. Roy,
  and G.~K. Dziugaite, ``Limitations of information-theoretic generalization
  bounds for gradient descent methods in stochastic convex optimization,'' in
  \emph{International Conference on Algorithmic Learning Theory (ALT)}.\hskip
  1em plus 0.5em minus 0.4em\relax PMLR, 2023, pp. 663--706.

\bibitem[Livni(2023)]{livni2023information}
R.~Livni, ``Information theoretic lower bounds for information theoretic upper
  bounds,'' in \emph{Advances in Neural Information Processing Systems
  (NeurIPS)}, 2023.

\bibitem[Dziugaite and Roy(2017)]{dziugaite2017computing}
G.~K. Dziugaite and D.~M. Roy, ``Computing nonvacuous generalization bounds for
  deep (stochastic) neural networks with many more parameters than training
  data,'' in \emph{Conference on Uncertainty in Artificial Intelligence (UAI)},
  2017.

\bibitem[Dziugaite and Roy(2018)]{dziugaite2018data}
------, ``Data-dependent {PAC}-{B}ayes priors via differential privacy,''
  \emph{Advances in Neural Information Processing Systems (NeurIPS)}, vol.~31,
  2018.

\bibitem[P{\'e}rez-Ortiz et~al.(2021)P{\'e}rez-Ortiz, Rivasplata, Shawe-Taylor,
  and Szepesv{\'a}ri]{perez2021tighter}
M.~P{\'e}rez-Ortiz, O.~Rivasplata, J.~Shawe-Taylor, and C.~Szepesv{\'a}ri,
  ``Tighter risk certificates for neural networks,'' \emph{Journal of Machine
  Learning Research}, vol.~22, no.~1, pp. 10\,326--10\,365, 2021.

\bibitem[Zhou et~al.(2019)Zhou, Veitch, Austern, Adams, and
  Orbanz]{zhou2019non}
W.~Zhou, V.~Veitch, M.~Austern, R.~P. Adams, and P.~Orbanz, ``Non-vacuous
  generalization bounds at the {I}magenet scale: A {PAC}-{B}ayesian compression
  approach,'' in \emph{International Conference on Learning Representations
  (ICLR)}, 2019.

\bibitem[Rodríguez-Gálvez et~al.(2024)Rodríguez-Gálvez, Thobaben, and
  Skoglund]{rodriguez2023morepac}
B.~Rodríguez-Gálvez, B.~Thobaben, and M.~Skoglund, ``More {PAC}-{B}ayes
  bounds: From bounded losses, to losses with general tail behaviors, to
  anytime validity,'' \emph{Journal of Machine Learning Research}, vol.~25,
  2024.

\bibitem[Grunwald et~al.(2021)Grunwald, Steinke, and
  Zakynthinou]{grunwald2021pac}
P.~Grunwald, T.~Steinke, and L.~Zakynthinou, ``{PAC}-{B}ayes, mac-{B}ayes and
  conditional mutual information: Fast rate bounds that handle general {VC}
  classes,'' in \emph{Conference on Learning Theory (COLT)}.\hskip 1em plus
  0.5em minus 0.4em\relax PMLR, 2021, pp. 2217--2247.

\bibitem[Hellstr{\"o}m and Durisi(2022)]{hellstrom2022new}
F.~Hellstr{\"o}m and G.~Durisi, ``A new family of generalization bounds using
  samplewise evaluated {CMI},'' in \emph{Advances in Neural Information
  Processing Systems (NeurIPS)}, 2022.

\bibitem[Neu et~al.(2021)Neu, Dziugaite, Haghifam, and Roy]{neu2021information}
G.~Neu, G.~K. Dziugaite, M.~Haghifam, and D.~M. Roy, ``Information-theoretic
  generalization bounds for stochastic gradient descent,'' in \emph{Conference
  on Learning Theory (COLT)}.\hskip 1em plus 0.5em minus 0.4em\relax PMLR,
  2021, pp. 3526--3545.

\bibitem[Dalalyan and Tsybakov(2008)]{dalalyan2008aggregation}
A.~Dalalyan and A.~B. Tsybakov, ``Aggregation by exponential weighting, sharp
  {PAC}-{B}ayesian bounds and sparsity,'' \emph{Machine Learning}, vol.~72, no.
  1-2, pp. 39--61, 2008.

\bibitem[Salmon and Dalalyan(2011)]{salmon2011optimal}
J.~Salmon and A.~Dalalyan, ``Optimal aggregation of affine estimators,'' in
  \emph{Conference on Learning Theory (COLT)}.\hskip 1em plus 0.5em minus
  0.4em\relax JMLR Workshop and Conference Proceedings, 2011, pp. 635--660.

\bibitem[Dalalyan and Salmon(2012)]{dalalyan2012sharp}
A.~S. Dalalyan and J.~Salmon, ``Sharp oracle inequalities for aggregation of
  affine estimators,'' \emph{Annals of Statistics}, vol.~40, no.~4, pp.
  2327--2355, 2012.

\bibitem[Maurer(2004)]{maurer2004note}
A.~Maurer, ``A note on the {PAC} {B}ayesian theorem,'' \emph{arXiv preprint
  cs/0411099}, 2004.

\bibitem[Germain et~al.(2015)Germain, Lacasse, Laviolette, March, and
  Roy]{germain2015risk}
P.~Germain, A.~Lacasse, F.~Laviolette, M.~March, and J.-F. Roy, ``Risk bounds
  for the majority vote: From a {PAC}-{B}ayesian analysis to a learning
  algorithm,'' \emph{Journal of Machine Learning Research}, vol.~16, no.~26,
  pp. 787--860, 2015.

\bibitem[Alquier and Biau(2013)]{alquier2013sparse}
P.~Alquier and G.~Biau, ``Sparse single-index model,'' \emph{Journal of Machine
  Learning Research}, vol.~14, no.~1, 2013.

\bibitem[Guedj and Alquier(2013)]{guedj2013pac}
B.~Guedj and P.~Alquier, ``{PAC}-{B}ayesian estimation and prediction in sparse
  additive models,'' \emph{Electronic Journal of Statistics}, vol.~7, pp.
  264--291, 2013.

\bibitem[Rivasplata et~al.(2020)Rivasplata, Kuzborskij, Szepesv{\'a}ri, and
  Shawe-Taylor]{rivasplata2020pac}
O.~Rivasplata, I.~Kuzborskij, C.~Szepesv{\'a}ri, and J.~Shawe-Taylor,
  ``{PAC}-{B}ayes analysis beyond the usual bounds,'' \emph{Advances in Neural
  Information Processing Systems (NeurIPS)}, vol.~33, pp. 16\,833--16\,845,
  2020.

\bibitem[Gelfand and Mitter(1991)]{gelfand1991recursive}
S.~B. Gelfand and S.~K. Mitter, ``Recursive stochastic algorithms for global
  optimization in $r^d$,'' \emph{SIAM Journal on Control and Optimization},
  vol.~29, no.~5, pp. 999--1018, 1991.

\bibitem[Welling and Teh(2011)]{welling2011bayesian}
M.~Welling and Y.~W. Teh, ``Bayesian learning via stochastic gradient
  {L}angevin dynamics,'' in \emph{International Conference on Machine Learning
  (ICML)}, 2011, pp. 681--688.

\bibitem[Haghifam et~al.(2020)Haghifam, Negrea, Khisti, Roy, and
  Dziugaite]{haghifam2020sharpened}
M.~Haghifam, J.~Negrea, A.~Khisti, D.~M. Roy, and G.~K. Dziugaite, ``Sharpened
  generalization bounds based on conditional mutual information and an
  application to noisy, iterative algorithms,'' \emph{Advances in Neural
  Information Processing Systems (NeurIPS)}, vol.~33, pp. 9925--9935, 2020.

\bibitem[Rodr{\'\i}guez-G{\'a}lvez et~al.(2020)Rodr{\'\i}guez-G{\'a}lvez,
  Bassi, Thobaben, and Skoglund]{rodriguez2020randomsubset}
B.~Rodr{\'\i}guez-G{\'a}lvez, G.~Bassi, R.~Thobaben, and M.~Skoglund, ``On
  random subset generalization error bounds and the stochastic gradient
  {L}angevin dynamics algorithm,'' in \emph{IEEE Information Theory Workshop
  (ITW)}.\hskip 1em plus 0.5em minus 0.4em\relax IEEE, 2020.

\bibitem[Audibert(2004)]{audibert2004better}
J.-Y. Audibert, ``A better variance control for {PAC}-{B}ayesian
  classification,'' \emph{Preprint}, vol. 905, 2004.

\bibitem[Germain et~al.(2009)Germain, Lacasse, Laviolette, and
  Marchand]{germain2009pac}
P.~Germain, A.~Lacasse, F.~Laviolette, and M.~Marchand, ``{PAC}-{B}ayesian
  learning of linear classifiers,'' in \emph{International Conference on
  Machine Learning (ICML)}, 2009, pp. 353--360.

\bibitem[B{\'e}gin et~al.(2016)B{\'e}gin, Germain, Laviolette, and
  Roy]{begin2016pac}
L.~B{\'e}gin, P.~Germain, F.~Laviolette, and J.-F. Roy, ``{PAC}-{B}ayesian
  bounds based on the r{\'e}nyi divergence,'' in \emph{Artificial Intelligence
  and Statistics}.\hskip 1em plus 0.5em minus 0.4em\relax PMLR, 2016, pp.
  435--444.

\bibitem[Juditsky et~al.(2008)Juditsky, Rigollet, and
  Tsybakov]{juditsky2008learning}
A.~B. Juditsky, P.~Rigollet, and A.~Tsybakov, ``Learning by mirror averaging,''
  \emph{Annals of Statistics}, vol.~36, no.~5, pp. 2183--2206, 2008.

\bibitem[Dudley(2010)]{dudley2010universal}
R.~M. Dudley, ``Universal donsker classes and metric entropy,'' in
  \emph{Selected Works of RM Dudley}.\hskip 1em plus 0.5em minus 0.4em\relax
  Springer, 2010, pp. 345--365.

\bibitem[Pollard(1984)]{pollard1984convergence}
D.~Pollard, \emph{Convergence of stochastic processes}.\hskip 1em plus 0.5em
  minus 0.4em\relax David Pollard, 1984.

\bibitem[Rodr{\'\i}guez-G{\'a}lvez et~al.(2024)Rodr{\'\i}guez-G{\'a}lvez,
  Rivasplata, Thobaben, and Skoglund]{rodriguez2024moments}
B.~Rodr{\'\i}guez-G{\'a}lvez, O.~Rivasplata, R.~Thobaben, and M.~Skoglund, ``A
  note on generalization bounds for losses with bounded moments,'' in
  \emph{IEEE International Symposium on Information Theory (ISIT)}.\hskip 1em
  plus 0.5em minus 0.4em\relax IEEE, 2024.

\bibitem[Rodr{\'\i}guez-G{\'a}lvez
  et~al.(2021{\natexlab{b}})Rodr{\'\i}guez-G{\'a}lvez, Thobaben, Bassi, and
  Skoglund]{rodriguez2021tighter}
B.~Rodr{\'\i}guez-G{\'a}lvez, R.~Thobaben, G.~Bassi, and M.~Skoglund, ``Tighter
  expected generalization error bounds via {W}asserstein distance,'' in
  \emph{Advances in Neural Information Processing Systems (NeurIPS)}, 2021.

\bibitem[Pensia et~al.(2018)Pensia, Jog, and Loh]{pensia2018generalization}
A.~Pensia, V.~Jog, and P.-L. Loh, ``Generalization error bounds for noisy,
  iterative algorithms,'' in \emph{IEEE International Symposium on Information
  Theory (ISIT)}.\hskip 1em plus 0.5em minus 0.4em\relax IEEE, 2018, pp.
  546--550.

\bibitem[Haghifam et~al.(2021)Haghifam, Dziugaite, Moran, and
  Roy]{haghifam2021towards}
M.~Haghifam, G.~K. Dziugaite, S.~Moran, and D.~Roy, ``Towards a unified
  information-theoretic framework for generalization,'' \emph{Advances in
  Neural Information Processing Systems (NeurIPS)}, vol.~34, pp.
  26\,370--26\,381, 2021.

\bibitem[Harutyunyan et~al.(2021)Harutyunyan, Raginsky, Ver~Steeg, and
  Galstyan]{harutyunyan2021information}
H.~Harutyunyan, M.~Raginsky, G.~Ver~Steeg, and A.~Galstyan,
  ``Information-theoretic generalization bounds for black-box learning
  algorithms,'' \emph{Advances in Neural Information Processing Systems
  (NeurIPS)}, vol.~34, pp. 24\,670--24\,682, 2021.

\bibitem[Wang and Mao(2023)]{Wang2023TighterIG}
Z.~Wang and Y.~Mao, ``Tighter information-theoretic generalization bounds from
  supersamples,'' in \emph{International Conference on Machine Learning
  (ICML)}, 2023.

\bibitem[Wang and Mao(2024)]{wang2024sample}
------, ``Sample-conditioned hypothesis stability sharpens
  information-theoretic generalization bounds,'' \emph{Advances in Neural
  Information Processing Systems (NeurIPS)}, vol.~36, 2024.

\bibitem[Asadi et~al.(2018)Asadi, Abbe, and Verd{\'u}]{asadi2018chaining}
A.~Asadi, E.~Abbe, and S.~Verd{\'u}, ``Chaining mutual information and
  tightening generalization bounds,'' \emph{Advances in Neural Information
  Processing Systems (NeurIPS)}, vol.~31, 2018.

\bibitem[Asadi and Abbe(2020)]{asadi2020chaining}
A.~R. Asadi and E.~Abbe, ``Chaining meets chain rule: Multilevel entropic
  regularization and training of neural networks,'' \emph{Journal of Machine
  Learning Research}, vol.~21, no.~1, pp. 5453--5484, 2020.

\bibitem[Chatzigeorgiou(2013)]{chatzigeorgiou2013bounds}
I.~Chatzigeorgiou, ``Bounds on the {L}ambert function and their application to
  the outage analysis of user cooperation,'' \emph{IEEE Communications
  Letters}, vol.~17, no.~8, pp. 1505--1508, 2013.

\bibitem[Vapnik and Chervonenkis(1971)]{vapnik1971uniform}
V.~Vapnik and A.~Y. Chervonenkis, ``On the uniform convergence of relative
  frequencies of events to their probabilities,'' \emph{Theory of Probability
  and its Applications}, vol.~16, no.~2, p. 264, 1971.

\bibitem[Devroye et~al.(1996)Devroye, Gy{\"o}rfi, Lugosi, Devroye, Gy{\"o}rfi,
  and Lugosi]{devroye1996vapnik}
L.~Devroye, L.~Gy{\"o}rfi, G.~Lugosi, L.~Devroye, L.~Gy{\"o}rfi, and G.~Lugosi,
  ``Vapnik-{C}hervonenkis theory,'' \emph{A probabilistic theory of pattern
  recognition}, pp. 187--213, 1996.

\bibitem[Hellstr{\"o}m and Durisi(2021{\natexlab{a}})]{hellstrom2021fast}
F.~Hellstr{\"o}m and G.~Durisi, ``Fast-rate loss bounds via conditional
  information measures with applications to neural networks,'' in \emph{IEEE
  International Symposium on Information Theory (ISIT)}.\hskip 1em plus 0.5em
  minus 0.4em\relax IEEE, 2021, pp. 952--957.

\bibitem[Steinke(2016)]{steinke2016upper}
T.~A. Steinke, ``Upper and lower bounds for privacy and adaptivity in
  algorithmic data analysis,'' Ph.D. dissertation, Harvard, 2016.

\bibitem[Zhou et~al.(2022)Zhou, Tian, and Liu]{zhou2022individually}
R.~Zhou, C.~Tian, and T.~Liu, ``Individually conditional individual mutual
  information bound on generalization error,'' \emph{IEEE Transactions on
  Information Theory}, vol.~68, no.~5, pp. 3304--3316, 2022.

\bibitem[Boucheron et~al.(2005)Boucheron, Bousquet, and
  Lugosi]{boucheron2005theory}
S.~Boucheron, O.~Bousquet, and G.~Lugosi, ``Theory of classification: A survey
  of some recent advances,'' \emph{ESAIM: probability and statistics}, vol.~9,
  pp. 323--375, 2005.

\bibitem[Wang et~al.(2019)Wang, Diaz, Santos~Filho, and
  Calmon]{wang2019information}
H.~Wang, M.~Diaz, J.~C.~S. Santos~Filho, and F.~P. Calmon, ``An
  information-theoretic view of generalization via {W}asserstein distance,'' in
  \emph{IEEE International Symposium on Information Theory (ISIT)}.\hskip 1em
  plus 0.5em minus 0.4em\relax IEEE, 2019, pp. 577--581.

\bibitem[Zhang et~al.(2021{\natexlab{b}})Zhang, Liu, and Tao]{zhang2021optimal}
J.~Zhang, T.~Liu, and D.~Tao, ``An optimal transport analysis on generalization
  in deep learning,'' \emph{IEEE Transactions on Neural Networks and Learning
  Systems}, 2021.

\bibitem[Orabona(2019)]{orabona2019modern}
F.~Orabona, ``A modern introduction to online learning,'' \emph{arXiv preprint
  arXiv:1912.13213}, 2019.

\bibitem[Palomar and Verd{\'u}(2008)]{palomar2008lautum}
D.~P. Palomar and S.~Verd{\'u}, ``Lautum information,'' \emph{IEEE Transactions
  on Information Theory}, vol.~54, no.~3, pp. 964--975, 2008.

\bibitem[Lopez and Jog(2018)]{lopez2018generalization}
A.~T. Lopez and V.~Jog, ``Generalization error bounds using {W}asserstein
  distances,'' in \emph{IEEE Information Theory Workshop (ITW)}.\hskip 1em plus
  0.5em minus 0.4em\relax IEEE, 2018, pp. 1--5.

\bibitem[Gao and Pavel(2018)]{gao2018properties}
B.~Gao and L.~Pavel, ``On the properties of the softmax function with
  application in game theory and reinforcement learning,'' 2018.

\bibitem[Steinwart and Christmann(2008)]{steinwart2008support}
I.~Steinwart and A.~Christmann, \emph{Support vector machines}.\hskip 1em plus
  0.5em minus 0.4em\relax Springer Science \& Business Media, 2008.

\bibitem[Chen et~al.(2014)Chen, Fox, and Guestrin]{chen2014stochastic}
T.~Chen, E.~Fox, and C.~Guestrin, ``Stochastic gradient {H}amiltonian {M}onte
  {C}arlo,'' in \emph{International Conference on Machine Learning
  (ICML)}.\hskip 1em plus 0.5em minus 0.4em\relax PMLR, 2014, pp. 1683--1691.

\bibitem[Li et~al.(2019)Li, Luo, and Qiao]{li2019generalization}
J.~Li, X.~Luo, and M.~Qiao, ``On generalization error bounds of noisy gradient
  methods for non-convex learning,'' in \emph{International Conference on
  Learning Representations (ICLR)}, 2019.

\bibitem[Wang et~al.(2021)Wang, Huang, Gao, and Calmon]{wang2021analyzing}
H.~Wang, Y.~Huang, R.~Gao, and F.~Calmon, ``Analyzing the generalization
  capability of {SGLD} using properties of {G}aussian channels,''
  \emph{Advances in Neural Information Processing Systems (NeurIPS)}, vol.~34,
  pp. 24\,222--24\,234, 2021.

\bibitem[Wang et~al.(2023)Wang, Gao, and Calmon]{wang2023generalization}
H.~Wang, R.~Gao, and F.~P. Calmon, ``Generalization bounds for noisy iterative
  algorithms using properties of additive noise channels,'' \emph{Journal of
  Machine Learning Research}, vol.~24, no.~26, pp. 1--43, 2023.

\bibitem[Futami and Fujisawa(2024)]{futami2024time}
F.~Futami and M.~Fujisawa, ``Time-independent information-theoretic
  generalization bounds for {SGLD},'' \emph{Advances in Neural Information
  Processing Systems (NeurIPS)}, vol.~36, 2024.

\bibitem[Wang and Mao(2021)]{wang2021generalization}
Z.~Wang and Y.~Mao, ``On the generalization of models trained with {SGD}:
  Information-theoretic bounds and implications,'' in \emph{International
  Conference on Learning Representations (ICLR)}, 2021.

\bibitem[Mou et~al.(2018)Mou, Wang, Zhai, and Zheng]{mou2018generalization}
W.~Mou, L.~Wang, X.~Zhai, and K.~Zheng, ``Generalization bounds of {SGLD} for
  non-convex learning: Two theoretical viewpoints,'' in \emph{Conference on
  Learning Theory (COLT)}.\hskip 1em plus 0.5em minus 0.4em\relax PMLR, 2018,
  pp. 605--638.

\bibitem[Hochreiter and Schmidhuber(1997)]{hochreiter1997flat}
S.~Hochreiter and J.~Schmidhuber, ``Flat minima,'' \emph{Neural computation},
  vol.~9, no.~1, pp. 1--42, 1997.

\bibitem[Keskar et~al.(2016)Keskar, Mudigere, Nocedal, Smelyanskiy, and
  Tang]{keskar2016large}
N.~S. Keskar, D.~Mudigere, J.~Nocedal, M.~Smelyanskiy, and P.~T.~P. Tang, ``On
  large-batch training for deep learning: Generalization gap and sharp
  minima,'' in \emph{International Conference on Learning Representations
  (ICLR)}, 2016.

\bibitem[Dinh et~al.(2017)Dinh, Pascanu, Bengio, and Bengio]{dinh2017sharp}
L.~Dinh, R.~Pascanu, S.~Bengio, and Y.~Bengio, ``Sharp minima can generalize
  for deep nets,'' in \emph{International Conference on Machine Learning
  (ICML)}.\hskip 1em plus 0.5em minus 0.4em\relax PMLR, 2017, pp. 1019--1028.

\bibitem[Izmailov et~al.(2018)Izmailov, Wilson, Podoprikhin, Vetrov, and
  Garipov]{izmailov2018averaging}
P.~Izmailov, A.~Wilson, D.~Podoprikhin, D.~Vetrov, and T.~Garipov, ``Averaging
  weights leads to wider optima and better generalization,'' in
  \emph{Conference on Uncertainty in Artificial Intelligence (UAI)}, 2018, pp.
  876--885.

\bibitem[He et~al.(2019)He, Huang, and Yuan]{he2019asymmetric}
H.~He, G.~Huang, and Y.~Yuan, ``Asymmetric valleys: Beyond sharp and flat local
  minima,'' \emph{Advances in Neural Information Processing Systems (NeurIPS)},
  vol.~32, 2019.

\bibitem[Chaudhari et~al.(2019)Chaudhari, Choromanska, Soatto, LeCun, Baldassi,
  Borgs, Chayes, Sagun, and Zecchina]{chaudhari2019entropy}
P.~Chaudhari, A.~Choromanska, S.~Soatto, Y.~LeCun, C.~Baldassi, C.~Borgs,
  J.~Chayes, L.~Sagun, and R.~Zecchina, ``Entropy-{SGD}: Biasing gradient
  descent into wide valleys,'' \emph{Journal of Statistical Mechanics: Theory
  and Experiment}, vol. 2019, no.~12, p. 124018, 2019.

\bibitem[Attias et~al.(2024)Attias, Dziugaite, Haghifam, Livni, and
  Roy]{attias2024information}
I.~Attias, G.~K. Dziugaite, M.~Haghifam, R.~Livni, and D.~M. Roy, ``Information
  complexity of stochastic convex optimization: Applications to generalization
  and memorization,'' \emph{arXiv preprint arXiv:2402.09327}, 2024.

\bibitem[Bun et~al.(2014)Bun, Ullman, and Vadhan]{bun2014fingerprinting}
M.~Bun, J.~Ullman, and S.~Vadhan, ``Fingerprinting codes and the price of
  approximate differential privacy,'' in \emph{ACM symposium on Theory of
  Computing (STOC)}, 2014, pp. 1--10.

\bibitem[Kamath et~al.(2019)Kamath, Li, Singhal, and
  Ullman]{kamath2019privately}
G.~Kamath, J.~Li, V.~Singhal, and J.~Ullman, ``Privately learning
  high-dimensional distributions,'' in \emph{Conference on Learning Theory
  (COLT)}.\hskip 1em plus 0.5em minus 0.4em\relax PMLR, 2019, pp. 1853--1902.

\bibitem[Shalev-Shwartz et~al.(2009)Shalev-Shwartz, Shamir, Srebro, and
  Sridharan]{shalev2009stochastic}
S.~Shalev-Shwartz, O.~Shamir, N.~Srebro, and K.~Sridharan, ``Stochastic convex
  optimization.'' in \emph{COLT}, vol.~2, no.~4, 2009, p.~5.

\bibitem[Cauchy et~al.(1847)]{cauchy1847methode}
A.~Cauchy \emph{et~al.}, ``M{\'e}thode g{\'e}n{\'e}rale pour la r{\'e}solution
  des systemes d’{\'e}quations simultan{\'e}es,'' \emph{Comp. Rend. Sci.
  Paris}, vol.~25, no. 1847, pp. 536--538, 1847.

\bibitem[Bubeck(2015)]{bubeck2015convex}
S.~Bubeck, ``Convex optimization: Algorithms and complexity,''
  \emph{Foundations and Trends{\textregistered} in Machine Learning}, vol.~8,
  no. 3-4, pp. 231--357, 2015.

\bibitem[Orabona(2020)]{lastiterate}
\BIBentryALTinterwordspacing
F.~Orabona. (2020) Last iterate of {SGD} converges (even in unbounded domains).
  [Online]. Available:
  \url{https://parameterfree.com/2020/08/07/last-iterate-of-sgd-converges-even-in-unbounded-domains/#lemmalast_average}
\BIBentrySTDinterwordspacing

\bibitem[Zhang(2004)]{zhang2004solving}
T.~Zhang, ``Solving large scale linear prediction problems using stochastic
  gradient descent algorithms,'' in \emph{International Conference on Machine
  Learning (ICML)}, 2004, p. 116.

\bibitem[Amir et~al.(2021)Amir, Koren, and Livni]{amir2021sgd}
I.~Amir, T.~Koren, and R.~Livni, ``{SGD} generalizes better than {GD} (and
  regularization doesn’t help),'' in \emph{Conference on Learning Theory
  (COLT)}, 2021, pp. 63--92.

\bibitem[Sekhari et~al.(2021)Sekhari, Sridharan, and Kale]{sekhari2021sgd}
A.~Sekhari, K.~Sridharan, and S.~Kale, ``{SGD}: The role of implicit
  regularization, batch-size and multiple-epochs,'' \emph{Advances in Neural
  Information Processing Systems (NeurIPS)}, vol.~34, pp. 27\,422--27\,433,
  2021.

\bibitem[Mitzenmacher and Upfal(2005)]{mitzenmacher2017probability}
M.~Mitzenmacher and E.~Upfal, \emph{Probability and Computing: Randomized
  Algorithms and Probabilistic Analysis}.\hskip 1em plus 0.5em minus
  0.4em\relax Cambridge university press, 2005.

\bibitem[Freund(1998)]{freund1998self}
Y.~Freund, ``Self bounding learning algorithms,'' in \emph{Conference on
  Computational Learning Theory (COLT)}, 1998, pp. 247--258.

\bibitem[Langford and Blum(2003)]{langford2003microchoice}
J.~Langford and A.~Blum, ``Microchoice bounds and self bounding learning
  algorithms,'' \emph{Machine Learning}, vol.~51, pp. 165--179, 2003.

\bibitem[Rivasplata(2022)]{rivasplata2022pac}
O.~Rivasplata, ``{PAC}-{B}ayesian computation,'' Ph.D. dissertation, UCL
  (University College London), 2022.

\bibitem[Langford and Caruana(2001)]{langford2001not}
J.~Langford and R.~Caruana, ``({N}ot) bounding the true error,'' \emph{Advances
  in Neural Information Processing Systems (NeurIPS)}, vol.~14, 2001.

\bibitem[Seldin et~al.(2012)Seldin, Laviolette, Cesa-Bianchi, Shawe-Taylor, and
  Auer]{seldin2012pac}
Y.~Seldin, F.~Laviolette, N.~Cesa-Bianchi, J.~Shawe-Taylor, and P.~Auer,
  ``{PAC}-{B}ayesian inequalities for martingales,'' \emph{IEEE Transactions on
  Information Theory}, vol.~58, no.~12, pp. 7086--7093, 2012.

\bibitem[de~la Peña et~al.(2007)de~la Peña, Lai, and Lai]{de2007pseudo}
V.~H. de~la Peña, M.~J.~K. Lai, and T.~L. Lai, ``Pseudo-maximization and
  self-normalized processes,'' \emph{Probability Surveys}, vol.~4, pp.
  172--192, 2007.

\bibitem[Kakade et~al.(2008)Kakade, Sridharan, and
  Tewari]{kakade2008complexity}
S.~M. Kakade, K.~Sridharan, and A.~Tewari, ``On the complexity of linear
  prediction: Risk bounds, margin bounds, and regularization,'' \emph{Advances
  in Neural Information Processing Systems (NeurIPS)}, vol.~21, 2008.

\bibitem[Chugg et~al.(2023)Chugg, Wang, and Ramdas]{chugg2023unified}
B.~Chugg, H.~Wang, and A.~Ramdas, ``A unified recipe for deriving
  (time-uniform) pac-bayes bounds,'' \emph{Journal of Machine Learning
  Research}, vol.~24, no. 372, pp. 1--61, 2023.

\bibitem[Guedj(2019)]{guedj2019primer}
B.~Guedj, ``A primer on {PAC}-{B}ayesian learning,'' in \emph{Proceedings of
  the French Mathematical Society}, vol.~33.\hskip 1em plus 0.5em minus
  0.4em\relax Soci{\'e}t{\'e} Math{\'e}matique de France, 2019, pp. 391--414.

\bibitem[Lattimore and Szepesv{\'a}ri(2020)]{lattimore2020bandit}
T.~Lattimore and C.~Szepesv{\'a}ri, \emph{Bandit algorithms}.\hskip 1em plus
  0.5em minus 0.4em\relax Cambridge University Press, 2020.

\bibitem[Seldin(2023)]{seldinNotes}
\BIBentryALTinterwordspacing
Y.~Seldin, ``Machine learning. the science of selection under uncertainty,''
  \emph{Lecture Notes}, 2023. [Online]. Available:
  \url{https://sites.google.com/site/yevgenyseldin/teaching?authuser=0}
\BIBentrySTDinterwordspacing

\bibitem[Marton(1996)]{marton1996measure}
K.~Marton, ``A measure concentration inequality for contracting {M}arkov
  chains,'' \emph{Geometric \& Functional Analysis GAFA}, vol.~6, no.~3, pp.
  556--571, 1996.

\bibitem[Kuzborskij et~al.(2024)Kuzborskij, Jun, Wu, Jang, and
  Orabona]{kuzborskij2024better}
I.~Kuzborskij, K.-S. Jun, Y.~Wu, K.~Jang, and F.~Orabona, ``Better-than-{KL}
  {PAC}-{B}ayes bounds,'' \emph{arXiv preprint arXiv:2402.09201}, 2024.

\bibitem[Amit et~al.(2022)Amit, Epstein, Moran, and Meir]{amit2022integral}
R.~Amit, B.~Epstein, S.~Moran, and R.~Meir, ``Integral probability metrics
  {PAC}-{B}ayes bounds,'' \emph{Advances in Neural Information Processing
  Systems (NeurIPS)}, vol.~35, pp. 3123--3136, 2022.

\bibitem[Haddouche and Guedj(2023{\natexlab{b}})]{haddouche2023wasserstein}
M.~Haddouche and B.~Guedj, ``Wasserstein {PAC}-{B}ayes learning: A bridge
  between generalisation and optimisation,'' \emph{arXiv preprint
  arXiv:2304.07048}, 2023.

\bibitem[Viallard et~al.(2024{\natexlab{a}})Viallard, Haddouche, Simsekli, and
  Guedj]{viallard2024learning}
P.~Viallard, M.~Haddouche, U.~Simsekli, and B.~Guedj, ``Learning via
  {W}asserstein-based high probability generalisation bounds,'' \emph{Advances
  in Neural Information Processing Systems (NeurIPS)}, vol.~36, 2024.

\bibitem[Viallard et~al.(2024{\natexlab{b}})Viallard, Haddouche,
  {\c{S}}im{\c{s}}ekli, and Guedj]{viallard2024tighter}
P.~Viallard, M.~Haddouche, U.~{\c{S}}im{\c{s}}ekli, and B.~Guedj, ``Tighter
  generalisation bounds via interpolation,'' \emph{arXiv preprint
  arXiv:2402.05101}, 2024.

\bibitem[Foong et~al.(2021)Foong, Bruinsma, Burt, and Turner]{foong2021tight}
A.~Foong, W.~Bruinsma, D.~Burt, and R.~Turner, ``How tight can {PAC}-{B}ayes be
  in the small data regime?'' \emph{Advances in Neural Information Processing
  Systems (NeurIPS)}, vol.~34, pp. 4093--4105, 2021.

\bibitem[B{\'e}gin et~al.(2014)B{\'e}gin, Germain, Laviolette, and
  Roy]{begin2014pac}
L.~B{\'e}gin, P.~Germain, F.~Laviolette, and J.-F. Roy, ``{PAC}-{B}ayesian
  theory for transductive learning,'' in \emph{Artificial Intelligence and
  Statistics}.\hskip 1em plus 0.5em minus 0.4em\relax PMLR, 2014, pp. 105--113.

\bibitem[Catoni(2015)]{catoni2015pac}
O.~Catoni, ``{PAC}-{B}ayes bounds for supervised classification,''
  \emph{Measures of Complexity: Festschrift for Alexey Chervonenkis}, pp.
  287--302, 2015.

\bibitem[Reeb et~al.(2018)Reeb, Doerr, Gerwinn, and Rakitsch]{reeb2018learning}
D.~Reeb, A.~Doerr, S.~Gerwinn, and B.~Rakitsch, ``Learning {G}aussian processes
  by minimizing {PAC}-{B}ayesian generalization bounds,'' \emph{Advances in
  Neural Information Processing Systems (NeurIPS)}, vol.~31, 2018.

\bibitem[Wu and Seldin(2022)]{wu2022split}
Y.-S. Wu and Y.~Seldin, ``Split-kl and {PAC}-{B}ayes-split-kl inequalities for
  ternary random variables,'' \emph{Advances in Neural Information Processing
  Systems (NeurIPS)}, vol.~35, pp. 11\,369--11\,381, 2022.

\bibitem[Jang et~al.(2023)Jang, Jun, Kuzborskij, and Orabona]{jang2023tighter}
K.~Jang, K.-S. Jun, I.~Kuzborskij, and F.~Orabona, ``Tighter {PAC}-{B}ayes
  bounds through coin-betting,'' \emph{arXiv preprint arXiv:2302.05829}, 2023.

\bibitem[Wu et~al.(2021)Wu, Masegosa, Lorenzen, Igel, and
  Seldin]{wu2021chebyshev}
Y.-S. Wu, A.~Masegosa, S.~Lorenzen, C.~Igel, and Y.~Seldin,
  ``Chebyshev-{C}antelli {PAC}-{B}ayes-{B}ennett inequality for the weighted
  majority vote,'' \emph{Advances in Neural Information Processing Systems
  (NeurIPS)}, vol.~34, pp. 12\,625--12\,636, 2021.

\bibitem[Mhammedi et~al.(2019)Mhammedi, Gr{\"u}nwald, and
  Guedj]{mhammedi2019pac}
Z.~Mhammedi, P.~Gr{\"u}nwald, and B.~Guedj, ``{PAC}-{B}ayes un-expected
  {B}ernstein inequality,'' \emph{Advances in Neural Information Processing
  Systems (NeurIPS)}, vol.~32, 2019.

\bibitem[Hellstr{\"o}m and
  Durisi(2021{\natexlab{b}})]{hellstrom2021corrections}
F.~Hellstr{\"o}m and G.~Durisi, ``Corrections to “{G}eneralization bounds via
  information density and conditional information density”,'' \emph{IEEE
  Journal on Selected Areas in Information Theory}, vol.~2, no.~3, pp.
  1072--1073, 2021.

\bibitem[Mai and Alquier(2015)]{mai2015bayesian}
T.~T. Mai and P.~Alquier, ``A {B}ayesian approach for noisy matrix completion:
  {O}ptimal rate under general sampling distribution,'' \emph{Electronic
  Journal of Statistics}, vol.~9, no.~1, pp. 823--841, 2015.

\bibitem[Kaufmann et~al.(2016)Kaufmann, Capp{\'e}, and
  Garivier]{kaufmann2016complexity}
E.~Kaufmann, O.~Capp{\'e}, and A.~Garivier, ``On the complexity of best arm
  identification in multi-armed bandit models,'' \emph{Journal of Machine
  Learning Research}, vol.~17, pp. 1--42, 2016.

\bibitem[Ville(1939)]{ville1939etude}
J.~Ville, ``{\'E}tude critique de la notion de collectif,'' \emph{Bull. Amer.
  Math. Soc}, vol.~45, no.~11, p. 824, 1939.

\bibitem[Bercu and Touati(2008)]{bercu2008exponential}
B.~Bercu and A.~Touati, ``Exponential inequalities for self-normalized
  martingales with applications,'' \emph{Annals of Applied Probability},
  vol.~18, pp. 1848--1869, 2008.

\bibitem[Braun et~al.(2009)Braun, Chatzikokolakis, and
  Palamidessi]{braun2009quantitative}
C.~Braun, K.~Chatzikokolakis, and C.~Palamidessi, ``Quantitative notions of
  leakage for one-try attacks,'' \emph{Electronic Notes in Theoretical Computer
  Science}, vol. 249, pp. 75--91, 2009.

\bibitem[Alvim et~al.(2020)Alvim, Chatzikokolakis, McIver, Morgan, Palamidessi,
  and Smith]{alvim2020science}
M.~S. Alvim, K.~Chatzikokolakis, A.~McIver, C.~Morgan, C.~Palamidessi, and
  G.~Smith, \emph{The Science of Quantitative Information Flow}.\hskip 1em plus
  0.5em minus 0.4em\relax Springer, 2020.

\bibitem[Meiser(2018)]{meiser2018approximate}
S.~Meiser, ``Approximate and probabilistic differential privacy definitions,''
  \emph{Cryptology ePrint Archive}, 2018.

\bibitem[Wasserman and Zhou(2010)]{wasserman2010statistical}
L.~Wasserman and S.~Zhou, ``A statistical framework for differential privacy,''
  \emph{Journal of the American Statistical Association}, vol. 105, no. 489,
  pp. 375--389, 2010.

\bibitem[Kairouz et~al.(2015)Kairouz, Oh, and
  Viswanath]{kairouz2015composition}
P.~Kairouz, S.~Oh, and P.~Viswanath, ``The composition theorem for differential
  privacy,'' in \emph{International Conference on Machine Learning
  (ICML)}.\hskip 1em plus 0.5em minus 0.4em\relax PMLR, 2015, pp. 1376--1385.

\bibitem[Dong et~al.(2021)Dong, Roth, and Su]{dong2021gaussian}
J.~Dong, A.~Roth, and W.~Su, ``{G}aussian differential privacy,'' \emph{Journal
  of the Royal Statistical Society}, 2021.

\bibitem[Yu et~al.(2014)Yu, Rybar, Uhler, and Fienberg]{yu2014differentially}
F.~Yu, M.~Rybar, C.~Uhler, and S.~E. Fienberg, ``Differentially-private
  logistic regression for detecting multiple-{SNP} association in {GWAS}
  databases,'' in \emph{International Conference on Privacy in Statistical
  Databases}.\hskip 1em plus 0.5em minus 0.4em\relax Springer, 2014, pp.
  170--184.

\bibitem[Eguchi et~al.(2009)Eguchi, Kuruvilla, Ogedegbe, Gerin, Schwartz, and
  Pickering]{eguchi2009optimal}
K.~Eguchi, S.~Kuruvilla, G.~Ogedegbe, W.~Gerin, J.~E. Schwartz, and T.~G.
  Pickering, ``What is the optimal interval between successive home blood
  pressure readings using an automated oscillometric device?'' \emph{Journal of
  hypertension}, vol.~27, no.~6, p. 1172, 2009.

\bibitem[Muntner et~al.(2019)Muntner, Einhorn, Cushman, Whelton, Bello, Drawz,
  Green, Jones, Juraschek, Margolis, et~al.]{muntner2019blood}
P.~Muntner, P.~T. Einhorn, W.~C. Cushman, P.~K. Whelton, N.~A. Bello, P.~E.
  Drawz, B.~B. Green, D.~W. Jones, S.~P. Juraschek, K.~L. Margolis
  \emph{et~al.}, ``Blood pressure assessment in adults in clinical practice and
  clinic-based research: {JACC} scientific expert panel,'' \emph{Journal of the
  American College of Cardiology}, vol.~73, no.~3, pp. 317--335, 2019.

\bibitem[Ma et~al.(2022)Ma, Zhong, Duan, Shen, Qin, and Hu]{ma2022development}
G.~Ma, Z.~Zhong, Y.~Duan, Z.~Shen, N.~Qin, and D.~Hu, ``Development and
  validation of a self-quantification scale for patients with hypertension,''
  \emph{Frontiers in Public Health}, vol.~10, p. 849859, 2022.

\bibitem[Boateng and Abaye(2019)]{boateng2019review}
E.~Y. Boateng and D.~A. Abaye, ``A review of the logistic regression model with
  emphasis on medical research,'' \emph{Journal of data analysis and
  information processing}, vol.~7, no.~4, pp. 190--207, 2019.

\bibitem[{Hershey} and {Olsen}(2007)]{hershey_2007_approx}
J.~R. {Hershey} and P.~A. {Olsen}, ``Approximating the {K}ullback {L}eibler
  divergence between {G}aussian mixture models,'' in \emph{2007 IEEE
  International Conference on Acoustics, Speech, and Signal Processing
  (ICASSP)}, vol.~4, Apr. 2007, pp. 317--320.

\bibitem[{Kaji}(2015)]{kaji_bounds_2015}
Y.~{Kaji}, ``Bounds on the entropy of multinomial distribution,'' in \emph{IEEE
  International Symposium on Information Theory (ISIT)}, Jun. 2015, pp.
  1362--1366.

\bibitem[Bu et~al.(2020{\natexlab{b}})Bu, Dong, Long, and Su]{bu2019deep}
Z.~Bu, J.~Dong, Q.~Long, and W.~J. Su, ``Deep learning with {G}aussian
  differential privacy,'' \emph{Harvard data science review}, vol. 2020,
  no.~23, 2020.

\bibitem[Kingma and Ba(2014)]{kingma2014adam}
D.~P. Kingma and J.~Ba, ``{Adam: A method for stochastic optimization},''
  \emph{arXiv preprint arXiv:1412.6980}, 2014.

\bibitem[Heikkil{\"a} et~al.(2019)Heikkil{\"a}, J{\"a}lk{\"o}, Dikmen, and
  Honkela]{heikkila2019differentially}
M.~Heikkil{\"a}, J.~J{\"a}lk{\"o}, O.~Dikmen, and A.~Honkela, ``Differentially
  private {M}arkov chain {M}onte {C}arlo,'' in \emph{Advances in Neural
  Information Processing Systems (NeurIPS)}, 2019, pp. 4113--4123.

\bibitem[Csisz{\'a}r and K{\"o}rner(2011)]{csiszar_2011_information}
I.~Csisz{\'a}r and J.~K{\"o}rner, \emph{{Information Theory: Coding Theorems
  for Discrete Memoryless Systems}}, 2nd~ed.\hskip 1em plus 0.5em minus
  0.4em\relax Cambridge University Press, 2011.

\bibitem[Flynn et~al.(2023)Flynn, Reeb, Kandemir, and Peters]{flynn2023pac}
H.~Flynn, D.~Reeb, M.~Kandemir, and J.~Peters, ``{PAC}-{B}ayes bounds for
  bandit problems: A survey and experimental comparison,'' \emph{IEEE
  Transactions on Pattern Analysis and Machine Intelligence}, 2023.

\end{thebibliography}
